%% file: rpithes.tex
\documentclass[chap]{thesis}
\usepackage{amsmath}

\usepackage{amssymb}
\usepackage{enumerate}
\usepackage{amsthm}
\usepackage{bm}
\usepackage{dsfont}
\usepackage{subcaption}
\usepackage{graphics,epsfig}
\usepackage{epsf}
\usepackage{url}
\usepackage{algorithm}
\usepackage{algpseudocode}
\usepackage{multirow}
\usepackage{cite}
\usepackage{graphicx}
\usepackage{epstopdf}
\newtheorem{theorem}{Theorem}
\newtheorem{lemma}{Lemma}

\newtheorem{defi}{Definition}

\usepackage{xcolor}
\definecolor{mygray}{gray}{0.95}
\usepackage{bm}
\usepackage{url}
\usepackage{psfrag}
\floatstyle{ruled}
\newfloat{algorithm}{tbp}{loa}
\usepackage{amsmath}
\usepackage{amsfonts}
\usepackage{amsthm}\usepackage{algpseudocode}
\usepackage[most]{tcolorbox}
\newcommand{\W[1]}{{\bfW}^{(#1)}}

\newcommand{\WL}{\bfW^{[\hat{\lambda}]} }  
\newcommand{\Wp}{\bfW^{[p]} }  

\newcommand{\de}{\tilde{\delta}}

\newcommand{\C}{\mathbb{C}}
\newcommand{\R}{\mathbb{R}}
\newcommand{\tcr}{\textcolor{red}}

\newcommand\blfootnote[1]{
  \begingroup
  \renewcommand\thefootnote{}\footnote{#1}%
  \addtocounter{footnote}{-1}%
  \endgroup
}      
\usepackage{titlesec}
\usepackage{titlecaps}

\def\bfa{{\boldsymbol a}}

\def\bfd{{\boldsymbol d}}
\def\bfe{{\boldsymbol e}}

\def\bfh{{\boldsymbol h}}

\def\bfl{{\boldsymbol l}}

\def\bfr{{\boldsymbol r}}
\def\bfs{{\boldsymbol s}}

\def\bfu{{\boldsymbol u}}
\def\bfv{{\boldsymbol v}}
\def\bfw{{\boldsymbol w}}
\def\bfx{{\boldsymbol x}}
\def\bfy{{\boldsymbol y}}
\def\bfz{{\boldsymbol z}}

\def\bfA{{\boldsymbol A}}
\def\bfB{{\boldsymbol B}}
\def\bfC{{\boldsymbol C}}
\def\bfD{{\boldsymbol D}}
\def\bfE{{\boldsymbol E}}

\def\bfG{{\boldsymbol G}}
\def\bfH{{\boldsymbol H}}
\def\bfI{{\boldsymbol I}}
\def\bfJ{{\boldsymbol J}}

\def\bfL{{\boldsymbol L}}
\def\bfM{{\boldsymbol M}}
\def\bfN{{\boldsymbol N}}

\def\bfP{{\boldsymbol P}}
\def\bfQ{{\boldsymbol Q}}

\def\bfS{{\boldsymbol S}}

\def\bfU{{\boldsymbol U}}
\def\bfV{{\boldsymbol V}}
\def\bfW{{\boldsymbol W}}
\def\bfX{{\boldsymbol X}}
\def\bfY{{\boldsymbol Y}}
\def\bfZ{{\boldsymbol Z}}

\newcommand{\cH}{\mathcal{H}}

\newcommand{\wcH}{\widetilde{\mathcal{H}}}
\newcommand{\wL}{{\widetilde{\bfL}}}
\newcommand{\wS}{{\widetilde{\bfS}}}
\newcommand{\wW}{{\widetilde{\bfW}}}

\newcommand{\tr}{{\tilde{r}}}
\newcommand{\tmu}{{\tilde{\mu}}}
\newcommand{\sL}[1][t-1]{|\lambda^*_{k+1}|+\Big(\frac{1}{2}\Big)^{#1}|\lambda^*_k|}

\renewcommand{\arraystretch}{1.5}

\def\v{{\bf v}}

\newcommand{\wt[1]}{\widetilde{\bfw}^{(#1)}}

\begin{document}
 
\include{./title/rpititlephd}   
\include{./acknowledge/rpiack}  
\include{./abstract/rpiabs} 

\chapter{\uppercase{Introduction}}
Non-convex optimization are widely involved in the modern learning problems, including matrix completion \cite{AFSU07,CT10, Netflix,CP10} and neural network learning \cite{SHA16,CW08,KSHE12}. In matrix completion problems, the natural constraint is the low-rankness of the ground-truth matrices. In training neural networks, the data itself is usually assumed to be represented by a non-linear model, and the objective is a non-convex function in general. Such non-convex objectives and constraints naturally describe the learning problems themselves but leave the optimization problem NP-hard if no other assumptions are imposed. Unlike convex optimization problems, there is no general framework for the theoretical analysis of non-convex optimization problems. 

\section{Robust Matrix Completion}
Robust Matrix completion (RMC) \cite{CJSC13} aims to recover a low-rank matrix from its partial observations of the measurements that may contain significantly large but sparse bad data. 
RMC problems are closely related to miss data recovery and principal component analysis, and sampling applications include video surveillance \cite{WGRPM09}, face recognition \cite{CLMW11}, image processing \cite{OCS15}, network traffic analysis \cite{MMG13}, and power system \cite{GWCGFSR16}. 
Direct applications are the data recovery of missing points, detection/correction of outliers and anomalies, predictions of the future points in network analysis, power system, and video processing.  Take the background subtraction in video processing as an instance. The background is almost statics across frames and can be viewed as a low-rank matrix if each frame is vectorized.  Meanwhile, the foreground is dynamic and only appears in a small fraction of pixels in each frame so that the foreground data can be formulated as sparse errors. 
Typical RMC is formulated as an optimization problem by decomposing the observed measurements into two matrices, where one matrix is low-rank (represent the clean data) and the other one is sparse (represent the bad data). The non-convexity comes naturally from the rank constraint for clean data and $\ell_0$ norm for bad data.
 
The conventional method is to relax the non-convex constraints into the corresponding convex ones, i.e., rank into nuclear, and $\ell_0$ into $\ell_1$. With properly selected regularization parameters, the ground truth are approximately or  exactly the solutions to the convex relaxation \cite{CT10,CR09,Gross11,CJSC13,KLT17,CLMW11, CXCS11,HKZ11}. However, the convex relaxation is time-consuming to solve. Another line of research focuses on developing non-convex algorithms \cite{GWL16,KC12,CW15,NNSAJ14,YPCC16}. While these approaches are much more computationally efficient than the convex alternatives, the convergence to the global optimum often has a higher requirement in the number of observations theoretically.

\subsection{Column-wise Corruptions} 
In a practical setting, the data matrix may experience column-wise loss or corruption. For instance, in power systems, the measurements in all PMUs may be lost at a certain period due to communication congestion. 
However, the current MC algorithm can locate the column at best but cannot recover the actual values if all the observations in one column are corrupted or missing  \cite{CJSC13,CGJ16,XCS12,CXCS11}.
Since every column is a data point in the $r$-dimensional column subspace, where $r$ is the rank of the data matrix,
even if the column subspace is correctly identified, at least $r$
linearly independent equations, i.e., $r$ entries of each column,
are needed to determine the exact values of that column.

In exploring the low-rankness, some low-rank matrix has been proved to have the additional low-rank Hankel property that the structured Hankel matrix from the original matrix is still low-rank. 
For instance, if the data matrices contain the time series of multiple output channels in a dynamical system that 
a reduced order linear system can approximate, then the structured Hankel matrix of the original data matrix is still approximately low-rank \cite{ZHWC17,ZHWC18}. 
In array signal processing, the structured Hankel
matrix of a spectrally sparse signal is low-rank,
and the rank depends on the number of sinusoidal components \cite{CC14,YX16}.
The low-rank Hankel property also holds for a class of finite
rate of innovation (FRI) signals, which are motivated by MRI
imaging \cite{Haldar14,OBJ18,YKJL17}.
 
\subsection{Multi-channel Hankel Matrix}
This thesis develops a new model to characterize the intrinsic structures of multiple time series generated by a linear dynamical system. 
This model can be viewed as an extension of the traditional low-rank Hankel matrix model in \cite{CC14,CWW17} to an $n_c$-channel Hankel matrix. The degree of freedom for a traditional low-rank Hankel matrix with rank $r$ is in the order of $r$, while for a rank $r$, $n_c$-channel Hankel matrix in this thesis is in the order of $n_cr$.
Besides linear dynamical system, this multi-channel Hankel matrix model also characterizes spectrally sparse signals in applications like radar imaging and magnetic resonance imaging, and the structured models are proved to be low-rank in practical settings.  

Exploiting the proposed multi-channel Hankel matrix model, several fast algorithms, all with linear convergence rates, are developed to recover the missing points and corruptions with a provable performance guarantee to the ground truth.  Specifically, the computational complexity to recover a rank $r$, $n_c$-channel Hankel matrix with $n$ measurements in each channel is in the order of $r^2 n_c n$ per iteration, and the number of iterations to achieve $\varepsilon$ accuracy is in the order of $\log(1/\varepsilon)$, which is far less than the conventional RMC methods as $n_c^2n/\varepsilon$. 
Moreover, the number of required samples for the recovery is in the order of $\mu r^3 \log n$, which is significantly smaller than the existing bound of $rn\log n$ for a general rank $r$ matrix.
In addition, 
the algorithms can tolerate up to the order of $1/r$ fraction of fully corrupted columns.



\section{Neural Network Learning}
 Recently, neural networks have come into the sight of the public and have demonstrated  their superior performance in machine
learning tasks, for instant, image classification \cite{KSHE12} and recognition \cite{LGTB97}, natural language processing \cite{CW08}, and   strategic game program \cite{SHA16}. A basic neural network model contains the input layer, output layer, and several hidden layers in between.
Due to the diversity of the machine learning tasks, there is a brunch of different neural network structures, i.e., fully connected neural network, convolutional neural network \cite{LBBH98}, graph neural networks \cite{GMS05,SGAHM08}, to name a few. Different neural network structures may lead to significant differences in the neurons' connections among layers and thus the expressive power of the neural networks. For instance, convolutional neural networks use shared weights architecture and find applications mainly in image and video classification. In comparison, graph neural networks consider the topology of structured data and are mainly used in analyzing social network and biology data.  No matter which kind of neural network structure, 
 the targets are all to simulate the mapping from input to output with the appropriate parameters for the hidden layers. 
Therefore, training a neural network needs to find appropriate parameters for the hidden layers using the training data and is achieved by minimizing a non-convex empirical loss function over the choices of the model parameters. The first-order gradient descent (GD) algorithm is usually used to solve the non-convex learning problem and has been succeeded in many numerical experiments  \cite{SHA16, BPLR16, YHCEHL18, DMIBH15, MDSG06, su2020does,XLHL20,CKSNH20,YJCPM19,ZGLC20}. 

\subsection{Generalization of the Learned Model}

Despite the numerical success, neural networks still lack transparency due to the absence of theoretical explanations. 
The ability to explain the performance of neural networks is essential because of the potential overfitting problems, which is also known as the generalizability of the learned model.  
For instance, a successful self-driving car under experimental environments is not guaranteed to gain the public's trust. The safety concern that how accurately the car  would perform in the unseen road conditions still prevents customers from buying it. 
Same in general machine learning tasks,
the potentials of overfitting, where the parameters learned from the training data may perform poorly on the testing data, are one of the biggest hurdle to the widespread acceptance of deep neural network learning.  The study of the generalization error, which measures how accurately the trained parameters predict the unseen data, provides insight into when the overfitting will happen and how to avoid the overfitting.  

The significant challenge lies in achieving the small training error and small generalization gap (difference of test and training errors) simultaneously.
First, from the perspective of optimization, the learning problem is formulated as a non-convex paradigm in general. 
However, the convergence to the global optimal by gradient descent is not guaranteed naturally due
to the existence of spurious local minima. Therefore, achieving small training errors is not a trivial problem. 
Second, 
classical generalization theories, like VC-dimension \cite{BM02} and Rademacher Complexity, indicate that models in low complexity tend to  have small generalization errors. 
However, this requires either the trainable parameters to be much less than the training samples or the ranges of trainable parameters are restricted to a small space, and either of the assumptions fits the current neural network training setup.
Therefore,
the analytical performance about the generalizability of the learned model is still missing.
 
\textbf{Feasibility of a small training error.} 
The feasibility of small training error is to answer the convergence to the global optimum. However, it is well-known that solving the optimization problem for a one-hidden-layer neural network with multiple neurons is even NP-Complete \cite{BR92} without any  assumption. Therefore, various assumptions are imposed in existing works to make the theoretical analysis feasible.

A critical line of the research is over parameterized neural network learning, where the number of parameters is far more than training samples, and the learned model are guaranteed to reach the global optimum from a random initialization under some mild assumptions. 
Earlier works focus on the landscape analysis, and the objective loss functions are proved to have no spurious local minima for shallow neural networks \cite{LSS14,ZBHRV16,SJL18}. The most recent development in training over parameterized networks is the Neural Tangent Kernel (NTK) framework \cite{ZLL19, JGH18, DLLW19,ADHLW19}, which builds the connections of infinitely wide neural networks and  Gaussian kernel methods. In NTK, the gradient descent step size is assumed to be infinitely small, and the parameters are initialized as random Gaussian variables. This setting makes it possible to consider neural network parameters in a relatively small region around the initialization when the neuron widths are large enough.  From the perspective of optimizations, the iterations stay close to the initialization, and the objective functions become approximately linear during the training process, which leads the non-convex problems to be convex.

Another important line of research studied neural networks from the perspective of model estimation \cite{ZSJB17,FCL20,ZWLC20_1,ZWLC20_2,ZWLC20_3}. The input and output data are assumed to be generated from an underlying model with unknown ground truth parameters. Then, the set of the ground truth model parameters is a global minimizer of the objective loss function. Therefore, if the parameters are estimated correctly, the generalization error (test error) is bounded naturally. In the local convex region of the ground truth, achieving a small test error is equivalent to obtaining a small training error. As the landscape is intractable for unknown input distributions \cite{Sh18}, the data are usually assumed to follow Gaussian distributions. If we further have  the hidden layer only containing one neuron, the local convex region near the ground-truth parameters is sufficiently large such that random initialization can converge to global optimal with at least a constant probability \cite{DLTPS17,DLT17,BG17}. However, as the number of neurons increases in the hidden layer, the number of  spurious local minima becomes sufficiently large and cannot be characterized. 
The local convex region near the ground truth for neural networks with multiple neurons is characterized in
\cite{ZSJB17}, and an initialization method via tensor decomposition is proposed to generate an initial point in the characterized local convex region. Therefore, the convergence to the global optimal is achieved, and the learned model has zero training and thus test error.  

\textbf{Generalization theories.}
In general, small training errors cannot  guarantee that the learned model can generalize well on unseen testing data. The classic generalization theories, such as VC-dimension and Rademacher complexity, heavily depends on the model complexity. The concept of ``model complexity'' can be interpreted in two aspects: (1) the dimension of the trainable parameters; (2) the range of the trainable parameters. Therefore, there is no bounded generalization for over-parameterized neural networks without any constraints. 
Nevertheless, with the assumptions in NTK framework, the iterations lie in a small regime near the initialization.
Hence, the model complexity can be bounded by the small range of the trainable parameters \cite{ADHLW19}. To transfer the bounded model complexity into bounded generalization through classic generalization theories, both the data and the target functions needs to satisfy some nice properties, such as data needs to be bounded, and the target functions need to have bounded Lipschitz constant and be infinitely differentiable \cite{ZLL19,ADHLW19}. While in the framework of model estimation, the generalization bound is characterized naturally by the distance of learned weights and the ground truth \cite{ZSJB17}.


\subsection{Linear Convergence with Generalization Guarantee: Convolutional Neural Networks and Graph Neural Networks}
This thesis exploited the convergence analysis and generalization guarantee of neural networks from the perspective of model estimation. To make the analysis feasible, this thesis mainly focuses on  one-hidden layer cases, e.g., shallow non-overlapping  convolutional neural network and graph neural networks. The traditional neural networks can be viewed as a special case of convolutional neural networks when the feature is the input itself or graph neural networks when the graph has no edges. 

Specifically, the key findings of these two neural network structures can be summarized as the first sample complexity analysis for achieving zero generalization error with linear convergence rate via the accelerated gradient descent algorithms.  The required sample complexity for the exact estimation is a linear function of the input (feature) dimension. For graph neural networks, the sample complexity is also an increasing function of the largest singular value of the normalized adjacency matrix and the degree of the graph, where the two factors are related to the number of edges in the graph. Therefore, the theorem explains the intuition that
more edges in the graph correspond to the stronger dependence of the labels on neighboring features, thus requiring more samples to learn these dependencies.

\subsection{Improved Generalization via Network Pruning: Lottery Ticket Hypothesis}

Although learning neural networks with enough training samples can guarantee zero generalization errors. 
However, the recent success of modern machine learning tasks is built upon larger neural networks, which leads to the demand for the large size of reliable training samples.
The neural network pruning becomes a popular approach to reduce the computational cost of model training and inference by removing unnecessary neurons or weights in the neural networks \cite{LDS90,HS93,DCP17,HPTD15,HPTTC16}. However, random pruned networks usually suffer from a significant loss over the original dense networks, while structured pruning methods are time-consuming and not practical. 

The recent  \textit{Lottery Ticket Hypothesis} (LTH)  \cite{FC18}
casts an insight into understanding network pruning with numerical verifications.
Specifically, LTH claims that a randomly  initialized dense neural network always  contains a so-called ``winning ticket'', which is a sub-network  bundled with the corresponding initialization, such that when trained in isolation, this winning ticket can achieve at least the same testing accuracy as that of the original network by running at most the same amount of training time.
Since the proposition of LTH, it has attracted a significant  amount  of recent research interests   \cite{RWKKR20,ZLLY19,MYSS20}.

This thesis provides the systematic analysis of learning pruned neural networks with a finite number of training samples in the {oracle-learner} setup, where the training data are generated by an unknown neural network, the \textit{{oracle}}, and another network, the \textit{{learner}}, is trained on the dataset. The analytical results  also justify the LTH from the perspective of the sample complexity. 

\subsection{Improved Generalization via Unlabeled Data: Semi-supervised Learning}
Semi-supervised learning is another approach to solve the contradictions of large model size and a small portion of trainable data, which combines a small amount of labeled data with a large amount of unlabeled data during training.
Due to the easy access of the unlabeled data, research has seen the diversified developments of the semi-supervised learning approaches \cite{Yar95,GB05,LA16,SBCZZRCKL20,LWHL19,MMLI18}, and the key differences lie in how to generate and make use of the unlabeled data. 

Self-training is one of the most powerful semi-supervised learning  algorithms, which augments a limited number of labeled data with  unlabeled data to achieve
improved   generalization performance   on test data, compared with the model trained by  supervised learning using the labeled data only. 
The terminology ``self-training'' has been used to describe various algorithms in the literature, however, this thesis is centered on  the classic iterative   self-training method in particular \cite{Lee13}.  In this setup \cite{ZWLCX22}, an  initial teacher model (learned from the labeled data) is applied to the unlabeled data to generate pseudo labels. One then trains a student model   by minimizing the weighted empirical risk of both the labeled and unlabeled data. The student model is then used as the new teacher model to update the pseudo labels of the unlabeled data. This process is repeated multiple times to improve the eventual student model.

This thesis provides the theoretical study of iterative self-training on nonlinear neural networks. Focusing on one-hidden-layer neural networks, this thesis provides a quantitative analysis of the generalization performance of iterative self-training as a function of the number of labeled and unlabeled samples. 





\section{Organization of This Thesis}
The rest of the thesis is organized as follows. 
Chapter \ref{chapter: 1} introduces the relative works of low-rank Hankel matrix completion from partially observed data and the proposed recovery algorithm. 
In Chapter \ref{chapter: 2}, the robust Hankel matrix completion from partially observed data and bad data and the proposed recovery algorithm are included. Chapter \ref{chapter: 3} introduces the convergence analysis of non-overlapping convolutional neural networks via the stochastic gradient descent and tensor initialization algorithms. 
Chapter \ref{chapter: 4} summarized the theoretical and experimental results of graph neural networks in solving regression and classification problems. 
In Chapter \ref{chapter: 5}, the theoretical analysis of a pruned neural network is provided, and the connections with the existing Lottery Ticket Hypothesis \cite{FC18} are briefly discussed. 
In Chapter \ref{chapter: 6}, this thesis revisits the classic self-training algorithms and provides theoretical insights into understanding the roles of unlabeled data in semi-supervised learning.

\include{./chapter_1/rpichap1} 
\include{./chapter_2/rpichap2}
\include{./chapter_3/rpichap3}
\include{./chapter_4/rpichap4}

\begingroup
\newcommand{\w}{\widetilde{\bfw}}
\include{./chapter_5/rpichap5}
\endgroup

\begingroup
\newcommand{\w[1]}{\widetilde{#1}}
\include{./chapter_6/rpichap6}
\endgroup

\include{./bibography/rpibib}

\include{./appendix/rpiapp}

\end{document}

%% file: title/rpititlephd.tex
%
    
%
\thesistitle{Non-convex Optimizations for Machine Learning with Theoretical Guarantee:\\Robust Matrix Completion and Neural Network Learning}        
\author{Shuai Zhang}        
\degree{Doctor of Philosophy}        
\department{Electrical, Computer, and Systems Engineering} 
     
\signaturelines{4}     
\thadviser{Meng Wang}
\memberone{Birsen Yazici}        
\membertwo{Ali Tajer}        
\memberthree{John E. Mitchell}

\submitdate{[December 2021]\\
Submitted November 2021}        
\copyrightyear{2021}   

%
\titlepage     
\copyrightpage         
\tableofcontents        
\listoftables          
\listoffigures         

%% file: acknowledge/rpiack.tex
 \newpage
\specialhead{\uppercase{Mathematical Notations}}
\begin{minipage}{\textwidth}
\centering
\centerline{\bf Numbers and Arrays}
\bgroup
\def\arraystretch{1.5}
\begin{tabular}{cp{4in}}
$\displaystyle a$ & A scalar (integer or real)\\
$\displaystyle \bfa$ & A vector\\
$\displaystyle \bfA$ & A matrix\\
$\displaystyle \bfI_n$ & Identity matrix with $n$ rows and $n$ columns\\
$\displaystyle \bfI$ & Identity matrix with dimensionality implied by context\\
$\displaystyle \bfe_i$ & Standard basis vector $[0,\dots,0,1,0,\dots,0]$ with a 1 at position $i$\\
$\displaystyle \text{diag}(\bfa)$ & A square, diagonal matrix with diagonal entries given by $\bfa$\\
\end{tabular}
\egroup
\index{Scalar}
\index{Vector}
\index{Matrix}
\index{Tensor}
\end{minipage}

\vspace{0.1in}
\begin{minipage}{\textwidth}
\centering
\centerline{\bf Sets}
\bgroup
\def\arraystretch{1.5}
\begin{tabular}{cp{4in}}
$\displaystyle \mathbb{A}$ & A set\\
$\displaystyle \mathbb{Z}$ & The set of integers \\
$\displaystyle \R$ & The set of real numbers \\
$\displaystyle \C$ & The set of complex numbers \\
$\displaystyle \{0, 1\}$ & The set containing 0 and 1 \\
$\displaystyle [n]$ & The set of all integers between $0$ and $n$ as  $\{0, 1, \dots, n \}$\\
$\displaystyle [a, b]$ & The real interval including $a$ and $b$\\
$\displaystyle (a, b]$ & The real interval excluding $a$ but including $b$\\
\end{tabular}
\egroup
\index{Scalar}
\index{Vector}
\index{Matrix}
\index{Tensor}
\index{Graph}
\index{Set}
\end{minipage}

\vspace{0.1in}
\begin{minipage}{\textwidth}
\centerline{\bf Indexing}
\bgroup
\def\arraystretch{1.5}
\begin{tabular}{cp{4in}}
$\displaystyle \bfa_i$ & Element $i$ of vector $\bfa$, with indexing starting at 1 \\
$\displaystyle \bfa_{-i}$ & All elements of vector $\bfa$ except for element $i$ \\
$\displaystyle \bfA_{i,j}$ & Element $i, j$ of matrix $\bfA$ \\
$\displaystyle \bfA_{i, :} \text{ or } \bfA_{i*}$ & Row $i$ of matrix $\bfA$ \\
$\displaystyle \bfA_{:, i} \text{ or } \bfa_i$ & Column $i$ of matrix $\bfA$ \\
\end{tabular}
\egroup
\end{minipage}

\vspace{0.1in}
\begin{minipage}{\textwidth}
\centerline{\bf Linear Algebra Operations}
\bgroup
\def\arraystretch{1.5}
\begin{tabular}{cp{5in}}
$\displaystyle \bfA^{\top}$ & Transpose of matrix $\bfA$ \\
$\displaystyle \bfA^\mathsf{H}$ & Conjugate transpose of matrix $\bfA$ \\
$\displaystyle \bfA^\dagger$ & Moore-Penrose pseudoinverse of $\bfA$\\
$\displaystyle \bfA \odot \bfB $ & Element-wise (Hadamard) product of $\bfA$ and $\bfB$\\
$\displaystyle \mathrm{det}(\bfA)$ & Determinant of $\bfA$ \\
\end{tabular}
\egroup
\index{Transpose}
\index{Element-wise product|see {Hadamard product}}
\index{Hadamard product}
\index{Determinant}
\end{minipage}

\vspace{0.1in}
\begin{minipage}{\textwidth}
\centerline{\bf Calculus}
\bgroup
\def\arraystretch{1.5}
\begin{tabular}{cp{5in}}
$\displaystyle\frac{d y} {d x}$ & Derivative of $y$ with respect to $x$\\ [2ex]
$\displaystyle \frac{\partial y} {\partial x} $ & Partial derivative of $y$ with respect to $x$ \\
$\displaystyle \nabla_\bfx y $ & Gradient of $y$ with respect to $\bfx$ \\
$\displaystyle \frac{\partial f}{\partial \bfx} $ & Jacobian matrix $\bfJ \in \R^{m\times n}$ of $f: \R^n \rightarrow \R^m$\\
$\displaystyle \nabla_\bfx^2 f(\bfx)$ & The Hessian matrix of $f$ at input point $\bfx$\\
$\displaystyle \int_\mathbb{S} f(\bfx) d\bfx$ & Definite integral with respect to $\bfx$ over the set $\mathbb{S}$ \\
\end{tabular}
\egroup
\index{Derivative}
\index{Integral}
\index{Jacobian matrix}
\index{Hessian matrix}
\end{minipage}

\begin{minipage}{\textwidth}
\centerline{\bf Probability}
\bgroup
\def\arraystretch{1.5}
\begin{tabular}{cp{3.7in}}
$\displaystyle p(\bfa)$ & A probability distribution over a continuous variable\\
$\displaystyle \bfa \sim P$ & Random variable $\bfa$ has distribution $P$\\
$\displaystyle  \mathbb{E}_{\bfx\sim P} [ f(x) ]\text{ or } \mathbb{E} f(x)$ & Expectation of $f(x)$ with respect to $P(\bfx)$ \\
$\displaystyle \text{Var}(f(x)) $ &  Variance of $f(x)$ under $P(\bfx)$ \\
$\displaystyle \bfx \sim \mathcal{N} (\mu , \Sigma)$ & Gaussian distribution %
over $\bfx$ with mean $\mu$ and covariance $\Sigma$ \\
\end{tabular}
\egroup
\index{Independence}
\index{Conditional independence}
\index{Variance}
\index{Covariance}
\index{Kullback-Leibler divergence}
\index{Shannon entropy}
\end{minipage}

\vspace{0.1in}
\begin{minipage}{\textwidth}
\centerline{\bf Functions}
\bgroup
\def\arraystretch{1.5}
\begin{tabular}{cp{4in}}
$\displaystyle f: \mathbb{A} \rightarrow \mathbb{B}$ & The function $f$ with domain $\mathbb{A}$ and range $\mathbb{B}$\\
$\displaystyle f \circ g $ & Composition of the functions $f$ and $g$ \\
  $\displaystyle f(\bfx ; \theta) $ & A function of $\bfx$ parametrized by $\theta$.
  (Sometimes we write $f(\bfx)$ and omit the argument $\theta$ to lighten notation) \\
$\displaystyle \log x$ & Natural logarithm of $x$ \\
$\displaystyle || \bfx ||_p $ & $\ell_p$ norm of $\bfx$ \\
$\displaystyle || \bfx || $ & $\ell_2$ norm of $\bfx$ \\
$\displaystyle || \bfX ||_F $ & Frobenius norm of $\bfx$ \\
$\displaystyle || \bfX ||_2 $ & Spectrum norm of $\bfx$ \\
\end{tabular}
\egroup
\index{Sigmoid}
\index{Softplus}
\index{Norm}
\end{minipage}

\vspace{0.1in}
\begin{minipage}{\textwidth}
\centerline{\bf Order Analysis}
\def\arraystretch{1.5}
\begin{tabular}{cp{4in}}
$\displaystyle f(n) = \mathcal{O}\big(g(n)\big)$ & $f(n)\leq Cg(n)$ holds for some constant $C>0$ when $n$ is sufficiently large\\
$\displaystyle f(n) = {\Omega}\big(g(n)\big)$ & $f(n)\geq Cg(n)$ holds for some constant $C>0$ when $n$ is sufficiently large\\
$\displaystyle f(n) = \Theta\big(\big(g(n)\big)$ & $f(n) = Cg(n)$ holds for some constant $C>0$ when $n$ is sufficiently large\\
\end{tabular}
\end{minipage}
\newpage
\specialhead{ACKNOWLEDGEMENTS}
 
I would like to take this valuable opportunity to express my appreciation to everybody helping me during my Ph.D. career at RPI. 

First and foremost, I would like to thank my Ph.D. advisor Meng Wang, and three IBM mentors Sijia Liu, Pin-Yu Chen, and Jinjun Xiong. During my Ph.D. study, Prof. Wang  provided me the assistant in learning mathematical tools and helped me pick the proper project that I could make contributions to. She is also a kind and attentive teacher, which keeps providing me encourages and passions exploring the world. Prof. Liu, Dr. Chen, and Prof. Xiong are talented mentors and wise researchers. They are knowledgeable and experts in neural network learning and related fields. I learned a lot from them on how to find and formulate the problems. All of them provided me tremendous help in writing essays and being an independent researcher.

I would like to thank my committee members, Prof. Ali Tajer, Prof. Birsen Yazici, and Prof. John E. Mitchell. I am fortunate to take the statistical courses of Prof. Tajer and Prof. Yazici, which provide me with a broader vision of statistical learning and motivate me to explore the possibility in the related field. 
I also feel lucky to have the opportunity to learn the basics of non-convex optimization from Prof. Mitchell before starting my research. All of them are intelligent and inspiring teachers, and I have learned useful mathematical tools from them, which helps me in exploring my researcher. 

Besides them, I also want to thank my lab mates during my Ph.D. studies: Pengzhi Gao, Yingshuai Hao, Wenting Li, Ren Wang, Ming Yi, Hongkang Li, Nowaz Rabbani, Yating Zhou and Jiawei Sun. 
I enjoy the smooth and cheerful atmosphere in the room of JEC 6308, which motives  me in finishing my Ph.D. study. 
Special thanks to Yingshuai who helped me a lot in nearly every aspect of my research and life at my early stage as a Ph.D. student, and my first work (in Chapter 2) in this thesis is built upon our collaborated paper. 

Moreover, I would like to provide  thanks to  the administrative staff at RPI for making a grate environment for graduate studies. Special thanks to Kelley Kritz and Prof. Alhussein Abouzeid for covering up  my mistake when formulating the thesis committee. 

In addition, I am very grateful to all of my friends and peers for making my Ph.D. years in U.S. the most memorable years of my life.
Special thanks to Si Cheng, Songyang Zhang, Yulong Wang, Xiaoxiao Wang, Qiang Shen, Lei Li, Junze Yuan,  Xi Fu, and Lei Li for the enjoyable memorizes when having trips in American. 
Special thanks to Yanwen Chen, Xiang Zhou, Xueqing Liu, Chunheng Jiang, Siwen Zhang, Hao Lu, and Feike Wu for a wonderful campus life at RPI. 
Special thanks to    William Tang, Kev, Moon, Aaron, Vincent Wong, and Zenan Jin for the enjoyable time when isolation during COVID-19.

Last but not the least, I would like to thank my parents for their selfless love and supports on my Ph.D. study.

%% file: abstract/rpiabs.tex
 
\specialhead{ABSTRACT}
Despite the recent development in machine learning, most learning systems are still under the concept of ``black box'', where the performance cannot be understood and derived.
With the rise of safety and privacy concerns in public, designing an explainable learning system has become a new trend in machine learning. 
In  general, many machine learning problems are formulated as minimizing (or maximizing) some loss function. Since real data are most likely generated from non-linear models, the loss function is non-convex in general. 
Unlike the convex optimization problem, gradient descent algorithms will be trapped in  spurious local minima in solving non-convex optimization. Therefore, it is challenging to provide explainable algorithms when studying non-convex optimization problems.  
In this thesis,  two popular non-convex problems are studied: (1) low-rank matrix completion and (2) neural network learning.
 
In low-rank matrix completion (MC), the objective is to recover a low-rank matrix from partial observations that may contain significant errors. 
When there are errors in the observations, the problem is also called Robust Matrix Completion (RMC). 
Both RMC and MC problems are  non-convex due to the natural constraint of low-rankness. 
However, the low-rank structure does not capture the temporal correlations in some time series, i.e., power system monitoring, magnetic resonance imaging, and array signal processing. 
As a result, low-rank MC cannot handle the whole column/row being fully lost. 
In this thesis,
a new model, termed multi-channel Hankel matrices, is proposed to characterize the intrinsic low-dimensional structures in some multichannel time series. By exploiting the new model in this thesis, two projected gradient-based algorithms are developed to solve the non-convex MC and RMC problems, respectively. The algorithms converge linearly to the ground truth data matrix as long as all the measurements are in at most a constant fraction of lost or corrupted columns. Theoretical results suggest that the required number of observations for successful estimation is significantly less than the existing bounds for conventional MC/RMC. 



Learning a neural network is to find appropriate parameters for the hidden layers using the training data, and the learned model is achieved by minimizing a non-convex empirical loss function over the choices of the model parameters. A reliable learning model requires a small generalization error, which simultaneously achieves a small training error and generalization gap. 
According to the classic generalization theories, such as Vapnik–Chervonenkis (VC) dimension,  the bounded generalization requires a larger number of training samples than the model complexity. However, solving the optimization problems with such a number of  samples is not guaranteed to find a minimum with a small training error due to the high non-convexity of the objective functions. Therefore,  studying the convergence to the global optimum when training neural networks is vital and challenging. 
This thesis provides the convergence analysis to the global optimum when the number of samples is larger than the model complexity for one-hidden-layer neural networks with Gaussian inputs. 
Also, the minimal required training samples to guarantee zero generalization error (or bounded generalization error for noisy measurements and binary classification problems)
are presented for various neural network architectures.

Although the sample complexity for zero generalization errors is established, there are cases where the training process is not accessible to adequate training samples due to the difficulty of generating reliable data. 
Therefore,
this thesis further explores the methods in dealing with a limited number of training samples, and the focuses are the network pruning and self-training algorithms. The motivation for studying self-training comes naturally from a semi-supervised framework, which leverages many unlabeled data to improve learning when the labeled data are limited. In contrast, the network pruning is mainly inspired by the recent Lottery Ticket Hypothesis (LTH), which claims that a good pruned network achieves a faster convergence rate and higher test accuracy than the original dense network. As the model complexity of the pruned network is highly reduced, the required sample complexity should decrease simultaneously for a bounded generalization. 
In this thesis, the sample complexity and convergence analysis of a pruned network are characterized, and the properties of a good pruned network, improved convergence rate and test accuracy,  can be explained through the given theoretical results. 
For self-training, this thesis proves the benefits of unlabeled data in both training convergence and generalization ability, and the improvement of the convergence rate and generalization accuracy are characterized quantitatively given some sufficiently large number of unlabeled data.

%% file: chapter_1/rpichap1.tex
 
\chapter{\uppercase{Multi-Channel Hankel Matrix Completion}}\label{chapter: 1}
\blfootnote{Portions of this chapter previously appeared as: S.~Zhang, Y.~Hao, M.~Wang, and J.~H. Chow, ``Multi-channel Hankel matrix completion through nonconvex
  optimization,'' \emph{IEEE J. Sel. Topics Signal Process.}, vol.~12,
  no.~4, pp. 617--632, Apr. 2018.}
\blfootnote{This chapter is based on the work collaborated with Yingshuai Hao. We contributed to the work equally.
In order to present the work coherently, this chapter includes all of my work and our collaborated work, and
part of Yingshuai’s work.}
\input{./chapter_1/introduction_publication_ready.tex}

\input{./chapter_1/formulation_publication_ready.tex}

\input{./chapter_1/algorithm_publication_ready.tex}
\input{./chapter_1/theoretical_publication_ready.tex}

\input{./chapter_1/numerical_publication_ready.tex}

\section{Summary}
This chapter characterizes the  intrinsic low-dimensional structures of correlated time series through multi-channel low-rank Hankel matrices. 

	Two iterative hard thresholding algorithms with linear convergence rates are proposed to solve the nonconvex missing data recovery problem. 
	Our bound of the required number of observed entries for successful recovery is $O (r^2 \log^2n)$,   significantly smaller than $O(rn \log^2 n)$  by conventional low-rank matrix completion methods. Our bound is slightly larger than the degree of the freedom $\Theta(n_cr)$, and 
	we suspect that the bound can be improved with better proof techniques. 
	The convergence rate is proved to be accelerated further by adding a heavy-ball step, which also increases the tolerable missing data percentage numerically. 
	
	One motivating application of our methods is power system synchrophasor data recovery.  Other applications include array signal processing and MRI image recovery. This chapter provides the first analytical characterization of multi-channel Hankel matrix completion methods, while existing works mostly focused on single-channel Hankel matrix recovery.  
	One future direction is to study data recovery from both data losses and corruptions, where partial measurements contain significant errors. 
	The bad data should be first located and removed before recovering the missing points.

%% file: chapter_1/introduction_publication_ready.tex
\section{Introduction}\label{sec:intro}
Missing data recovery is an important task in various applications such as covariance estimation from partially observed correlations in  remote sensing \cite{CP10}, multi-class learning in machine learning \cite{AFSU07,CT10},  the   Netflix Prize \cite{Netflix} problem and other similar
questions in collaborative filtering \cite{GNOT92}. Moreover, the recent framework of super-resolution enables accurate signal recovery from   sparsely sampled measurements \cite{CF14}. Example applications include magnetic resonance imaging (MRI) \cite{Haldar14,JLY16,SLOE14} and target localization in radar imaging \cite{CC14,YX16}. In power system monitoring, Phasor Measurement Units (PMU) \cite{PT08} can measure voltage and current phasors directly at various locations and transmit the measurements  to the operator for state estimation \cite{AI09,DCTP10} or disturbance identification \cite{MZFBD10}. 
Some PMU data points, however, do not reach the operator due to PMU malfunction or communication congestions. These missing data points should be recovered for the subsequent applications on PMU data \cite{GWGCFS16}.

Since  practical datasets often have intrinsic low-dimensional structures, the missing data  recovery problem can be formulated as a low-rank matrix completion problem, which is nonconvex due to the rank constraint. Its convex relaxation, termed Nuclear Norm Minimization (NNM) problem, 
has been extensively investigated  \cite{CT10,CR09,Fazel01,Gross11}.   Given an $n_c\times n$ ($n_c \leq n$) matrix with rank $r$ ($r\ll n$), as long as $O(rn \log^2 n)$ entries are observed, one can recover   the remaining entries accurately by solving NNM \cite{CT10,CR09,Gross11}.  

Although elegant theoretical analyses exist, 
 convex approaches like NNM have high computational complexity and poor convergence rate. For
example, to decompose an $n_c \times n$ matrix, the per-iteration complexity of the best specialized
implementation is $O(n_c^2n)$ \cite{NNSAJ14}. To reduce the computational complexity, first-order algorithms like \cite{JNS13} have been developed to solve the non-convex problem directly. 
Despite the numerical superiority, the
theoretical analyses of the convergence and recovery performance of these nonconvex methods are still open problems.    Only a few recent work such as \cite{CWW17,JNS13} provided such analyses on a case-by-base basis.



The low-rank matrix model, however,  does not capture the temporal correlations in time series.
A permutation of measurements at different time steps would result in  different time series, but   the rank of the data matrix remains the same.  As a result, low-rank matrix completion methods require at least $r$ entries   in each column/row to recover the missing points and would fail if a complete column/row was lost.  
They cannot recover  simultaneous data losses among all channels. Simultaneous data losses are not uncommon in power systems due to communication congestions. 

There is limited  study of  the coupling of low-dimensional models and   temporal correlations. Parametric models like hidden Markov models \cite{ML13,MSR10} and autoregression (AR) models \cite{HW15,MSPD15} are employed to model  temporal correlations. The accuracy of the algorithms   depends on  the correct  estimation of  model parameters, and no theoretical analysis is reported. 


In this chapter, a new model is developed  to characterize the intrinsic structures of   multiple time series that are generated by a linear dynamical system.
Our model of \textit{multi-channel low-rank Hankel matrix} characterizes the temporal correlations in time series like PMU data   without directly modeling the dynamical systems and estimating the system parameters. 
Our model can also be viewed as an extension of  the single-channel low-rank Hankel matrix model with a $\Theta(r)$ degree of freedom in  \cite{CC14,CWW17}  to an $n_c$-channel matrix with a $\Theta(n_cr)$ degree of freedom. It can also characterize spectrally sparse signals in applications like radar imaging \cite{PEPC10} and magnetic resonance imaging \cite{LDP07}. 

Building upon the FIHT algorithm \cite{CWW17}, this thesis proposes two fast algorithms, termed accelerated multi-channel fast iterative hard thresholding (AM-FIHT) and robust AM-FIHT (RAM-FIHT) for multi-channel low-rank Hankel matrix completion.  They can recover missing points for simultaneous data losses. 
The heavy ball method \cite{P87, LESB16} is employed to accelerate the convergence rate, and the acceleration is evaluated  theoretically and numerically.  
 Our algorithms converge  linearly with a low per iteration complexity  $O(r^2n_cn+rn_cn\log n +r^3)$ to the original matrix (noiseless measurements) or a sufficiently close matrix depending on the noise level (noisy measurements). Theoretical analyses of FIHT with only noiseless measurements are reported \cite{CWW17}. Moreover, the recovery is successful as long as the number of observed measurements is $O(r^2\log^2n)$, significantly lower than $O(rn\log^2 n)$ for general rank-$r$ matrices. This number is also a constant fraction of the   required number  of measurements by applying the single-channel Hankel matrix completion methods like FIHT \cite{CWW17} to each channel separately. 

%% file: chapter_1/formulation_publication_ready.tex
\section{Problem Formulation}\label{c1sec:model}
Consider an $n_p$-th order linear dynamical system after an impulse response. Let $\bfs_t\in \C^{n_p}$ and $\bfx_t\in \C^{n_c}$ denote deviations of state variables and observations at time $t$ from the equilibrium point. Then we have 
\vspace{-0.05in}
\begin{equation}\label{A}
\begin{split}
\bfs_{t+1} & = \bfA\bfs_t, \ \ \bfx_t= \bfC\bfs_t,\quad t=1,2,\cdots,n,
\end{split}
\end{equation}
where $\bfA \in \C^{n_p\times n_p}$, and $\bfC \in \C^{n_c \times n_p}$.
Let   $\bfX$ contain the measurements from time 1 to $n$, 
\vspace{-0.05in}
\begin{equation}\label{M}
\bfX=[\bfx_1,\  \bfx_2,\ \cdots, \ \bfx_n]\in \mathbb{C}^{n_c\times n}.
\end{equation}
Further, the Hankel matrix of $\bfX$ is defined as  
\small
\begin{equation}\label{eqn:Hankel}
\mathcal{H}(\bfX)=
\begin{bmatrix}
\bfx_1 & \bfx_2 & \cdots & \bfx_{n_2}\\
\bfx_2 & \bfx_3 & \cdots & \bfx_{n_2+1}\\
\vdots & \vdots & \ddots & \vdots\\
\bfx_{n_1} & \bfx_{n_1+1} & \cdots & \bfx_n
\end{bmatrix}\in \C^{n_cn_1 \times n_2},	
\end{equation}
\normalsize
where $n_1+n_2=n+1$.

Suppose $\bfA$ could be diagonalized, denoted by $\bfA=\bfP{\bf\Lambda}\bfP^{-1}$, 
where
 $\bfP=[\bfl_1, \bfl_2, \cdots,  \bfl_{n_p}]$, 
 $\bfP^{-1}=[\bfr_1, \bfr_2,  \cdots, \bfr_{n_p}]^{H}$, and $(\cdot)^H$ stands for the conjugate transpose. ${\bf \Lambda}=\textrm{diag}(\lambda_1,\cdots,\lambda_{n_p})$ contains the eigenvalues of $\bfA$. Then
{
\begin{equation}\label{eqn:x}
\begin{split}
 \bfx_{t+1} 
 =\bfC \underbrace{\bfA \cdots \bfA}_{t \text{ times}} \bfs_1
 &= \bfC \bfA^{t} \bfs_1
 =\bfC \bfP{\bf\Lambda}^{t}\bfP^{-1} \bfs_1
 =\sum_{i=1}^{n_p}\lambda_i^t\bfr_i^H \bfs_1\bfC\bfl_i.
\end{split}
 \end{equation}}

 All $n_p$ modes of the system are considered in \eqref{eqn:x}. In practice, a mode might be highly damped ($|\lambda_i|\approx 0$), or
not excited by the input ($|\bfr_i^H  \bf s_1| \approx $ 0), or not directly measured ($\|\bfC \bfl_i\|\approx 0$). If only $r$ ($r \ll n$) out of $n$ modes are  significant, assuming these modes to be $\lambda_1$,..., $\lambda_r$ for simplicity, we have  
\begin{equation}\label{dynamic model}
\bfx_{t+1}\simeq\sum_{i=1}^{r}\lambda_i^t\bfr_i^H \bfs_1\bfC\bfl_i.
\end{equation}
Then the corresponding Hankel matrix can be written as
\small
\begin{equation}\label{Hankel X}
\mathcal{H}\bfX
=\begin{bmatrix}
\bfx_1 & \bfx_2 & \cdots & \bfx_{n_2}\\
\bfx_2 & \bfx_3 & \cdots & \bfx_{n_2+1}\\
\vdots & \vdots & \ddots & \vdots\\
\bfx_{n_1} & \bfx_{n_1+1} & \cdots & \bfx_n
\end{bmatrix}=\bfP_L{\bf\Gamma}\bfP_R^T,
\end{equation}
\normalsize
where
\small
\begin{equation}\label{P_l}
\bfP_L=\begin{bmatrix}
\bfI_{n_c} & \bfI_{n_c} & \cdots & \bfI_{n_c}\\
\lambda_1\bfI_{n_c} & \lambda_2\bfI_{n_c} & \cdots & \lambda_r\bfI_{n_c}\\
\vdots & \vdots & \ddots & \vdots\\
\lambda_1^{n_1-1}\bfI_{n_c} & \lambda_2^{n_1-1}\bfI_{n_c} & \cdots & \lambda_r^{n_1-1}\bfI_{n_c}
\end{bmatrix} \in \C^{n_c n_1\times n_c r},
\end{equation}
\begin{equation}
\boldsymbol{\Gamma}=\begin{bmatrix}
\bfr_1^H\bfx_1\bfC\bfl_1 & \bf0  & \cdots  & \bf0\\
\bf0 &\bfr_2^H\bfx_1\bfC\bfl_2&  \cdots & \bf0   \\
\vdots &  \vdots  & \ddots& \vdots\\
\bf0 &\bf0   & \cdots &\bfr_r^H\bfx_1\bfC\bfl_{r}\\
\end{bmatrix}\in \C^{n_c r\times r},
\end{equation}
\normalsize
and
\small
\begin{equation}
\bfP_R=\begin{bmatrix}
1 & 1 & \cdots & 1\\
\lambda_1 & \lambda_2 & \cdots & \lambda_r\\
\vdots & \vdots & \ddots & \vdots\\
\lambda_1^{n_2-1} & \lambda_2^{n_2-1} & \cdots & \lambda_r^{n_2-1}
\end{bmatrix}\in \C^{n_2\times  r},
\end{equation}
\normalsize
where $\bfI_{n_c}\in \mathbb{C}^{n_c \times n_c}$ is the identity matrix. One can check that both $\bfX$ and $\mathcal{H}\bfX$ are rank $r$ matrices\footnote{We assume $\mathcal{H}(\bfX)$ is exactly rank $r$ throughout this chapter. The methods analyses can be extended to approximately  low-rank matrices with minor modifications. If $\mathcal{H}(\bfX)$ is approximately low-rank, i.e., its rank-$r$ approximation error is very small, we seek to find the best rank-$r$ approximation to $\mathcal{H}(\bfX)$. Then the recovery error is at least the approximation error.}.

 Let $\bfN \in \C^{n_c \times n}$ denote the measurement noise. $\bfM=\bfX+\bfN$ denotes the noisy measurements.  Some entries of $\bfM$ are not observed due to data losses. Let $\widehat{\Omega}$ denote the index set of observed entries. The objective of missing data recovery is to reconstruct the missing data based on the observed entries $\mathcal{P}_{\widehat{\Omega}}(\bfM)$.   Since the rank of $\mathcal{H}\bfX$ is $r$, the data recovery problem can be formulated as
\begin{equation}\label{m2}
\min_{{\bfZ}\in \mathbb{C}^{n_c\times n}}\left\| \mathcal{P}_{\widehat{\Omega}}({\bfZ}-{\bfM})\right\|^2_F \ \  \text{subject to}\  \text{rank}(\mathcal{H}\big({\bfZ})\big)=r,
\end{equation}
where   $\mathcal{P}_{\widehat{\Omega}}(\cdot)$ is the sampling operator with
 $(\mathcal{P}_{\widehat{\Omega}}(\bfZ))_{ij}=Z_{ij}$ if $(i,j)\in\widehat{\Omega}$ and $0$ otherwise. (\ref{m2}) is a nonconvex problem due to the rank constraint.  It reduces to the conventional matrix completion problem when $n_1=1$.

Clearly, the recovery is impossible if $\bfX$ is in the null space of $\mathcal{P}_{\widehat{\Omega}}(\cdot)$. Additionally, the definition of $\mu$ follows the standard   incoherence assumption in low-rank matrix completion \cite{CR09}:
\begin{defi}
	A matrix $\bfZ\in\mathbb{C}^{l_1 \times l_2}$ with singular value decomposition {\textrm (SVD)} as ${{\bfZ}}={\bfU}{\bf\Sigma}{\bfV}^H$, is said to be incoherent with parameter $\mu$ if 
	\begin{equation}\label{mu_1}
	\max_{1\le k_1\le l_1}\left\|{\bfe}_{k_1}^H{{\bfU}}\right\|^2\le\frac{\mu r}{l_1}, \  \max_{1\le k_2\le l_2}\left\|{\bfe}_{k_2}^H{{\bfV}}\right\|^2\le\frac{\mu r}{l_2},
	\end{equation}
	where ${\bfe}_{k_1},{\bfe}_{k_2}$ are the coordinate unit vectors.
\end{defi}
The incoherence definition guarantees that the singular vectors of the matrix are sufficiently spread, and  $\mathcal{P}_{\widehat{\Omega}}(\cdot)$ samples enough information about the matrix. The focus is to recover $\mu$-incoherence matrices here.

\section{Background and Related Work}\label{sec: related work}

The low-rank   property of a Hankel matrix is also recently exploited in the direction of arrival (DOA) problem in array signal processing \cite{CC14,YX16}, MRI image recovery from undersampled measurements \cite{Haldar14,OBJ18,YKJL17}, video inpainting \cite{DSC07} and system identification \cite{FPST13}. 
To see the connection with our model, the $k$-th row of $\bfX$ in \eqref{M}, denoted by $\bfX_{k*}$, can   be equivalently viewed as the discrete samples of a spectrally sparse signal $g_k(t)$, which is a weighted sum of $r$  damped or undamped sinusoids at $t=\{0,...,n-1\}$, where 
\vspace{-2mm}
\begin{equation}\label{eqn:model}
\vspace{-2mm}
g_k(t)=\sum_{i=1}^{r}d_{k,i}e^{(2\pi \imath f_i-\tau_i)t}, k=1,..., n_c,
\end{equation} 
and $f_i$  and $d_{k,i}$ 
are the frequency and  the normalized complex amplitude of the $i$th sinusoid, respectively. $\imath$ is the imaginary unit. 
The connection between \eqref{eqn:model} and \eqref{M} is that $\lambda_i=e^{2\pi \imath f_i-\tau_i}$ and $d_{k,i}=\bfr_i^H\bfs_1\bfC_{k*}\bfl_i$.

The signal of interest itself in array signal processing is spectrally sparse. In MRI imaging, 
if a signal reduces to a sparse linear combination of Dirac delta functions under some transformations, then its Fourier transform   is a sum of a few sinusoids \cite{JLY16,OJ16,YKJL17}. Most existing work on low-rank Hankel matrices studied single-channel signals, i.e., $n_c=1$ in our setup. References \cite{CC14,JY18,OBJ18,OJ16} considered 2-dimensional (2-D) and higher-dimensional  
signals, while a 2-D signal is still a sum of $r$ 2-D sinusoids, and the degree of freedom is still $\Theta(r)$. The focus of this chapter is multi-channel signals with $n_c > 1$. Each signal is a weighted sum of the same set of $r$ sinusoids, while the weights $d_{k,i}$ are different for each channel $k=1,...,n_c$. The degree of the freedom of  \eqref{eqn:model} is $\Theta(n_c r)$.

 The multi-channel signal in \eqref{eqn:model} is related to the multiple measurement vector (MMV) problem \cite{CREK05}. References \cite{LC16,YX16} considered data recovery of MMV when the signals are linear combinations of undamped sinusoids, i.e., $\tau_i=0$ for all $i$ in  \eqref{eqn:model}.  The data recovery is achieved in \cite{LC16,YX16} through atomic norm minimization, which requires solving large-scale semidefinite programs. Besides the high computational complexity,   it is not clear how the atomic norm can be extended to handle damped sinusoids, i.e., $\tau_i\neq 0$. 
 References \cite{BMJ17,JLY16} studied multi-channel signal recovery using Hankel structures and can thus handle damped sinusoids.   Despite the numerical evaluations, there is no theoretical analysis of the recovery guarantee in \cite{BMJ17,JLY16}. This chapter provides analytical recovery guarantees for  multi-channel damped and undamped sinusoids.

The recovery of a low-rank Hankel matrix can be formulated as a convex optimization, for example, nuclear norm minimization for missing data recovery \cite{DSC07,FPST13,JLY16,UC16,OBJ18}, \cite{YKJL17} and minimizing a weighted sum of the nuclear norm and the $\ell_1$ norm for bad data correction \cite{JY18}. Since it is computationally challenging to solve these convex problems for high-dimensional Hankel matrices,   fast algorithms to recover missing points in single-channel  \cite{CWW17} and multi-channel Hankel matrices \cite{BOJ16,DSC07} are proposed recently. Although   numerical results are reported in \cite{BOJ16,DSC07}, only \cite{CWW17} provides the theoretical performance analysis of the proposed fast iterative hard thresholding (FIHT) algorithm for single-channel Hankel matrix recovery. FIHT is a projected gradient descent method. In each iteration, the algorithm updates the estimate along the gradient descent direction and then projects it to a rank-$r$ matrix. To reduce the computational complexity, instead of solving singular value decomposition (SVD) directly, FIHT first projects a matrix onto a $2r$-dimensional subspace and then computes the SVD of the rank-$2r$ matrix.  The per-iteration complexity of FIHT is $O(r^2n+rn\log n+r^3)$.     

Motivated by PMU data analysis in power systems, this chapter connects   dynamical systems with low-rank Hankel matrices. It develops fast data recovery algorithms for multi-channel Hankel matrices with provable performance guarantees.

%% file: chapter_1/algorithm_publication_ready.tex
\vspace{-2mm}
\section{Data Recovery Algorithms}\label{sec:block}
Here, two algorithms are described to solve (\ref{m2}), and the theoretical analyses is defered to Section \ref{sec:analyses}. One is accelerated multi-channel fast iterative hard thresholding algorithm (AM-FIHT), and the other one is robust AM-FIHT (RAM-FIHT).
Both algorithms are    built upon the FIHT \cite{CWW17} with some major differences. 
First, FIHT recovers the missing points of one spectrally sparse signal, while (R)AM-FIHT  recovers the missing points of $n_c$ signals simultaneously. The simultaneous recovery can reduce the required number of measurements, as quantified in Theorem \ref{coherent}. Second, (R)AM-FIHT has a heavy-ball step \cite{P87,LESB16}, e.g., term $\beta({\bfW}_{l-1}-{\bfW}_{l-2})$ in line 5 of Algorithm \ref{HB} and line 14 of Algorithm \ref{FIHT2}, while FIHT does not. The basic idea of the heavy ball method is to compute the search direction using a linear combination of the gradient at the current iterate and the update direction in the previous step, rather than being memoryless of the past iterates' trajectory \cite{LESB16}. With the heavy-ball step, AM-FIHT is proved to converge faster while maintaining the recovery accuracy (Theorem \ref{t4}). Third, the theoretical guarantee of data recovery when the measurements are noisy is presented (Theorem \ref{t3}), while   \cite{CWW17} only has the performance guarantee of FIHT using noiseless measurements. 

{ In both algorithms, $\bfM$, $\bfX_l$, $\bfG_l\in\mathbb{C}^{n_c\times n}$, and $\bfW_{l}$, $\Delta\bfW_{l}$,  $\bfL_{l}\in\mathbb{C}^{n_cn_1\times n_2}$.  
	 $\bfL_l$ is a rank-$r$ matrix and its SVD is denoted as $\bfL_l=\bfU_l\boldsymbol{\Sigma}_l\bfV_l^*$, where ${\bfU}_l\in \mathbb{C}^{n_cn_1\times r}$, ${\bfV}_l\in \mathbb{C}^{n_2\times r}$ and $\boldsymbol{\Sigma_l}\in\mathbb{C}^{r\times r}$.
	$\mathcal{S}_l$ is the tangent subspace of the rank-$r$ Riemannian manifold at ${\bfL}_l$, and for any matrix ${\bfZ}\in \mathbb{C}^{n_cn_1 \times n_2}$, the projection of ${\bfZ}$ onto $\mathcal{{\bfS}}_l$ is defined as
	\begin{equation}\label{eqn: projection of matrix}
	\mathcal{P}_{{\bf\mathcal{S}}_l}({\bfZ})={\bfU}_l{\bfU}_l^*{\bfZ}+{\bfZ}{\bfV}_l{\bfV}_l^*-{\bfU}_l{\bfU}_l^*{\bfZ}{\bfV}_l{\bfV}_l^*.
	\end{equation}
	$\mathcal{Q}_r$ finds the best rank-$r$ approximation as
	\begin{equation}
	\mathcal{Q}_r({\bfZ})=\sum_{i=1}^{r}\sigma_i {\bfu}_i{\bfv}_i^* ,
	\end{equation} if ${\bfZ}=\sum_{i}^{}\sigma_i {\bfu}_i{\bfv}_i^*$ is the SVD of $\bfZ$ with $\sigma_1\ge\sigma_2\ge\cdots$.  $\mathcal{H}^{\dagger}$ is the Moore-Penrose pseudoinverse of ${\mathcal{H}}$. For any matrix $\bfZ\in \mathbb{C}^{n_cn_1\times n_2}$, $(\mathcal{H}^{\dagger}\bfZ)\in\mathbb{C}^{n_c\times n}$ satisfies
	\begin{equation}\label{w}
	\langle\mathcal{H}^{\dagger}\bfZ,\bfe_k \bfe_t^*\rangle=\frac{1}{w_t}\sum_{k_1+k_2=t+1}Z_{(k_1-1)n_c+k,k_2}, 
	\end{equation}
	where $w_t=\# \{(k_1,k_2)|k_1+k_2=t+1,1\le k_1\le n_1, 1\le k_2\le n_2 \}$ as the number of elements in the $t$-th anti-diagonal of an $n_1\times n_2$ matrix. }

\begin{algorithm}[h!]
	\caption{AM-FIHT for Data Recovery from Noiseless Measurements}\label{HB}
	\begin{algorithmic}[1]
 	\Require  $\mathcal{P}_{\widehat{\Omega}}(\bfM)$, $n_1$, $n_2$, $r$
		\State Set $\bfW_{-2}=\bf0$, $\bfW_{-1}=p^{-1}{\mathcal{H}}{\mathcal{P}}_{\widehat{\Omega}}({\bfM})$, ${{\bfL}}_0=\mathcal{Q}_r(\bfW_{-1})$;
		\State Initialize ${\bfX}_0=\mathcal{H}^{\dagger}{\bfL}_0$;
		\For{$l=0,1,\cdots$} 
		\State ${\bfG}_l=\mathcal{P}_{\widehat{\Omega}}(\bfM-{\bfX}_l)$;
		\State ${\bfW}_l=\mathcal{P}_{{\bf\mathcal{S}}_l}\left(\mathcal{H}({\bfX}_l+p^{-1}{\bfG}_l)+\beta({\bfW}_{l-1}-{\bfW}_{l-2})\right)$;
		\State ${\bfL}_{l+1}=\mathcal{Q}_r({\bfW}_l)$;
		\State ${\bfX}_{l+1}=\mathcal{H}^{\dagger}{\bfL}_{l+1}$;
		\EndFor
		\State \Return $\bfX_{l}$ 
	\end{algorithmic}
\end{algorithm}

The key steps in AM-FIHT are as follows. 
Here the measurements are  noiseless, thus $\bfM=\bfX$. 
 In each iteration, the current ${\bfX}_l$ is updated along the gradient descent direction ${\bfG}_l$, with a step size $p^{-1}=\frac{n_cn}{m}$, where $m$ is the number of observed entries. To improve the convergence rate, the update is further combined with an additional heavy-ball term $\beta({\bfW}_{l-1}-{\bfW}_{l-2})$, which represents the update direction in the previous iteration. 
    Next, $\mathcal{H}({\bfX}_l+p^{-1}{\bfG}_l)+\beta({\bfW}_{l-1}-{\bfW}_{l-2})$
is projected to a rank-$r$ matrix. To reduce the computational complexity, the iteration is first projected to the $2r$-dimensional space $\mathcal{S}_l$ and then apply SVD on the rank-$2r$ matrix \cite{CWW17}, instead of directly computing its SVD. 
The rank-$r$ matrix ${\bfL}_{l+1}$ is obtained in line 7 by thresholding the singular values of the rank-$2r$ matrix $\bfW_l$.
Finally,  ${\bfX}_{l+1}$ is updated by  $\mathcal{H}^{\dagger}{\bfL}_{l+1}$.

 The analysis of the computational complexity of AM-FIHT is similar to that of FIHT \cite{CWW17} with some modifications for the $n_c$-channel signal and the heavy-ball step. 
 Details are in the supplementary materials. The computational complexity of solving SVD of a matrix in $\mathbb{C}^{n_cn_1\times n_2}$  is generally $O(n_cn^2r)$. Due to the low rank structure of the matrices in $\mathcal{S}_l$,   the SVD of $\bfW_l\in \mathbb{C}^{n_cn_1\times n_2}$ can be computed in $O(n_cnr^2+r^3)$ via QR decompositions and SVD on a $2r\times 2r$ matrix \cite{CWW17}. 
 Moreover, it is not necessary to construct Hankel matrices following \eqref{eqn:Hankel} explicitly. The matrix multiplication of $\bfU_l^*\mathcal{H}\bfX_l\in \mathbb{C}^{r\times n_2}$ and $(\mathcal{H}\bfX_l)\bfV_l\in \mathbb{C}^{n_cn_1\times r}$ in line 5 can be completed via fast convolution algorithms with $O(n_cnr\log(n))$ flops, instead of the conventional complexity of $O(n_cn^2r)$.  
 Similar analysis can be applied to line 7, which costs $O(n_cnr\log(n))$ flops to compute $\bfX_{l+1}$ from the SVD of $\bfL_{l+1}$ directly. 
 
 With the heavy ball term, since the SVDs of $\bfW_{l-1}$ and $\bfW_{l-2}$ have been obtained in the last two steps, $\mathcal{P}_{\mathcal{S}_l}(\bfW_{l-1})-\mathcal{P}_{\mathcal{S}_l}(\bfW_{l-2})$ can be calculated as shown in line 5. From \eqref{eqn: projection of matrix}, the computation of ${\bfU}_l{\bfU}_l^*{\bfZ}{\bfV}_l{\bfV}_l^*$ plays the dominant part in computing $\mathcal{P}_{\mathcal{S}_l}(\bfZ)$. Let  $\bfW_{l}=\bfU_{\bfW_l}\mathbf{\Sigma}_{\bfW_l}\bfV_{\bfW_l}^*$ denote the SVD of $\bfW_l$, where $\bfU_{\bfW_l}\in \mathbb{C}^{n_cn_1\times 2r}, \bfV_{\bfW_l}\in\mathbb{C}^{n_2\times 2r}$. Then  computing $\bfU_l^*\bfU_{\bfW_{l-1}}$ and $\bfV_{\bfW_{l-1}}^*\bfV_l$ requires $O(n_cnr^2)$ and $O(nr^2)$ flops, respectively. Computing ${\bfU}_l^*\bfU_{\bfW_{l-1}}\mathbf{\Sigma}_{\bfW_{l-1}}\bfV_{\bfW_{l-1}}^*{\bfV}_l$ further requires $O(r^3)$ flops.
  
From the above analysis, line 4 requires $O(n_cn)$ flops. The complexity of line 5 is $O(n_cnr\log(n)+n_cnr^2+r^3)$. Line 6 requires $O(n_cnr^2+r^3)$ flops, and line 7 requires $O(n_cnr\log(n))$ flops.	Thus, the total per-iteration complexity of AM-FIHT is $O(r^2n_cn+rn_cn\log n +r^3)$.

\begin{algorithm}[h]
	\caption{RAM-FIHT}\label{FIHT2}
	\begin{algorithmic}[1]
	 \Require  $\mathcal{P}_{\widehat{\Omega}}(\bfM)$, $n_1$, $n_2$, $r$, $\mu$, and $\beta$.
		\State  Partition $\widehat{\Omega}$ into $L+1$ disjoint sets $\widehat{\Omega}_0, \widehat{\Omega}_1, \cdots, \widehat{\Omega}_L$ of equal size $\widehat{m}$, let $\widehat{p}=\frac{\widehat{m}}{n}$.
		\State Set
		$\bfW_{-2}=\bf0$, $\bfW_{-1}=\widehat{p}^{-1}{\mathcal{H}}{\mathcal{P}}_{\widehat{\Omega}_0}({\bfM})$, ${{\bfL}}_0=\mathcal{Q}_r(\bfW_{-1})$;
		\For{$l=0,1,\cdots$, $L-1$};
		\State $[{{\bfU}}_{l},{{\bf\Sigma}}_{l},{{\bfV}}_{l}]=$SVD$({{\bfL}}_{l})$;
		\For{$i=1,2,\cdots,n_cn_1$}
		\State $({\bfA}_{l})_{i*}=\displaystyle\frac{({\bfU}_{l})_{i*}}{\left\| ({\bfU}_{l})_{i*}\right\|}\min\bigg\{\left\|({\bfU}_{l})_{i*} \right\|, \sqrt{\frac{\mu r}{n_cn_1}} \bigg\}$;
		\EndFor
		\For{$i=1,2,\cdots,n_2$}
		\State $({\bfB}_{l})_{i*}=\displaystyle\frac{({\bfV}_{l})_{i*}}{\left\|({\bfV}_{l})_{i*}\right\|}\min\bigg\{\left\|({{\bfV}}_{l})_{i*} \right\|, \sqrt{\frac{\mu r}{n_2}} \bigg\}$;
		\EndFor
		\State ${{{\bfL}}_{l}^{\prime}}={{\bfA}}_{l}{{\bf\Sigma}}_{l}{{\bfB}}_{l}^*$;
		\State $\widehat{{\bfX}}_{l}={\mathcal{H}}^{\dagger}{{{\bfL}}_{l}^{\prime}}$;
		\State ${\bfG}_l=\mathcal{P}_{\widehat{\Omega}_{l+1}}({\bfM}-\widehat{\bfX}_l)$;
		\State ${\bfW}_l=\mathcal{P}_{\mathcal{S}^{\prime}_l}(\mathcal{H}(\widehat{{\mathbf{X}}}_l+\widehat{p}^{-1}{\bfG}_l)+\beta({\bfW}_{l-1}-{\bfW}_{l-2}))$;
		\State ${\bfL}_{l+1}=\mathcal{Q}_r({\bfW}_l)$;
		\EndFor
		\State \Return ${\bfX}_{L}=\mathcal{H}^{\dagger}{\bfL}_{L}$;
	\end{algorithmic}
\end{algorithm}


RAM-FIHT differs from AM-FIHT mainly in resampling (line 1) and trimming (lines 5-10). The resampling and trimming  are used in \cite{CWW17} to improve the initialization of FIHT. Here, these ideas are applied in the data recovery algorithm and are proved in Theorem \ref{t3} that the resulting RAM-FIHT can recover the matrix even when the observed measurements are noisy. There is no analytical analysis of FIHT on noisy measurements in \cite{CWW17}. Moreover, compared with AM-FIHT,  a tighter bound of the required observations amount for RAM-FIHT is provided 
(comparing Theorems \ref{t1} and   \ref{t2}).  

In RAM-FIHT, the sampling set  $\widehat{\Omega}$ is divided into $L$  disjoint subsets $\widehat{\Omega}_l$'s. During the $l$-th iteration, ${\bfL}_l$ is updated using the observed entries in $\widehat{\Omega}_l$,    instead of using all the entries in $\widehat{\Omega}$ as in AM-FIHT. 
The partition of the sampling set is a standard technique in analyzing matrix completion (MC) problems \cite{R11}. The disjointness of ${\bfL}_l$'s in different iterations simplifies the theoretical analyses\footnote{In fact,   the necessary condition is 
the mutual independence among the subsets $\widehat{\Omega}_l$'s, and the disjoint partition is a sufficient condition.}, since it ensures the independence between $\bfX_l$ and $\bfX_{l+1}$.
 The trimming procedure ensures that the estimate in each iteration remains close to $\mu$-incoherent, which in turn helps  to obtain tighter bounds of the recovery performance in Theorem \ref{t2}.  
 The resampling and trimming steps in RAM-FIHT are introduced mainly to simplify the theoretical analyses and obtain tighter bounds, while numerical observation suggests that AM-FIHT and RAM-FIHT perform similarly in Section \ref{sec:simu}. The per iteration computational complexity of RAM-FIHT is $O(r^2n_cn+rn_cn\log n +r^3)$.

%% file: chapter_1/theoretical_publication_ready.tex
\section{Theoretical Analyses}\label{sec:analyses}

The theoretical analyses of the convergence rates and recovery accuracy of AM-FIHT and RAM-FIHT are summarized in the following four theorems. All the proofs are deferred to the Appendix. Theorem \ref{t1}  records the recovery performance of AM-FIHT using noiseless measurements with $\beta=0$. Theorem \ref{t4} shows that the convergence rate  of AM-FIHT  can be further improved by using a small positive $\beta$.    Theorems \ref{t2} and \ref{t3} discuss  the recovery performance of RAM-FIHT from noiseless and noisy measurements, respectively.   
Additionally, the comparison with the recovery performance via recovering missing points on each individual row of $\bfX$ separately are provided to quantify the performance gain of our algorithms in Theorem \ref{coherent}. 

\begin{theorem}(AM-FIHT with noiseless measurements.)\label{t1}
	Assume ${\mathcal{H}}{\bfX}$ is $\mu$-incoherent. Let $0<\varepsilon_0<\frac{1}{10}$ be a numerical constant and $\nu=6\varepsilon_0<1$. Then with probability at least $1-3n_cn^{-2}$, the iterates $\bfX_l$'s generated by AM-FIHT with $\beta=0$ satisfy
	\vspace{-1mm} 
	\begin{equation}\label{AM-FIHT linear convergent rate}
	\vspace{-1mm} 
	 \left\|{\bfX}_l-{\bfX}\right\|_F\le{\nu^{l-1}}\left\|{{\bfL}}_0-{\mathcal{H}}{\bfX}\right\|_F, 
	\end{equation}
	provided that
	\vspace{-1mm} 
			\begin{equation}\label{eqn:m1}
	 		m\ge   C_1\max\left\{\frac{\mu c_sr\log(n)}{\varepsilon_0^{2}}, 
		 \frac{1+\varepsilon_0}{\varepsilon_0}(n_c\mu c_s n)^{\frac{1}{2}}\kappa r \log^{\frac{3}{2}}(n)\right\}
 		\end{equation}
	for some constant $C_1>0$, where $\kappa=\frac{\sigma_{\rm{max}}({\mathcal{H}}{\bfX})}{\sigma_{\rm{min}}({\mathcal{H}}{\bfX})}$ denotes the condition number of ${\mathcal{H}}{\bfX}$ and $c_s=\max\{\frac{n}{n_1},\frac{n}{n_2}\}$.
\end{theorem}

Theorem \ref{t1} indicates that if  the number of noiseless observations is $O(rn_c^{1/2}n^{1/2}\log^{3/2}(n))$, then AM-FIHT  is    guaranteed to recover $\bfX$ exactly. Moreover, from \eqref{AM-FIHT linear convergent rate}, the iterates generated by  AM-FIHT converge linearly to the groundtruth $\bfX$, and the rate of convergence is $\nu$.
Since $\bfX$ is rank $r$, if one directly applies a conventional low-rank matrix completion method  such as NNM (\cite{CT10,CR09,Gross11}), the required number of observations is $O(rn\log^2(n))$. Thus, when $n$ is large, by exploiting the low-rank Hankel structure of correlated time series, the required number of measurements is   significantly reduced.
Note that the degree of freedom of $\bfX$ is $\Theta(n_cr)$, as one can see from \eqref{eqn:model}, 
the required number of observations by Theorem \ref{t1} is suboptimal due to the dependence upon $n$. This results from the artefacts in our  proof techniques. A tighter bound is provided in   Theorem \ref{t2} for RAM-FIHT. 

The required number of measurements depends on $c_s$, which is minimized when   $n_1=n_2=\frac{n+1}{2}$.
In practice, the selection of $n_1$ and $n_2$ of the Hankel matrix is also affected by the accuracy of the low-rank approximation.

$\beta$ is set as $0$ in Theorem 1 to simplify the analyses.  The improvement of the convergence rate by using a positive $\beta$ is quantified in the following theorem. 
\begin{theorem}(Faster convergence with a heavy-ball step)\label{t4}
	Given any $\beta \in [0, \tau)$ for some $\tau>0$, let $\bfX_l$'s denote the convergent iterates returned by AM-FIHT.  There exists an integer $s_0$, a   constant $q \in (0,1)$ that depends on $\beta$ such that 
	\begin{equation}\label{eqn:heavy}
	\|\bfX_{s_0+k}-\bfX\|_F\le 
	c(\delta)(q(\beta)+\delta)^k, \ \forall k \geq 0
	\end{equation}
holds	for any  $\delta \in (0, 1-q(\beta))$ and  a positive $c(\delta)$ that depends on $\delta$. Moreover,
\vspace{-2mm} 
\begin{equation}\label{eqn:q}
\vspace{-1mm}
q(0)>q(\beta),\ \forall \beta \in (0, \tau).
\end{equation}
\end{theorem}
The exact expressions of $q$ and $\tau$ are deferred to the proofs in Appendix (equation \eqref{eqn: expression of q}). 
Theorem \ref{t4} indicates that by adding a heavy-ball term, when close enough to the ground-truth $\bfX$,  the iterates converge linearly to $\bfX$, and the rate of convergence is $q(\beta)+\delta$. Moreover, from \eqref{eqn:q}, with a small positive $\beta$, the iterates converge faster than those without the heavy-ball step.  Such improvement  is  numerically evaluated in Section \ref{sec:simu}.

\begin{theorem}(RAM-FIHT with noiseless measurements)\label{t2} 
	Assume ${\mathcal{H}}{\bfX}$ is $\mu$-incoherent. Let $0<\varepsilon_0<\frac{1}{2}$ and 
	\begin{equation}\label{eqn:L}
 L=\Big\lceil \varepsilon_0^{-1}\log\Big(\frac{\sigma_{\max}({\mathcal{H}}{\bfX})}{128\kappa^3\varepsilon}\Big)\Big\rceil.
	\end{equation}
	 Define $\nu=2\varepsilon_0<1$. Then with probability at least $1-(2L+3)n_cn^{-2}$,  for any arbitrarily small constant $\varepsilon>0$, the iterates 
	 $\bfL_l$'s and $\bfX_L$ generated by RAM-FIHT with $\beta=0$ satisfy
	 \begin{equation*}
	 \begin{split}
	 \left\|{\bfL}_l-{\bfX}\right\|_F& \le{\nu^{l}}\left\|{{\bfL}}_0-{\mathcal{H}}{\bfX}\right\|_F, 1\leq l\leq L,\\
\text{and }	 \left\|{\bfX}_L-{\bfX}\right\|_F& \le{\nu^{L}}\left\|{{\bfL}}_0-{\mathcal{H}}{\bfX}\right\|_F\le\varepsilon, 
	 \end{split}
	 \end{equation*}	 
	provided that
	\begin{equation}\label{eqn:m2}
	m\ge C_2\varepsilon_0^{-3}\mu c_s \kappa^6r^2\log(n)\log \Big(\frac{\sigma_{\max}({\mathcal{H}}{\bfX})}{\kappa^3\varepsilon}\Big)
	\end{equation}
	for some constant $C_2>0$.
\end{theorem}

Theorem \ref{t2} shows that the iterates of  RAM-FIHT converge to the ground truth $\bfX$ with a linear convergence rate, and the number of required measurements is further reduced from that needed by AM-FIHT. 
To see this, since $\sigma_{\max}({\mathcal{H}}{\bfX})\leq \sqrt{n_cn}\|\bfX\|_\infty$. If $\|\bfX\|_\infty$ is a constant, selecting $\varepsilon=O(n^{-\alpha})$ with a positive constant $\alpha$ yields $L=O(\log(n))$ from \eqref{eqn:L} and $m \geq O(r^2\log^2n)$ from \eqref{eqn:m2}. 
Compared with the bound of $O(rn_c^{1/2}n^{1/2}\log^{3/2}(n))$ in Theorem \ref{t1}, the dependence on $n$ is significantly reduced to $\log^2n$, while the dependence on $r$ is worse, from $r$ to $r^2$. Since $r$ is usually very small, and $n$ is much larger, $O(r^2\log^2n)$ by Theorem \ref{t2} is tighter than $O(rn_c^{1/2}n^{1/2}\log^{3/2}(n))$ by Theorem \ref{t1}. Since the degree of freedom of $\Theta(n_cr)$, the bound could be further improved using better proof techniques than ours.

\begin{theorem}(RAM-FIHT with noisy measurements)\label{t3} 
	Assume ${\mathcal{H}}{\bfX}$ is $\mu$-incoherent and 
	\begin{equation}\label{eqn:N}
	\left\|\bfN \right\|_{\infty}\le
	\frac{\varepsilon_0\left\|\mathcal{H} \bfX\right\|}{2048\kappa^3r^{1/2}n_c^{1/2}n}.
	\end{equation} Let $L=\Big\lceil \varepsilon_0^{-1}\log\big(\frac{\sigma_{\max}({\mathcal{H}}{\bfX})}{128\kappa^3\varepsilon}\big)\Big\rceil$ and $0<\varepsilon_0<\frac{1}{4}$. Define $\nu=2\varepsilon_0<\frac{1}{2}$. Then with probability at least $1-(3L+3)n_cn^{-2}$ and for any arbitrarily small constant $\varepsilon>0$, the iterates $\bfL_l$'s ($l=1,...,L$) generated by RAM-FIHT with $\beta=0$ satisfies
	\vspace{-2mm}
	\begin{equation*}
	\vspace{-2mm}
		\begin{split}
		\|{\bfL}_l-\mathcal{H}{\bfX}\|_F \le&{\nu^{l}}\left\|{{\bfL}}_0-{\mathcal{H}}{\bfX}\right\|_F
		+128n_c^{1/2}n\left\|\bfN \right\|_{\infty}+8r^{1/2}\|{\mathcal{H}}\bfN\|,
		\end{split}	
	\end{equation*}
	\begin{equation}\label{eqn:noise}
\textrm{and } \left\|{\bfL}_L-\mathcal{H}{\bfX}\right\|_F
	\le\varepsilon+128n_c^{1/2}n\left\|\bfN \right\|_{\infty}+8r^{1/2}\left\|{\mathcal{H}}\bfN \right\|,
	\end{equation}
	provided that
	\vspace{-3mm}
	\begin{equation}\label{eqn:m3}
	\vspace{-3mm}
	m\ge C_3\varepsilon_0^{-3}\mu c_s \kappa^6r^3\log(n)\log \Big(\frac{\sigma_{\max}({\mathcal{H}}{\bfX})}{\kappa^3\varepsilon}\Big)
	\end{equation}
	for some constant $C_3>0$.
\end{theorem}

Theorem \ref{t3} explores the performance of RAM-FIHT in the noisy case. Note that $\left\|\mathcal{H}\bfX \right\|_{\infty}=\left\|\bfX \right\|_{\infty}$, and 
 \[\frac{n_c^{1/2}n}{\mu c_sr}\left\|\mathcal{H}\bfX \right\|_{\infty}\le\left\|\mathcal{H}\bfX \right\|\le n_c^{1/2}n\left\|\mathcal{H}\bfX \right\|_{\infty}.\]
 If $\mu$ and $r$ are both constants, \eqref{eqn:N} implies that $\|\bfN\|_\infty$ can be as large as a constant fraction of $\left\|\bfX \right\|_{\infty}$.
When the number of observations is at least $O(r^3\log^2(n))$,
the error between the ground truth and the iterates returned by RAM-FIHT is controlled by the noise level. To evaluate the optimality of this error bound, 
consider a special case that $\bfX$ is a constant matrix with each entry being $c$, and $\bfN$ is a constant matrix with each entry being $-c$. Then every observation is zero, and the estimated matrix from partial observations by any recovery method would be a zero matrix. Then the recovery error is $\left\|\mathcal{H}\bfN\right\|_F=\sqrt{n_cn_1n_2}|c|=\sqrt{n_cn_1n_2}\left\| \bfN\right\|_{\infty}$. The sum of the second and the third term in the right hand side of \eqref{eqn:noise} is bounded by $(128c_s+8r^{1/2})\sqrt{n_cn_1n_2}\left\| \bfN\right\|_{\infty}$. Thus, the error bound of RAM-FIHT is in the same order of the minimum error by any method.

{\bf Comparison with single-channel missing data recovery}. FIHT \cite{CWW17} is a single-channel Hankel matrix completion method. When $n_c=1$, Theorems \ref{t1} and \ref{t2} reduce to the results in \cite{CWW17}. One can apply FIHT to recover the missing points of each row of $\bfX$  and solve $n_c$ data recovery problems separately.
Let $\mathcal{H}\bfX_{k*}$ denote the single-channel Hankel matrix constructed from the $k$th row of $\bfX$. Suppose $\mathcal{H}\bfX_{k*}$  is $\mu_0$-incoherent for every $1\le k\le n_c$. Then, setting $n_c=1$ in Theorems \ref{t1} and \ref{t2} (or using Theorems 1 and 2 in \cite{CWW17}), if each $\mathcal{H}\bfX_{k*}$ is recovered separately, the required number of measurements   is proportional to $\sqrt{\mu_0}$ (AM-FIHT) or $\mu_0$  (RAM-FIHT). Then the total number of observations to recover $\bfX$ is proportional to $n_c\sqrt{\mu_0}$ or $n_c\mu_0$. In contrast, the required number of observations by our methods is proportional to $\sqrt{n_c\mu}$ (AM-FIHT) or $\mu$ (RAM-FIHT). Thus, the ratio of the number of measurements by our method to FIHT is 
$\sqrt{\frac{\mu}{n_c\mu_0}}$ (or $\frac{\mu}{n_c\mu_0}$). 
To this end, our method only requires a constant   fraction of the measurements by using FIHT through the following theorem. 

 \begin{theorem}\label{coherent}
 \begin{equation}\label{eqn:ratio}
\frac{\mu}{n_c\mu_0}<1. 
 \end{equation}
 If it further holds that $(1-\delta)|\hat{d}|\le |d_{k,i}|\le (1+\delta)|\hat{d}|, \forall k \in \{1,...,n_c\}, i \in\{1,...,r\}$ for some $\delta\in(0,1)$ and $\hat{d}\in \mathbb{C}$, where $d_{k,i}=\bfr_i^*\bfs_1\bfC_{k*}\bfl_i$, we have
  	\begin{equation}\label{eqn:ratio_improvement}
 	\frac{\mu}{n_c\mu_0}\le\frac{1}{1+(n_c-1)\frac{(1-\delta)^2}{\kappa_L^2(1+\delta)^2}},
 	\end{equation}
 	where $\kappa_L$ is the conditional number of $\bfP_L$ when $n_c=1$.
 \end{theorem}

Theorem \ref{coherent}  indicates that the required number of measurements is reduced when collectively processing $\bfX$. 
 Note that $\mu_0$ is independent of the amplitude parameters $d_{k,i}$'s and depends only on the separations of the frequencies $f_i$'s in \eqref{eqn:model}. As a direct corollary of Theorem 2 in \cite{LF16},  if the separation among frequencies $f_i$'s is at least   $1/c_sn$, then $\mu_0$ is a constant. In contrast, $\mu$  depends on both $d_{k,i}$'s and $f_i$'s. \eqref{eqn:ratio} shows that  $\mu$ is always less than $n_c\mu_0$. 
 Moreover, in the special case that $d_{k,i}$'s are all in a small range, $\mu/(n_c\mu_0)$ can be reduced to approximately $\kappa_L^2/n_c$ from \eqref{eqn:ratio_improvement}. With well separated frequencies, the maximum and minimum singular values of $\bfP_L$ when $n_c=1$ are both proportional to $\sqrt{n_1}$ \cite{LF16}. 
 That implies  $\kappa_L$ is a constant. Then, $\kappa_L^2/n_c$ is in the order of $1/n_c$ for large $n_c$, and $\mu/\mu_0=O(1/n_c)$ holds. Combining these results with the arguments before Theorem 3, one can see that the required number of measurements is significantly reduced by collective processing.

%% file: chapter_1/numerical_publication_ready.tex
\section{Numerical Results}\label{sec:simu}
The numerical performances of AM-FIHT and RAM-FIHT are summarized in this section. The simulations are implemented in MATLAB on a desktop with 3.4 GHz Intel Core i7 and 16 GB memory.  In all the experiments, some data points are deleted in the datasets and test the recovery performance.  Three modes of missing data patterns are considered, as illustrated in Fig.~\ref{fig: missing data modes}. Given a data loss percentage, 
\begin{itemize} 
	\item Mode 1: Data losses occur randomly and independently across time and channels.
	\item Mode 2: At randomly selected time instants,  the data points in all channels are lost simultaneously. 
	\item Mode 3: Starting from a randomly selected time instant, in half of the channels that are randomly selected, the data points are lost simultaneously and consecutively lost.  	
\end{itemize}

\begin{figure}[h]
	\centering
	\includegraphics[width=0.6\textwidth]{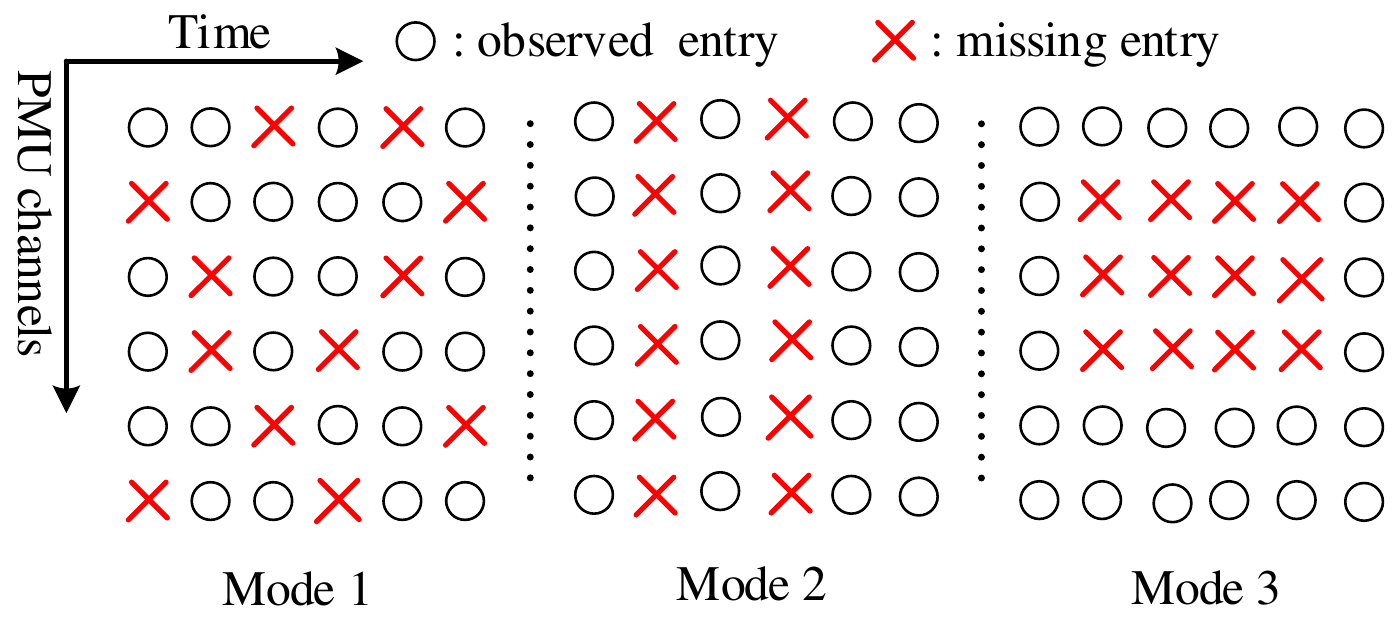}
	\centering
	\caption{Three modes of missing data} 
	\label{fig: missing data modes}
\end{figure}

\subsection{Numerical Experiments on Synthetic Data}
The synthetic data experiments are conducted with spectrally sparse signals. Each row of the matrix $\bfX\in\mathbb{C}^{n_c\times n}$ is a weighted sum of $r$ sinusoids as shown in \eqref{eqn:model}.  Each $f_i$ is randomly selected from $(0, 1)$. $\tau_i$ is $0$ for all $i$. The complex coefficient $d_{k,i}$ has its angle randomly selected from $(0, 2\pi)$  and its magnitude  chosen as $1+10^{a_{k,i}}$, where $a_{k,i}$ is randomly selected from $(0, 1)$.

\subsubsection{(R)AM-FIHT with Noiseless Measurements} 
The performance of  AM-FIHT and RAM-FIHT with noiseless measurements are compared first. For RAM-FIHT, instead of dividing the observation set into disjoint subsets, the entire observation set is used in every iteration.  Hence, RAM-FIHT differs from AM-FIHT in the trimming step, {and the thresholding is set as the ground truth $\mu$ throughout this section}. AM-FIHT is tested with both $\beta=0$ and $\beta=(1-p)^2/5$, while only $\beta=(1-p)^2/5$ is tested on RAM-FIHT.  An algorithm terminates if 
\begin{equation}\label{eqn:termination}
	\|\mathcal{P}_{\widehat{\Omega}}(\bfX_l-\bfX_{l-1})\|_F/\|\mathcal{P}_{\widehat{\Omega}}(\bfX_{l-1})\|_F\leq 10^{-6}
\end{equation}
is satisfied before reaching  the maximum iteration number, which is set as 300 here. 

Figs.~\ref{fig: comparison of RAM and AM in mode 1}, \ref{fig: comparison of RAM and AM in mode 2}, and \ref{fig: comparison of RAM and AM in mode 3} show the recovery phase transitions of AM-FIHT and RAM-FIHT with missing data patterns following different modes. $n=300$, $n_1=150$, and $n_c=30$. The $x$-axis is the fraction of observations $p=\frac{m}{n_cn}$.  The $y$-axis is the rank $r$. For each fixed $p$ and $r$, $100$ independent realizations of synthetic data matrices and data erasures are generated. 
The recovery is said to be successful in a test case if 
\begin{equation}\label{ratio}
\|\mathcal{P}_{\widehat{\Omega}^c}(\bfX_l-\bfX)\|_F/\|\mathcal{P}_{\widehat{\Omega}^c}(\bfX)\|_F<10^{-3}
\end{equation}
holds when the algorithm terminates after $l$-th iteration, and
$\widehat{\Omega}^c$ is the complement of $\widehat{\Omega}$.    A white block corresponds to $100\%$ success, and a black one means failures in all $100$ tests.
\begin{figure}[h!]
	\centering
	\includegraphics[width=0.8\textwidth]{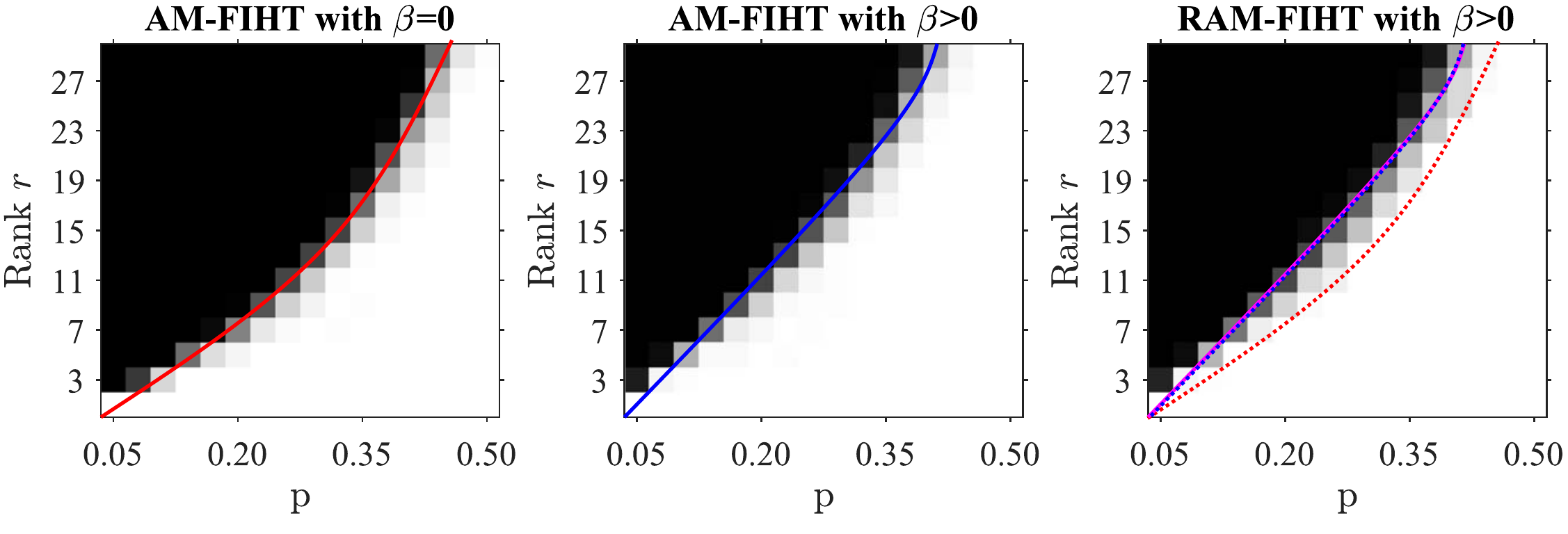}	
	\caption{Phase transition under mode 1} 
	\label{fig: comparison of RAM and AM in mode 1}
\end{figure}
\begin{figure}[h!]
	\centering
	\includegraphics[width=0.80\textwidth]{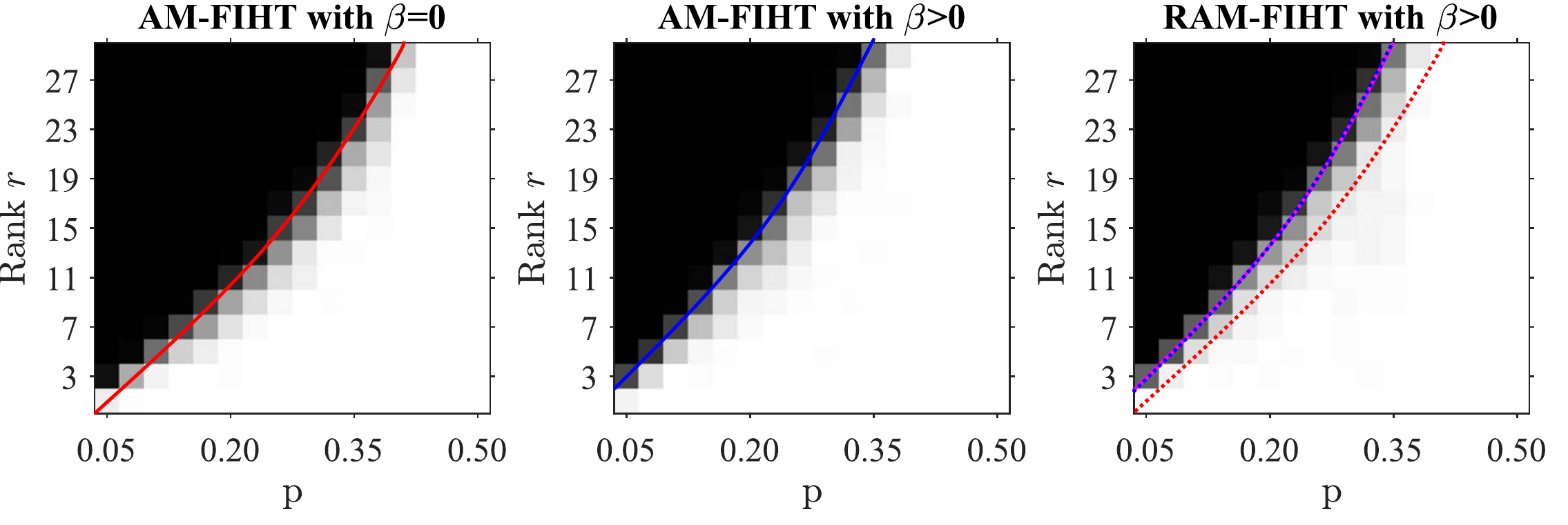}
	
	\caption{Phase transition under mode 2} 
	\label{fig: comparison of RAM and AM in mode 2}
\end{figure}
\begin{figure}[h!]
	\centering
	\includegraphics[width=0.8\textwidth]{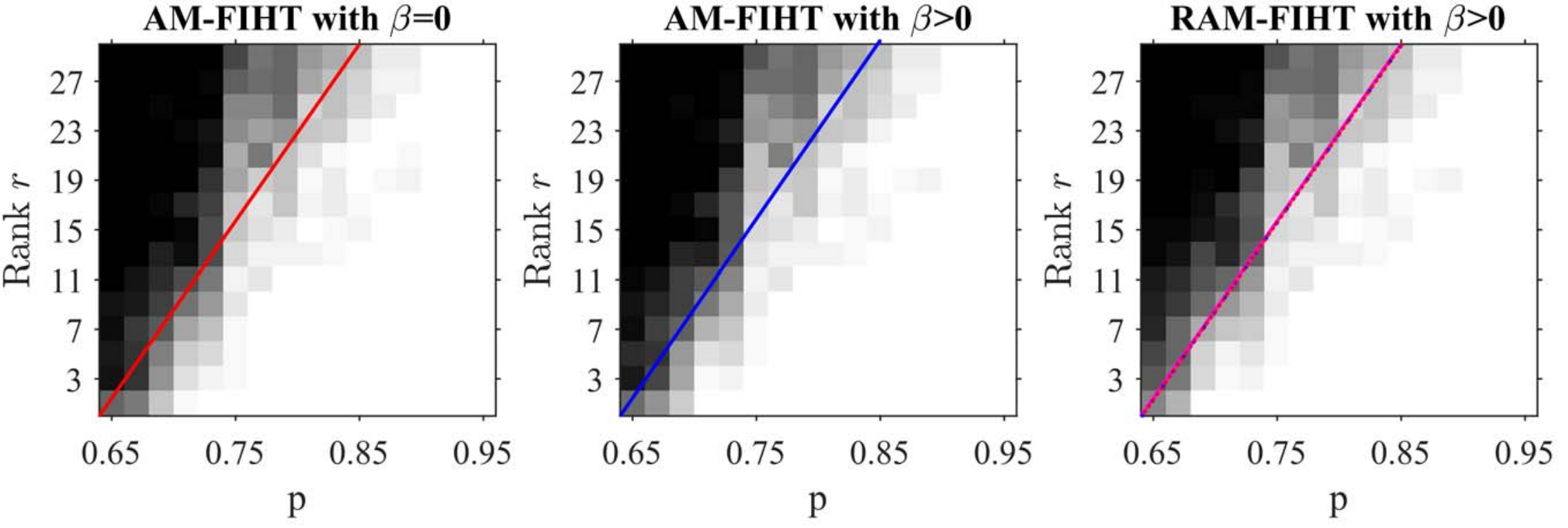}
	\caption{Phase transition under mode 3} 
	\label{fig: comparison of RAM and AM in mode 3}
\end{figure}

 {Auxiliary solid lines (red for AM-FIHT with $\beta=0$, blue for AM-FIHT with $\beta>0$, and magenta for RAM-FIHT) are added in Figs.~\ref{fig: comparison of RAM and AM in mode 1}, \ref{fig: comparison of RAM and AM in mode 2}, and \ref{fig: comparison of RAM and AM in mode 3} to highlight the phase transition. In the subfigures for RAM-FIHT, the phase transition curves for AM-FIHT are repeated in dotted curves to compare. 
 Both AM-FIHT and RAM-FIHT with $\beta>0$ preform very similarly, as the blue dotted line and the magenta solid line coincide in all three modes. }

{The phase transition threshold of $\beta>0$ is higher than that of $\beta=0$ for all the modes.   The recovery improvement by the heavy-ball step can be intuitively explained as follows. 
Note that Theorem \ref{t4} shows that the heavy ball can speed up the convergence by reducing $q(0)$ to $q(\beta)$. With a certain percentage of data losses, it might hold that $q(\beta)<1<q(0)$, which indicates that  the iterates with $\beta>0$ are still convergent, while those with $\beta=0$ may be divergent. }


{One can see from the phase transition lines that the required ratio of observations is approximately linear in $r$ when other parameters are fixed.  Note that the degree of freedom of the signal is $\Theta(n_c r)$.  Although our bound   of the required number of measurements $O(r^2\log n)$ in Theorem \ref{t2} is not order-wise optimal due to the artefacts of the proof, the required number of measurements in  numerical experiments is approximately linear in the degree of freedom.  }

\subsubsection{Comparison with Existing Algorithms}
Here,  our methods are compared with FIHT \cite{CWW17} and Singular Value Thresholding (SVT) \cite{CCS10}. The rank $r$ is fixed as $15$, and $n=600, n_1=300, n_c=20.$ All the other setups remain the same for the proposed algorithm.  Since FIHT recovers the missing points in a single channel, each row of $\bfX$ is converted into a Hankel matrix with the size of $300 \times 301$ and apply FIHT separately. SVT solves the convex NNM problem approximately, and the algorithm is applied on the original observation matrix and the constructed Hankel matrix, respectively.
The relative recovery error is calculated as $\|\mathcal{P}_{\widehat{\Omega}^c}(\bfX_l-\bfX)\|_F/\|\mathcal{P}_{\widehat{\Omega}^c}(\bfX)\|_F$.

\begin{figure}[H]
	\centering
	\includegraphics[width=0.9\textwidth]{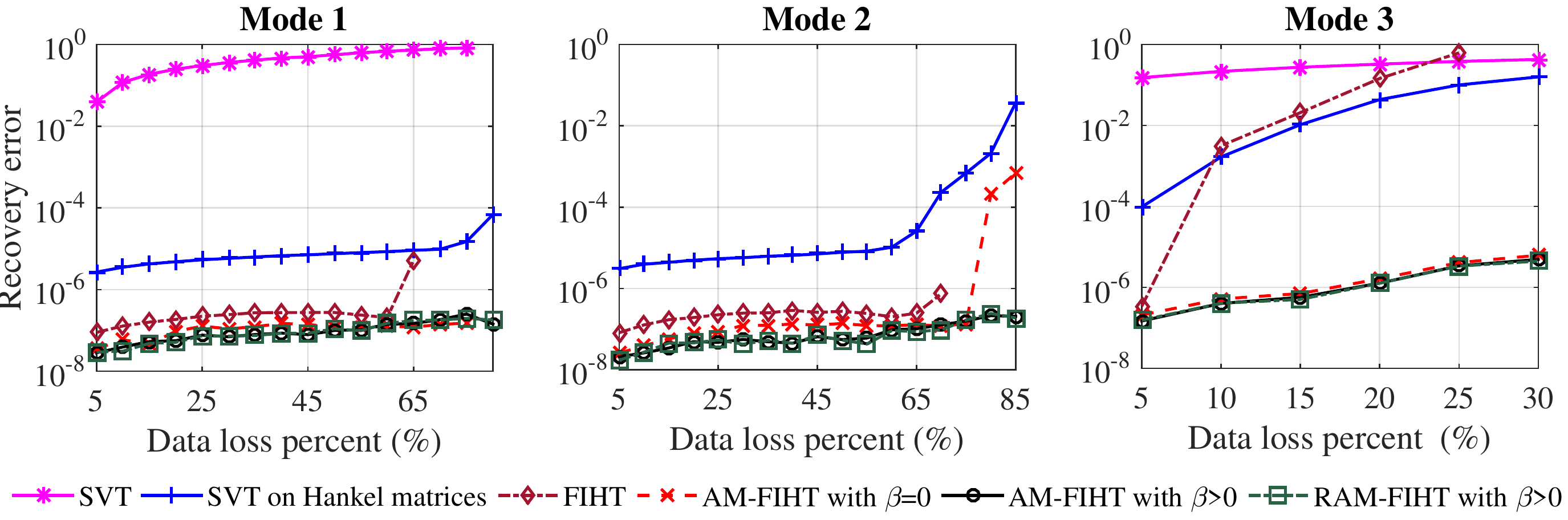}
	\caption{Performance comparison of recovery methods in noiseless setting} 
	\label{fig: sparse spectral signal recovery error noiseless}
\end{figure}
    \begin{figure}[h]
	\centering
	\includegraphics[width=0.8\textwidth]{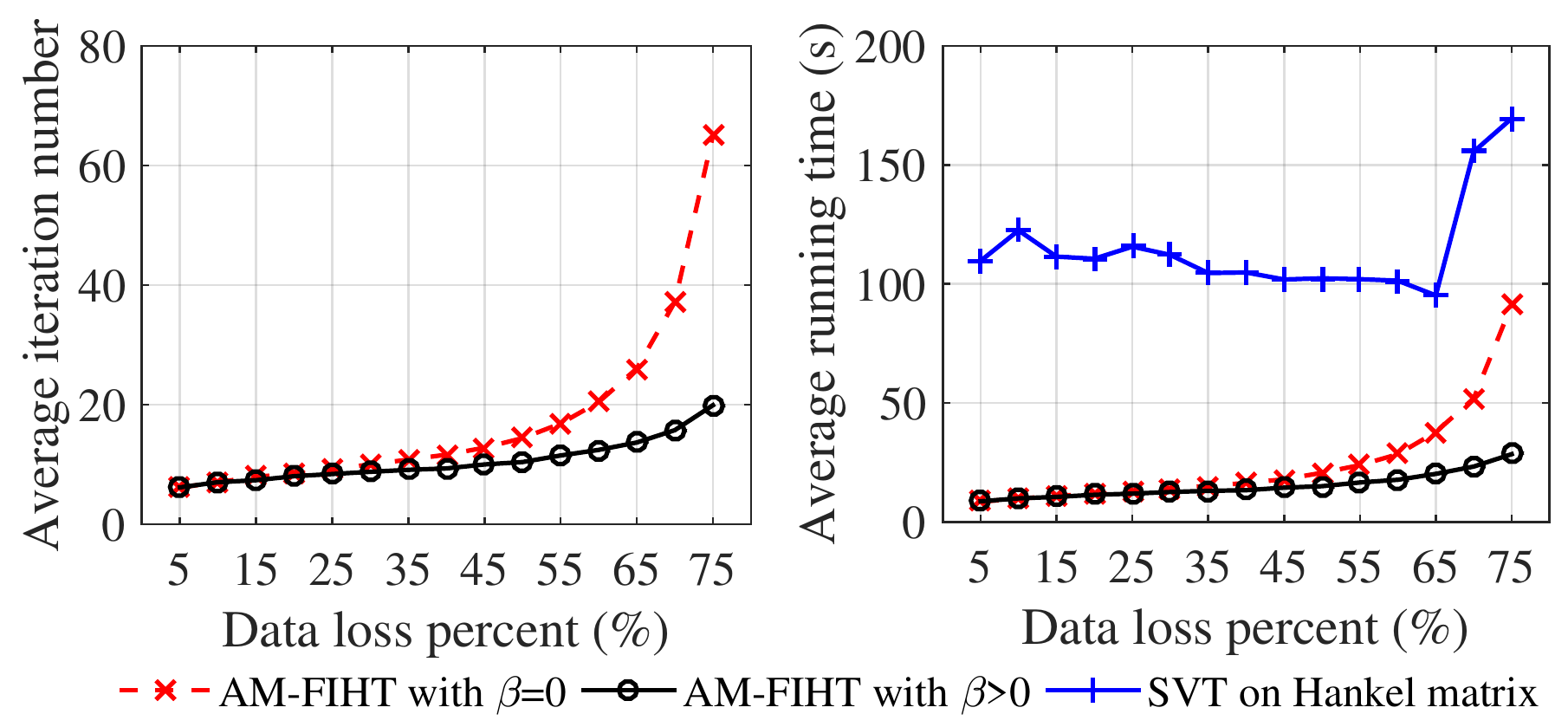}
	\caption{Running time comparison in mode $1$} 
	\label{fig: sparse spectral signal iteration number}
\end{figure} 

Fig.~\ref{fig: sparse spectral signal recovery error noiseless}
	compares the relative recovery error of convergent tests by different methods with noiseless measurements and different data loss patterns. (R)AM-FIHT with a nonzero $\beta$ performs the best among all the methods. As the original data matrix is not low-rank, SVT fails in all cases. When applied to the constructed Hankel matrix, SVT exhibits better performance,  however,  the recovery errors are still much larger compared with (R)AM-FIHT. SVT also needs the much longer running time, as shown in Fig. \ref{fig: sparse spectral signal iteration number}. To achieve the error bound of $10^{-5}$, SVT requires around $100$ iterations at a time cost of $100$ seconds, while AM-FIHT with a nonzero $\beta$ only takes less than 12 seconds to obtain an error bound of $10^{-7}$. With $65\%$ data loss in Mode $1$, $13.3\%$ tests of FIHT diverge. In contrast, all tests of AM-FIHT are convergent, even in the case with $75\%$ data loss.
   A nonzero $\beta$ also increases the percentage of convergent tests. With 80\% of data loss, only $76.7\%$ tests of AM-FIHT with $\beta=0$ converge, while all the tests of AM-FIHT with $\beta>0$ are convergent.  Moreover, AM-FIHT performs much better than FIHT and SVT in Mode 3.
\begin{figure}[h]
	\centering
	\includegraphics[width=0.99\textwidth]{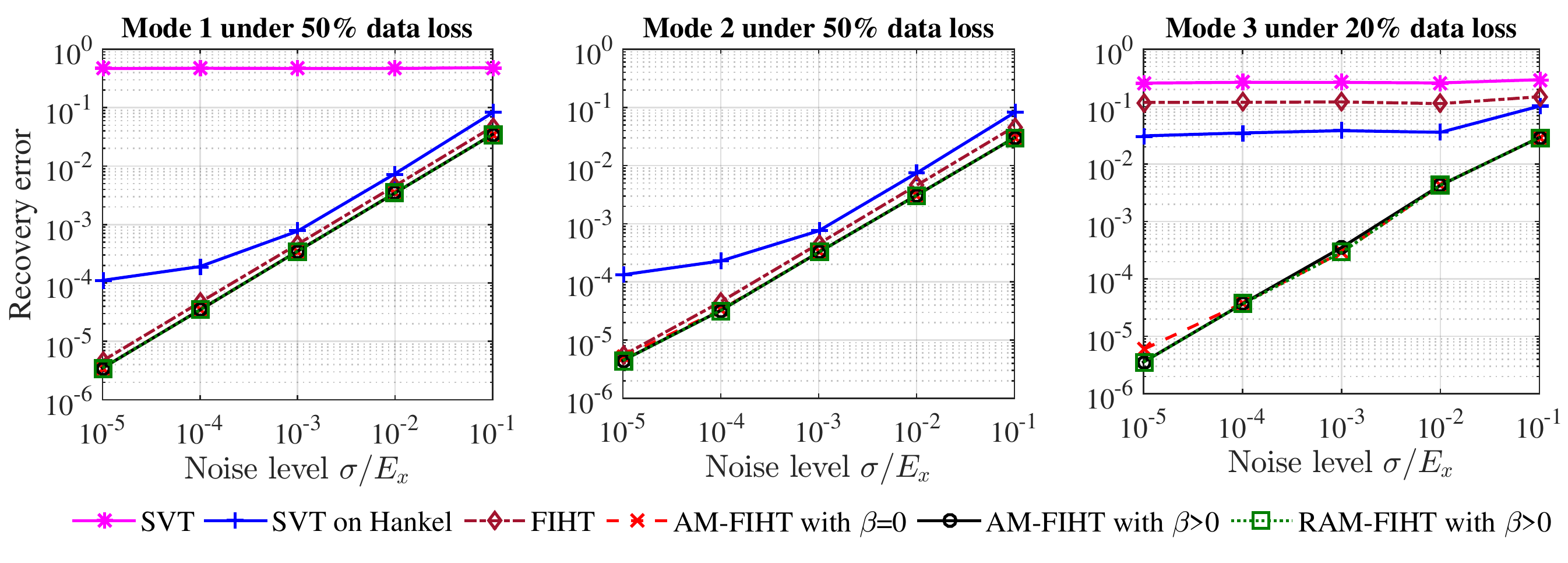}
	\caption{Performance comparison of recovery methods in noisy setting under 50\% data loss in modes 1, 2 and 20\% data loss in mode 3} 
	\label{fig: sparse spectral signal recovery error noisy}
\end{figure}

{When measurements are noisy, every entry of $\bfN$ is independently drawn from Gaussian $\mathcal{N}(0, \sigma^2)$, where $\sigma$ is the standard deviation.  
	Fig.~\ref{fig: sparse spectral signal recovery error noisy}
shows the relative recovery error of convergent tests against the relative noise level $\sigma/E_x$, where $E_x$ is the average energy of $\bfX$ calculated as $E_x=||\bfX||_F/\sqrt{n_cn}$. The data loss percentage is fixed as 50\% in modes 1, 2 and 20\% for mode 3, respectively. 
In all three modes, AM-FIHT and RAM-FIHT perform similarly and achieve the smallest error among all the methods.  The relative recovery error is proportional to the relative noise level, with a  ratio between $0.3$ to $0.4$.
FIHT is sightly worse than these two methods in modes 1 and 2, but its performance degrades significantly in mode 3.  SVT has a better performance when applied on the Hankel matrix instead of the data matrix, but it is  still worse than (R)AM-FIHT. }

\subsection{Numerical Experiments on Actual PMU Data}
The low-rank property of the Hankel PMU data matrix is verified on a recorded PMU dataset in Central New York \cite{GWGCFS16}. $11$ voltage phasors are measured at a rate of $30$ samples per second. 
Fig.~\ref{fig: voltage mag ang samples} shows the recorded voltage magnitudes and angles of a $10$-second dataset that contains a disturbance at $2.3$s.   
Fig.~\ref{fig: approximation erros} shows the approximation errors of $\bfX$ and $\mathcal{H}\bfX$ by rank-$r$ matrices with varying $n_1$ and $r$. 
The approximation error of $\bfX$ with the rank-$r$ matrix $\mathcal{Q}_r(\bfX)$ is defined as   $\|\bfX-\mathcal{Q}_r(\bfX)\|_F/\|\bfX\|_F$, and likewise for $\mathcal{H}\bfX$. One can see from Fig.~\ref{fig: approximation erros} that  all these data matrices can be approximated by rank-8 matrices with negligible errors. 
\begin{figure}[ht]
	\centering
	\includegraphics[width=0.9\textwidth]{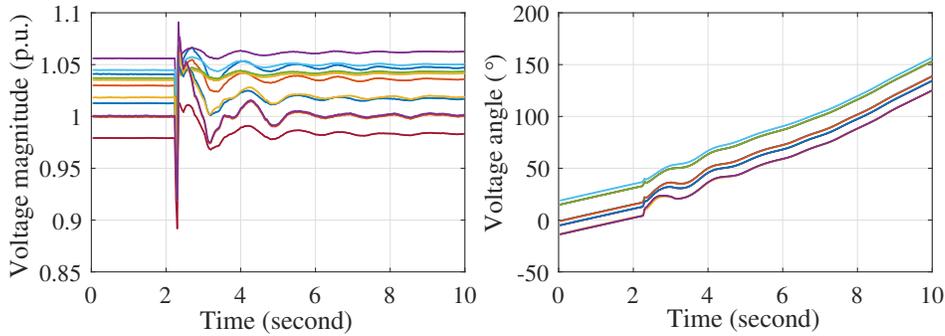}
	\caption{The measured voltage phasors of 11 channels} 
	\label{fig: voltage mag ang samples}
\end{figure}
	
\begin{figure}[h!]
	\centering
	\includegraphics[width=0.7\textwidth]{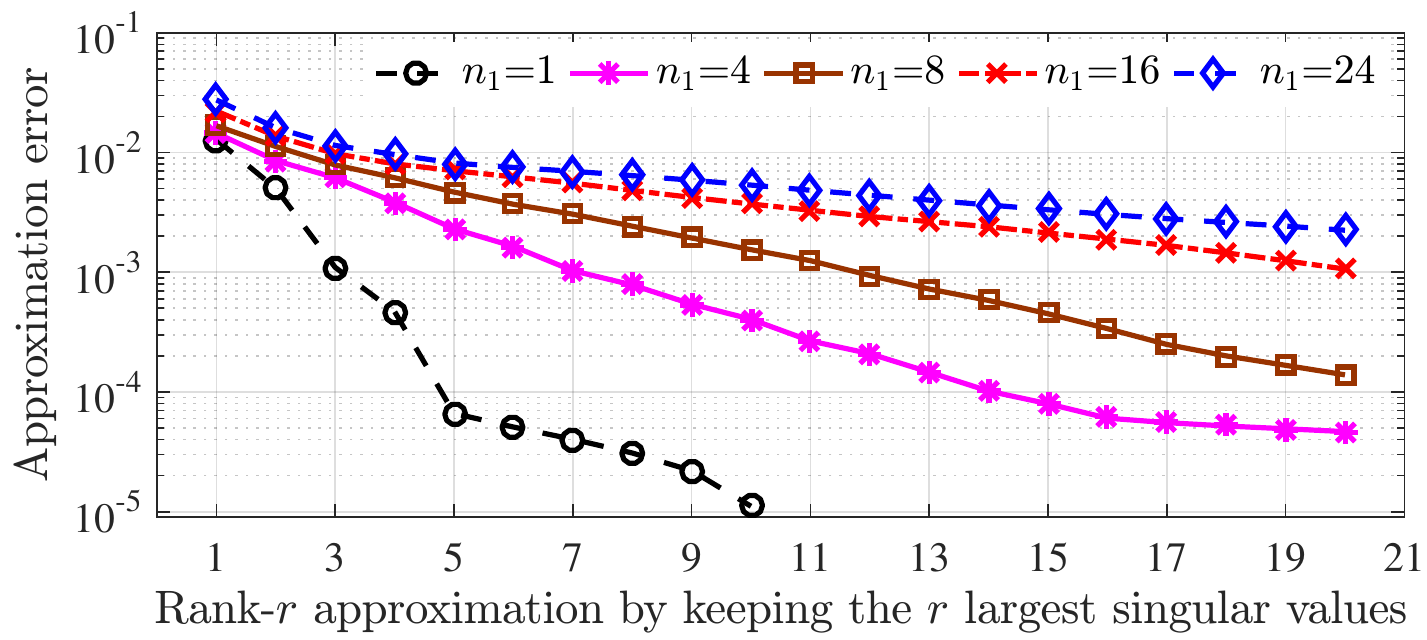}
	\caption{The approximation errors of the data block and the Hankel matrices} 
	\label{fig: approximation erros}
\end{figure}

\begin{figure}[h!]
	\centering
	\includegraphics[width=0.9\textwidth]{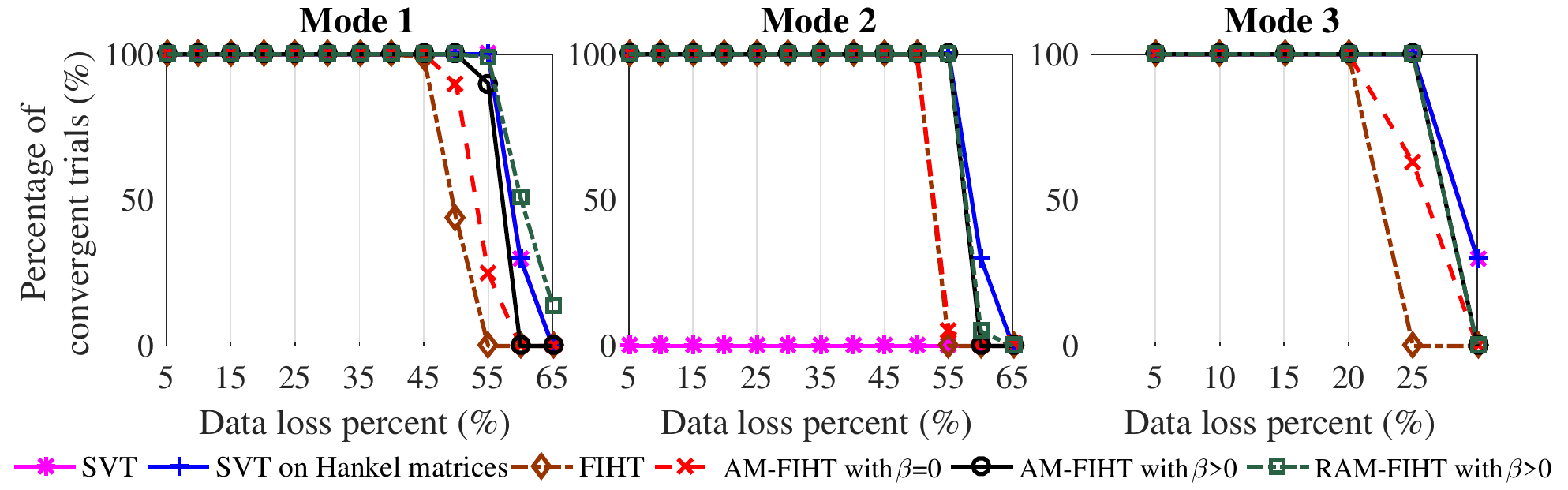}
	\caption{Percentage of convergent trials of recovery algorithms} 
	\label{fig: simulation result 2}
\end{figure}

\begin{figure}[h!]
	\centering
	\includegraphics[width=0.9\textwidth]{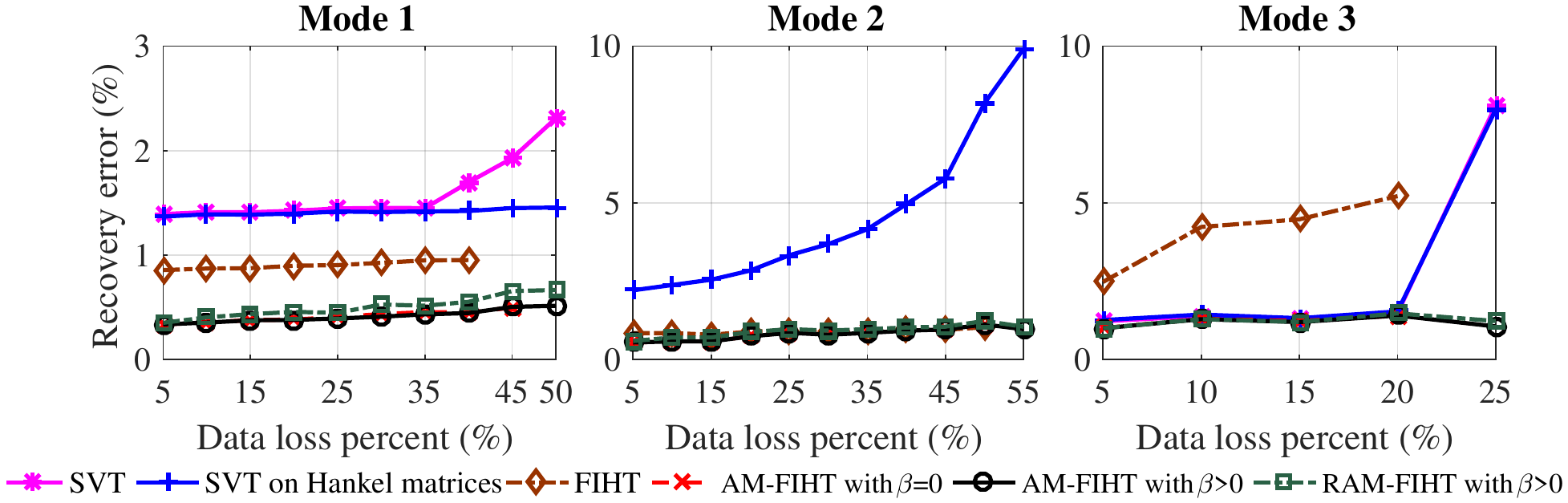}
	\caption{Performance comparison of recovery algorithms} 
	\label{fig: simulation result}
\end{figure}

The parameters are set as $n_1=8, r=8$ and test both $\beta=0$ and $\beta=(1-p)/5$ in the simulation. 
Fig.~\ref{fig: simulation result 2} shows the percentage of convergent runs out of $100$ runs  for different algorithms. Fig.~\ref{fig: simulation result} compares the average recovery error of convergent runs.
Overall, AM-FIHT with $\beta>0$ achieves a small recovery error, tolerates a high data loss rate, and does not require   much computation. 
For example, when the data loss rate is $55\%$ in Mode 2,  
AM-FIHT with $\beta=(1-p)/5$ converges every time.   The number of iterations is $47.2$ on average, and the running time is $0.62$ seconds. It takes $4.34$ seconds to run 400 iterations of SVT on Hankel matrices.
AM-FIHT with $\beta=0$ diverges for $95\%$ of the runs.  FIHT diverges completely. Similar result to Fig.~\ref{fig: sparse spectral signal iteration number} about the average iteration numbers of AM-FIHT with $\beta=0$ and $\beta>0$ from respective successful trials is observed as well, thus a small positive $\beta$ helps improve the convergence rate.
 
There are minor differences between AM-FIHT $(\beta>0)$ and RAM-FIHT $(\beta>0)$ in mode 1.   RAM-FIHT tolerates a slightly higher data loss percentage,  while its average  recovery error of convergent runs is slightly larger than that of AM-FIHT.   AM-FIHT and RAM-FIHT perform similarly in modes 2 and 3.

%% file: chapter_2/rpichap2.tex
 
\chapter{\uppercase{Multi-Channel Hankel Robust Matrix Completion with corrupted columns}\label{chapter: 2}}
\blfootnote{Portions of this chapter previously appeared as: S. Zhang and M. Wang, ``Correction of corrupted columns through fast robust Hankel matrix
  completion,'' \emph{{IEEE} Trans. Signal Process.}, vol.~67, no.~10, pp.
  2580--2594, Mar. 2019.}
\input{./chapter_2/Introduction.tex}

\input{./chapter_2/Problem_formulation.tex}

\input{./chapter_2/Background_and_Related_Work.tex}
\input{./chapter_2/Algorithm.tex}

\input{./chapter_2/Theoretical_result.tex}

\input{./chapter_2/Simulation.tex}

\section{Summary}
The multi-channel low-rank Hankel matrix naturally characterizes the correlations among columns of a matrix in addition to the low-rankness. Exploiting the low-rank Hankel structure, we develop a non-convex approach to recover the low-rank matrix  from partial observations in this chapter, even when a constant fraction of the columns are all corrupted simultaneously and consecutively. The proposed algorithm converges to the ground-truth matrix linearly. The required number of observations is significantly smaller than all the existing bounds for robust matrix completion. Our method applies to power system monitoring, MRI imaging, and array signal processing.  

%% file: chapter_2/Introduction.tex
\section{Introduction}\label{sec:intro}

 Robust matrix completion (RMC) \cite{CJSC13} aims to recover a low-rank matrix $\bfX^*$ in $\mathbb{C}^{n_c\times n}$ ($n_c\le n$) from partial observations  of measurements $\bfM=\bfX^*+\bfS^*$ 
, where the sparse matrix $\bfS^*$ in $\mathbb{C}^{n_c\times n}$ represents arbitrary errors. Due to the wide existence of low-rank matrices, RMC finds applications in areas like  video surveillance \cite{WGRPM09}, face recognition \cite{CLMW11}, MRI image processing \cite{OCS15}, network traffic analysis \cite{MMG13}, and  power systems   \cite{GWCGFSR16}. 
{ For  instance, each row of $\bfX^*$ represents the measurements from one phasor measurement unit (PMU) in power systems, and each column corresponds to the time-synchronized measurements from multiple PMUs \cite{GWGCFS16}.} 
 $\bfS^*$ represents the bad measurements. 

 Let $\widehat{\Omega} \subseteq \{1, \cdots, n_c\}\times \{1, \cdots, n\}$ contain the indices of the observed entries. 
  If $\bfX^*$ is at most rank $r$ and $\bfS^*$ contains at most $s$ nonzero entries, RMC
  can be formulated as a nonconvex optimization problem,  
 \begin{equation}\label{eqn:rpca}
 \begin{split}
 &\min_{{\bfX}, {\bfS} \in \C^{n_c\times n}} \       \sum_{(i,j)\in\widehat{\Omega}}|M_{i,j}- X_{i,j}-S_{i,j}|^2 \\
  &\qquad\textrm{s.t. }\quad     \text{rank}({\bfX})\le r \quad \textrm{and} \quad \|{\bfS}\|_0\le s,
 \end{split}
  \end{equation} 
  where $\|\bfS\|_0$ measures the number of nonzero entries in $\bfS$.  
  If all the entries are observed, i.e., $\widehat{\Omega}$ contains all the indices,
   \eqref{eqn:rpca} reduces to the robust principal component analysis (RPCA) problem, which decomposes a low-rank matrix and a sparse matrix from their sum. If $\bfS^*$ is a zero matrix,    \eqref{eqn:rpca} reduces to the low-rank matrix completion problem. 

One line of research is to relax the nonconvex rank and $\ell_0$-norm terms in   \eqref{eqn:rpca} into the corresponding approximated convex nuclear norm and $\ell_1$-norm.
 Under mild assumptions, $\bfX^*$ and $\bfS^*$ are indeed the solution to the convex relaxation (see e.g., \cite{CJSC13,KLT17} for RMC and \cite{CLMW11, CXCS11,HKZ11} for RPCA). 
Since the convex relaxation is still time-consuming to solve,  fast algorithms based on alternating minimization or gradient descent are developed recently to solve the nonconvex problem directly, for example \cite{GWL16,KC12} for RMC and \cite{CW15,NNSAJ14,YPCC16} for RPCA.  These   approaches are more computationally efficient than the convex alternatives.  

 If the fraction of nonzeros in each column and row of $\bfS^*$ is at most $\Theta(\frac{1}{r})$, then both the  convex method in \cite{CJSC13} and the nonconvex method in \cite{CGJ16} are proven to be able to recover $\bfX^*$ successfully\footnote{$f(n)=O(g(n))$ means that if for some constant $C>0$,  $f(n)\leq Cg(n)$ holds when $n$ is sufficiently large.   $f(n)=\Theta(g(n))$ means that   for some constants $C_1>0$ and $C_2>0$,  $C_1g(n)\leq f(n)\leq C_2g(n)$ holds when $n$ is sufficiently large.}. 
If all the  observations in a column are corrupted, however, with the prior assumption that the matrix is low-rank,  one can locate the corrupted column at best but cannot recover the actual values in either RPCA \cite{XCS12} or RMC \cite{CXCS11}. Since every column is a data point in the $r$-dimensional column subspace,  even if the  column subspace is correctly identified, at least $r$ linearly independent equations, i.e., $r$ entries of each column, are needed to determine the exact values of that column.
 
 {
 In many  applications, the low-rank matrix has the additional low-rank Hankel property. For instance, if $\bfX^*$ contains the time series of $n_c$ output channels in a dynamical system, then the   Hankel matrix of $\bfX^*$ is approximately low-rank, provided that the dynamical system can be approximated by a reduced-order linear system. As demonstrated in \cite{ZHWC18}, the  Hankel matrix of the spatial-temporal blocks of PMU data in power systems is   low-rank. 
 In array signal processing, the Hankel matrix of a spectrally sparse signal is   low-rank \cite{CC14,CWW17,YX16}, and the rank depends on the number of sinusoidal components.  
 The low-rank Hankel property also holds for a class of  finite rate of innovation (FRI) signals, which are motivated by MRI imaging \cite{Haldar14,JLY16,OJ16,OBJ18,YKJL17}.
 
 The low-rank Hankel property has been exploited 
 for data recovery and error correction.   Refs. \cite{CC14,CWW17,ZHWC18} studied the low-rank Hankel matrix completion problem from missing data and proved analytically that the required number of measurements by their respective approaches are significantly smaller than that needed to recover a general low-rank matrix.   
  Error correction by exploiting the low-rank Hankel structure has been exploited in RPCA  \cite{JY18} and RMC \cite{CC14}. 
   Ref.~\cite{CC14} provides the analytical guarantee of low-rank Hankel matrix recovery from randomly located  data losses and corruptions. No analytical guarantee is provided for fully corrupted channels in \cite{CC14}. Moreover, the recovery approach in \cite{CC14} requires solving Semidefinite Programming (SDP), which is computationally expensive in large datasets. 
 }
This chapter solves the RMC problem of a low-rank Hankel matrix. 
 Extending from the methods in \cite{CGJ16,NNSAJ14}, this chapter 
  develops an alternation-projection-based algorithm,
and the iterates are proved to converge to the ground-truth data matrix linearly with a complexity of   $O\big(r^2n_cn\log(n)\log(1/\varepsilon)\big)$, where $\varepsilon$ is the recovery error of  $\bfX^*$. The computational cost is significantly smaller than the approach in \cite{CC14}. The required number of observations for the successful recovery is $O(\mu^2r^3\log^2(n)\log(1/\varepsilon))$, where $\mu$ is the incoherence of the corresponding Hankel matrix. This number is   significantly smaller than the existing bound of $O(rn\log^2(n))$ for recovering a general rank-$r$ matrix \cite{R11}.

Our data model follows the multi-channel Hankel matrix studied in \cite{ZHWC18}, which models multiple signals with common sinusoidal components. The multi-channel Hankel matrix is different from the single-channel Hankel matrix studied in \cite{CC14,CWW17}, where the recovery of only one spectrally sparse signal is considered.
 Our work provides the  first algorithmic development with the theoretical performance guarantee for multi-channel  low-rank Hankel matrix recovery from corrupted measurements. Our method can tolerate up to $\Theta(\frac{1}{r})$ fraction of corruptions per row and does not have any constraint on the number of corruptions per column. In fact, our method can recover $\bfX^*$ accurately even if  $\bfS^*$ contains a constant fraction of fully corrupted columns.
 Full corrupted columns happen in many applications. For example, simultaneous bad data across all channels can happen  due to device malfunctions, communication errors, or cyber data attacks in power systems.


%% file: chapter_2/Problem_formulation.tex
\section{Problem Formulation}\label{sec:formulation}

\newtheorem{assumption}{Assumption}
Let $\bfX^*=[\bfx_1, \bfx_2, \cdots, \bfx_n] \in \mathbb{C}^{n_c\times n}$ denote the actual data.  
We define a linear operator $\mathcal{H}_{n_1}: \mathbb{C}^{n_c\times n}\rightarrow \mathbb{C}^{n_cn_1\times n_2}$ that maps a matrix into its corresponding Hankel matrix, i.e.,
{\small\begin{equation}\label{eqn:Hankel_structure}
\mathcal{H}_{n_1}(\bfX^*)=
\begin{pmatrix}
\bfx_1&\bfx_2&\cdots&\bfx_{n_2}\\
\bfx_2&\bfx_3&\cdots&\bfx_{n_2+1}\\
\vdots&\vdots&\ddots&\vdots\\
\bfx_{n_1}&\bfx_{n_1+1}&\cdots&\bfx_{n}
\end{pmatrix}\in\mathbb{C}^{n_cn_1\times n_2}
\end{equation}}

\noindent with $n_1+n_2=n+1$.  We say $\bfX^*$ satisfies the low-rank Hankel property if rank$(\mathcal{H}_{n_1}(\bfX^*))\leq r$  for { some $r\ll  n$  and some integer $n_1$ in $[r, n+1-r]$}.  
Throughout this chapter, we assume   $n_1>1$ is known and fixed and use $\mathcal{H}\bfX^*$ instead of $\mathcal{H}_{n_1}(\bfX^*)$ for simplicity.

Let $\bfS^*$ denote the additive errors in the measurements. We assume at most $s$ measurements are corrupted, i.e.,   $\|\bfS^*\|_0\leq s$. The values of the nonzero entries can be arbitrary. The measurements are presented by 
\begin{equation}
\bfM=\bfX^*+\bfS^*.
\end{equation}
Define the operator $\mathcal{P}_{\widehat{\Omega}}$ with  
$\mathcal{P}_{\widehat{\Omega}}(\bfM)_{i,j}=M_{i,j}$ if $(i,j)\in\widehat{\Omega}$, and $0$ otherwise.
 The robust low-rank Hankel matrix completion problem aims to recover $\bfX^*$ from $\mathcal{P}_{\widehat{\Omega}}(\bfM)$.  We formulate it as the following nonconvex optimization problem,
\begin{equation}\label{eqn:OP_partial_observation}
\begin{split}
&\min_{{\bfX}, {\bfS}}\quad \left\| \mathcal{P}_{\widehat{\Omega}}(\bfM-{\bfX}-{\bfS})\right\| _F \\
&\text{s.t.} \quad \text{rank}(\mathcal{H}{\bfX})\le r \quad \text{and} \quad \|{\bfS}\|_0\le s,
\end{split}
\end{equation}
where the nonconvexity results from the constraints. 
  \begin{defi}\label{def:mu}
	A rank-$r$ matrix ${{\bfL}}\in\mathbb{C}^{l_1\times l_2}$, with  its Singular Value Decomposition  (SVD)   ${{\bfL}}={\bfU}{\bf\Sigma} {\bfV}^H$, is $\mu$-incoherent if 
	\begin{equation}
	\max_{1\le i\le l_1} \|{\bfe_i^T\bfU} \|^2\le \frac{\mu r}{l_1},\quad   \max_{1\le j\le l_2} \|\bfe_j^T{\bfV} \|^2\le \frac{\mu r}{l_2}.
	\end{equation} 
\end{defi}

 The incoherence assumption   is   standard   in analyzing RPCA and MC problem, see, e.g., \cite{CR09,XCS12}.  
 { If a matrix  is both low-rank and sparse, like $\bfe_i\bfe_j^T$,  then there is no way to separate the sparse component and the low-rank component. The incoherence assumption  prevents the low-rank matrix $L$ to be sparse itself. The incoherence measures the closeness of $\bfL$ to matrices like $\bfe_i\bfe_j^T$. 
 If $\bfL=\bfe_i\bfe_j^T$, the corresponding $\mu$ is as large as $\max\{l_1, l_2\}/r$. 
 When $\mu$ is small, 
 the energy  of $\bfL$ is spread over all its entries.
 } 



The $\bfS^*$ and $\bfX^*$ are assumed to satisfy the following assumption throughout the chapter. We will show our method can accurately recover $\bfX^*$ based on this assumption.
\begin{assumption}
	Each row of $\bfS^*$ contains at most $\alpha$ fraction of non-zero entries with $\alpha\le\frac{C_1}{\mu c_s r} $  for some small positive constant 
	$C_1\le \frac{1}{840}$, 
	where $c_s=\max\Big(\frac{n}{n_1},\frac{n}{n_2}\Big)$; $\mathcal{H}\bfX^*$ is rank-$r$ and $\mu$-incoherent.   
\end{assumption}

In the successful recovery of $\bfX^*$ in  conventional RPCA, $\bfS^*$ can 
have at most $\Theta(\frac{1}{r})$ fraction of 
  nonzeros in each row and each column \cite{NNSAJ14}. In contrast, Assumption 1   only requires an upper limit for each row, while  the entries in one column of $\bfS^*$ can   all be nonzero. In fact, $\bfS^*$ can contain  $\alpha$ fraction of consecutive columns with all nonzero entries.
   If $n_1$ and $n_2$ are in the same order, i.e.,   both proportional to $n$, then $c_s$ is a constant. $\alpha$ could be as large as $\Theta(\frac{1}{r})$.
  Thus, our method can handle bad data across all the channels consecutively for a nearly constant fraction of time if each row of $\bfX^*$ corresponding to a time series.

%% file: chapter_2/Background_and_Related_Work.tex
\section{Applications and Related Work}\label{sec:background}
\subsection{Low-rank Hankel Matrices}
The low-rank Hankel property has been recently exploited  in different areas including array signal processing \cite{CC14,YX16}, dynamic system monitoring \cite{ZHWC18}, and magnetic resonance imaging (MRI) \cite{JLY16,OJ16,YKJL17}.

 {One 
  example of  the low-rank Hankel property is the class of spectrally sparse signals \cite{CWW17}, 
which are  weighted sums of $r$  damped or undamped  sinusoids.  The mathematical expression of an one-dimensional spectrally sparse signal is  
 \begin{equation}\label{eqn:single_sss}
	g[t]=\sum_{i=1}^{r}d_{i}e^{(2\pi\imath f_i-\tau_i)t},\quad t\in\mathbb{N},
\end{equation}
where $f_i$ and $d_{i}$ are the frequency and the normalized complex amplitude of the $i$-th sinusoid, respectively, and $\imath$ is the imaginary unit. As $g[t]$ is the sum of $r$ sinusoids,  its degree of freedom is $\Theta(r)$. The one-dimensional spectrally sparse signal $g[t]$ can be viewed as a special case of $\bfX^*$ in this chapter. Specifically, 
   $\bfX^*$ only contains one row, i.e., $n_c=1$, and let its $i$-th entry be $g[i]$. We follow \cite{ZHWC18} and refer to the resulting Hankel matrix as a single-channel Hankel matrix to differentiate from our general model of a  multi-channel Hankel matrix with $n_c>1$ in \eqref{eqn:Hankel_structure}.  

Ref.~\cite{CC14} also considers two-dimensional (2-D) and higher-dimensional spectrally sparse signals that are the sums of $r$ 2-D or higher-dimensional sinusoids. The data matrix $\bfX^*$ of a 2-D spectrally sparse signal in \cite{CC14} can be represented as
\begin{equation}\label{eqn: 2-D sss}
X^*_{t_1,t_2}= 
\sum_{i=1}^{r}d_{i}e^{(2\pi\imath f_{1i}-\tau_{1i})t_1+(2\pi\imath f_{2i}-\tau_{2i})t_2},
\end{equation}
where $X_{t_1,t_2}$ is the entry in row $t_1$ and column $t_2$.
Note that the degree of freedom of $\bfX$ is still  $\Theta(r)$ for a 2-D signal. 

The second example of the low-rank Hankel property is the outputs of linear dynamic system  discussed in \cite{ZHWC18}. Consider  a discrete-time system with the state vector $\bfs_t\in\mathbb{C}^{n_p}$, and the observation vector $\bfx_t\in\mathbb{C}^{n_c}$, 
\begin{equation}\label{eqn:dynamic}
\begin{split}
\bfs_{t+1}=\bfA\bfs_t,\qquad \qquad\\
\bfx_{t+1}=\bfC\bfs_{t+1}, \quad t=0,1,\cdots, n.
\end{split}
\end{equation}
As described in \cite{ZHWC18}, the data matrix 
$\bfX^*=[\bfx_1, \bfx_2, \cdots, \bfx_n]$ satisfies low-rank Hankel property.
$\mathcal{H}_{n_1}(\bfX^*)$ is rank $r$ for $n_1 \in [r, n+1-r]$, and the rank $r$ is the number of  the observed modes of the dynamical system.   If $r< n_c$, $\mathcal{H}_{n_1}(\bfX^*)$ is also rank $r$ for any $n_1 \in [1, n+1-r]$.  
Each row of $\bfX^*$ can be represented by an one-dimensional spectrally sparse signal. 
All rows share the same set of sinusoids but have different weights. 
The entry in row $k$ and column $t$, denoted by  $X^*_{k,t}$, can be written as
  \begin{equation}\label{eqn: spectrally_sparse_signal}
 X^*_{k,t}= 
 \sum_{i=1}^{r}d_{k,i}e^{(2\pi\imath f_i-\tau_i)t}, \quad k=1,\cdots, n_c,
 \end{equation}
 where $d_{k,i}$ is the normalized complex amplitude of the $i$-th sinusoid in the $k$-th signal.
The degree of freedom  of $\bfX^*$ is $\Theta(n_cr)$.
We also remark that this chapter considers the recovery of $\bfX^*$, which is irrelevant to the observability and identifiability of the linearly dynamical system  in \eqref{eqn:dynamic}. Our method directly recovers the data from partial observations and does not need to estimate the system model. 


In MRI imaging, 
a signal is called finite rate of innovation (FIR) \cite{OJ16} if   there exists a finite sequence $\bfh[t]$ such that
\begin{equation}\label{eqn: FRI_signal}
 (\bfx^**\bfh)[t]=0, \quad \forall t,
\end{equation} 
where $(\cdot)*(\cdot)$ computes the convolution of two signals. Such $\bfh[t]$ is also known as annihilating filter. If the length of $\bfh[t]$ is   $r+1$, then the Hankel matrix $\mathcal{H}_{n_1}(\bfx)$ is a rank-$r$ matrix for some $n_1\in[r, n-r+1]$. The MRI images satisfy \eqref{eqn: FRI_signal} after some transformation. The low-rank Hankel property has been exploited in MRI image recovery \cite{JLY16,OJ16,YKJL17}.  

}

\subsection{Robust Matrix Completion}
When $n_1=1$, \eqref{eqn:OP_partial_observation}   reduces to the conventional RMC problem studied in 
\cite{CLMW11,CGJ16,CJSC13}, \cite{CXCS11,Char12,Li13,GWL16,KC12,KLT17}.
If all the measurements are available, RMC reduces   to the RPCA problem. The state-of-art RPCA algorithms such as \cite{HKZ11,NNSAJ14} can recover the low-rank matrix even if at most $O(\frac{1}{r})$ fraction of  entries per row and per column are corrupted. This bound is also proved to be order-wise  optimal \cite{NNSAJ14}. If no corruptions exist, RMC reduces to the low-rank matrix completion problem, and  $O(\mu_0rn \log n)$ measurements are needed to recover an $n_c$-by-$n$ ($n_c<n$) rank-$r$ matrix with incoherence $\mu_0$ \cite{CR09}.

For the general RMC problem, one approach is to relax  the nonconvex rank and $\ell_0$-norm into the convex nuclear norm and $\ell_1$ norm and then solve the resulting convex optimization problem \cite{CJSC13,CLMW11,KLT17,Li13}. Refs.~\cite{CLMW11,CJSC13} show that the convex approach can correct a constant fraction of randomly distributed outliers, provided that a constant fraction of the matrix entries are observed. 
Based on a stronger requirement on the incoherence of the matrix,  ref.~\cite{Li13} improves the theoretical bound such that only $O(\mu_0 r n \log^2(n))$ observed entries are required while   tolerating a constant fraction of bad data. Although fully corrupted columns are considered in \cite{CXCS11, KLT17}, both p{a}pers cannot recover the corrupted columns. 
Ref.~\cite{CXCS11} shows that when $\bfS^*$ contains fully corrupted columns, the convex approach can recover non-corrupted columns and estimate the column subspace accurately. However, their approach can only locate the corrupted column but cannot recover its actual entries.  Ref.~\cite{KLT17} provides an upper bound of RMC when the measurements contain noise. The error bound is large when some columns are fully corrupted, because the recovery of corrupted columns is not accurate.

Fast algorithms to solve the nonconvex formulation directly have been recently developed. Ref.~\cite{Char12} adds a nonconvex penalty function to speed up the minimization through a shrinkage operator, but no analytical analysis is reported.
Ref.~\cite{KC12}  proposes a projected gradient descent algorithm over the nonconvex sets. Ref.~\cite{GWL16} proposes an alternating minimization algorithm.  
 Both \cite{GWL16} and \cite{KC12} prove the proposed algorithms converge under the assumption of Restrict Isometric Property (RIP), but no theoretical analyses of the recovery performance are provided.
 
 {
 	The low-rank Hankel has been exploited in missing data recovery but not much in error corrections.  Refs.~\cite{CWW17} analyze the matrix completion performance for single-channel Hankel matrices, i.e., $n_c=1$. Ref.~\cite{ZHWC18} extends the analysis to multi-channel Hankel matrices with $n_c>1$. If $\bfS^*$ is a zero matrix, one can recover $\bfX^*$ from   $O(\mu r^3 \log n)$ observations \cite{ZHWC18}, where $\mu$  is the incoherence of $\mathcal{H}\bfX^*$.   Theorem 5 of \cite{ZHWC18} indicates that  $\mu$ is a constant for a group of well separated frequencies $f_{i}$'s and concentrated normalized amplitude $d_{k,i}$'s. 
 	
 	Only Refs.~\cite{CC14} and \cite{JY18} consider the RMC problems for the low-rank Hankel matrix. The nonconvex rank  and $\ell_0$-norm are relaxed into the convex nuclear norm and $\ell_1$-norm, respectively  in both \cite{CC14} and \cite{JY18},   and only Ref. \cite{CC14} provides the theoretical guarantee.  Although  high-dimensional spectral sparse signals are considered in \cite{CC14}, the degree of freedom of these signals is still $\Theta(r)$, which corresponds to single-channel Hankel matrices in our setup. 
 	We consider multi-channel Hankel matrix where $n_c>1$ in this chapter. Moreover,
 	Ref. \cite{CC14} assumes the locations of the corrupted entries are randomly distributed and does not provide any theoretical  recovery guarantee when column-wise corruptions exist. This chapter provides the first theoretical study of RMC and RPCA for multi-channel low-rank Hankel matrices with fully corrupted columns. Furthermore, the convex approach in  \cite{CC14} requires  solving SDP, 
 	which 
 	is time-consuming for large-scale problems. The computational complexity of solving the SDP  to recover the Hankel matrix $\mathcal{H}\bfX^*$
 	is $\mathcal{O}(n_c^3n^3)$, while the computational complexity of our algorithm is $O\big(r^2n_cn\log(n)\log(1/\varepsilon)\big)$, where $\varepsilon$ is the approximation error.
 	 }
 
 \subsection{Rank-based Stagewise (R-RMC) Algorithm}
 Ref.~\cite{CGJ16} proposed a nonconvex algorithm called Rank-based stagewise (R-RMC) algorithm to solve RMC. The R-RMC algorithm is directly extended from  the  AltProj algorithm in \cite{NNSAJ14} for RPCA by adjusting to partial measurements. 
 R-RMC  contains two loops of iterations.
 In the $k$-th stage of the outer loop, it decomposes $\bfM$ into a rank-$k$ matrix and a sparse matrix. The resulting matrices are used for initial points in the $(k+1)$-th stage. 
 In the $t$-th iteration of the inner loop, it updates the sparse matrix $\bfS_{t}$ and the rank-$k$ matrix $\bfL_{t+1}$ based on $\bfS_{t-1}$ and $\bfL_{t}$. 
 $\bfS_t$ is obtained by a hard thresholding  over the residual error between $\bfM$ and $\bfL_{t}$. 
 $\bfL_{t+1}$ is updated by 
 first moving along the gradient descent direction 
 and then truncating it to a rank-$k$ matrix.    
 The reason of using an outer loop instead of directly decomposing into a rank-$r$ and a sparse matrix is that by the initial thresholding, the remaining sparse corruptions in the residual is in the order of $\sigma_1(\bfX^*)$.  These corruptions would lead to large errors in the estimation of the lower singular values of $\bfX^*$. Through the outer loop, the algorithm recovers the lower singular values  after the corruptions at higher values are already removed.  The computational complexity of R-RMC is $O\Big((mr^2+nr^3)\log(1/\varepsilon)\Big)$, where $m$ is the number of observed measurements.  
 
  To achieve a recovery accuracy of $\varepsilon$, 
  R-RMC requires at least $O(\mu_0^2 r^3n\log^2(n)\log(1/\varepsilon))$ observed measurements. The percentage of  outliers per row and per column is at most $O(\frac{1}{r})$.
 
 This chapter develops an algorithm based upon R-RMC \cite{CGJ16} to solve the nonconvex problem \eqref{eqn:OP_partial_observation}. By exploiting the Hankel structure, our algorithm can correct fully corrupted columns, which cannot be corrected by R-RMC. Moreover, the required number of measurements by our method is significantly less than that by R-RMC.

%% file: chapter_2/Algorithm.tex
\section{Structured Alternating Projection (SAP) Algorithm}\label{sec:algorithm}

 Here we present the structured alternating projections (SAP) algorithm to solve \eqref{eqn:OP_partial_observation}. 
 In the algorithm, $\bfM, \bfX_t, \bfS_t\in\mathbb{C}^{n_c\times n}$, and $\bfW_{t}, \bfL_{t}\in\mathbb{C}^{n_cn_1\times n_2}$. 
  $\mathcal{T}_{\xi}$ is the hard thresholding operator, 
 \begin{equation}
 \mathcal{T}_{\xi}(\bfZ)_{i,j}
 = Z_{ij}\quad \text{if}\quad { |Z_{ij}|\ge \xi},\quad \text{and }0\quad \text{otherwise}. 
 \end{equation}
 Let $\bfZ=\sum_{i=1}\sigma_i\bfu_i\bfv_i^H$ denote the SVD of $\bfZ$ with $\sigma_1\ge\sigma_2\ge \cdots$.  $\mathcal{Q}_k$ finds the best rank-$k$ approximation to $\bfZ$, i.e., 
\begin{equation}
\mathcal{Q}_k(\bfZ)=\sum_{i=1}^{k}\sigma_i\bfu_i\bfv_i^H.
\end{equation}  
$\mathcal{H}^{\dagger}$ denote the Moore-Penrose pseudoinverse of ${\mathcal{H}}$. Given any matrix $\bfZ\in \mathbb{C}^{n_cn_1\times n_2}$, $\mathcal{H}^{\dagger}(\bfZ)\in\mathbb{C}^{n_c\times n}$ satisfies
\begin{equation}
(\mathcal{H}^{\dagger}(\bfZ))_{i,j}=\frac{1}{w_j}\sum_{k_1+k_2=j+1}Z_{(k_1-1)n_c+i,k_2}, 
\end{equation}
where $w_j$ denotes the number of elements in the $j$-th anti-diagonal of an $n_1\times n_2$ matrix. 

\begin{algorithm}[h]
	\caption{Structured Alternating Projections (SAP)}\label{Alg3}
	\begin{algorithmic}[1]
		\State  {\bf Input} Observations $\mathcal{P}_{\widehat{\Omega}}(\boldsymbol{M})$,  thresholding parameter $\varepsilon$, 
		 the largest singular value $\sigma_{1}=\sigma_{1}(\mathcal{H}\bfX^*)$, and { convergence criterion $\eta=\frac{4\mu c_s r}{\sqrt{n_c}n}$}.
		\State  {\bf Initialization}  $\bfX_0=\bf0$, $\xi_{0}=\eta\sigma_{1}$.
		\State  {\bf Partition} $\widehat{\Omega}$ into disjoint subsets $\widehat{\Omega}_{k,t}$ ($1\le k \le r, 0\le t\le T$) of equal size $\widehat{m}$, let $\hat{p}=\frac{\widehat{m}}{n_cn}$.
		\For	{Stage $k=1, 2, \cdots, r$}
		\For	{$t=0,1,\cdots, T=\log(\eta\sqrt{n_c}n\sigma_{1}/\varepsilon)$}
		\State $\bfS_{t}=\mathcal{T}_{\xi_{t}}\big(\mathcal{P}_{\widehat{\Omega}_{k,t}}(\bfM-\bfX_t)\big)$;
		\State 
		$\bfW_{t}=\cH\Big(\bfX_t+\hat{p}^{-1}\big(\mathcal{P}_{\widehat{\Omega}_{k,t}}(\bfM-\bfX_t)-\bfS_t\big)\Big)$;
		\State
		$\xi_{t+1}=\eta\big(\sigma_{k+1}(\bfW_t)+(\frac{1}{2})^t\sigma_{k}(\bfW_t)\big)$;
		\State $\bfL_{t+1}=\mathcal{Q}_k(\bfW_t)$;
		\State $\bfX_{t+1}=\cH^{\dagger}\bfL_{t+1}$;
		\EndFor	
		\If {$\eta\sigma_{k+1}(\bfW_{T})\le\frac{\varepsilon}{\sqrt{n_c}n}$}
		\State	{\bf Return} $\bfX_{T+1}$; 
		\EndIf
		\State $\bfX_0=\bfX_{T+1}$, $\xi_0=\xi_{T+1}$;
		\EndFor
	\end{algorithmic}
\end{algorithm}
SAP is built upon  Rank-based stagewise (R-RMC) \cite{CGJ16}. 
The major differences of SAP from R-RMC are the additional Hankel structure. The main contribution of this chapter is   the analytical performance guarantee of SAP, which we defer  to Section \ref{sec:guarantee}. 
    The key steps are summarized as follows. Similar to R-RMC \cite{CGJ16}, SAP also contains two stages of iterations.
    In the $t$-th iteration of the inner loop, it updates the estimated sparse error matrix $\bfS_{t}$ and data matrix $\bfX_{t+1}$ based on $\bfS_{t-1}$ and $\bfX_{t}$. 
    $\bfS_t$ is obtained by a hard thresholding  over the residual error between $\bfM$ and $\bfX_{t}$. The thresholding $\xi_{t}$ decreases  as $t$ increases. The entire sampling set $\widehat{\Omega}$ is first divided into several disjoint subsets. The disjointness guarantees the independence across $\bfX_{t}$ and $\bfX_{t+1}$, which is a standard analysis trick in solving RMC (see  \cite{R11}).  
    To obtain $\bfX_{t+1}$, we first updated $X_t$ by 
     moving  along the gradient descent direction with a step size $\hat{p}^{-1}=\frac{n_cn}{|\widehat{\Omega}_{k,t}|}$. Then, $\bfW_{t}$ is calculated as the  projection of the updated $\bfX_{t}$ to the Hankel matrix space. Finally, $\bfX_{t+1}$ is obtained by $\mathcal{H}^{\dagger}\bfL_{t+1}$, and $\bfL_{t+1}$ is updated by  truncating $\bfW_{t}$ to a rank-$k$ matrix. 
     { The maximum  number of iterations in each inner loop, denoted as $T$, is set as $\log(\eta\sqrt{n_c}n\sigma_{1}/\varepsilon)$.}
     { In practice, 
     the algorithm can exit the loop before reaching the maximum number of iterations if 
     $\bfX_{t+1}$  is already very close to $\bfX_{t}$.}  
       In the $k$-th iteration of the outer loop, the target rank increases from $1$ gradually, and the resulting matrices are used {as the  initialization} in the $(k+1)$-th stage. 
     
     The reason of using an outer loop instead of directly applying rank-$r$ approximation when calculating $\bfL_{t}$ is the same as that in R-RMC \cite{CGJ16} and AltProj \cite{NNSAJ14}. By the initial thresholding, the remaining sparse corruptions in the residual is in the order of $\sigma_1(\mathcal{H}\bfX^*)$, the largest singular value of $\mathcal{H}\bfX^*$. These corruptions would lead to large errors in the estimation of the lower singular values of $\mathcal{H}\bfX^*$. Through the outer loop, the algorithm recovers the lower singular values  after the corruptions with higher values are already removed. 
     
     Calculating the best rank-$k$ approximation in line $9$ dominates the computation complexity.
     Generally for a matrix $\bfW_t\in \mathbb{C}^{n_cn_1\times n_2}$, the best rank-$k$ approximation can be solved in $O(kn_cn^2)$, and $n_cn^2$ results from calculating  $\bfW_{t}\bfz$ for $\bfz\in\mathbb{C}^{n_2}$. Here, due to the Hankel structure of $\bfW_{t}$, a fast convolution algorithm (see \cite{CWW17,ZHWC18}) only requires computational complexity at $O(n_cn\log(n))$ to compute  $(\mathcal{H}\bfZ)\bfz$ for any $\bfZ\in\mathbb{C}^{n_c\times n}$ and $\bfz\in\mathbb{C}^{n_2}$. 
     The fast convolution can also applied to reduce the computational time to $O\big(rn_cn\log(n)\big)$ when calculating $\bfX_{t+1}=\mathcal{H}^{\dagger}\bfL_{t+1}$ with stored SVD of $\bfL_{t+1}$\cite{CWW17,ZHWC18}. Hence, the computational complexity per iteration is  $O\big(rn_cn\log(n)\big)$, and the total computational complexity is $O\big(r^2n_cn\log(n)\log(1/\varepsilon)\big)$.
     
     One can directly apply R-RMC on the structured Hankel matrix. 
     The resulting algorithm differs from SAP in line 7 and 10 that the updated rank-$k$ matrix is not projected to the Hankel matrix space. 
     Based on the analysis in \cite{CGJ16}, the computational time per iteration of R-RMC on Hankel matrix is $O(m_sr+n_cnr^2)$, and $m_s$ is the number of observed measurements in the structured Hankel matrix.
     With full observations,  the computational complexity per iteration of R-RMC on Hankel matrices is as large as $O(n_cn^2r)$. By downsampling the observation set to its theoretical limit in Theorem 2 of \cite{CGJ16}, the computational complexity per iteration  can be reduced to $O(\mu^2r^3n_cn\log^2(n))$. However, it is still larger than $O(rn_cn\log(n))$ of SAP. Moreover, the constant item of the theoretical limit in theorem 2 \cite{CGJ16} is hard to determine in practice. Furthermore, downsampling will increases the iteration number numerically. Though the complexity per iteration is reduced by downsampling, the computational time may increase, which is  reported in Fig. 1(b) \cite{CGJ16} as well.

{  We remark that $\sigma_{1}(\mathcal{H}\bfX^*)$ and $\mu$ may not be computed directly.   $\sigma_1(\mathcal{H}\bfX^*)$ is only used to obtain the initial  estimation of the sparse matrix.  In practice,  we use $p^{-1}(\mathcal{H}\mathcal{P}_{\widehat{\Omega}}(\bfM))$ to estimate $\sigma_1(\mathcal{H}\bfX^*)$.  This  idea is borrowed from \cite{NNSAJ14,CGJ16}.  As long as the estimated value is in the same order as   $\sigma_{1}(\mathcal{H}\bfX^*)$, all the theoretical results in the following Theorem  1 still hold,   with a different constant $C_1$ in Assumption 1.  $\mu$ is only used in $\eta$ as $\eta =\frac{\mu c_s r}{\sqrt{n_c}n}$. If the estimated incoherence is in the same order as $\mu$, all the results still hold with a different constant  $C_1$ in Assumption 1 and a different constant $C_2$ in \eqref{eqn:num_measurements_theoretical_bound}. In practice, one can estimate $\eta$ by $\frac{r}{\sqrt{n_cn_1n_2}}$   for incoherent matrices without computing $\mu$. This idea has been used in    \cite{CLMW11,NNSAJ14,CGJ16}, which all require $\mu$ in their algorithms but do not actually compute it.  Thus, we present SAP using  $\sigma_{1}(\mathcal{H}\bfX^*)$ and $\mu$ to simply the following theoretical analysis, and one can replace them with estimated values in implementation.

%
%
%
%
  }  

{
We also note that the recovery of $\bfX^*$ is irrelevant to the observability and identifiability of the linearly dynamical system  in \eqref{eqn:dynamic}, when $\bfX^*$ contains the output time series of \eqref{eqn:dynamic}. Our method directly recovers the data from partial observations and does not need to estimate the system model. }

%% file: chapter_2/Theoretical_result.tex
\section{Recovery Guarantee of SAP}\label{sec:guarantee}
The recovery guarantee of SAP is summarized in  Theorem \ref{Theorem: bad data and missing data}, and the proof is deferred to the Appendix.
\begin{theorem}\label{Theorem: bad data and missing data}
	{ Suppose $\bfX^*$, $\bfS^*$ satisfy the Assumption 1, and the support of the sampling set $\widehat{\Omega}$ is randomly selected. 
	 Let $\eta=\frac{4\mu c_s r}{\sqrt{n_c}n}$ and $T=\log(\eta\sqrt{n_c}n\sigma_{1}/\varepsilon)$} in Algorithm \ref{Alg3}. If
	\begin{equation}\label{eqn:num_measurements_theoretical_bound}
	\begin{split}
	m\ge C_2\mu^2 r^3\log^2(n)\log\big({\mu c_s r\sigma_1}/{\varepsilon}\big),
	\end{split}
	\end{equation} 
	 with probability at least $1-\frac{rn_cT\log^3(n_cn)}{n^{2}}$,  its output $\boldsymbol{{X}}$ and $\bfS$
	satisfy:
	\begin{equation}\label{eqn:THM1}
	\begin{split}	
	&\qquad \qquad \qquad \qquad  \|\boldsymbol{{X}}-\bfX^*\|_F\le\varepsilon\\
	&	\|\bfS-\mathcal{P}_{\widehat{\Omega}}(\bfS^*) \|_F\le \varepsilon, \quad {Supp}(\bfS)\subseteq Supp\big(\mathcal{P}_{\widehat{\Omega}}(\bfS^*)\big)
	\end{split}
	\end{equation}
	for some large constant $C_2>0$.
\end{theorem}

 Theorem \ref{Theorem: bad data and missing data} indicates that the resulting $\bfX$ returned by SAP can be arbitrarily close to the ground truth $\bfX^*$ as long as the number of observations exceeds $O\Big( \mu^2r^3\log^2n\log({\mu r\sigma_1}/{\varepsilon})\Big)$, and each row of $\bfS^*$ has at most $\Theta(\frac{1}{\mu r})$ fraction of outliers. 
{ If $\bfX^*$ contains spectrally sparse signals as shown in \eqref{eqn: spectrally_sparse_signal}, then  $\bfX^*$ is also rank $r$. If we directly apply a low-rank MC method via convex relaxation \cite{R11} to recover $\bfX^*$ from $\mathcal{P}_{\widehat{\Omega}}(\bfX^*)$,  the required number of  observations is at least ${O}(\mu_0 rn\log^2(n))$.} Since $n\gg r$, SAP reduces the required number of observations significantly by exploiting the Hankel structure. Moreover, SAP can identity and correct fully corrupted columns up to a fraction at $\Theta({1}/{\mu r})$. In contrast, traditional RMC methods 
can locate the fully corrupted columns but cannot recover the corrupted columns \cite{CXCS11,XCS12}.
The number of iterations $rT$ depends on  $\log(1/\varepsilon)$, where $\varepsilon$ is desired accuracy. Therefore, the algorithm also enjoys a linear convergent rate. 



If there is no bad data, i.e. $\bfS^*=\bf0$, \eqref{eqn:OP_partial_observation} is reduced to the MC problem. Under the setup of spectrally sparse signals in \eqref{eqn: spectrally_sparse_signal},
according to the Theorem 5 in \cite{ZHWC18}, a group of well separated frequencies $f_{i}$'s can guarantee that the incoherence $\mu<O(n_c)$.
If we further assume on the normalized amplitude $d_{k,i}$, say that  $d_{k,i}$'s are close to each other, the incoherence $\mu$ is a constant. 
The degree of freedom depends linearly on the rank $r$, while the theoretical bound in \eqref{eqn:num_measurements_theoretical_bound} relies on $(r^3)$. When $r$ is small, the theoretical bound in \eqref{eqn:num_measurements_theoretical_bound} is nearly optimal.
Compared with our  algorithm AM-FIHT in \cite{ZHWC18},
SAP does not have the heavy-ball step and increases the rank gradually instead of keeping fixed rank. To achieve a recovery error of $\varepsilon$, AM-FIHT requires $O\Big(\mu\kappa^6r^2\log(n)\log\big(\frac{\sigma_{1}}{\kappa^3\varepsilon}\big)\Big)$ observations. 
In contrast, SAP depends on $r^3$ but does not rely on the conditional number $\kappa$, where $\kappa$ is defined as the ratio of the largest to smallest singular values of $\mathcal{H}\bfX^*$.   

If there is no missing data, i.e. $\widehat{\Omega}=\{(k,t)| 1\le k\le n_c, 1\le t\le n  \}$, \eqref{eqn:OP_partial_observation} is reduced to RPCA problem. Each row of $\bfS^*$ can have up to $\alpha\le \frac{C_1}{\mu c_s r}$ fraction of corrupted entries. If we choose $n_1=n_2$, $c_s$ is constant. 
The existing results in \cite{NNSAJ14,HKZ11} for RPCA can tolerate at most $\Theta(\frac{1}{r})$ fraction of outliers per row and per column.  
SAP also tolerates at most $\Theta(\frac{1}{r})$ fraction of outliers in each row. Moreover, SAP can recover fully corrupted columns. There is no upper bound of the number of corruptions per column.  
One can directly apply a general RPCA algorithm such as AltProj \cite{NNSAJ14} on the structured Hankel matrix $\mathcal{H}(\bfM)$, Altproj can recover the corrupted data correctly based on the same analysis  as in \cite{NNSAJ14}.
 However, the computational time  per iteration of Altproj  is  $O(rn_cn^2)$,  which is much large  than  $O(rn_cn\log(n))$ by SAP. 


%% file: chapter_2/Simulation.tex
\section{Numerical Results}\label{sec:numerical}
We evaluate the performance of SAP numerically. The experiments are implemented in MATLAB 2015 on a desktop with 3.4GHz Intel Core i7-4770 CPU. 
Here, we study several modes of missing data and bad data as shown in Figs. \ref{fig:missing data modes} and \ref{fig:bad data modes}.  For  each pair of data loss and bad data modes, the supports of the bad data matrix $\bfS^*$ and the observed indices $\widehat{\Omega}$ are generated independently. The models are summarized as: 
\begin{itemize}
	\item M1/B1: Missing data or bad data occur randomly across the all  channels and times;
	\item M2/B2: Missing data or bad data occur in all channel simultaneously s at randomly selected time indices;
	\item B3: Bad data occurs simultaneously and consecutively in all the channels. The starting point is selected randomly. 
\end{itemize}
\begin{figure}[h]
	\centering
	\includegraphics[width=0.5\textwidth]{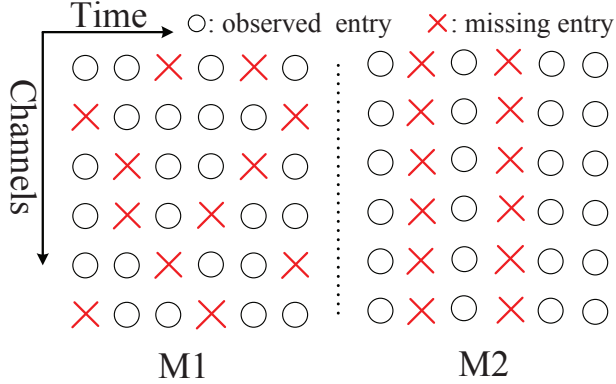}
	\centering
	\caption{Two modes of missing data} 
	\label{fig:missing data modes}
\end{figure}
\begin{figure}[h]
	\centering
	\includegraphics[width=0.6\textwidth]{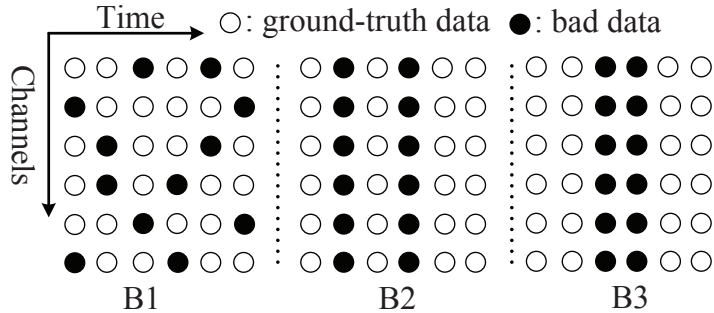}
	\centering
	\caption{Three modes of bad data} 
	\label{fig:bad data modes}
\end{figure}

The performance is tested on the spectrally sparse signals as shown in \eqref{eqn: spectrally_sparse_signal}.
Each $f_i$ in \eqref{eqn: spectrally_sparse_signal} is randomly selected from $(0,1)$. $\tau_i$ is set as $0$ for all $i$. For the complex coefficient $d_{k,i}$, its angle is randomly selected from $(0, 2\pi)$, and its magnitude is set as $1+10^{0.5a_{k,i}}$, where $a_{k,i}$ is randomly selected from $(0,1)$. For each non-zero entry in the bad data matrix $\bfS^*$, its angle is randomly selected from $(0, 2\pi)$ ({ except for Fig.~\ref{fig: Phase transition of SAP with outliers restricted in  Quadrants I}}), 
and its magnitude is randomly selected from $(\bar{X}^*, 5\bar{X}^*)$, where $\bar{X}^*=\|\bfX^*\|_F/\sqrt{n_cn}$ is the average energy of $\bfX^*$.  Unless otherwise stated, the size of the data matrix $\bfX^*\in\mathbb{C}^{n_c \times n}$ is set as $n_c=30$ and $n=300$, and $n_1=n/2=150$.

In SAP, the SVD algorithm  for a structured Hankel matrix  is computed via PROPACK \cite{La04}. 
PROPACK provides a general framework to compute the partial SVD of a structured matrix that denoted by $\bfA$, and
the user is required to implement the functions to compute $\bfA\bfy_1$ and $\bfA^H\bfy_2$. For a Hankel matrix $\mathcal{H}\bfZ_1$, the function to compute $(\mathcal{H}\bfZ_1)\bfz_2$ is implemented by calculating the convolution of $\bfz_2$ and each row of $\bfZ_1$. 
Since $\sigma_{1}(\mathcal{H}\bfX^*)$ is unknown, we use $p^{-1}(\mathcal{H}\mathcal{P}_{\widehat{\Omega}}(\bfM))$ to approximate $\sigma_1(\mathcal{H}\bfX^*)$ in our experiments following the same idea in \cite{NNSAJ14,CGJ16}.
Also, during each iteration, we use the entire observed set rather than the disjoint subsets as shown in line 3 of Alg. \ref{Alg}. 
In each inner loop, instead of keeping a fixed number of iterations, SAP will jump out of the current inner loop if 
\begin{equation}
\frac{\|\mathcal{P}_{\widehat{\Omega}}(\bfX_{t+1}-\bfX_{t})\|_F}{\|\mathcal{P}_{\widehat{\Omega}}(\bfX_{t})\|_F}\le 10^{-3}
\end{equation}
before reaching the maximum iteration number, which is set  as 200. 
SAP finally terminates if $\sigma_{k+1}(\bfW_t)\le 10^{-3}$ holds.

The results in Figs. 3-10 are all obtained by averaging over $100$ independent trials for each block. We say that the trial is successful if the returned $\bfX$ satisfies that
\begin{equation}\label{eqn:success_criterion}
\frac{\|\mathcal{P}_{\widehat{\Omega}^c}(\bfX-\bfX^*)\|_F}{\|\mathcal{P}_{\widehat{\Omega}^c}(\bfX^*)\|_F}\le 10^{-2},
\end{equation}
where $\widehat{\Omega}^c$ is the complementary set of $\widehat{\Omega}$ over $\{1,2, \cdots, n_c\} \times \{1,2, \cdots, n\}$.
A white block means that all 100 trials are successful, while all trials fail in a black block.    
\begin{figure}[h]
		\centering
		\includegraphics[width=1.0\textwidth]{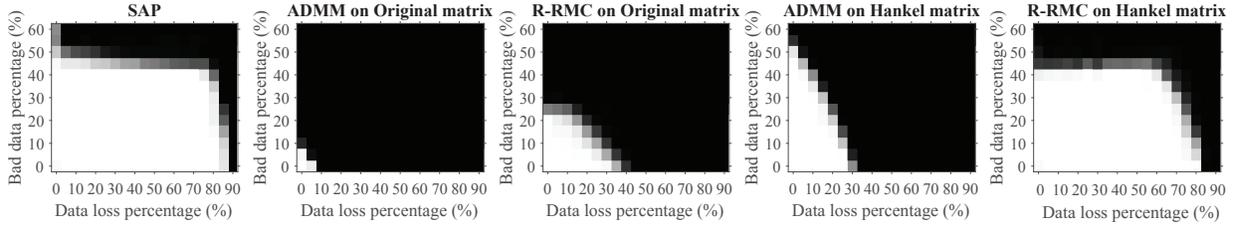}
		\centering
		\caption{Phase transition of SAP, ADMM and R-RMC under mode  M1$\times$B1} 
		\label{fig: Phase transition of SAP, ADMM and R-RMC under mode 1}
\end{figure}

\subsection{Performance of SAP}\label{section: performance of SAP}
In this experiment, we vary the rank and bad data percentage to test the performance of SAP for several combined modes, 
 M1$\times$B1, M2$\times$B2, and M2$\times$B3.   M1$\times$B1 means missing data model M1 and bad data mode B1. We only provide the simulation results of these three combined modes because
the performances of SAP are almost the same under modes M1$\times$B1, M2$\times$B1, M1$\times$B2 and M2$\times$B2. 
 The data loss percentage is fixed as $50\%$. 

  Fig.~\ref{fig: Phase transition of SAP with random outliers} shows the recovery performance  when    the angles of nonzero entries in $\bfS^*$ is randomly selected from $(0, {2\pi})$.
The $x$-axis is the bad data percentage, and the $y$-axis is the rank.
The results under M1$\times$B1 and M2$\times$B2 are included in  Fig.~\ref{fig: Phase transition of SAP with random outliers} to illustrate the similarity of SAP under these modes, and the similarity also shows that columnwise corruptions and missing entries do not affect the performance of SAP. 
Under mode M2$\times$B3, we test SAP under simultaneous and consecutive bad data. It can tolerate $9\%$ outliers for a rank-17 matrix, and $27$ out of $300$ consecutive columns are corrupted. 
\begin{figure}[h]
	\centering
	\includegraphics[width=0.9\textwidth]{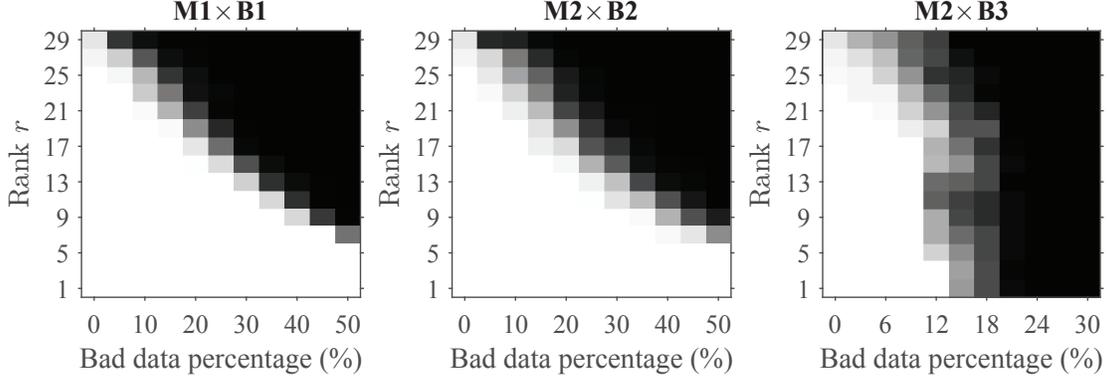}
	\centering
	\caption{Phase transition of SAP with random outliers} 
	\label{fig: Phase transition of SAP with random outliers}
\end{figure}

Fig.~\ref{fig: Phase transition of SAP with outliers restricted in  Quadrants I}   shows the recovery performance when the angles of nonzero entries in $\bfS^*$ is randomly selected from $(0, {\pi}/{2})$ such that both the real and imaginary parts of $\bfS^*$ are   positive.  Comparing  Figs.~\ref{fig: Phase transition of SAP with random outliers} and \ref{fig: Phase transition of SAP with outliers restricted in  Quadrants I},   one can see that SAP performs very similar when the corruptions have random signs and when the corruptions have positive signs. The recovery performance with random signs is slightly better in all three modes.

\begin{figure}[h]
	\centering
	\includegraphics[width=0.9\textwidth]{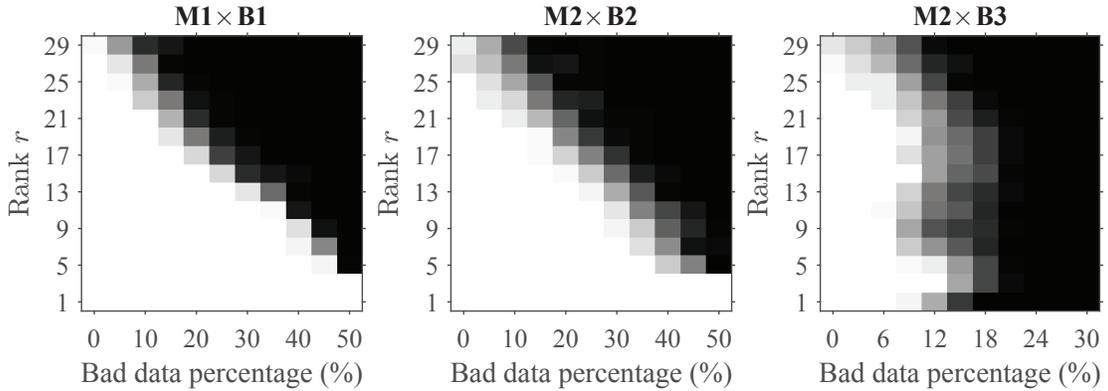}
	\centering
	\caption{Phase transition of SAP with outliers restricted in  quadrant  I} 
	\label{fig: Phase transition of SAP with outliers restricted in  Quadrants I}
\end{figure}
%
 
\subsection{Comparison with Existing RMC Methods}
We compare SAP with two other RMC methods to recover $\bfX^*$ from $\mathcal{P}_{\widehat{\Omega}}(\bfM)$. One is R-RMC \cite{CGJ16}, and the other is the convex relaxation of \eqref{eqn:OP_partial_observation} by relaxing rank and $\ell_0$-norm to the approximated convex nuclear norm and $\ell_1$-norm, and the convex optimization is solved by  Alternating Direction Method of Multipliers (ADMM) \cite{CLMW11}.
\footnote{We downloaded the codes from \textit{https://github.com/andrewssobral} for R-RMC 
	and  \textit{https://github.com/dlaptev/RobustPCA} for ADMM}
Under M1$\times$B1, we apply ADMM and R-RMC on both $\mathcal{P}_{\widehat{\Omega}}(\bfM)$ and the Hankel matrix $\mathcal{H}\big(\mathcal{P}_{\widehat{\Omega}}(\bfM)\big)$.
 Since ADMM and R-RMC cannot tolerate columnwise data losses or corruptions, they can not recover $\bfX^*$ under M2$\times$B2 and M2$\times$B3. Hence, we only test ADMM and R-RMC on $\mathcal{H}\mathcal{P}_{\widehat{\Omega}}(\bfM)$ under these two modes.  
The phase transitions in Fig. \ref{fig: Phase transition of SAP, ADMM and R-RMC under mode 1} are obtained by varying the data loss  and bad data percentages, and the rank is set as $5$.

\setcounter{figure}{5}
\begin{figure}[h]
	\centering
	\includegraphics[width=0.9\textwidth]{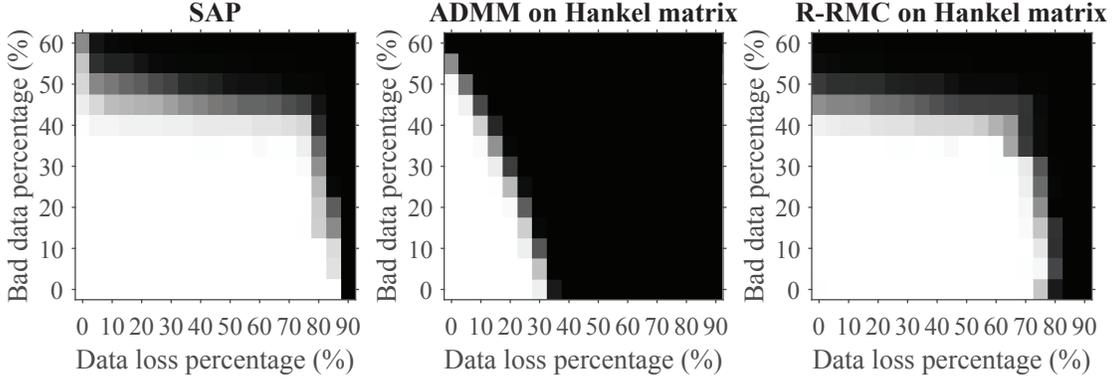}
	\centering
	\caption{Phase transition of SAP, ADMM and R-RMC on Hankel matrix under mode M2$\times$B2} 
	\label{fig: Phase transition of SAP, ADMM and R-RMC on Hankel matrix under mode 2}
\end{figure}
\begin{figure}[h]
	\centering
	\includegraphics[width=0.9\textwidth]{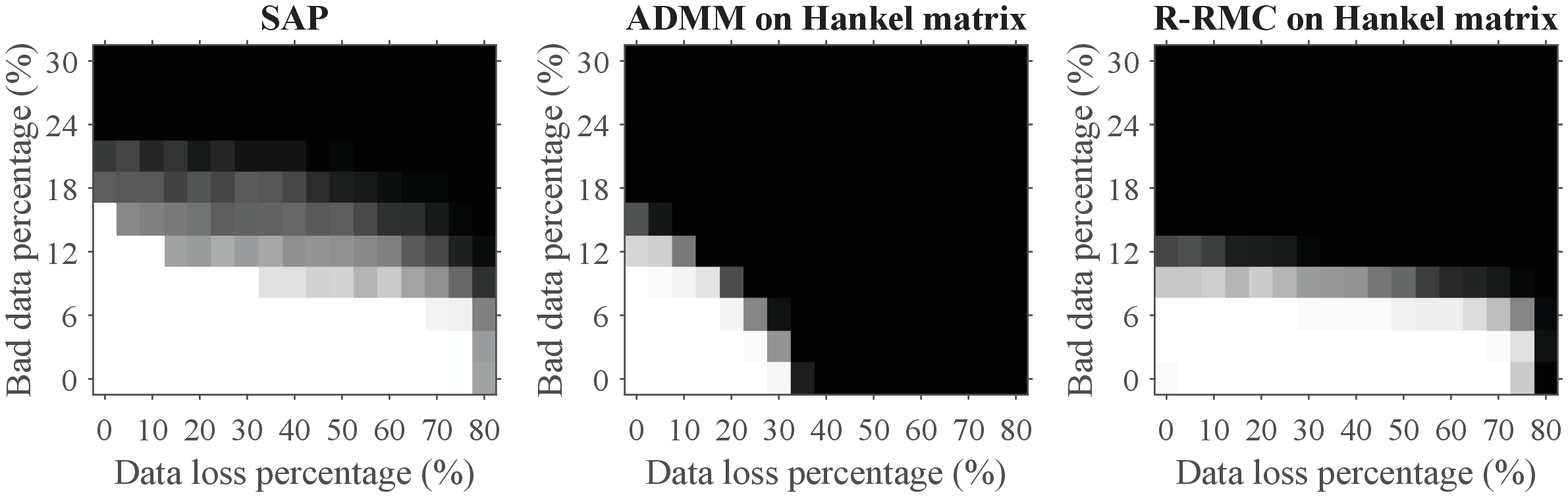}
	\centering
	\caption{Phase transition of SAP, ADMM and R-RMC on Hankel matrix under mode  M2$\times$B3} 
	\label{fig: Phase transition of SAP, ADMM and R-RMC on Hankel matrix under mode 3}
\end{figure}

From the results shown in Figs. \ref{fig: Phase transition of SAP, ADMM and R-RMC under mode 1}-\ref{fig: Phase transition of SAP, ADMM and R-RMC on Hankel matrix under mode 3}, under all these three modes, SAP performs the best among all methods. In Fig. \ref{fig: Phase transition of SAP, ADMM and R-RMC under mode 1}, ADMM and R-RMC are both applied on the original observed data matrix $\mathcal{P}_{\widehat{\Omega}}(\bfM)$, and the performances are much worse than SAP. When applying ADMM and R-RMC on the structured Hankel matrices, they both achieve higher success rates as shown in Figs. \ref{fig: Phase transition of SAP, ADMM and R-RMC under mode 1}, \ref{fig: Phase transition of SAP, ADMM and R-RMC on Hankel matrix under mode 2} and \ref{fig: Phase transition of SAP, ADMM and R-RMC on Hankel matrix under mode 3}. However, under modes M1$\times$B1 and M2$\times$B2, ADMM can only handle up to $30\%$ data loss even on the structured Hankel matrix, while SAP can still recover the data matrix under $80\%$ data loss. R-RMC performs slightly worse than SAP when applying on the structured Hankel matrix in modes 1 and 2, but SAP obtains a much larger rate of success in mode M2$\times$B3.

Moreover, SAP is significantly faster than ADMM and R-RMC on structured Hankel matrix as shown in Fig. \ref{fig: computational time of SAP, R-RMC}. We vary number of columns $n$ from $2000$ to $8000$ with a step size of $1000$, and the results are averaging over 100 successful independent trials for each $n$. Since the computational complexities of all these methods depend linearly on $n_c$, we keep $n_c=1$. 
The rank is fixed as $5$, and $n_1$ is set as $n/2$ throughout the experiments. The size of the  Hankel matrix is approximately $\frac{n}{2}\times\frac{n}{2}$.  We consider the mode of M1$\times$B1 where the locations of both bad data and miss data are generated randomly,
and the bad data percentage is set as $20\%$. Since the computational complexity of R-RMC depends on   the size of observed set, we study both $50\%$ and $95\%$ data loss percentages. The computational time of ADMM with $95\%$ data loss is not included since it does not converge in this setting. 
\begin{figure}[h!]
	\centering
	\includegraphics[width=0.9\textwidth]{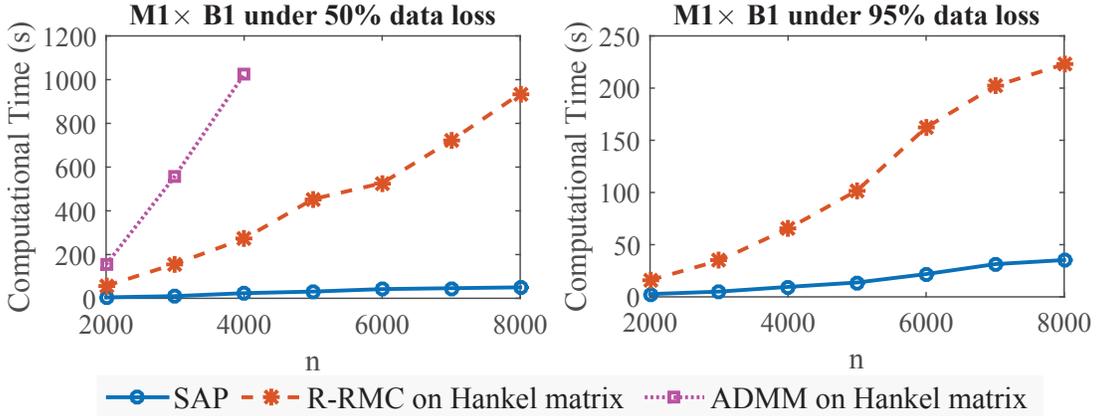}
	\centering
	\caption{Computational time of SAP, ADMM and R-RMC  on the structured Hankel matrix} 
	\label{fig: computational time of SAP, R-RMC}
\end{figure}

The computational time of SAP increases the least as the matrix size increases among all the methods. The convex method ADMM is the slowest with $50\%$ data loss.
ADMM takes over $1000$ seconds to recover the Hankel matrix of size $2000\times 2000$. R-RMC takes around $935$ seconds to recover the Hankel matrix of size $4000\times 4000$. 
With $95\%$ data loss, the computational time of R-RMC decreases by applying fast algorithms to compute the sparse matrix multiplication. 
It takes much more time than SAP.
 For example, SAP takes less than $40$ seconds to recover the Hankel matrix of size $4000\times 4000$, while R-RMC takes around $227$ seconds in the same setting. 

\subsection{Comparison with AM-FIHT in MC}
In this experiment, we compare SAP with AM-FIHT \cite{ZHWC18} to solve MC problem. We do not include other MC methods, such as SVT, because AM-FIHT is demonstrated to outperform other methods in both recovering accuracy and computational time \cite{ZHWC18}. 
We fix rank as 5. Since the number of observed entries for successful recovery depends on the conditional number $\kappa$, we consider both well-conditioned matrices, where $\kappa$ is small, and ill-conditioned matrices, where $\kappa$ is large. To generate a well-conditioned matrix, we just follow the same setup for generating $\bfX^*$ in the previous experiments. To generate a ill-conditioned matrix, we enlarges the amplitude of the first sinusoid $d_{1,i}$ by a factor of $r$ in all channels.

\begin{figure}[h]
	\centering
	\includegraphics[width=0.8\textwidth]{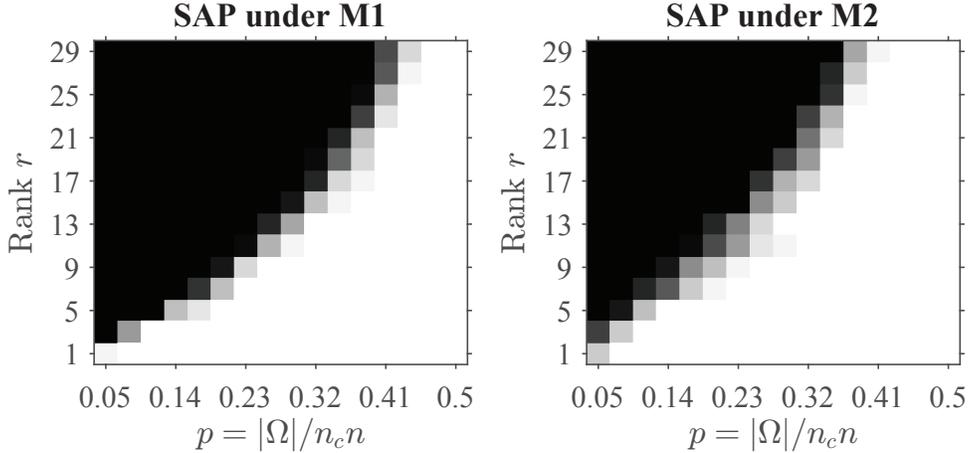}
	\centering
	\caption{Phase transition of SAP for ill-conditioned matrix} 
	\label{fig: Phase transition of SAP for ill-conditioned matrix}
\end{figure}

\begin{figure}[h]
	\centering
	\includegraphics[width=0.8\textwidth]{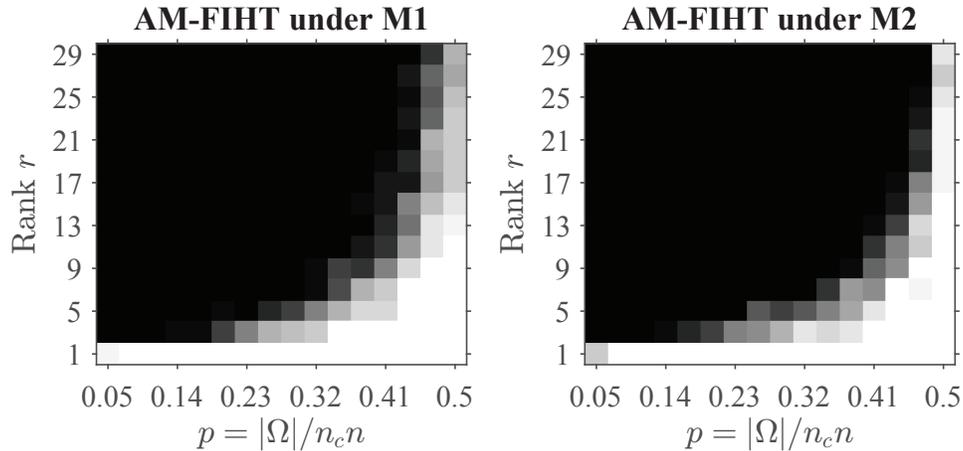}
	\centering
	\caption{Phase transition of AM-FIHT for ill-conditioned matrix} 
	\label{fig: Phase transition of AM-FIHT for ill-conditioned matrix}
\end{figure}

Fig. \ref{fig: Phase transition of SAP for ill-conditioned matrix} shows the performance of SAP when recovering ill-conditioned matrices. 
When the matrix is well-conditioned, both SAP and AM-FIHT perform very similarly. Moreover, SAP performs similarly on both well-conditioned and ill-conditioned matrices. 
This verify our result in \eqref{eqn:num_measurements_theoretical_bound} that the performance of SAP does not depend on $\kappa$.
We do not include the results of SAP and AM-FIHT in well-conditioned matrices because they are both similar to Fig. \ref{fig: Phase transition of SAP for ill-conditioned matrix}. When the matrix is ill-conditioned, AM-FIHT is much worse than SAP.

%% file: chapter_3/rpichap3.tex
 
\chapter{\uppercase{Training Convolutional Neural Networks with Generalizability Guarantees: A One-hidden-layer Case}}\label{chapter: 3}
\blfootnote{Portions of this chapter previously appeared as: S.~Zhang, M.~Wang, J.~Xiong, S.~Liu, and P.-Y. Chen, ``Improved linear
  convergence of training CNNs with generalizability guarantees: A
  one-hidden-layer case,'' \emph{IEEE Trans. Neural Networks Learn. Syst.}, vol.~32, no.~6, pp. 2622--2635, Jul. 2020.}
\input{./chapter_3/TNNLS-2019-P-12087-Intro.tex}
\input{./chapter_3/TNNLS-2019-P-12087-Problem_formulation.tex}
\input{./chapter_3/TNNLS-2019-P-12087-Theoretical.tex}

\input{./chapter_3/TNNLS-2019-P-12087-Simulation.tex}

\section{Summary}
We have  analyzed the  performance of (accelerated) gradient descent methods in learning one-hidden-layer non-overlapping convolutional neural networks with multiple nodes and ReLU activation function. We have shown that if the number of samples exceeds our provided sample complexity,  
gradient descent methods 	with the tensor initialization find the ground-truth parameters  with a linear convergence rate. The parameters can be estimated exactly when the data are noiseless. Moreover, accelerated gradient descent is proved to converge faster than vanilla gradient descent. {One future direction is to extend the analysis framework to multi-layer overlapping convolutional neural networks.}

%% file: chapter_3/TNNLS-2019-P-12087-Intro.tex
\section{Introduction}
Neural networks, especially convolutional neural networks (CNNs), have demonstrated superior performance in machine learning for  image classification \cite{KSHE12} and recognition \cite{LGTB97}, natural language processing \cite{CW08}, and   strategic game program \cite{SHA16}. Compared with fully connected neural networks,  CNNs   require fewer  coefficients and can better capture local features \cite{LBBH98}, and thus perform well in applications like image and video processing.


Learning a neural network needs to find appropriate parameters for  the hidden layers  using the training data  and is achieved by minimizing a non-convex  empirical   loss function  over the choices of the model parameters. 
The non-convex learning problem is usually solved by a first-order gradient descent (GD) algorithm. The convergence to the global optimal, however, is not guaranteed naturally due to the existence of spurious local minima.
Another major hurdle to the widespread acceptance of deep learning is the lack of analytical
performance guarantees about whether the parameters learned from the training data perform well on the
testing data, i.e., the generalizability of the learned model.
A learned model generalizes well to the testing data provided that it is a global minimizer of the population
loss function, which takes the expectation over the distribution of testing samples. Since the distribution
is unknown, one minimizes the empirical loss function of the training data assuming that the training data
are drawn from the same distribution.
 Moreover, a large number of
training samples are required to obtain a network model with powerful feature representation capability
\cite{DDSL09}, while the method may perform poorly when the number of training samples is small \cite{CJLJGP16}. The theoretical
characterization of the required size of the  training data for a given network architecture is vastly unavailable.

To analyze the learning performance, one line of research focuses on the over parameterized case  that the number of parameters in the   neural network is larger than the  number of training samples \cite{BE02,HRS16,KMNST16,LSS14,NBMS17,RHW88,SJL18,ZLL19}. In particular, the optimization problem has no spurious local minima \cite{LSS14,ZBHRV16,SJL18}, and GD methods can indeed find the global minimum of the empirical loss function. 
Nevertheless, the over-parameterized models may experience overfitting issues in practice \cite{YS19,ZBHRV16}. Moreover, when over-parameterized,  there is no guarantee by VC-dimension learning theory that the empirical loss function is close to the population loss and thus, the generalizability of the learned model to the testing data is unknown. Ref.~\cite{ZLL19} develops a new analysis tool to explore the generalizability under over-parameterization assumption. The convergence rate provided by  \cite{ZLL19} is sub-linear, and the sizes of neural networks increase as a polynomial function of the inverse of the desired testing error, which implies a high computational cost. Moreover, the  training error and the generalization error are analyzed separately,  and it is not clear if both a small training error and a small generalization error can be achieved simultaneously. 

Refs.~\cite{LB18,WGGC19} study the convergence to the global optimal for shallow neural networks when the data is linearly separable. Assuming the Rectified Linear Unit (ReLU)  activation function and the hinge loss function, 
ref.~\cite{WGGC19} can detect all the spurious local minima and saddle points, and the generalization error of the learned model approaches zero when the number of samples goes to infinite. 
However, if the data are linearly separable, simple algorithms, such as Perceptron \cite{R58}, can find a classifier in finite steps. Moreover, the detection method of the spurious local minima and saddle points in ~\cite{WGGC19} only apply 
the ReLU activation function and hinge loss function, and the method does not extend to other activation functions and loss functions.

One recent line of research assumes 
the existence of a ground-truth model    that maps the input data to the output data. Then the set of the ground-truth model parameters is a global minimizer of both the population and the empirical risk functions. The learning problem can be viewed as a model estimation problem.  
If the parameters are accurately estimated,  the generalizability to the testing data is guaranteed. 
This chapter follows this line of research. 

To simplify the analysis, one standard trick in this line of research is to assume that the number of input data is infinite so that the empirical loss function is simplified to the population loss function that is easier to analyze. 
Most existing  theoretical results are centered on one-hidden-layer shallow neural networks as the analyses quickly become intractable when the number of layers increases.   The input data are usually assumed to follow the Gaussian distribution   \cite{Sh18}  or some distributions that are rationally invariant \cite{DLT17}.
Refs. \cite{BG17,DLTPS17,Tian17} analyze   the landscape of  the population loss function 
of a simple  one-hidden-layer neural network with only one or two nodes  
 and show   that there exists a considerably large   convex  region near the global optimum. Then a  random initial point lies in this region with a constant probability, and gradient descent algorithms converge to the global minimum.   
 This result does not easily generalize to general neural networks  as  spurious local minima are fairly common for neural networks with even one hidden layer but multiple nodes \cite{SS17}.
Some works \cite{GLM17,LY17,LSL18} seek to obtain a good optimization landscape through changing the neural network structure. Ref.~\cite{LY17} adds an identity mapping   after the hidden layer to improve the convergence of GD algorithm.   An additional regularization term is added to the loss function in \cite{GLM17} such that the ground-truth parameters are still close to the global minimum, and spurious local minima are excluded.  
  An exponential node is added in each layer  of an arbitrary neural network such that all local minima are global minima  \cite{LSL18}. {Another work \cite{GKM18} developed a new iterative algorithm named Convotron, which applied a modified gradient desecnt update in each iteration and does not require initialization.}

In the practical case of a finite number of samples, the nice properties of the population loss function do not directly generalize to the empirical loss function. Some recent works study the training performance with a finite number of samples \cite{WDS19,FCL20,ZYWG18,ZSD17,ZSJB17}. 
If the number of samples is greater than a certain threshold, referred to the sample complexity, ref. \cite{WDS19} shows that the iterates converge to the ground-truth parameters for one-hidden-layer neural networks. However, the sample complexity is sub-optimal as it is a high order polynomial with respect to the dimension of the input data.
With the tensor initialization method \cite{ZSJB17},  GD algorithms are proved to converge to the ground-truth parameters linearly in one-hidden-layer neural networks, and the sample complexity is   nearly linear in the dimension of the input data \cite{FCL20,ZYWG18,ZSD17,ZSJB17}. However, the analyses in \cite{FCL20,ZSD17,ZSJB17} are limited to smooth activation functions {and exclude} the widely used   non-smooth activation function, ReLU. 
Among them, only ref.~\cite{ZYWG18} studies the ReLU activation function, 
but focuses on fully connected neural networks. 
Ref.~\cite{ZYWG18} can only guarantee the convergence to the ground truth up to some nonzero estimation error, even when the data are noiseless.  

The majority of the above works assume that data are noiseless, which may not be realistic in practice. 
	Only \cite{GLM17} and \cite{ZYWG18}   consider the cases that the output data contain additive noise that is independent of the input.   The noise is assumed to be zero mean in \cite{GLM17}, and the authors analyze the stochastic gradient descent through expectation. Thus, the noise does not affect their analyses and results. 
The result in \cite{ZYWG18} guarantees  the convergence of GD provided that the initialization  is sufficiently close to the ground-truth parameters, but no discussion is provided about whether the initialization in \cite{ZYWG18} satisfies this assumption or not. 


All the aforementioned works analyze standard GD algorithms. It is well known that  Accelerated Gradient Descent (AGD) methods such as Nesterov's accelerated gradient (NAG) method \cite{Nesterov13}  and Heavy ball method \cite{P87}  converge faster than vanilla GD. { However, the analyses for GD   do not generalize directly to AGD  because of the additional  momentum term introduced in AGD. 
Only refs. \cite{SMDH13} and \cite{YLL16} explore the  numerical performance of AGD  in  neural networks. No theoretical analysis  of AGD is reported in \cite{SMDH13}.  Ref.~\cite{YLL16} analyzes AGD from a general  optimization perspective, and it is not clear whether 
the   neural network learning problem satisfies the   assumptions in \cite{YLL16}. 
}

In this chapter, we provide novel contributions to the theoretical analyses of neural networks in three aspects. First, in this chapter, we provide the first theoretical analysis of AGD methods in learning neural networks. We prove analytically that the AGD method can converge to the ground-truth parameters linearly, and its convergence rate is faster  than vanilla GD. Second, it is the first work that explicitly proves the convergence of the proposed learning algorithm     to the ground-truth parameters  (or nearby) when the { data contain noise}. 
We characterize the relationship between the learning accuracy and the noise level quantitatively.  Our error bound is much tighter than that in \cite{ZYWG18}, and \cite{ZYWG18} makes assumptions about the initialization without any justification. 
In the special case of noiseless data,  our parameter estimation is exact, while the method in \cite{ZYWG18} is not.  Third,  it  provides the first tight generalizability  analysis of the   widely used   convolutional neural networks with the nonsmooth ReLU activation functions. Specifically, we prove that for one-hidden-layer non-overlapping convolutional neural networks, if initialized using the  tensor method, and the number of samples exceeds our characterized sample complexity,  both GD and ADG converge to a global minimum linearly    up to the noise level. Our sample complexity is order-wise optimal with respect to the dimension of the node parameters. Our estimation error bound of the ground-truth parameters is much tighter than a direct application of the existing results for  fully connected neural networks such as  \cite{ZYWG18} to CNN.   

%% file: chapter_3/TNNLS-2019-P-12087-Problem_formulation.tex
\section{Problem Formulation}
	Following \cite{ZSD17}, we consider the regression setup in this chapter as follows.
	Given $N$ input data $\bfx_n\in \mathbb{R}^{p}$, $n=1,2,\cdots, N$, that are independent and identically distributed (i.i.d.) from the standard Gaussian distribution $\mathcal{N}({\boldsymbol{0}},\bfI_{p\times p})$, the resulting outputs $y_n\in\mathbb{R}$, $n=1,2,\cdots, N$, are generated from $\{\bfx_n\}_{n=1}^{N}$ by a one-hidden-layer non-overlapping convolutional neural network shown in Fig. \ref{fig:CNN}. 
The hidden layer has $K$ nodes. 
We use the vector $\bfw^*_{j}\in \mathbb{R}^{d}$ to denote the weight parameters for the $j$-th node in the hidden layer and define the weight matrix $\bfW^*=\begin{bmatrix}
	\bfw^*_1,& \bfw^*_2,& \cdots,&\bfw^*_K \end{bmatrix}\in \R^{d\times K}$. 
	{Followed by the hidden layer, there is a pooling layer with ground-truth parameters $\bfv^*\in \mathbb{R}^{d}$.}
	 We assume $K<d$ throughout this chapter because $K$ is the constant, while $d$ increases as the dimension of the input data increases. 
		 $\sigma_{i} = \sigma_{i}(\bfW^*)$ denotes the $i$-th largest singular value of $\bfW^*$. We 
		define $\kappa = \sigma_{1}(\bfW^*)/\sigma_{K}(\bfW^*)$ as the conditional number of $\bfW^*$ and $\gamma=\displaystyle\Pi_{j=1}^{K}\big(\sigma_{j}(\bfW^*)/\sigma_{K}(\bfW^*)\big)$.
	\begin{figure}[h]
		\centering
		\includegraphics[width=0.65\linewidth]{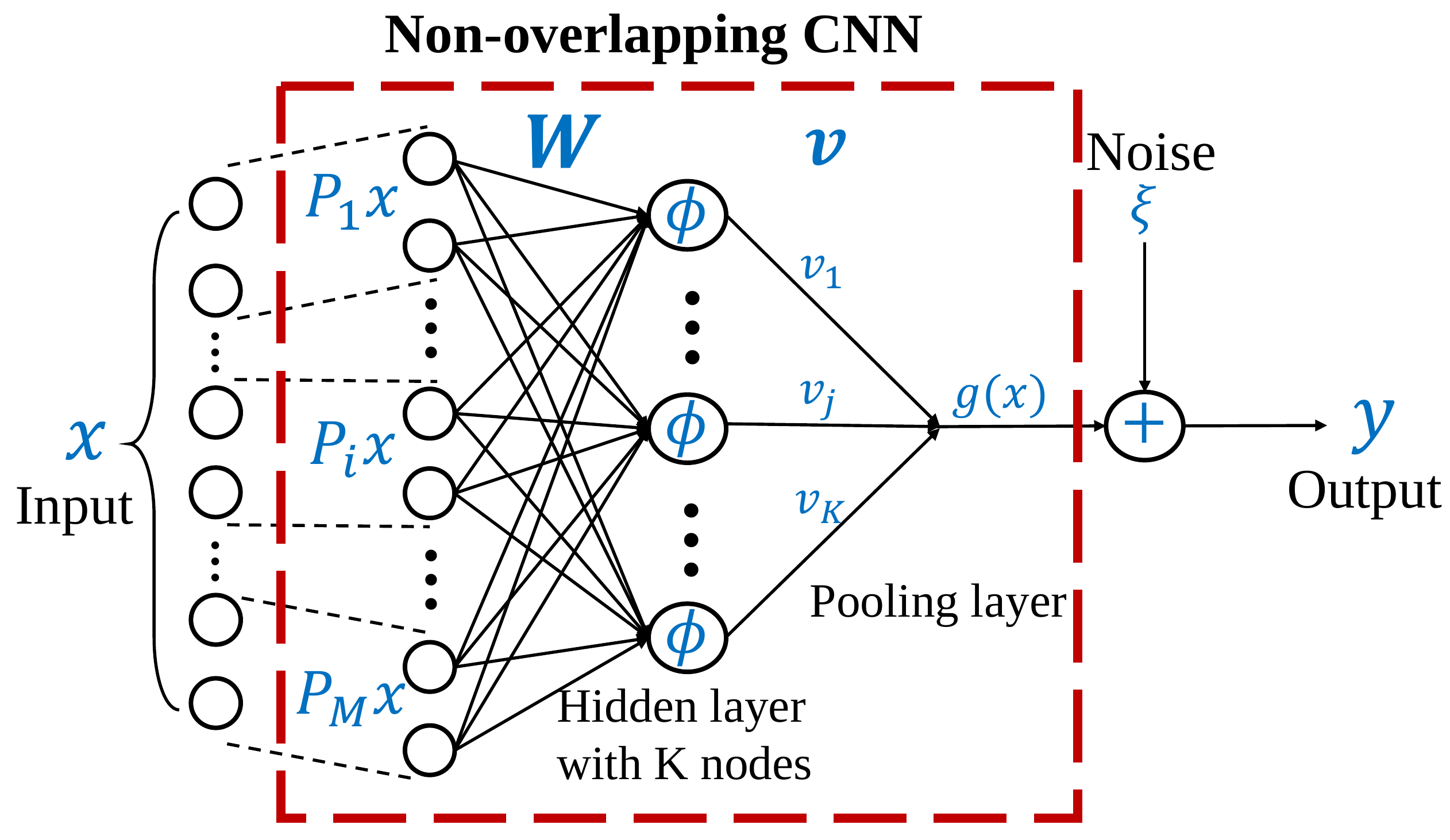}
		\centering
		\caption{One-hidden-layer non-overlapping CNN}
		\label{fig:CNN}
	\end{figure}
	
Each input data $\bfx_n$ is partitioned into $M$ non-overlapping patches, denoted by $\bfP_i\bfx_n\in\mathbb{R}^d$, $i=1,\cdots,M$.  
  $\bfP_i\in \mathbb{R}^{d\times p}$, $ i=1,\cdots,M$,  are a series of matrices that satisfy the following properties: (1) there exists one and only one non-zero entry with value $1$ in each row of $\bfP_i$; (2) $\langle\bfP_{i_1}, \bfP_{i_2} \rangle=0$ for $i_1\neq i_2$.
  \footnote{  Such requirement on $\bfP_i$ guarantees the independence of each patches and will be used in the proofs.}
   A simple   example of $\{\bfP_i\}_{i=1}^{M}$ is 
	\begin{equation*}
	\begin{split}
		&\bfP_i
		=\begin{bmatrix}
		\smash{\underbrace{\begin{matrix}{\boldsymbol{0}}_{d\times d}&\cdots&{\bf0}_{d\times d}\end{matrix}}_{\displaystyle  (i-1)  \text{ submatrices}}}
		&\begin{matrix}
			\bfI_{d\times d}
		 \end{matrix}
		&\smash{\underbrace{\begin{matrix}{\boldsymbol{0}}_{d\times d}&\cdots&{\bf0}_{d\times d}\end{matrix}}_{\displaystyle (M-i) \text{ submatrices}}}
		\end{bmatrix}.
	\end{split}
	\end{equation*}
\vspace{2mm}\\
	The output $y_n$ can be written as
	\begin{equation}\label{eqn: major_equ}
	y_n=g(\bfx_n)+\xi_n=\sum_{j=1}^{K}\sum_{i=1}^{M}v^*_j\phi({\bfw^*_j}^T\bfP_i\bfx_n) + \xi_n
	\end{equation}
	for $1\le n \le N$, where $\xi_n$ is the additive stochastic noise.

	Throughout this chapter, we assume bounded noise  with zero mean and  use  $|\xi|$ to denote the upper bound such that $|\xi_n|\leq|\xi|$ for all $n$. In practice, the mapping from the input to output data may not be  modeled exactly by a neural network due to the random fluctuations or measurement errors in the data. The additive  noise   better characterizes the relations in real datasets.
	
	The activation function $\phi(z)=\max\{z, 0 \}$ is the ReLU function, which is 
	widely used in various applications \cite{GBB11,HS01,YYG15,MHN13}.
	Note that if the activation function is homogeneous, such as ReLU, one can assume $v_j^*$ to be either $+1$ or $-1$ without loss of generality. That is 
	because $v_j^*\phi({\bfw_j^*}^T\bfP_i \bfx_n) = \text{sign}(v_j^*)\phi(|v_j^*|{\bfw_j^*}^T\bfP_i \bfx_n)$ for a homogeneous $\phi$. We can just let
	$\widetilde{\bfw}^*_j =|v_j^*|\bfw^*_j$  and    $\widetilde{v}^*_j=\text{sign}(v^*_j)$ and use $\{\widetilde{\bfw}^*_j\}_{j=1}^K$ and $\{\widetilde{v}^*_j\}_{j=1}^K$ as ground-truth parameters equivalently. Therefore,  we  assume $v_j^*\in\{+1,-1 \}$ for any $1\le j\le K$ throughout the chapter.	

	Given any estimated {$\bfW\in \R^{d\times K}$ and $\bfv\in \R^{K} $} of the weight matrix {$\bfW^*$ and $\bfv^*$}, the empirical squared loss function\footnote{{Besides the mean squared error, another  choice of the loss function, especially in classification problems, is the cross entropy, see. e.g.,   \cite{FCL20}.}} of the training set $\mathcal{D}=\{\bfx_n, y_n \}_{n=1}^N$ is defined as 
	{\begin{equation}\label{eqn: sample_risk}
	\begin{split}
	\hat{f}_{\mathcal{D}}(\bfW,\bfv)
	=&\frac{1}{2N}\sum_{n=1}^{N}\Big(\sum_{j=1}^{K}\bfv_j\sum_{i=1}^{M}\phi(\bfw_j^T\bfP_i\bfx_n)-y_n\Big)^2.
	\end{split}
	\end{equation}}Our goal is to estimate the ground-truth weight matrix { $\bfW^*$ and $\bfv^*$} via solving the following problem:
	{ \begin{equation}\label{eqn: optimization_problem}
		 \begin{split}
		 &\min_{\bfW\in \mathbb{R}^{d\times K},\bfv\in \mathbb{R}^{K}}:\quad \hat{f}_{\mathcal{D}}(\bfW,\bfv).
		 \end{split}
	\end{equation}}Clearly  { ($\bfW^*$, $\bfv^*$)} is a global minimizer  to \eqref{eqn: optimization_problem} when measurements are noiseless, i.e., $\xi_n=0$ for all $n$.  
	However, 
	\eqref{eqn: optimization_problem} is a non-convex optimization problem and is not easy to solve. 

%% file: chapter_3/TNNLS-2019-P-12087-Theoretical.tex
\section{Algorithm and Theoretical Results}
 We propose to solve  the non-convex problem \eqref{eqn: optimization_problem} via 
 the Heavy Ball method {\cite{P87}}.
 {The algorithm is initialized via  the tensor method \cite{ZSJB17}. Although the tensor initialization   is designed for fully connected neural networks in \cite{ZSJB17}, we can extend it to non-overlapping convolutional neural networks with minor changes. $\hat{\bfv}$ is estimated through the tensor initialization. 
 	During each iteration, we  update $\bfW$ through the AGD algorithm.} 
 Compared with the vanilla gradient descent, in the $(t+1)$-th iteration, an additional   momentum term, denoted by $\beta(\bfW^{(t)}- \bfW^{(t-1)})$, is added to the update, where $\bfW^{(t)}$ is the estimation in iteration $t$.  
 The momentum represents the direction of the previous iterations. 
Hence, besides moving along the gradient descent direction with a step size of  {$\eta$}, $\bfW^{(t)}$ is further moved along the direction of  previous steps with a parameter of $\beta$. 
During each iteration, a fresh subset of data is applied to estimate the gradient descent. Such disjoint subsets guarantee the independence of $\hat{f}_{\mathcal{D}_t}$ over the iterations. This is a standard analysis technique \cite{ZSD17,ZSJB17} and  not necessary in numerical experiments. The initialization algorithm is summarized  in Section \ref{Sec:IIIA3}, and Algorithm \ref{Alg4} summarizes our proposed algorithm
to solve \eqref{eqn: optimization_problem}. 
\begin{algorithm}[h]
	\caption{Accelerated Gradient Descent Algorithm with Tensor Initialization}\label{Alg4}
	\begin{algorithmic}[1]
		\State \textbf{Input:} training data $\mathcal{D}=\{(\bfx_n, y_n) \}_{n=1}^{N}$, gradient step size $\eta$, momentum parameter $\beta$, and thresholding error parameter $\varepsilon$;
		\State \textbf{Initialization:} {$\bfW^{(0)}, \hat{\bfv}$} through Tensor   Initialization via Subroutine 1;
		\State {Partition} $\mathcal{D}$ into $T=\log(1/\varepsilon)$ disjoint subsets, denoted as $\{\mathcal{D}_i\}_{i=1}^T$;
		\For  {$t=1, 2, \cdots, T$}		
		\State {$\bfW^{(t+1)}=\bfW^{(t)}-\eta\nabla \hat{f}_{\mathcal{D}_t}(\bfW^{(t)},\hat{\bfv})$$+\beta(\bfW^{(t)}-\bfW^{(t-1)})$}
		\EndFor\\
		\Return $\bfW^{(T)}$ and $\hat{\bfv}$.
		
	\end{algorithmic}
	
\end{algorithm}

\subsection{Initialization via Tensor Method}\label{Sec:IIIA3}
In this section, we first briefly introduce the tensor initialization method that is built upon Algorithm 1 in \cite{ZSJB17}. We then provide the first  theoretical performance guarantee of the tensor initialization method when the output contains noise in Lemma \ref{Thm: initialization}, while the result in \cite{ZSJB17} only applies to noiseless measurements. 

	The tensor initialization method in \cite{ZSJB17} is designed for  the fully connected neural networks. To handle the convolutional neural networks, the definitions of the high-order moments (see \eqref{eqn: M_13}-\eqref{eqn: M_33}) are modified by replacing $\bfx$ in  Definition 5.1 in \cite{ZSJB17} with $\bfP_i\bfx$. All the other steps mainly follow \cite{ZSJB17}.

 Following \cite{ZSJB17},  we define a special outer product, denoted by $\widetilde{\otimes}$. For any vector $\bfv\in \mathbb{R}^{d_1}$ and $\bfZ\in \mathbb{R}^{d_1\times d_2}$,   
 \begin{equation}
 	\bfv\widetilde{\otimes} \bfZ=\sum_{i=1}^{d_2}(\bfv\otimes \bfz_i\otimes \bfz_i +\bfz_i\otimes \bfv\otimes \bfz_i + \bfz_i\otimes \bfz_i\otimes \bfv),
 \end{equation} 
 where $\otimes$ is the outer product and $\bfz_i$ is the $i$-th column of $\bfZ$.
 Next, we pick any $i\in\{1, 2, \cdots, K\}$ and define
\begin{equation}\label{eqn: M_13}
\bfM_{i,1} = \mathbb{E}_{\bfx}\{y \bfx \}\in \mathbb{R}^{d},
\end{equation}
\begin{equation}\label{eqn: M_23}
\bfM_{i,2} = \mathbb{E}_{\bfx}\Big\{y\big[(\bfP_i\bfx)\otimes (\bfP_i\bfx)-\bfI\big]\Big\}\in \mathbb{R}^{d\times d},
\end{equation}
\begin{equation}\label{eqn: M_33}
\bfM_{i,3} = \mathbb{E}_{\bfx}\Big\{y\big[(\bfP_i\bfx)^{\otimes 3}- (\bfP_i\bfx)\widetilde{\otimes} \bfI  \big]\Big\}\in \mathbb{R}^{d\times d \times d},
\end{equation}
where $\bfz^{\otimes 3} := \bfz \otimes \bfz \otimes \bfz$, and $\mathbb{E}_{\bfx}$ is the expectation over $\bfx$.

{ 
From Claim 5.2 in \cite{ZSJB17},  there exist some known constants $\psi_i, i =1, 2, 3$, such that
\begin{equation}\label{eqn: M_1_23}
\bfM_{i,1} = \sum_{j=1}^{K} \psi_1\cdot v_j^*\|{\bfw}_{j}^*\|\cdot\overline{\bfw}_j^*,
\end{equation}
\begin{equation}\label{eqn: M_2_23}
\bfM_{i,2} = \sum_{j=1}^{K} \psi_2\cdot v_j^*\|{\bfw}_{j}^*\|\cdot\overline{\bfw}_{j}^* \overline{\bfw}_{j}^{*T},
\end{equation}\label{eqn: M_3_23}
\begin{equation}\label{eqn:M_3_v23}
\bfM_{i,3} = \sum_{j=1}^{K} \psi_3\cdot v_j^*\|{\bfw}_{j}^*\|\cdot\overline{\bfw}_{j}^{*\otimes3},
\end{equation} 
where $\overline{\bfw}^*_{j}={\bfw}_{j}^*/\|{\bfw}_{j}^*\|_2$ in \eqref{eqn: M_13}-\eqref{eqn: M_33} is the normalization of ${\bfw}_{j}^*$.}

 $\bfM_{i,1}$, $\bfM_{i,2}$ and $\bfM_{i,3}$ can be estimated through the samples $\big\{(\bfx_n, y_n)\big\}_{n=1}^{N}$, and let $\widehat{\bfM}_{i,1}$, $\widehat{\bfM}_{i,2}$, $\widehat{\bfM}_{i,3}$ denote the corresponding estimates. 
{First, we will decompose the rank-$k$ tensor $\bfM_{i,3}$ and obtain the $\{\overline{\bfw}^*_{j}\}_{j=1}^K$. By applying the tensor decomposition method \cite{KCL15} to $\widehat{\bfM}_{i,3}$, the outputs, denoted by $\widehat{\overline{\bfw}}_{j}^*$, are the estimations of $\{\bfs_j\overline{\bfw}^*_{j}\}_{j=1}^K$, where $s_j$ is an unknown sign.
Second, we will estimate $s_j$, $\bfv_j^*$ and $\|\bfw_j^*\|_2$ through $\bfM_{i,1}$ and $\bfM_{i,2}$. Note that $\bfM_{i,2}$ does not contain the information of $\bfs_j$ because $\bfs_j^2$ is always $1$. Then, through solving the following two optimization problem:
\begin{equation}\label{eqn: int_op3}
\begin{gathered}
\widehat{\boldsymbol{\alpha}}_1 = \arg\min_{\boldsymbol{\alpha}_1\in\mathbb{R}^K}:\quad  \Big|\widehat{\bfM}_{i,1} - \sum_{j=1}^{K}\psi_1 \alpha_{1,j} \widehat{\overline{\bfw}}^*_{j}\Big|,\\
\widehat{\boldsymbol{\alpha}}_2 = \arg\min_{\boldsymbol{\alpha}_2\in\mathbb{R}^K}:\quad  \Big|\widehat{\bfM}_{i,2} - \sum_{j=1}^{K}\psi_2 \alpha_{2,j} \widehat{\overline{\bfw}}^*_{j} {\widehat{\overline{\bfw}}_j^{*T}}\Big|,
\end{gathered}
\end{equation}
The estimation of $s_j$ can be given as $\hat{s}_j = \text{sign}(\widehat{\alpha}_{1,j}/\widehat{\alpha}_{2,j})$. Also, we know that $|\widehat{\alpha}_{1,j}|$ is the estimation of $\|\bfw_j^*\|$ and $\hat{v}_j = \text{sign}(\widehat{\alpha}_{1,j}/\bfs_j)$. Thus, 
$\bfW^{(0)}$ is given as 
 $[
 \text{sign}(\widehat{\alpha}_{2,1})\widehat{\alpha}_{1,1}\widehat{\overline{\bfw}}^*_{1}, \cdots,\\
  \text{sign}(\widehat{\alpha}_{2,K})\widehat{\alpha}_{1,K}
 \widehat{\overline{\bfw}}^*_{K}]
 $.
}


To reduce the computational complexity of tensor decomposition, one can project $\widehat{\bfM}_{i,3}$ to a lower-dimensional tensor \cite{ZSJB17}. The idea is to first estimate the subspace spanned by $\{\bfw_{j}^* \}_{j=1}^{K}$, and let $\widehat{\bfV}$ denote the estimated subspace. Then, from \eqref{eqn: M_33} and \eqref{eqn:M_3_v23}, we know that
$\bfM_{i,3}(\widehat{\bfV}, \widehat{\bfV}, \widehat{\bfV})\in\mathbb{R}^{K\times K \times K}$ is represented by 
\begin{equation}\label{eqn: M_3_v13}
\begin{split}
	 \bfM_{i,3}(\widehat{\bfV}, \widehat{\bfV}, \widehat{\bfV}) 
	=& \mathbb{E}_{\bfx} \Big\{ y\big[(\widehat{\bfV}^T\bfP_i\bfx)^{\otimes 3}
	   - (\widehat{\bfV}^T\bfP_i\bfx)\widetilde{\otimes} \bfI  \big]\Big\}\\
	=&\sum_{j=1}^{K} \psi_3(\widehat{\bfV}^T{\bfw}_{j}^*)\cdot(\widehat{\bfV}^T\overline{\bfw}_{j}^*)^{\otimes3}
\end{split}
\end{equation}
and can be estimated by training samples as well. Next, one can decompose the estimate $\widehat{\bfM}_{i,3}(\widehat{\bfV}, \widehat{\bfV}, \widehat{\bfV})$ to obtain unit vectors $\{\widehat{\bfu}_j \}_{j=1}^K \in \mathbb{R}^{K}$. Since $\overline{\bfw}^*$ lies in the subspace $\bfV$, we have $\bfV\bfV^T\overline{\bfw}_j^*=\overline{\bfw}_j^*$. Then, $\widehat{\bfV}\widehat{\bfu}_j$ is an estimate of $s_j\overline{\bfw}_j^*$. The initialization process is summarized in Subroutine 1.

 \floatname{algorithm}{Subroutine}
 \setcounter{algorithm}{0}
\begin{algorithm}[h]
	\caption{Tensor Initialization Method}\label{Alg: initia_CNN}
	\begin{algorithmic}[1]
		\State \textbf{Input:} training data $\mathcal{D}=\{(\bfx_n, y_n) \}_{n=1}^{N}$;
		\State Partition $\mathcal{D}$ into three disjoint subsets $\mathcal{D}_1$, $\mathcal{D}_2$, $\mathcal{D}_3$;
		\State Calculate $\widehat{\bfM}_{i,1}$, $\widehat{\bfM}_{i,2}$ {following \eqref{eqn: M_13}, \eqref{eqn: M_23}} using $\mathcal{D}_1$, $\mathcal{D}_2$, respectively;
		\State Obtain the estimate subspace $\widehat{\bfV}$ of $\widehat{M}_{i,2}$;
		\State Calculate $\widehat{\bfM}_{i,3}(\widehat{\bfV},\widehat{\bfV},\widehat{\bfV})$ {using \eqref{eqn: M_3_v1}} through $\mathcal{D}_3$;
		\State Obtain $\{ \widehat{\bfu}_j \}_{j=1}^K$ via {tensor decomposition method \cite{KCL15}};
		\State Obtain $\widehat{\boldsymbol{\alpha}}_1$, $\widehat{\boldsymbol{\alpha}}_2$ by solving  optimization problem $\eqref{eqn: int_op3}$;
		\State \textbf{Return:}  $\bfw^{(0)}_j=\text{sign}(\widehat{\alpha}_{2,j})\widehat{\alpha}_{1,j}\widehat{\bfV}\widehat{\bfu}_j$ and $\hat{\bfv}=\text{sign}(\widehat{\boldsymbol{\alpha}}_{2}) $, {$j=1,...,K$.}
	\end{algorithmic}
\end{algorithm}

\subsection{Parameter Estimation Through Accelerated Gradient Descent}
 {In this part, we provide the major theoretical results. 
 Lemma \ref{Thm: initialization} provides the first error bound of the initialization using the tensor initialization method in the presence of noise. 
Based on the tensor initialization method,
 Theorem \ref{Thm: major_thm} summarizes the recovery accuracy of $\bfW^*$ using Algorithm  \ref{Alg4}.
}
{
\begin{lemma}\label{Thm: initialization}
	Assume the noise level $|\xi|\le KM\sigma_{1}$ and the number of samples 
	\begin{equation}
	    N\ge C_1 \kappa^8M^2Kd\log^4 d    
	\end{equation}
	for some large positive constant $C_1$, the tensor initialization method in Subroutine 1 outputs 	$\hat{\bfv}$, $\bfW^{(0)}$  such that
	\begin{equation}
	\hat{\bfv} = \bfv^*,
	\end{equation}
	and
	\begin{equation}\label{eqn: ini}
	\|\bfW^{(0)} -\bfW^* \|_2\le C_2 \kappa^6 \sqrt{\frac{K^4d\log d}{N}}(KM\sigma_{1}+|\xi|)
	\end{equation}
	with probability at least $1-d^{-10}$.
\end{lemma}
}
\begin{theorem}\label{Thm: major_thm}
	{Let $\{ \bfW^{(t)} \}_{t=1}^T$ be the sequence generated in Algorithm \ref{Alg4} with $\eta= \frac{1}{12M^2K}$. }
	Suppose the noise level $|\xi|\le KM\sigma_{1}$ and the number of samples satisfies 
	\begin{equation}\label{eqn: sample_complexity}
		N\ge C_3 \varepsilon_0^{-2} \kappa^9 \gamma^3 M^3K^8d\log^4 d\log(1/\varepsilon)
	\end{equation} 
	 for some constants $C_3>0$ and $\varepsilon_0\in(0,\frac{1}{2})$. Then $\{ \bfW^{(t)} \}_{t=1}^{T}$ converges linearly to $\bfW^*$ with probability at least $1-K^2M^2T\cdot d^{-10}$ as  
	\begin{equation}\label{eqn: linear_convergence}
			\begin{split}
			\|	\W[t]-\bfW^*\|_2
			\le\nu(\beta)^t\|	\W[0]-\bfW^*\|_2
			+ C_4 \sqrt{\frac{\kappa^2\gamma MK^2d\log d}{N}} \cdot |\xi|,
			\end{split}
	\end{equation}
	and
	\begin{equation} \label{eqn: linear_convergence2}
			\begin{gathered}
			\|	\W[T]-\bfW^*\|_2
			\le \varepsilon \|\bfW^* \|_2 
			+ C_4 \sqrt{\frac{\kappa^2\gamma MK^2d\log d}{N}}\cdot |\xi|,
		\end{gathered}
	\end{equation}
	where $\nu(\beta)$ is the convergence rate that depends on $\beta$, and $C_4$ is some positive constant.
	Moreover, we have
	\begin{equation}\label{eqn: accelerated_rate}
	\nu(\beta)< \nu(0) \text{\quad for some small nonzero } \beta,
	\end{equation} 
	Specifically, let $\beta^* = \Big(1-\sqrt{\frac{1-\varepsilon_0}{132\kappa^2 \gamma KM}}\Big)^2$, we have
	\begin{equation}
	\begin{gathered}
	1-\frac{1-\varepsilon_0}{{132\kappa^2 \gamma KM}}\le \nu(0)\le 1-\frac{1-2\varepsilon_0}{{132\kappa^2 \gamma KM}},\\
	   \nu(\beta^*)\le 1-\frac{1-\varepsilon_0}{\sqrt{132\kappa^2 \gamma KM}}.	
	\end{gathered}
 	\end{equation}
\end{theorem}

\noindent
{\textbf{Remark 1 (Zero generalization error of learned model):}
Lemma \ref{Thm: initialization} shows that  the weight vector $\bfv^*$ of the second layer can be exactly recovered when the noise is bounded, and there exist enough samples. 
Theorem \ref{Thm: major_thm} shows that the iterates returned by Algorithm \ref{Alg} converge to $\bfW^*$ exactly in the noiseless case or approximately  in noisy case. For the convenience of presentation, we refer to the second   term  on the right-hand side of \eqref{eqn: linear_convergence} and \eqref{eqn: linear_convergence2} as the noise error term.
Specifically, when the relation of input $\bfx$ and the output $y$ can be exactly described by the CNN model, i.e., the   noise $\xi=0$, then the noise error term vanishes, and the ground-truth $\bfW^*$ can be estimated exactly with a finite number of samples. 
When the   noise is not zero,  the noise error term decreases as the number of samples $N$ increases in the order of $\sqrt{1/N}$. With a sufficiently large sample size, the iterates can approach $\bfW^*$ for an arbitrarily small error.}  
{With the number of samples satisfies \eqref{eqn: sample_complexity}}, 
{the second error term  on the right-hand side of \eqref{eqn: linear_convergence} 
 is proportional to the  noise magnitude $|\xi|$.
 From the definition of $g(\cdot)$, one can check that $\kappa KM\sigma_{1} \le\mathbb{E}_{\bfx}|g(\bfx)|\le KM\sigma_{1}$ when $\bfx$ follows $\mathcal{N}(0,1)$. Then the condition in Lemma \ref{Thm: initialization} and Theorem  \ref{Thm: major_thm}  
 that $|\xi|\le KM\sigma_{1}$ means that the noise   can be as high as the order of the average energy of the noiseless output  $g(\bfx)$.}
 
\noindent 
{\textbf{Remark 2 (Faster linear convergence rate than GD in learning neural networks):}  
Theorem \ref{Thm: major_thm} indicates that the Heavy Ball step can accelerate the rate of convergence as shown in \eqref{eqn: accelerated_rate}.  Without the second momentum term, i.e., $\beta=0$, the rate of convergence is $1-\Theta\big(\frac{1}{KM}\big)$ for the vanilla GD. If $\beta$ is selected appropriately, the rate of convergence is improved and upper bounded by $1-\Theta\big(\frac{1}{\sqrt{KM}}\big)$.  
This is the first attemp to provide theoretical guarantees for the convergence of AGD methods in learning neural networks.}  

\noindent 
{\textbf{Remark 3 (Sample complexity analysis):} 
 Theorem \ref{Thm: major_thm} requires  $O\big(M^3 K^8  d \log^4 d \log(\frac{1}{\varepsilon})\big)$ number of samples for the successful estimation. 
  $K$ is the number of nodes in the hidden layer and usually a fixed constant for a given neural network. $d$ is the dimension of patches and scales with the size of input data. $\varepsilon$ is the estimation error of $\bfW^*$. 
Note that the degree of freedom of $\bfW^*$ is $Kd$.   The required number of samples in Theorem \ref{Thm: major_thm} depends on $d\log^4 d$ and thus is nearly optimal with respect to $d$.}

\subsection{Comparisons with Related Works}
{

We compare our results with all the exiting works to the best of our knowledge that provide generalizability guarantees. We focus on the following three aspects.

\textbf{(1) Tensor initialization method and AGD algorithm:} Tensor initialization method is first introduced and analyzed in \cite{ZSJB17} for fully connected neural networks with homogeneous activation functions. Ref. \cite{FCL20} extends the analysis to  the  non-homogeneous   sigmoid activation. 
However, both works only consider noiseless settings. 
When  reduced  to the case of fully connected neural networks without noise, i.e., $\xi = 0$ and $M=1$,   the bound in \eqref{eqn: ini} is as tight as that in \cite{ZSJB17}. 

Existing works only consider the convergence of GD instead of AGD in neural networks. 
Due to the additional momentum term, the analysis of GD does not directly generalize to AGD. Specifically, the convergence of GD is based on establishing  $\|\bfW^{(t+1)}-\bfW^* \|_2\le \nu\|\bfW^{(t)}-\bfW^* \|_2$ for some $|\nu|<1$,  so this analysis does not directly apply to AGD. Instead, our analysis of AGD is based on the augmented iteration as 
$\begin{bmatrix}
\bfW^{(t+1)}-\bfW^*\\
\bfW^{(t)}-\bfW^*
\end{bmatrix}
$, and the convergence rate is calculated as a function of $\beta$. Note our analysis also applies to the special case that $\beta=0$, i.e., the GD algorithm. 
}

{\textbf{(2) Noisy outputs:}
Refs.~\cite{GLM17}, \cite{ZYWG18} consider noisy outputs in fully connected neural networks.
	In \cite{GLM17},   the authors analyze  stochastic gradient descent through expectation, and the noise is assumed to be zero mean. Thus, the noise level does not appear in the theoretical bounds. 
	In \cite{ZYWG18}, the authors  assume the existence of a proper initialization, but there is no theoretical guarantee in \cite{ZYWG18} about whether their proposed initialization method in the noisy setting   can return a desirable initialization. 
	Moreover, our error bound \eqref{eqn: linear_convergence} is tighter than that in \cite{ZYWG18}.  Specifically, the second term on the right-hand side of  \eqref{eqn: linear_convergence} only depends on noise factor $\xi$. In contrast, 
	eqn. (4.1) in \cite{ZYWG18}  shows
  that the GD algorithm  converges to $\bfW^*$ up to an estimation error that depends on both $\|\bfW^* \|_F$ and the noise level. 
  Even when there is no noise, the additional error term in eqn. (4.1) of \cite{ZYWG18} is nonzero. }

{
\textbf{(3) Theoretical guarantees:} 
As most existing works only focus on GD algorithm with noiseless outputs, we compare with these works  by   reducing to $\beta=0$ and $\xi = 0$ in Theorem \ref{Thm: major_thm}.
Refs. \cite{BG17,DLTPS17,FCL20,ZSD17} consider {one-hidden-layer} non-overlapping convolutional neural networks. 
Refs.~\cite{BG17} and \cite{DLTPS17} show that the GD algorithm converges to the ground-truth with a constant probability from one random initialization, but the result only applies to the case of one node in the hidden layer, i.e.,   $K=1$. Moreover, the analyses assume an infinite number of input samples and do not consider the sample complexity. 
 Based on the tensor initialization method \cite{ZSJB17},  refs. \cite{FCL20} and \cite{ZSD17}  show that the GD algorithm converges to the ground-truth with a linear convergence rate, but the result  only  applies to smooth activation functions, like sigmoid functions, and excludes ReLU functions. 
Refs.~\cite{GLM17,ZYWG18} provide the sample complexity analysis with ReLU activation function but focus on {one-hidden-layer} fully connected neural networks, which  can be viewed as a special case of the convolutional neural network studied in this chapter by selecting $M=1$.
The sample complexity in \cite{GLM17} with respect to $d$ is $\text{poly}(d)$, but the power of $d$ is not provided explicitly. 
Moreover, the convergence rate   in \cite{GLM17} is sub-linear, while our theorem shows that {both GD and AGD} enjoy    linear convergence rates.}

%% file: chapter_3/TNNLS-2019-P-12087-Simulation.tex
\section{Simulation}
The input data $\{\bfx_n\}_{n=1}^{N}$ are randomly selected from the Gaussian distribution $\mathcal{N}( 0 , \bfI )$. 
The number of patches $M$ is selected as  
a factor of the signal dimension $p$,  and all the patches have the same size $d$ with $d=p/M$. 
Entries of the weight matrix $\bfW^*$ are i.i.d generated from   $\mathcal{N}(0, 1^2)$. 
The noise $\{\xi_n \}_{n=1}^N$ are i.i.d from $\mathcal{N}(0, \sigma^2)$,
 and the noise level is measured by $\sigma/E_y$, where $E_y$ is the average energy of the noiseless outputs $\{g(\bfx_n)\}_{n=1}^{N}$ calculated as $E_y = \sqrt{\frac{1}{N}\sum_{n=1}^{N}|g(\bfx_n)|^2}$.   
{The output data $\{y_n\}_{n=1}^N$ are generated by \eqref{eqn: major_equ}.  }  
In the following numerical experiments, {the whole dataset $\{\bfx_n, y_n \}_{n=1}^N$ instead of a fresh subset is used  to calculate the gradient in each iteration.}
The initialization is randomly selected from $\big\{ \bfW_0\big| \| \bfW_0 - \bfW^* \|_F/\|\bfW^* \|_F <0.5 \big\}$ and $\bfv^{(0)} =\bfv^*$ to reduce the computation. As shown in \cite{FCL20,ZYWG18},
random initialization and the tensor method have very similar numerical performance.

If not otherwise specified, we use the following parameter setup. 
 $p$ is chosen as $50$, and $M$ is selected as $5$. Hence, $d=p/M$  is $10$.  
The number of nodes in hidden layer $K$ is chosen as $5$.
The number of samples $N$ is chosen as $200$. 
The step size of the gradient {$\eta$} is   $\frac{2K}{M^2}$, and $\beta$ is selected as $(1-\frac{1}{\sqrt{KM}})^2$. 
All the simulations are implemented in MATLAB 2015a on a desktop with 3.4 GHz Intel Core i7.
 
\subsection{Performance of AGD with Different $\bfv^*$} 
Figs. \ref{Fig: sign_1} and \ref{Fig: sign_2} show the performance of AGD with different $v_j^*$, and the results are averaged over 100 independent trials.
In Fig. \ref{Fig: sign_1}, the relative error is defined as 
$	{\| \bfW^{(t)}- \bfW^*  \|_F}/{\|\bfW^*\|_F}$, where $\bfW^{(t)}$ is the estimate in the $t$-th iteration. 
In Fig. \ref{Fig: sign_2}, each trial is called a success if the relative error is less than $10^{-6}$.
We generate two cases of $\bfv^*$. In Case 1, all the entries of $\bfv^*$ are $1$, while  each entry is i.i.d. selected from $\{+1, -1 \}$ with equal probability in Case 2. $k$ is set as $5$, and $d$ is set as $60$ with $p=300$. In both figures, the results of Case 1 is shown by the lines marked as ``$v_j=+1$'', and the second group is marked as ``$v_j\in\{+1, -1 \}$''.  We can see that the performances of these two cases are almost the same. 
In the following experiments, we fix  $\bfv_j^*$ as $1$ for all $j$.
\begin{figure}[h] 
 		\centering
 		\includegraphics[width=0.6\linewidth]{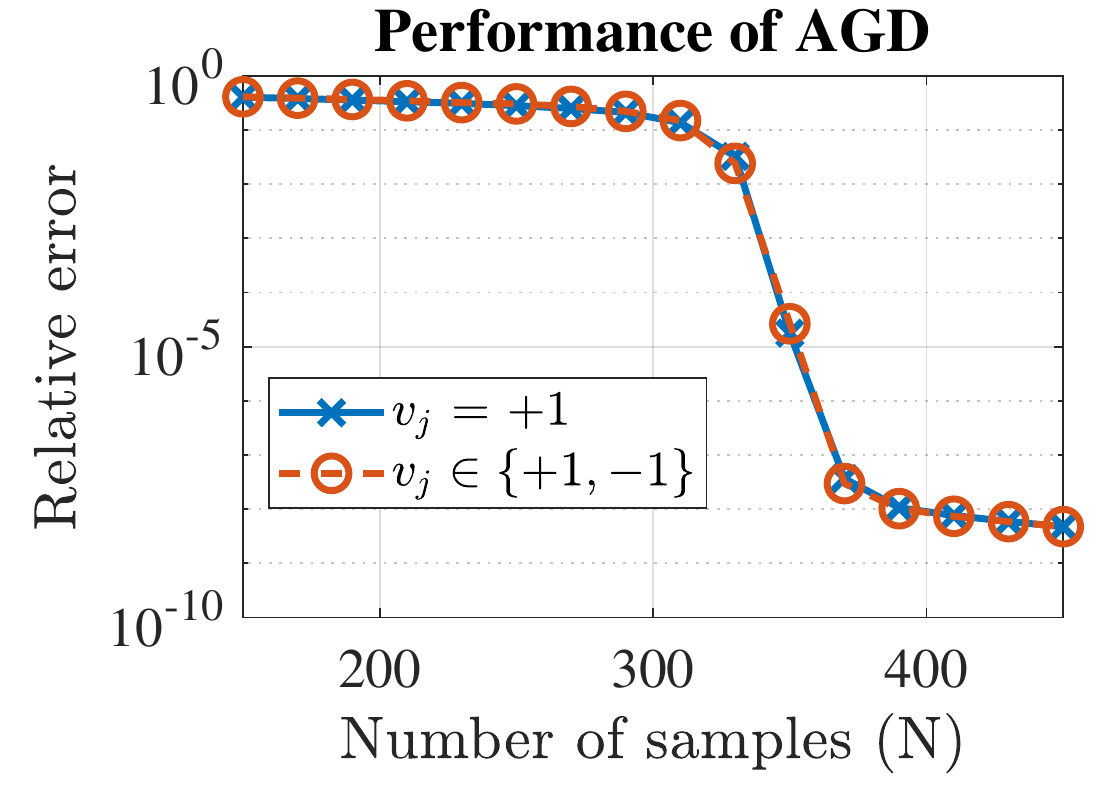}
 		\vspace{-2.5mm}
 		\caption{Recovery error of AGD under different $\bfv^*$}
 		\label{Fig: sign_1}
 \end{figure}
 \begin{figure}[h]
 		\centering
 		\includegraphics[width=0.6\linewidth]{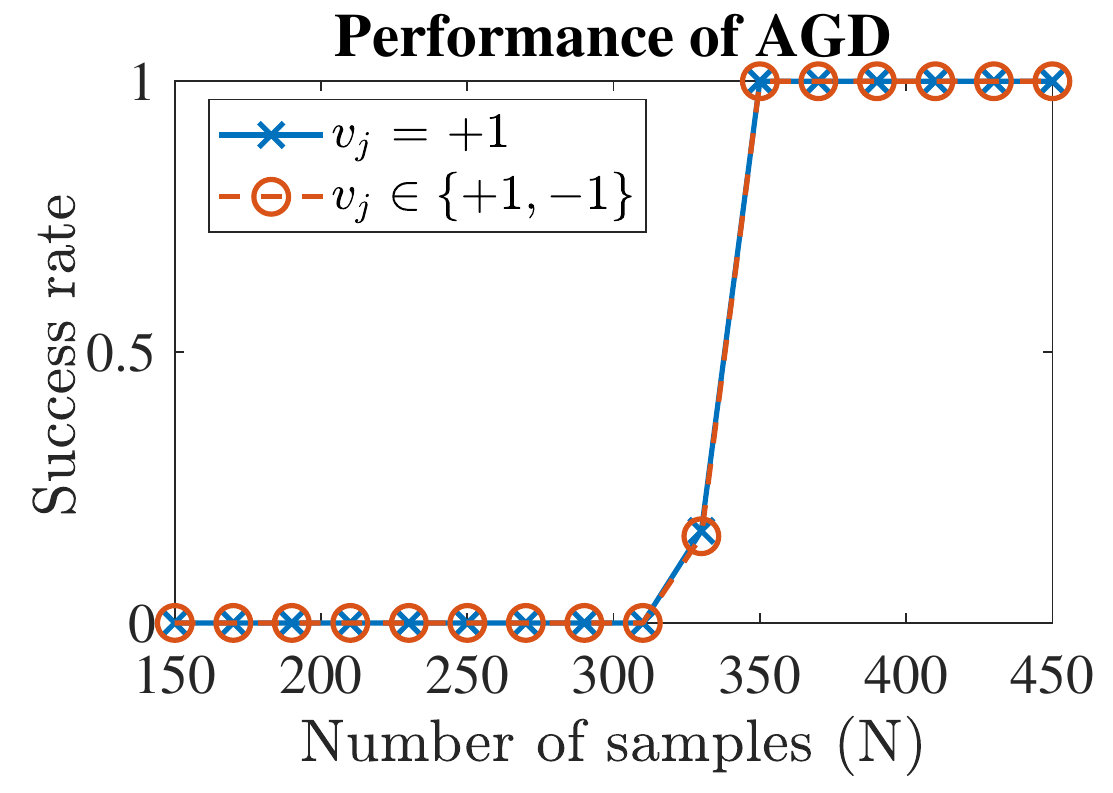}
 		\vspace{-2.5mm}
 		\caption{Success rate of AGD under different $\bfv^*$}
 		\label{Fig: sign_2}
 \end{figure}
\subsection{Performance of AGD with Noiseless Output}
Figs. \ref{Fig: convergence_vs_K} and \ref{Fig: convergence_vs_M} show the convergence of AGD by varying $K$ and $M$. 
In Fig. \ref{Fig: convergence_vs_K},   $\eta$, $\beta$ are calculated based the value of $K$, and other parameters are fixed.  
For each $K$, we conducted independent trials with random selected $\bfx_n$, $\bfW^*$ and the corresponding $y_n$. 
Given $K$, the convergence rates of different trials vary slightly. Fig. \ref{Fig: convergence_vs_K} shows one example of these trials for each $K$.  
We can see that the convergence rate decreases as $K$ increases. Similarly, Fig. \ref{Fig: convergence_vs_M} shows that the convergence rate decreases as $M$ increases. 
\begin{figure}[h] 
		\centering
		\includegraphics[width=0.6\linewidth]{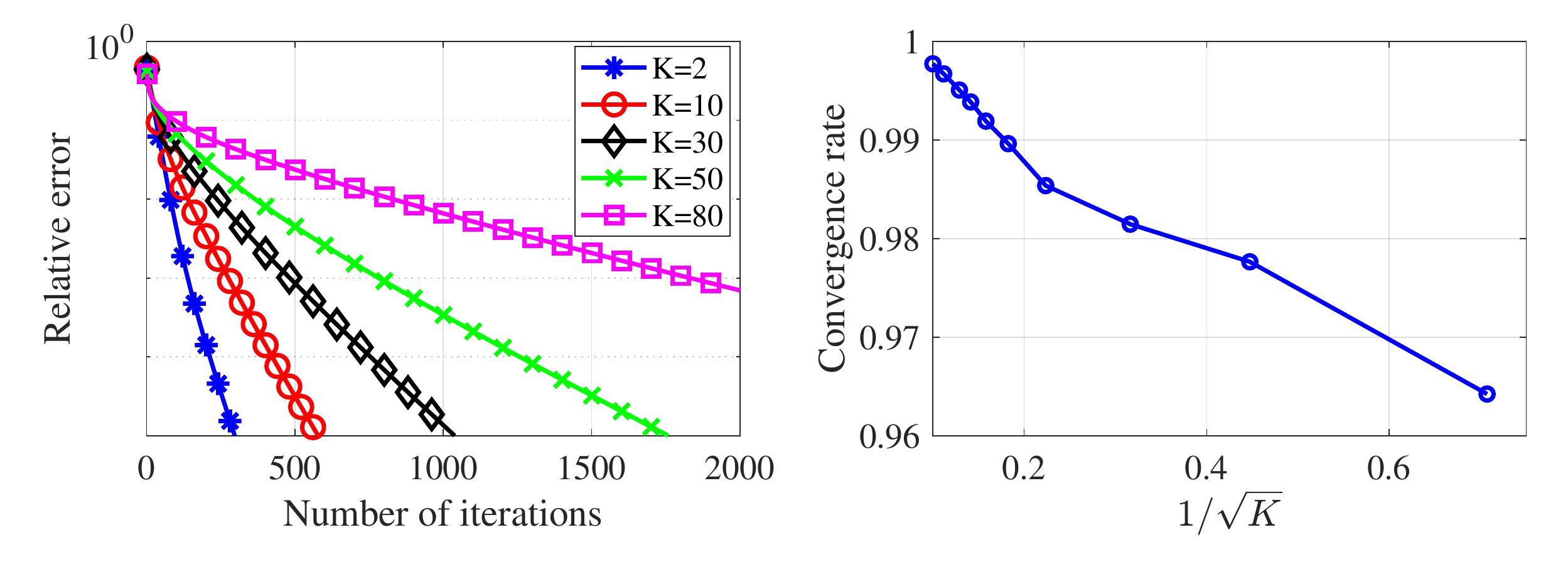}
		\vspace{-3mm}
		\caption{Convergence of AGD with different $K$}
		\label{Fig: convergence_vs_K}
\end{figure}
\begin{figure}[h] 
		\centering
		\includegraphics[width=0.6\linewidth]{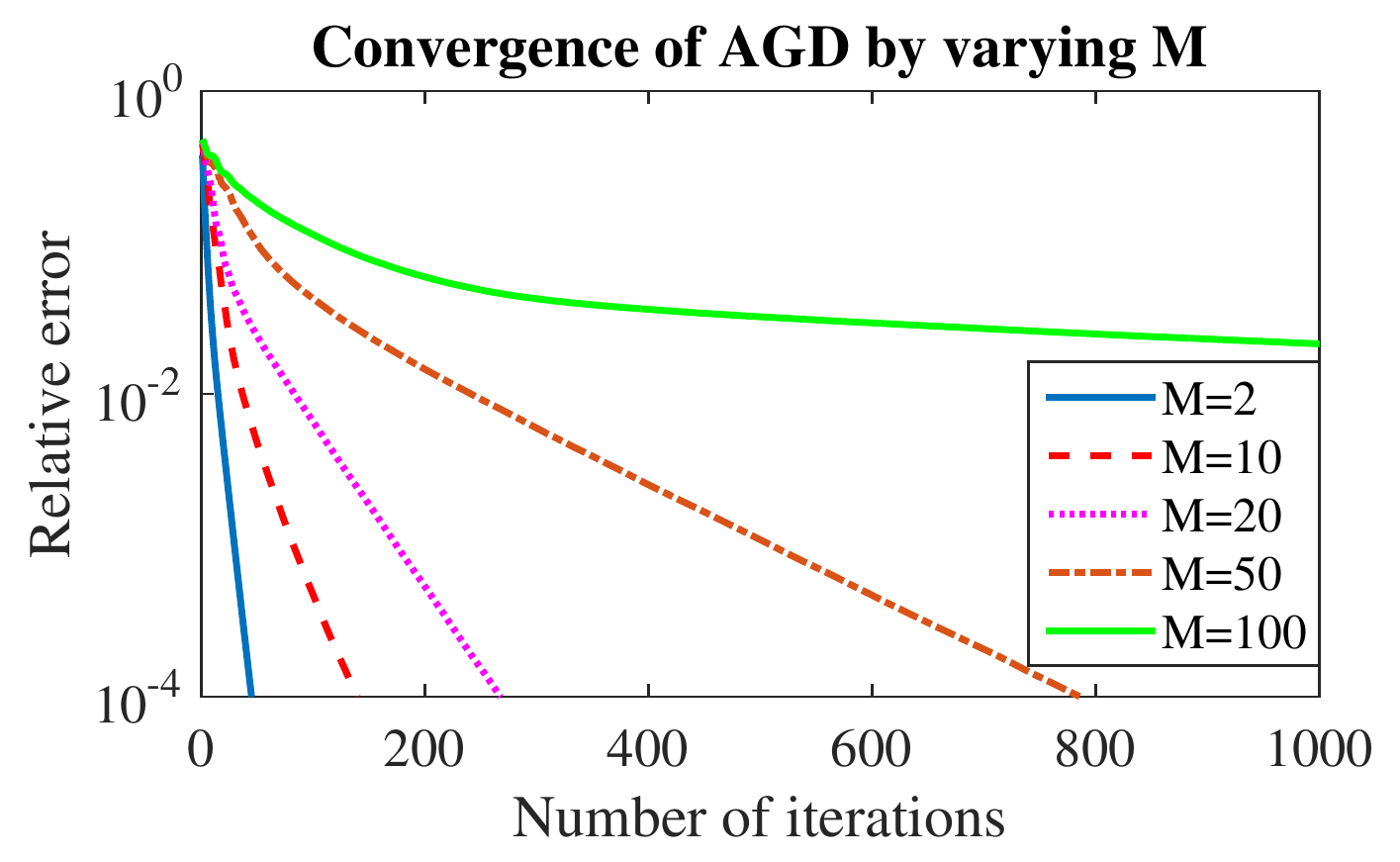}
		\vspace{-2.4mm}
		\caption{Convergence of AGD with different $M$}
		\label{Fig: convergence_vs_M}
\end{figure}

Figs. \ref{Fig: ps of N_vs_d} and \ref{Fig: ps of N_vs_K} show  the phrase transition where the number of samples $N$, the dimension of patches $d$, and the number of nodes in the hidden layer $K$ change. 
All the other parameters except $N$ and $d$ $(\text{or }k) $ remain the same as the default values.
For each $({N}, {d})$ or $(N, K)$ pair, we conduct $100$ independent trials.  Each trial is called a success if the relative error is less than $10^{-6}$. 
A white block means all the trails are successful, while a black one means all the trials fail. 
\begin{figure}[h] 
		\centering
		\includegraphics[width=0.6\linewidth]{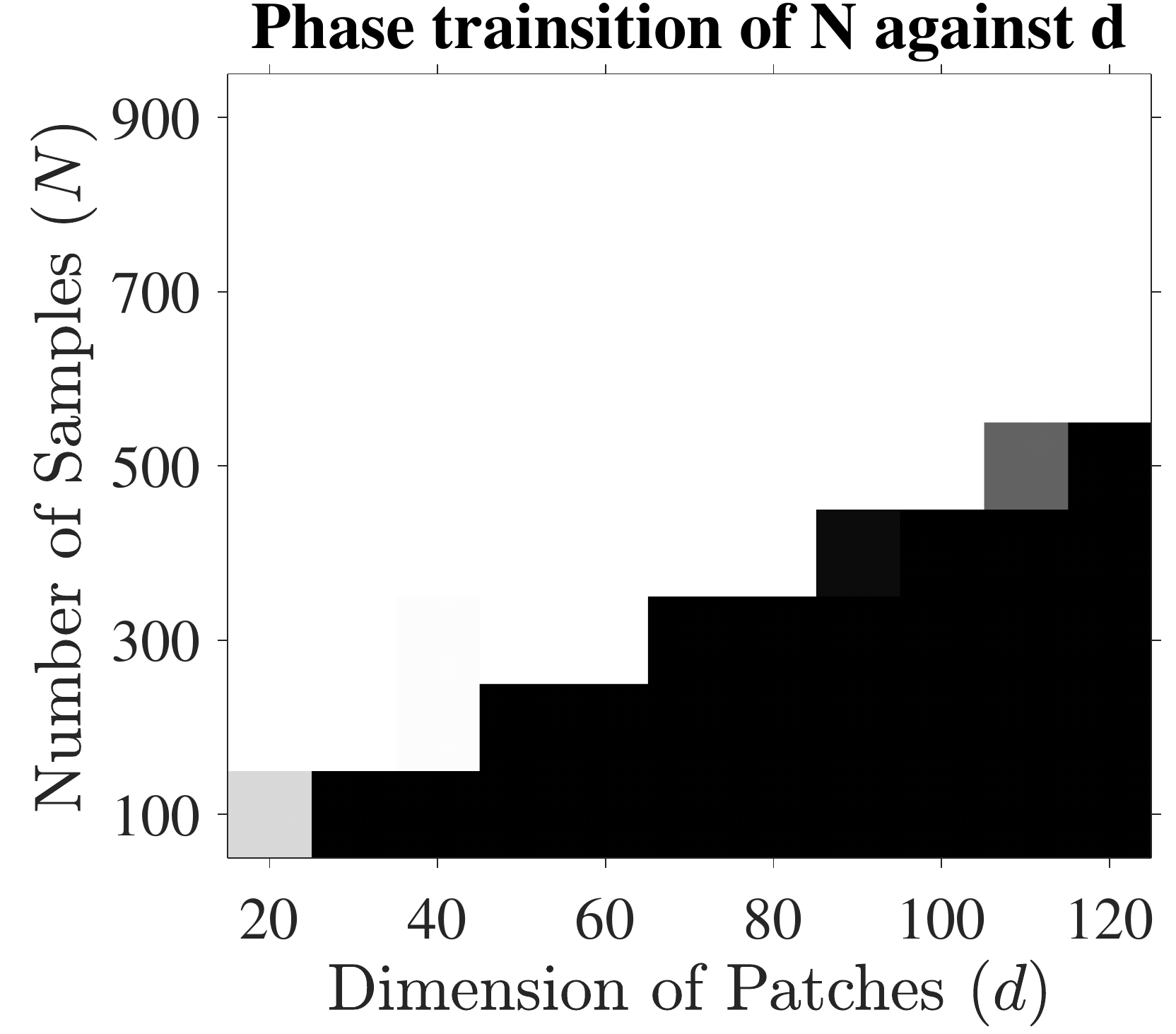}
		\caption{Phrase transition of $N$ against $d$}
		\label{Fig: ps of N_vs_d}
\end{figure}
\begin{figure}[h] 
		\centering
		\includegraphics[width=0.6\linewidth]{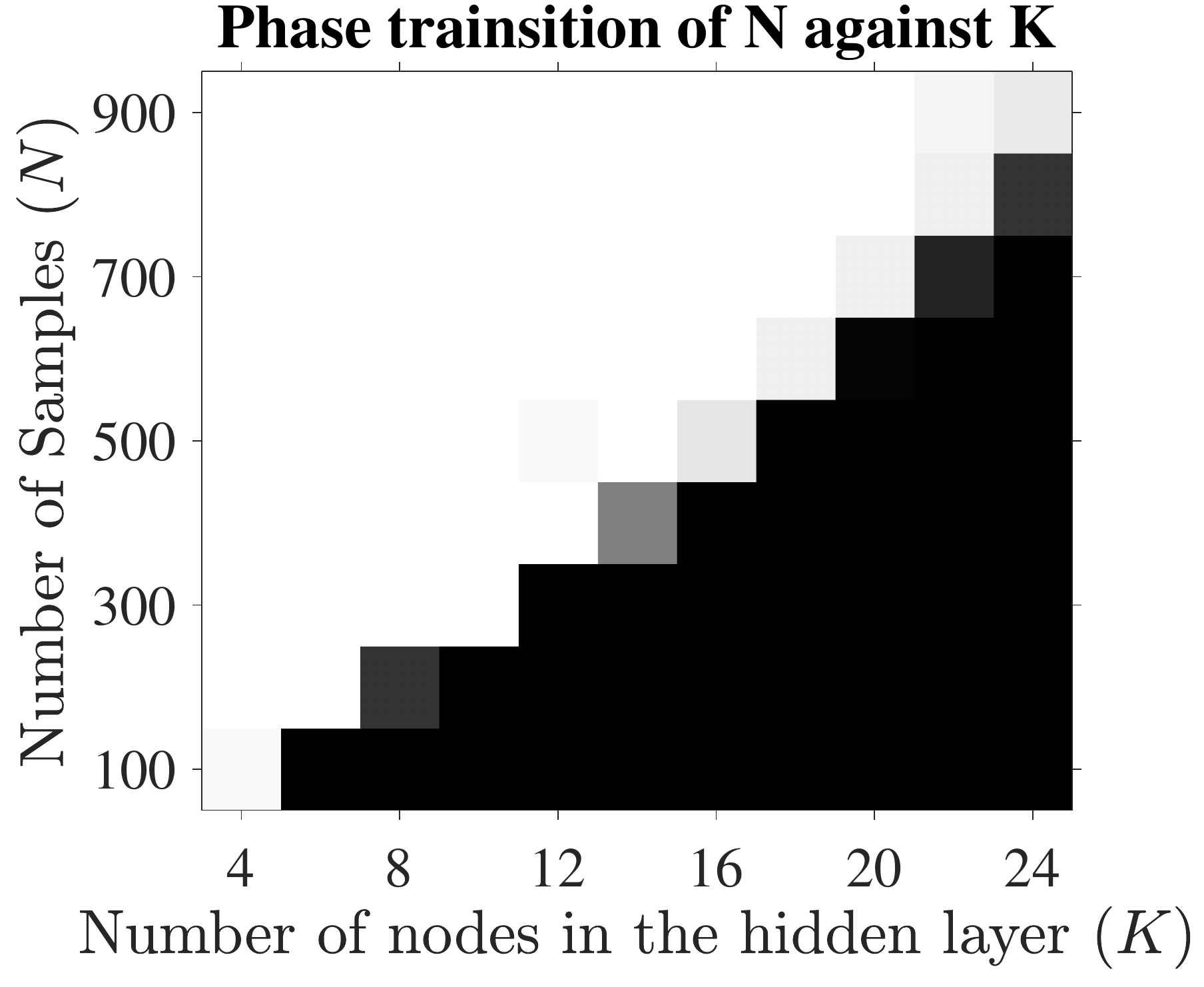}
		\caption{Phrase transition of $N$ against $K$}
		\label{Fig: ps of N_vs_K}
\end{figure}

\subsection{Performance of AGD with Noisy Output}
Fig. \ref{fig:err_vs_noise} shows the relative error of AGD algorithm by varying the number of samples $N$ in the noisy case. 
$K$ is set as $5$, and  $d$ is set as $60$ with $p=300$. Hence, the degree of freedom of $\bfW^*$ is $300$.
Y-axis stands for the relative error, and the results are averaging over $100$ independent trials. 
We can see that the relative errors are high when $N$ is less than the degree of freedom as $300$. Once the number of samples exceeds the degree of freedom, the relative error decreases dramatically in both noisy and noiseless settings. As $N$ increases,  the relative error in the noisy setting  converges fast to  the noise level. 

\begin{figure}[h]
	\centering
	\includegraphics[width=0.5\linewidth]{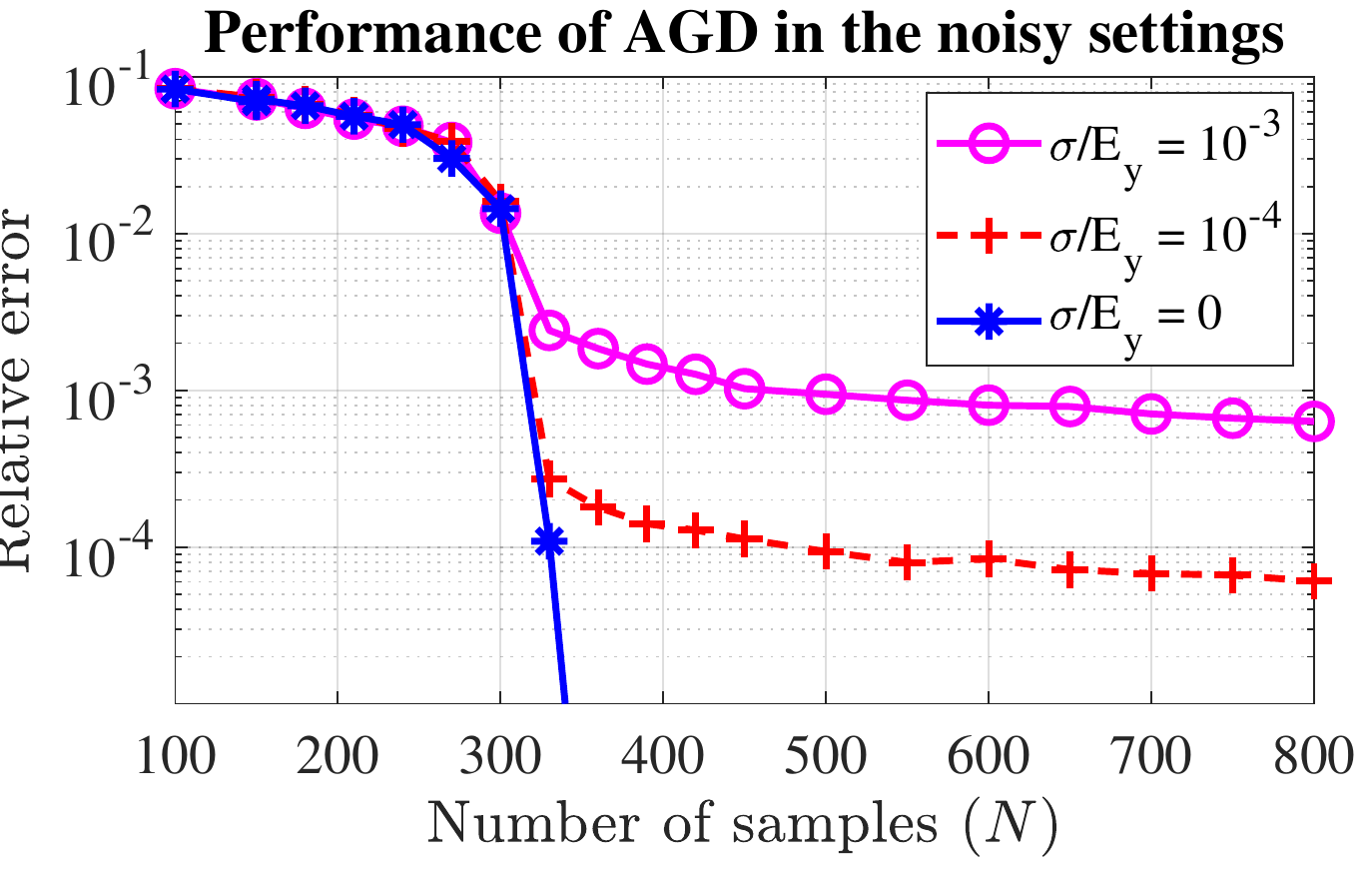}
	\caption{The performance of Algorithm \ref{Alg4} with noisy measurements}
	\label{fig:err_vs_noise}
\end{figure}

Fig.~\ref{Fig: ps_noise} shows the phrase transition of $N$ against $d$ with different noise levels.
A trial is considered successful if the returned $\bfW$ satisfies $\|\bfW - \bfW^*\|_2/\|\bfW^* \|_2\le \sigma/E_y$ (or $10^{-6}$ in noiseless settings). 
 As  $d$ increases, the required number of samples for all successful estimations increases as well. 
 Also, with a higher noise level, the success region becomes smaller. 
\begin{figure}[h] 
	\centering
	\includegraphics[width=1.0\linewidth]{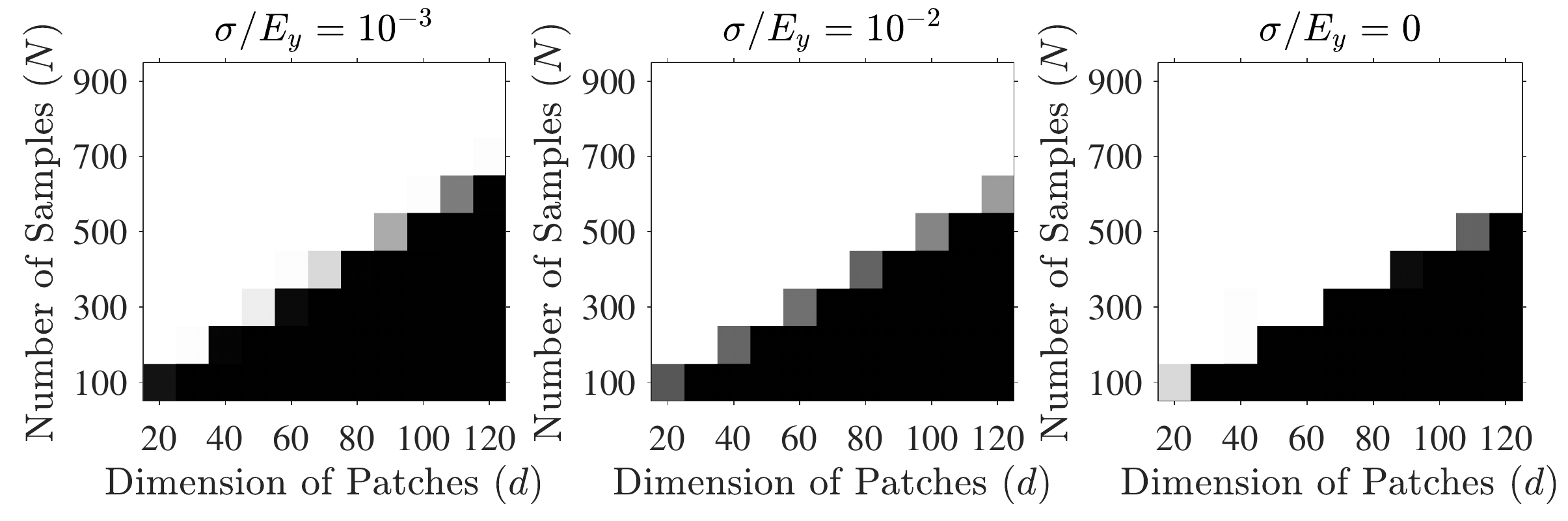}
	\caption{The phrase transition of AGD  in noisy settings}
	\label{Fig: ps_noise}
\end{figure}

\subsection{Comparison of GD and AGD}
Fig.~\ref{fig:CASE} shows the progress of both GD and AGD methods across iterations.  We fix the same initialization for GD and AGD in  Fig.~\ref{fig:CASE}(a) and (b), respectively. In both cases, $\beta$ and other parameters except for $\eta$ are fixed as the default values.  The only difference is that the step size $\eta$ is $\frac{2K}{M^2}$ in  Fig.~\ref{fig:CASE}(a) and $\frac{3K}{M^2}$ in Fig.~\ref{fig:CASE}(b).  One can see that starting from the same initialization,  GD sometimes diverges in (b) with a large step size. By adding the heavy-ball term, the AGD method can converge to the global minimum. Moreover, when both GD and AGD converge, AGD converges faster than GD.

\begin{figure}[H] 
	\begin{minipage}[c]{0.48\textwidth}
		\centering
		\includegraphics[width=0.9\linewidth]{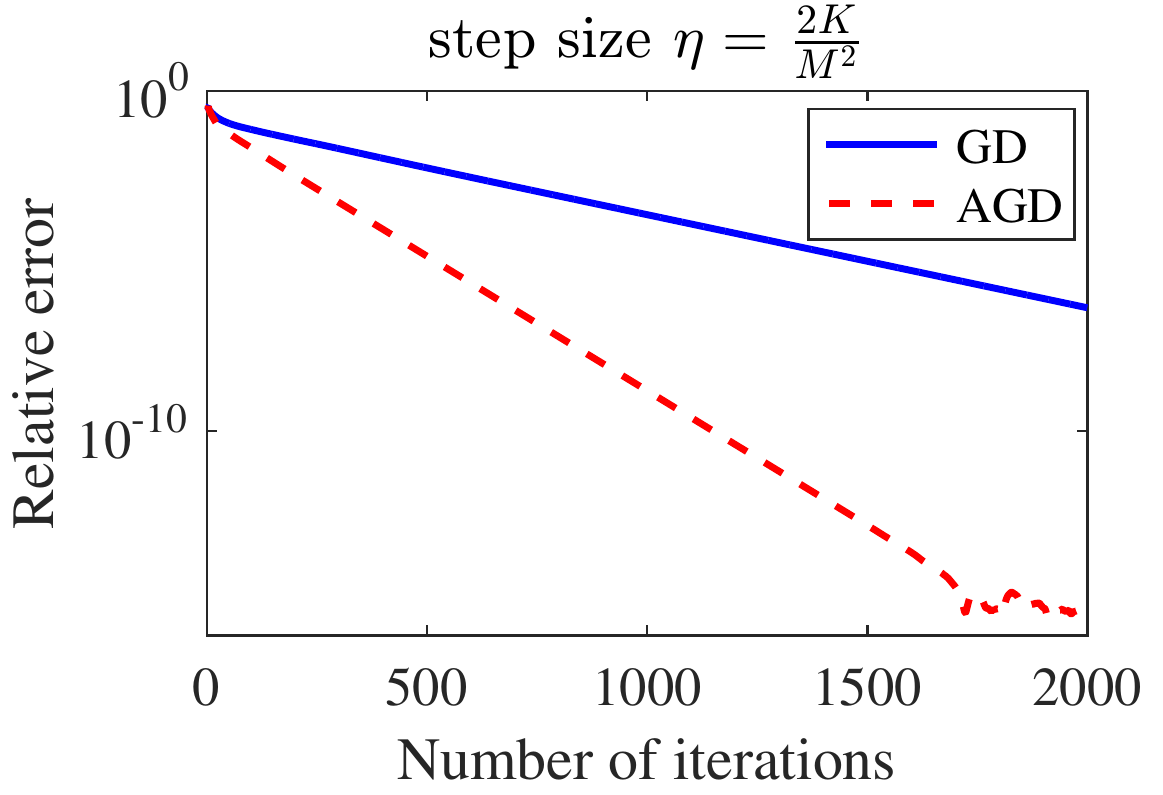}
		\vspace{-4mm}
		\scriptsize \center{(a)}
	\end{minipage}
	\begin{minipage}[c]{0.48\textwidth}
		\centering
		\includegraphics[width=0.9\linewidth]{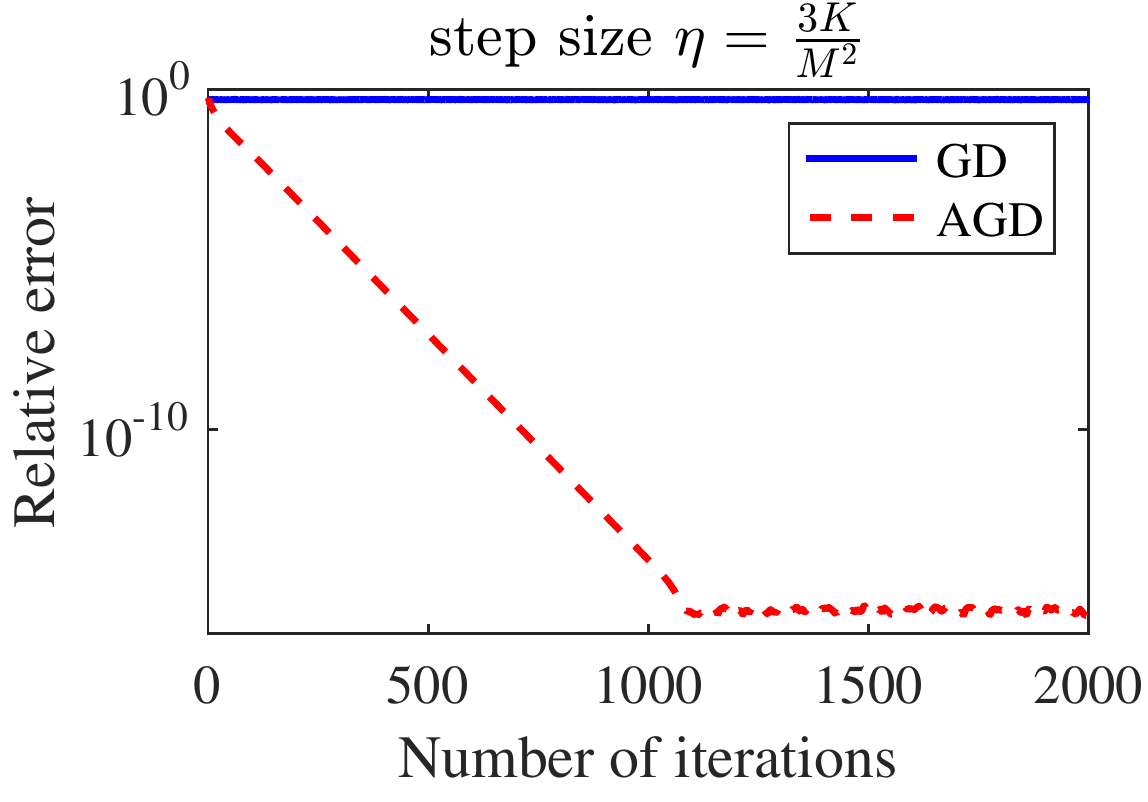}
		\vspace{-4mm}
		\scriptsize \center{(b)}
	\end{minipage}
	\caption{ Performance of AGD and GD under different $\eta$}
	\label{fig:CASE}
\end{figure}

Fig.~\ref{fig: N_vs_err} compares the convergence rates of AGD and GD. The number of samples $N$ is set as $500$, 
and other parameters are  the default values.   Each   point means the smallest number of iterations needed to reach the corresponding estimation error, and the results are averaged over $100$ independent trials. AGD requires a smaller number of the iterations than GD  to achieve the same relative error.

\begin{figure}[h]
	\centering
	\includegraphics[width=0.5\linewidth]{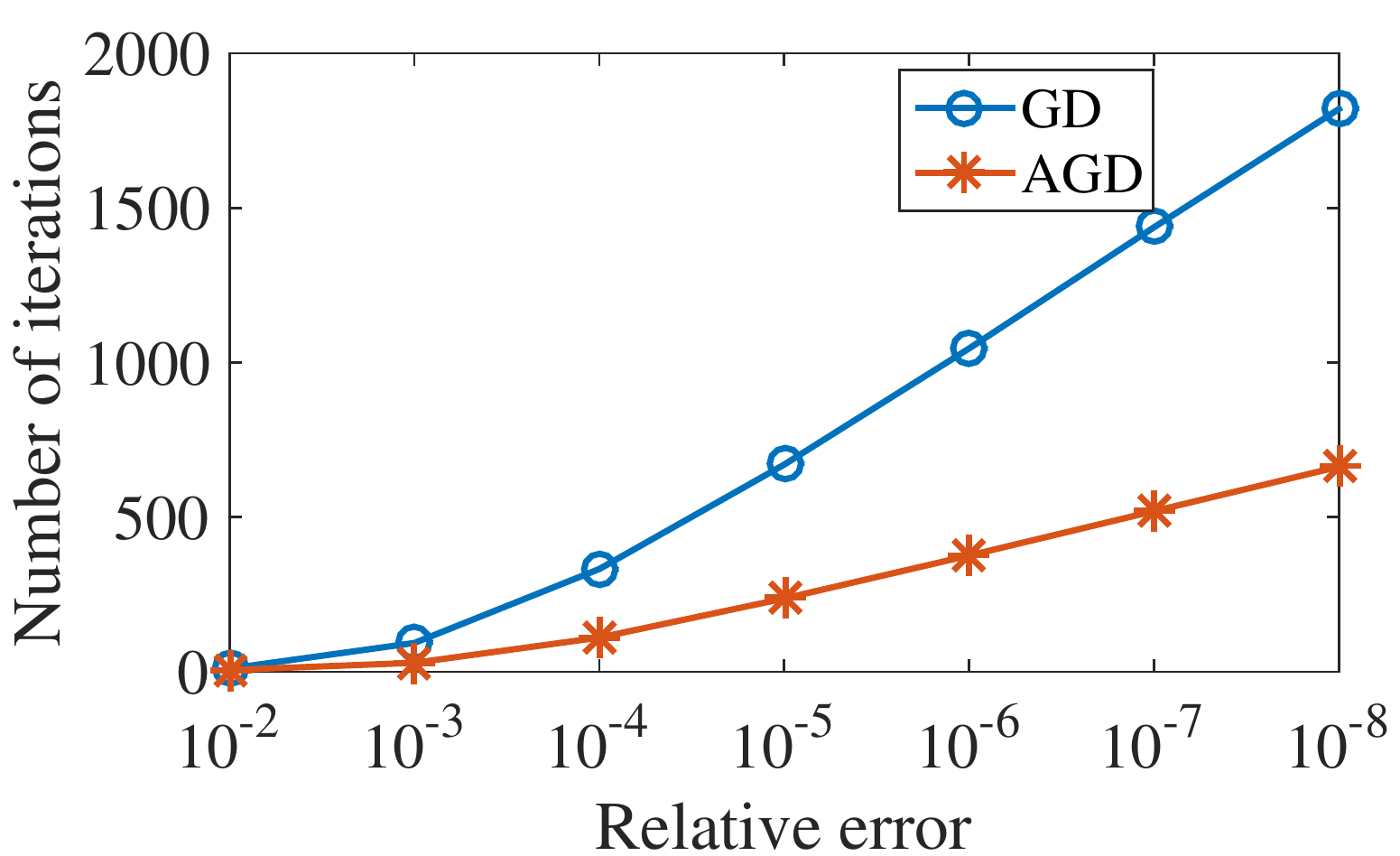}
	\caption{Comparison of AGD and GD in number of iterations }
	\label{fig: N_vs_err}
\end{figure}

Fig.~\ref{fig:phase_com} shows the phrase transition of GD and AGD by varying $N$ and $d$ when the output is noiseless. 
AGD has a larger successful region than GD so that AGD requires a smaller number of samples to guarantee successful recovery for a given $d$. 
\begin{figure}[H]
	\centering
 	\includegraphics[width=0.7\linewidth]{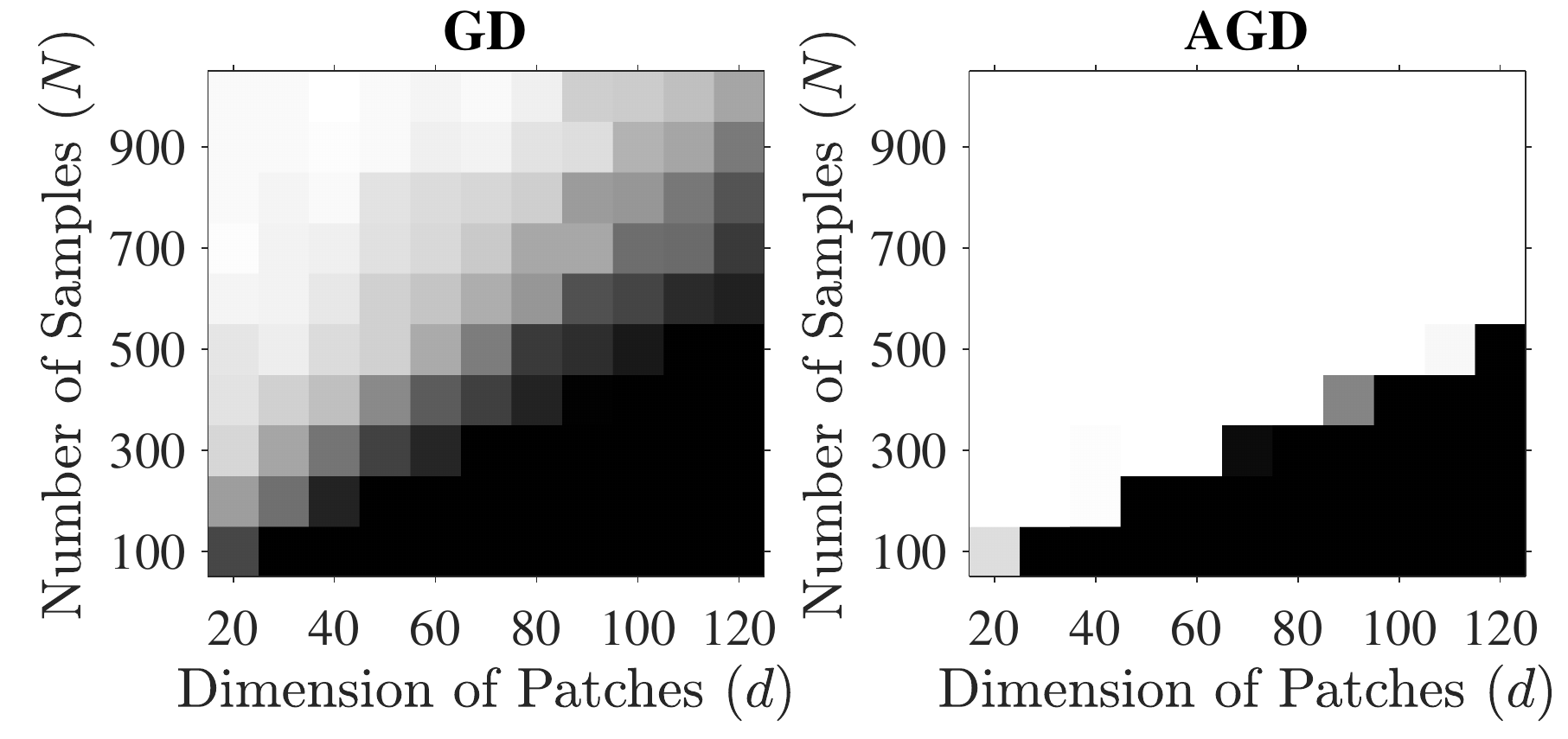}
	\caption{The phase transition of GD and AGD}
	\label{fig:phase_com}
\end{figure}

%% file: chapter_4/rpichap4.tex
 
\chapter{\uppercase{Training Graph Neural Networks with Guaranteed Generalizability: One-hidden-layer Case}}\label{chapter: 4}
\blfootnote{Portions of this chapter previously appeared as: S.~Zhang, M.~Wang, S.~Liu, P.~Chen, J.~Xiong, ``Fast learning of graph neural networks with guaranteed generalizability: One-hidden-layer case,'' in \emph{Proc. Int. Conf. Mach. Learn.}, 2020, pp. 11268-11277.}
\input{./chapter_4/intro_MW.tex}
\input{./chapter_4/Problem_MW.tex}

\input{./chapter_4/Alg.tex}

\input{./chapter_4/Theorem_MW.tex}
\input{./chapter_4/Simu_MW.tex}

\section{Summary}
Despite the practical success of graph neural networks in learning graph-structured data, the theoretical guarantee of the generalizability of graph neural networks is still elusive. Assuming the existence of a ground-truth model, this chapter shows theoretically, for the first time,
learning a one-hidden-layer graph neural network  with a generation error that is zero for  regression or approximately zero for binary classification. With the tensor initialization, we prove that the accelerated gradient descent method converges to the ground-truth model  exactly for regression or approximately for binary classification at a linear rate. We also characterize the required number of training samples   as a function of the feature dimension, the model size, and the   graph structural properties. One future direction is to extend the analysis to multi-hidden-layer neural networks. 

%% file: chapter_4/intro_MW.tex
\section{Introduction}

 Graph neural networks (GNNs)   \cite{GMS05,SGAHM08} have demonstrated great practical performance in learning with graph-structured data. 
Compared with traditional (feed-forward) neural networks, GNNs introduce an additional neighborhood aggregation layer, where the features of each node are aggregated with the features of  the neighboring nodes \cite{GSRVD17,XLTSTKKJ18}. GNNs have a better learning  performance in
  applications including physical reasoning \cite{BPLR16}, recommendation systems \cite{YHCEHL18},  biological analysis \cite{DMIBH15}, and compute vision \cite{MDSG06}. 
Many variations of GNNs, such as Gated Graph Neural Networks (GG-NNs) \cite{LTBZ16}, Graph Convolutional Networks (GCNs) \cite{KW17} and   others  \cite{HYL17, VCCRLB18}     have recently  been  developed to enhance the learning performance on graph-structured data.



Despite the numerical success, the theoretical understanding of the generalizability of the learned GNN models to the testing data is very limited. Some works \cite{XLTSTKKJ18,XHLJ19,WSZFYW19,MRFMW19} analyze the expressive power of GNNs but do not provide learning algorithms that are guaranteed to return the desired GNN model with proper parameters.  Only    few works  \cite{DHPSWX19,VZ19} explore the generalizabilty of GNNs, under the one-hidden-layer setting, as  even with one hidden layer  the models are  already    complex to analyze,  not to mention the multi-layer setting. Both works show that for regression problems, the generalization gap of the training error and the testing error decays with respect to the number of training samples at a sub-linear rate. The  analysis in   Ref.~\cite{DHPSWX19} analyzes GNNs through Graph Neural Tangent Kernels (GNTK)  
{which is an extension of Neural Tangent kernel (NTK) model \cite{JGH18,CB18, NS19,CG20}.   When over-parameterized, this line of works   shows sub-linear convergence to the global optima of the learning problem with assuming enough filters in the hidden layer \cite{JGH18,CB18}. }  
Ref.~\cite{VZ19} only applies to the case of one single filter in the hidden layer, and the activation function  needs to be smooth, excluding the popular ReLU activation function. 
Moreover, refs.~\cite{DHPSWX19,VZ19} do not consider classification and do not discuss {if a small training error and a small generalization error can be achieved simultaneously.}


 One recent line of research analyzes the generalizability of neural networks (NNs) from the perspective of model estimation \cite{BG17,DLTPS17,DLT17,FCL20,GLM17,SS17,ZSJB17,ZSD17}. These works assume   
the existence of a ground-truth NN model with some unknown parameters that maps the input features to the output labels for both training and testing samples. Then the learning objective is to estimate  the ground-truth model parameters from the training data, and this ground-truth model is guaranteed to have  a zero generalization                                             error on the testing data. The analyses are  focused on one-hidden-layer NNs, assuming the input features following the Gaussian   distribution  \cite{Sh18}.  
 If one-hidden-layer NNs only have one filter in the hidden layer, gradient descent (GD) methods can learn the ground-truth parameters with a high probability \cite{DLTPS17,DLT17,BG17}. When there are multiple filters in the hidden layer, the learning problem is much more challenging to solve because of the common spurious local minima \cite{SS17}. \cite{GLM17} revises the learning objective and shows the global convergence of GD   to the global optimum of the new learning problem. 
 The required number for training samples, referred to as the sample complexity in this chapter, is a high-order polynomial function of the   model size. A few works  \cite{ZSJB17,ZSD17,FCL20} study a learning algorithm that initializes using the  tensor initialization method \cite{ZSJB17} and iterates using GD. This algorithm is proved to  converge  to the ground-truth model parameters with a zero generalization error for the one-hidden-layer NNs with multiple filters, and the sample complexity is shown to be linear in the model size. All these works only consider NNs rather than GNNs.

We provides the first algorithmic design and theoretical analysis to learn a GNN model with a zero generalization error, assuming the existence of such a ground-truth model. 
  We study GNNs  in semi-supervised learning,  
 and the results apply to  both regression and binary classification problems. 
Different from NNs, each output label on the graph depends on multiple neighboring features in GNNs, and such dependence significantly complicates the analysis of the learning problem.  Our proposed algorithm uses the tensor initialization \cite{ZSJB17} and updates by accelerated gradient descent (AGD). We prove that with a sufficient number of training samples, our algorithm returns the ground-truth model with the zero generalization error for regression problems. For binary classification problems, our algorithm returns a model sufficiently close to the ground-truth model, and its distance to the ground-truth model decays to zero as the number of samples increases. Our algorithm converges linearly, with a rate that is proved to be faster than that of vanilla GD. We quantifies the dependence of the sample complexity on the model size and the underlying graph structural properties. The required number of samples is linear in the model size. It is also a polynomial function of the graph degree and the largest singular value of the normalized adjacency matrix. Such dependence of the sample complexity on graph parameters is exclusive to GNNs and does not exist in NNs. 

%% file: chapter_4/Problem_MW.tex
\section{Problem Formulation}\label{Section: pf_1}
Let $\mathcal{G}=\{ \mathcal{V}, \mathcal{E} \}$ denote an un-directed graph, where $\mathcal{V}$ is the set of nodes with size $|\mathcal{V}|=N$ and $\mathcal{E}$  is the set of edges. Let $\delta$ and $\delta_{\textrm{ave}}$ denote the maximum and average node degree  of $\mathcal{G}$, respectively.  Let $\boldsymbol{\tilde{A}}\in \{0,1\}^{N\times N}$ be the adjacency matrix of  $\mathcal{G}$ with added self-connections. Then,   $\tilde{A}_{i,j}=1$ if and only if  there exists an edge between node $v_i$ and node $v_j$, $i,j \in [N]$, and $\tilde{A}_{i,i}=1$ for all $i\in [N]$.  Let  $\bfD$ be the degree matrix with diagonal elements $D_{i,i} = \sum_{j}\tilde{A}_{i,j}$ and zero entries otherwise. $\bfA$     denotes  the normalized adjacency matrix with $\bfA= \bfD^{-1/2}\boldsymbol{\tilde{A}}\bfD^{-1/2}$, and $\sigma_1(\bfA)$ is the largest singular value of   $\bfA$.  

 Each node $v_n$ in $\mathcal{V}$  $(n=1,2,\cdots, N)$ corresponds to an input feature vector, denoted by $\bfx_n\in \mathbb{R}^{d}$, and a label $y_n\in\mathbb{R}$. $y_n$ depends on not only $\bfx_n$ but also all $\bfx_j$ where $v_j$ is a neighbor of $v_n$. Let $\bfX = \begin{bmatrix}
\bfx_1, \bfx_2, \cdots, \bfx_N
\end{bmatrix}^T\in \mathbb{R}^{N\times d}$ denote the feature matrix.  Following the analyses of NNs \cite{Sh18}, we assume $\bfx_n$'s are i.i.d. samples from the standard Gaussian distribution $\mathcal{N}(\boldsymbol{0}, \bfI_{d})$.    
 For GNNs,  we consider the \ typical  semi-supervised learning  problem setup.  Let $\Omega\subset [N]$ denote the set of node indices with known labels, and let  $\Omega^c$ be its complementary set.  
The objective of the GNN is to predict $y_i$ for every $i$ in $\Omega^c$. 

Suppose there exists a one-hidden-layer GNN that maps  node  features to labels, as shown in Figure \ref{Figure: GNN}. There are $K$ filters\footnote{We assume $K\leq d$   to simplify the representation of the analysis, while the result still holds for $K>d$ with   minor changes. }
 in    the hidden layer, and the weight matrix is denoted by 
	 $\bfW^* = \begin{bmatrix} 
	\bfw_1^*&\bfw_2^*&\cdots&\bfw_K^*
	\end{bmatrix}\in \mathbb{R}^{d\times K}$ . The hidden layer is followed by a pooling layer. Different from NNs, GNNs have an additional aggregation layer with $\bfA$ as the aggregation factor matrix \cite{KW17}.   For every node $v_n\in\mathcal{V}$, the input to the hidden layer is $\bfa_n^T\bfX$, where $\bfa_n^T$ denotes the $n$-th row of $\bfA$. When there is no edge in $\mathcal{V}$, $\bfA$  is reduced to the identify matrix, and a GNN model is reduced to an NN model. 
	\begin{figure}[h]
		\centering
		\includegraphics[width=0.65\textwidth]{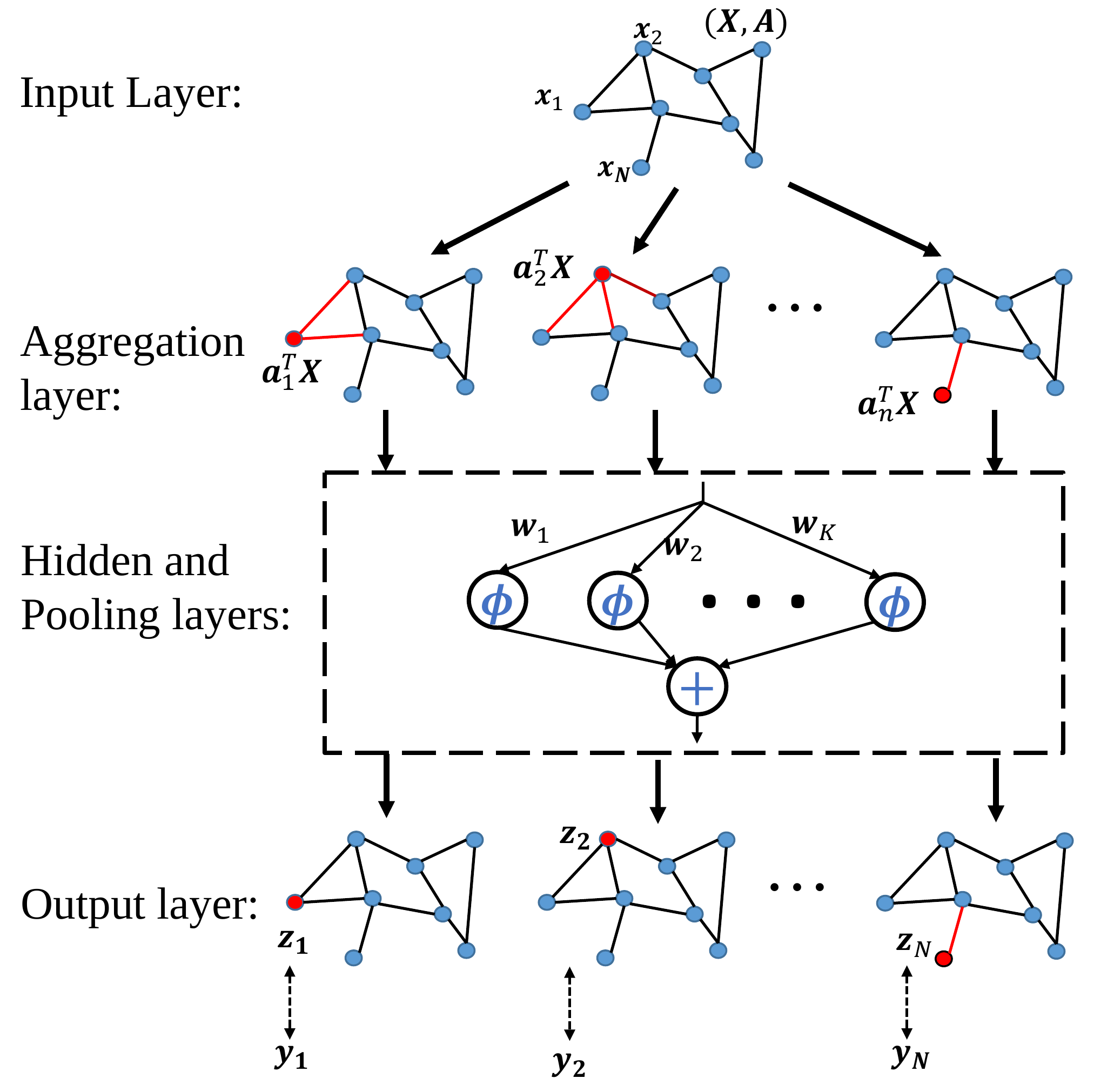}
		\caption{Structure of the graph neural network}
		\label{Figure: GNN}
	\end{figure}	
 
The output $z_n$ of the node $v_n$ of the GNN is 
	 	\begin{equation}\label{eqn: y_n}
	 \begin{gathered}
	z_n= g(\bfW^*;\bfa_n^T\bfX) = \frac{1}{K}\sum_{j=1}^{K}\phi(\bfa_n^T\bfX{\bfw_j^*}), \forall n \in[N],
	 \end{gathered}
	 \end{equation}
	 where $\phi(\cdot)$ is  the activation function. We consider both regression and binary classification in this chapter. For regression,  $\phi(\cdot)$ is the ReLU function\footnote{Our result can be extended to the sigmoid activation function with minor changes.} 
	 $\phi(x)=\max\{x , 0  \}$, 
	 and  
	 $y_n = z_n$. 
	 For binary classification, we consider the sigmoid activation function where $\phi(x)=1/(1+e^{-x})$. Then $y_n$ is  a binary variable generated from $z_n$ by  $\text{Prob}\{y_n=1\}=z_n$, and $\text{Prob}\{y_n=0\}=1-z_n$.

Given $\bfX$, $\bfA$, and $y_i$ for all $i\in \Omega$, the learning objective is to  estimate $\bfW^*$, which  is assumed to have a zero generalization error. The  {training objective} is to minimize the empirical risk function,
	\begin{equation}\label{eqn: optimization}
	\min_{\bfW\in\mathbb{R}^{d\times K}}  \hat{f}_{\Omega}(\bfW):=\frac{1}{|\Omega|}\sum_{n\in \Omega} \ell(\bfW;\bfa_n^T\bfX),
	\end{equation}
	where $\ell$  is the loss function. 
For regression,	 we use the squared loss function
, and \eqref{eqn: optimization} is written as
	\begin{equation}\label{eqn:linear_regression}
	\min_{\bfW}: \quad \hat{f}_{\Omega}(\bfW)=\frac{1}{2|\Omega|}\sum_{n\in \Omega}\Big| y_n - g(\bfW;\bfa_n^T\bfX) \Big|^2.
	\end{equation}
For classification, we use the cross entropy loss function, and \eqref{eqn: optimization} is written as
	\begin{equation}\label{eqn: classification}
		\begin{split}
		\min_{\bfW}: \quad \hat{f}_{\Omega}(\bfW)
		=\frac{1}{|\Omega|}\sum_{n\in \Omega} -y_n \log \big(g(\bfW;\bfa_n^T\bfX)\big)
		-(1&-y_n) \log \big(1-g(\bfW;\bfa_n^T\bfX)\big).
		\end{split}
	\end{equation}

	Both \eqref{eqn:linear_regression} and \eqref{eqn: classification} are nonconvex due to the nonlinear function $\phi$. 
 Moreover,   while 	 $\bfW^*$ is a global minimum of \eqref{eqn:linear_regression},   $\bfW^*$ is	 not necessarily  a global minimum of \eqref{eqn: classification}\footnote{$\bfW^*$ is a global minimum if replacing all $y_n$ with $z_n$ in \eqref{eqn: classification}, but $z_n$'s are unknown in practice. }. Furthermore, compared with NNs, the additional difficulty of analyzing the generalization performance of GNNs lies in the fact that each label $y_n$ is correlated with all the input features that are connected to node $v_n$, as shown in the risk functions in \eqref{eqn:linear_regression} and \eqref{eqn: classification}.

Note that our model with $K=1$ is equivalent to  the one-hidden-layer convolutional network (GCN)  \cite{KW17} for binary classification. 
To study the  multi-class classification problem,  the GCN model in  \cite{KW17} has $M$ nodes for $M$ classes in the second layer and employs the softmax activation function at the output. Here, our model has a pooling layer and uses the sigmoid function for binary classification.
Moreover,
 we consider both regression and binary classification problems   using the same model architecture with different activation functions. 
We consider one-hidden-layer networks following the state-of-art works in NNs \cite{DLTPS17,DLT17,BG17,ZSJB17,ZSD17,FCL20} and GNNs \cite{DHPSWX19,VZ19} because the theoretical analyses are extremely complex and still being developed for multiple hidden layers.

%% file: chapter_4/Alg.tex
\section{Proposed Learning Algorithm}
 

  In what follows, we illustrate the algorithm used for solving problems \eqref{eqn:linear_regression} and \eqref{eqn: classification}, summarized in Algorithm \ref{Alg}. Algorithm \ref{Alg} has two   components: a) accelerated gradient descent and b) tensor initialization. We initialize  $\bfW$  using the tensor initialization method    \cite{ZSJB17} with minor modification for GNNs  and update iterates by the Heavy Ball method \cite{P87}.


	\textbf{Accelerated gradient descent}. Compared with the vanilla GD method,
  each iterate in the Heavy Ball method is updated along the combined directions of both the gradient and the moving direction of the previous iterates. 
 	 Specifically, one computes the difference of the  estimates in the previous two iterations, and the difference is scaled by a constant $\beta$. This additional  momentum term   is added to the gradient descent update. When $\beta$ is $0$, AGD reduces to GD.
	 
	During each iteration, a fresh subset 
	of data is applied to estimate the gradient.  The assumption of disjoint subsets   is  standard     to simplify the analysis \cite{ZSD17,ZSJB17} but  not necessary in numerical experiments. 
	\floatname{algorithm}{Algorithm}
	\setcounter{algorithm}{04}
	\begin{algorithm}[h]
		\caption{Accelerated Gradient Descent Algorithm with Tensor Initialization}
		\label{Alg}
		\begin{algorithmic}[1]
			\State \textbf{Input:} $\bfX$, $\big\{y_n\big\}_{n\in\Omega}$, $\bfA$, 
			the  step size $\eta$, the momentum constant $\beta$, and the  error tolerance  $\varepsilon$;
			\State \textbf{Initialization:} Tensor   Initialization via Subroutine 1;
			\State {Partition} ${\Omega}$ into $T=\log(1/\varepsilon)$ disjoint subsets, denoted as $\{{\Omega}_t\}_{t=1}^T$;
			\For  {$t=1, 2, \cdots, T$}		
			\State $\bfW^{(t+1)}=\bfW^{(t)}-\eta\nabla \hat{f}_{{\Omega}_t}(\bfW^{(t)}) +\beta(\bfW^{(t)}-\bfW^{(t-1)})$
			\EndFor
		\end{algorithmic}
	\end{algorithm}
	
	
  	\textbf{Tensor initialization}.  The main idea of the tensor initialization method \cite{ZSJB17} is to utilize the homogeneous property of an activation function such as ReLU to estimate the magnitude and direction separately for each $\bfw_j^*$ with $j\in[K]$. 
   A non-homogeneous function can be approximated by piece-wise linear functions,  if the function is strictly monotone with lower-bounded derivatives \cite{FCL20}, like the sigmoid function. Our  initialization is similar to those in \cite{ZSJB17,FCL20} for NNs with some definitions {are changed} to handle the graph structure, and the initialization process is summarized in Subroutine 1.   
   
Specifically, following \cite{ZSJB17},  we define a special outer product, denoted by $\widetilde{\otimes}$, such that for any vector $\bfv\in \mathbb{R}^{d_1}$ and $\bfZ\in \mathbb{R}^{d_1\times d_2}$,   
 \begin{equation}
 	\bfv\widetilde{\otimes} \bfZ=\sum_{i=1}^{d_2}(\bfv\otimes \bfz_i\otimes \bfz_i +\bfz_i\otimes \bfv\otimes \bfz_i + \bfz_i\otimes \bfz_i\otimes \bfv),
 \end{equation} 
 where $\otimes$ is the outer product and $\bfz_i$ is the $i$-th column of $\bfZ$.
 Next, we define
 \footnote{ $\mathbb{E}_{\bfX}$ stands for the expectation over the distribution of random variable $\bfX$.}
\begin{equation}\label{eqn: M_1_GNN}
\bfM_{1} = \mathbb{E}_{\bfX}\{y \bfx \} \in \mathbb{R}^{d},
\end{equation}
\begin{equation}\label{eqn: M_2_GNN}
\bfM_{2} = \mathbb{E}_{\bfX}\Big\{y\big[(\bfa_n^T\bfX)\otimes (\bfa_n^T\bfX)-\bfI\big]\Big\}\in \mathbb{R}^{d\times d},
\end{equation}
\begin{equation}\label{eqn: M_3_GNN}
\bfM_{3} = \mathbb{E}_{\bfX}\Big\{y\big[(\bfa_n^T\bfX)^{\otimes 3}- (\bfa_n^T\bfX)\widetilde{\otimes} \bfI  \big]\Big\}\in \mathbb{R}^{d\times d \times d},
\end{equation}
where $\bfz^{\otimes 3} := \bfz \otimes \bfz \otimes \bfz$. {The tensor $M_3$ is used to identify the directions of $\{w_j^*\}_{j=1}^K$. $M_1$ depends on both the magnitudes and directions of $\{w_j^*\}_{j=1}^K$. We will sequentially estimate the directions  and magnitudes of
$\{w_j^*\}_{j=1}^K$ from $M_3$ and $M_1$. The matrix $M_2$ is  used to identify
the subspace spanned by $w_j^*$. We will project to this subspace to reduce
the computational complexity of decomposing $M_3$.}
\floatname{algorithm}{Subroutine}
 \setcounter{algorithm}{1}
\begin{algorithm}[t]
	\caption{Tensor Initialization Method}\label{Alg: initia}
	\begin{algorithmic}[1]
		\State \textbf{Input:}  $\bfX$, $\big\{y_n\big\}_{n\in\Omega}$ and $\bfA$;
		\State Partition $\Omega$ into three disjoint subsets $\Omega_1$, $\Omega_2$, $\Omega_3$;
		\State Calculate $\widehat{\bfM}_{1}$, $\widehat{\bfM}_{2}$ {following \eqref{eqn: M_1_GNN}, \eqref{eqn: M_2_GNN}} using $\Omega_1$, $\Omega_2$, respectively;
		\State Estimate $\widehat{\bfV}$  by orthogonalizing the eigenvectors with respect to the $K$ largest eigenvalues of  $\widehat{\bfM}_{2}$;
		\State Calculate $\widehat{\bfM}_{3}(\widehat{\bfV},\widehat{\bfV},\widehat{\bfV})$ {using \eqref{eqn: M_3_v1}} through $\Omega_3$;
		\State Obtain $\{ \widehat{\bfu}_j \}_{j=1}^K$ via {tensor decomposition method \cite{KCL15}};
		\State Obtain $\widehat{\boldsymbol{\alpha}}$ by solving  optimization problem $\eqref{eqn: int_op_GNN}$;
		\State \textbf{Return:} $\bfw^{(0)}_j=\widehat{\alpha}_j\widehat{\bfV}\widehat{\bfu}_j$, {$j=1,...,K$.}
	\end{algorithmic}
\end{algorithm} 

{Specifically, the values of $\bfM_1$, $\bfM_2$ and $\bfM_3$ are all estimated through samples, and let $\widehat{\bfM}_{1}$, $\widehat{\bfM}_{2}$, $\widehat{\bfM}_{3}$ denote the corresponding estimates of these high-order momentum.}
Tensor decomposition method \cite{KCL15} provides the estimates of {the vectors $\bfw_j^*/\|\bfw_j^*\|_2$ 
from $\widehat{\bfM}_{3}$, and the estimates are denoted as $\widehat{\overline{\bfw}}_{j}^*$.}

However, the computation complexity of estimate through $\widehat{\bfM}_{3}$ depends on poly($d$). To reduce the computational complexity of tensor decomposition, $\widehat{\bfM}_{3}$ is in fact first projected to a lower-dimensional tensor \cite{ZSJB17} through a matrix $\widehat{\bfV}\in \mathbb{R}^{d\times K}$. {$\widehat{\bfV}$ is the estimation of matrix $\bfV$ and can be computed from the right singular vectors of $\widehat{\bfM}_{2}$.} 
The column vectors of ${\bfV}$ form a basis for the subspace spanned by $\{\bfw_{j}^* \}_{j=1}^{K}$, which indicates that $\bfV\bfV^T\bfw_j^* = \bfw_j^*$ for any $j\in[K]$.
Then, from \eqref{eqn: M_3_GNN}, 
$\bfM_{3}(\widehat{\bfV}, \widehat{\bfV}, \widehat{\bfV})\in\mathbb{R}^{K\times K \times K}$  is defined as
 \begin{equation}\label{eqn: M_3_v1}
 \begin{split}
 	 &\bfM_{3}(\widehat{\bfV}, \widehat{\bfV}, \widehat{\bfV}) 
 	:= \mathbb{E}_{\bfX} \Big\{ y\big[(\bfa_n^T\bfX\widehat{\bfV})^{\otimes 3}
 	   - (\bfa_n^T\bfX\widehat{\bfV})\widetilde{\otimes} \bfI  \big]\Big\}.
\end{split}
\end{equation}
 
Similar to the case of $\widehat{\bfM}_{3}$, by applying the tensor decomposition method in $\widehat{\bfM}_{3}(\widehat{\bfV}, \widehat{\bfV}, \widehat{\bfV})$, one can obtain a series of normalized vectors, denoted as $\{\widehat{\bfu}_j \}_{j=1}^K \in \mathbb{R}^{K}$, which are the estimates of $\{\bfV^T\overline{\bfw}_j^*\}_{j=1}^K$.
Then, $\widehat{\bfV}\widehat{\bfu}_j$ is an estimate of $\overline{\bfw}_j^*$ since $\overline{\bfw}^*_j$ lies in the column space of $\bfV$ with $\bfV\bfV^T\overline{\bfw}_j^*=\overline{\bfw}_j^*$. 

From \cite{ZSJB17}, \eqref{eqn: M_1_GNN} can be written as 
\begin{equation}\label{eqn: M_1_2_GNN}
\bfM_1 = \sum_{j=1}^{K} \psi_1(\overline{\bfw}_{j}^*){\|\bfw_j^*\|_2{\bar{\bfw}}_j^*},
\end{equation}
where $\psi_1$ depends on the distribution of $\bfX$. 
Since the distribution of $\bfX$ is known, the values of $\psi(\widehat{\overline{\bfw}}^*_j)$ can be calculated exactly. Then, the magnitudes of $\bfw_j^*$'s are estimated through solving the following optimization problem:
\begin{equation}\label{eqn: int_op_GNN}
	\widehat{\boldsymbol{\alpha}}=\arg\min_{\boldsymbol{\alpha}\in\mathbb{R}^K}:\quad  \Big|\widehat{\bfM}_{1} - \sum_{j=1}^{K}\psi(\widehat{\overline{\bfw}}^*_j) \alpha_j \widehat{\overline{\bfw}}^*_{j}\Big|.
\end{equation}
 Thus,  $\bfW^{(0)}$ is given as 
 $\begin{bmatrix}
 \widehat{\alpha}_1\widehat{\overline{\bfw}}^*_{1}, &\cdots,& \widehat{\alpha}_K\widehat{\overline{\bfw}}^*_{K}
 \end{bmatrix}$.


%% file: chapter_4/Theorem_MW.tex
	\section{Main Theoretical Results}
	
	
	Theorems \ref{Thm: major_thm_lr} and \ref{Thm: major_thm_cl} state our major results about the GNN model  for regression  and binary classification, respectively. Before formally presenting the results, we first summarize the key findings as follows. 
	
	\textbf{1. Zero generalization error of the learned model.} Algorithm ~\ref{Alg} can return $\bfW^*$ exactly for   regression (see \eqref{eqn: linear_convergence_lr})   and approximately for binary classification (see \eqref{eqn: linear_convergence2_cl}). 
	Specifically, since $\bfW^*$ is {often} not  a solution to \eqref{eqn: classification}, 
	Algorithm~\ref{Alg} returns a critical point $\widehat{\bfW}$ that is sufficiently close to $\bfW^*$, and the distance decreases with respect to the number of samples in the order of $\sqrt{1/|\Omega|}$. 
	Thus, with a sufficient number of samples, $\widehat{\bfW}$ {will be close to $\bfW^*$ and}  achieves a zero generalization error approximately for binary classification.  Algorithm~\ref{Alg} always returns $\bfW^*$ exactly for   regression, a zero {generalization} error is thus achieved.
	


\textbf{2. 	Fast linear convergence of  Algorithm~\ref{Alg}.}  Algorithm~\ref{Alg} is proved to converge linearly to $\bfW^*$ for regression and $\widehat{\bfW}$ for classification, as shown in \eqref{eqn: linear_convergence_lr} and \eqref{eqn:converge2}. That means the distance of the estimate during the iterations to $\bfW^*$ (or $\widehat{\bfW}$) decays exponentially. Moreover,  Algorithm~\ref{Alg} converges faster than the vanilla GD. The rate of convergence is $1-\Theta\big(\frac{1}{\sqrt{K}}\big)$
for regression
and $1-\Theta\big(\frac{1}{K}\big)$ for classification, where $K$ is the number of filters in the hidden layer. In comparison, the convergence rates of GD are $1-\Theta\big(\frac{1}{K}\big)$ and $1-\Theta\big(\frac{1}{K^2}\big)$, respectively. Note that a smaller value of the rate of convergence corresponds to faster convergence. 
 We remark that this is the first theoretical guarantee of  AGD methods for learning GNNs.

\textbf{3. Sample complexity analysis.} $\bfW^*$ can be estimated exactly for regression and approximately for classification, provided that the number of  samples is in the order of $(1+\delta^2) \textrm{poly}(\sigma_1(\bfA), K) d\log N\log(1/\varepsilon)$, as shown in \eqref{eqn: sample_complexity_lr} and \eqref{eqn: sample_complexity_cl}, where $\varepsilon$ is the desired estimation error tolerance. $\bfW^*$ has $Kd$ parameters,  where $K$ is the number of nodes in the hidden layer, and $d$ is the feature dimension.  Our sample complexity is order-wise optimal with respect to $d$  and  only logarithmic with respect to the total number of features $N$. 
{We further show that} the sample complexity is also positively   associated  with $\sigma_1(\bfA)$ and $\delta$.  That characterizes the relationship between the sample complexity and graph structural properties.  From Lemma \ref{Lemma: sigma_1}, we know that given $\delta$, $\sigma_1(\bfA)$ is positively correlated with the average node degree $\delta_{\text{ave}}$. Thus, the required number of samples increases when the maximum and average degrees of the graph increase. That coincides with the intuition that more edges in the graph corresponds to the stronger dependence of the labels on neighboring features, {thus requiring} more samples to learn these dependencies. Our sample complexity quantifies this intuition explicitly. 
 
Note that the graph structure affects this bound only through $\sigma_1(\bfA)$ and $\delta$. Different graph structures may require a similar number of samples to estimate $\bfW^*$, as long as they have similar $\sigma_1(\bfA)$ and $\delta$. We will verify this property on different graphs numerically in Figure~\ref{Figure: N_vs_graph}.
  
        \begin{lemma}\label{Lemma: sigma_1}
 	Give an un-directed graph $\mathcal{G}=\{\mathcal{V}, \mathcal{E} \}$ and the   normalized adjacency matrix $\bfA$ as defined in Section \ref{Section: pf_1}, 
 	the largest singular value $\sigma_1(\bfA)$ of $\bfA$ satisfies
		\begin{equation}
			\frac{ 1+\delta_{\textrm{ave}} }{1+ \delta_{\max} } \le \sigma_1(\bfA)\le 1, 
		\end{equation}
 		where $\delta_{\text{ave}}$ and $\delta$ are the average and maximum node degree, respectively. 
	\end{lemma}
 
 \subsection{Formal Theoretical Guarantees}\label{sec:theorem}
	To formally present the results, some parameters in the results are defined as follows. $\sigma_j(\bfW^*)$ ($j\in [N]$) is the $j$-th singular value of $\bfW^*$. $\kappa=\sigma_1(\bfW^*)/\sigma_K(\bfW^*)$ is the conditional number of $\bfW^*$. 
	$\gamma$ is defined as $\prod_{j=1}^K\sigma_j(\bfW^*)/\sigma_K(\bfW^*)$. For a fixed $\bfW^*$, both $\gamma$ and $\kappa$ can be viewed as constants and do not affect the order-wise analysis.

	\begin{theorem}\label{Thm: major_thm_lr}
		(Regression) Let $\{ \bfW^{(t)} \}_{t=1}^T$ be the sequence generated by  Algorithm \ref{Alg} to solve \eqref{eqn:linear_regression} with $\eta= {K}/{(8\sigma_1^2(\bfA))}$. 
		Suppose the number of samples satisfies 
		\begin{equation}\label{eqn: sample_complexity_lr}
			|\Omega|\ge C_1\varepsilon_0^{-2} \kappa^9\gamma^2 { (1+\delta^2)\sigma_1^4(\bfA)} K^8 d\log N \log(1/\varepsilon)
		\end{equation} 
		for some constants $C_1>0$ and $\varepsilon_0\in(0,\frac{1}{2})$. Then $\{ \bfW^{(t)} \}_{t=1}^{T}$ converges linearly to $\bfW^*$ with probability at least ${1-K^2T}\cdot N^{-10}$ as  
		\begin{equation}\label{eqn: linear_convergence_lr}
		\begin{gathered}
		\|	\W[t]-\bfW^*\|_2
		\le\nu(\beta)^t\|	\W[0]-\bfW^*\|_2,\\
		\text{and}\quad \|	\W[T]-\bfW^*\|_2
		\le \varepsilon \|\bfW^* \|_2 
		\end{gathered}
		\end{equation}
		where $\nu(\beta)$ is the  rate of convergence that depends on $\beta$.
		Moreover, we have
		\begin{equation}\label{eqn: accelerated_rate_lr}
		\nu(\beta)< \nu(0) \text{\quad for some small nonzero } \beta.
		\end{equation} 
		Specifically, let $\beta^* = \Big(1-\sqrt{\frac{1-\varepsilon_0}{88\kappa^2 \gamma }}\Big)^2$, we have
		\begin{equation}\label{eqn:rate}
		\begin{gathered}
		 \nu(0)\ge 1-\frac{1-\varepsilon_0}{{88\kappa^2 \gamma K}},\quad 
		\nu(\beta^*)= 1-\frac{1-\varepsilon_0}{\sqrt{88\kappa^2 \gamma  K}}.	
		\end{gathered}
		\end{equation}
	\end{theorem}
	\begin{theorem}\label{Thm: major_thm_cl}
	(Classification)	Let $\{ \bfW^{(t)} \}_{t=1}^T$ be the sequence generated by Algorithm \ref{Alg} to solve \eqref{eqn: classification} with $\eta= {1}/{(2\sigma_1^2(\bfA))}$. 
		Suppose the number of samples satisfies 
		\begin{equation}\label{eqn: sample_complexity_cl}
		|\Omega|\ge C_2\varepsilon_0^{-2}(1+\delta^2)\kappa^8\gamma^2 \sigma_1^4(\bfA)K^8d\log N \log(1/\varepsilon)
		\end{equation} 
		for some positive constants $C_2$ and $\varepsilon_0\in(0,1)$. Then, let $\widehat{\bfW}$ be the nearest critical point of  \eqref{eqn: classification} to $\bfW^*$, we have that $\{ \bfW^{(t)} \}_{t=1}^{T}$ converges linearly to $\widehat{\bfW}$ with probability at least {$1-K^2T\cdot N^{-10}$} as  
		\begin{equation}\label{eqn:converge2}
		\begin{gathered}
		\|\bfW^{(t)} -\widehat{\bfW} \|_2 \le \nu(\beta)^t\|\bfW^{(0)} -\widehat{\bfW}\|_2,\\
		\textrm{ and }\|\bfW^{(T)} -\widehat{\bfW} \|_2 \le \varepsilon\|\bfW^{(0)} -\widehat{\bfW}\|_2.
		\end{gathered}
		\end{equation}
		The distance between $\widehat{\bfW}$ and $\bfW^*$ is bounded by
		\begin{equation} \label{eqn: linear_convergence2_cl}
		\| \widehat{\bfW} -\bfW^*  \|_2 \le C_3(1-\varepsilon_0)^{-1}\kappa^2 \gamma K \sqrt{\frac{(1+\delta^2)d\log N}{|\Omega|}},
		\end{equation}
		where $\nu(\beta)$ is the rate of convergence   that depends on $\beta$, and $C_3$ is some positive constant.
		Moreover, we have
		\begin{equation}\label{eqn: accelerated_rate_cl}
		\nu(\beta)< \nu(0) \text{\quad for some small nonzero } \beta,
		\end{equation} 
		Specifically, let $\beta^* = \Big(1-\sqrt{\frac{1-\varepsilon_0}{11\kappa^2\gamma K^2}}\Big)^2$, we have
		\begin{equation}\label{eqn: accelerated_rate1_cl}
		\begin{gathered}
		\nu(0)= 1-\frac{1-\varepsilon_0}{11\kappa^2\gamma K^2}, \quad
		\nu(\beta^*)= 1-\sqrt{\frac{1-\varepsilon_0}{11\kappa^2\gamma K^2}}.	
		\end{gathered}
		\end{equation}
	\end{theorem}

\subsection{Comparison with Existing Works}
 
Only \cite{VZ19,DHPSWX19} analyze the generalization error of one-hidden-layer GNNs in regression problems, while there is no existing work about the generalization error in classification problems. \cite{VZ19,DHPSWX19} show that the difference between the risks in the testing data and the training data
decreases in the order of $1/\sqrt{|\Omega|}$ as the sample size increases.  The GNN model in \cite{VZ19} only has one filter in the hidden layer, i.e., $K=1$, and the loss function is required to be a smooth function, excluding ReLU. Ref.~\cite{DHPSWX19} only considers infinitely wide GNNs.  In contrast, $\bfW^*$  returned by Algorithm \ref{Alg} can achieve zero risks for both training data and testing data in regression problems. Our results apply to an arbitrary number of filters and the ReLU activation function. 
 Moreover, this chapter is the first work that characterizes the {generalization} error of GNNs for binary classification.

When $\delta$ is zero, our model reduces to one-hidden-layer NNs, and the corresponding sample complexity is $O\Big( \textrm{poly}(K) d\log N\log(1/\varepsilon)\Big)$. Our results are at least comparable to, if not better than, the state-of-art theoretical guarantees {that from the prespective of model estimation} for NNs. 
For example, \cite{ZSJB17} considers   one-hidden-layer NNs for regression  and proves the linear convergence of their algorithm to the ground-truth model parameters. The sample complexity of ~\cite{ZSJB17} is also linear in  $d$, but the activation function must be smooth. \cite{ZYWG18} considers one-hidden-layer 
NNs with the ReLU  activation function for regression, but the algorithm cannot converge to the ground-truth parameters exactly but up to a statistical error.  Our result in Theorem \ref{Thm: major_thm_lr} applies to the nonsmooth ReLU function and can recover $\bfW^*$ exactly.
   \cite{FCL20} considers   one-hidden-layer NNs for classification and proves linear convergence of their algorithm to a critical point sufficiently close to   $\bfW^*$ with the distance bounded by   $O(\sqrt{1/|\Omega|})$. The convergence rate in \cite{FCL20} is $1-\Theta({1}/{K^2})$, while Algorithm~\ref{Alg} has a faster convergence   rate of $1-\Theta({1}/{K})$.

%% file: chapter_4/Simu_MW.tex
\section{Numerical Results}
We verify our results on synthetic graph-structured data.
We consider four types of graph structures as shown in Figure ~\ref{Figure: graph structure}: (a) a connected-cycle graph {having} each node connecting to its $\delta$ closet neighbors;   (b) a two-dimensional grid {having} each node connecting to its nearest neighbors in axis-aligned directions;  (c) a random $\delta$-regular graph {having} each node connecting to $\delta$ other nodes randomly;   (d) a random graph with bounded degree   {having} each node degree selected from $0$ with probability $1-p$ and $\delta$ with probability $p$ for some $p\in[0,1]$. 
\begin{figure}[h]
    \centering
    \begin{subfigure}[h]{0.4\textwidth}
        \centering
        \includegraphics[width=0.8\textwidth]{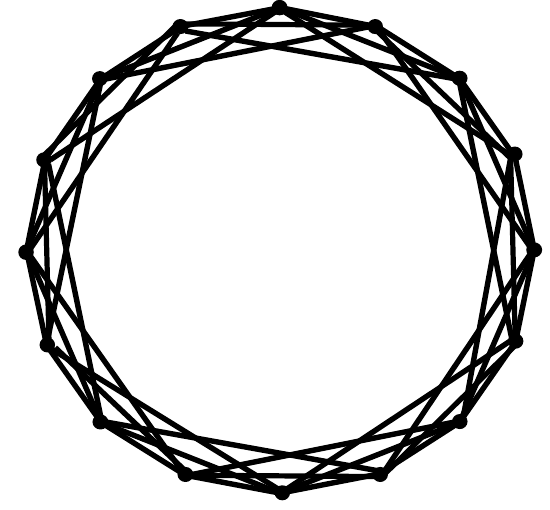}
        \caption{Connected-cycle graph}
    \end{subfigure}%
    ~ 
    \begin{subfigure}[h]{0.4\textwidth}
        \centering
        \includegraphics[width=0.8\textwidth]{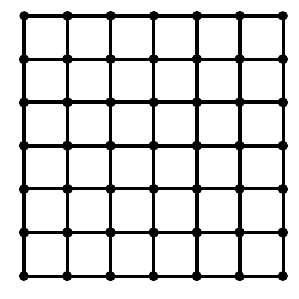}
        \caption{Two-dimensional grid}
    \end{subfigure}
    ~
    \begin{subfigure}[h]{0.4\textwidth}
        \centering
        \includegraphics[width=0.8\textwidth]{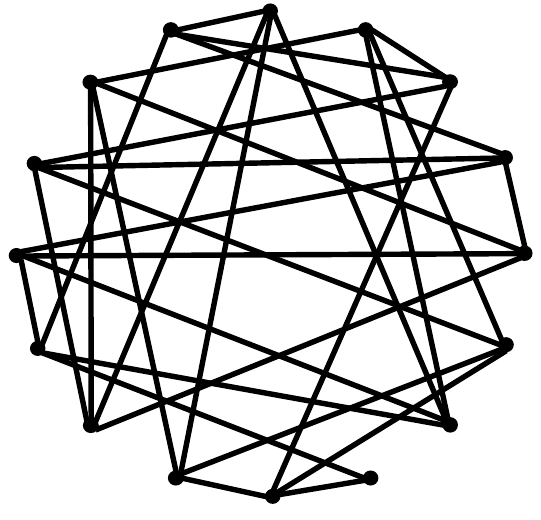}
        \caption{Random regular graph}
    \end{subfigure}%
    ~ 
    \begin{subfigure}[h]{0.4\textwidth}
        \centering
        \includegraphics[width=0.8\textwidth]{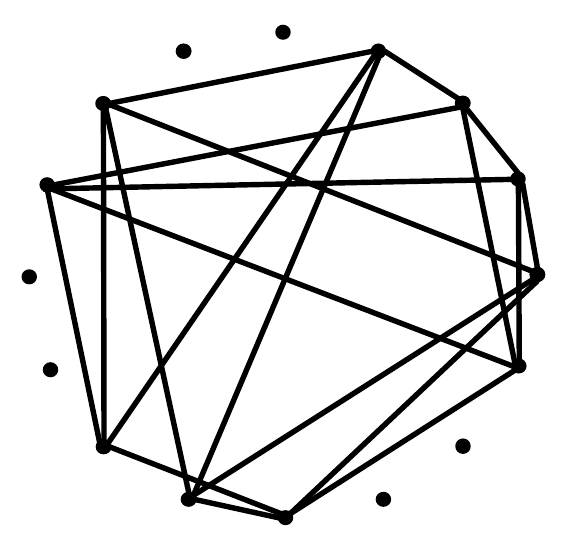}
        \caption{Random graph with a bounded degree}
    \end{subfigure}
    \caption{Different graph structures}
    \label{Figure: graph structure}
\end{figure}
The feature vectors $\{\bfx_n\}_{n=1}^{N}$ are randomly generated from the standard Gaussian distribution $\mathcal{N}(0 , \bfI_{d\times d})$. 
Each entry of  $\bfW^*$ is generated from   $\mathcal{N}(0, 5^2)$ independently. $\{z_n\}_{n=1}^N$ are computed based on \eqref{eqn: y_n}. 
The labels $\{y_n\}_{n=1}^N$ are generated by $y_n=z_n$ and  $\text{Prob}\{y_n=1\}=z_n$ for regression and classification problems, respectively.

During each iteration of Algorithm~\ref{Alg}, we use the whole training data  to calculate the gradient.
The initialization is randomly selected from $\big\{ \bfW^{(0)}\big| \| \bfW^{(0)} - \bfW^* \|_F/\|\bfW^* \|_F <0.5 \big\}$ to reduce the computation. As shown in \cite{FCL20,ZYWG18}, the
random initialization and the tensor initialization have very similar numerical performance. We consider the learning algorithm to be successful in estimation if the relative error,  defined as $\|\bfW^{(t)}-\bfW^*\|_F/\|\bfW^*\|_F$,  is less than $10^{-3}$, where $\bfW^{(t)}$ is the weight matrix returned by Algorithm ~\ref{Alg} when it terminates.


\subsection{Convergence Rate}
We first verify the linear convergence of Algorithm~\ref{Alg}, 
 as shown in \eqref{eqn: linear_convergence_lr} and \eqref{eqn:converge2}.
{Figure \ref{Figure: convergence_k} (a) and (b) show the convergence rate of Algorithm \ref{Alg} when varying the number of nodes in the hidden layer $K$. The dimension $d$ of the feature vectors   is chosen as $10$, and the sample size $|\Omega|$ is chosen as $2000$. We consider the connected-cycle graph in Figure~\ref{Figure: graph structure} (a) with  $\delta = 4$. All cases   converge to $\bfW^*$ with the exponential decay. Moreover, from Figure \ref{Figure: convergence_k}, we can also see that the  rate of convergence is almost a linear function of  $1/\sqrt{K}$. That verifies our theoretical result of the convergence rate of $1-O(1/\sqrt{K})$ in \eqref{eqn:rate}.}
\begin{figure}[h]
    \centering
    \begin{subfigure}[h]{0.5\textwidth}
        \centering
        \includegraphics[width=1.0\textwidth]{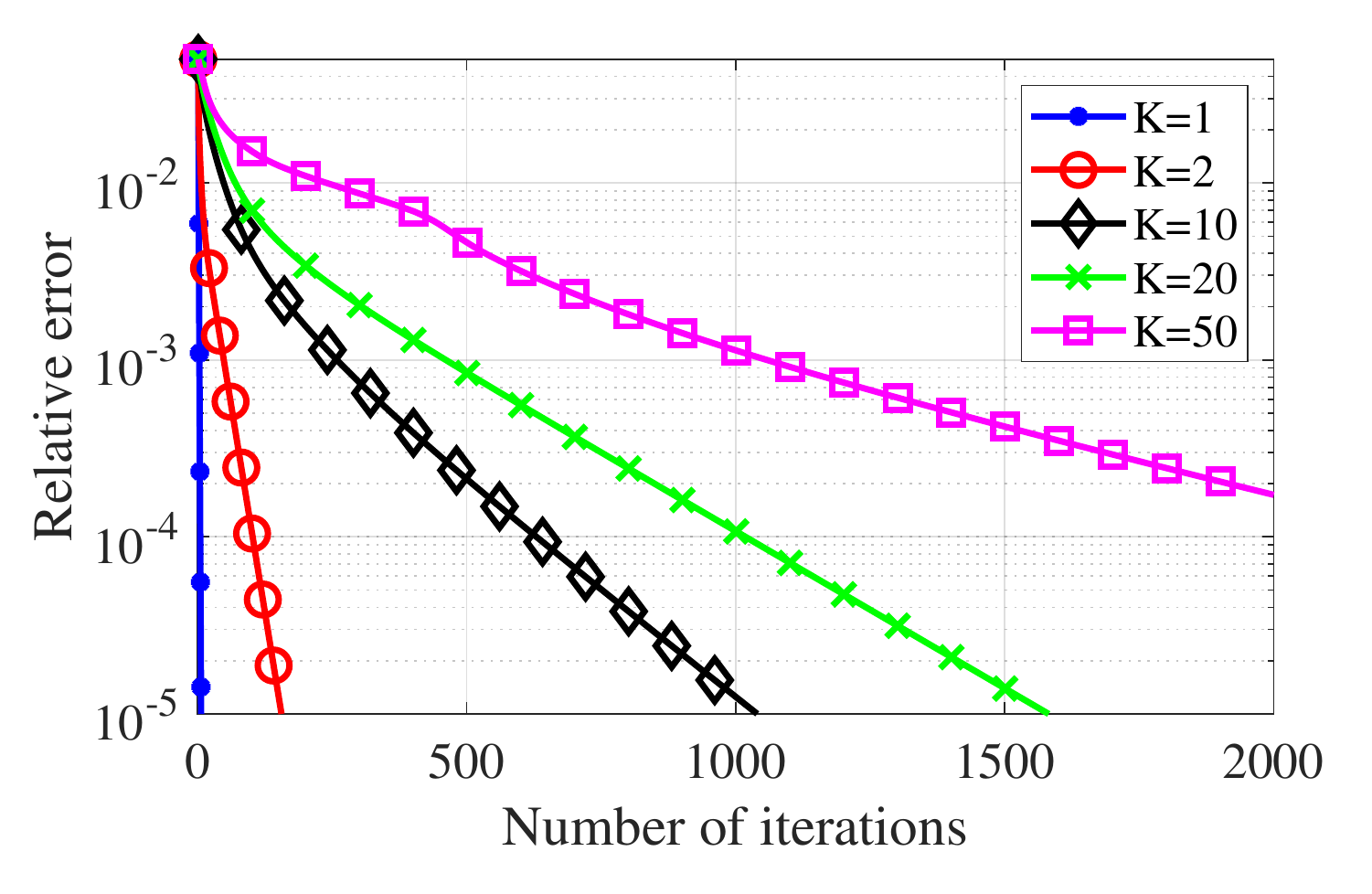}
        \caption{Relative error against iterations}
    \end{subfigure}%
    ~ 
    \begin{subfigure}[h]{0.48\textwidth}
        \centering
        \includegraphics[width=1.0\textwidth]{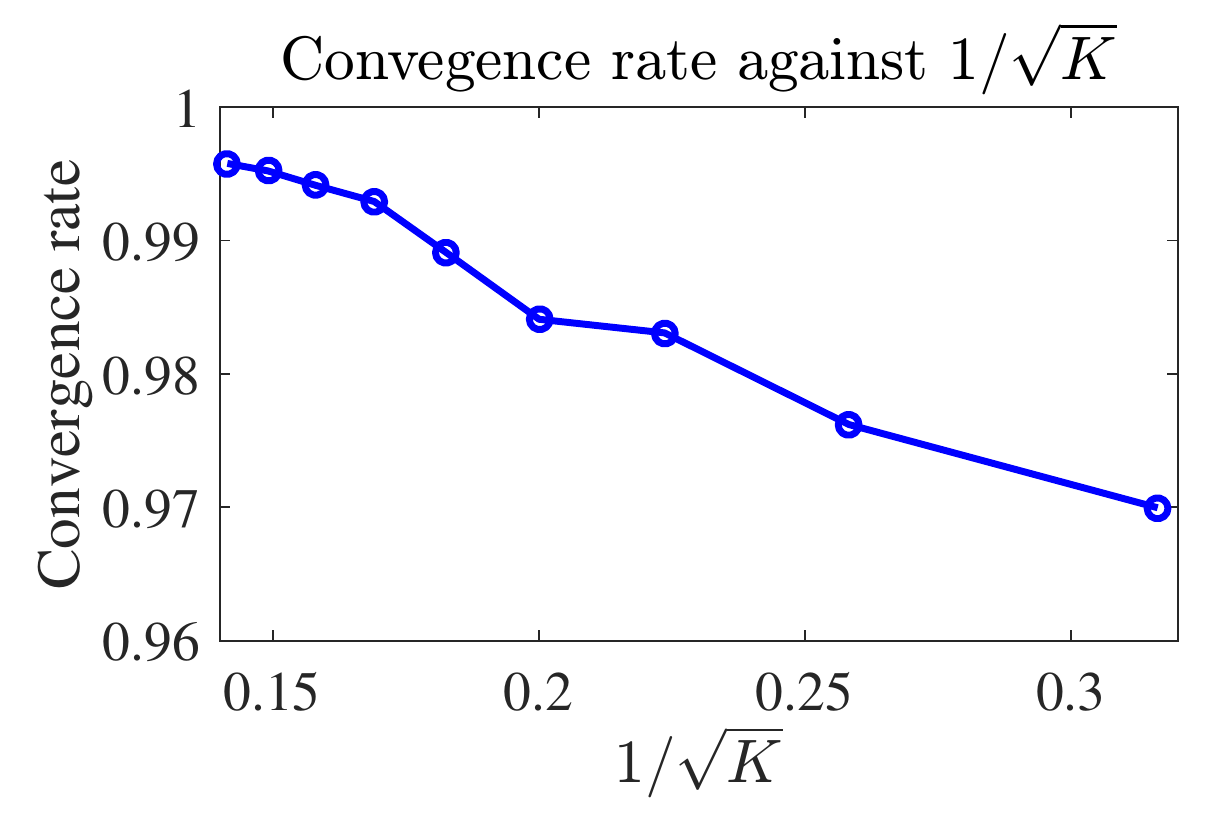}
        \caption{Convergence rate against $1/\sqrt{K}$}
    \end{subfigure}
     	\caption{Convergence rate with different $K$ for a connected-circle graph}
 	\label{Figure: convergence_k}
\end{figure}

Figure~\ref{fig: N_vs_err} compares the  rates of convergence of AGD and GD in regression problems. We consider a connected-cycle graph with  $\delta=4$. The number of samples  $|\Omega|=500$,  $d=10$, and $K=5$.   Starting with the same initialization, we show the smallest number of iterations needed to reach a certain estimation error, and the results are averaged over $100$ independent trials. Both AGD and GD converge linearly. AGD requires a smaller number of the iterations than GD  to achieve the same relative error.

\begin{figure}[h]
	\centering
	\includegraphics[width=0.5\linewidth]{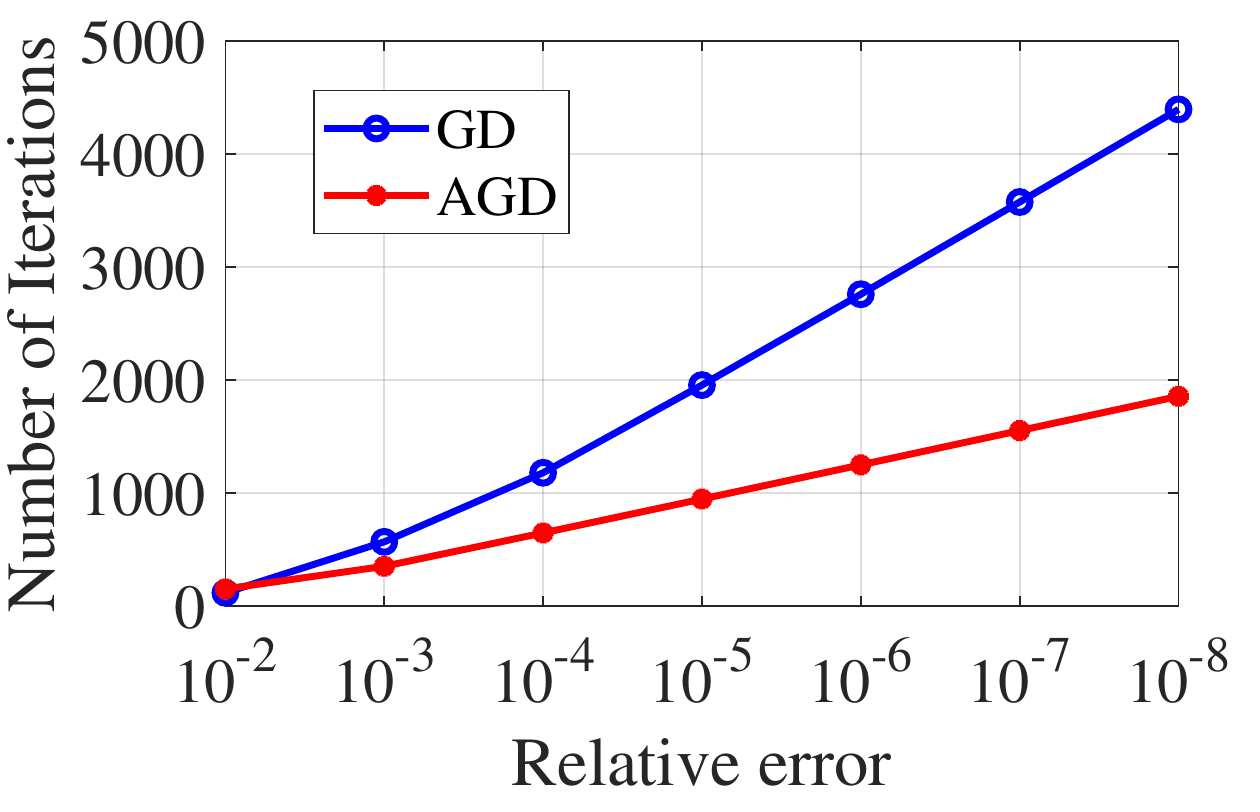}
	\caption{Convergence rates of AGD and GD}
	\label{fig: N_vs_err}
\end{figure}

\subsection{Sample Complexity}
We next study the influence of   $d$, $\delta$, $\delta_{\text{ave}}$,  {and} different graph structures  on the  estimation performance 
of Algorithm \ref{Alg}. These relationships are summarized in the sample complexity analyses in \eqref{eqn: sample_complexity_lr} and \eqref{eqn: sample_complexity_cl} of section \ref{sec:theorem}. We have similar numerical results for both regression and classification,  and here we only present the regression case.


Figures \ref{Figure: Phrase transition of number of samples} (a) and (b)  show  the successful estimation rates  when the degree of graph  $\delta$ and the feature dimension  $d$ changes. We consider the connected-cycle graph in Figure~\ref{Figure: graph structure} (a), and  
$K$ is kept as 5. $d$ is  $40$ in Figure \ref{Figure: Phrase transition of number of samples} (a),  and  $\delta$ is   $4$ in Figure \ref{Figure: Phrase transition of number of samples} (b).
The results are averaged over $100$ independent trials. White block means all trials are successful while black block means all trials fail. We can see that the required number of samples for successful estimation increases as $d$ and $\delta$ increases. 
\begin{figure}[H]
    \centering
    \begin{subfigure}[h]{0.48\textwidth}
        \centering
        \includegraphics[width=0.96\textwidth]{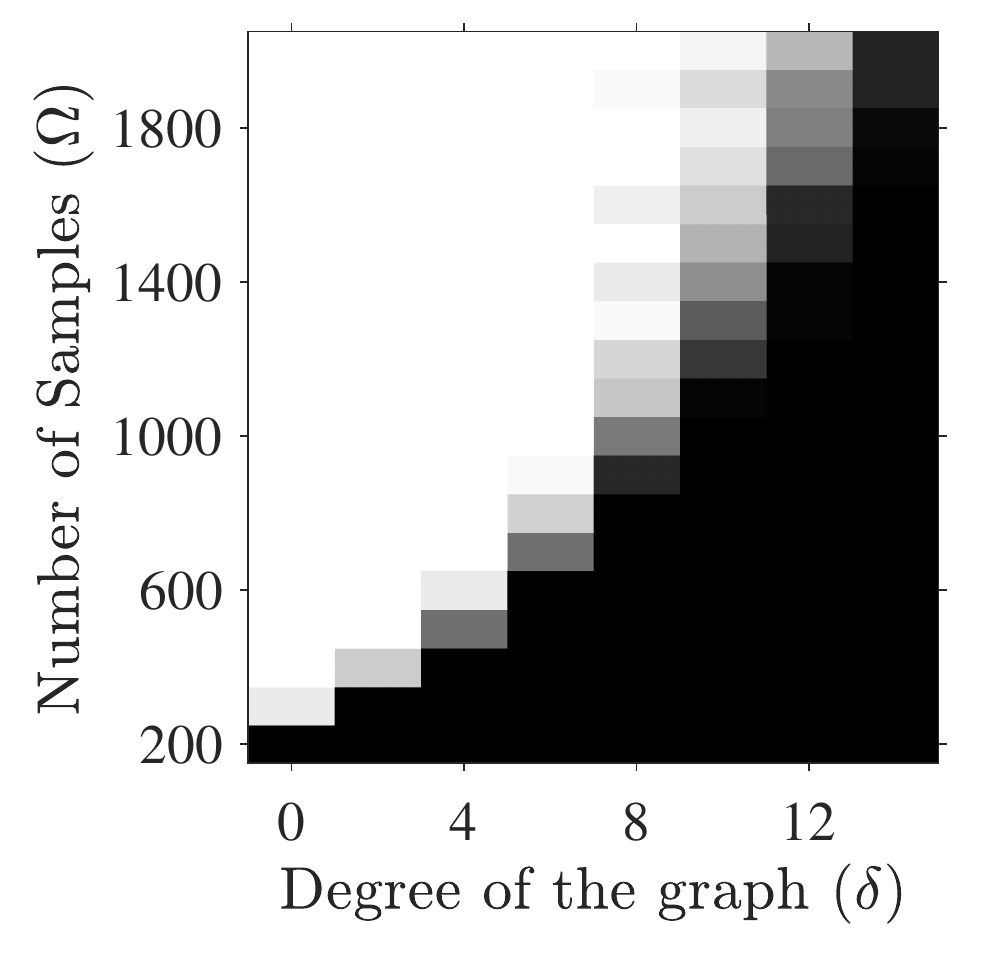}
        \caption{}
    \end{subfigure}%
    ~ 
    \begin{subfigure}[h]{0.48\textwidth}
        \centering
        \includegraphics[width=1.0\textwidth]{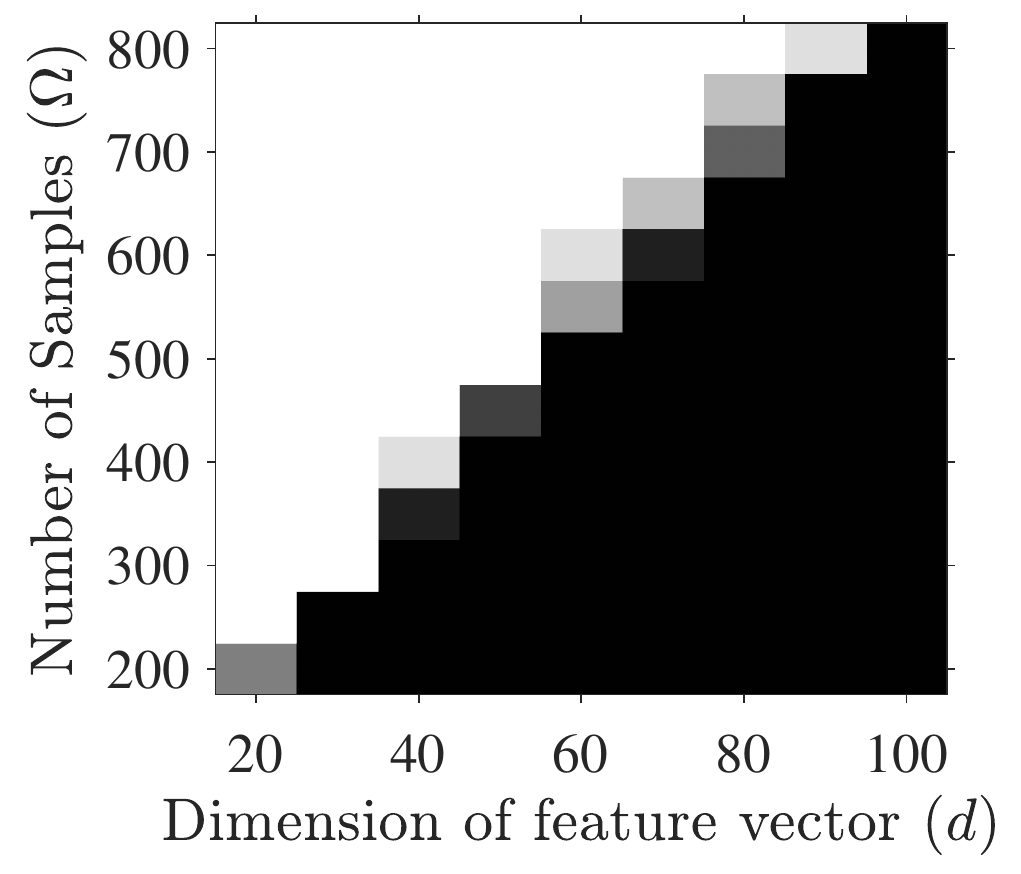}
        \caption{}
    \end{subfigure}
    \caption{Successful estimation rate for varying the required number of samples,  $\delta$, and $d$ in a connected-circle graphs }
    \label{Figure: Phrase transition of number of samples}
\end{figure}

Figure \ref{Figure: sigma} shows the success rate against the sample size $|\Omega|$ for the random graph in Figure~\ref{Figure: graph structure}(d) with different average node degrees.  We vary $p$ to change the average node degree  $\delta_{\text{ave}}$.  $K$ and $d$ are fixed as $5$ and $40$, respectively. The successful rate is calculated based on $100$ independent trials.    
We can see that more samples are needed for successful recovery for a larger $\delta_{\text{ave}}$ when the maximum degree $\delta$ is fixed. 
\begin{figure}[h]
	\centering
	\includegraphics[width=0.65\textwidth]{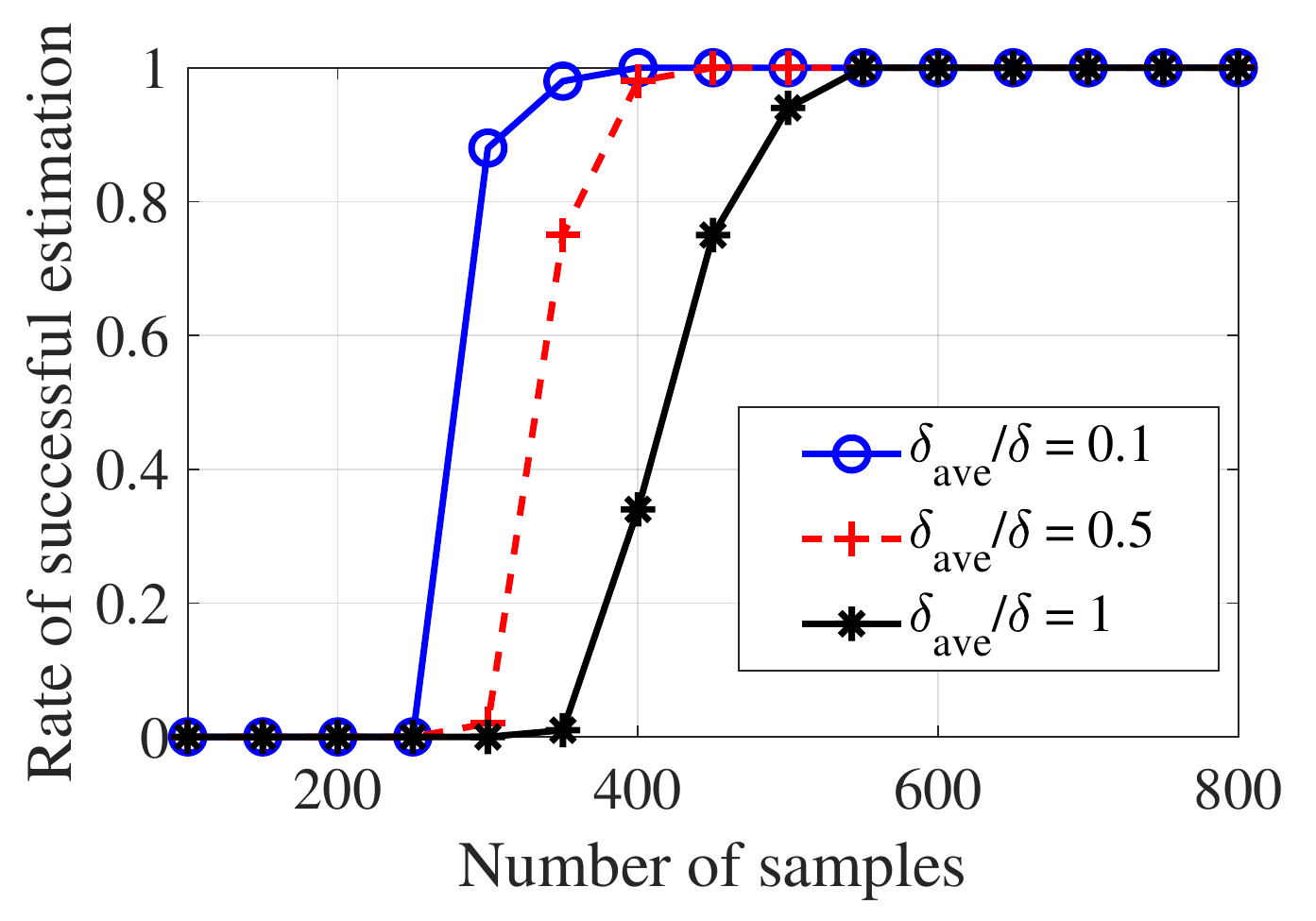}
	\caption{The success rate against the number of  samples for different $\delta_{\text{ave}}/ \delta$}
	\label{Figure: sigma}
\end{figure}

Figure \ref{Figure: N_vs_graph} shows the success rate against $|\Omega|$ for three different graph structures, including a connected cycle, a two-dimensional grid,  and  a random  regular graph  in Figure \ref{Figure: graph structure} (a), (b), and (c).  The maximum degrees of these graphs are all fixed with $\delta=4$. The average degrees of the connected-circle and the random $\delta$-regular graphs are also $\delta_{\text{ave}}=4$.  $\delta_{\text{ave}}$ is very close to $4$ for the two-dimensional grid when the graph size is large enough, because only the boundary nodes have smaller degrees, and the percentage of boundary nodes decays as the graph size increases. Then from Lemma \ref{Lemma: sigma_1}, we have $\sigma_1(\bfA)$ is $1$ for all these graphs. Although these graphs have different structures, the required numbers of samples to estimate $\bfW^*$ accurately are  the same, because both $\delta$  and $\sigma_1(\bfA)$ are the same. One can verify this property from Figure~\ref{Figure: N_vs_graph} where all three curves almost coincide.

\begin{figure}[h]
	\centering
	\includegraphics[width=0.60\textwidth]{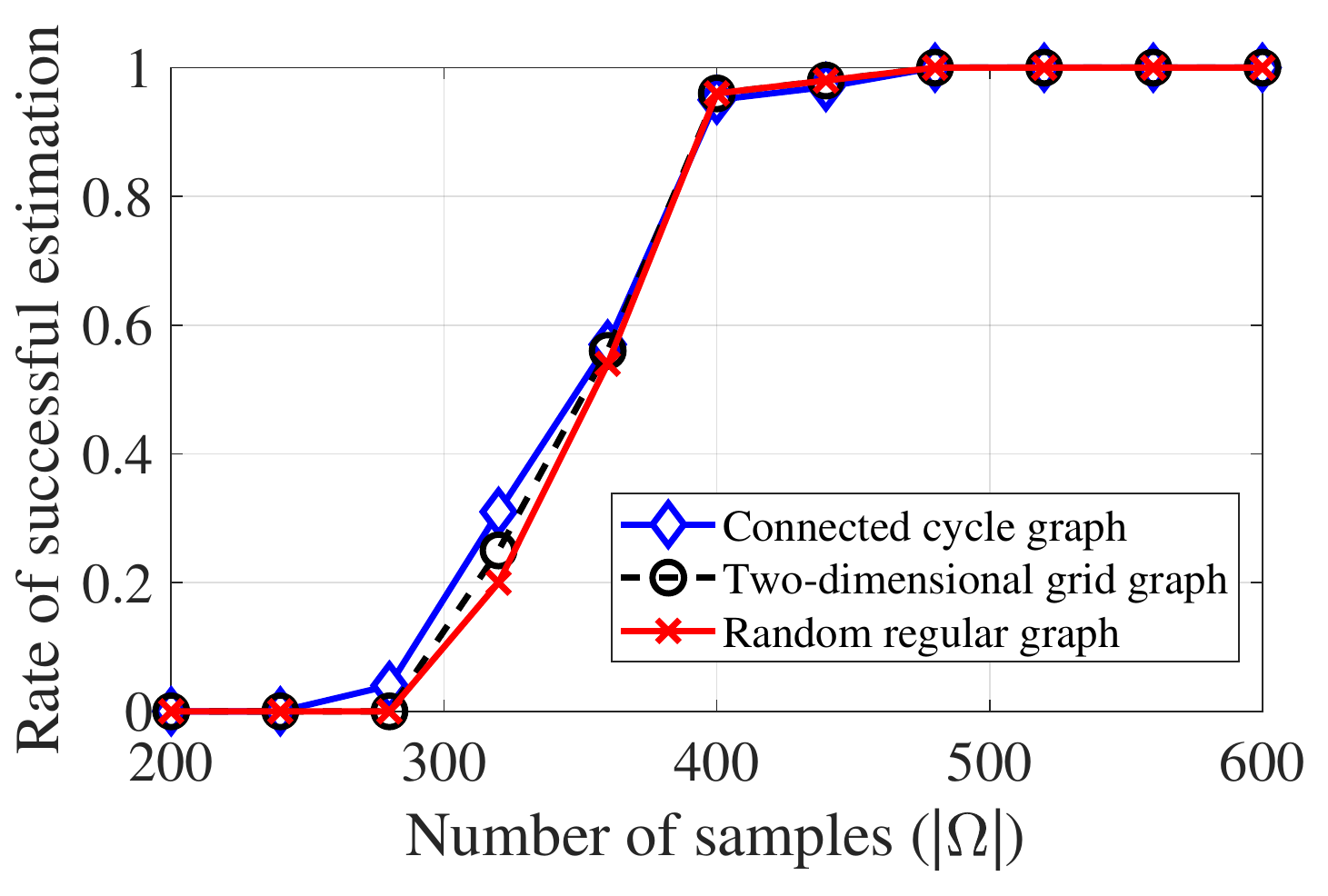}
	\caption{The success rate with respect to sample complexity for various graph structures} 
	\label{Figure: N_vs_graph}
\end{figure}

\subsection{Accuracy in Learning $\bfW^*$}
We study the learning accuracy of $\bfW^*$, characterized in 
  \eqref{eqn: linear_convergence_lr} for regression and \eqref{eqn: linear_convergence2_cl}  for classification. 
For   regression problems, we simulate the general cases when the labels are noisy, i.e., $y_n = z_n+\xi_n$.
The noise $\{\xi_n \}_{n=1}^N$ are i.i.d. from $\mathcal{N}(0, \sigma^2)$,
and the noise level is measured by $\sigma/E_z$, where $E_z$ is the average energy of the noiseless labels $\{z_n\}_{n=1}^{N}$, calculated as $E_z = \sqrt{\frac{1}{N}\sum_{n=1}^{N}|z_n|^2}$.
The number of hidden nodes $K$ is   5, and the dimension of each feature   $d$ is  as $60$. We consider a connected-circle graph with $\delta=2$.  Figure \ref{Figure: noise_case} shows the performance of Algorithm \ref{Alg} in the noisy case. We can see that when the number of samples exceeds $Kd=300$, which is the degree of freedom of $\bfW^*$, the relative error decreases dramatically. Also, as $N$ increases, the relative error converges to the noise level. When there is no noise, the estimation of $\bfW^*$ is accurate.

\begin{figure}[h]
	\centering
	\includegraphics[width=0.60\textwidth]{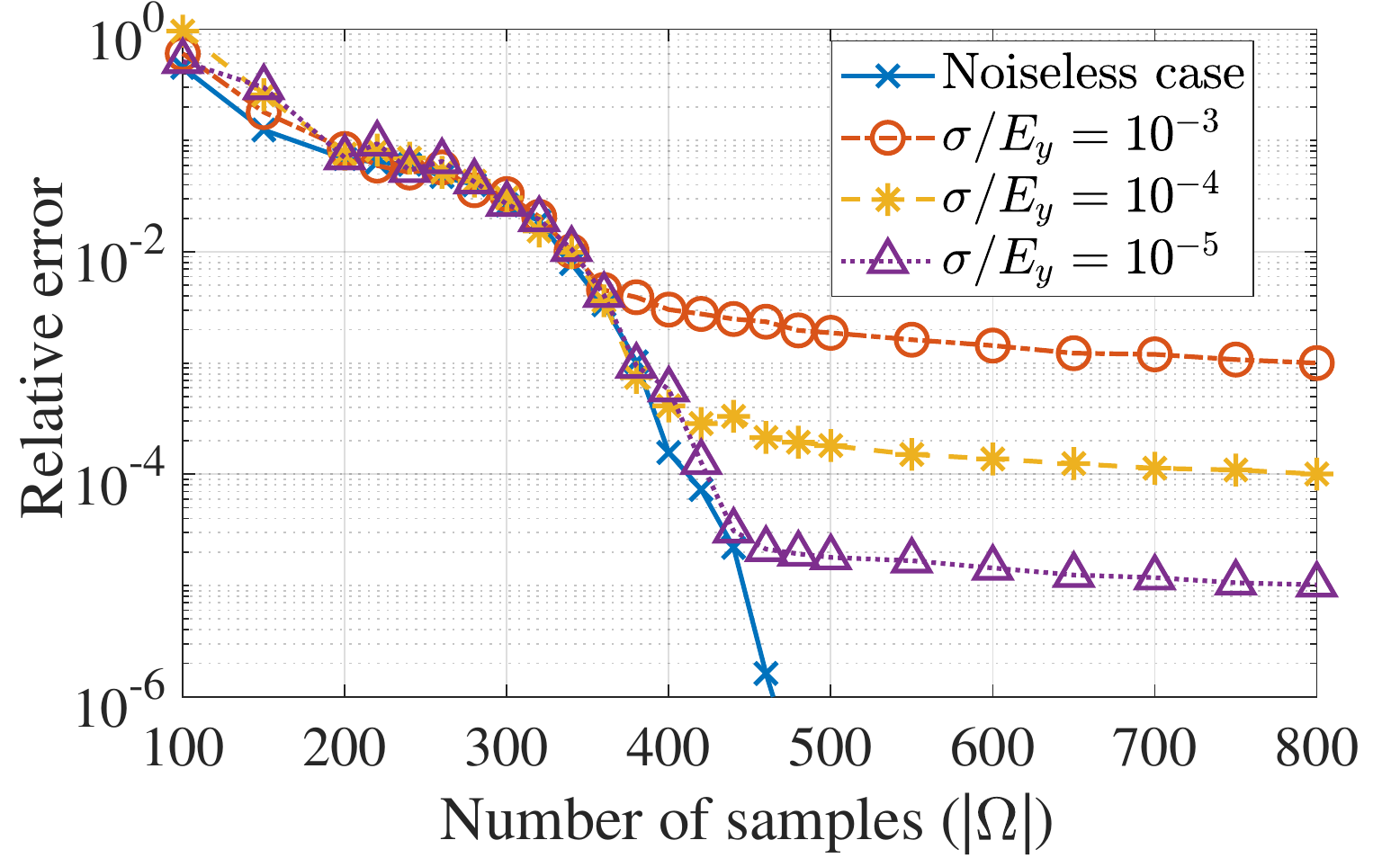}
	\caption{Learning accuracy of Algorithm~\ref{Alg} with noisy measurements for regression}
	\label{Figure: noise_case}
\end{figure}

For binary classification problems, Algorithm~\ref{Alg} returns  the nearest critical point $\widehat{\bfW}$ to $\bfW^*$. 
We show  the distance between the returned model and the ground-truth model $\bfW^*$   against the number of samples
in Figure \ref{Figure: N_error_cl}. 
We consider  a connected-cycle graph with the degree $\delta=2$. 
$K=3$ and $d=20$. 
The relative error $\|\widehat{\bfW}-\bfW^*\|_F/\|\bfW^*\|_F$ is averaged over  $100$ independent trials. 
We can see that the  distance between the returned model and the ground-truth model indeed decreases as the number of samples increases.

\begin{figure}[h]
	\centering
	\includegraphics[width=0.60\textwidth]{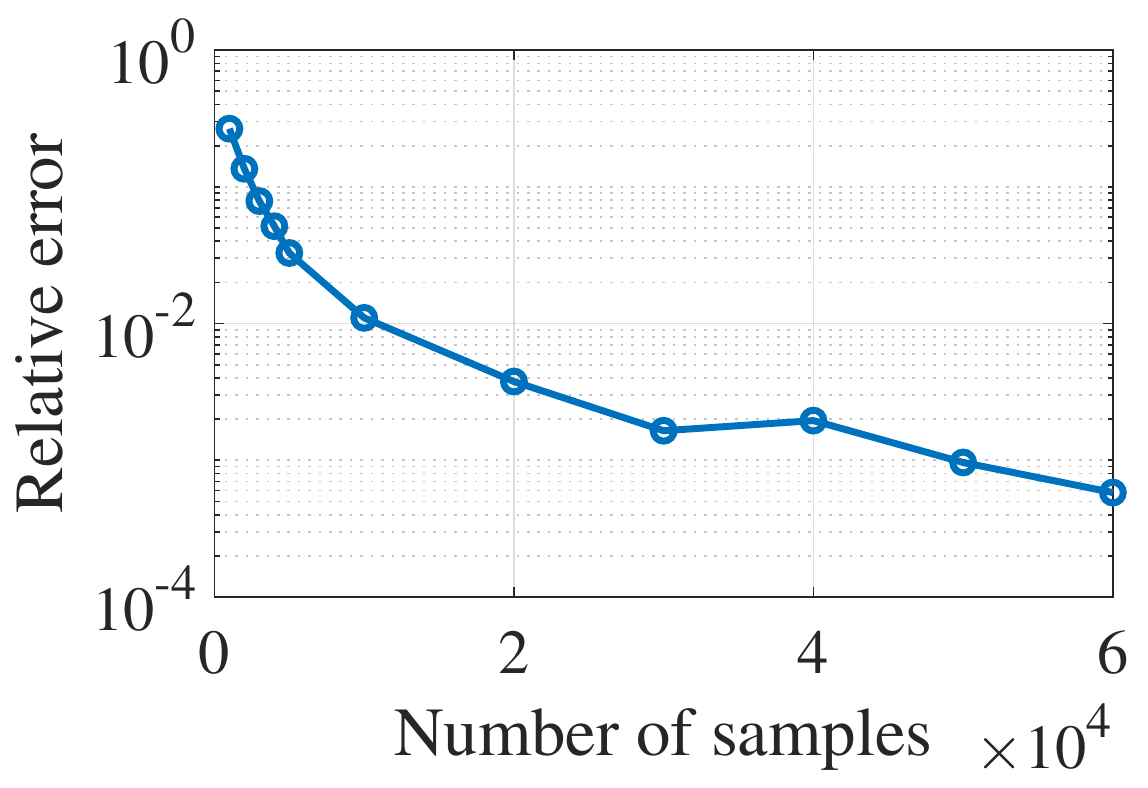}
	\caption{Distance between the returned model by Algorithm \ref{Alg} and the ground-truth model for binary  classification}
	\label{Figure: N_error_cl}
\end{figure}

%% file: chapter_5/rpichap5.tex
\chapter{\uppercase{A Theoretical Perspective of Sample Complexity on Pruned Neural Networks: One-hidden-layer Case}}\label{chapter: 5}
\blfootnote{Portions of this chapter previously appeared as: S.~Zhang, M. Wang, S.~Liu, P. Chen, J. Xiong, ``Why lottery ticket wins? A theoretical perspective of sample complexity on pruned neural networks'',  2021, \emph{arXiv: 2110.05667}.}

\input{./chapter_5/intro_new}

\input{./chapter_5/Prob}

\input{./chapter_5/structure}

\input{./chapter_5/Algorithm}
\input{./chapter_5/Simu_new}

\section{Summary}
This section provides the first theoretical analysis of learning one-hidden-layer pruned neural networks, which offers formal justification of the improved generalization of winning ticket observed from empirical findings in LTH. We characterize analytically the impact of the number of remaining weights in a pruned network on the required number of samples for training, the convergence rate of the  learning algorithm, and the accuracy of the learned model. We also provide extensive numerical validations of our theoretical findings.

%% file: chapter_5/intro_new.tex
\section{Introduction}
Neural network pruning  can reduce the computational cost of {model training and inference} significantly and potentially lessen the chance of overfitting 
\cite{LDS90,HS93,DCP17,HPTD15,HPTTC16,MAV17}, \cite{SB15,YCS17}. 
The recent  \textit{Lottery Ticket Hypothesis} (LTH)  \cite{FC18} 
claims that a randomly  initialized dense neural network always  contains a so-called ``winning ticket,'' which is a sub-network  bundled with the corresponding initialization, such that when trained in isolation, this winning ticket can achieve at least the same testing accuracy as that of the original network {by running at most the same amount of training time.}
This so-called   ``improved generalization of winning tickets'' is verified empirically in  \cite{FC18}. 
LTH has attracted a significant  amount  of recent research interests   \cite{RWKKR20,ZLLY19,MYSS20}.
Despite the empirical success \cite{MGE20,YLXF19,WZG19,CFCLZWC20}, the theoretical justification of winning tickets  remains elusive  except for a few recent works.  
 \cite{MYSS20}  provides the first
theoretical evidence that within a randomly initialized neural network, there exists a good sub-network that can achieve the same test performance as the original network. 
 Meanwhile, recent work \cite{N20} trains neural network by adding the $\ell_1$ regularization term to obtain a relatively sparse neural network, which has a better performance  numerically.  


However, {the theoretical foundation of 
{network pruning}
is limited. The existing theoretical works usually focus on  finding a sub-network  that achieves a tolerable loss in either expressive power or training accuracy, compared with the original dense network \cite{AGNZ18,ZVAAO18,TWL20,OHR20,BLGFR19,BLGFR18,LBLFR19,BOBZF20,TGNZKL20}.
To the best of our knowledge, {there exists no theoretical support for the}
\textit{improved} generalization achieved by
winning tickets, i.e., pruned networks with faster convergence and better  test accuracy. 
}

\textbf{Contributions}: This chapter provides the \textit{first} systematic   analysis of learning pruned neural networks with a finite number of training samples in the {oracle-learner} setup, where the training data are generated by a unknown neural network, the \textit{{oracle}}, and another network, the \textit{{learner}}, is trained on the dataset. Our analytical results  also provide  a  justification  of the LTH from the perspective of the sample complexity. 
In particular,
we provide the \textit{first} 
theoretical justification of the improved generalization of winning tickets.    Specific contributions include:


1. \textbf{Pruned neural network learning 
via accelerated gradient descent (AGD)}: 
We propose an AGD algorithm  with tensor initialization to learn the  {pruned} model from  training samples.
Our algorithm converges to the {oracle} model linearly, which has guaranteed generalization.

2. \textbf{First sample complexity analysis for pruned networks}: We characterize the required number of samples for successful convergence, termed as the \textit{sample complexity}. Our sample complexity bound  depends linearly on the number of {the non-pruned weights} and   is a significant reduction  from directly applying conventional complexity bounds in 
\cite{ZSJB17,ZWLC20_1,ZWLC20_3}.


3. \textbf{Characterization of the benign optimization landscape of pruned networks}: We show analytically that the empirical risk function has an \textit{enlarged} convex region for a pruned network, 
justifying 
the importance of a good sub-network {(i.e., the winning ticket)}.

4.  \textbf{Characterization of the improved generalization of winning tickets}: We show that gradient-descent methods converge faster to the {{oracle} model} when the neural network is properly pruned, or equivalently, learning on a pruned network returns a model closer to the {{oracle} model} with the same number of iterations, indicating the improved generalization of winning tickets.


\subsection{Related Work}

{\textbf{Network pruning}. Network pruning methods seek a compressed model 
while maintaining the expressive power. 
Numerical experiments have shown that  over 90\% of the parameters can be pruned without a significant performance loss \cite{CFCLZCW20}.  
Examples of pruning methods include {irregular weight pruning \cite{HPTD15}, structured weight pruning \cite{WWWYL16},} 
neuron-based pruning \cite{HPTTC16}, and projecting the weights to a low-rank subspace   \cite{DSDR13}.}  


\textbf{Winning tickets}.  \cite{FC18} employs an \textit{Iterative Magnitude Pruning} (IMP) algorithm to obtain the proper sub-network and initialization.  
IMP and its variations \cite{FDRC19,RFC19} 
succeed  in deeper networks like Residual Networks (Resnet)-50 and Bidirectional Encoder Representations from Transformers (BERT) network \cite{CFCLZWC20}. 
\cite{FDRC19_1} shows that IMP succeeds in finding the ``winning ticket'' if the ticket  is stable to stochastic gradient descent noise. In parallel, \cite{LSZHD18} shows numerically that the ``winning ticket'' initialization does not 
improve over a random initialization once the correct sub-networks are found, 
suggesting that the benefit  of ``winning ticket'' mainly comes from 
the sub-network structures.  {\cite{EKT20} analyzes the sample complexity of IMP from the perspective of recovering a sparse vector in a linear model rather than learning neural networks. } 

{\textbf{Feature sparsity}. High-dimensional data often contains   redundant features, and only a subset of the features is used in training \cite{BPSH15,DH20,HZRS16,YS18,ZHW15}. Conventional approaches like wrapper and filter methods score the importance of each feature in a certain way and select the ones with highest scores  \cite{GE03}. Optimization-based methods add variants of the $\ell_0$ norm as a   regularization to promote feature sparsity \cite{ZHW15}. Different from network pruning where the feature dimension   still remains high during training, the feature dimension is significantly reduced in training when promoting feature sparsity. } 

\textbf{Over-parameterized model.}
When  the number of weights in a neural network is  much larger than the number of training samples,  the landscape of the objective function of the learning problem  has no spurious local minima, and  first-order algorithms converge  to one of the global optima \cite{LSS14, OS20, ZBHRV16, SJL18,CRBD18,SM17,LMLLY20}. However, the global optima 
is
not guaranteed to generalize well 
 on testing data \cite{YS19,ZBHRV16}.

\textbf{Generalization analyses.}
The existing generalization analyses mostly fall within three categories. One line of research employs the Mean Field approach to model the training process   by a differential equation assuming infinite network width and infinitesimal training step size 
\cite{CB18b,MMN18, WLLM18}. 
Another approach is the neural tangent kernel (NTK) \cite{JGH18}, which requires   strong and probably unpractical over-parameterization such that the nonlinear neural network model behaves as its linearization around the initialization \cite{ALS19,DZPS19,ZCZG20,ZG19}. The third line of works follow the {oracle-learner} setup, where the data are generated by an unknown {oracle} model, and the  learning objective is to estimate the {oracle} model, which has a generalization guarantee on testing data. However, the objective function   has {intractably} many spurious local minima even for one-hidden-layer neural networks \cite{Sh18,SS17,ZBHRV16}. Assuming an infinite number of training samples, \cite{BG17,DLTPS17,Tian17} develop learning methods to estimate the {oracle} model. \cite{FCL20,ZSJB17,ZWLC20_1,ZWLC20_3} extend to the practical case of a finite number of samples and characterize  the sample complexity for recovering the {oracle} model. Because the analysis complexity explodes when the number of hidden layers increases, all the analytical results about estimating the {oracle} model are limited to one-hidden-layer neural networks, and the input distribution is often assumed to be the standard Gaussian distribution.

%% file: chapter_5/Prob.tex
\section{Problem Formulation}\label{Sec: Problem}

{In an oracle-learner model,}
given any input $\bfx \in \R^d$, the corresponding output $y$ 
	is generated  
 	by a pruned one-hidden-layer neural network,   called \textbf{\textit{{oracle}}}, 
  as shown in Figure \ref{fig:SNN}.
The {oracle} network is equipped with $K$ neurons where the $j$-th neuron is connected to  any arbitrary $r_j^*$ ($r_j^* \leq d$)
input features.
Let ${\bfW}^* =[{\bfw}^{*}_1,\cdots, {\bfw}^{*}_K]\in
\mathbb{R}^{d\times K}$ denotes all the  weights ({pruned ones are represented by zero}). 
The number of non-zero entries in $\bfw_j^*$ is at most $r_j^*$. The {oracle} network is not unique because permuting neurons together with the corresponding weights does not change the output. Therefore, the  output label $y$ obtained by the {oracle} network satisfies \footnote{It is extendable to binary classification, and the output is generated by $\text{Prob}\big(y_n=1|\bfx_n\big)= g(\bfx_n;{\bfW}^*)$.} 
\begin{equation}\label{eqn: SNN_model}
\begin{split}
\vspace{-1mm}
y =&  \frac{1}{K} \sum_{j=1}^K \phi( {{\bfw}^{*T}_j}\bfx) +  \xi
:=   g(\bfx;{\bfW}^*) + \xi = g(\bfx;{\bfW}^*\bfP) + \xi, 
\vspace{-1mm}
\end{split}
\end{equation}
where $\xi$ is arbitrary unknown {additive noise} bounded by some constant $|\xi|$,  $\phi$ is the rectified linear unit (ReLU) activation function with $\phi(z) = \max\{z , 0 \}$, and   $\bfP\in \{0,1\}^{K\times K}$ is any permutation matrix.
$\bfM^*$ is a \textbf{\textit{mask matrix}}  for the {oracle} network, such that $M^*_{j,i}$ equals to $1$ if the weight $\bfw^*_{j,i}$ is not pruned, and $0$ otherwise.  
Then, $\bfM^*$ is an indicator matrix for the non-zero entries of ${\bfW}^*$
with
 $\bfM^* \odot {\bfW}^* = {\bfW}^*$, where $\odot$ is entry-wise multiplication.

\begin{figure}
\centering
  	\includegraphics[width=0.5\linewidth]{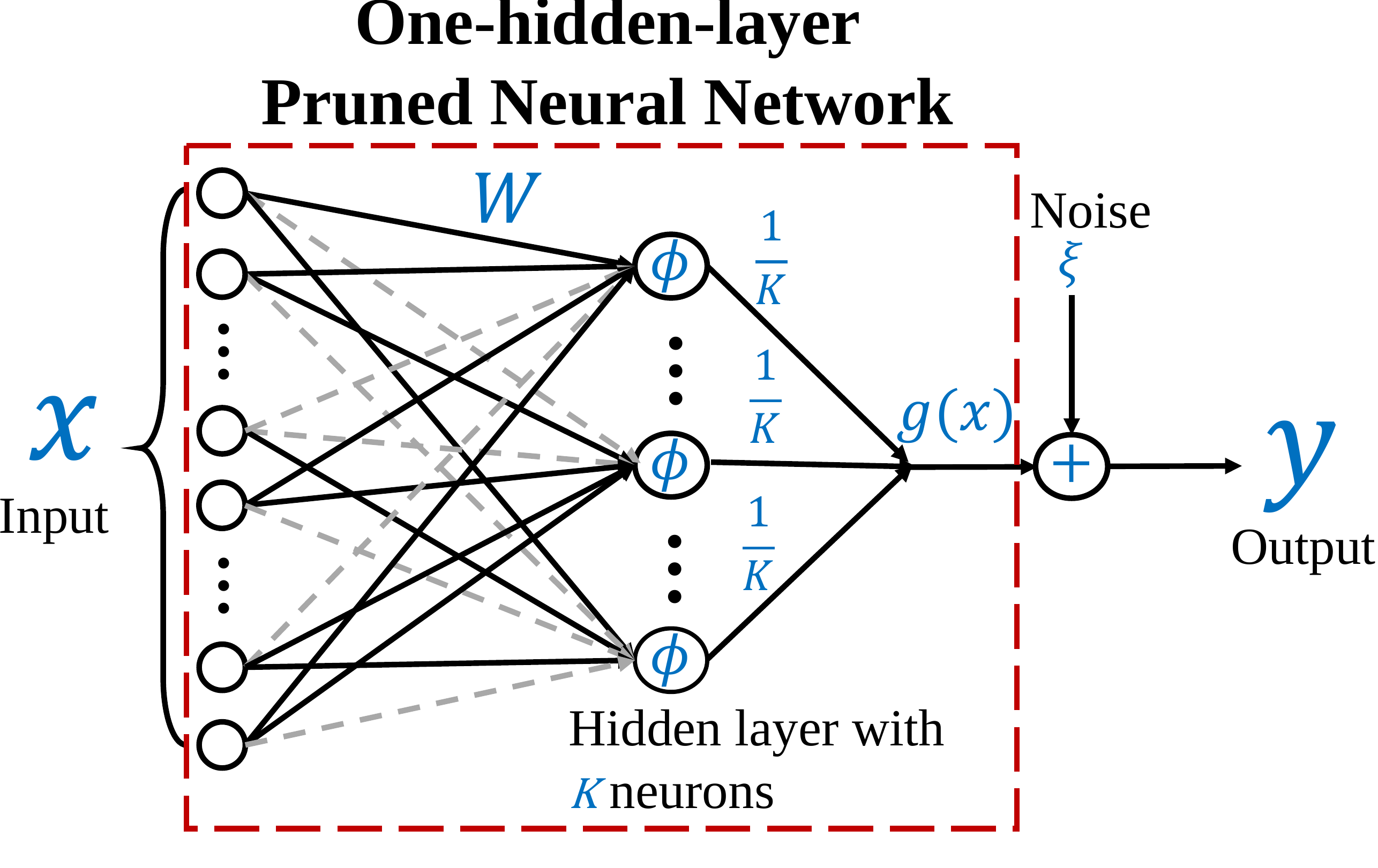}
  	\caption{Illustration of the model}
  	\label{fig:SNN}
  	\vspace{-4mm}
\end{figure}

{Based on  
$N$ pairs of training samples $\mathcal{D}=\{\bfx_n, y_n\}_{n=1}^N$ generated by the {oracle}, 
 we train on a \textbf{\textit{{learner}}} network equipped with the same number of neurons in the {oracle} network. However, the $j$-th neuron in the {learner} network is connected to $r_j$   input features rather than $r_j^*$. Let $r_{\min}$,   $r_{\max}$, and  $r_{\textrm{ave}}$ denote the minimum, maximum, and average value of $\{r_j\}_{j=1}^K$, respectively.  
Let $\bfM$ denote the mask matrix with respect to the {learner} network, and $\bfw_j$ is the $j$-th column of $\bfW$. } 
The empirical risk function is defined as
\begin{equation}\label{eqn: sample_risk}
\begin{split}
\hat{f}_{\mathcal{D}}(\bfW)
=&\frac{1}{2N}\sum_{n=1}^{N}\big(\frac{1}{K} \sum_{j=1}^K\phi({\bfw}_j^T\bfx_{n})-y_n\big)^2.
\end{split}
\end{equation}
When the mask $\bfM$ is given, the learning objective is to estimate a proper weight matrix ${\bfW}$ for the \textit{{learner}} network from the training samples $\mathcal{D}$ via solving
\begin{equation}\label{eqn: optimization_problem_5}
\begin{split}
&{\textstyle\min_{{\bfW}\in \mathbb{R}^{d\times K}} } \quad \hat{f}_{\mathcal{D}}({\bfW}) \qquad \text{s.t.}\quad \bfM\odot \bfW =\bfW.
\end{split}
\end{equation} 
$\bfM$ is called an \textbf{\textit{accurate 
mask}} if the support of $\bfM$ covers the support of a permutation of $\bfM^*$, i.e., there exists a permutation matrix $\bfP$ such that  $(\bfM^*\bfP) \odot \bfM = \bfM^*$. 
When $\bfM$ is accurate, and $\xi=0$,  there exists a permutation matrix $\bfP$ such that $\bfW^*\bfP$ is a global optimizer to \eqref{eqn: optimization_problem_5}.
Hence, if $\bfW^*\bfP$ can be estimated 
by solving  \eqref{eqn: optimization_problem_5},  one can learn the {oracle} network accurately, which has guaranteed generalization performance on the testing data. 

We assume $\bfx_n$ is independent and identically distributed from the standard Gaussian distribution $\mathcal{N}({\boldsymbol{0}},\bfI_{d\times d})$. The Gaussian assumption is motivated by the data whitening \cite{LBOM12} and batch normalization techniques \cite{IS15} that are commonly used in practice to improve learning performance.
Moreover, training one-hidden-layer neural network with multiple neurons has intractable many fake minima \cite{SS17} without any input distribution assumption. In addition,
the theoretical results in Section \ref{sec:algorithm_and_theorem} assume an accurate mask, and inaccurate mask is evaluated empirically in Section \ref{sec:simu}.

The questions addressed in this chapter include: 1. \textbf{what algorithm} to  solve \eqref{eqn: optimization_problem_5}? 
2. 
what is the \textbf{sample complexity}  for the accurate estimate of
{the weights in the {oracle} network}?
 3. what is the \textbf{impact of the network pruning} on the difficulty of the learning problem and the performance of the {learned} model?

%% file: chapter_5/structure.tex
\section{Algorithm and Theoretical Results}\label{sec:algorithm_and_theorem}
Section \ref{sec:structure}  studies the geometric structure of \eqref{eqn: optimization_problem_5},  and  the main results are in Section \ref{sec: algorithm}. Section \ref{sec: proof_sketch} briefly introduces the proof sketch and technical novelty..
\subsection{Local Geometric Structure}\label{sec:structure}
Theorem \ref{Thm: convex_region}  characterizes the local convexity of $\hat{f}_{\mathcal{D}}$ in \eqref{eqn: optimization_problem_5}. It has two important implications. 

1. \textbf{Strictly locally convex near ground truth}:  $\hat{f}_{\mathcal{D}}$  is strictly convex near   $\bfW^*\bfP$ for some permutation matrix $\bfP$, and the radius of the convex ball is negatively correlated with $\sqrt{\widetilde{r}}$, where $\widetilde{r}$  is in the order of $r_\textrm{ave}$. 
Thus, the convex ball enlarges as  any $r_j$  decreases.

2. \textbf{Importance of the winning ticket architecture}: Compared with training on the dense network directly, training on a properly pruned sub-network has a larger local convex region near   $\bfW^*\bfP$, which   may 
lead to   easier estimation of $\bfW^*\bfP$.    To some extent, this result can be viewed as a theoretical validation of  the importance of the winning architecture (a good sub-network) in \cite{FC18}.  
Formally, we have
\begin{theorem}[Local Convexity]
\label{Thm: convex_region}
Assume the mask $\bfM$ of the {learner} network is accurate. 
Suppose constants $\varepsilon_0$, $\varepsilon_1\in(0,1)$ and  the number of samples satisfies
    ~
	\begin{equation}\label{eqn: sample_complexity_1}
	N   = \Omega\big(\varepsilon_1^{-2} K^4 {\widetilde{r}}\log q\big), 
	\end{equation}
    ~	
	for some large constant $q>0$, where 
	\begin{equation}\label{eqn: r_tilde}
	\widetilde{r} = \frac{1}{8K^4}\Big({\textstyle \sum_{k=1}^K\sum_{j=1}^K}{(1+\delta_{j,k})}({r_j+r_k})^{\frac{1}{2}}\Big)^2,
	\end{equation}
	$\delta_{j,k}$ 
	is $1$ if the indices of non-pruned weights in the $j$-th and $k$-th neurons overlap and 0 otherwise.
	Then, there exists a permutation matrix $\bfP$ such that for any $\bfW$ that satisfies
	\begin{equation}\label{eqn: convex_region}
	 \| \bfW -\bfW^*\bfP \|_F  = \mathcal{O}\big(\textstyle\frac{\varepsilon_0}{K^2}\big),
	 \quad \text{and } \bfM\odot\bfW = \bfW,
	\end{equation} 
	its Hessian of $\hat{f}_{\mathcal{D}}$, with probability at least $1- K\cdot q^{-r_{\min}}$, is bounded as: 
	\begin{equation}\label{eqn: convex_condition}
	\Theta\Big({\frac{1-\varepsilon_0-\varepsilon_1}{K^2}} \Big) \bfI \preceq
	\nabla^2\hat{f}_{\mathcal{D}}({\bfW}) 
	\preceq
	 \Theta \Big({\frac{1}{K}}\Big)\bfI.  
	\end{equation}
\end{theorem}
\noindent

\textbf{Remark 1.1 (Parameter $\widetilde{r}$)}: Clearly ${\widetilde{r}}$ is a monotonically increasing function of any $r_j$ from \eqref{eqn: r_tilde}. Moreover, one can check that 
    $\frac{1}{8}r_{\text{ave}}\le\widetilde{r}\le r_{\text{ave}}$. 
Hence, $\widetilde{r}$ is in the order of $r_{\text{ave}}$.

\textbf{Remark 1.2 (Local landscape)}:
Theorem \ref{Thm: convex_region} shows that  with   enough samples as shown in \eqref{eqn: sample_complexity_1},  in a local region of  $\bfW^*\bfP$   as shown in \eqref{eqn: convex_region},   all the eigenvalues of the Hessian matrix of the empirical risk function are 
lower and upper bounded by two positive constants. 
This property  is useful in designing efficient algorithms to recover $\bfW^*\bfP$, as shown 
in Section \ref{sec: algorithm}.

\textbf{Remark 1.3 (Size of the convex region)}:  When the number of samples $N$ is fixed and $r$ changes,  $\varepsilon_1$ can be {$\Theta(\sqrt{\widetilde{r}/N})$} while \eqref{eqn: sample_complexity_1} is still met.     $\varepsilon_0$  in (\ref{eqn: convex_condition}) can be arbitrarily close to but smaller than $1-\varepsilon_1$ so that the Hessian matrix is still positive definite.  Then from \eqref{eqn: convex_region},  the radius of the convex ball is  $ \Theta(1)-\Theta(\sqrt{\widetilde{r}/N})$, indicating an enlarged region when $\widetilde{r}$ decreases. The enlarged convex region serves as an important component in proving the faster convergence rate, summarized in  Theorem \ref{Thm: major_thm}. Besides this,  as Figure 1 shown in [20], the authors claim that the learning is stable if the linear interpolation of the learned models with SGD noises still remain  similar in performance, which is summarized as the concept ``linearly connected region.'' Intuitively, we conjecture that the winning ticket shows a better performance in the stability analysis because it has a larger convex region. In the other words, a larger convex region indicates that the learning is more likely to be stable in the linearly connected region.

%% file: chapter_5/Algorithm.tex
\subsection{Convergence Analysis with Accelerated Gradient Descent}
 \label{sec: algorithm}




We propose to solve  the non-convex problem \eqref{eqn: optimization_problem_5} via 
the accelerated   gradient descent (AGD) algorithm, summarized in Algorithm \ref{Alg5}.
Compared with the vanilla gradient descent (GD) algorithm, AGD has an additional  momentum term, denoted by $\beta(\bfW^{(t)}- \bfW^{(t-1)})$,   in each iteration. AGD enjoys a faster convergence rate than vanilla GD   in solving    optimization problems, including learning neural networks \cite{ZWLC20_2}.  
Vanilla GD can be viewed as a special case of AGD by letting $\beta = 0$. The initial point $\bfW^{(0)}$ can be obtained through a tensor method..

\floatname{algorithm}{Algorithm}
\begin{algorithm}[h]
	\caption{Accelerated Gradient Descent (AGD) Algorithm}\label{Alg5}
	\begin{algorithmic}[1]
		\State \textbf{Input:} training data $\mathcal{D}=\{(\bfx_n, y_n) \}_{n=1}^{N}$, gradient step size $\eta$, momentum parameter $\beta$, and an initialization $\bfW^{(0)}$ by the tensor initialization method;
		\State {Partition} $\mathcal{D}$ into $T=\log(1/\varepsilon)$ disjoint subsets, denoted as $\{\mathcal{D}_t\}_{t=1}^T$;
		\For  {$t=1, 2, \cdots, T$}		
		\State $\bfW^{(t+1)}=\bfW^{(t)}-\eta \cdot \bfM \odot\nabla \hat{f}_{\mathcal{D}_t}(\bfW^{(t)})		  +\beta(\bfW^{(t)}-\bfW^{(t-1)})$
		\EndFor
		\State\textbf{Return:} $\bfW^{(T)}$
	\end{algorithmic}
\end{algorithm}
The theoretical analyses of our algorithm are summarized in Theorem \ref{Thm: major_thm} (convergence) and Lemma \ref{Thm: initialization} (Initialization). The significance of these results can be interpreted from the following aspects.

\noindent
1. \textbf{Linear convergence to the {oracle} model}: Theorem \ref{Thm: major_thm} implies that if initialized in the local convex region,  the iterates generated by AGD   converge linearly to    $\bfW^*\bfP$ for some $\bfP$ when noiseless. When there is noise, they converge to a point $\bfW^{(T)}$. The distance between  $\bfW^{(T)}$ and $\bfW^*\bfP$ is proportional to the noise level and     scales in terms of $\mathcal{O}(\sqrt{\widetilde{r}/N})$.    Moreover, when $N$ is fixed, the convergence rate of AGD is   $\Theta( \sqrt{{\widetilde{r}}/{K}})$.  Recall that Algorithm \ref{Alg5} reduces to the vanilla GD by setting $\beta = 0$.  The rate for the vanilla GD algorithm here is $\Theta({\sqrt{\widetilde{r}}}/{K})$ by setting $\beta = 0$  by Theorem \ref{Thm: major_thm}, indicating a slower convergence than AGD.  
Lemma \ref{Thm: initialization} shows  the tensor initialization method indeed returns {an} initial point in the convex region. 



\noindent
2. \textbf{Sample complexity for accurate  estimation}: We show that the required number of samples for successful estimation of the {oracle} model is   $\Theta\big(\widetilde{r}\log q\log (1/\varepsilon)\big)$ for some large constant $q$ and estimation accuracy $\varepsilon$. 
Our sample complexity is much less than the conventional bound of $\Theta(d\log q\log(1/\varepsilon))$ for one-hidden-layer networks \cite{ZSJB17,ZWLC20_1,ZWLC20_3}. This is the first theoretical characterization of learning a pruned network from the perspective of sample complexity.  

\noindent
3. \textbf{Improved generalization of winning tickets}: 
We prove that with a fixed number of training samples, 
training on a properly pruned sub-network converges faster to   $\bfW^*\bfP$ than training on the original dense network. Our theoretical analysis justifies that training on the winning ticket can meet or exceed the same test  accuracy within the same number of iterations. To the best of our knowledge, our result here provides  the first theoretical justification for this {intriguing empirical finding of} ``improved generalization of winning tickets'' by  \cite{FC18}.  


\begin{theorem}[Convergence]
\label{Thm: major_thm}
{Assume the mask  $\bfM$ of the {learner} network is accurate.} 	Suppose $\bfW^{(0)}$ satisfies \eqref{eqn: convex_region} and the number of samples satisfies  
	\begin{equation}\label{eqn: sample_complexity_main}
	N = \Omega\big(\varepsilon_0^{-2}  K^6 \widetilde{r}\log q \log(1/\varepsilon)\big)
	\end{equation} 
	for some $\varepsilon_0\in(0,{1}/{2})$. Let $\eta= {K}/{14}$ in Algorithm 1.  Then, the iterates $\{ \bfW^{(t)} \}_{t=1}^{T}$ returned by Algorithm 1 converges linearly to $\bfW^*$ up to the noise level with probability at least $1-K^2T\cdot q^{-r_{\min}}$  
	\begin{equation}\label{eqn: linear_convergence_lr}
	\begin{split}
	\|	\W[t]-\bfW^*\bfP\|_F
	\le&\nu(\beta)^t\|	\W[0]-\bfW^*\bfP\|_F 
	+ \mathcal{O}\Big({{\sum_{j}\sqrt{\frac{ r_j\log q}{N}}}}\Big) \cdot |\xi|, 
	\end{split} 
	\end{equation}
    ~
	\begin{equation}\label{eqn: linear_convergence_lr_2}
	\begin{split}
	\text{and \quad}
	\|	\W[T]-\bfW^*\bfP\|_F
	\le &\varepsilon \|\bfW^* \|_F 
	+\mathcal{O}\Big({{\sum_{j}\sqrt{\frac{ r_j\log q}{N}}}}\Big) \cdot |\xi|,
	\end{split}
	\end{equation}
    for a fixed permutation matrix $\bfP$,
	where $\nu(\beta)$ is the  rate of convergence that depends on $\beta$ with { $\nu(\beta^*)= 1-\Theta\big(\frac{1-\varepsilon_0}{\sqrt{K}}\big)$  for some non-zero $\beta^*$ and $\nu(0) = 1-\Theta\big(\frac{1-\varepsilon_0}{{K}}\big)$.}
\end{theorem}

\begin{lemma}[Initialization]
\label{Thm: initialization}
    Assume the noise $|\xi|\le \|\bfW^*\|_2$ and the number of samples
$N = \Omega\big( \varepsilon_0^{-2} K^5 r_{\max} \log  q \big)$
    for $\varepsilon_0>0$ and large constant $q$, the tensor initialization method 
    outputs  $\bfW^{(0)}$  such that (\ref{eqn: convex_region}) holds, i.e., 
$	\| \bfW^{(0)} -\bfW^* \|_F  = \mathcal{O}\big(\frac{\varepsilon_0\sigma_K}{K^2}\big)$ with probability at least $1-q^{-r_{\max}}$.
\end{lemma}



 \textbf{Remark 2.1 (Faster convergence on pruned network)}:
With a fixed number of samples, when $\widetilde{r}$ decreases, $\varepsilon_0$ can increase as $\Theta(\sqrt{\widetilde{r}})$   while     \eqref{eqn: sample_complexity_main} is still met. 
{Then $\nu(0)=\Theta(\sqrt{\widetilde{r}}/K)$ and $\nu(\beta^*)=\Theta(\sqrt{\widetilde{r}/K})$.} Therefore,  when $\widetilde{r}$ decreases,  both the stochastic  and   accelerated gradient descent converge faster.  Note that as long as $\bfW^{(0)}$  is initialized in the local convex region,  not necessarily  by the tensor method,  Theorem \ref{Thm: major_thm} guarantees the accurate recovery. 
\cite{ZWLC20_1,ZWLC20_3} analyze AGD on  convolutional neural networks, 
while this chapter focuses on network pruning.

\textbf{Remark 2.2} (\textbf{Sample complexity of initialization}): From Lemma \ref{Thm: initialization}, the required number of samples for a proper initialization is $\Omega\big( \varepsilon_0^{-2} K^5 r_{\max} \log  q \big)$. Because $r_{\max}\le Kr_{\text{ave}}$ and  $\widetilde{r}=\Omega(r_{\text{ave}})$, this number is no greater than the sample complexity in (\ref{eqn: sample_complexity_main}). Thus, provided that (\ref{eqn: sample_complexity_main}) is met, Algorithm 1 can estimate the {oracle} network model accurately.

{{\bf Remark 2.3 (Inaccurate mask)}: The above analyses are based on the assumption that the mask of the {learner} network is accurate. In practice, a mask can be obtained by an iterative pruning method  such as \cite{FC18} or a one-shot pruning method  such as \cite{WZG19}. In Appendix, we prove that the magnitude pruning method can obtain an accurate mask with enough training samples. Moreover, empirical experiments in Section \ref{sec: grasp} and \ref{sec:IMP} suggest that 
	even if the mask is not  accurate, the three properties (linear convergence, sample complexity with respect to the network size, and improved generalization of winning tickets) can still hold. Therefore, our theoretical results provide some insight into the empirical success of network pruning. }

\subsection{The Sketch of Proofs and Technical Novelty}\label{sec: proof_sketch}
 {Our proof outline {is inspired by} \cite{ZSJB17} on fully connected neural networks, however,  major technical changes are made in this chapter   to generalize the analysis to an arbitrarily pruned network. To characterize the local convex region of $\hat{f}_{\mathcal{D}}$ (Theorem 1), the idea is to bound the Hessian matrix of the population risk function, which is the expectation of the empirical risk function, locally and then  characterize the distance between the empirical and population risk functions through the  concentration bounds. Then, the convergence of AGD (Theorem 2) is established based on the desired local curvature, which in turn determines the   sample complexity. Finally, to initialize in the local convex region (Lemma 1), we construct tensors that contain the weights information and apply a decomposition method to estimate the weights.}

{Our technical novelties are as follows.  First, a direct application of the results in \cite{ZSJB17} leads to a sample complexity bound that is linear in the feature dimension $d$. We develop new techniques to tighten the sample complexity bound to be linear in $\tilde{r}$, which can be significantly smaller than $d$ for  a sufficiently pruned network. Specifically, we develop new concentration bounds (Lemmas \ref{Lemma: second_order} and \ref{Lemma: first_order} in   Appendix) to bound the distance between the population and empirical risk functions rather than using the bound in \cite{ZSJB17}. 
Second, instead of restricting the acitivation to be smooth for convergence analysis, we study the case of ReLU function which is non-smooth.
Third, new tensors are constructed for pruned networks (see (\ref{eqn: M_1})-(\ref{eqn: M_3}) in Appendix) in computing the initialization, and our new concentration bounds are employed  to reduce the required number of samples for a proper initialization.
Last, Algorithm 1 employs AGD and is proved to converge faster than the GD algorithm in \cite{ZSJB17}.}

%% file: chapter_5/Simu_new.tex
\section{Numerical Experiments}\label{sec:simu}
 
The theoretical results are first verified on synthetic data, and we then analyze the pruning performance 
on both synthetic and real datasets. 
{In Section 4.1, Algorithm \ref{Alg5} is implemented   
with minor modification, 
such that, the initial point is randomly selected as
$ \| \bfW^{(0)} - \bfW^* \|_F/\|\bfW^* \|_F < \lambda$
for some $\lambda>0$ to reduce the computation.}
Algorithm   \ref{Alg5} terminates when $\|\bfW^{(t+1)}-\bfW^{(t)} \|_F/\|\bfW^{(t)}\|_F$ is smaller than $10^{-8}$ or reaching $10000$   iterations. In Sections   \ref{sec: grasp} and \ref{sec:IMP},   the
Gradient Signal Preservation
(GraSP) algorithm  \cite{WZG19} and IMP algorithm \cite{CFCLZCW20,FC18}\footnote{The source codes used are downloaded from https://github.com/VITA-Group/CV\_LTH\_Pre-training.} are implemented to prune the neural  networks. 
As   many works like \cite{CFCLZWC20,CFCLZCW20,FC18}  
 have already verified the faster convergence   and better generalization {accuracy} of the winning tickets empirically,  we only include  the results of some representative experiments, such as training MNIST and CIFAR-10 on Lenet-5 \cite{lenet5} and Resnet-50 \cite{HZRS16} networks, to verify our theoretical findings.

The synthetic data are generated using a {oracle} model in Figure  \ref{fig:SNN}. The input $\bfx_n$'s are randomly generated from Gaussian distribution $\mathcal{N}(0, \bfI_{d\times d})$ independently, and indices of non-pruned weights of the $j$-th neuron  are obtained by randomly selecting  $r_j$ numbers without replacement from  $[d]$.  
For the convenience of generating specific $\widetilde{r}$, the indices of non-pruned weights are almost overlapped ($\sum_j\sum_k\delta_j\delta_k > 0.95 K^2$) except for Figure \ref{Figure: delta_N}.
In Figures \ref{fig: r_eps} and \ref{fig: sparse}, $r_j$ is selected uniformly from $[0.9\widetilde{r}, 1.1\widetilde{r}]$ for a given $\widetilde{r}$, and $r_j$  are the same in value for all $j$ in other figures. 
Each non-zero entry of   $\bfW^*$ is randomly selected from $[-0.5, 0.5]$ independently.   The noise $\xi_n$'s are i.i.d. from $\mathcal{N}(0, \sigma^2)$,
 and the noise level is measured by $\sigma/E_y$, where $E_y$ is the root mean square of the noiseless outputs.



\subsection{Evaluation of Theoretical Findings on  Synthetic Data}

\textbf{Local convexity near $\bfW^*$.}  We set the  number of neurons $K= 10$, the dimension of the data $d=500$ and the sample size $N=5000$. Figure \ref{fig: r_eps} illustrates the success rate of Algorithm \ref{Alg5} when  $\widetilde{r}$ changes. The $y$-axis is the relative distance of the initialization $\bfW^{(0)}$ to the ground-truth. For each pair of $\widetilde{r}$ and the initial distance,  $100$ trails are constructed with  the network weights, training data and the initialization $\bfW^{(0)}$ are all generated independently in each trail. Each trail is called successful if the  relative error of the solution $\bfW$ returned by Algorithm \ref{Alg5}, measured by $\|\bfW -\bfW^*\|_F/\|\bfW^*\|_F$, is less than   $10^{-4}$.
 A black block means Algorithm 1 fails in estimating $\bfW^*$ in all trails while a white block indicates all success. As Algorithm \ref{Alg5} succeeds if   $\bfW^{(0)}$ is in the local convex region near $\bfW^*$, we can see that the radius of   convex region is indeed  linear in $-\widetilde{r}^{\frac{1}{2}}$, as predicted by Theorem \ref{Thm: convex_region}.

\textbf{Convergence rate.} Figure \ref{fig: con_r} shows the convergence rate of Algorithm \ref{Alg5} when   $\widetilde{r}$ changes.   $N=5000$, $d=300$,  $K=10$, $\eta=0.5$, and $\beta=0.2$.  Figure \ref{fig: con_r}(a) shows that the relative error decreases exponentially as the number of iterations increases, indicating the linear convergence of Algorithm \ref{Alg5}. {As shown in Figure \ref{fig: con_r}(b), 
the results are averaged over $20$ trials with different initial points, and the areas in low transparency represent the standard deviation errors.}
We can see that the  convergence rate is almost linear in $1/\sqrt{\widetilde{r}}$, as predicted by Theorem \ref{Thm: major_thm}. We also compare with GD by setting $\beta$ as 0. One can see that AGD has a smaller convergence rate than GD, indicating faster convergence. 
\begin{figure}
    \centering
	\includegraphics[width=0.5\linewidth]{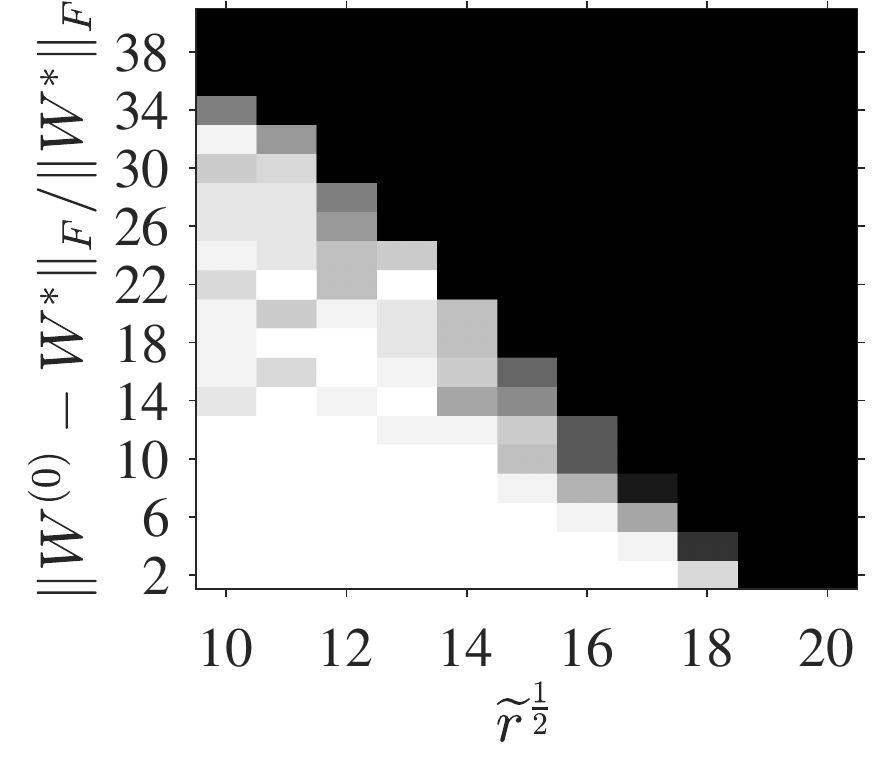}
	\caption{The radius of the local convex region against $\widetilde{r}^{\frac{1}{2}}$}
	\label{fig: r_eps}
\end{figure}

\textbf{Sample complexity.} Figures \ref{fig: sparse} and \ref{Figure: delta_N} show the success rate of Algorithm \ref{Alg5} when varying $N$ and $\widetilde{r}$.  $d$ is fixed as $100$. In Figure \ref{fig: sparse},  we construct $100$ independent trails for each pair of $N$ and $\widetilde{r}$, where the ground-truth model and training data are generated independently in each trail. 
One can see that the required number of samples for successful estimation is   linear in $\widetilde{r}$, as predicted by (\ref{eqn: sample_complexity_main}). In Figure \ref{Figure: delta_N}, $r_j$ is fixed as $20$ for all neurons, but different network architectures after pruning are considered. One can see that although the number of remaining weights is the same,  $\widetilde{r}$  can be different in different architectures, and the sample complexity increases as $\widetilde{r}$ increases, as predicted by (\ref{eqn: sample_complexity_main}). 
\begin{figure}[h]
\begin{minipage}{1.0\linewidth}
	\centering
	\includegraphics[width=0.8\linewidth]{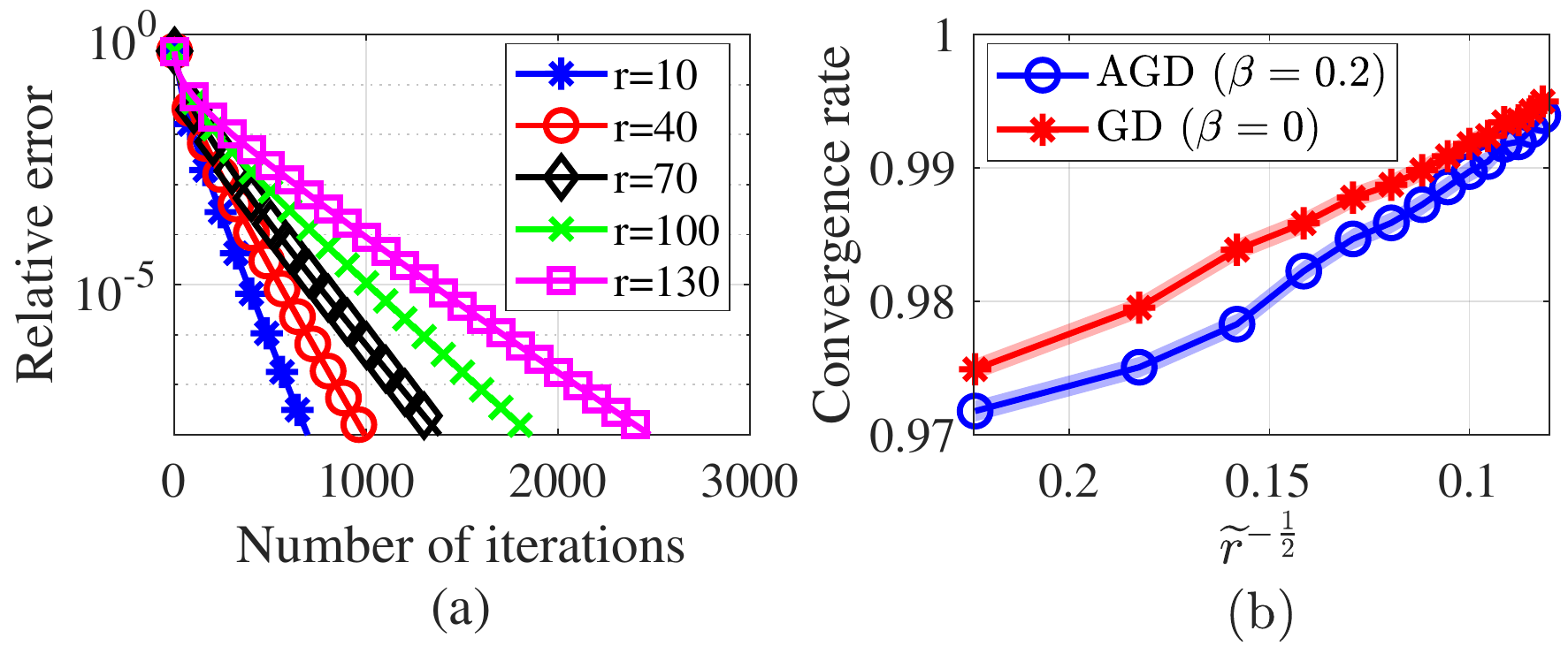}
	\caption{Convergence rate when $\widetilde{r}$ changes}
	\label{fig: con_r}
\end{minipage}
\end{figure} 

\begin{figure}
    \centering
    \includegraphics[width=0.6\linewidth]{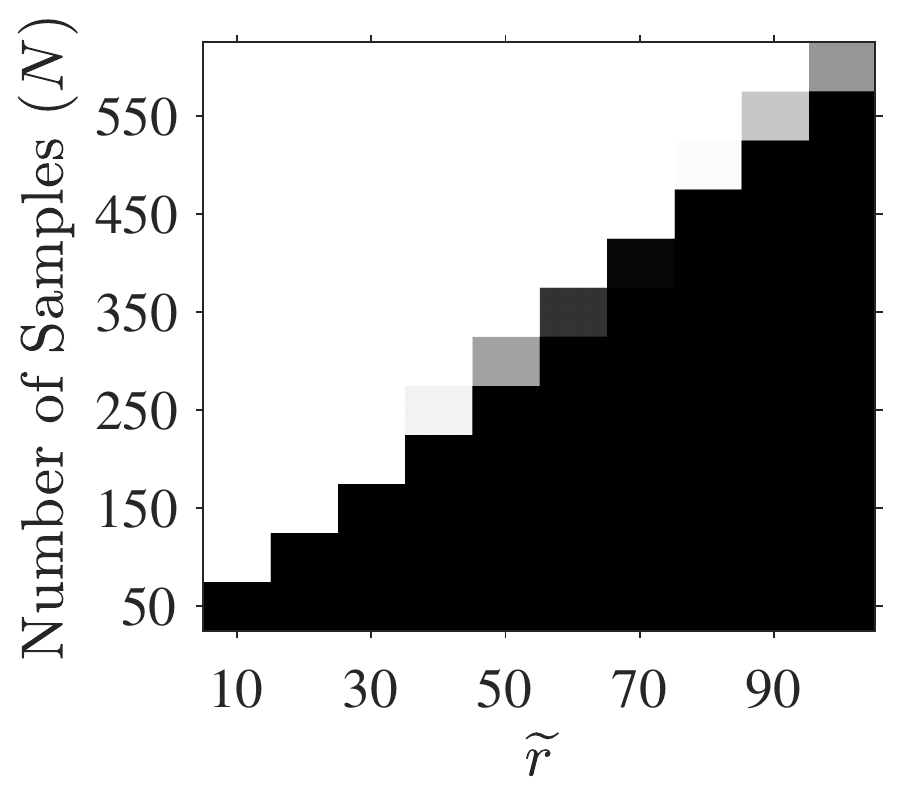}
	\caption{Sample complexity when $\widetilde{r}$ changes}
	\label{fig: sparse} 
\end{figure}

\begin{figure}
        \centering
	\includegraphics[width=0.6\textwidth]{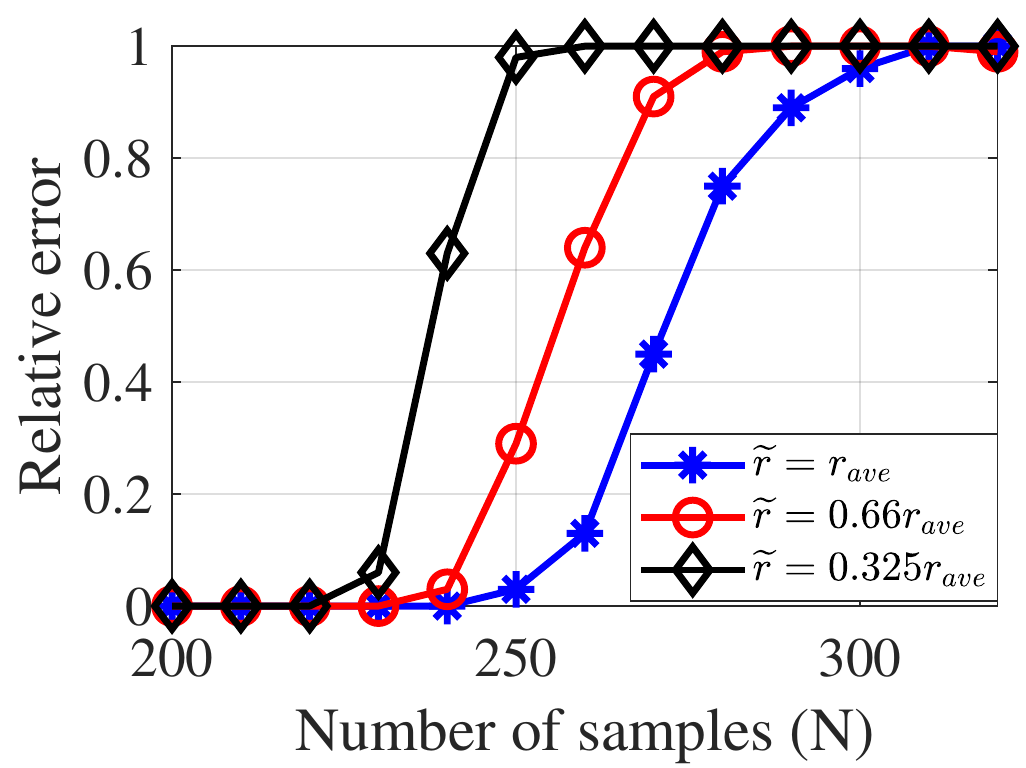}
	\caption{Relative error against $\widetilde{r}$}
	\label{Figure: delta_N}
\end{figure}

\begin{figure}
        \centering
	\includegraphics[width=0.6\textwidth]{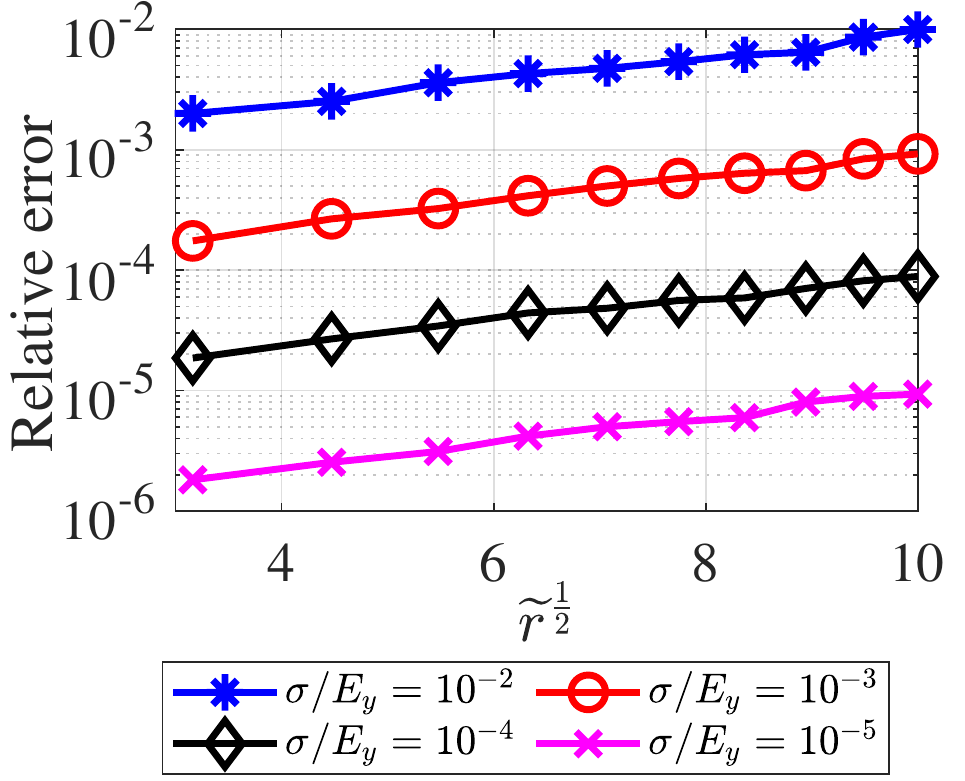}
	\caption{Relative error against $\widetilde{r}^{\frac{1}{2}}$ at different noise level}
	\label{Figure: noise_case}
\end{figure}

\textbf{Performance in noisy case.} Figure \ref{Figure: noise_case} shows the relative error of the learned model by Algorithm \ref{Alg5} from noisy measurements  when $\widetilde{r}$ changes. 
$N=1000$, $K=10$, and $d=300$. The results are averaged over $100$ independent trials, and standard deviation is around $2\%$ to $8\%$ of the corresponding relative errors.
The relative error is linear in $\widetilde{r}^{\frac{1}{2}}$, as predicted by (\ref{eqn: linear_convergence_lr}). Moreover, the relative error is proportional to the noise level $|\xi|$. 

\subsection{Performance with Inaccurate Mask on Synthetic Data}\label{sec: grasp}
The performance of Algorithm \ref{Alg5} is evaluated when  the mask $\bfM$ of the {learner} network is inaccurate.  The number of neurons $K$ is $5$. The dimension of inputs $d$ is  $100$. $r_j^*$ of the {oracle} model is $20$ for all $j\in[K]$.
GraSP algorithm  \cite{WZG19} is employed to find  masks based only on early-trained weights in $20$ iterations of AGD.
The mask accuracy is measured by $\|\bfM^* \odot \bfM\|_0/\|\bfM^* \|_0$, where $\bfM^*$ is the mask of the {oracle} model. The pruning ratio is defined as $(1-{r_{\textrm{ave}}}/{d}) \times 100\%$. The number of training samples   $N$ is $200$.
{The model returned by  Algorithm \ref{Alg5}  is evaluated on $N_{\text{test}}=10^5$ samples, and the test error is measured by  $\sqrt{ \sum_{n} |y_n - \hat{y}_n|^2/ N_{\text{test}} }$, where $\hat{y}_n$ is the   output of the learned model with the input  $\bfx_n$, and $(\bfx_n, y_n)$ is the $n$-th testing sample generated by the {oracle} network. }




{\bf Improved generalization by 
{GraSP}.
} Figure \ref{fig: g_n} shows the test error with different pruning ratios.  
 For each pruning ratio, we randomly generate $1000$ independent trials. Because the mask of the {learner} network in each trail is generated independently, 
we compute the average  test error of the learned models in all the trails with same mask accuracy. If there are less than $10$ trails for  certain mask accuracy, the result of that mask accuracy is not reported as it is statistically meaningless. 
The test error decreases as the mask accuracy increases. More importantly, at   fixed mask accuracy, the test error decreases as the pruning ratio increases. That means the generalization performance improves when $\widetilde{r}$ deceases, even if the mask is not accurate. 

\begin{figure}
    \centering
	\includegraphics[width=0.5\linewidth]{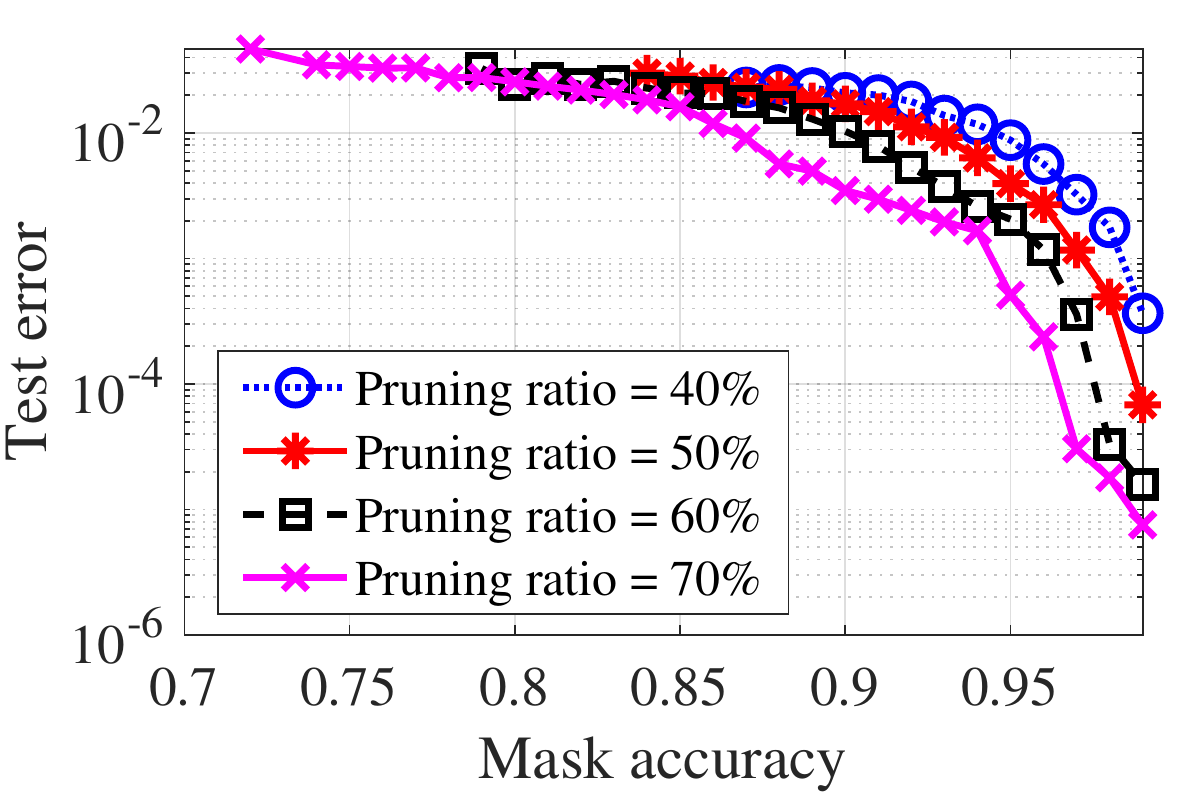}
	\caption{Test error against mask accuracy with different pruning ratios}
	\label{fig: g_n}
\end{figure}

\begin{figure}
    	\centering
	\includegraphics[width=0.5\linewidth]{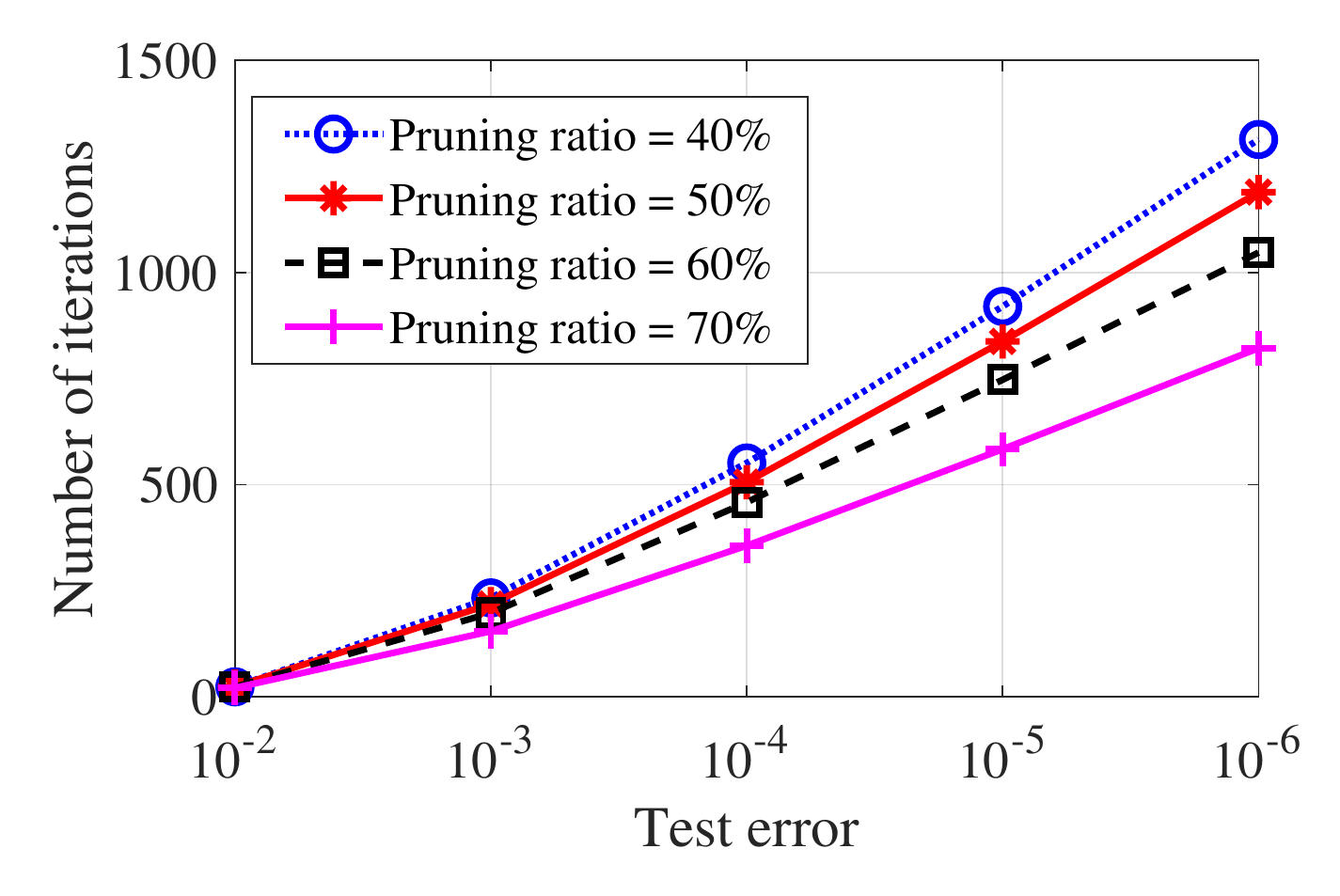}
	\caption{Convergence rate with mask accuracy in $[0.85,0.9]$}
	\label{fig: ite_rate}
\end{figure}

\begin{figure}
    	\centering
	\includegraphics[width=0.5\linewidth]{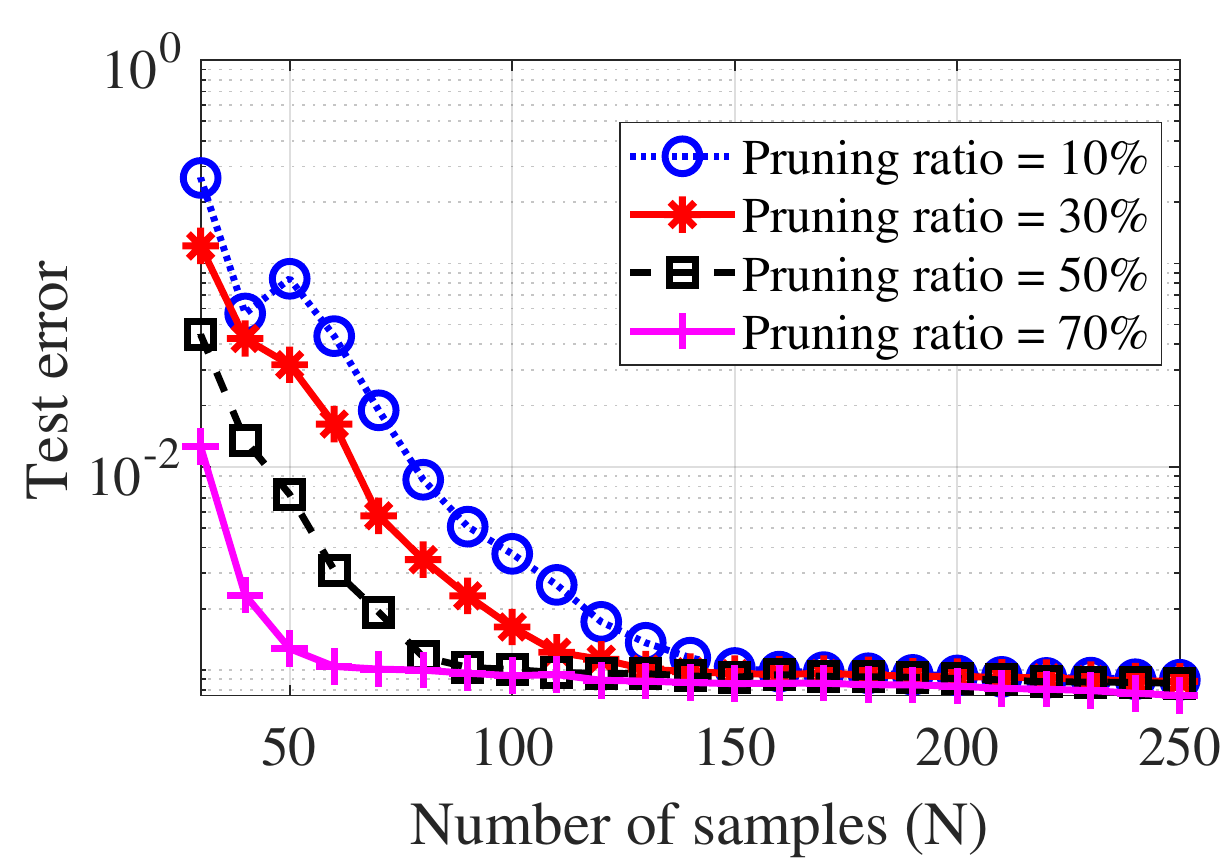}
	\caption{Test error against the number of samples with mask accuracy in $[0.85,0.9]$}
	\label{fig: sample}
\end{figure}




{{\bf Linear convergence.} Figure \ref{fig: ite_rate} shows the convergence rate of Algorithm \ref{Alg5} with different pruning ratios. 
We show the smallest number of iterations required to achieve a certain test error of the learned model, and the results are averaged over the  independent trials with mask accuracy between $0.85$ and $0.90$. Even with inaccurate mask, the test error converges linearly. Moreover,   as the pruning ratio increases, Algorithm \ref{Alg5} converges faster.}

{{\bf Sample complexity with respect to the pruning ratio.} Figure \ref{fig: sample} shows the test error when the number of training samples  $N$ changes. 
All the other parameters except $N$ remain the same. 
The results are averaged over the trials  with mask accuracy between $0.85$ and $0.90$. We can see the test error decreases  when $N$ increases. More importantly, as the pruning ratio increases, the required number of samples to achieve the same test error (no less than $10^{-3}$) decreases dramatically.  That means the sample complexity decreases as $\widetilde{r}$ decreases even if the mask is inaccurate.}

\subsection{Performance of IMP on  Synthetic, MNIST and CIFAR-10 Datasets}\label{sec:IMP}

We implement the {IMP} algorithm    to obtain pruned networks on  synthetic, MNIST and CIFAR-10 datasets.
Figure \ref{fig: test} shows the test performance of a pruned network   on synthetic data with different sample sizes. Here in the {oracle} network model, $K=5, d=100$, and $r_j^*=20$ for all $j\in[K]$. 
The noise level $\sigma/E_y = 10^{-3}$. One observation is that for a fixed sample size $N$ greater than 100, the test error decreases as  the pruning ratio increases. This verifies that the {IMP} algorithm indeed prunes the network properly. It also shows that the learned model improves as   the pruning progresses, verifying our  theoretical result in Theorem \ref{Thm: major_thm} that the difference of the learned model from the {oracle} model decreases  as  $r_j$ decreases. The second observation is that the test error decreases as $N$ increases for any fixed pruning ratio. This verifies our result in Theorem \ref{Thm: major_thm} that the difference of the learned model from the {oracle} model decreases as the number of training samples increases. When the pruning ratio is too large (greater than 80\%), 
the pruned network cannot explain the data properly, and thus the test error is large for all $N$. When the number of samples is too small, like  $N=100$, the test error is always large, because it does not meet the sample complexity requirement for estimating the {oracle} model even though the network is properly pruned. 

{Figures \ref{fig: test_minst} and \ref{fig: test_real}   show the test performance of learned models by implementing    the IMP algorithm on MNIST  and CIFAR-10 using Lenet-5 \cite{lenet5} and Resnet-50 \cite{HZRS16} architecture, respectively.
The experiments follow  the standard setup in \cite{CFCLZCW20} except for the size of the training sets. 
{To demonstrate the effect of sample complexity,}
we randomly selected $N$ samples from the original training set without replacement.
{As we can see, a properly pruned network (i.e., winning ticket) helps reduce the  sample complexity required to reach the test accuracy of the original dense model. For example, 
 training on a pruned network    returns a model (e.g., $P_1$ and $P_3$ in Figures \ref{fig: test_minst} and \ref{fig: test_real}) that has better testing performance than a dense model (e.g., $P_2$ and $P_4$ in Figures \ref{fig: test_minst} and \ref{fig: test_real}) trained on a larger data set. 
 Given the number of samples, we consistently find the characteristic behavior of winning tickets:}
 That is,  the test accuracy could increase  when the pruning ratio increases, indicating the effectiveness of   pruning. The test accuracy then  drops when the network is overly pruned.}
{The results show that our theoretical characterization of sample complexity is well aligned with the empirical performance of pruned neural networks and explains the improved generalization observed in LTH.}
\begin{figure}[h]
	\centering
	\includegraphics[width=0.6\linewidth]{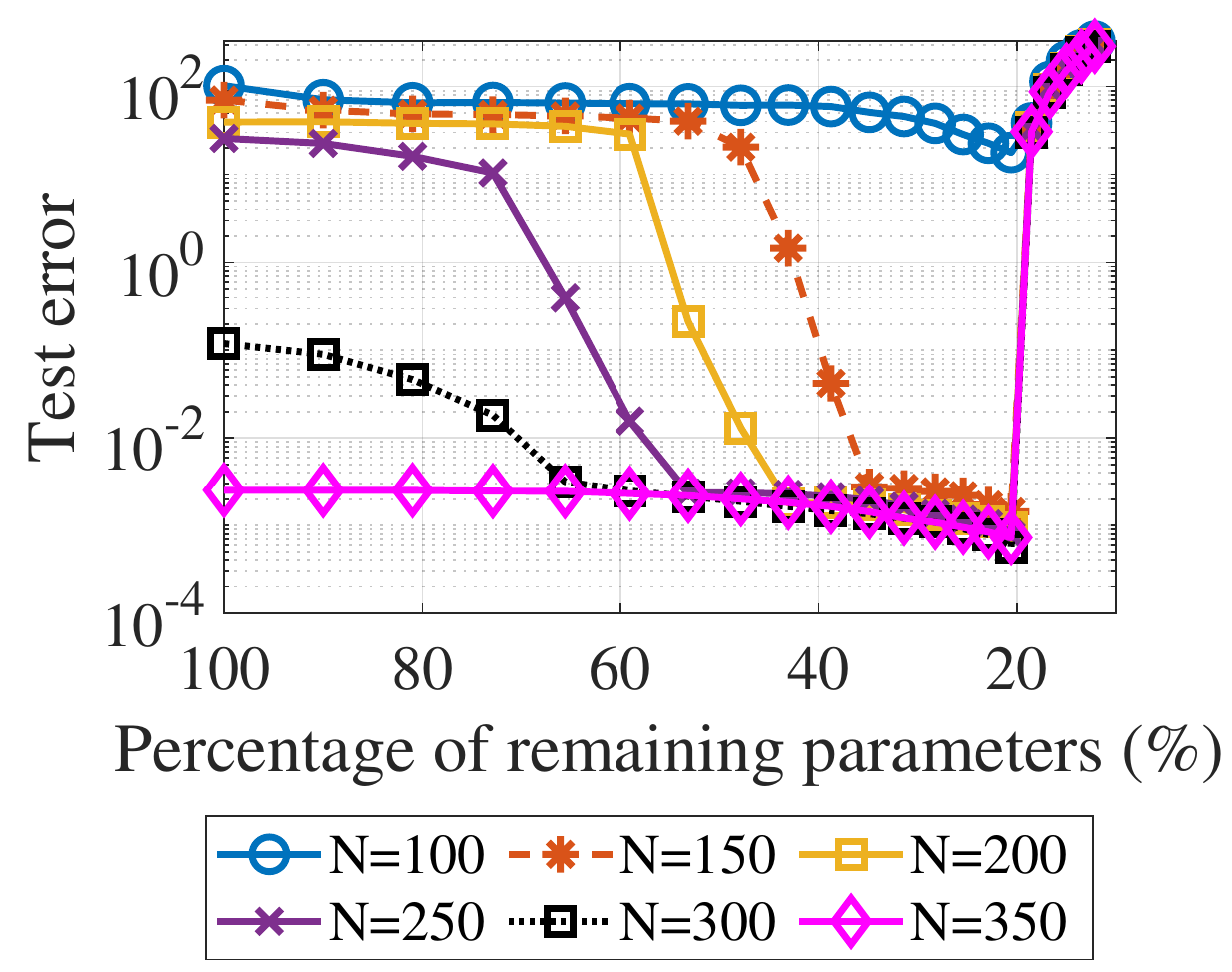}
	\caption{Test error of pruned models on the synthetic dataset}
	\label{fig: test}
\end{figure}

\begin{figure}[h]
	\centering
	\includegraphics[width=0.6\linewidth]{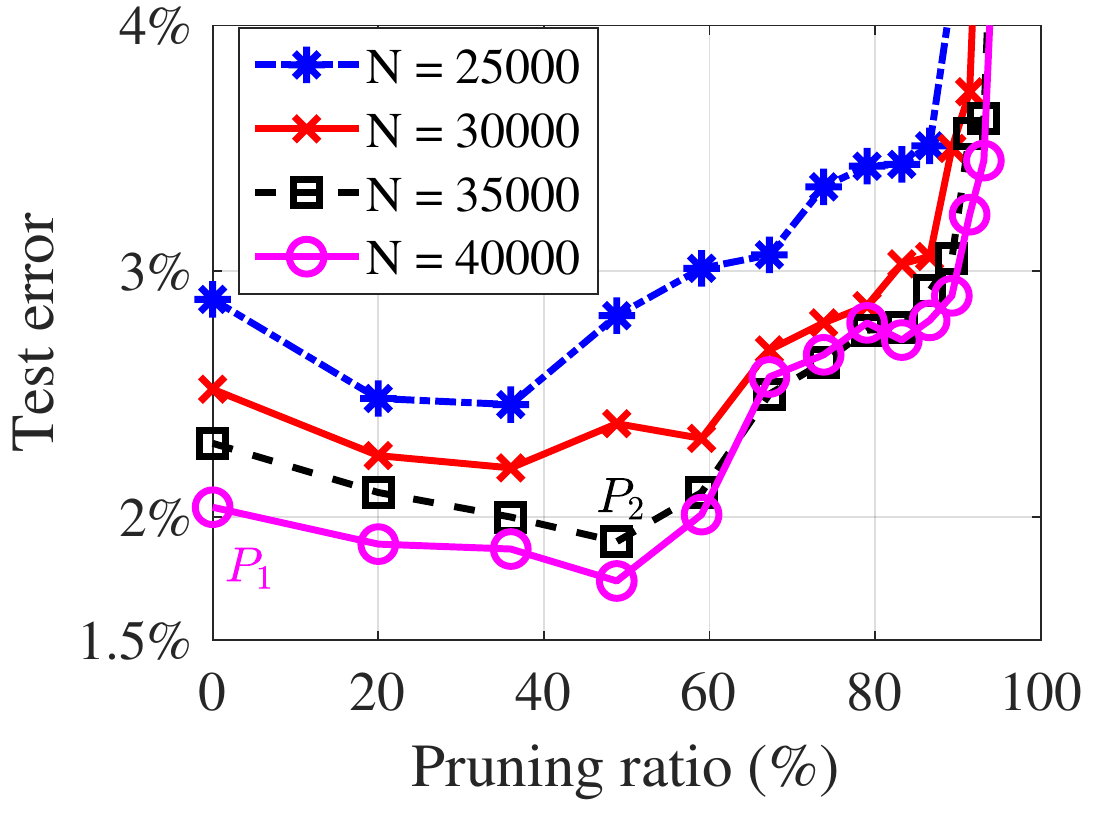}
 \caption{Test accuracy of pruned LeNet-5 on Mnist}
	\label{fig: test_minst}
\end{figure} 
\begin{figure}[h]
	\centering
	\includegraphics[width=0.6\linewidth]{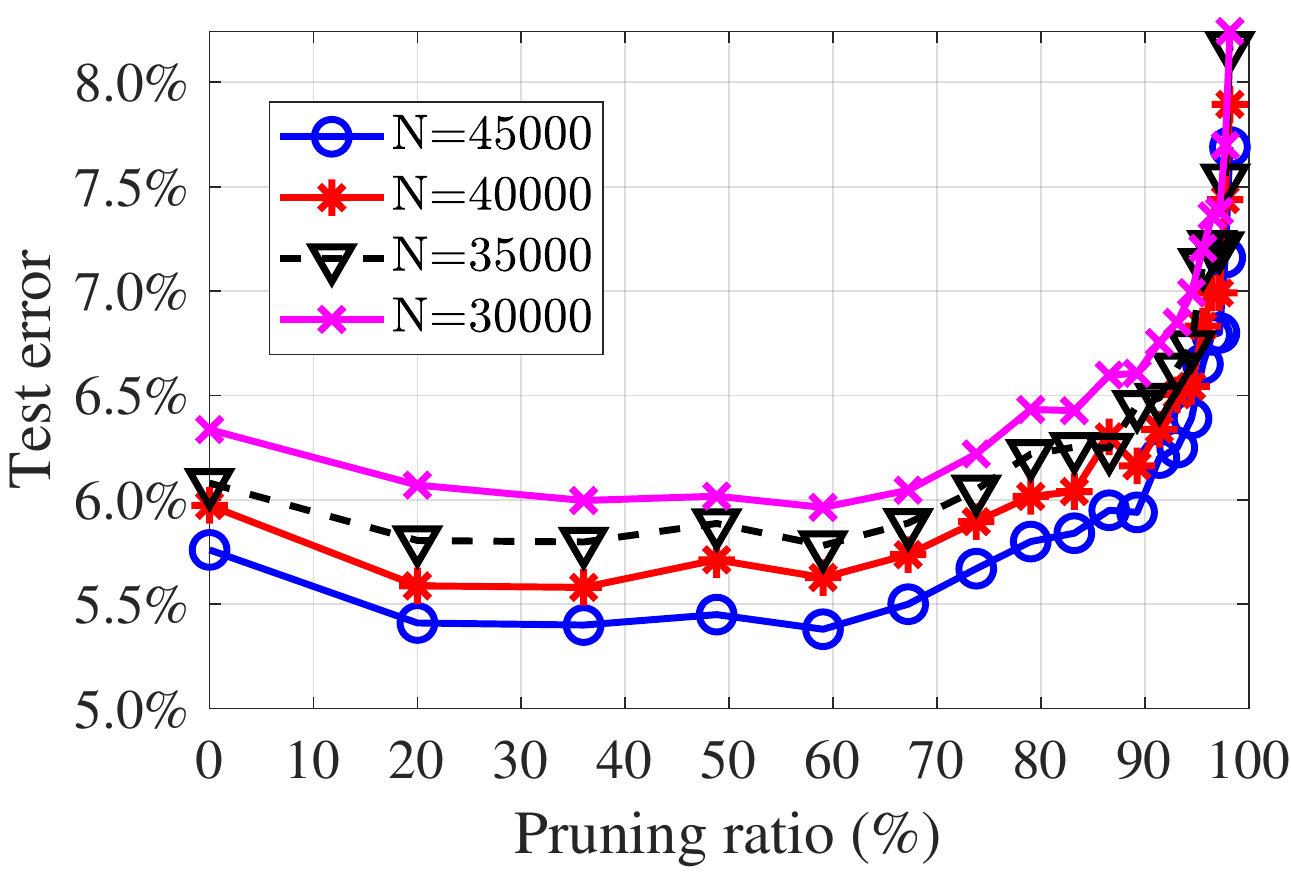}
 \caption{Test accuracy of pruned Resnet-50 on Cifar-10}
	\label{fig: test_real}
\end{figure}

%% file: chapter_6/rpichap6.tex
\chapter{\uppercase{Generalization analysis of Self-training with Unlabeled data: A one-hidden-layer Case}}\label{chapter: 6}
\blfootnote{Portions  of this chapter  have been submitted to: S.~Zhang, M.~Wang, S.~Liu, P.-Y. Chen, and J.~Xiong, ``How unlabeled data
  improve generalization in self-training? A one-hidden-layer theoretical
  analysis,''  submitted for publication.}

\input{./chapter_6/Introduction.tex}
\input{./chapter_6/problem_formulation}

\input{./chapter_6/Theoretical_results.tex}

\input{./chapter_6/simulation.tex}

\section{Summary}

In this chapter, a new theoretical insight is provided  in understanding the influence of unlabeled data in  the iterative self-training algorithm.
The improved generalization error and convergence rate is proved to be a linear function of $1/\sqrt{M}$, where $M$ is the number of unlabeled data. Moreover, compared with  supervised learning, using unlabeled data reduces the required sample complexity of labeled data for achieving zero generalization error. Future directions include generalizing the analysis to multi-layer neural networks and other semi-supervised learning problems such as domain adaptation.

%% file: chapter_6/Introduction.tex

\section{Introduction}

Self-training \cite{Scu65,Yar95,Lee13,HLW19}, one of the most powerful \underline{semi}-\underline{s}upervised \underline{l}earning (SemiSL) algorithms, augments a limited number of labeled data with  unlabeled data so as to achieve
improved   generalization performance   on test data, compared with the model trained by  supervised learning using the labeled data only.  
Self-training has shown empirical success in diversified applications such as few-shot image classification \cite{su2020does,XLHL20,CKSNH20,YJCPM19,ZGLC20}, objective detection \cite{RHS05}, robustness-aware model training against adversarial attacks \cite{CRSD19}, continual lifelong learning \cite{lee2019overcoming} and natural language processing \cite{HGSR19,klh20}.  The terminology ``self-training'' has been used to describe various SemiSL algorithms in the literature, while this chapter is centered on  the commonly-used iterative   self-training method in particular.  In this setup, an  initial teacher model (learned from the labeled data) is applied to the unlabeled data to generate pseudo labels. One then trains a student model   by minimizing the weighted empirical risk of both the labeled and unlabeled data. The student model is then used as the new teacher model to update the pseudo labels of the unlabeled data. This process is repeated multiple times to improve the eventual student model. {We refer readers to Section \ref{sec: problem} for algorithmic details.}

Despite the empirical achievement of self-training  methods {with neural networks}, the theoretical justification of such success is very limited, even in the field of SemiSL. 
The majority of the theoretical results on general SemiSL are limited to linear networks \cite{CWKM20,OADAR11,OG20}, \cite{ RXYDL20}.  
 The authors in \cite{BMB10} 
 show that unlabeled data can improve the generalization bound if the unlabeled data distribution and target model are compatible. For instance, the unlabeled data need to be well-chosen such that the target function for labeled data can separate the unlabeled data clusters, which, however, may not be able to be verified ahead.
\cite{WSCM20}   analyzes SemiSL on nonlinear neural networks  and prove that an infinite number of unlabeled data can improve the generalization compared with training with labeled data only.  
However, \cite{WSCM20} considers  single shot  rather than iterative SemiSL, and the training problem aims to  minimize the consistency regularization rather than the risk function in the conventional self-training method \cite{Lee13}. Moreover, 
\cite{WSCM20}  directly analyzes  the global optimum of the nonconvex training problem without any discussion about how to achieve the global optimum.   To the best of our knowledge, there exists no analytical characterization of how the unlabeled data affect the generalization of the learned model by iterative self-training on nonlinear neural networks.


\begin{figure}[h]
   	\centering
   	\includegraphics[width=0.4\linewidth]{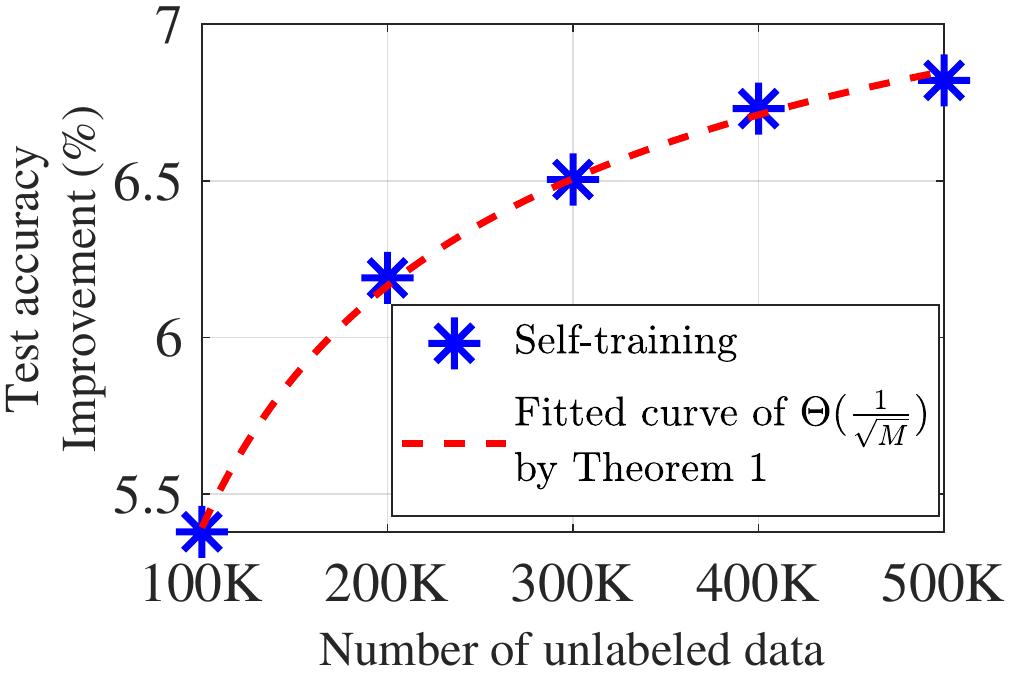}
   	\vspace{-2mm}
     \caption{The trend of test accuracy improvement (\%) on CIFAR-10 by self-training  on CIFAR-10 (labeled) with different amount of unlabeled data from 80 Million Tiny Images matches our theoretical prediction}
     \label{fig: intro}
\end{figure}
\textbf{Contributions.} This chapter provides the first theoretical study of iterative self-training on nonlinear neural networks. Focusing on one-hidden-layer neural networks, this chapter provides a quantitative analysis of the generalization performance of iterative self-training as a function of the number of labeled and unlabeled samples. 
Specifically, our contributions include

\textbf{1. Quantitative justification of generalization improvement by unlabeled data.} Assuming the existence of a ground-truth model with weights $\bfW^*$ that maps the features to the corresponding labels, we prove that the learned model {via iterative self-training} moves closer to $\bfW^*$ as the number $M$ of unlabeled data increases, indicating a better testing performance. Specifically, we prove that the Frobenius distance to $\bfW^*$, which is approximately linear in the generalization error, decreases in the order of $1/\sqrt{M}$. 
{As an example, Figure\,\ref{fig: intro} shows that the   proposed theoretical bound matches the empirical self-training performance versus the number of unlabeled data for image classification; see details in Section \ref{sec: result_practice}.}



\textbf{2. Analytical justification of iterative self-training over single shot alternative.}  We prove that the student models returned by the iterative self-training method \textit{converges linearly} to a model close to $\bfW^*$, {with the rate improvement in the order of $1/\sqrt{M}$}. 



\textbf{3. Sample complexity analysis of labeled and unlabeled data for learning a proper model.} 
{We quantify the impact of labeled and unlabeled data on the generalization of the learned model. In particular, we prove that the sample complexity of labeled data   can be reduced compared with supervised learning.}

\subsection{Related Works}\label{sec: related}

\textbf{Semi-supervised learning.}  Besides self-training, many recent SemiSL algorithms exploit either consistency regularization or entropy minimization.  Consistency regularization is  based on the assumption that the learned model will return same or similar output when the input is perturbed \cite{LA16,BAP14,SJT16,TV17,RLASER15}.  \cite{GB05}  claims that the unlabeled data are more informative if the 
pseudo labels of the unlabeled data have lower entropy. Therefore, a line of works \cite{GB05,MMKI18} adds a regularization term that minimizes the entropy of the outputs of the unlabeled data. 
In addition, hybrid algorithms that unify both the above regularizations have been developed like  \cite{BCNKSZR19,BCGOR19,SBCZZRCKL20}. 

\textbf{Domain adaptation.} Domain adaptation exploits abundant data in the source domain to learn a model for the target domain, where only limited training data are available \cite{LS10,VLMPG13,ZSMW13,LCWJ15,THZSD14}. Source and target domain are related but different. Unsupervised domain adaptation \cite{GL15,GUAGLLML16,GGS13,BTSKE16}, where   training data in target domain are unlabeled,  is similar to SemiSL, and self-training   methods have been used  for analysis \cite{ZYKW18,TRLK12,FMF18}. 
However,  self-training and unsupervised domain adaptation are fundamentally different. The former   learns a model for the domain where there is limited labeled data, with the help of a large number of \textit{unlabeled} data from a different domain.  The latter learns a model for the domain where the training data are unlabeled, with the help of sufficient \textit{labeled} data from a different domain. 

\textbf{Generalization analysis of supervised learning.}
In theory, the testing error is upper bounded by the training error plus the generalization gap between training and testing. These two quantities are often analyzed separately and cannot be proved to be small simultaneously for deep neural networks. For example, neural tangent kernel (NTK) method \cite{JGH18,DZPS18,LBNSSPS18} shows the training error can be zero, and the Rademacher complexity in \cite{BM02} bounds  the generalization gap \cite{ADHLW19}. For one-hidden-layer neural networks \cite{SS18}, the testing error can be proved to be zero under mild conditions. One common assumption is that the input data belongs to the Gaussian distribution \cite{ZSJB17,GLM17,kkms08,BJW19,ZLJ16,BG17,LY17,SJL18}. Another line of approaches \cite{BGM18,LL18,WGC19} consider  linearly separable data.

%% file: chapter_6/problem_formulation.tex
\section{Formalizing {Self-Training}: Notation, Formulation, and Algorithm}\label{sec: problem}




\textbf{Problem formulation.} 
Given $N$ labeled data sampled from distribution $P_{l}$, denoted by $\mathcal{D} = \{ \bfx_n, y_n\}_{n=1}^N$, and $M$ unlabeled data drawn from distribution $P_{u}$, denoted by $\widetilde{\mathcal{D}} = \{ \widetilde{\bfx}_m \}_{m=1}^M$. The objective function is to find a neural network model $g(\bfW)$, where $\bfW$ denotes the trainable weights, that minimize the testing error on data sampled from $P_{l}$.

\textbf{Iterative self-training.} In each iteration, given the current teacher predictor $g(\W[\ell])$, 
the pseudo-labels for the unlabeled data in $\w[\mathcal{D}]$ are computed as $\tilde{y}_m = g(\W[\ell];\w[\bfx]_m)$. 
The method then minimizes  the weighted empirical risk  $\hat{f}_{\mathcal{D},\w[\mathcal{D}]}(\bfW)$ of both labeled and unlabeled data through stochastic gradient descent, where
\begin{equation}\label{eqn: empirical_risk_function}
\vspace{-1mm}
    \hat{f}_{\mathcal{D},\w[\mathcal{D}]}(\bfW) = \frac{\lambda}{2N}\sum_{n=1}^N\big(y_n - g(\bfW;\bfx_n) \big)^2 + \frac{\w[\lambda]}{2M}\sum_{m=1}^M \big(\w[y]_m - g(\bfW;\w[\bfx]_m) \big)^2, 
\end{equation}
and $\lambda+\w[\lambda]= 1$. The learned student model $g(\W[\ell+1])$   is used as the teacher model in the next iteration. The initial model $g(\W[0])$ is learned from labeled data. 
The formal algorithm is summarized as followed. 

    
    
    

\textbf{Model and assumptions.} In this chapter, 
$g$ is a one-hidden-layer fully connected neural network equipped with $K$ neurons. Namely, given the input $\bfx\in\mathbb{R}^d$ and weights $\bfW = [\bfw_1, \bfw_2, \cdots, \bfw_K]\in \mathbb{R}^{d\times K}$, we have
\begin{equation}\label{eqn: teacher_model}
\vspace{-1mm}
    g(\bfW;\bfx) : = \frac{1}{K}\sum_{j=1}^K \phi(\bfw_j^{T}\bfx),
\end{equation}
where $\phi$ is the ReLU activation function,  and $\phi(z) = \max\{ z, 0 \}$ for any 
input $z\in \R$.

Moreover, we assume an unknown ground-truth model with weights $\bfW^*$ that maps all the features to the corresponding labels drawn from $P_l$, i.e., $y = g(\bfW^*;\bfx)$, where $(\bfx, y)\sim P_l$. 
The generalization function (GF) with respect to $g(\bfW)$  is defined as 
\begin{equation}\label{eqn: GF}
    I\big(g(\bfW)\big) = \mathbb{E}_{(\bfx, y)\sim P_l }\big(y - g(\bfW;\bfx)\big)^2 = \mathbb{E}_{(\bfx, y)\sim P_l }\big( g(\bfW^*;\bfx) - g(\bfW;\bfx)\big)^2.
\end{equation}
By definition $I\big(g(\bfW^*)\big)$ is zero. 
Clearly, $\bfW^*$ is not unique because any column permutation of $\bfW^*$, which corresponds to permuting neurons, represents the same function as $\bfW^*$ and minimizes GF
in \eqref{eqn: GF} too. To simplify the representation, we follow the convention and abuse the notation that 
the distance from $\bfW$ to $\bfW^*$, denoted by  $\|\bfW-\bfW^*\|_F$, means the smallest distance from $\bfW$ to any permutation of $\bfW^*$.  Additionally, some important notations are summarized in Table \ref{table:problem_formulation}.

We assume the inputs of both the labeled and unlabeled data belong to the zero mean Gaussian distribution,  
i.e., $\bfx \sim\mathcal{N}(0, \delta^2\bfI_{d})$, and $\w[\bfx]\sim \mathcal{N}(0, \de^2\bfI_d)$.  
The Gaussian assumption is motivated by the data whitening \cite{LBOM12} and batch normalization techniques \cite{IS15} that are commonly used in practice to improve learning performance. 
 
\textbf{The focus of this chapter.} 
This chapter will analyze two aspects about self-training:  
(1) the generalization performance of $\W[L]$, the returned model by self-training after $L$ iterations, measured by $\|\W[L]-\bfW^*\|_F$
\footnote{We use this metric because  $I\big(g(\bfW)\big)$ is shown to be linear in $\|\W[L]-\bfW^*\|_F$ numerically when $\W[L]$ is close to  $\bfW^*$, see Figure \ref{fig:GF_1M}.}; 
and (2) the impact of unlabeled data on the training and generalization performance.  

\begin{table}[h]
    \centering
        \caption{Some important notations for iterative self-training}
    \begin{tabular}{c|p{11cm}}
    \hline
    \hline
    $\mathcal{D} =\{ \bfx_n, \bfy_n \}_{n=1}^N$ & Labeled dataset with $N$  number of samples;\\
    \hline
    $\widetilde{\mathcal{D}} =\{ \widetilde{\bfx}_m \}_{m=1}^M$ & Unlabeled dataset with $M$  number of samples;\\
    \hline
    $d$ & Dimension of the input $\bfx$ or $\widetilde{\bfx};$
    \\
    \hline
    $K$ & Number of neurons in the hidden layer;
    \\
    \hline
    $\W[\ell]$ & Model returned by self-training after $\ell$ iterations; $\W[0]$ is the initial model;\\
    \hline
    $\bfW^*$ & Weights of the ground truth model;\\
    \hline
    $\WL$ & $\WL = \hat{\lambda} \bfW^* + (1-\hat{\lambda})\W[0]$;\\
    \hline
    \hline
    \end{tabular}
    \label{table:problem_formulation}
\end{table}

%% file: chapter_6/Theoretical_results.tex
\section{Theoretical Results}\label{sec: theoretical}


\paragraph{{Beyond supervised learning: Challenge of self-training.}}
The existing theoretical works such as  \cite{ZSJB17} verify that for one-hidden-layer neural networks, if only labeled data are available, and $\bfx$ are drawn from the standard Gaussian distribution, then supervised learning by minimizing \eqref{eqn: empirical_risk_function} with $\lambda=1$ 
 can return a model with ground-truth weights $\bfW^*$ (up to column permutation), as long as  the number of labeled data $N$ is at least  {$N^*$, which  depends on $\kappa, K$ and $d$}. In contrast, this chapter focuses 
on the \textbf{low labeled-data regime} when $N$ is less than $N^*$. 
Specifically,
\begin{equation}\label{eqn: definition_N*}
   \max\{N^*/K^2, N^* /4\} < N \leq N^*. 
\end{equation}%

Intuitively, if $N<N^*$, the landscape of the empirical risk of the labeled data becomes highly nonconvex, even in a neighborhood   of $\bfW^*$, thus, the existing analyses for supervised learning do not hold in this region.  With additional unlabeled data, the landscape of the weighted empirical risk becomes smoother near  $\bfW^*$. Moreover, as $M$ increases, and starting from a nearby initialization, the returned model $\W[L]$ by iterative self-training  can converge to a local minimum that is closer to $\bfW^*$ (see  illustration in Figure \ref{fig: landscape_1}). 
\begin{figure}
    \centering
    \includegraphics[width=0.60
    \linewidth]{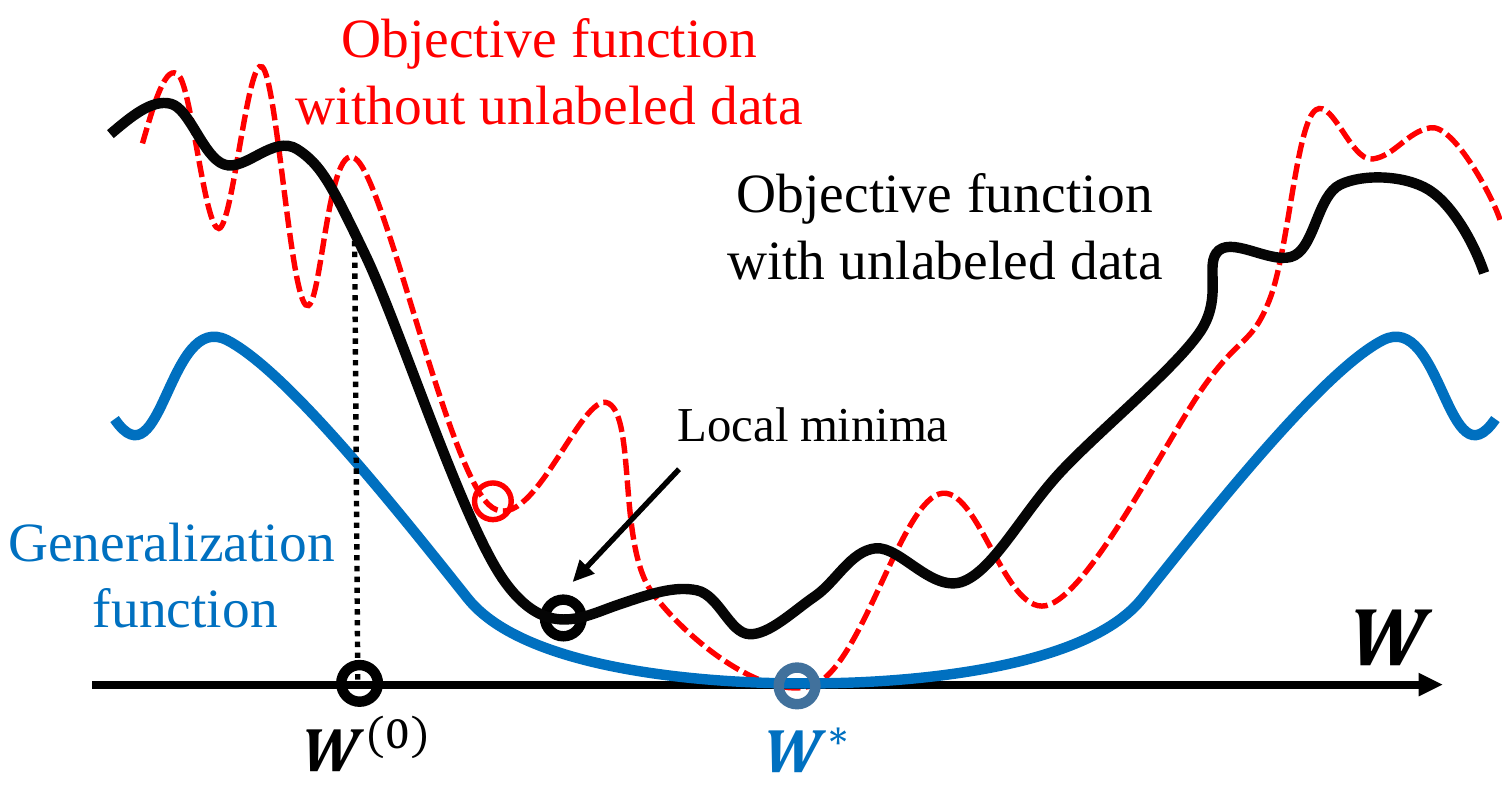}
     \caption{Adding unlabeled data in the empirical risk function drives its local minimum closer to $\bfW^*$, which minimizes the generalization function}
    \label{fig: landscape_1}
\end{figure} 

Compared with supervised learning, the formal analyses of self-training need to handle new technical challenges from two aspects. First, the existing analyses of supervised learning exploit the fact that   the GF and the empirical risk have the same minimizer, i.e., $\bfW^*$. This property does not hold for self-training as $\bfW^*$ no longer minimizes the weighted empirical risk in (\ref{eqn: empirical_risk_function}). Second, the iterative manner of self-training complicates the analyses. Specifically, the empirical risk   in each iteration is different and depends on the model trained in the previous iteration through the pseudo labels.  

{In what follows, we   provide theoretical insights and the formal theorems. 
Some important quantities $\hat{\lambda}$ and $\mu$  are defined below 
\vspace{-1mm}
\begin{equation}\label{eqn: definition_lambda}
    \hat{\lambda} : =\textstyle\frac{\lambda\delta^2}{\lambda\delta^2+\w[\lambda]\de^2},
    \text{\quad and \quad}
    \mu
    = \mu( \delta, \de) 
    := \sqrt{\frac{\lambda\delta^2+\w[\lambda]\de^2}{\lambda\rho(\delta)+\w[\lambda]\rho(\de)}},
    \vspace{-1mm}
\end{equation}
 where $\rho$ is a positive function defined in \eqref{eqn: rho}.
$\hat{\lambda}$ is an increasing function of $\lambda$. Also,
 from  Lemma \ref{lemma: rho_order} (in appendix),   $\rho(\delta)$ is in the order of $\delta^2$ when $\delta\leq 1$   for ReLU activation functions. Thus,   $\mu$ is a fixed constant, denoted by $\mu^*$, for all $\delta, \de \leq 1$. 
 When $\delta$ and $\de$ are large,   $\mu$ increases as they increase.} {The formal definition of $N^*$ in \eqref{eqn: definition_N*} is $c(\kappa)\mu^{*2}K^3d\log q$, where $c(\kappa)$ is some polynomial function of $\kappa$ and can be viewed as constant.}


\subsection{Informal Key Theoretical Findings}

 \begin{figure}
    \centering
    \includegraphics[width=0.70\linewidth]{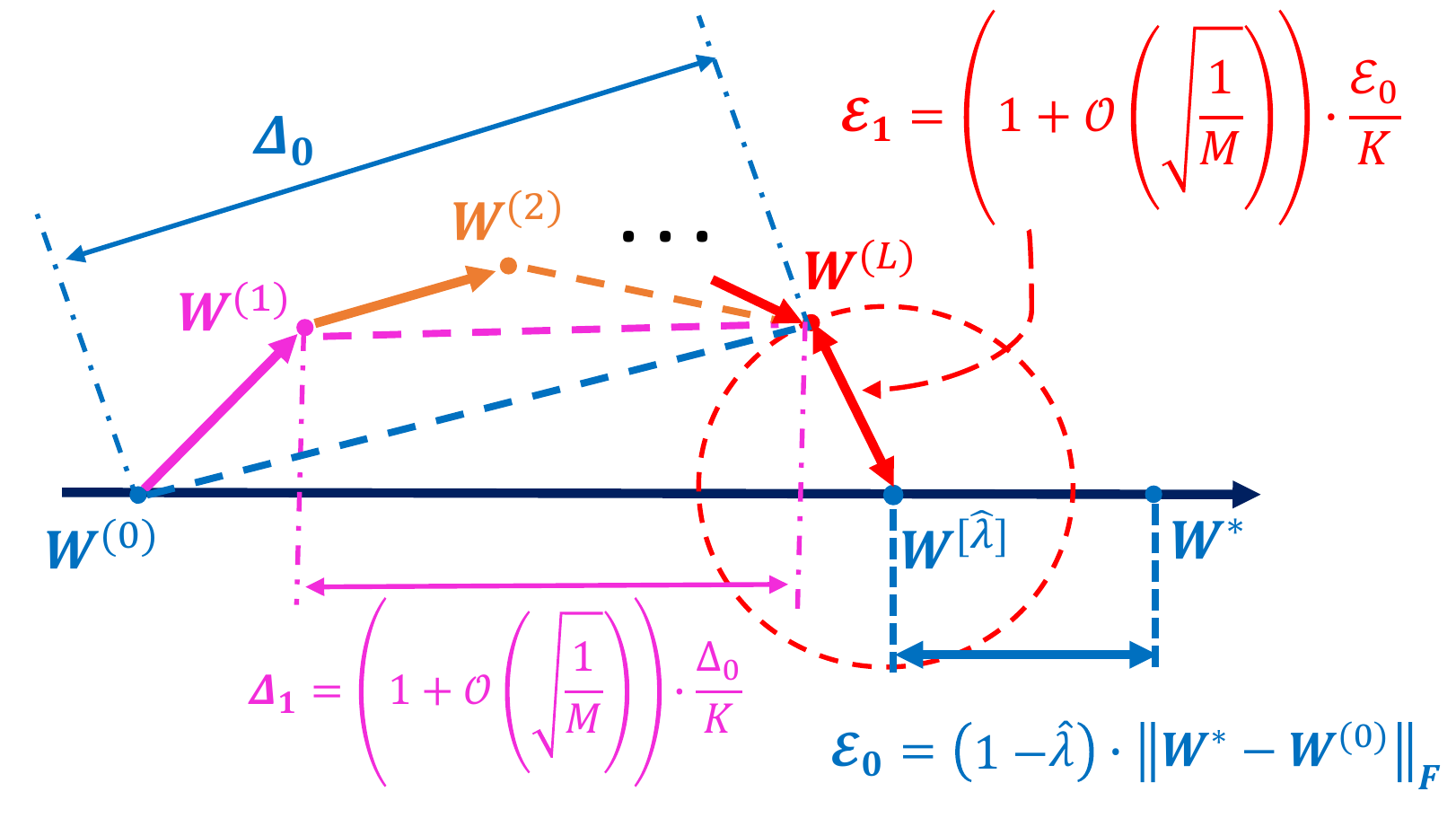}
    \caption{Illustration of the (1) ground truth $\bfW^*$, (2) iterations $\{ \W[\ell] \}_{\ell=0}^L$, (3) convergent point $\W[L]$ and (4) $\WL = \hat{\lambda}\bfW^*+(1-\hat{\lambda})\W[0]$}
    \label{fig: Thm2_region}
\end{figure}
To the best of our knowledge,   Theorems \ref{Thm: p_convergence} and \ref{Thm: sufficient_number} provide the first  theoretical characterization of iterative self-training on nonlinear neural networks. Before formally presenting them, we summarize the highlights as follows. 

\textbf{1. Linear convergence of the learned models.} 
The learned models 
converge linearly to a model close to $\bfW^*$. 
Thus, the iterative approach returns a model with better generalization than that by the single-shot method. Moreover,    the convergence rate is a constant term plus a term in the order of $1/\sqrt{M}$ (see $\Delta_1$ in Figure \ref{fig: Thm2_region}), indicating a faster convergence with more unlabeled data. 

\textbf{2. Returning a model with guaranteed generalization   in the low labeled-data regime.}  Even when the number of labeled data is much less than the required sample complexity to obtain $\bfW^*$ in supervised learning, we prove that with the help of unlabeled data, 
 the iterative self-training can return a model in 
 the neighborhood 
 of $\WL$, where $\WL$ is in the line segment of  $\W[0]$ ($\hat{\lambda}=0$) and ground truth $\bfW^*$ ($\hat{\lambda}=1$). 
Moreover, 
$\hat{\lambda}$ is upper bounded by $\sqrt{N/N^*}$. Thus $\W[L]$ moves closer to $\bfW^*$ as $N$ increases ($\mathcal{E}_0$ in Figure \ref{fig: Thm2_region}), indicating a better generalization performance with more labeled data. 



\textbf{3. Guaranteed generalization improvement by unlabeled data.} 
The distance between $\W[L]$ and $\WL$ ($\mathcal{E}_1$ in Figure \ref{fig: Thm2_region})
scales in the order of $1/\sqrt{M}$. With a larger number of unlabeled data $M$, $\W[L]$ moves closer to $\WL$ and thus $\bfW^*$,  indicating an improved generalization performance 
(Theorem \ref{Thm: p_convergence}).
When $N$ is close to $N^*$ but still smaller as defined in (\ref{eqn: main_sample}), 
 both $\W[L]$ and $\WL$ converge to $\bfW^*$, and thus the learned model achieves zero generalization error (Theorem \ref{Thm: sufficient_number}).  

\subsection{Formal Theory in Low Labeled-data Regime}


{\textit{Takeaways of Theorem\,1:}
Theorem 1 characterizes the convergence rate of the proposed algorithm and the accuracy of the learned model $\W[L]$ 
in a low labeled-data regime. Specifically, the iterates converge linearly, and the learned model is close to  $\WL$ and guaranteed to outperform the initial model $\bfW^{(0)}$.} 


\begin{theorem}\label{Thm: p_convergence}
{Suppose the initialization  $\bfW^{(0)}$ and the number of labeled data  satisfy 
    \begin{equation}\label{eqn: definition_p}
        \| \bfW^{(0)} -\bfW^* \|_F \le  p^{-1}\cdot \frac{\|\bfW^*\|_F}{c(\kappa) \mu^2 K^{3/2}} \quad \text{with} \quad p\in \big(\frac{1}{2}, 1\big],
    \end{equation}
    \begin{equation} \label{eqn:Nbound}
\hspace{-1in} \textrm{ and }\quad         \max\Big\{ \frac{1}{K}, p - \frac{2p-1}{\mu\sqrt{K}}\Big\}^2 \cdot N^* \le N \le N^*. 
    \end{equation}
If the value of $\hat{\lambda}$ in \eqref{eqn: definition_lambda} and unlabeled data amount $M$ satisfy
    \begin{equation}\label{eqn: thm1_condition1}
        \max\Big\{ \frac{1}{K}, p - \frac{2p-1}{\mu\sqrt{K}}\Big\}\le\hat{\lambda} \le \min \Big\{\sqrt{\frac{N}{N^*}}, p+ \frac{2p-1}{\mu\sqrt{K}}\Big\},
    \end{equation}
    \begin{equation}\label{eqn: thm1_condition2}
   \hspace{-1in}     \text{ and } \quad  M\ge (2p-1)^{-2} c(\kappa) \mu^2 \big( 1-  \hat{\lambda}\big)^2 K^3 d\log q.
    \end{equation}
    Then, when the number $T$ of SGD iterations  is large enough in each loop $\ell$, with probability at least $1-q^{-d}$,
     the iterates $\{ \bfW^{(\ell)} \}_{\ell=0}^{L}$ converge to $\WL$ as 
    \begin{equation}\label{eqn: p_convergence}
    \begin{split}
    \| \W[L]-\WL \|_F 
            \le& \bigg(\Big(1+\Theta\big(\frac{\mu(1-\hat{\lambda})}{\sqrt{M}}\big)\Big)\cdot \hat{\lambda} \bigg)^L \cdot \| \bfW^{(0)}-\WL \|_2\\  &+\Big(1+\Theta\big(\frac{\mu(1-\hat{\lambda})}{\sqrt{M}}\big)\Big) \cdot \| \bfW^* - \WL \|_F,
    \end{split}
    \end{equation}
    where $\WL = \hat{\lambda}\bfW^* + (1-\hat{\lambda})\W[0]$. Typically, when the iteration number $L$ is sufficient large, we have
    \begin{equation}\label{eqn: p_convergence_*}
         \begin{split}
        \vspace{-1mm}
             \| \W[L]-\bfW^* \|_F 
            \le \Big(1+\Theta\big(\frac{\mu(1-\hat{\lambda})}{\sqrt{M}}\big)\Big) \cdot 2(1-\hat{\lambda})\cdot \| \bfW^* - \bfW^{(0)} \|_F.
         \end{split}
    \end{equation}}
    \vspace{-1mm}
\end{theorem}

The accuracy of the learned model $\W[L]$ with respect to $\bfW^*$ is characterized as
\eqref{eqn: p_convergence}, and the learning model is better than initial model as in
\eqref{eqn: p_convergence_*} if the following conditions hold.
First, the weights $\lambda$ in \eqref{eqn: empirical_risk_function} are properly chosen as in \eqref{eqn: thm1_condition1}. Second, the number of unlabeled data is sufficiently large as in \eqref{eqn: thm1_condition2}. 

{
 \textbf{Selection of $\lambda$ in self-training algorithms}.   When $\hat{\lambda}$ increases, the required number of unlabeled data  is reduced from  \eqref{eqn: thm1_condition2}, and  the convergence point $\bfW^{(L)}$ becomes closer to $\bfW^*$ from \eqref{eqn: p_convergence_*}, which indicates a smaller generalization error. Thus, a large  $\hat{\lambda}$ within its feasible range \eqref{eqn: thm1_condition1} is desirable.  When the initial model $\bfW^{(0)}$  is closer to  $\bfW^*$ (corresponding to a larger $p$), and the number of labeled data $N$ increases, the upper bound in \eqref{eqn: thm1_condition1} increases, and thus, one can select a larger $\hat{\lambda}$.}

{
\textbf{The initial model $\bfW^{(0)}$}.  The 
tensor initialization from \cite{ZSJB17} can return  a  $\bfW^{(0)}$ that satisfies \eqref{eqn: definition_p} when the number of labeled data is $N = p^2N^*$ (see Lemma \ref{Thm: p_initialization} in Appendix). Combining with the requirement in \eqref{eqn:Nbound}, Theorem 1 applies to the case that $N$ is at least $N^*/4$.}

\subsection{Formal Theory of Achieving Zero Generalization Error}
\textit{Takeaways of Theorem\,2:} 
{Theorem \ref{Thm: sufficient_number} indicates the model returned by the proposed algorithm converges linearly to the ground truth $\bfW^*$. Thus the distance between the learned model and the ground truth can be arbitrarily small with the ability to achieve zero generalization error. The required sample complexity is reduced by a constant factor compared with supervised learning.}

\begin{theorem}\label{Thm: sufficient_number}
Consider the number of unlabeled data satisfies
   \begin{equation}\label{eqn: main_sample}
    \begin{split}
    \big(1 - {1}/{(\mu\sqrt{K})} \big)^2\cdot N^* \leq    N \leq    N^*,
    \end{split}
    \end{equation}
    we choose $\hat{\lambda}$ such that
    \begin{equation}\label{eqn: lambda_hat_s}
   1 - {1}/{(\mu\sqrt{K})} \leq     \hat{\lambda} \leq \sqrt{{N}/{N^*}}. 
    \end{equation}
Suppose the initial model $\bfW^{(0)}$  and the number of unlabeled data $M$ satisfy
\begin{equation}\label{eqn: Thm2_condition}
    \| \bfW^{(0)} -\bfW^* \|_F  \le \frac{\|\bfW^*\|_F}{c(\kappa)\mu^{2}K^{3/2}} \quad \text{and}\quad  M \geq c(\kappa)\mu^{2}(1-\hat{\lambda})^2 K^3 d\log q,
\end{equation}
    the iterates $\{ \bfW^{(\ell)} \}_{\ell=0}^{L}$ converge to the ground truth $\bfW^*$  as  follows,
    \begin{equation}\label{eqn: convergence}
    \vspace{-1mm}
         \| \W[L] -\bfW^* \|_F \le \Big[ \big(1+\frac{c(\kappa)\hat{\lambda}}{\sqrt{N}}+\frac{c(\kappa)(1-\hat{\lambda})}{\sqrt{M}}\big)\cdot\mu\sqrt{K}(1-\hat{\lambda}) \Big]^L 
        \cdot \| \W[0]-\bfW^* \|_F.
    \vspace{-1mm}
    \end{equation}
\end{theorem}

The models $\W[\ell]$'s  converge linearly  to the ground truth $\bfW^*$ as \eqref{eqn: convergence} when the number of labeled data 
satisfies \eqref{eqn: main_sample}. In contrast, supervised learning requires at least $N^*$ labeled samples to estimate $\bfW^*$ accurately without unlabeled data, which suggests self-training at least saves a constant fraction of labeled data.

\subsection{The Main Proof Idea}
Our proof  builds upon and extends one recent line of works on supervised learning such as \cite{ZSJB17,ZWLC20_2,ZWLC20_3}. The standard framework of these works is first to show that the generalization function $I(g(\bfW))$ in \eqref{eqn: GF} is locally convex near $\bfW^*$, which is its global minimizer. Then, when $M=0$ and $N$ is sufficiently large, the empirical risk function using labeled data only can approximate  $I(g(\bfW))$ well  in the neighborhood of $\bfW^*$. Thus, if initialized in this local convex region, iterations returned by applying gradient descent on the empirical risk converges to $\bfW^*$.

The technical challenge here is that in self-training, when unlabeled data are paired with pseudo labels, $\bfW^*$ is no longer a global minimizer of the empirical risk $\hat{f}_{\mathcal{D},\widetilde{\mathcal{D}}}$ in \eqref{eqn: empirical_risk_function}, and $\hat{f}_{\mathcal{D},\widetilde{\mathcal{D}}}$ does not approach $I(g(\bfW))$ even when $M$ and $N$ increase to infinity. Our new idea is to design a population risk function $f(\bfW; \hat{\lambda})$ in \eqref{eqn: p_population_risk_function} (see appendix), which is a lower bound of $\hat{f}_{\mathcal{D},\widetilde{\mathcal{D}}}$ when $M$ and $N$ are infinity. $f(\bfW; \hat{\lambda})$ is locally convex around its minimizer $\WL$, and $\WL$ approaches $\bfW^*$ as $\hat{\lambda}$ increases. Then we show the iterates generated by $\hat{f}_{\mathcal{D},\widetilde{\mathcal{D}}}$ stay close to $f(\bfW; \hat{\lambda})$, and the returned model $\W[L]$ is close to $\WL$. New technical tools are developed to bound the distance between the functions $\hat{f}_{\mathcal{D},\widetilde{\mathcal{D}}}$ and $f(\bfW; \hat{\lambda})$.

%% file: chapter_6/simulation.tex
\section{Empirical Results}\label{sec: empirical_results}
\subsection{Synthetic Data Experiments} We generate a ground-truth   neural network with the width $K=10$.  Each entry of $\bfW^*$ is uniformly selected from $[-2.5, 2.5]$. 
The input of    labeled  data $\bfx_n$ 
are generated from Gaussian distribution $\mathcal{N}(0, \bfI_{d})$ independently, and the corresponding label $y_n$ is generated through \eqref{eqn: teacher_model} using $\bfW^*$. The unlabeled data 
$\w[\bfx]_m$ 
are generated from $\mathcal{N}(0, \widetilde{\delta}^2 \bfI_{d})$ independently with $\widetilde{\delta}=1$ except in Figure \ref{fig: delta}. $d$ is set as $50$ except in Figure \ref{fig: M_pt}. 
The value of $\lambda$ is selected as $\sqrt{N/(2Kd)}$ except in Figure \ref{fig: lambda}.
We consider one-hidden-layer except in Figure \ref{fig:GF_1M}.
The initial teacher model $\bfW^{(0)}$ in self-training is randomly selected from $\{\bfW|\|\bfW-\bfW^* \|_F/\|\bfW^*\|_F \le 0.5  \}$ to reduce the computation. In each iteration, the  maximum number of SGD steps $T$ is   $10$. Self-training terminates if $\|\bfW^{(\ell+1)}-\bfW^{(\ell)}\|_F/\|\bfW^{(\ell)}\|_F\le 10^{-4}$ or reaching $1000$  iterations. In Figures \ref{fig: M_N} to \ref{fig: lambda}, all the points on the curves are averaged over $1000$ independent trials, and the regions in lower transparency indicate  the corresponding  one-standard-deviation error bars. Our \textbf{empirical observations} are summarized below.

\textbf{(a) GF (testing performance) proportional to $\|\bfW -\bfW^*\|_F$.}
Figure \ref{fig:GF_1M} illustrates the GF in \eqref{eqn: GF} against the distance to the ground truth $\bfW^*$. To visualize results for different networks together, 
GF is normalized  in $[0, 1]$, divided by its   largest value for each network architecture.  
All the results are averaged over $100$ independent choice of $\bfW$. 
One can see that for one-hidden-layer neural networks,  in a 
large region near  $\bfW^*$,   GF is almost linear in $\|\bfW -\bfW^*\|_F$. When the number of hidden layers increases, this region decreases, but the linear dependence still holds locally. This is an empirical justification of using $\|\bfW -\bfW^*\|_F$  to evaluate the GF and, thus, the testing error in Theorems 1 and 2. 
\begin{figure}[h]
    \centering
    \includegraphics[width=0.6\textwidth]{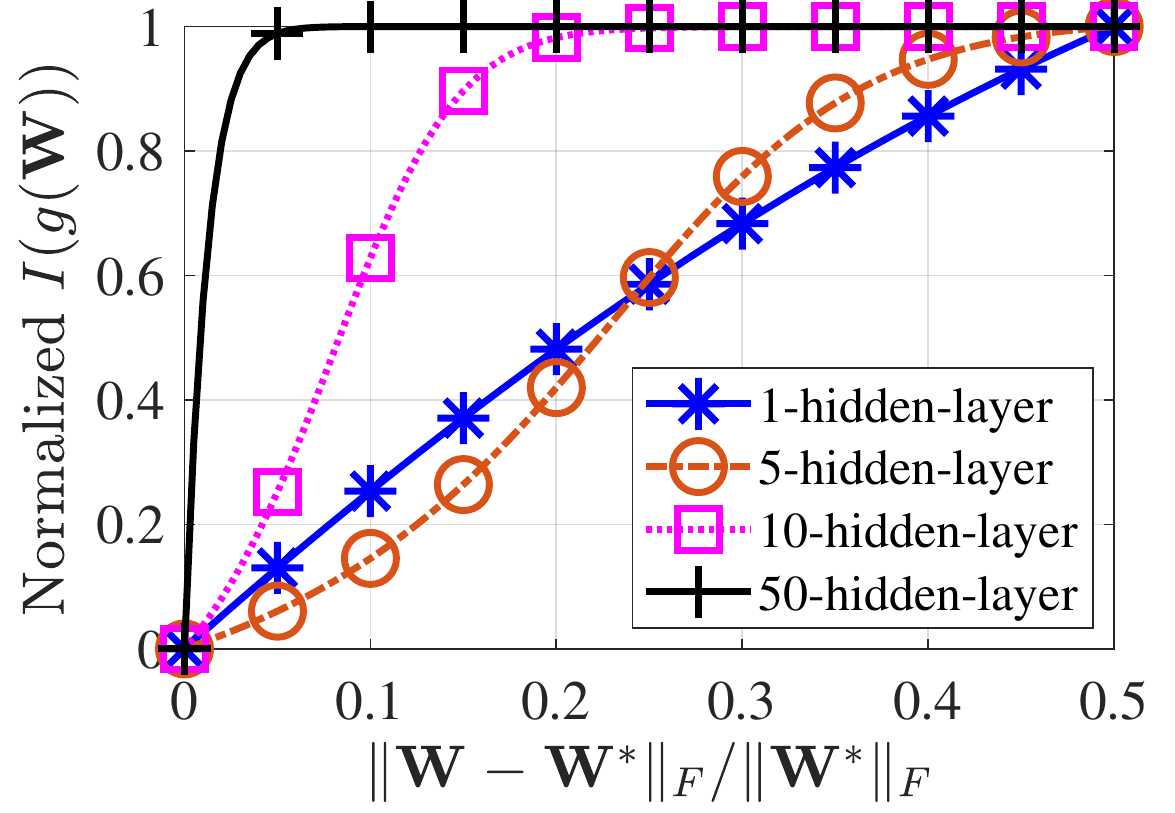}
    \caption{The generalization function against the distance to the ground truth neural network}
    \label{fig:GF_1M}
\end{figure}

\textbf{(b) $\|\W[L]-\bfW^*\|_F$ as a linear function of $1/\sqrt{M}$.}
 Figure \ref{fig: M_N} shows the relative error $\|\W[L]-\bfW^*\|_F/\| \bfW^* \|_F$ when the number of unlabeled data and labeled data changes. 
 One can see that the relative error decreases when either $M$ or $N$ increases. Additionally, the dash-dotted lines represent the best fitting of the linear functions of $1/\sqrt{M}$ using the least square method. Therefore, the relative error is indeed 
 a linear function of $1/\sqrt{M}$, as predicted by our results in \eqref{eqn: p_convergence_*} and \eqref{eqn: convergence}.  
\begin{figure}[h]
  	    \centering
  	    \includegraphics[width=0.6\linewidth]{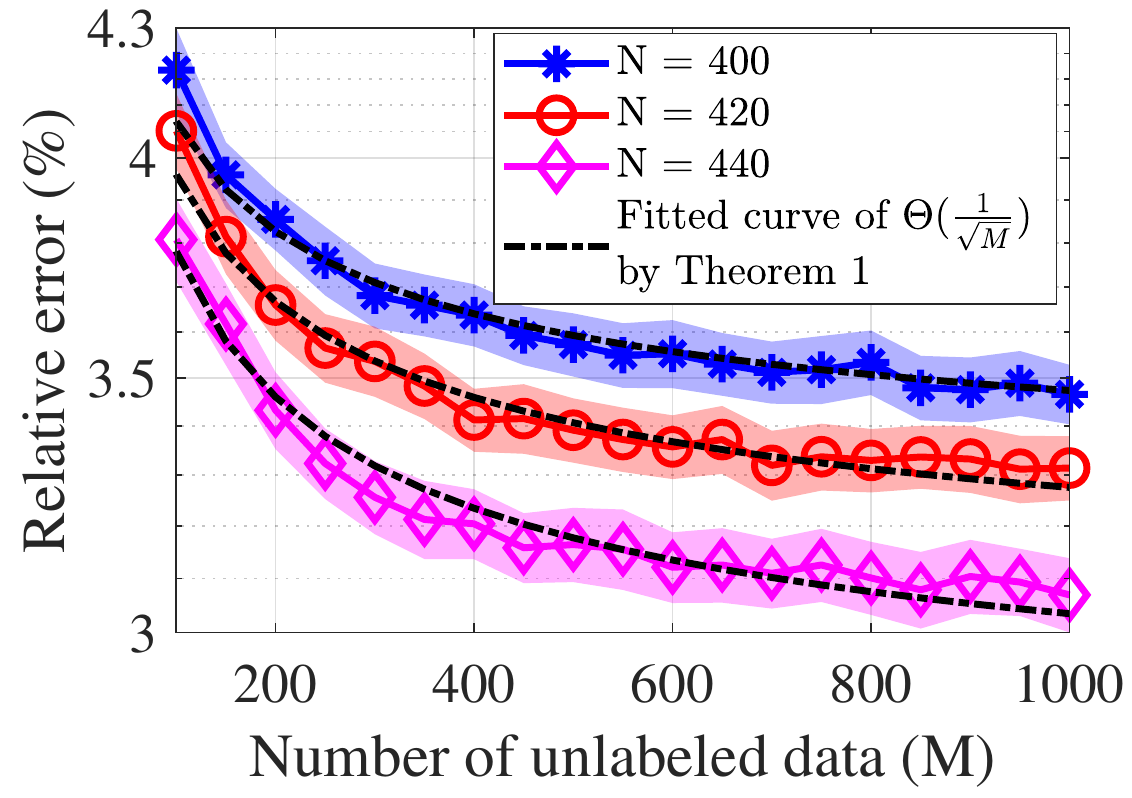}
  	    \caption{The relative error against the number of unlabeled data}
  	    \label{fig: M_N}
\end{figure}
\begin{figure}[h]
        \centering
  	    \includegraphics[width=0.6\linewidth]{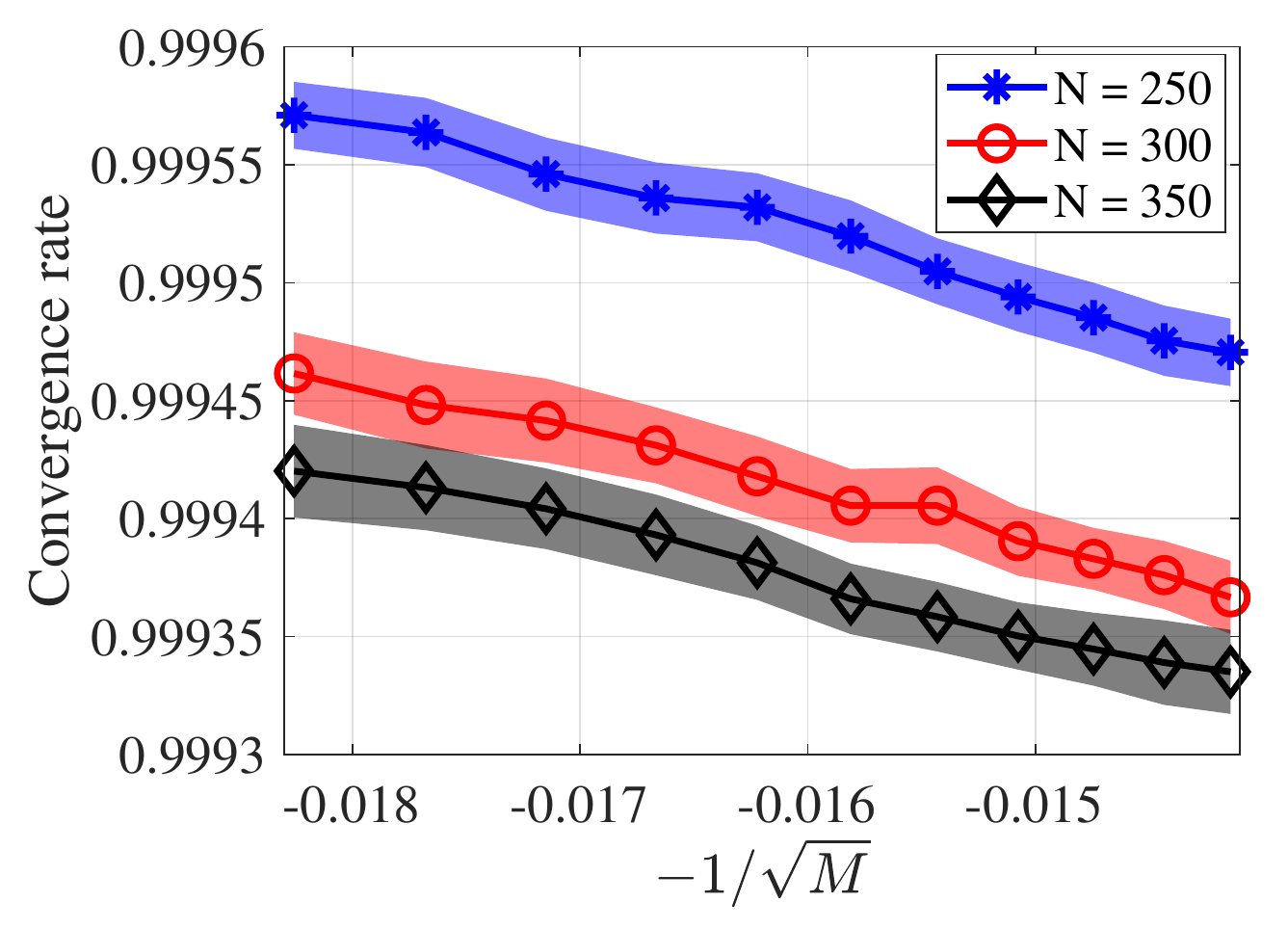}
        \caption{The convergence rate with different $M$ when $N<N^*$}
        \label{fig:M_ite}
\end{figure}

\textbf{(c) Convergence rate as a linear function of $1/\sqrt{M}$.}
Figure \ref{fig:M_ite} illustrates the convergence rate when $M$ and $N$ change. 
We can see that the convergence rate is a linear function of ${1}/{\sqrt{M}}$, as predicted by our results \eqref{eqn: p_convergence_*} and \eqref{eqn: convergence}. When $M$ increases, {the convergence rate is improved, and the method converges faster.} 

\textbf{(d) Increase of $\widetilde{\delta}$ slows down convergence.}  Figure \ref{fig: delta} shows that the convergence rate {becomes worse} when the variance of the unlabeled data $\widetilde{\delta}$ increases from $1$. When $\widetilde{\delta}$ is less than $1$, the convergence rate almost remains the same, which is consistent with our characterization in \eqref{eqn: p_convergence} that the convergence rate is linear in $\mu$. From the discussion after \eqref{eqn: definition_lambda}, $\mu$ increases as $\widetilde{\delta}$ increases from $1$ and stays constant when $\widetilde{\delta}$ is less than $1$.

\textbf{(e) Both $\| \W[L] - \bfW^* \|_F/\|\bfW^* \|_F$ and $\| \W[L] - \WL \|_F/\|\WL \|_F$ {are} linear functions of $-\hat{\lambda}$.}   Figure \ref{fig: lambda} shows that the relative errors of $\W[L]$ with respect to both $\bfW^*$ and $\W[L]$ decrease as a linear function of $-\hat{\lambda}$,  which are consistent with   the theoretical bounds in \eqref{eqn: p_convergence} and \eqref{eqn: p_convergence_*}.

\begin{figure}[h]
        \centering
        \includegraphics[width=0.6\linewidth]{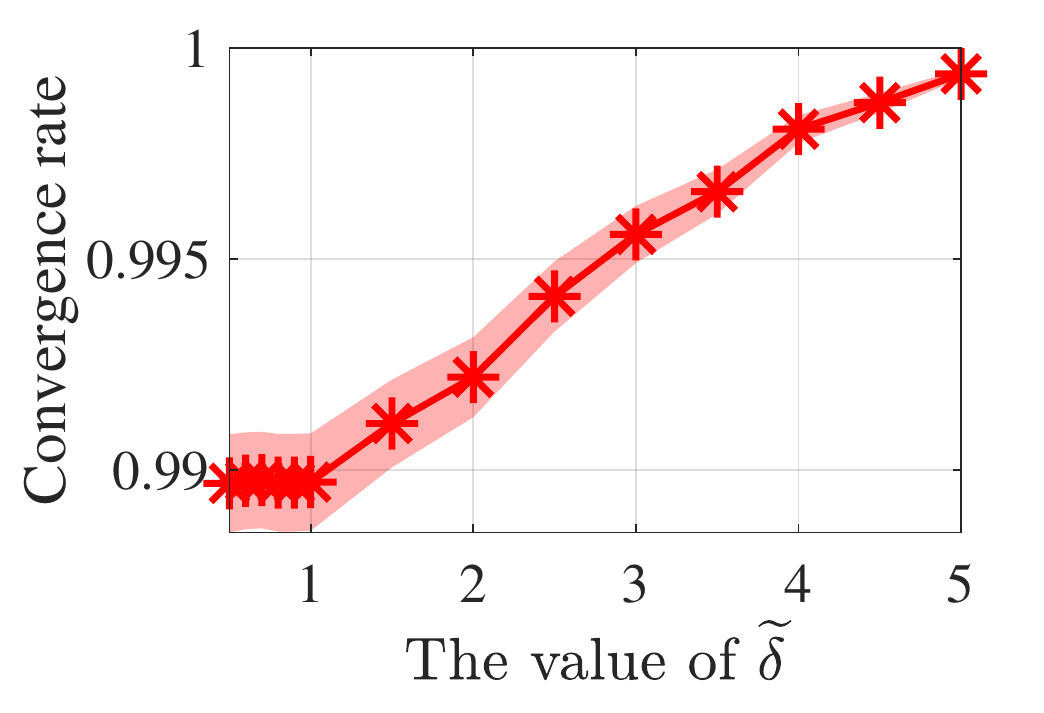}
  	    \caption{Convergence rate with different $\hat{\delta}$}
  	    \label{fig: delta}
\end{figure}
\begin{figure}[h]
        \centering
        \includegraphics[width=0.6\linewidth]{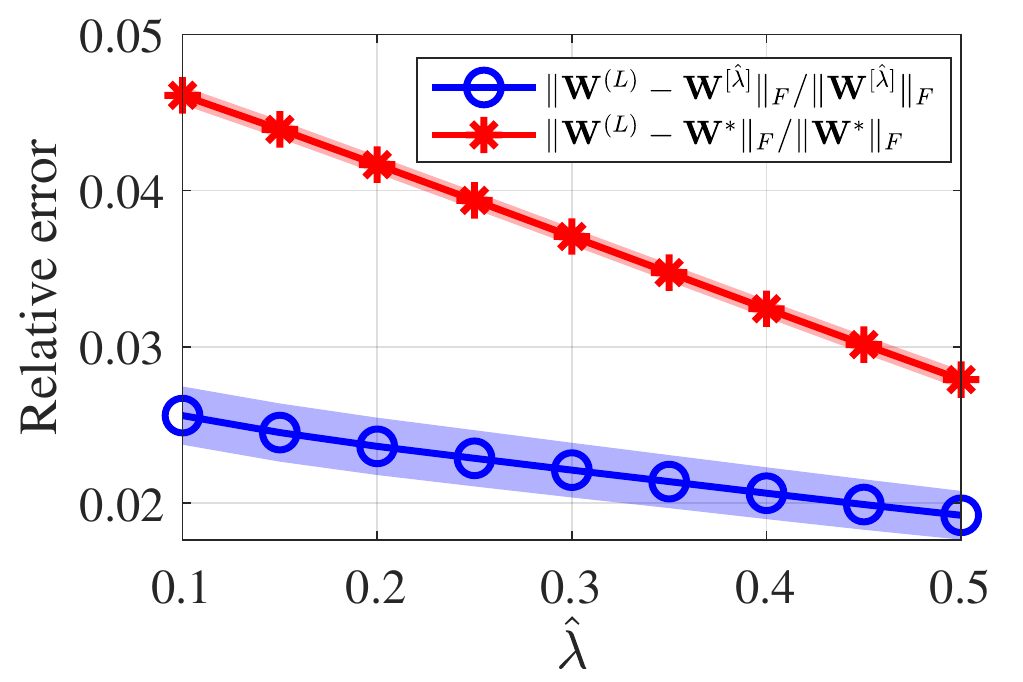}
  	    \vspace{-4mm}
  	    \caption{Relative errors of $\W[L]$ with respect to $\bfW^*$ and $\WL$ with different $\hat{\lambda}$}
  	    \label{fig: lambda}
\end{figure}

\textbf{(f) Unlabeled data reduce the sample complexity to learn $\bfW^*$.} Figure \ref{fig: M_pt} depicts 
the phase transition of {returning $\W[L]$}. 
For every pair of $d$ and $N$, we construct $100$ independent trials, and each trial is said to be successful if $\|\W[L]-\bfW^*\|_F/\| \bfW^* \|_F\le 10^{-2}$. The block in white depicts all the trials are successful, while the block in black indicates all failed. When $d$ increases, the required number of labeled data to learn $\bfW^*$ is linear in $d$. Thus, the sample complexity bound in \eqref{eqn: main_sample} is order-wise optimal for $d$. Moreover,  the phase transition line when $M=1000$ is below the one when $M=0$. Therefore, with unlabeled data, the required sample complexity of $N$ is reduced.

\begin{figure}[h]
    \centering
  	    \includegraphics[width=0.6\linewidth]{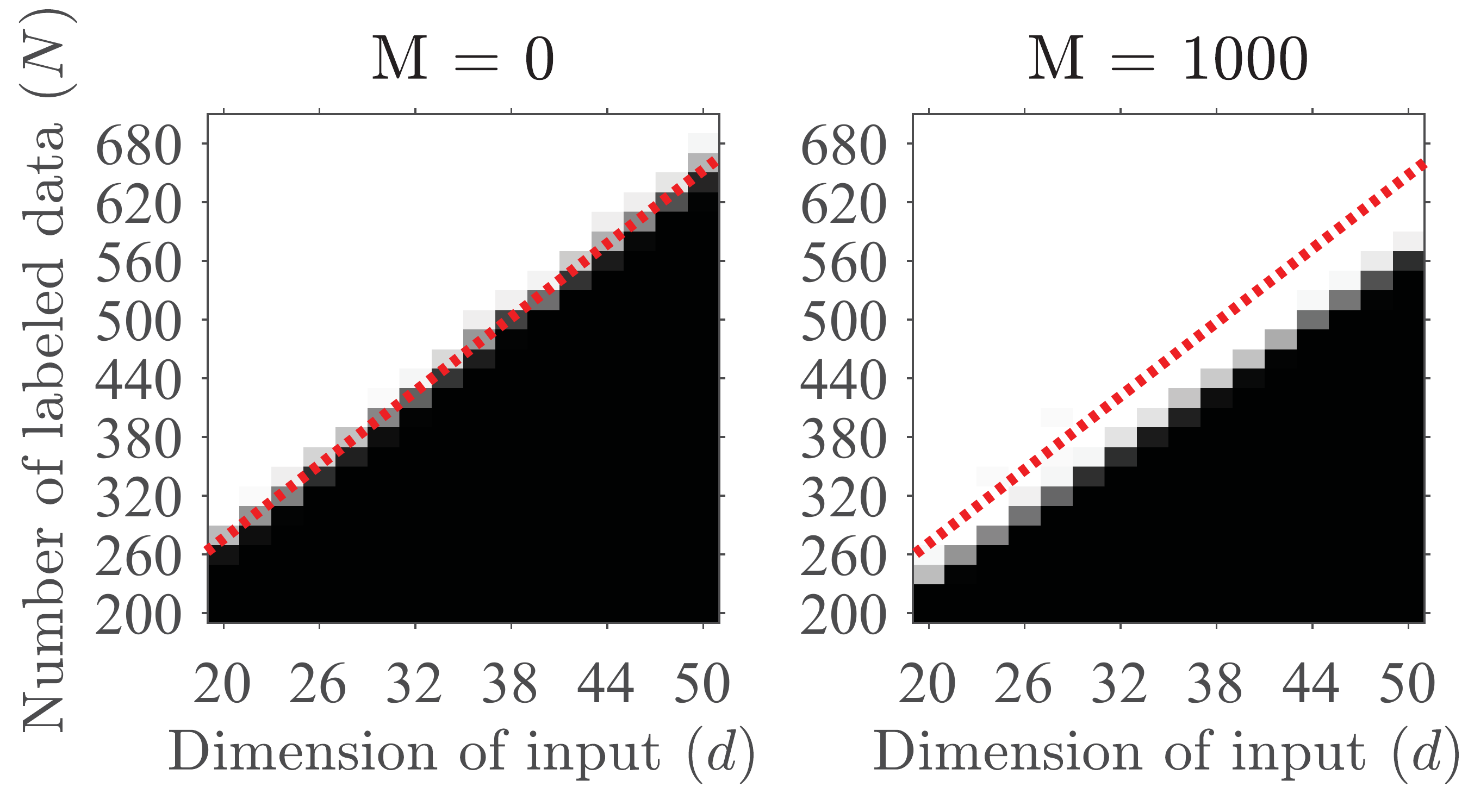}
  	\caption{Empirical phase transition of the curves with (a) $M = 0$ and (b) $M = 1000$}
  	\label{fig: M_pt}  
\end{figure}

\subsection{Image Classification on Augmented CIFAR-10 Dataset}
\label{sec: result_practice}
We evaluate  self-training  on the augmented CIFAR-10 dataset, which has 50K labeled data. The unlabeled data are mined from 80 Million Tiny Images following the setup in \cite{CRSD19}\footnote{The codes are downloaded from https://github.com/yaircarmon/semisup-adv}, and additional 50K images are selected for each class, which is a total of 500K images, to form the unlabeled data. The self-training method is the same implementation as that in \cite{CRSD19}. $\lambda$ and $\widetilde{\lambda}$ is selected as ${N}/{(M+N)}$ and ${M}/{(N+M)}$, respectively, and the algorithm stops after $200$ epochs.
In Figure \ref{fig: real_accuracy}, the dash lines stand for the best fitting of the linear functions of $1/\sqrt{M}$ via the least square method. One can see that the test accuracy is improved by up to $7\%$ using unlabeled data, and the empirical evaluations match the theoretical predictions. Figure \ref{fig: real_rate} shows the convergence rate calculated based on the first $50$ epochs, and the convergence rate is almost a linear function of $1/\sqrt{M}$, as predicted by \eqref{eqn: p_convergence}.

\begin{figure}[h]
        \centering
        \includegraphics[width=0.6\linewidth]{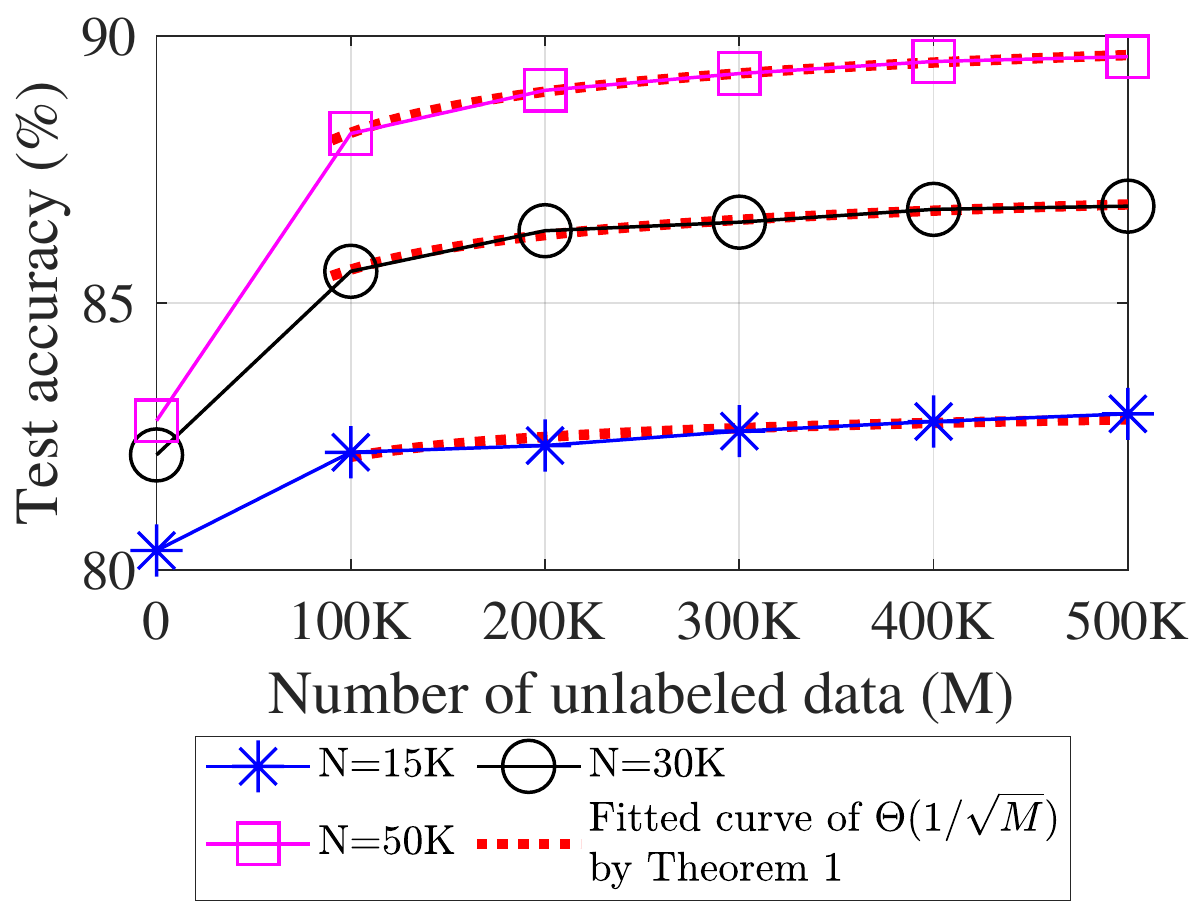}
        \caption{The test accuracy against the number of unlabeled data}
        \label{fig: real_accuracy}
\end{figure}
\begin{figure}[h]
        \centering
        \includegraphics[width=0.6\linewidth]{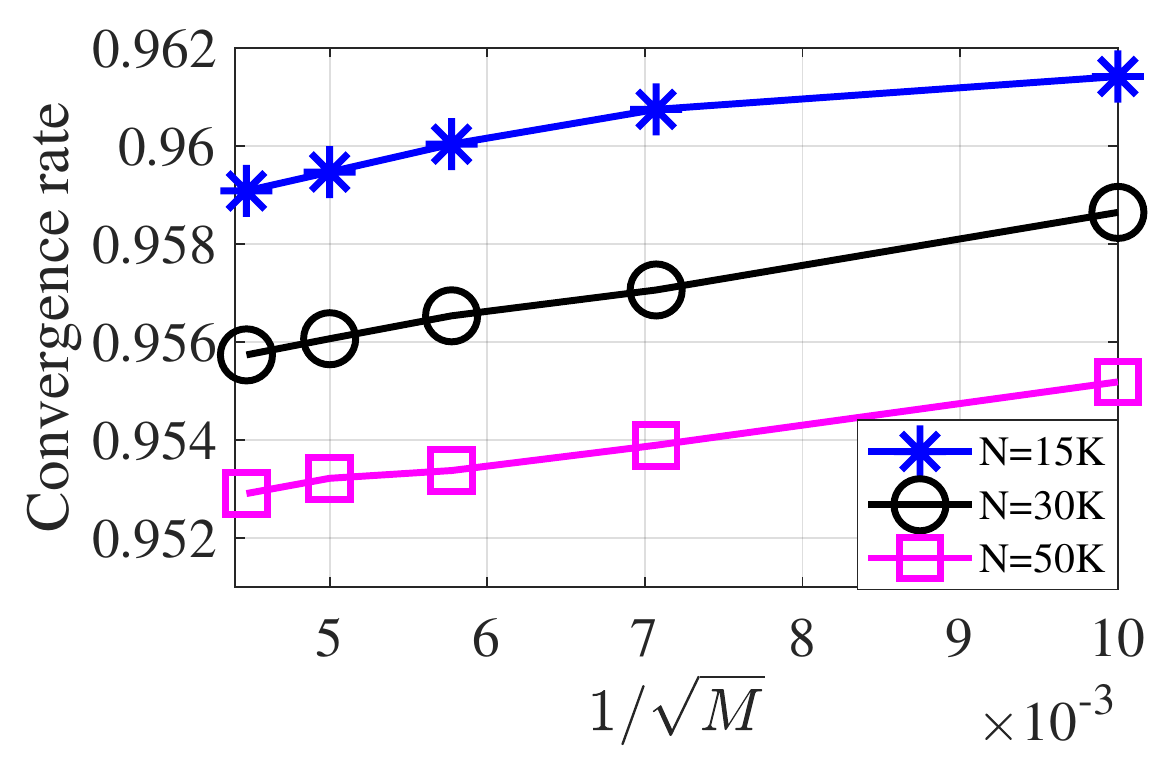}
        \caption{The convergence rate against the number of unlabeled data}
        \label{fig: real_rate}
\end{figure}

%% file: bibography/rpibib.tex
 
%
\specialhead{REFERENCES}
\bibliographystyle{IEEEtranN}

%% file: appendix/rpiapp.tex
\appendix    
\addtocontents{toc}{\parindent0pt\vskip12pt APPENDICES} 

\chapter{\uppercase{Supplementary proofs of the theorems in Chapter 2}}

\input{appendix/c1p1}
\input{appendix/c1p2}

\input{appendix/c1p3}

\chapter{\uppercase{Supplementary proofs of the theorems in Chapter 3}}

\input{appendix/c2p1}

\chapter{\uppercase{Supplementary proofs of the theorems in Chapter 4}}

\input{appendix/c3p1}

\chapter{\uppercase{Supplementary proofs of the theorems in Chapter 5}}

\input{appendix/c4p1}
\input{appendix/c4p2}

\begingroup
\newcommand{\w}{\widetilde{\bfw}}
\chapter{\uppercase{Supplementary proofs of the theorems in Chapter 6}}

\input{appendix/c5p1}

\input{appendix/c5p2}
\input{appendix/c5p3}
\input{appendix/c5p4}
\endgroup

\begingroup

\newcommand{\w[1]}{\widetilde{#1}}

\chapter{\uppercase{Supplementary proofs of the theorems in Chapter 7}}

\input{appendix/c6p1}
\input{appendix/c6p2}
\endgroup

%% file: appendix/c1p1.tex
In this part, part of the proofs of the theorems in chapter 1 are provided, while the full proofs can be found in \cite{ZHWC18,ZHWC17}. In Appendix \ref{sec: c1n1}, some important notations and techniques are introduced first to provide a comprehensive understanding about the roadmap of the proofs. In what followed, the proofs of Theorems 1, 2, 3, 4 and 5 are presented in Appendix  \ref{sec: c1p1}, \ref{sec: c1p2}, \ref{sec: c1p3}, \ref{sec: c1p4} and \ref{sec: c1p5}, respectively. 
\section{Notations and Assumptions}\label{sec: c1n1}
We first introduce notations used in the proof. 
For matrix $\bfZ_1\in\mathbb{C}^{n_c\times n}$, we define the \textbf{Block Hankel Operator} $\widetilde{\mathcal{H}}$ as
\vspace{-1mm}
$$\widetilde{\mathcal{H}}\bfZ_1
\vspace{-1mm}
=\begin{bmatrix}
\mathcal{H}\bfZ_1&\mathcal{H}\bfZ_1&\cdots&\mathcal{H}\bfZ_1
\end{bmatrix}
\in\mathbb{C}^{n_cn_1\times n_cn_2}.$$
$\widetilde{\mathcal{H}}\bfZ_1$ is an $n_c$-block Hankel matrix.
 $\widetilde{\mathcal{H}}^*$ is the adjoint operator of $\widetilde{\mathcal{H}}$. For any matrix $\bfZ_2\in \mathbb{C}^{n_cn_1\times n_cn_2}$, $(\widetilde{\mathcal{H}}^*\bfZ_2)\in\mathbb{C}^{n_c\times n}$ satisfies
\vspace{-3mm}
\begin{equation*}
\langle\widetilde{\mathcal{H}}^*\bfZ_2,\bfe_k \bfe_t^H\rangle=\sum_{l=0}^{n_c-1}\sum_{k_1+k_2=t+1}\langle\bfZ, \bfe_{(k_1-1)n_c+k}\bfe_{k_2+ln_2}^H\rangle.
\end{equation*}
Define  $\mathcal{\widetilde{D}}^2:=\widetilde{\mathcal{H}}^*\mathcal{\widetilde{H}}$, an operator from an $n_c\times n$ matrix $\bfZ$ to an $n_c \times n$ matrix with $$\widetilde{\mathcal{D}}^2\bfZ=\sum_{t=1}^{n}\sum_{k=1}^{n_c}n_cw_t\langle{\bfZ},\bfe_{k}\bfe_{t}^H\rangle\bfe_{k}\bfe_{t}^H,$$ 
where $w_t$ is defined in \eqref{w}. Then the Moore-Penrose pseudoinverse of $\widetilde{\mathcal{H}}$, denoted as $\widetilde{\mathcal{H}}^{\dagger}$, equals to $\widetilde{\mathcal{D}}^{-2}\widetilde{\mathcal{H}}^*$.  Further, we define $\mathcal{\widetilde{G}}=\mathcal{\widetilde{H}}\mathcal{\widetilde{D}}^{-1},$ then the adjoint operator of $\mathcal{\widetilde{G}}$ is defined as $\mathcal{\widetilde{G}}^*=\mathcal{\widetilde{D}}^{-1}\mathcal{\widetilde{H}}^*$. 
Additionally,
\begin{equation}\label{y_l}
\vspace{-1mm}
\bfY :=\widetilde{\mathcal{D}}{\bfX} \quad \text{and} \quad {\bfY}_l:=\widetilde{\mathcal{D}}{\bfX}_l.
\end{equation}
 
For any matrix $\bfZ\in\mathbb{C}^{n_c\times n}$, one can check that $ \| \widetilde{\mathcal{H}}{\bfX}\|=\sqrt{n_c}\left\| {\mathcal{H}}{\bfX}\right\|$ and $\| \widetilde{\mathcal{H}}{\bfX}\|_F=\sqrt{n_c}\left\| {\mathcal{H}}{\bfX}\right\|_F $. Immediately, 
$\widetilde{\mathcal{H}}\bfX$ and ${\mathcal{H}}\bfX$ share the same conditional number $\kappa$. Moreover, it is clear that $\widetilde{\mathcal{G}}$ is a unit operator as $\mathcal{\widetilde{G}}^*\mathcal{\widetilde{G}}=\mathcal{{I}}$, and $\mathcal{\widetilde{H}}{\bfX}=\mathcal{\widetilde{G}}{\bfY}$.

The following proofs will be established on Block Hankel Operator $\widetilde{\mathcal{H}}$. Consider AM-FIHT in terms of $\widetilde{\mathcal{H}}$,
the initialization can be written as  $\widetilde{{\bfL}}_0=p^{-1}\mathcal{Q}_r(\widetilde{{\mathcal{H}}}\mathcal{P}_{\Omega}({\bfX}))$. Further, the major steps can be represented as
\vspace{-1mm}
\begin{equation*}
\vspace{-2mm}
\widetilde{{\bfW}}_l=\mathcal{P}_{\widetilde{{\bf\mathcal{S}}}_l}\widetilde{\mathcal{H}}({\bfX}_l+p^{-1}{\bfG}_l+\beta\Delta\widetilde{\bfW}),\quad \widetilde{{\bfL}}_{l+1}=\mathcal{Q}_r(\widetilde{{\bfW}}_l),
\end{equation*}
where $\widetilde{\mathcal{{\bfS}}}_l$ is the tangent subspace at $\widetilde{{\bfL}}_l$. The resulting $\widetilde{\bfX}_l$ returned by AM-FIHT in terms of $\widetilde{\mathcal{H}}$ is given as $\widetilde{\mathcal{H}}^{\dagger}\widetilde{\bfL}_{l}$. 

Moreover, by replacing $\mathcal{H}$ with $\widetilde{\mathcal{H}}$,  AM-FIHT returns the same result as $\widetilde{\bfX}_l=\bfX_l$. In other words, 
we can show  that 
\begin{equation}\label{bar}
\widetilde{{\bfL}}_l=\begin{bmatrix}
{\bfL}_l &{\bfL}_l &\cdots &{\bfL}_l
\end{bmatrix}, \forall l\ge0.
\end{equation}
To see this, it is clear that (\ref{bar}) holds for $l=0$ from the definition of $\widetilde{\bfL}_0$.	Then suppose (\ref{bar}) holds when $l=k$. Immediately, we have $\widetilde{{\bfW}}_{k}=\begin{bmatrix}
	{\bfW}_{k} &{\bfW}_{k} &\cdots &{\bfW}_{k}
	\end{bmatrix}$. Let  ${\bfW}_{k}=\sum_{i=1}^{\min(n_cn_1, n_2)}\sigma_i {\bfu}_i{\bfv}_i^*$ be the SVD with $\sigma_1\ge\sigma_2\ge\cdots\ge\sigma_{\min(n_cn_1, n_2)}$.  Then for $l=k+1$, 
	\small
	\vspace{-2mm}
	\begin{equation*}
	\vspace{-2mm}
	\begin{split}
	\widetilde{{\bfL}}_{k+1}=\mathcal{Q}_r(\widetilde{{\bfW}}_{k})
	&=\sum_{i=1}^{r}\sqrt{n_c}\sigma_i {\bfu}_i \frac{1}{\sqrt{n_c}}
	\begin{bmatrix}
	{\bfv}_i^*& \cdots& {\bfv}_i^*
	\end{bmatrix}\\
	&=\sum_{i=1}^{r}\begin{bmatrix}
	\sigma_i {\bfu}_i{\bfv}_i^*& \cdots& \sigma_i {\bfu}_i{\bfv}_i^*
	\end{bmatrix}.\\
	&=\begin{bmatrix}
	{\bfL}_{k+1} &{\bfL}_{k+1} &\cdots &{\bfL}_{k+1}
	\end{bmatrix}.
	\end{split}
	\end{equation*}
	\normalsize
Hence, the connection between $\bfX_l$ and $\widetilde{\bfL}_l$ can be given as 
\begin{equation}
\begin{split}
&\|\bfX_l-\bfX\|_F
=\|\widetilde{\mathcal{D}}^{-1}({\bfY}_l-{\bfY})\|_F\le \frac{1}{\sqrt{n_c}}\|{\bfY}_l-{\bfY}\|_F\\
&~~~=\frac{1}{\sqrt{n_c}}\|\widetilde{\mathcal{G}}^{*}(\widetilde{\bfL}_l-\widetilde{\mathcal{G}}{\bfY})\|_F\le \frac{1}{\sqrt{n_c}}\|\widetilde{\bfL}_l-\widetilde{\mathcal{G}}{\bfY}\|_F.
\end{split}
\end{equation}

For RAM-FIHT, 
similarly define an $n_c$-block matrix 
 $\widetilde{\bfL}^{\prime}_l=\begin{bmatrix}
\bfL^{\prime}_l&\bfL^{\prime}_l&\cdots&\bfL^{\prime}_l
\end{bmatrix}$. From the discussion above, one can verify that (\ref{bar}) also holds in RAM-FIHT. Then 
the SVD of $\widetilde{\bfL}_l$ in \eqref{bar} is $\widetilde{\bfL}_l=\widetilde{\bfU}_l\widetilde{\bf\Sigma}_l\widetilde{\bfV}_l^*
=
\bfU_l \big(\sqrt{n_c}{\bf\Sigma}_l\big)
\big(\frac{1}{\sqrt{n_c}}[\bfV_l^* \cdots \bfV_l^*]
\big)$. Then $\widetilde{\bfA}_l$ and $\widetilde{\bfB}_l$ are defined as 
\small
\begin{equation}\label{trimming1}
(\widetilde{\bfA}_{l})_{i*}=\displaystyle\frac{(\widetilde{\bfU}_{l})_{i*}}{\left\| (\widetilde{\bfU}_{l})_{i*}\right\|}\min\bigg\{\left\|(\widetilde{\bfU}_{l})_{i*} \right\|, \sqrt{\frac{\mu r}{n_cn_1}} \bigg\},
\end{equation}
\begin{equation}\label{trimming2}
(\widetilde{\bfB}_{l})_{i*}=\displaystyle\frac{(\widetilde{\bfV}_{l})_{i*}}{\left\| (\widetilde{\bfV}_{l})_{i*}\right\|}\min\bigg\{\left\|(\widetilde{\bfV}_{l})_{i*} \right\|, \sqrt{\frac{\mu r}{n_cn_2}} \bigg\}.
\end{equation}
\normalsize

{\bf Sampling model with replacement}. As shown in  \cite{R11}, 
due to the duplications,  
 the number of observed entries in a sampling model with replacement is less than or equal to that in a sampling model without replacement. 
 Thus, it is sufficient to study the sampling model with replacement. 
 To distinguish $\widehat{\Omega}$ in \eqref{m2}, which represents the sampling set without replacement, let $\Omega$ be $m$ unions of indices sampled uniformly from set $\{1,2,\cdots,n_c\}\times\{1,2,\cdots,n\}$ with replacement, and
 \vspace{-2mm}
 \begin{equation}\label{sampling with replacement}
 \vspace{-2mm}
 \mathcal{P}_{\Omega}(\bfZ_1)=\sum_{a=1}^{m}\langle\bfZ_1, \bfe_{k_a}\bfe_{t_a}^H\rangle\bfe_{k_a}\bfe_{t_a}^H,
 \end{equation}
 for any $\bfZ_1\in\mathbb{C}^{n_c\times n}$. By changing the sampling model, several critical lemmas can be derived from Bernstein Inequality.
\begin{lemma}[\cite{T12}, Theorem 1.6]\label{prob}
	Consider a finite sequence $\{{\bfZ}_k\}$ of independent, random matrices with dimensions $d_1\times d_2$. Assume that such random matrix satisfies
	\vspace{-2mm}
	\begin{equation*}
	\vspace{-2mm}
	\mathbb{E}({\bfZ_k})=0\quad and \quad \left\|{\bfZ_k}\right\|\le R \quad almost ~ surely.	
	\end{equation*}
	Define
	\vspace{-2mm}
	\begin{equation*}
	\vspace{-2mm}
	\sigma^2:=\max\Big\{\Big\|\sum_{k}\mathbb{E}({\bfZ}_k{\bfZ}_k^H)\Big\|,\Big\|\displaystyle\sum_{k}\mathbb{E}({\bfZ}_k^H{\bfZ}_k)\Big\|\Big\}.
	\end{equation*}
	Then for all $t\ge0$,
	\begin{equation*}
	\mathbb{P}\Bigg\{ \left\|\displaystyle\sum_{k}{\bfZ}_k\right\|\ge t \Bigg\}\le(d_1+d_2)\exp\Big(\frac{-t^2/2}{\sigma^2+Rt/3}\Big).
	\end{equation*}
\end{lemma}
Suppose that $t\le \sigma^2/R$, then the right-hand side can be released as $(d_1+d_2)\exp(-\frac{3}{8}t^2/\sigma^2)$. Such kind of manipulation will be adopted in several proofs.

Note that the set of matrices $$\big\{\widetilde{{\bfH}}_{k.t}|\widetilde{{\bfH}}_{k,t}=\frac{1}{\sqrt{n_cw_t}}\widetilde{\mathcal{H}}({\bfe}_k{\bfe}_t^H),1\le k \le n_c, 1\le t \le n\big\}$$ forms an orthonormal basis of the $n_c$-block Hankel matrix, where
$$\widetilde{\mathcal{H}}\bfX=\sum_{k=1}^{n_c}\sum_{t=1}^{n}\langle\widetilde{{\bfH}}_{k,t}, \widetilde{\mathcal{H}}\bfX\rangle\widetilde{{\bfH}}_{k,t}.
$$
Then, for all $(k_a,t_a)\in\Omega$, $\mathcal{P}_{\Omega}$ is also used as the operator
 \begin{equation*}
 \mathcal{P}_{\Omega}({\bfZ}_2)=\sum_{a=1}^{m}\langle{\bfZ}_2,\widetilde{{\bfH}}_{k_a,t_a}\rangle\widetilde{{\bfH}}_{k_a,t_a},
 \end{equation*}
 for any ${\bfZ}_2\in\mathbb{C}^{n_cn_1 \times n_cn_2}$. In spite of abuse of notation, the meaning of $\mathcal{P}_{\Omega}$ is clear from context. By such definition,  $\mathcal{P}_{\Omega}(\widetilde{\mathcal{H}}{\bfZ}_1)=\widetilde{\mathcal{H}}\mathcal{P}_{\Omega}({\bfZ}_1)$ for any ${\bfZ}_1\in\mathbb{C}^{n \times n_c}$. Additionally,
$\widetilde{{\bfH}}_{k,t}$ only has $n_c{w}_t$ nonzero entries of magnitude $1/\sqrt{n_c{w}_t}$, so $\|\widetilde{{\bfH}}_{k,t} \|_F =1$. The following lemma can be established directly from the definition of incoherence.
\begin{lemma}[\cite{ZHWC18}, Lemma 2]
\label{def}
	Let $\mathcal{\widetilde{H}}{\bfX}={\widetilde{\bfU}\widetilde{\bf\Sigma} \widetilde{\bfV}^H}$ be the SVD of $\widetilde{\mathcal{H}}\bfX$. Assume $\mathcal{\widetilde{H}}{\bfX}$ is $\mu$-incoherent. Then\\
	\begin{equation}\label{co_1}
	\left\|{\bfe}_{k_1}^H\widetilde{{\bfU}}\right\|^2\le\frac{\mu c_sr}{n_cn} \quad, \quad \left\| {\bfe}_{k_2}^H\widetilde{{\bfV}}\right\|^2\le\frac{\mu c_sr}{n_cn},
	\end{equation}
	\begin{equation}\label{co_2}
	\left\|\mathcal{P}_{\widetilde{{\bfU}}}(\widetilde{{\bfH}}_{k,t})\right\|_F^2\le \displaystyle\frac{\mu c_sr}{n_cn} \quad ,\quad   \left\|\mathcal{P}_{\widetilde{{\bfV}}}(\widetilde{{\bfH}}_{k,t})\right\|_F^2\le \displaystyle\frac{\mu c_sr}{n_cn}.
	\end{equation}
	where ${\bfe}_{k_1},{\bfe}_{k_2}$ are the coordinate unit vectors.
\end{lemma}


%% file: appendix/c1p2.tex
\section{Proof of Theorem \ref{t1}}\label{sec: c1p1}
We first present some supporting lemmas to prove Theorem \ref{t1}. Lemma \ref{condition1} shows that the maximum number of repetitions is bounded by ${O}(\log n)$ with high probability in uniform sampling. {Lemma \ref{condition2}} derives the properties of $p^{-1}\mathcal{P}_\mathcal{\widetilde{\mathcal{S}}}\widetilde{\mathcal{G}}{\mathcal{P}}_{\Omega}\widetilde{\mathcal{G}}^*\mathcal{P}_\mathcal{\widetilde{\mathcal{S}}}$, and the random operator can be close enough to its mean $\mathcal{P}_\mathcal{\widetilde{\mathcal{S}}}\widetilde{\mathcal{G}}\widetilde{\mathcal{G}}^*\mathcal{P}_\mathcal{\widetilde{\mathcal{S}}}$ with a significant large amount of observed entries. Lemma \ref{sub1} connects the angle of two subspaces, represented as $\| \mathcal{P}_{\mathcal{\widetilde{\mathcal{S}}}_l}-\mathcal{P}_{\mathcal{\widetilde{\mathcal{S}}}}\|$, with $\|\widetilde{{\bfL}}_l-\widetilde{\mathcal{G}}{\bfY}\|_F$, and shows the angle decreases as $\bfL_l$ approaches to the ground truth. Lemma \ref{condition3} indicates the distance between the initial point and ground truth.  Lemmas \ref{condition2} and \ref{condition3} are built upon Lemmas 5 and 2 of \cite{CWW17}, respectively,  by extending   from single-channel signals to multi-channel signals.
 When $n_c=1$, Lemmas \ref{condition2} and \ref{condition3} are reduced to corresponding lemmas in \cite{CWW17}.
 
 The proof of Theorem \ref{t1} is extended from  that of Theorem 3 in \cite{CWW17} with some modifications. 
The majority of the  efforts are devoted to the ``Inductive Step'' to build the connections between $\bfW_{l-1}$ and $\bfW_{l}$  through \eqref{I_4}. In \eqref{I_4}, the major issue is to bound 
$I_1, I_2, I_3$ and $I_4$, and \eqref{eqn:t1.I124} and \eqref{eqn:t1.I3} provide critical steps in bounding these items. This inductive step is built upon a similar analysis for 
 $\bfL_l$'s in \cite{CWW17}. Here we study   $\bfW_{l}$'s instead of $\bfL_{l}$'s since the analysis of Theorem 2 is based on analyzing $\bfW_{l}$'s. 
 Although Lemma \ref{condition3} provides the theoretical bound for $\bfL_0$ directly,   a similar result for $\bfW_{0}$ is lacking. Thus, some efforts to analyze $\bfW_{0}$ is needed in the ``Base Case'' part. \eqref{eqn:m3} in Theorem \ref{t1} is obtained from  \eqref{eqn:t1.bound}, which provides the theoretical bound for the required number of observations to ensure successful recovery. We include detailed steps as follows for the completeness of the proof.
\begin{lemma}[\cite{ZHWC18}, Lemma 3]
	\label{condition1}
	Under sampling with replacement model, the maximum number of repetitions of any entry in $\Omega$ is less than $3\log(n)$ with probability at least $1-n_cn^{-2}$ for $n\ge 12$.
\end{lemma}
\begin{lemma}[\cite{ZHWC18}, Lemma 4]
	\label{condition2}
	Let $\widetilde{\mathcal{S}}$ be the tangent subspace of $\widetilde{\mathcal{H}}\bfX$. Assume $\mathcal{\widetilde{H}}{\bfX}$ is $\mu$-incoherent. Then with $m\ge32\mu c_sr\log(n)$,
	\begin{equation*}	
	\left\| \mathcal{P}_{\mathcal{\widetilde{\mathcal{S}}}}\widetilde{\mathcal{G}}\widetilde{\mathcal{G}}^*\mathcal{P}_{\mathcal{\widetilde{\mathcal{S}}}}-p^{-1}\mathcal{P}_{\mathcal{\widetilde{\mathcal{S}}}}\widetilde{\mathcal{G}}{\mathcal{P}}_{\Omega}\widetilde{\mathcal{G}}^*\mathcal{P}_{\mathcal{\widetilde{\mathcal{S}}}} \right\|\le\displaystyle\sqrt{\frac{32\mu c_sr\log(n)}{m}}
	\end{equation*}
	holds with probability at least $1-n_cn^{-2}$.
\end{lemma}

\begin{lemma}[\cite{WCCL16}, Lemma 4.1]\label{sub1}
	Let ${\bfZ}_l$ be a rank-$r$ matrix and $\mathcal{S}_l$ be the tangent subspace of $\bfZ_l$. If $\bfZ$ is also a rank-$r$ matrix and its tangent subspace is denoted as $\mathcal{S}$, then 
	\begin{equation}\label{sub1.1}
	\left\| ({\mathcal{I}}-\mathcal{P}_{\mathcal{S}_l})(\bfZ_l-\bfZ) \right\|_F\le\displaystyle\frac{\left\|\bfZ_l-\bfZ\right\|_F^2 }{\sigma_{\rm min}(\bfZ)},\\
	\end{equation}
	\vspace{-2mm}
	\begin{equation}\label{sub1.3}
	\left\| \mathcal{P}_{\mathcal{S}_l}-\mathcal{P}_{\mathcal{S}}\right\|\le\displaystyle\frac{2\left\|\bfZ_l-\bfZ\right\|_F }{\sigma_{\rm min}(\bfZ)}.
	\end{equation}
\end{lemma}
\begin{lemma}[\cite{ZHWC18}, Lemma 6]
\label{condition3}
	Assume $\mathcal{\widetilde{H}}{\bfX}$ is $\mu$-incoherent. With the initial point $\widetilde{{\bfL}}_0=\mathcal{Q}_r(p^{-1}\mathcal{\widetilde{H}}{\mathcal{P}}_{\Omega}({\bfX}))$, if $m\ge 16 \mu c_sr\log(n)$, we have
	\begin{equation*}
	\left\| \widetilde{{\bfL}}_0-\widetilde{\mathcal{H}}{\bfX} \right\|\le \displaystyle \sqrt{\frac{64\mu c_sr\log(n)}{m}}\left\| \widetilde{\mathcal{H}}{\bfX} \right\|
	\end{equation*} 
	holds with probability at least $1-n_cn^{-2}$.
\end{lemma}


\begin{proof}[Proof of Theorem \ref{t1}]
	As $\widetilde{{\bfL}}_{l+1}=\mathcal{Q}_r(\widetilde{{\bfW}_l})$, $\widetilde{{\bfL}}_{l+1}$ is the best rank-$r$ approximation to $\widetilde{{\bfW}}_l$. Then we have 
	\begin{equation}\label{W_L}
	\begin{split}
	\|\widetilde{{\bfL}}_{l+1}-\mathcal{\widetilde{G}}{\bfY}\|_F
	&\le \|\widetilde{{\bfW}}_l-\widetilde{{\bfL}}_{l+1}\|_F+\|\widetilde{{\bfW}}_{l}-\mathcal{\widetilde{G}}{\bfY}\|_F 
	\le 2\|\widetilde{{\bfW}}_{l}-\mathcal{\widetilde{G}}{\bfY}\|_F.
	\end{split}
	\end{equation}
	\normalsize
	Therefore, it is sufficient to bound $\|\widetilde{{\bfW}}_{l}-\mathcal{\widetilde{G}}{\bfY}\|_F$.
	{Lemma \ref{condition1}} suggests that with probability at least $1-n_cn^{-2}$, 
	\vspace{-1mm}
	\begin{equation}\label{c1}
	\|\mathcal{P}_{\Omega}\|\le 3\log(n)
	\vspace{-1mm}
	\end{equation}
	holds. 
	{Lemma \ref{condition2}} suggests as long as $m\ge 32\varepsilon_0^{-2}\mu c_sr\log(n)$,
	\vspace{-1mm}
	\begin{equation}\label{c2}
	\left\|\mathcal{P}_\mathcal{\widetilde{\mathcal{S}}}\widetilde{\mathcal{G}}\widetilde{\mathcal{G}}^*\mathcal{P}_\mathcal{\widetilde{\mathcal{S}}}-p^{-1}\mathcal{P}_\mathcal{\widetilde{\mathcal{S}}}\widetilde{\mathcal{G}}{\mathcal{P}}_{\Omega}\widetilde{\mathcal{G}}^*\mathcal{P}_\mathcal{\widetilde{\mathcal{S}}}\right\|\le\varepsilon_0
	\vspace{-1mm}
	\end{equation}  
	holds with probability at least $1-n_cn^{-2}$.
	
	Now we will show that the following inequality holds by mathematical induction, 
	\begin{equation}\label{c4}
	\displaystyle\frac{\|\widetilde{\bfW}_k-\widetilde{\mathcal{G}}{\bfY}\|_F }{\sigma_{\min}(\widetilde{\mathcal{G}}{\bfY})}\le\displaystyle\frac{p^{1/2}\varepsilon_0}{12\log(n)(1+\varepsilon_0)}.
	\end{equation}
	{\bf Inductive Step:} Suppose (\ref{c4}) holds for $k=l-1$. 
	Recall that
	\begin{equation*}
	\begin{split}
	\widetilde{{\bfW}}_l&=\mathcal{P}_{\mathcal{\widetilde{\mathcal{S}}}_l}\mathcal{\widetilde{H}}({\bfX}_l+p^{-1}\mathcal{P}_{\Omega}({\bfX}-{\bfX}_l))
	=\mathcal{P}_{\mathcal{\widetilde{\mathcal{S}}}_l}\mathcal{\widetilde{G}}({\bfY}_l+p^{-1}\mathcal{P}_{\Omega}({\bfY}-{\bfY}_l)).
	\end{split}
	\end{equation*}
	Then, for all $ l\ge1$, we have
	\small
	\begin{equation}\label{I_4}
		\begin{split}
		&\left\|\widetilde{{\bfW}}_{l}-\mathcal{\widetilde{G}}{\bfY}\right\|_F\\
		=&\left\|\mathcal{P}_{\mathcal{\widetilde{\mathcal{S}}}_l}\mathcal{\widetilde{G}}({\bfY}_l+p^{-1}\mathcal{P}_{\Omega}({\bfY}-{\bfY}_l)-\mathcal{\widetilde{G}}{\bfY}\right\|_F\\
		=& \left\|\mathcal{P}_{\mathcal{\widetilde{\mathcal{S}}}_l}\mathcal{\widetilde{G}}{\bfY}-\mathcal{\widetilde{G}}{\bfY}+(\mathcal{P}_{\mathcal{\widetilde{\mathcal{S}}}_l}\mathcal{\widetilde{G}}-p^{-1}\mathcal{P}_{\mathcal{\widetilde{\mathcal{S}}}_l}\mathcal{\widetilde{G}}\mathcal{P}_{\Omega})({\bfY}_l-{\bfY})\right\|_F\\
		\stackrel{(a)}{\le}&\left\|(\mathcal{I}-\mathcal{P}_{\mathcal{\widetilde{\mathcal{S}}}_l})(\widetilde{{\bfL}}_l-\widetilde{\mathcal{G}}{\bfY})\right\|_F
		+\left\|(\mathcal{P}_{\mathcal{\widetilde{\mathcal{S}}}_l}\mathcal{\widetilde{G}}\mathcal{\widetilde{G}}^*-p^{-1}\mathcal{P}_{\mathcal{\widetilde{\mathcal{S}}}_l}\mathcal{\widetilde{G}}\mathcal{P}_{\Omega}\mathcal{\widetilde{G}}^*)(\mathcal{P}_{\widetilde{\mathcal{S}}_l}\widetilde{{\bfW}}_{l-1}-\mathcal{\widetilde{G}}{\bfY})\right\|_F\\
		\le& \left\|(\mathcal{I}-\mathcal{P}_{\mathcal{\widetilde{\mathcal{S}}}_l})(\widetilde{{\bfL}}_l-\widetilde{\mathcal{G}}{\bfY})\right\|_F+\left\|\mathcal{P}_{\mathcal{\widetilde{\mathcal{S}}}_l}\mathcal{\widetilde{G}}\mathcal{\widetilde{G}}^*(\mathcal{I}-\mathcal{P}_{\mathcal{\widetilde{\mathcal{S}}}_l})(\widetilde{{\bfL}}_l-\widetilde{\mathcal{G}}{\bfY})\right\|_F\\
		&+\left\|(\mathcal{P}_{\mathcal{\widetilde{\mathcal{S}}}_l}\mathcal{\widetilde{G}}\mathcal{\widetilde{G}}^*\mathcal{P}_{\mathcal{\widetilde{\mathcal{S}}}_l}-p^{-1}\mathcal{P}_{\mathcal{\widetilde{\mathcal{S}}}_l}\mathcal{\widetilde{G}}\mathcal{P}_{\Omega}\mathcal{\widetilde{G}}^*\mathcal{P}_{\mathcal{\widetilde{\mathcal{S}}}_l})(\widetilde{{\bfW}}_{l-1}-\mathcal{\widetilde{G}}{\bfY})\right\|_F\\
		&+p^{-1}\left\|\mathcal{P}_{\mathcal{\widetilde{\mathcal{S}}}_l}\mathcal{\widetilde{G}}\mathcal{P}_{\Omega}\mathcal{\widetilde{G}}^*(\mathcal{I}-\mathcal{P}_{\mathcal{\widetilde{\mathcal{S}}}_l})(\widetilde{{\bfL}}_l-\widetilde{\mathcal{G}}{\bfY})\right\|_F\\
		& := I_1+I_2+I_3+I_4,
		\end{split}
		\end{equation}
	where (a) holds since $\widetilde{\bfL}_l=\mathcal{P}_{\mathcal{\widetilde{\mathcal{S}}}_l}\widetilde{\bfW}_{l-1}$.
	\normalsize
	By applying (\ref{sub1.1}),
	\begin{equation}\label{eqn:t1.I124}
	\begin{split}
	 I_1+I_2+I_4
	\le&\frac{2\|\widetilde{{\bfL}}_l-\mathcal{\widetilde{G}}{\bfY}\|_F^2 }{\sigma_{\min}(\mathcal{\widetilde{G}}{\bfY})}+p^{-1}\left\| \mathcal{P}_{\Omega}\mathcal{\widetilde{G}}^*\mathcal{P}_{\mathcal{\widetilde{\mathcal{S}}}_l}\right\|\frac{\|\widetilde{{\bfL}}_l-\mathcal{\widetilde{G}}{\bfY}\|_F^2 }{\sigma_{\min}(\mathcal{\widetilde{G}}{\bfY})} \\
	\le&\frac{8\|\widetilde{{\bfW}}_{l-1}-\mathcal{\widetilde{G}}{\bfY}\|_F^2 }{\sigma_{\min}(\mathcal{\widetilde{G}}{\bfY})}+4p^{-1}\left\| \mathcal{P}_{\Omega}\mathcal{\widetilde{G}}^*\mathcal{P}_{\mathcal{\widetilde{\mathcal{S}}}_l}\right\|\frac{\|\widetilde{{\bfW}}_{l-1}-\mathcal{\widetilde{G}}{\bfY}\|_F^2 }{\sigma_{\min}(\mathcal{\widetilde{G}}{\bfY})}.
	\end{split}
	\end{equation}
	Next, we will bound $\| \mathcal{P}_{\Omega}\mathcal{\widetilde{G}}^*\mathcal{P}_{\mathcal{\widetilde{\mathcal{S}}}_l}\|$. 
	For any ${\bfZ}\in \mathbb{C}^{n_cn_1\times n_cn_2}$, 
	\vspace{-1mm}
 	\begin{equation}
		\begin{split}
		\|\mathcal{P}_{\Omega}\widetilde{\mathcal{G}}^*\mathcal{P}_\mathcal{\widetilde{\mathcal{S}}}({\bfZ})\|^2
		=&\langle\mathcal{P}_{\Omega}\widetilde{\mathcal{G}}^*\mathcal{P}_\mathcal{\widetilde{\mathcal{S}}}({\bfZ}),\mathcal{P}_{\Omega}\widetilde{\mathcal{G}}^*\mathcal{P}_\mathcal{\widetilde{\mathcal{S}}}({\bfZ})\rangle		\\
		\stackrel{(b)}{\le}& 3\log(n)\langle\widetilde{\mathcal{G}}^*\mathcal{P}_\mathcal{\widetilde{\mathcal{S}}}({\bfZ}),\mathcal{P}_{\Omega}\widetilde{\mathcal{G}}^*\mathcal{P}_\mathcal{\widetilde{\mathcal{S}}}({\bfZ})\rangle\\
		=&3\log(n)\langle{\bfZ},\mathcal{P}_\mathcal{\widetilde{\mathcal{S}}}\widetilde{\mathcal{G}}\mathcal{P}_{\Omega}\widetilde{\mathcal{G}}^*\mathcal{P}_\mathcal{\widetilde{\mathcal{S}}}({\bfZ})\rangle\\
		\stackrel{(c)}{\le}& 3\log(n)(1+\varepsilon_0)p\left\|{\bfZ}\right\|_F^2.
		\end{split}
	\end{equation}
	\normalsize	
	where (b) holds due to (\ref{c1}), and (c) holds due to (\ref{c2}). Hence, $$\left\|\mathcal{P}_\mathcal{\widetilde{\mathcal{S}}}\widetilde{\mathcal{G}}\mathcal{P}_{\Omega}\right\|=\left\|\mathcal{P}_{\Omega}\widetilde{\mathcal{G}}^*\mathcal{P}_\mathcal{\widetilde{\mathcal{S}}}\right\|\le \sqrt{3\log(n)(1+\varepsilon_0)p},$$
	and
	\small
	\begin{equation}\label{sc1}
		\begin{split}
		\left\|\mathcal{P}_{\Omega}\widetilde{\mathcal{G}}^*\mathcal{P}_{\mathcal{\widetilde{\mathcal{S}}}_l}\right\|
		\le&\left\|\mathcal{P}_{\Omega}\widetilde{\mathcal{G}}^*(\mathcal{P}_{\mathcal{\widetilde{\mathcal{S}}}_l}-\mathcal{P}_{\mathcal{\widetilde{\mathcal{S}}}})\right\|+\left\|\mathcal{P}_{\Omega}\widetilde{\mathcal{G}}^*\mathcal{P}_{\mathcal{\widetilde{\mathcal{S}}}}\right\|\\
		\stackrel{(a)}{\le}& 3\log(n)\frac{2\|\widetilde{{\bfL}}_l-\mathcal{\widetilde{G}}{\bfY}\|_F }{\sigma_{\min}(\mathcal{\widetilde{G}}{\bfY})}+\left\|\mathcal{P}_{\Omega}\widetilde{\mathcal{G}}^*\mathcal{P}_{\mathcal{\widetilde{\mathcal{S}}}}\right\|\\
		\stackrel{(b)}{\le}& 3\log(n)\frac{p^{1/2}\varepsilon_0}{3\log(n)(1+\varepsilon_0)}+\sqrt{3\log(n)(1+\varepsilon_0)p}\\
		\le& 3\log(n)(1+\varepsilon_0)p^{1/2},
		\end{split}
	\end{equation}
	\normalsize
	where (a) holds due to \eqref{sub1.3} and (\ref{c1}), and (b) holds due to (\ref{W_L}) and the inductive hypothesis.\\
	Hence, $I_1+I_2+I_4	\le 2\varepsilon_0 \|\widetilde{{\bfW}}_{l-1}-\mathcal{\widetilde{G}}{\bfY}\|_F$.
	Moreover, 
	\small
	\vspace{-1mm}
	\begin{align}\label{eqn:t1.I3}
		\begin{split}
		&\left\|\mathcal{P}_{\mathcal{\widetilde{\mathcal{S}}}_l}\widetilde{\mathcal{G}}\widetilde{\mathcal{G}}^*\mathcal{P}_{\mathcal{\widetilde{\mathcal{S}}}_l}-p^{-1}\mathcal{P}_{\mathcal{\widetilde{\mathcal{S}}}_l}\widetilde{\mathcal{G}}\mathcal{P}_{\Omega}\widetilde{\mathcal{G}}^*\mathcal{P}_{\mathcal{\widetilde{\mathcal{S}}}_l}\right\|\\
		\le&\left\|\mathcal{P}_{\mathcal{\widetilde{\mathcal{S}}}}\widetilde{\mathcal{G}}\widetilde{\mathcal{G}}^*\mathcal{P}_{\mathcal{\widetilde{\mathcal{S}}}}-p^{-1}\mathcal{P}_{\mathcal{\widetilde{\mathcal{S}}}}\widetilde{\mathcal{G}}\mathcal{P}_{\Omega}\widetilde{\mathcal{G}}^*\mathcal{P}_{\mathcal{\widetilde{\mathcal{S}}}}\right\|
		+\left\|(\mathcal{P}_\mathcal{\widetilde{\mathcal{S}}}-\mathcal{P}_{\mathcal{\widetilde{\mathcal{S}}}_l})\widetilde{\mathcal{G}}\widetilde{\mathcal{G}}^*\mathcal{P}_{\mathcal{\widetilde{\mathcal{S}}}_l}\right\|\\
		&+\left\|\mathcal{P}_\mathcal{\widetilde{\mathcal{S}}}\widetilde{\mathcal{G}}\widetilde{\mathcal{G}}^*(\mathcal{P}_\mathcal{\widetilde{\mathcal{S}}}-\mathcal{P}_{\mathcal{\widetilde{\mathcal{S}}}_l})\right\|+\left\|p^{-1}(\mathcal{P}_\mathcal{\widetilde{\mathcal{S}}}-\mathcal{P}_{\mathcal{\widetilde{\mathcal{S}}}_l})\widetilde{\mathcal{G}}\mathcal{P}_{\Omega}\widetilde{\mathcal{G}}^*\mathcal{P}_{\mathcal{\widetilde{\mathcal{S}}}_l}\right\|
		+\left\|p^{-1}\mathcal{P}_\mathcal{\widetilde{\mathcal{S}}}\mathcal{\widetilde{G}}\mathcal{P}_{\Omega}\widetilde{\mathcal{G}}^*(\mathcal{P}_\mathcal{\widetilde{\mathcal{S}}}-\mathcal{P}_{\mathcal{\widetilde{\mathcal{S}}}_l})\right\|\\
		\stackrel{(c)}{\le}&\varepsilon_0+\frac{2\|\widetilde{{\bfL}}_l-\widetilde{\mathcal{G}}{\bfY}\|_F}{\sigma_{\min}(\widetilde{\mathcal{G}}{\bfY})}\left(2+p^{-1}\|\mathcal{P}_{\Omega}\mathcal{\widetilde{G}}^*\mathcal{P}_{\mathcal{\widetilde{\mathcal{S}}}_l}\|+p^{-1}\|\mathcal{P}_\mathcal{\widetilde{\mathcal{S}}}\mathcal{\widetilde{G}}\mathcal{P}_{\Omega}\|   \right)\\
		\stackrel{(d)}{\le}& 4\varepsilon_0. 
		\end{split}
		\end{align}
		\normalsize
	where (c) comes from {Lemma \ref{sub1}}, and (d) comes from (\ref{c4}) and (\ref{sc1}). Then, $I_3$ can be bounded as
	\begin{equation}
	\vspace{-1mm}
	I_3\le 4\varepsilon_0\|\widetilde{{\bfW}}_{l-1}-\mathcal{\widetilde{G}}{\bfY}\|_F. 
	\end{equation}
	By putting pieces together, we have
	\begin{equation}\label{p_1}
	\|\widetilde{{\bfW}}_{l}-\mathcal{\widetilde{G}}{\bfY}\|_F\le\nu\|\widetilde{{\bfW}}_{l-1}-\mathcal{\widetilde{G}}{\bfY}\|_F.
	\end{equation} 
	Hence, (\ref{c4}) still holds for $k=l$.\\
	{\bf Base Case}: Let us assume
	\begin{equation}\label{initial}
	\|\widetilde{{\bfL}}_{0}-\mathcal{\widetilde{G}}{\bfY}\|_F\le\displaystyle\frac{p^{1/2}\varepsilon_0}{6\log(n)(1+\varepsilon_0)}.
	\end{equation}
	Then, similar to $I_1+I_2+I_3+I_4$ in \eqref{I_4}, we have
	\small
	\begin{equation*}
	\begin{split}
	&\left\|\widetilde{{\bfW}}_{0}-\mathcal{\widetilde{G}}{\bfY}\right\|_F
	=\left\|\mathcal{P}_{\mathcal{\widetilde{\mathcal{S}}}_0}\mathcal{\widetilde{G}}({\bfY}_0+p^{-1}\mathcal{P}_{\Omega}({\bfY}-{\bfY}_0))-\mathcal{\widetilde{G}}{\bfY}\right\|_F\\
	\le& \left\|(\mathcal{I}-\mathcal{P}_{\mathcal{\widetilde{\mathcal{S}}}_0})(\widetilde{{\bfL}}_0-\widetilde{\mathcal{G}}{\bfY})\right\|_F+\left\|\mathcal{P}_{\mathcal{\widetilde{\mathcal{S}}}_0}\mathcal{\widetilde{G}}\mathcal{\widetilde{G}}^*(\mathcal{I}-\mathcal{P}_{\mathcal{\widetilde{\mathcal{S}}}_0})(\widetilde{{\bfL}}_0-\widetilde{\mathcal{G}}{\bfY})\right\|_F\\
	&+\left\|(\mathcal{P}_{\mathcal{\widetilde{\mathcal{S}}}_0}\mathcal{\widetilde{G}}\mathcal{\widetilde{G}}^*\mathcal{P}_{\mathcal{\widetilde{\mathcal{S}}}_0}-p^{-1}\mathcal{P}_{\mathcal{\widetilde{\mathcal{S}}}_0}\mathcal{\widetilde{G}}\mathcal{P}_{\Omega}\mathcal{\widetilde{G}}^*\mathcal{P}_{\mathcal{\widetilde{\mathcal{S}}}_0})(\widetilde{{\bfL}}_0-\mathcal{\widetilde{G}}{\bfY})\right\|_F\\
	&+p^{-1}\left\|\mathcal{P}_{\mathcal{\widetilde{\mathcal{S}}}_0}\mathcal{\widetilde{G}}\mathcal{P}_{\Omega}\mathcal{\widetilde{G}}^*(\mathcal{I}-\mathcal{P}_{\mathcal{\widetilde{\mathcal{S}}}_0})(\widetilde{{\bfL}}_0-\widetilde{\mathcal{G}}{\bfY})\right\|_F\\
	\le& 5\varepsilon_0\|\widetilde{\bfL}_0-\widetilde{\mathcal{G}}\bfY\|_F.	
	\end{split}
	\end{equation*}
	\normalsize
	Since $\varepsilon_0\in (0, 1/10)$, we have
	\vspace{-3mm} $$\|\widetilde{{\bfW}}_{0}-\mathcal{\widetilde{G}}{\bfY}\|_F\le\frac{1}{2}\|\widetilde{{\bfL}}_{0}-\mathcal{\widetilde{G}}{\bfY}\|_F\le\displaystyle\frac{p^{1/2}\varepsilon_0}{12\log(n)(1+\varepsilon_0)},$$
    which completes the induction part.
    
    Then the only thing is to check the assumption (\ref{initial}). Using Lemma \ref{condition3} and  $\|\widetilde{{\bfL}}_0-\mathcal{\widetilde{G}}{\bfY}\|_F\le\sqrt{2r}\|\widetilde{{\bfL}}_0-\mathcal{\widetilde{G}}{\bfY}\|$, with probability at least $1-n_cn^{-2}$,
    \small
    \begin{equation*}
    \begin{split}
    \frac{\|\widetilde{{\bfL}}_0-\mathcal{\widetilde{G}}{\bfY}\|_F}{\sigma_{\min}(\mathcal{\widetilde{G}}{\bfY})}
    &\le\frac{\sqrt{2r}\|\widetilde{{\bfL}}_0-\mathcal{\widetilde{G}}{\bfY}\|}{\sigma_{\min}(\mathcal{\widetilde{G}}{\bfY})}
    =\kappa\sqrt{\frac{128\mu c_sr^2\log(n)}{m}}.
    \end{split}
    \end{equation*}
    \normalsize
    Therefore, to guarantee (\ref{initial}), we need
    \begin{equation}\label{eqn:t1.bound}
    \kappa\sqrt{\frac{128\mu c_sr^2\log(n)}{m}}\le \frac{p^{1/2}\varepsilon_0}{6\log(n)(1+\varepsilon_0)},
    \end{equation}
    That is  $m\ge C_1(1+\varepsilon_0)\varepsilon_0^{-1}n_c^{1/2}\mu^{1/2}c_s^{1/2}\kappa rn^{1/2}\log^{3/2}(n)$ with $C_1=48\sqrt{2}$.\\
    Hence, with probability at least $1-(2l+1)n_cn^{-2}$, from \eqref{p_1},
	\begin{equation}\label{eqn:t1.linear_convergence}
	\begin{split}
	\left\|{\bfY}_l-{\bfY}\right\|_F 
	&=\|\mathcal{\widetilde{G}}^*(\widetilde{{\bfL}}_l-\mathcal{\widetilde{G}}{\bfY})\|_F
	\le\|\widetilde{{\bfL}}_l-\mathcal{\widetilde{G}}{\bfY}\|_F\\
	&\le2\|\widetilde{{\bfW}}_{l-1}-\mathcal{\widetilde{G}}{\bfY}\|_F\le2\nu^{l-1}
	\|\widetilde{{\bfW}}_0-\mathcal{\widetilde{G}}{\bfY}\|_F \\
	&\le \nu^{l-1}
	\|\widetilde{{\bfL}}_0-\mathcal{\widetilde{G}}{\bfY}\|_F,
	\end{split} 
	\end{equation}
	where ${\bfY}_l=\widetilde{\mathcal{G}}^*\widetilde{{\bfL}}_l,\mathcal{\widetilde{G}}^*\mathcal{\widetilde{G}}=\mathcal{I}$ and $\|\mathcal{\widetilde{G}}^*\|\le 1$.
\end{proof}
\section{Proof of Theorem \ref{t4}}\label{sec: c1p2}
First, we extend the eigenvalues and eigenvectors of linear operators on vector spaces to the eigenvalues and eigenmatrices of linear operators on matrix spaces, defined as follows. 
\begin{defi} \label{defi: eigenmatrix}
	Let $\mathcal{A}$ denote a linear operator from $\mathbb{C}^{l_1\times l_2}$ to $\mathbb{C}^{l_1\times l_2}$,
	for any matrix $\bfM$ in the space and $\bfM\neq \mathbf{0}$, if $\mathcal{A}\bfM=\lambda \bfM$ holds, then $\bfM$ is one eigenmatrix of operator $\mathcal{A}$, and $\lambda$ is the corresponding eigenvalue. 
\end{defi}
Let $\mathcal{L}$ denote the following linear operator on the matrix space $\mathbb{C}^{n_cn_1\times n_cn_2}$, $$\mathcal{L}=\mathcal{P}_{\widetilde{\mathcal{S}}}\widetilde{\mathcal{G}}\widetilde{\mathcal{G}}^*\mathcal{P}_{\widetilde{\mathcal{S}}}-p^{-1}\mathcal{P}_{\widetilde{\mathcal{S}}}\widetilde{\mathcal{G}}\mathcal{P}_\Omega\widetilde{\mathcal{G}}^*\mathcal{P}_{\widetilde{\mathcal{S}}}.$$ 
We first introduce Lemmas \ref{lemma: linear approximation of update rule}, \ref{lemma: real eigenvalues} and \ref{lemma: convergence rate with eigenvalue} that are useful in the proof of Theorem \ref{t4}. 

\begin{lemma}[\cite{ZHWC18}, Lemma 7]
\label{lemma: linear approximation of update rule}
Suppose that for any  $\epsilon>0$,  there always exists an integer $s_\epsilon$ such that for any integer $k\geq 0$,  the iterates $\widetilde{\bfW}_{s_\epsilon+k}$ generated by AM-FIHT satisfy $\|\widetilde{\bfW}_{s_\epsilon+k}-\widetilde{\mathcal{G}}\bfY\|_F\leq \epsilon$. Then	with any $l>s_\epsilon+1$, the updated rule can be denoted as
	\begin{equation*}
	\begin{split}
	& \begin{bmatrix}
	\widetilde{\bfW}_l-\widetilde{\mathcal{G}}{\bfY}\\
	\widetilde{\bfW}_{l-1}-\widetilde{\mathcal{G}}{\bfY}
	\end{bmatrix}
	= \begin{bmatrix}
	\mathcal{L}(\widetilde{\bfW}_{l-1}-\widetilde{\mathcal{G}}{\bfY}) +\beta \mathcal{P}_{\widetilde{\mathcal{S}}} (\widetilde{\bfW}_{l-1}-\widetilde{\bfW}_{l-2})\\
	\widetilde{\bfW}_{l-1}-\widetilde{\mathcal{G}}{\bfY} 
	\end{bmatrix}+\widetilde{\bfZ}_{l-1},
	\end{split}
	\end{equation*}
	\normalsize
	where 
	\small
	\begin{equation*}
	\left\|\widetilde{\bfZ}_{l-1}\right\|_F=o\left(
	\left\|
	\begin{bmatrix}
	\widetilde{\bfW}_{l-1}-\widetilde{\mathcal{G}}{\bfY}\\
	\widetilde{\bfW}_{l-2}-\widetilde{\mathcal{G}}{\bfY}
	\end{bmatrix}
	\right\|_F
	\right).
	\end{equation*}
\end{lemma}
\begin{lemma}[\cite{ZHWC18}, Lemma 8]
\label{lemma: real eigenvalues}
	All the eigenvalues of operator $\mathcal{L}$
	are real numbers.
\end{lemma}
\begin{lemma}[\cite{P87}, Lemma 2.1]\label{lemma: convergence rate with eigenvalue}
	Let $\mathcal{A}$ be a linear operator from $\mathbb{C}^{l_1\times l_2}$ to $\mathbb{C}^{l_1\times l_2}$, and let $\lambda_1,\cdots,\lambda_n$ be its eigenvalues, let $\rho(\mathcal{A})=\max_{1\leq i\leq n}|\lambda_i|$, if $\rho(\mathcal{A})<1$, then there exists some constant $c(\delta)$ such that $\|\mathcal{A}^k\|\leq c(\delta)(\rho(\mathcal{A})+\delta)^k$ holds for all integers $k$, where $0< \delta<1-\rho(\mathcal{A}).$
\end{lemma}

\begin{proof}[Proof of Theorem \ref{t4}]	
	First, we claim that AM-FIHT is still convergent with a small $\beta\in (0,1).$
	Based on the proof of Theorem \ref{t1}, 
	\vspace{-1mm}
	$$\| \widetilde{\bfW}_l-\widetilde{\mathcal{G}}{\bfY}\| _F\leq \nu \| \widetilde{\bfW}_{l-1}-\widetilde{\mathcal{G}}{\bfY}\| _F,$$ 
	\normalsize
	where $\nu=6\varepsilon_0, 0<\varepsilon_0< {1}/{10}$. 	
	A loose bound from a direct derivation with $\beta\neq 0$ is 
	\begin{equation*}
	\|\widetilde{\bfW}_l-\widetilde{\mathcal{G}}{\bfY}\|_F\leq (\nu+\beta) \|\widetilde{\bfW}_{l-1}-\widetilde{\mathcal{G}}{\bfY}\|_F+\beta\|\widetilde{\bfW}_{l-2}-\widetilde{\mathcal{G}}{\bfY}\|_F. 
	\end{equation*}
	\normalsize
	Thus if $\nu+2\beta<1,$ i.e., $\beta\in(0,\frac{1}{5}),$ the iteration is still convergent. Thus for any $\epsilon>0$, we can always find such an $l$ that $\|\widetilde{\bfW}_{l-2+k}-\widetilde{\mathcal{G}}\bfY\|_F\leq \epsilon, \forall k\geq0.$
	Then following {Lemma \ref{lemma: linear approximation of update rule}}, 
	if we ignore $\widetilde{\bfZ}_{l-1}$,
	then 
	\begin{equation*}
	\widetilde{\bfW}_l-\widetilde{\mathcal{G}}{\bfY}=\mathcal{L}(\widetilde{\bfW}_{l-1}-\widetilde{\mathcal{G}}{\bfY}) +\beta \mathcal{P}_{\widetilde{\mathcal{S}}} (\widetilde{\bfW}_{l-1}-\widetilde{\bfW}_{l-2}).
	\end{equation*}
	Thus $\mathcal{P}_{\widetilde{\mathcal{S}}}(\widetilde{\bfW}_l-\widetilde{\mathcal{G}}{\bfY})=\widetilde{\bfW}_l-\widetilde{\mathcal{G}}{\bfY}$.
	With $\mathcal{P}_{\widetilde{\mathcal{S}}}(\widetilde{\mathcal{G}}{\bfY})=\widetilde{\mathcal{G}}{\bfY},$ we have $\mathcal{P}_{\widetilde{\mathcal{S}}}(\widetilde{\bfW}_l)=\widetilde{\bfW}_l.$
	The update rule of AM-FIHT can be further simplified as 
	\small
	\begin{equation*}
	\begin{split}
	\begin{bmatrix}
	\widetilde{\bfW}_l-\widetilde{\mathcal{G}}{\bfY}\\
	\widetilde{\bfW}_{l-1}-\widetilde{\mathcal{G}}{\bfY}
	\end{bmatrix}
	=& \begin{bmatrix}
	\mathcal{L}(\widetilde{\bfW}_{l-1}-\widetilde{\mathcal{G}}{\bfY}) +\beta(\widetilde{\bfW}_{l-1}-\widetilde{\bfW}_{l-2})\\
	\widetilde{\bfW}_{l-1}-\widetilde{\mathcal{G}}{\bfY} 
	\end{bmatrix}
	:=
	\widetilde{\mathcal{L}}\begin{bmatrix}
	\widetilde{\bfW}_{l-1}-\widetilde{\mathcal{G}}{\bfY}\\
	\widetilde{\bfW}_{l-2}-\widetilde{\mathcal{G}}{\bfY}
	\end{bmatrix}.
	\end{split}
	\end{equation*}
	\normalsize
	Following Lemma \ref{condition2}, we have  $\left\|\mathcal{L} \right\|<1$, if $m>32\mu c_sr\log(n)$. Based on the definitions of $\rho(\mathcal{L})$ and $\|\mathcal{L}\|$, we have $\rho(\mathcal{L})\leq \|\mathcal{L}\|<1$.
	Our aim is to prove $\rho(\widetilde{\mathcal{L}})<\rho(\mathcal{L})$. 
	Let $\lambda$ denote one nonzero eigenvalue of $\widetilde{\mathcal{L}}$, the corresponding eigenmatrix is 
	$
	\begin{bmatrix}
	\bfM_1\\
	\bfM_2
	\end{bmatrix}
	$, then
	\begin{equation*}
	\widetilde{\mathcal{L}}
	\begin{bmatrix}
	{\bfM}_1\\
	{\bfM}_2\\
	\end{bmatrix}
	=\begin{bmatrix}
	\mathcal{L}\bfM_1 +\beta (\bfM_1-\bfM_2)\\
	\bfM_1 	\end{bmatrix}
	=\lambda 
	\begin{bmatrix}
	{\bfM}_1\\
	{\bfM}_2\\
	\end{bmatrix},
	\end{equation*}
	\normalsize
	we have ${\bfM}_1=\lambda {\bfM}_2, \mathcal{L}({\bfM}_1)+\beta {\bfM}_1-\beta {\bfM}_2=\lambda {\bfM}_1$.  Therefore, $\lambda \mathcal{L}({\bfM}_2)+\lambda \beta {\bfM}_2-\beta {\bfM}_2=\lambda^2 {\bfM}_2$. With $\lambda\neq 0$, 
	\begin{equation*}
	\mathcal{L}(\bfM_2)=(\lambda-\beta+\beta/\lambda)\bfM_2 :=\eta_i \bfM_2,
	\end{equation*}
	thus ${\bfM}_2$ is an eigenmatrix of operator $\mathcal{L},$ with the corresponding eigenvalue as $\eta_i$. Lemma \ref{lemma: real eigenvalues} shows $\eta_i\in \mathbb{R}$, then we have
	\begin{equation*}
	\lambda^2-\eta_i \lambda-\beta \lambda+\beta=0,
	\end{equation*}
	\begin{equation*}
	\begin{split}
	& \lambda_{i1}=\frac{\eta_i+\beta+ \sqrt{(\eta_i+\beta)^2-4\beta}}{2},\\
	& \lambda_{i2}=\frac{\eta_i+\beta- \sqrt{(\eta_i+\beta)^2-4\beta}}{2}.
	\end{split}
	\end{equation*}
	Here we analyze in two cases:
	\begin{enumerate}
		\item for any $\eta_i$ that satisfies $(\eta_i+\beta)^2-4\beta\leq 0$, the modulus $|\lambda_{i1}|=|\lambda_{i2}|=\sqrt{\beta}$. 
		\item for any $\eta_i$ that satisfies $(\eta_i+\beta)^2-4\beta>0,\eta_i$ cannot be zero for any $\beta\in(0,1)$.   
		 With $\rho(\mathcal{L})=\max_i |\eta_i|<1$, it holds that $\eta_i<1$. In this case, with $\beta\in (0,1)$, we have
		\begin{equation*}
		\begin{split}
		(\eta_i+\beta)^2-4\beta =(\eta_i-\beta)^2-4(1-\eta_i)\beta<(\eta_i-\beta)^2,
		\end{split}
		\end{equation*}
		\begin{equation*}
		\begin{split}
		&  \max\{|\lambda_{i1}|,|\lambda_{i2}|\}= \frac{|\eta_i+\beta|+ \sqrt{(\eta_i+\beta)^2-4\beta}}{2}\\
		&< \frac{|\eta_i+\beta|+|\eta_i-\beta|}{2} = \max\{|\eta_i|, \beta\} \leq \max\{|\eta_i|, \sqrt{\beta}\}.
		\end{split}
		\end{equation*}
	\end{enumerate}
	Combining the two cases, if we choose a positive $\beta$ that satisfies $\beta<\left(\max_i\{|\eta_i|\}\right)^2=\rho^2(\mathcal{L})$, let
	\begin{equation}\label{eqn: expression of q}
	q(0)=\rho(\mathcal{L}), q(\beta)=\rho(\widetilde{\mathcal{L}}), \tau= \min\left\{1/5, \rho^2(\mathcal{L})\right\},
	\end{equation}
	then we have $q(0)>q(\beta), \forall \beta\in (0,\tau).$
\end{proof}

\section{Proof of Theorem \ref{t2}}\label{sec: c1p3}
Lemma \ref{t2c1} derives the properties of {\small$\widehat{p}^{-1}\mathcal{P}_{\widehat{\mathcal{S}}_l}\widetilde{\mathcal{G}}{\mathcal{P}}_{\Omega_{l+1}}\widetilde{\mathcal{G}}^*(\mathcal{{P}}_{\widetilde{{\bfU}}}-\mathcal{{P}}_{\widehat{{\bfU}}_l})$}, and the random operator can be close enough to its mean {\small $\mathcal{P}_{\widehat{\mathcal{S}}_l}\widetilde{\mathcal{G}}\widetilde{\mathcal{G}}^*(\mathcal{{P}}_{\widetilde{{\bfU}}}-\mathcal{{P}}_{\widehat{{\bfU}}_l})$} with a significant large amount of observed entries. Lemma \ref{t2c2} illustrates the relation between $\widetilde{\bfL}_l$ and $\widetilde{\bfL}^{\prime}_l$ and gives the bound on the incoherence of $\widetilde{\bfL}^{\prime}_l$, which is obtained after the trimming part (line $5$ to $10$).  Lemmas \ref{t2c1} and \ref{t2c2} are built upon Lemmas 9 and 10 in \cite{CWW17} by extending from single-channel signals to multi-channel signals. Similar to the proof of Theorem 1, the proof of Theorem 3 is built upon that of Lemma 3 in \cite{CWW17}, which is originally proposed as an initialization strategy. The major steps are devoted to bounding $I_5, I_6$ and $I_7$ in \eqref{I_567}, and the corresponding results are presented in \eqref{eqn:I_5}, \eqref{eqn:I_6}, and \eqref{eqn:I_7}. We include some details for the completeness of this proof. 

\begin{lemma}[\cite{ZHWC18}, Lemma 10]
\label{t2c1}
	Let  $\widetilde{{\bfL}}^{\prime}_l=\widehat{{\bfU}}_l\widehat{{\bf\Sigma}}_l\widehat{{\bfV}}_l^*$ and $\widetilde{\mathcal{G}}\bfY=\widetilde{{\bfU}}\widetilde{{\bf\Sigma}}\widetilde{{\bfV}}^*$ be the SVD of $\widetilde{{\bfL}}_l^{\prime}$ and $\widetilde{\mathcal{G}}\bfY$. Further let $\widehat{\mathcal{S}}_l$ be the tangent subspace of $\widetilde{{\bfL}}^{\prime}_l$. Assume there exists a constant $\mu$ such that
	\vspace{-2mm}
	\begin{equation*}
	\vspace{-1mm}
	\left\| \mathcal{P}_{\widehat{{\bfU}}_l}\widetilde{{\bfH}}_{k,t}\right\|_F^2\le \frac{\mu c_sr}{n_cn}, \quad \left\| \mathcal{P}_{\widehat{{\bfV}}_l}\widetilde{{\bfH}}_{k,t}\right\|_F^2\le \frac{\mu c_sr}{n_cn},
	\end{equation*}
	and
	\vspace{-1mm}
	\begin{equation*}
	\left\| \mathcal{P}_{\widetilde{{\bfU}}}\widetilde{{\bfH}}_{k,t}\right\|_F^2\le \frac{\mu c_sr}{n_cn}, \quad \left\| \mathcal{P}_{\widetilde{{\bfV}}}\widetilde{{\bfH}}_{k,t}\right\|_F^2\le \frac{\mu c_sr}{n_cn}.
	\end{equation*}
	for all $1\le t \le n, 1\le k \le n_c$. Let $\Omega_{l+1}=\{(k_a,t_a)|a=1,\cdots,\widehat{m}\}$ be a set of indices sampled with replacement. If ${\mathcal{P}}_{\Omega_{l+1}}$ is independent of $\widetilde{{\bfU}}$, $\widetilde{{\bfV}}$, $\widehat{{\bfU}}_l$ and $\widehat{{\bfV}}_l$, then
	\vspace{-2mm}
	\begin{equation*}
	\left\|\mathcal{P}_{\widehat{\mathcal{S}}_l}\widetilde{\mathcal{G}}({\mathcal{I}}-\widehat{p}^{-1}{\mathcal{P}}_{\Omega_{l+1}})\widetilde{\mathcal{G}}^*(\mathcal{{P}}_{\widetilde{{\bfU}}}-\mathcal{{P}}_{\widehat{{\bfU}}_l})\right\|\le \sqrt{\frac{160\mu c_s r \log(n)}{\widehat{m}}}
	\end{equation*}
	with probability at least $1-n_cn^{-2}$, if $\widehat{m}\ge\frac{125}{18}\mu c_s r \log(n)$.
\end{lemma} 
\begin{lemma}[\cite{ZHWC18}, Lemma 11]
\label{t2c2}
	Let $\widetilde{{\bfL}}_l=\widetilde{{\bfU}}_l\widetilde{{\bf\Sigma}}_l\widetilde{{\bfV}}_l^*$ and $\widetilde{\mathcal{G}}{\bfY}=\widetilde{{\bfU}}\widetilde{{\bf\Sigma}}\widetilde{{\bfV}}^*$ be the SVD of $\widetilde{{\bfL}}_l$ and $\widetilde{\mathcal{G}}\bfY$. Assume 
	$$\max_{k_1}||\widetilde{{\bfU}}_{k_1*}||^2\le \frac{\mu c_s r}{n_cn} \  \text{and}\  \max_{k_2}||\widetilde{{\bfV}}_{k_2*}||^2\le 
	\frac{\mu c_s r}{n_cn}.$$	
	Suppose $\widetilde{{\bfL}}_l$ and $\widetilde{\mathcal{G}}{\bfY}$ are both rank-$r$ matrices satisfying
	\vspace{-2mm}
	\begin{equation*}
	\vspace{-1mm}
	\|\widetilde{{\bfL}}_l-\widetilde{\mathcal{G}}{\bfY}\|_F\le\frac{\sigma_{\min}(\widetilde{\mathcal{G}}{\bfY})}{10\sqrt{2}}. 
	\end{equation*}
	Then the matrix $\widetilde{{\bfL}}^{\prime}_l=\widehat{{\bfU}}_l\widehat{{\bf\Sigma}}_l\widehat{{\bfV}}_l^*$, denoting the SVD of $\widetilde{{\bfL}}_l^{\prime}$, that is obtained after trimming in RAM-FIHT satisfies 
	\vspace{-1mm}
	\begin{equation}
	\label{eqn: t2c2_1}
	\begin{split}
	\| \widetilde{{\bfL}}^{\prime}-\widetilde{\mathcal{G}}{\bfY}\|_F\le 8\kappa\|\widetilde{{\bfL}}_l-\widetilde{\mathcal{G}}{\bfY}\|_F,   
	\end{split} 
	\end{equation}
	\begin{equation}\label{eqn: t2c2_2}
	\max_{k_1,k_2}\bigg\{\|\widehat{{\bfU}}_{k_1*}\|^2,\|\widehat{{\bfV}}_{k_2*} \|^2  \bigg\}\le\frac{100\mu c_s r}{81n_cn},  
	\end{equation}
	\normalsize
	where $\kappa$ denotes the condition number of $\widetilde{\mathcal{G}}{\bfY}$.
\end{lemma}

\begin{proof}[Proof of Theorem \ref{t2}]
	First, we show the following inequality holds with high probability by mathematical induction.
	\begin{equation}\label{3.1}
	\|\widetilde{{\bfL}}_k-\widetilde{\mathcal{G}}{\bfY}\|_F\le\frac{\varepsilon_0\sigma_{\min}(\widetilde{\mathcal{G}}{\bfY})}{128\kappa^2}. 
	\end{equation}
	{\bf Inductive Step:} Suppose (\ref{3.1}) holds when $k=l$ and $l\ge0$. Then \eqref{eqn: t2c2_2} in Lemma \ref{t2c2} holds.  
	Further, we can conclude
	\small
	\begin{equation}\label{3.3}
	\left\|\mathcal{P}_{\widehat{{\bfU}}_l}\widetilde{{\bfH}}_{k,t} \right\|_F^2 \le \frac{100\mu c_s r}{81n_cn}\ 
	\normalsize \text{and}  
	\ \small \left\|\mathcal{P}_{\widehat{{\bfV}}_l}\widetilde{{\bfH}}_{k,t} \right\|_F^2 \le \frac{100\mu c_s r}{81n_cn}.
	\end{equation}
	\normalsize
	Recall that ${\bfY}=\widetilde{\mathcal{D}}{\bfX}$ and $\widetilde{\mathcal{G}}{\bfY}=\widetilde{\mathcal{H}}{\bf X}$. Define $\widehat{{\bfY}}_l=\widetilde{\mathcal{D}}\widehat{{\bfX}}_l$. Since the measurements are noiseless, then $\bfM=\bfX$ and
	\small
	\begin{equation*}
	\begin{split}
	\widetilde{\mathcal{H}}(\widehat{{\bfX}}_l+\widehat{p}^{-1}\widetilde{\mathcal{P}}_{\Omega_{l+1}}({\bfX}-\widehat{{\bfX}}_l)) =
	\widetilde{\mathcal{G}}(\widehat{{\bfY}}_l+\widehat{p}^{-1}\widetilde{\mathcal{P}}_{\Omega_{l+1}}({\bfY}-\widehat{{\bfY}}_l)).
	\end{split}
	\end{equation*}
	\normalsize
	Then,
	\small
	\begin{equation}\label{I_567}
	\begin{split}
	&\left\|\widetilde{{\bfL}}_{l+1}-\widetilde{\mathcal{G}}{\bfY} \right\|_F\\
	\le& 2\left\|\mathcal{P}_{\widehat{\mathcal{S}}_l}\widetilde{\mathcal{G}}(\widehat{{\bfY}}_l+\widehat{p}^{-1}{\mathcal{P}}_{\Omega_{l+1}}({\bfY}-\widehat{{\bfY}}_l))-\widetilde{\mathcal{G}}{\bfY} \right\|_F \\
	\le& 2\left\|\mathcal{P}_{\widehat{\mathcal{S}}_l}\widetilde{\mathcal{G}}{\bfY}-\widetilde{\mathcal{G}}{\bfY} \right\|_F
	+2\left\|\big(\mathcal{P}_{\widehat{\mathcal{S}}_l}\widetilde{\mathcal{G}}-\widehat{p}^{-1}\mathcal{P}_{\widehat{\mathcal{S}}_l}\widetilde{\mathcal{G}}{\mathcal{P}}_{\Omega_{l+1}}\big)({\bfY}-\widehat{{\bfY}}_l) \right\|_F  \\
	=&2\left\|\big({\mathcal{I}}-\mathcal{P}_{\widehat{\mathcal{S}}_l}\big)\widetilde{\mathcal{G}}{\bfY} \right\|_F
	+2\left\|\big(\mathcal{P}_{\widehat{\mathcal{S}}_l}\widetilde{\mathcal{G}}\widetilde{\mathcal{G}}^*-\widehat{p}^{-1}\mathcal{P}_{\widehat{\mathcal{S}}_l}\widetilde{\mathcal{G}}{\mathcal{P}}_{\Omega_{l+1}}\widetilde{\mathcal{G}}^*\big)(\widetilde{{\bfL}}^{\prime}_l-\widetilde{\mathcal{G}}{\bfY}) \right\|_F  \\
	\le&2\left\|\big({\mathcal{I}}-\mathcal{P}_{\widehat{\mathcal{S}}_l}\big)(\widetilde{{\bfL}}^{\prime}_l-\widetilde{\mathcal{G}}{\bfY}) \right\|_F
	+2\left\|\big(\mathcal{P}_{\widehat{\mathcal{S}}_l}\widetilde{\mathcal{G}}\widetilde{\mathcal{G}}^*\mathcal{P}_{\widehat{\mathcal{S}}_l}-\widehat{p}^{-1}\mathcal{P}_{\widehat{\mathcal{S}}_l}\widetilde{\mathcal{G}}{\mathcal{P}}_{\Omega_{l+1}}\widetilde{\mathcal{G}}^*\mathcal{P}_{\widehat{\mathcal{S}}_l}\big)(\widetilde{{\bfL}}^{\prime}_l-\widetilde{\mathcal{G}}{\bfY})  \right\|_F\\
	&+2\left\|\mathcal{P}_{\widehat{\mathcal{S}}_l}\widetilde{\mathcal{G}}({\mathcal{I}}-\widehat{p}^{-1}{\mathcal{P}}_{\Omega_{l+1}})\widetilde{\mathcal{G}}^*({\mathcal{I}}-\mathcal{P}_{\widehat{\mathcal{S}}_l})(\widetilde{{\bfL}}^{\prime}_l-\widetilde{\mathcal{G}}{\bfY})\right\|_F\\ 
	:=&I_5+I_6+I_7.
	\end{split}
	\end{equation}
	where the first inequality comes from (\ref{W_L}).\\
	With (\ref{eqn: t2c2_1}) and (\ref{3.1}), $I_5$ can be bounded as
	\vspace{-1mm}
	\begin{equation}\label{eqn:I_5}
	\vspace{-1mm}
	I_5\le\frac{2\|\widetilde{{\bfL}}^{\prime}_l-\widetilde{\mathcal{G}}{\bfY} \|_F^2 }{\sigma_{\min}(\widetilde{\mathcal{G}}{\bfY})}\le \varepsilon_0\|\widetilde{{\bfL}}_l-\widetilde{\mathcal{G}}{\bfY} \|_F,
	\end{equation}
	As for the item $I_6$, Lemma \ref{condition2} along with (\ref{eqn: t2c2_1}) suggests
	\begin{equation}\label{eqn:I_6}
	\begin{split}
	I_6
	&\le 2\sqrt{\frac{3200\mu c_s r\log(n)}{81\widehat{m}}}\| \widetilde{{\bfL}}^{\prime}_l-\widetilde{\mathcal{G}}{\bfY}\|_F
	\le 16\kappa\sqrt{\frac{3200\mu c_s r\log(n)}{81\widehat{m}}}\| \widetilde{{\bfL}}_l-\widetilde{\mathcal{G}}{\bfY}\|_F 
	\end{split}
	\end{equation}
	with probability at least $1-n_cn^{-2}$. 
	To bound $I_7$,
	\small	
	\begin{equation*}
	\begin{split}
	&\Big({\mathcal{I}}-\mathcal{P}_{\widehat{\mathcal{S}}_l}\Big)\Big(\widetilde{{\bfL}}^{\prime}_l-\widetilde{\mathcal{G}}{\bfY}\Big)
	=\Big({\bfI}-\widehat{{\bfU}}_l\widehat{{\bfU}}_l^*\Big)(-\widetilde{\mathcal{G}}{\bfY})\Big({\bfI}-\widehat{{\bfV}}_l\widehat{{\bfV}}_l^*\Big)\\
	=&\Big(\widetilde{{\bfU}}\widetilde{{\bfU}}^*-\widehat{{\bfU}}_l\widehat{{\bfU}}_l^*\Big)\Big(\widetilde{{\bfL}}^{\prime}_l-\mathcal{\widetilde{G}}{\bfY}\Big)\Big({\bfI}-\widehat{{\bfV}}_l\widehat{{\bfV}}_l^*\Big)\\
	=&\Big(\mathcal{P}_{\widetilde{{\bfU}}}-\mathcal{P}_{\widehat{{\bfU}}_l}\Big)({\mathcal{I}}-\mathcal{P}_{\widetilde{{\bfV}}})\Big(\widetilde{{\bfL}}^{\prime}_l-\mathcal{\widetilde{G}}{\bfY}\Big).
	\end{split}
	\end{equation*}
	\normalsize
Hence, by {Lemma \ref{t2c1}}, with probability at least $1-n_cn^{-2}$,
\footnotesize
\begin{equation}\label{eqn:I_7}
	\begin{split}
	I_7
	&=2\left\|\mathcal{P}_{\widehat{\mathcal{S}}_l}\widetilde{\mathcal{G}}({\mathcal{I}}-\widehat{p}^{-1}{\mathcal{P}}_{\Omega_{l+1}})\widetilde{\mathcal{G}}^* \Big(\mathcal{P}_{\widetilde{{\bfU}}}-\mathcal{P}_{\widehat{{\bfU}}_l}\Big)({\mathcal{I}}-\mathcal{P}_{\widetilde{{\bfV}}})\Big(\widetilde{{\bfL}}^{\prime}_l-\mathcal{\widetilde{G}}{\bfY}\Big)\right\|_F \\
	&\le 2\left\|\mathcal{P}_{\widehat{\mathcal{S}}_l}\widetilde{\mathcal{G}}({\mathcal{I}}-\widehat{p}^{-1}{\mathcal{P}}_{\Omega_{l+1}})\widetilde{\mathcal{G}}^* \Big(\mathcal{P}_{\widetilde{{\bfU}}}-\mathcal{P}_{\widehat{{\bfU}}_l}\Big)\right\| \left\|\widetilde{{\bfL}}^{\prime}_l-\mathcal{\widetilde{G}}{\bfY}\right\|_F\\
	&\le 16\kappa \sqrt{\frac{16000\mu c_s r \log (n)}{81\widehat{m}}}\left\|\widetilde{{\bfL}}^{\prime}_l-\mathcal{\widetilde{G}}{\bfY}\right\|_F.
	\end{split}
\end{equation}
\normalsize
Therefore, if $\widehat{m}\ge C_4\varepsilon_0^{-2}\mu c_s\kappa^2r\log(n)$ for some constant $C_4$,
\begin{equation*}
\begin{split}
I_6+I_7\le &326\kappa\sqrt{\frac{\mu c_s r \log (n)}{\widehat{m}}}\|\widetilde{{\bfL}}_l-\widetilde{\mathcal{G}}{\bfY} \|_F
	   \le \varepsilon_0\|\widetilde{{\bfL}}_l-\widetilde{\mathcal{G}}{\bfY} \|_F.
\end{split}
\end{equation*}
\normalsize
Putting pieces together gives 
\vspace{-2mm}
\begin{equation}
\vspace{-1mm}
	\begin{split}
	\| \widetilde{{\bfL}}_{l+1}-\widetilde{\mathcal{G}}{\bfY}\|_F
	&\le 2\varepsilon_0\|\widetilde{{\bfL}}_l-\widetilde{\mathcal{G}}{\bfY} \|_F 
	\end{split}
\end{equation}
	with probability at least $1-2n_cn^{-2}$. Hence, (\ref{3.1}) also holds when $k={l+1}$.\\
	{\bf Base Case:}
	Since $\widetilde{{\bfL}}_0=\mathcal{Q}_r\big(\widehat{p}^{-1}\widetilde{\mathcal{H}}{\mathcal{P}}_{\Omega_0}({\bfX})\big)$, we can follow the same idea in the proof of base case in Theorem \ref{t1}. Thus, when $k=0$, (\ref{3.1}) is valid with probability at least $1-n_cn^{-2}$ provided $\widehat{m}\ge C_5 \varepsilon_0^{-2}\mu c_s \kappa^6 r^2 \log(n)$
	for some constant $C_5$. 	
	
	Let $C_2=\max\{C_4,C_5\}$. If 	$\widehat{m}\ge C_2 \varepsilon_0^{-2}\mu c_s \kappa^6 r^2 \log(n)$, then for each $l\ge0$, we have
	\begin{equation*}
		\begin{split}
		\| \widetilde{{\bfL}}_{l+1}-\widetilde{\mathcal{G}}{\bfY}\|_F\le 2\varepsilon_0\|\widetilde{{\bfL}}_l-\widetilde{\mathcal{G}}{\bfY} \|_F. 
		\end{split}
	\end{equation*}
	with probability at least $1-2n_cn^{-2}$.
	Directly the following inequality is obtained with probability $1-(2L+1)n_cn^{-2}$,
	\begin{equation}
		\begin{split}
		\| \widetilde{{\bfL}}_L-\widetilde{\mathcal{G}}{\bfY}\|_F\le \nu^{L}\|\widetilde{{\bfL}}_0-\widetilde{\mathcal{G}}{\bfY} \|_F\le\nu^L\frac{\varepsilon_0\sigma_{\min}(\widetilde{\mathcal{G}}{\bfY})}{128\kappa^2}. 
		\end{split}
	\end{equation}
	If we take $L=\Big\lceil \varepsilon_0^{-1}\log \Big(\frac{\sigma_{\max}(\mathcal{H}{\bfX})}{128\kappa^3\varepsilon}\Big)\Big\rceil$ with an arbitrarily small positive constant $\varepsilon$,
	since $\sigma_{\max}(\widetilde{\mathcal{G}}{\bfY})=\sqrt{n_c}\sigma_{\max}(\mathcal{H}{\bfX})$,
	\begin{equation}\label{t2.end}
		 \| \widetilde{{\bfL}}_L-\widetilde{\mathcal{G}}{\bfY}\|_F\le n_c^{1/2}\varepsilon,
	\end{equation}
	which completes the proof of Theorem \ref{t2}.
\end{proof}

\section{Proof of Theorem \ref{t3}}\label{sec: c1p4}
The proof of Theorem \ref{t3} is similar to {Theorem \ref{t2}}, except some modification to handle the noise matrix $\bfN$. We first present the following lemma that will be useful in the proof.
\vspace{-1mm}
\begin{lemma}[\cite{ZHWC18}, Lemma 12]
\label{condition3.1}
	Suppose $m\ge 16\log(n)$, then
	\vspace{-2mm}
	\begin{equation*}
	\left\|p^{-1}\mathcal{\widetilde{H}}{\mathcal{P}}_{\Omega}({\bfN})-\widetilde{\mathcal{H}}{\bfN} \right\|\le \displaystyle \sqrt{\frac{16\log(n)}{m}}n_cn\left\| \widetilde{\mathcal{H}}{\bfN} \right\|_{\infty}
	\end{equation*} 
	with probability at least $1-n_cn^{-2}$.
\end{lemma}
\begin{proof}[Proof of Theorem \ref{t3}]
	For the noisy case where $\bfM=\bfX+\bfN$, we have assumed $ \left\|\bfN \right\|_{\infty}\le\displaystyle
	\frac{\varepsilon_0\left\|\mathcal{H} \bfX\right\|}{2048\kappa^3r^{1/2}n_c^{1/2}n}$. Recall \eqref{y_l}, define $\bfS=\widetilde{\mathcal{D}}\bfN$, then  
	\vspace{-2mm}
	\begin{equation*}
	\begin{split}
	   \| \widetilde{\mathcal{G}}\bfS\|_{\infty}
	 =&\| \widetilde{\mathcal{H}}\bfN\|_{\infty}=\left\| \bfN\right\|_{\infty}
	\le\frac{\varepsilon_0 \sigma_{\min}(\widetilde{\mathcal{G}}\bfY)}{2048\kappa^2r^{1/2}n_cn}.
	\end{split}
	\end{equation*}
	Similar to the derivation of (\ref{W_L}), we have
	\small
	\begin{equation*}
	\begin{split}
		   &\|\widetilde{{\bfL}}_{l+1}-\widetilde{\mathcal{G}}{\bfY} \|
		\le 2\left\|\mathcal{P}_{\widehat{\mathcal{S}}_l}\widetilde{\mathcal{G}}(\widehat{{\bfY}}_l+\widehat{p}^{-1}{\mathcal{P}}_{\Omega_{l+1}}({\bfY}+\bfS-\widehat{{\bfY}}_l))-\widetilde{\mathcal{G}}{\bfY} \right\|\\
		\le& 2\left\|\mathcal{P}_{\widehat{\mathcal{S}}_l}\widetilde{\mathcal{G}}(\widehat{{\bfY}}_l+\widehat{p}^{-1}{\mathcal{P}}_{\Omega_{l+1}}({\bfY}-\widehat{{\bfY}}_l))-\widetilde{\mathcal{G}}{\bfY} \right\|\\
		   &+2\|\widehat{p}^{-1}\mathcal{P}_{\widehat{\mathcal{S}}_l}\widetilde{\mathcal{G}} {\mathcal{P}}_{\Omega_{l+1}}(\bfS)\|. 
	\end{split}
	\end{equation*}
	\begin{equation*}
		\begin{split}
		&\|\widetilde{{\bfL}}_{l+1}-\widetilde{\mathcal{G}}{\bfY} \|_F\le\sqrt{2r}\|\widetilde{{\bfL}}_{l+1}-\widetilde{\mathcal{G}}{\bfY} \|\\
		\le& 2\sqrt{2r}\left\|\mathcal{P}_{\widehat{\mathcal{S}}_l}\widetilde{\mathcal{G}}(\widehat{{\bfY}}_l+\widehat{p}^{-1}{\mathcal{P}}_{\Omega_{l+1}}({\bfY}+\bfS-\widehat{{\bfY}}_l))-\widetilde{\mathcal{G}}{\bfY} \right\|\\
		\le& 2\sqrt{2r}\left\|\mathcal{P}_{\widehat{\mathcal{S}}_l}\widetilde{\mathcal{G}}(\widehat{{\bfY}}_l+\widehat{p}^{-1}{\mathcal{P}}_{\Omega_{l+1}}({\bfY}-\widehat{{\bfY}}_l))-\widetilde{\mathcal{G}}{\bfY} \right\|\\
		&+2\sqrt{2r}\|\widehat{p}^{-1}\mathcal{P}_{\widehat{\mathcal{S}}_l}\widetilde{\mathcal{G}} {\mathcal{P}}_{\Omega_{l+1}}(\bfS)\|\\
		\le	&2\sqrt{2r}\left\|\big({\mathcal{I}}-\mathcal{P}_{\widehat{\mathcal{S}}_l}\big)(\widetilde{{\bfL}}^{\prime}_l-\widetilde{\mathcal{G}}{\bfY}) \right\|_F\\
		&+2\sqrt{2r}\left\|\big(\mathcal{P}_{\widehat{\mathcal{S}}_l}\widetilde{\mathcal{G}}\widetilde{\mathcal{G}}^*\mathcal{P}_{\widehat{\mathcal{S}}_l}-\widehat{p}^{-1}\mathcal{P}_{\widehat{\mathcal{S}}_l}\widetilde{\mathcal{G}}{\mathcal{P}}_{\Omega_{l+1}}\widetilde{\mathcal{G}}^*\mathcal{P}_{\widehat{\mathcal{S}}_l}\big)(\widetilde{{\bfL}}^{\prime}_l-\widetilde{\mathcal{G}}{\bfY})  \right\|_F\\
		&+2\sqrt{2r}\left\|\mathcal{P}_{\widehat{\mathcal{S}}_l}\widetilde{\mathcal{G}}({\mathcal{I}}-\widehat{p}^{-1}{\mathcal{P}}_{\Omega_{l+1}})\widetilde{\mathcal{G}}^*({\mathcal{I}}-\mathcal{P}_{\widehat{\mathcal{S}}_l})(\widetilde{{\bfL}}^{\prime}_l-\widetilde{\mathcal{G}}{\bfY})\right\|_F\\		
		&+2\sqrt{2r}\|\widehat{p}^{-1}\mathcal{P}_{\widehat{\mathcal{S}}_l}\widetilde{\mathcal{G}} {\mathcal{P}}_{\Omega_{l+1}}(\bfS)\|
		\end{split}
		\end{equation*}	
		\normalsize
		\vspace{-4.5mm}
		\begin{flalign}\label{t4.I_5}
		:=\sqrt{2r}(I_5+I_6+I_7)+I_9, &&
		\end{flalign}	
	where $I_5+I_6+I_7$ has been defined in \eqref{I_567}.\\
	Similar to the proof of {Theorem \ref{t2}}, we show that the following inequality holds with high probability by induction.
	\vspace{-2mm}
		\begin{equation}\label{4.1}
		\|\widetilde{{\bfL}}_k-\widetilde{\mathcal{G}}{\bfY}\|_F\le\frac{\varepsilon_0\sigma_{\min}(\widetilde{\mathcal{G}}{\bfY})}{128\sqrt{2}\kappa^2r^{1/2}}. 
		\end{equation}
		\normalsize	
		{\bf Inductive Step:} Suppose (\ref{4.1}) holds when $k=l$ and $l\ge0$. 
		By {Lemma \ref{sub1} } and (\ref{3.1}), we have
		\begin{equation*}
		\begin{split}
		\sqrt{2r}I_5
		\le&\frac{2\sqrt{2r}\|\widetilde{{\bfL}}^{\prime}_l-\widetilde{\mathcal{G}}{\bfY} \|_F^2 }{\sigma_{\min}(\widetilde{\mathcal{G}}{\bfY})}
		\le \frac{128\sqrt{2}\kappa^2\sqrt{r}\|\widetilde{{\bfL}}_l-\widetilde{\mathcal{G}}{\bfY} \|_F^2 }{\sigma_{\min}(\widetilde{\mathcal{G}}{\bfY})}\\
		\le& \varepsilon_0\|\widetilde{{\bfL}}_l-\widetilde{\mathcal{G}}{\bfY} \|_F.
		\end{split}
		\end{equation*}
		\normalsize
		Similar to $I_6$ and $I_7$,
		\vspace{-2mm}
		\begin{equation*}
		\begin{split}
		\sqrt{2r}(I_6+I_7)
		\le &326\kappa\sqrt{\frac{2\mu c_s r^2 \log (n)}{\widehat{m}}}\|\widetilde{{\bfL}}_l-\widetilde{\mathcal{G}}{\bfY} \|_F.
		\end{split}
		\end{equation*}
		Hence, if $\widehat{m}\ge C_6\varepsilon_0^{-2}\mu c_s\kappa^2r^2\log(n)$ for some constant $C_6$, 
		\begin{equation*}
		\begin{split}
		\sqrt{2r}(I_5+I_6+I_7)&\le 2\varepsilon_0\|\widetilde{{\bfL}}_l-\widetilde{\mathcal{G}}{\bfY} \|_F 
		\end{split}
		\end{equation*}
		with probability at least $1-2n_cn^{-2}$.
		On the other hand, if $m\ge 256r\log(n)$, then with probability at least $1-n_cn^{-2}$
		\vspace{-1mm}
		\begin{equation}\label{I_9}
		\begin{split}
		\vspace{-3mm}
		I_9\le&2\sqrt{2r}\|\widehat{p}^{-1}\widetilde{\mathcal{G}} {\mathcal{P}}_{\Omega_{l+1}}(\bfS)\|\\
		\le&2\sqrt{2r}\|\widehat{p}^{-1}\widetilde{\mathcal{G}} {\mathcal{P}}_{\Omega_{l+1}}(\bfS)-\widetilde{\mathcal{G}}\bfS\|+2\sqrt{2r}\| \widetilde{\mathcal{G}}\bfS\| \\
		\le&8\sqrt{2}\sqrt{\frac{r\log(n)}{m}}n_cn\|\widetilde{\mathcal{G}}\bfS \|_{\infty}+2\sqrt{2r}n_cn\|\widetilde{\mathcal{G}}\bfS \|_{\infty}\\
		\le&\frac{1}{16\sqrt{2}}\sqrt{\frac{r\log(n)}{m}}\frac{\varepsilon_0\sigma_{\min}(\widetilde{\mathcal{G}}\bfY)}{\kappa^2r^{1/2}}+\frac{\varepsilon_0\sigma_{\min}(\widetilde{\mathcal{G}}{\bfY})}{512\sqrt{2}\kappa^2r^{1/2}}\\
		\le&\frac{\varepsilon_0\sigma_{\min}(\widetilde{\mathcal{G}}{\bfY})}{256\sqrt{2}\kappa^2r^{1/2}},
		\end{split}
		\end{equation}
		\normalsize
		where the second last inequality comes from {Lemma \ref{condition3.1}}.\\
		Following  $\nu=2\varepsilon_0\le1/2$ and (\ref{4.1}), with probability at least $1-3n_cn^{-2}$, we can bound $\|\widetilde{{\bfL}}_{l+1}-\widetilde{\mathcal{G}}{\bfY} \|_F$ by
		\begin{equation*}
		\begin{split}
		\frac{1}{2}\|\widetilde{{\bfL}}_{l}-\widetilde{\mathcal{G}}{\bfY} \|_F+\frac{\varepsilon_0\sigma_{\min}(\widetilde{\mathcal{G}}{\bfY})}{256\sqrt{2}\kappa^2r^{1/2}}
		\le\frac{\varepsilon_0\sigma_{\min}(\widetilde{\mathcal{G}}{\bfY})}{128\sqrt{2}\kappa^2r^{1/2}}.
		\end{split}
		\end{equation*}
		Hence, (\ref{4.1}) also holds when $k={l+1}$.
		
\noindent{\bf Base Case:}
		Since $\widetilde{{\bfL}}_0=\mathcal{Q}_r\big(\widehat{p}^{-1}\widetilde{\mathcal{H}}{\mathcal{P}}_{\Omega_0}({\bfX+\bfN})\big)$, then with probability at least $1-n_cn^{-2}$,
		\small
		\begin{equation*}		
		\begin{split}
		&\|\widetilde{{\bfL}}_0-\mathcal{\widetilde{G}}{\bfY}\|_F\le\sqrt{2r}\|\widetilde{{\bfL}}_0-\mathcal{\widetilde{G}}{\bfY}\|\\
		\le&\sqrt{2r}\left\|p^{-1}\mathcal{\widetilde{G}}\mathcal{P}_{\Omega}({\bfY}+\bfS)-\widetilde{{\bfL}}_0\right\|
		+\sqrt{2r}\left\|p^{-1}\mathcal{\widetilde{G}}\mathcal{P}_{\Omega}({\bfY}+\bfS)-\mathcal{\widetilde{G}}{\bfY}\right\|\\
		\le& 2\sqrt{2r}\left\|p^{-1}\mathcal{\widetilde{G}}\mathcal{P}_{\Omega}({\bfY})-\mathcal{\widetilde{G}}{\bfY}\right\|+2\sqrt{2r}\left\| p^{-1}\mathcal{\widetilde{G}}\mathcal{P}_{\Omega}({\bfS})\right\|\\
		\le&\sqrt{\frac{512\mu c_sr^2\log(n)}{m}}\|\mathcal{\widetilde{G}}{\bfY}\|+\frac{\varepsilon_0\sigma_{\min}(\widetilde{\mathcal{G}}{\bfY})}{256\sqrt{2}\kappa^2\sqrt{r}}.
		\end{split}		
		\end{equation*}
		\normalsize
		where the last inequality comes from $(\ref{I_9})$ and {Lemma \ref{condition3}}.
		To guarantee that (\ref{4.1}) holds with $k=0$, we need 
		\begin{equation}
		\sqrt{\frac{512\mu c_sr^2\log(n)}{m}}\|\mathcal{\widetilde{G}}{\bfY}\|\le\frac{\varepsilon_0\sigma_{\min}(\widetilde{\mathcal{G}}{\bfY})}{256\sqrt{2}\kappa^2\sqrt{r}}.
		\end{equation}
		That is 
		$\widehat{m}\ge C_7 \varepsilon_0^{-2}\mu c_s \kappa^6 r^3 \log(n)$ for some constant $C_7$.	
			
		Let $C_3=\max\{C_6,C_7\}$, if $\widehat{m}\ge C_3 \varepsilon_0^{-2}\mu c_s \kappa^6 r^3 \log(n)$, for each $l\ge0$, with probability $1-2n_cn^{-2}$, we have
		\vspace{-1mm}
		\begin{equation}
		\vspace{-1mm}
		\begin{split}
		\| \widetilde{{\bfL}}_{l+1}-\widetilde{\mathcal{G}}{\bfY}\|_F\le &2\varepsilon_0\|\widetilde{{\bfL}}_l-\widetilde{\mathcal{G}}{\bfY} \|_F+\Delta. 
		\end{split}
		\end{equation}
		where $\Delta=32\sqrt{2}n_cn\|\widetilde{\mathcal{G}}\bfS \|_{\infty}+2\sqrt{2}r^{1/2}\|\widetilde{\mathcal{G}}\bfS \|$. Then
		\begin{equation*}
		\begin{split}
		\| \widetilde{{\bfL}}_{l+1}-\widetilde{\mathcal{G}}{\bfY}\|_F-\frac{\Delta}{1-\nu}\le \nu\Big(\|\widetilde{{\bfL}}_l-\widetilde{\mathcal{G}}{\bfY} \|_F-\frac{\Delta}{1-\nu}\Big). 
		\end{split}
		\end{equation*}
		\normalsize
		Therefore, with probability $1-(3L+1)n_cn^{-2}$, 
		\begin{equation}
		\begin{split}
		\| \widetilde{{\bfL}}_L-\widetilde{\mathcal{G}}{\bfY}\|_F&\le \nu^{L}\|\widetilde{{\bfL}}_0-\widetilde{\mathcal{G}}{\bfY}\|_F+\frac{\Delta}{1-\nu}.
		\end{split}
		\end{equation}
		Similar to \eqref{t2.end}, take $L=\Big\lceil \varepsilon_0^{-1}\log \Big(\frac{\sigma_{\max}(\mathcal{H}{\bfX})}{128\kappa^3\varepsilon}\Big)\Big\rceil$ with an arbitrarily small positive constant $\varepsilon$,
		since $\nu\le1/2$,
		\begin{equation*}
		\begin{split}
		\| \widetilde{{\bfL}}_L-\widetilde{\mathcal{G}}{\bfY}\|_F
		\le& {n}_c^{1/2}\varepsilon+64\sqrt{2}n_cn\|\widetilde{\mathcal{G}}\bfS \|_{\infty}+4\sqrt{2}r^{1/2}\|\widetilde{\mathcal{G}}\bfS \|\\
		\le& {n}_c^{1/2}\varepsilon+128n_cn\|\widetilde{\mathcal{G}}\bfS \|_{\infty}+8{r}^{1/2}\|\widetilde{\mathcal{G}}\bfS \|.
		\end{split}
		\end{equation*}
		which completes the proof of {Theorem \ref{t3}}.
\end{proof}

%% file: appendix/c1p3.tex
\section{Proof of Theorem \ref{coherent}}\label{sec: c1p5}
We first introduce some useful lemmas. 
\begin{lemma}[\cite{HJ85}, Corollary 7.7.4(a)]\label{matrix}
	If ${\bfA},{\bfB}\in\mathbb{C}^{n}$ are positive-definite, then ${\bfA}\succeq{\bfB}$ if and only if ${\bfB}^{-1}\succeq{\bfA}^{-1}$
\end{lemma}
\begin{lemma}[\cite{MK04}, Theorem 7]\label{msp}
	Let $\lambda_1\ge\cdots\ge\lambda_n$ be eigenvalues of ${\bfA}$, denoted by $\lambda_i({\bfA})=\lambda_i$. 
	Let ${\bfA}$ and ${\bfB}$ be Hermitian positive semi-definite $n \times n$ matrices. If
	$1\le k \le i\le n$ and $1\le l \le n-i+1 $, then 
	\vspace{-1mm}
	\begin{equation*}
	\vspace{-1mm}
	\lambda_{i+l-1}({\bfA})\lambda_{n-l+1}({\bfB})\le \lambda_i({\bfA\bfB})\le \lambda_{i-k+1}({\bfA})\lambda_{k}({\bfB}).
	\end{equation*}
	In particular,
	\vspace{-1mm}
	\begin{equation*}
	\vspace{-1mm}
	\lambda_n({\bfA})\lambda_n({\bfB})\le\lambda_n({\bfA\bfB}), \quad
	\lambda_1({\bfA\bfB})\le\lambda_1({\bfA})\lambda_1({\bfB}).
	\end{equation*}
\end{lemma}

\begin{proof}[Proof of Theorem \ref{coherent}]
	 Consider $n_c=1$, the definition of $\mathcal{H}$ can be extended to a row vector, which corresponds with one channel data. Let $\mathcal{H}\bfX_{k*}=\bfU_k{\bf\Sigma}_k{\bfV}_k$ and $\mathcal{H}\bfX=\bfU{\bf\Sigma}{\bfV}$ be the SVD of $\mathcal{H}\bfX_{k*}$ and $\mathcal{H}\bfX$. Then, $\mu_0$ is defined as
	 \vspace{-1mm}
	\begin{equation*}
	\vspace{-1mm}
	\max_{k_1}\left\|{\bfe}_{k_1}^*{\bfU}_k\right\|^2\le\frac{\mu_0 r}{n_1}, \quad 
	\max_{k_2}\left\|{\bfe}_{k_2}^*{\bfV}_k\right\|^2\le\frac{\mu_0 r}{n_2}.
	\end{equation*}
	Notice that all $\mathcal{H}\bfX_{k*}$ share the same column space and row space, they have the same incoherence $\mu_0$. 
	It is trivial $\mu=\mu_0$ if we consider the incoherence of row spaces. Hence, we only focus on the incoherence of column space.
	
	By (\ref{Hankel X}), $\mathcal{H}\bfX=\bfP_L{\bf\Gamma}\bfP_R^T$. 
	Define a series of diagonal matrices ${\bfD}_k$ as ${\bfD}_k=\text{diag}({\bfd}_{k})$ with $1\le k\le n_c$, and ${\bfd}_k=[d_{k,1},\cdots,  d_{k,r}]$, where $d_{k,i}=\bfr_i^*\bfs_1\bfC_{k*}\bfl_i$. We need one mild assumption that $d_{k,i}\ne0$. It guarantees that each ${\bfD}_k$ is full rank. Thus, $\mathcal{H}\bfX_{k*}=\bfE_L\bfD_k\bfP_R$, where $\bfE_L=\bfP_L$ with $n_c=1$. There exists a row-switching matrix $\bfQ_1$ satisfying
	\small
	\begin{equation*}
	\bfQ_1(\mathcal{H}\bfX)=\begin{bmatrix}
	\mathcal{H}\bfX_{1*}\\
	\mathcal{H}\bfX_{2*}\\
	\vdots\\
	\mathcal{H}\bfX_{n_c*}
	\end{bmatrix}=\begin{bmatrix}
	\bfE_L\bfD_1\\
	\bfE_L\bfD_2\\
	\vdots\\
	\bfE_L\bfD_{n_c}
	\end{bmatrix}\bfP_R:=\widetilde{\bfE}_L\bfP_R.
	\end{equation*}
	\normalsize
	Define a mapping $f:\{1,2,\cdots,n_cn_1\}\mapsto\{1,2,\cdots,n_cn_1\}$, $f(z)=w$ with $\bfe_z=\bfQ_1\bfe_w$, then $f$ is a bijective mapping.
	Hence, we have
	\small
	\begin{equation}\label{pom}
	\vspace{-2mm}
	\begin{split}
	&\max_{l_1}\left\|{\bfe}_{l_1}^*{{\bfU}}\right\|^2
	=\max_{ k_1}(\bfQ_1{\bfe}_{k_1})^*\widetilde{{\bfE}}_L(\widetilde{{\bfE}}_L^*\widetilde{{\bfE}}_L)^{-1}\widetilde{{\bfE}}_L^*(\bfQ_1{\bfe}_{k_1})\\
	&=\max_{k_1}{\bfe}_{k_1}^*\widetilde{{\bfE}}_L\left(\sum_{k=1}^{n_c}{\bfD}_k^*{\bfE}_L^*{\bfE}_L{\bfD}_k\right)^{-1}\widetilde{{\bfE}}_L^*{\bfe}_{k_1}.	
	\end{split}
	\end{equation}
	\normalsize
	Consider $1\le k_1\le n_1$,  we know that ${\bfe}_{k_1}^*\widetilde{{\bfE}}_L=\widehat{{\bfe}}_{k_1}^*{\bfE}_L{\bfD}_1$, where ${\bfe}_{k_1} \in \mathbb{C}^{n_cn_1}$ and $\widehat{{\bfe}}_{k_1} \in \mathbb{C}^{n_1}$ are both coordinate vectors. Additionally, it is easy to show that symmetric matrices $\{{\bfD}_k^*{\bfE}_L^*{\bfE}_L{\bfD}_k\}_{k=1}^{n_c}$ are positive definite since $\{{\bfE}_L{\bfD}_k\}_{k=1}^{n_c}$ are full rank. Also, following Lemma \ref{matrix}, we have 
	\small
	\begin{equation}\label{S_ps}
	\begin{split}
	\sum_{k=1}^{n_c}{\bfD}_k^*{\bfE}_L^*{\bfE}_L{\bfD}_k& \succ{\bfD}_1^*{\bfE}_L^*{\bfE}_L{\bfD}_1\succ 0,\\
	\Big(\sum_{k=1}^{n_c}{\bfD}_k^*{\bfE}_L^*{\bfE}_L{\bfD}_k\Big)^{-1}&\prec({\bfD}_1^*{\bfE}_L^*{\bfE}_L{\bfD}_1)^{-1}.
	\end{split}	
	\end{equation}
	\small
	\vspace{-3mm}
	Then,
	\begin{equation*}
	\begin{split}
	\left\|{\bfe}_{k_1}^*{{\bfU}}\right\|^2
	&=\widehat{{\bfe}}_{k_1}^*{\bfE}_L{\bfD}_1\Big(\sum_{k=1}^{n_c}{\bfD}_k^*{\bfE}_L^*{\bfE}_L{\bfD}_k\Big)^{-1}{\bfD}_1^*{\bfE}_L^*\widehat{{\bfe}}_{k_1}\\
	&<\widehat{{\bfe}}_{k_1}^*{\bfE}_L{\bfD}_1({\bfD}_1^*{\bfE}_L^*{\bfE}_L{\bfD}_1)^{-1}{\bfD}_1^*{\bfE}_L^*\widehat{{\bfe}}_{k_1}\\
	&=\widehat{{\bfe}}_{k_1}^*{\bfE}_L({\bfE}_L^*{\bfE}_L)^{-1}{\bfE}_L^*\widehat{{\bfe}}_{k_1}\le\frac{u_0r}{n_1}=\frac{(n_cu_0)r}{n_cn_1}.
	\end{split}
	\end{equation*}
	\normalsize
	
	Similarly, we can prove $\left\|{\bfe}_{k_1}^*{{\bfU}}\right\|^2<\displaystyle
	\frac{(n_cu_0)r}{n_cn_1}$
	for all $i$ satisfying $1\le i\le n_cn_1$, which leads to \eqref{eqn:ratio}.
	\vspace{0.75mm}
	
	Moreover, we can provide a tighter bound on $\mu$ with a stronger assumption. Suppose there exists a $\widehat{d}\in\mathbb{C}$ and a real number $\delta\in (0,1)$ satisfying $(1-\delta)|\widehat{d}|\le |d_{k,i}| \le (1+\delta)|\widehat{d}|$. \\
	By {Lemma \ref{msp}}, define $\kappa_L=\frac{\sigma_{\max}({\bfE}_L)}{\sigma_{\min}({\bfE}_L)}$, then:
	\small
	\begin{equation*}
	\begin{split}
	&\lambda_{\max}({\bfD}_1^*{\bfE}_L^*{\bfE}_L{\bfD}_1)
	= \lambda_{\max}({\bfE}_L^*{\bfE}_L{\bfD}_1{\bfD}_1^*)\\
	&\le \lambda_{\max}({\bfE}_L^*{\bfE}_L)\lambda_{\max}({\bfD}_1{\bfD}_1^*)\\
	&\le \frac{\kappa_L^2(1+\delta)^2}{(1-\delta)^2} \lambda_{\min}({\bfE}_L^*{\bfE}_L)\lambda_{\min}({\bfD}_k{\bfD}_k^*)\\
	&\le \frac{\kappa_L^2(1+\delta)^2}{(1-\delta)^2} \lambda_{\min}({\bfD}_k^*{\bfE}_L^*{\bfE}_L{\bfD}_k).\\
	\end{split}
	\end{equation*}
	\normalsize
	since even the minimum eigenvalue of ${\bfD}_k^*{\bfE}_L^*{\bfE}_L{\bfD}_k$ is larger than the maximum one of $\frac{(1-\delta)^2}{\kappa_L^2(1+\delta)^2}{\bfD}_1^*{\bfE}_L^*{\bfE}_L{\bfD}_1$, we have
	$$\sum_{k=1}^{n_cn}{\bfD}_k^*{\bfE}_L^*{\bfE}_L{\bfD}_k\succeq [1+(n_c-1)\frac{(1-\delta)^2}{\kappa_L^2(1+\delta)^2}]{\bfD}_1^*{\bfE}_L^*{\bfE}_L{\bfD}_1.$$
	Similarly, we can establish the following relation between $\mu$ and $\mu_0$,
	\begin{equation*}
	\mu\le\frac{n_c\mu_0}{1+(n_c-1)\frac{(1-\delta)^2}{\kappa_L^2(1+\delta)^2}}.
	\end{equation*}
\end{proof}

%% file: appendix/c2p1.tex
In this part, part of the proofs of the theorems in chapter 2 are provided, while the full proofs can be found in \cite{ZW18, ZW19}. In Appendix \ref{sec:c2n1}, some important notations and assumptions are provided. Next, the supporting lemmas  for proving heorem \ref{Theorem: bad data and missing data} are presented in Appendix \ref{sec: c2l1} to illsutrate the roadmap of proof. In what followed, the proof of Theorem \ref{Theorem: bad data and missing data} are summarized in Appendix \ref{sec: c2p1}.

\section{Notations and Technical Assumptions}\label{sec:c2n1}
\textbf{Sampling model with replacement.} As a standard technique in solving RMC problem \cite{R11}, the model of sampling with replacement assumes that every entry is sampled independently with replacement. 
In this model, one entry can be sampled multiple times.
 To  distinguish from $\widehat{\Omega}$ defined in Section \ref{sec:formulation}, let $\Omega$ be the union of indices that uniformly sampled from $\{1,2, \cdots, n_c\}\times \{1,2, \cdots, n\}$ following the sampling model with replacement. 
Due to the repetitions in the sampling model with replacement, $|\Omega|\ge|\widehat{\Omega}|$ should hold for successful recovery \cite{R11}. Hence, the required number of observations for successful recovery under sampling model with replacement is sufficient to guarantee successful recovery under sampling model without replacement.

\textbf{Symmetric Hankel Operator.}
Here, we introduce the operator $\wcH$, which is the symmetric extension of Hankel operator $\cH$. For any $\bfZ\in\mathbb{C}^{n_c\times n}$,  $\widetilde{\mathcal{H}}(\bfZ)\in\mathbb{C}^{n_c(n_1+n_2)\times n_c(n_1+n_2)}$ is defined as
{\begin{equation}\label{eqn: widetilde_Hankel}
	\widetilde{\mathcal{H}}(\bfZ)
	=\begin{pmatrix}
	\begin{matrix}
	\bf0~~&\cdots&~~\bf0\\
	\vdots~~&\ddots&~~\vdots\\
	\bf0~~&\cdots&~~\bf0\end{matrix}&\begin{matrix}
	(\mathcal{H}(\bfZ))^H\\\vdots\\(\mathcal{H}(\bfZ))^H
	\end{matrix}\\
	\smash{\underbrace{\begin{matrix}\mathcal{H}(\bfZ)&\cdots&\mathcal{H}(\bfZ)\end{matrix}}_{\displaystyle n_c~\text{copies}}}&\bf0\\
	\end{pmatrix}.
	\vspace{4mm}
\end{equation}
}Define $\mathcal{H}\bfX^*=\bfU\boldsymbol{\Sigma}\bfV^H$ as the SVD of $\mathcal{H}\bfX^*$, then $\wcH\bfX^*$ can be written as 
\begin{equation}\label{eqn: SVD_widetilde_Hankel}
\wcH\bfX^*=\frac{1}{\sqrt{2}}\begin{pmatrix}
\widetilde{\bfV}&\widetilde{\bfV}\\
\bfU&\bfU
\end{pmatrix}
\begin{pmatrix}
\sqrt{n_c}\boldsymbol{\Sigma}&\bf0\\
\bf0&-\sqrt{n_c}\boldsymbol{\Sigma}
\end{pmatrix}
\frac{1}{\sqrt{2}}
\begin{pmatrix}
\widetilde{\bfV}&\widetilde{\bfV}\\
\bfU&\bfU
\end{pmatrix}^H,
\end{equation}
where $\widetilde{\bfV}=\frac{1}{\sqrt{n_c}}[\bfV^H\quad  \cdots \quad \bfV^H]^H$.
Therefore, $\wcH\bfX$ is a rank-$2r$ matrix.  Moreover, if  $\mathcal{H}\bfX^*$ is $\mu$-incoherent, one can easily check that 
the incoherence $\tilde{\mu}$ of $\wcH\bfX$ satisfies $\tilde{\mu}\le\frac{c_s}{2}\mu$. 
When $n_1$ and $n_2$ are in the same order, $c_s$ is a constant. 
\\
The key steps (lines  5-9) of Alg. \ref{Alg} can be represented equivalently based on $\wcH$ as:
\begin{equation}\label{eqn:eqv}
\begin{split}
&\widetilde{\bfS}_t=\mathcal{T}_{\xi_{t}}(\bfM-\bfX_t);\\
&\widetilde{\bfW}_{t}=\wcH\big(\bfX_t+p^{-1}\mathcal{P}_{\Omega_{k,t} }(\bfM-\bfX_t-\widetilde{\bfS}_t)\big);\\
&\xi_{t+1}=\frac{\eta}{\sqrt{n_c}}\Big(|\lambda_{2k}(\widetilde{\bfW}_t)|+\big(\frac{1}{2}\big)^t|\lambda_{2k+2}(\widetilde{\bfW}_t)|\Big);\\
&\widetilde{\bfL}_{t+1}=\mathcal{Q}_{2k}(\widetilde{\bfW}_t);\\
&\bfX_{t+1}=\wcH^{\dagger}(\widetilde{\bfL}_{t+1});
\end{split}
\end{equation}
The Pseudoinverse operator $\wcH^{\dagger}$ can be calculated from
\begin{equation}
(\wcH^{\dagger}(\bfZ))_{i,j}=\frac{1}{n_cw_j}\langle\wcH(\bfe_i\bfe_j^T), \bfZ\rangle. 
\end{equation}
In fact, \eqref{eqn:eqv} generates the same $\bfX_{t}$'s as lines 5-9 in Alg. \ref{Alg}. 
\eqref{eqn:eqv} differs from lines 5-9 in Alg. \ref{Alg} mainly in two aspects : (1) $\widetilde{\bfS}_{t}$ is updated based on the full observation of $\bfM$; (2) $\wW_{t}$ lies in the space defined by $\wcH$. Though we cannot calculate $\wS_t$ from $\mathcal{P}_{\Omega}(\bfM)$ in practice, $\wS_t$ is introduced to simplify our analysis and does not affect the update of $\bfX_t$. To see this,
we first assume the values of $\bfX_{t-1}$ are the same for \eqref{eqn:eqv} and lines 5-9 in Alg. \ref{Alg}. Then, the threshold $\xi_{t}$ remains the same as well.  Next, we have 
\begin{equation}
\mathcal{P}_{\Omega_{k,t} }(\widetilde{\bfS}_t)=\bfS_t,
\end{equation} 
which suggests $\mathcal{P}_{\Omega_{k,t} }(\bfM-\bfX_{t})-{\bfS}_t=\mathcal{P}_{\Omega_{k,t} }(\bfM-\bfX_t-\wS_t)$.
Operator $\wcH$ does not affect the update rule of $\bfX_{t}$, either. Similarly, suppose $\bfX_{t-1}$ remains the same for some $t$, then it is easy to verify that  
$$\widetilde{\bfL}_t=\begin{pmatrix}
	\begin{matrix}
	\bf0~&\cdots&~\bf0\\
	\vdots~&\ddots&~\vdots\\
	\bf0~&\cdots&~\bf0\end{matrix}&
	\begin{matrix}
	\bfL_t^H\\\vdots\\\bfL_t^H
	\end{matrix}\\
	\smash{\underbrace{ \begin{matrix}~\bfL_t&\cdots&\bfL_t\end{matrix}}_{\displaystyle n_c~\text{copies}}}&\bf0\\
	\end{pmatrix}\in\mathbb{C}^{n_c(n_1+n_2)\times n_c(n_1+n_2)}.
	\vspace{4mm}$$
(Since $n+1=n_1+n_2$, we use $n$ to replace $n_1+n_2$ for convenience in all the sections of Appendix.) Moreover, $\widetilde{\bfL}_t$ has duplicated eigenvalues as $|\lambda_{2i-1}(\widetilde{\bfL}_t)|=|\lambda_{2i}(\widetilde{\bfL}_t)|$ for $1\le i \le r$, where $\lambda_{i}(\widetilde{\bfL}_t)$ is the $i$-th largest eigenvalues (in absolute value) of $\widetilde{\bfL}_t$. Furthermore, let $\sigma_{i}({\bfL}_t)$ be the $i$-th largest singular value of $\bfL_t$, from \eqref{eqn: SVD_widetilde_Hankel} we have 
\begin{equation}\label{eqn: SV_EV}
\sigma_{i}(\bfL_t)=\frac{1}{\sqrt{n_c}}|\lambda_{2i-1}(\widetilde{\bfL}_t)|=\frac{1}{\sqrt{n_c}}|\lambda_{2i}(\widetilde{\bfL}_t)|.
\end{equation} 
Similar results can be derived for $\widetilde{\bfW}_t$.
From the definition of $\wcH^{\dagger}$ and the structure of $\wL_t$, it is straightforward that $\bfX_{t+1}$ returned by lines 5-9 in Alg. \ref{Alg} and \eqref{eqn:eqv} are equivalent. 
In conclusion, if we start with the same initial point $\bfX_{0}=\bf0$, the update rule in \eqref{eqn:eqv} will generate the same $\bfX_{t}$'s as those by lines 5-9 in Alg. \ref{Alg}, and we also have $\mathcal{P}_{\Omega_{k,t} }(\wS_{t})=\bfS_{t}$.

\textbf{Definition of $\bfH_{1,t}$, $\bfH_{2,t}$ and $\bfH_{t}$.}
From \eqref{eqn:eqv}, we know that
\begin{equation*}
\begin{split}
\widetilde{\bfL}_{t+1}
=&\mathcal{Q}_{2k}\big(\wcH\bfX_t  +\hat{p}^{-1}\wcH\mathcal{P}_{\Omega_{k,t} }(\bfM-\bfX_t-\wS_t)\big)\\
=&\mathcal{Q}_{2k}\big(\wcH\bfX_t  +\hat{p}^{-1}\wcH\mathcal{P}_{\Omega_{k,t} }(\bfX^*+\bfS^*-\bfX_t-\wS_t)\big)\\
=&\mathcal{Q}_{2k}\big(\wcH\bfX^* +  \wcH(\mathcal{I}-\hat{p}^{-1}\mathcal{P}_{\Omega_{k,t} }) (\bfX_t+\wS_t-\bfX^*-\bfS^*)\\
&~~~~~~~~~~~~~+\wcH(\bfS^*-\wS_t)\big).\\
\end{split}
\end{equation*}
Let $\bfH_{t}=\bfH_{1,t}+\bfH_{2,t}$,
where 
\begin{equation}\label{def:H1}
\begin{split}
&\qquad \qquad \bfH_{1,t}=\wcH({\bfS^*}-\wS_t),
\\
&\bfH_{2,t}=\wcH(\mathcal{I}-\hat{p}^{-1}\mathcal{P}_{\Omega_{k,t}}) (\bfX_t+\wS_t-\bfX^*-\bfS^*).
\end{split}
\end{equation}
Then, we have $$\widetilde{\bfL}_t=\mathcal{Q}_{2k}(\wcH\bfX^* +\bfH_t)=\mathcal{Q}_{2k}(\wcH\bfX^* +\bfH_{1,t}+ \bfH_{2, t}).$$

 \section{Key Lemmas in Proving Theorem \ref{Theorem: bad data and missing data}}\label{sec: c2l1}
We first introduce the key lemma in the whole proof.
  Lemma \ref{Lemma:key_lemma_L} is presented to 
  bound the $\ell_2$-norm of $\bfe_i^T(\bfH_{t})^a\bfZ$ by the $\ell_2$-norm of $\bfe_i^T\bfZ$ for all $1\le  a \le \log(n) $. 
  Although Lemma \ref{Lemma:key_lemma_L} is not directly used in proving Theorem \ref{Theorem: bad data and missing data}, 
  Lemma \ref{Lemma:key_lemma_L} is paramount important in showing the recovery error of $\bfX_{t}$ decreases as $t$ increases that summarized in
   Lemma \ref{Lemma: L}.

 Lemma 15 in \cite{CGJ16} provides a similar result for general matrix but with a more complicated proof.
  Ref. \cite{CGJ16} focused on bounding all entries of $\bfe_i^T(\bfH_{1,t}+\bfH_{2, t})^a\bfZ$, so Ref. \cite{CGJ16} needed to write the closed forms of all entries in $(\bfH_{1,t}+\bfH_{2, t})^a$. The closed forms are hard to determine, and several cases should be discussed separately.
  However, we will prove \eqref{eqn: H_t} by mathematical induction over $a$. Only two items, $\|\bfe_i^T\bfH_{1,t}(\bfH_{t})^{a-1}\bfZ\|_2$
  and $\|\bfe_i^T\bfH_{2,t}(\bfH_{t})^{a-1}\bfZ\|_2$, need to be bounded in the inductive step. Also, the conclusion of Lemma \ref{Lemma:key_lemma_L} in \eqref{eqn: H_t} can be extended to general matrices, though $\bfH_t=\bfH_{1,t}+\bfH_{2, t}$ is the Hankel matrix as defined in \eqref{def:H1}. 
  
  
  There are two lemmas used in the inductive steps of proving Lemma \ref{Lemma:key_lemma_L}. \footnote{ The proof of these two lemmas are presented in the supplementary material}
   Lemma \ref{Lemma: H_1} is built on the sparsity assumption with respect to $\bfH_{1, t}$.
    Moreover, Lemma \ref{Lemma: H_1} is a special case that $a=1$ of Lemma 5 \cite{NNSAJ14}, and all the steps are straightforward from \cite{NNSAJ14}. However, instead of discussing a special $\bfU$ like in \cite{NNSAJ14}, we consider a general matrix $\bfZ$ here.  Lemma \ref{Lemma: H_2} provides similar result for matrices with zero mean and bounded high moments, and it is used to bound $\bfH_{2, t}$. 
  The technique in proving Lemma \ref{Lemma: H_2} is similar as  that of Lemma 9 \cite{JN15}. 
  Rather than bounding each entry of $\bfe_i^T(\wcH\bfY)\bfZ$ separately as \cite{JN15}, we bound the $\ell_2$ norm of $\bfe_i^T(\wcH\bfY)\bfZ$ directly, which leads to a tighter bound by a factor of $r^{-1}$. The same trick is applied in \cite{CGJ16} as well.
  
  \begin{lemma}[\cite{ZW19}, Lemma 1]
  \label{Lemma:key_lemma_L}
  	Suppose the assumptions in Theorem \ref{Theorem: bad data and missing data}.
  	If we further assume that $Supp(\wS_{t}-\bfS^*)\subseteq Supp(\bfS^*)$, then for $1\le a \le \log(n_cn)$ and any $\bfZ \in\mathbb{C}^{n_cn\times l}$, with probability at least $1-\frac{n_c\log(n_cn)}{n^2}$,  we have
  	{
  		\begin{equation}\label{eqn: H_t}
  		\begin{split}
  		&\max_{i}\|\bfe_i^T(\bfH_{t})^a\bfZ\|_2\\
  		\le&\Big(C_3\beta_t\log(n)+ \alpha n_cn\|\bfH_{1,t}\|_{\infty}   \Big)^a\max_i\|\bfe_i^T\bfZ\|_{2},
  		\end{split}
  		\end{equation}
  	}where $\beta_t=\sqrt{\frac{n_cn}{\hat{p}}}\left\| \bfX_t+\wS_t-\bfX-\bfS\right\| _{\infty}$ and $C_3$ is a constant that greater than $e^4$.
  \end{lemma}

  \begin{lemma}[\cite{ZW19}, Lemma 2]
  \label{Lemma: H_1}
  	Assume each row and column of $\bfH\in\mathbb{C}^{n_cn\times n_cn}$ has at most $s$ nonzero entries, then for any $\bfZ \in\mathbb{C}^{n_cn\times l}$, 
  	\begin{equation}
  	\max_{1\le i\le n_cn}\|\bfe_i^T\bfH\bfZ\|_2\le(s\|\bfH\|_{\infty})\max_{1\le j\le n_cn}\|\bfe_j^T\bfZ\|_2.
  	\end{equation}
  \end{lemma}
  
  \begin{lemma}[\cite{ZW19}, Lemma 3]
  \label{Lemma: H_2}
  	Assume  each entry  of $\bfY\in\mathbb{C}^{n_c\times n}$ is drawn independently with
  	\begin{equation}\label{eqn: lemma_H_2_assumption}
  	\mathbb{E}(Y_{i,j})=0,\qquad \mathbb{E}(|Y_{i,j}|^k)\le\frac{1}{n_cn}
  	\end{equation}
  	for all $1\le i\le n_c$, $1\le j\le n$ and $k\ge2$. Then, for any $\bfZ\in\mathbb{C}^{n_cn\times l}$,  we have
  	\begin{equation}
  	\max_{1\le i\le n_cn}\|\bfe_i^T(\wcH\bfY)\bfZ\|_2\le C_3\log (n)\max_{1\le j\le n_cn}\|\bfe_j^T\bfZ\|_{2},
  	\end{equation}
  	with probability $1-n_cn^{-3}$. 
  \end{lemma}

In the following lemmas, $\widetilde{\bfS}_{t}$, $\bfX_{t}$, $\bfH_t$, $\wW_{t}$ and $\wL_t$ are generated in the $k$-th outer loop unless otherwise specified. For convenience, we use  ${\lambda}_i^*$ to denote ${\lambda}_{2i-1}(\wcH\bfX^*)$, which is the $(2i-1)$-{th} largest eigenvalue (in absolute value) of $\wcH\bfX^*$. Similarly, ${\lambda}_i^{(t)}$ stands for ${\lambda}_{2i-1}(\wW_t)$, which is the $(2i-1)$-{th} largest eigenvalue (in absolute value) of $\wW_t$.

Lemma \ref{Lemma:sigma_W} proves that the assumptions \eqref{eqn:noL_1} and \eqref{eqn:noL_2} are equivalent.
Lemma \ref{Lemma: L} shows the reduction of $\|\bfX_{t+1}-\bfX^*\|_{\infty}$ as $t$ increases. Moreover, the error bound of $\|\widetilde{\bfS}_{t+1}-\bfS^*\|_{\infty}$ is given in  Lemma \ref{Lemma: S} based on the bound of $\|\bfX_{t+1}-\bfX^*\|_{\infty}$.

\begin{lemma}[Weyl's inequality]\label{Lemma:com_singular value}
	Suppose $\bfA$, $\bfB\in \mathbb{C}^{n\times n}$ are two symmetric matrices satisfying $\bfB=\bfA+\bfE$. 
	Then,
	\begin{equation}
	|\lambda_i(\bfB)-\lambda_i(\bfA)|\le\|\bfE\|_2, \quad 1\le i\le n.
	\end{equation}
\end{lemma}

\begin{lemma}[\cite{ZW19}, Lemma 5]
\label{Lemma:sigma_W}
	Suppose the assumptions in Theorem \ref{Theorem: bad data and missing data} and  
	\begin{equation}
	\begin{split}
	&\|\wS_t-\bfS^*\|_{\infty}\le\frac{7\tmu\tr}{n_cn}\Big(|\lambda^*_{k+1}|+\Big(\frac{1}{2}\Big)^{t-1}|\lambda^*_k|\Big),\\
	&\quad  Supp(\wS_t-\bfS^*)\subseteq Supp(\bfS^*),
	\end{split}
	\end{equation}
	\begin{equation}\label{eqn:spectral_norm_H_t}
	\|\bfX_t-\bfX^*\|_{\infty}\le\frac{2\tmu\tr}{n_cn}\Big(|\lambda^*_{k+1}|+\Big(\frac{1}{2}\Big)^{t-1}|\lambda^*_k|\Big).
	\end{equation}
	With probability at least $1-n_cn^{-2}$, we have
	\begin{equation}
	\|\bfH_{t}\|_2\le \frac{1}{60}\Big(|\lambda^*_{k+1}|+\Big(\frac{1}{2}\Big)^{t-1}|\lambda^*_k|\Big).
	\end{equation}
	provided that $\widehat{m}\ge C_4\tmu^2\tr^2\log(n)$.
\end{lemma}

\begin{lemma}[\cite{ZW19}, Lemma 6]
\label{Lemma: L}
	Suppose the assumptions in Theorem \ref{Theorem: bad data and missing data} and
	\begin{equation}\label{eqn:noL_1}
	\begin{split}
	&\|\wS_t-\bfS\|^*_{\infty}\le\frac{8\tmu\tr}{n_cn}\Big(\sL[t-1]\Big),\\
	&\quad  Supp(\wS_t)\subseteq Supp(\bfS^*),\\
	&\|\bfX_t-\bfX^*\|_{\infty}\le\frac{2\tmu\tr}{n_cn}\Big(\sL[t-1]\Big).
	\end{split}
	\end{equation}
	With probability at least $1-\frac{n_c\log^3(n_cn)}{n^{2}}$, we have
	\begin{equation}
	\|\bfX_{t+1}-\bfX^*\|_{\infty}\le\frac{2\tmu\tr}{n_cn}\Big(\sL[t]\Big)
	\end{equation}
	provided that $\widehat{m}\ge C_5\tmu^2\tr^2\log^2(n)$.
\end{lemma}
\begin{lemma}[\cite{ZW19}, Lemma 7]
\label{Lemma: S}
	Suppose the assumptions in Theorem \ref{Theorem: bad data and missing data} and
	\begin{equation}\label{eqn:noL_2}
	\begin{split}
	&\quad\|\bfH_{t}\|_2\le \frac{1}{60}\Big(|\lambda^*_{k+1}|+\Big(\frac{1}{2}\Big)^{t-1}|\lambda^*_k|\Big),\\
	\quad&\|\bfX_{t+1}-\bfX^*\|_{\infty}\le\frac{2\tmu\tr}{n_cn}\Big(\sL[t]\Big).
	\end{split}
	\end{equation}
	Then, we have 
	\begin{equation}
	\begin{split}
	&\|\wS_{t+1}-\bfS^*\|_{\infty}\le\frac{7\tmu\tr}{n_cn}\Big(\sL[t]\Big),\\
	and& \qquad Supp(\wS_{t+1}-\bfS^*)\subseteq Supp(\bfS^*).
	\end{split}
	\end{equation}
\end{lemma}
\section{Proof of Theorem \ref{Theorem: bad data and missing data}}\label{sec: c2p1}
The proof of Theorem \ref{Theorem: bad data and missing data} follows the similar framework established in AltProj \cite{NNSAJ14} by inductions over $k$ and $t$. 
Here, we are mainly focused on the inductions over $k$ and $t$ for \eqref{eqn:induction over t}. 
The induction over $k$ follows naturally for the selected $T$, which is the iteration number of inner loop.
The key part is
the induction over $t$, and the critical steps are verified by applying Lemmas \ref{Lemma:sigma_W}, \ref{Lemma: L} and \ref{Lemma: S} recursively. Lemmas \ref{Lemma: L} and \ref{Lemma: S} play the similar roles as Lemmas 7 and 9 in \cite{NNSAJ14}.
However, we need an extra lemma \ref{Lemma:sigma_W} to handle the additional item $\bfH_{t}$ caused by the partial observation since AltProj only considers the case of full observation.  
 
\begin{proof}[Proof of Theorem \ref{Theorem: bad data and missing data}]
	The proof is based on induction over $t$ and $k$ for the following equation:
	\begin{equation}\label{eqn:induction over t}
	\begin{split}
	&\|\wS_t^{(k)}-\bfS^*\|_{\infty}\le\frac{7\tmu\tr}{n_cn}\Big(\sL\Big),\\
	&\quad Supp(\wS_t^{(k)}-\bfS^*)\subseteq Supp(\bfS^*),\\
	&\|\bfX_{t}^{(k)}-\bfX^*\|_{\infty}\le\frac{2\tmu\tr}{n_cn}\Big(\sL[t-1]\Big),
	\end{split}
	\end{equation} where we use $\bfX_t^{(k)}$ to represent the iteration $\bfX_t$ generated in the $k$-th outer loop, similar for $\wS_{t}^{(k)}$, $\bfH_{t}^{(k)}$ and $\xi_{t}^{(k)}$.\\
	\textbf{Base Case:} When $k=1$ and $t=0$, $\bfX_{0}^{(1)}$ is initialized as $\bf0$. Since $\wcH\bfX^*$ is $\tmu$-incoherent, we have 
	\begin{equation}\label{eqn:initial_X}
	\|\bfX^*-\bfX_{0}^{(1)}\|_{\infty}=\|\bfX^*\|_{\infty}=\|\wcH\bfX^*\|_{\infty}\le\frac{\tmu\tr}{n_cn}\lambda_{1}^*.
	\end{equation}
	Note that the hard thresholding $\xi_{0}^{(1)}$ is initialized as $\frac{4\tmu\tr}{n_cn}\lambda_{1}^*$, for $\wS_{0}^{(1)}$, we consider three cases:\\
	\textbf{Case 1}: If $S_{i,j}^*=0$, then $(\wS_{0}^{(1)})_{i,j}=\mathcal{T}_{\xi_{0}^{(1)}}(X_{i,j}^*)$.
	\begin{equation}
	\begin{split}
	|X_{i,j}^*|\le\frac{\tmu\tr}{n_cn}\lambda_{1}^*\le\xi_{0}^{(1)}.
	\end{split}
	\end{equation} 
	Hence, $(\wS_{0}^{(1)})_{i,j}=0$.\\
	\textbf{Case 2:} If $S_{i,j}^*\ne0$ and $|M_{i,j}|>\xi_{0}^{(1)}$, then $(\wS_{0}^{(1)})_{i,j}=S_{i,j}^*+X_{i,j}^*.$
	\begin{equation}
	\begin{split}
	|(\wS_{0}^{(1)})_{i,j}-S_{i,j}^*|=|X_{i,j}^*|\le \frac{\tmu\tr}{n_cn}\lambda_{1}^*.
	\end{split}
	\end{equation}\\
	\textbf{Case 3:} If $S_{i,j}^*\ne0$ and $|M_{i,j}|\le\xi_{0}^{(1)}$, then $(\wS_{0}^{(1)})_{i,j}=0.$
	\begin{equation}
	\begin{split}
	|(\wS_{0}^{(1)})_{i,j}-S_{i,j}^*|
	=&|S_{i,j}^*|
	\le\xi_{0}^{(1)}+|X_{i,j}^*|
	\le\frac{5\tmu\tr}{n_cn}\lambda_{1}^*.
	\end{split}
	\end{equation}
	Hence, we have 
	\begin{equation}\label{eqn:initial_S}
	\begin{split}
	&\|\wS_{0}^{(1)}-\bfS^*\|_{\infty}\le \frac{5\tmu\tr}{n_cn}\lambda_{1}^*,\\
	&Supp(\wS_{0}^{(1)}-\bfS^*)=Supp(\bfS^*).
	\end{split}\end{equation}
	From \eqref{eqn:initial_X} and \eqref{eqn:initial_S}, we know that \eqref{eqn:induction over t} is true in the base case.
	\\
	\textbf{Induction over $t$:} For any fixed $k\ge0$, suppose that $\wS_{t}^{(k)}$ and $\bfX_{t}^{(k)}$ satisfy  \eqref{eqn:induction over t}
	for some $t\ge0$.
	Then, according to Lemma \ref{Lemma: L}, we have 
	\begin{equation}\label{eqn:not_indection over X}
	\|\bfX_{t+1}^{(k)}-\bfX^*\|_{\infty}\le\frac{2\tmu\tr}{n_cn}\Big(\sL[t]\Big).
	\end{equation}
	Note in Lemma \ref{Lemma:sigma_W}, \eqref{eqn:induction over t} suggests that 
	\begin{equation}\label{eqn:not_indection over H}
	\|\bfH_{t}^{(k)}\|_2\le \frac{1}{60}\Big(\sL[t-1]\Big),
	\end{equation} 
	with high probability. By Lemma \ref{Lemma: S}, \eqref{eqn:not_indection over X} and \eqref{eqn:not_indection over H} give that 
	\begin{equation}\label{eqn:E_t+1}	
	\begin{split}
	&\|\wS_{t+1}^{(k)}-\bfS^*\|_{\infty}\le\frac{7\tmu\tr}{n_cn}\Big(\sL[t]\Big),\\
	&\quad Supp(\wS_{t+1}^{(k)}-\bfS^*)\subseteq Supp(\bfS^*).
	\end{split}
	\end{equation}
	Hence, \eqref{eqn:induction over t} is still valid for $\wS_{t+1}^{(k)}$ and $\bfX_{t+1}^{(k)}$.\\
	\textbf{Induction over $k$: } Suppose at $k^{th}$ stage, the initialization $\bfX_0^{(k)}$ and $\wS_{0}^{(k)}$ satisfy \eqref{eqn:induction over t}, that is
	\begin{equation}\label{eqn:induction over k}
	\begin{split}
	&\|\wS_{0}^{(k)}-\bfS^*\|_{\infty}\le\frac{7\tmu\tr}{n_cn}\Big(|{\lambda}_{k+1}^*|+2|{\lambda}_{k}^*|\Big),\\
	&\quad Supp(\wS_{0}^{(k)}-\bfS^*)\subseteq Supp(\bfS^*),\\
	and & \quad\|\bfX_{0}^{(k)}-\bfX^*\|_{\infty}\le\frac{2\tmu\tr}{n_cn}\Big(|{\lambda}_{k+1}^*|+2|{\lambda}_{k}^*|\Big).
	\end{split}
	\end{equation}
	From the discussion of induction over $t$ above, we know that,
	\begin{equation}\label{eqn:end of induction over k}
	\begin{split}
	&\|\wS_T^{(k)}-\bfS^*\|_{\infty}\le\frac{7\tmu\tr}{n_cn}\Big(\sL[T-1]\Big),\\
	&\quad Supp(\wS_T^{(k)}-\bfS^*)\subseteq Supp(\bfS^*),\\
    &\|\bfX_{T+1}^{(k)}-\bfX^*\|_{\infty}\le\frac{2\tmu\tr}{n_cn}\Big(\sL[T]\Big),
	\end{split}
	\end{equation} 
	where $T=\log({4\tmu\tr|\lambda_{1}^*|}/{\varepsilon})$.\\
	Then, from Lemmas \ref{Lemma:com_singular value} and \ref{Lemma:sigma_W}, we have 
	\begin{equation}\label{eqn: temp1}
	\begin{split}
	\Big||{\lambda}_{k+1}^{(T)}|-|{\lambda}_{k+1}^*|\Big|
	\le&\|\bfH_T\|_2\\
	\le&\frac{1}{60}\Big(\sL[T-1]\Big)\\
	\le&\frac{1}{60}\Big(|\lambda_{k+1}^*|+\frac{\varepsilon}{2\tmu\tr}\Big).
	\end{split}
	\end{equation}
	Now, we consider two cases,\\
	\textbf{Case 1}: if $\frac{\eta}{\sqrt{n_c}}{\lambda}_{k+1}^{(T)}\le\frac{\varepsilon}{n_cn}$, \eqref{eqn: temp1} implies that $|{\lambda}_{k+1}^*|\le\frac{\varepsilon}{2\tmu\tr}$. Hence,
	\begin{equation}
	\|\bfX_{T+1}^{(k)}-\bfX^*\|_{\infty}\le\frac{2\tmu\tr}{n_cn}\Big(\sL[T]\Big)\le\frac{\varepsilon}{n_cn}.
	\end{equation}
	Similar results can be established for $\bfS_T$. Therefore, $\bfX=\bfX_{T+1}^{(k)}$ and $\bfS=\bfS_{T}^{(k)}$ satisfy \eqref{eqn:THM1} in Theorem \ref{Theorem: bad data and missing data}.\\
	\textbf{Case 2:} if $\frac{\eta}{\sqrt{n_c}}{\lambda}_{k+1}^{(T)}>\frac{\varepsilon}{n_cn}$, then \eqref{eqn: temp1} implies that $|{\lambda}^*_{k+1}|\ge\frac{\varepsilon}{6\tmu\tr}$. Hence, 
	\begin{equation}
	\begin{split}
	\|\bfX_{T+1}^{(k)}-\bfX^*\|_{\infty}
	\le&\frac{2\tmu\tr}{n_cn}\Big(\sL[T]\Big)\\
	\le&\frac{2\tmu\tr}{n_cn}\Big(|{\lambda}_{k+1}^*|+\frac{\varepsilon}{4\tmu\tr}\Big)\\
	\le&\frac{2\tmu\tr}{n_cn}\Big(|{\lambda}_{k+2}^*|+2|{\lambda}_{k+1}^*|\Big).\\
	\end{split}
	\end{equation}
	Suppose we have an extra step $$\wS_{T+1}^{(k)}=\mathcal{T}_{\xi_{T+1}^{(k)}}(\bfM-\bfX_{T+1}).$$
	 From Lemma \ref{Lemma: S}, we have
	\begin{equation}
	\begin{split}
	\|\wS_{T+1}^{(k)}-\bfS^*\|_{\infty}
	\le&\frac{7\tmu\tr}{n_cn}\Big(|{\lambda}_{k+2}^*|+2|{\lambda}_{k+1}^*|\Big),\\
	\quad Supp(\wS_{T+1}^{(k)}&-\bfS^*)\subseteq Supp(\bfS^*).
	\end{split}
	\end{equation}
	$\bfX_{0}^{(k+1)}=\bfX_{T+1}^{(k)}$ is clear from Alg. \ref{Alg}, and it can also be verified that $\wS_{0}^{(k+1)}=\wS_{T+1}^{(k)}$. Hence, $\bfX_{0}^{(k+1)}$ and $\wS_{0}^{(k+1)}$ satisfy \eqref{eqn:induction over t} as well.\\
	Since $\wcH\bfX^*$ is at most rank-$2r$,  we will meet the terminating condition anyway. If not, from case 2, the contradiction arises
	\begin{equation}
	0=|\lambda^*_{r+1}|\ge\frac{\varepsilon}{6\tmu\tr}>0.
	\end{equation}
	Hence, the algorithm has at most $r\cdot T$ iterations, and we need the size of samplings satisfies 
	\begin{equation}
	m\ge rT\widehat{m}\ge\max(C_4, C_5)\mu^2 c_s^2 r^3\log^2(n)T,
	\end{equation}
	where the requirement on  $\widehat{m}$ comes from Lemmas \ref{Lemma:sigma_W} and \ref{Lemma: L}.
\end{proof}

%% file: appendix/c3p1.tex
\section{Roadmap of the Proofs}

Some high-level ideas in proving Theorem \ref{Thm: major_thm} are summarized first before presenting the technical proof. Following the recent line of research such as \cite{ZSD17,ZSJB17}, the idea is to initialize the weights $\bfW$ near the ground-truth $\bfW^*$ and then gradually converge to it.  Our initialization is similar to \cite{ZSJB17}. However, our proof is more involved than that of \cite{ZSJB17} to handle the additional noise item, the non-smooth ReLU functions, the additional momentum term in accelerated gradient descent, and different neural network structures.

As for the convergence analysis, 
refs.~\cite{ZSD17,ZSJB17}   apply the intermediate value theorem over $\nabla\hat{f}_{\mathcal{D}_t}$ at each iterate $\bfW_t$ as 
$$\nabla \hat{f}_{\mathcal{D}_t}(\bfW^{(t)}) \simeq \langle \nabla^2\hat{f}_{\mathcal{D}_t}(\widehat{\bfW}^{(t)}), \bfW^{(t)}-\bfW^* \rangle$$ 
for some $\widehat{\bfW}^{(t)}$ between ${\bfW}^{(t)}$ and $\bfW^*$ and analyze $\nabla^2\hat{f}_{\mathcal{D}_t}$ to obtain a recursive inequality of $\W[t]-\bfW^*$ over $t$.
The intermediate value theorem only applies to the continuous functions, and their analyses do not extend to our setup because with the ReLU activation function, the resulting $\nabla\hat{f}_\mathcal{D}$ is non-continuous. 
Instead,  we will first prove that the population loss function $f$, which is defined as
\begin{equation}\label{eqn: exp_risk}
\begin{split}
{f}(\bfW)
:=&\mathbb{E}_{\mathcal{D}_t} \hat{f}_{\mathcal{D}_t}(\bfW)
=\mathbb{E}_{\bfx}\bigg(\frac{1}{K}\sum_{j=1}^{K}\sum_{i=1}^{M}\phi(\bfw_j^T\bfP_i\bfx)-y\bigg)^2,
\end{split}
\end{equation}
is locally convex near $\bfW^*$, and the gradient  of $\hat{f}_{\mathcal{D}_t}$ is close enough to   $\nabla f$. We will then show that the iterates  based on $\nabla \hat{f}_{\mathcal{D}_t}$ converge  to $\bfW^*$. 

The following two lemmas are important for our proof, and their proofs can be found in \cite{ZWLC20_3}. 

{\begin{lemma}\label{Lemma: local_convexity}
	For any $\bfW$ that satisfies
	\begin{equation}\label{eqn:initialization}
	\|\bfW-\bfW^*\|_2\le \frac{\varepsilon_0\sigma_K}{44\kappa^2\gamma M},
	\end{equation}
	 we have 
	\begin{equation}
	\frac{(1-\varepsilon_0)M}{11\kappa^2\gamma}\bfI \le \nabla^2 f(\bfW) \le {6M^2}{K} \bfI.
	\end{equation}
\end{lemma}
\begin{lemma}\label{Lemma: sampling_approximation_error}
	Suppose a fixed point $\bfW$ satisfies \eqref{eqn:initialization}. Then, for a training set $\mathcal{D}$ with $N>d\log d$ samples, we have 
	\begin{equation}\label{eqn: lemma3}
	\begin{split}
		&\left\|\nabla f(\bfW)-\nabla \hat{f}_{\mathcal{D}}(\bfW)\right\|_2
		\lesssim MK \sqrt{\frac{d\log d}{N}}\Big( MK\|\bfW-\bfW^* \|_2 + |\xi|\Big),
	\end{split}
	\end{equation}
	with probability at least $1-K^2M^2\cdot d^{-10}$.
\end{lemma}
}

Lemma \ref{Lemma: local_convexity} shows that the population loss function $f(\bfW)$ is locally convex near $\bfW^*$.
Then, the analysis of AGD algorithm over the empirical loss function  $\hat{f}_{\mathcal{D}}(\bfW)$ is based on the analysis over   $f(\bfW)$ and the error bound between $\nabla\hat{f}_{\mathcal{D}}(\bfW)$ and $\nabla f(\bfW)$ as shown in \eqref{eqn: major_thm_key_eqn}. 

Lemma \ref{Lemma: sampling_approximation_error} describes the error bound between $\nabla f(\bfW)$ and $\nabla\hat{f}_{\mathcal{D}}(\bfW)$, and \eqref{eqn: lemma3} shows that $\nabla\hat{f}_{\mathcal{D}}(\bfW)$ converges to $\nabla f(\bfW)$ in a small neighborhood of $\bfW^*$ when $N$ is large enough. 
A similar result is stated in  Lemma 5.3 of \cite{ZYWG18} 
 for fully connected neural networks with ReLU activation function. 
 {Fully connected neural networks can be viewed as a special kind of convolutional neural networks with $M=1$. Moreover, even when reducing our model to the case $M=1$, the error bound presented in \eqref{eqn: lemma3} is much tighter than that in Lemma 5.3 of \cite{ZYWG18}. }

%

Combining Lemmas \ref{Thm: initialization}, \ref{Lemma: local_convexity}  and \ref{Lemma: sampling_approximation_error}, we will  show the convergence of GD in solving \eqref{eqn: optimization_problem} by mathematical induction. 
Conditioned on the assumption that $\bfW^{(t)}$ satisfies \eqref{eqn:initialization},  we show that $\|\bfW^{(t+1)}-\bfW^*\|_2$ is related to $\|\bfW^{(t)}-\bfW^*\|_2$ by \eqref{eqn: induction}.
The acceleration of Heavy-Ball steps is analyzed through \eqref{eqn:heavy_ball_beta}, and the result is summarized in \eqref{eqn: Heavy_ball_result}.
The next step is to show $\eqref{eqn: induction}$ holds for all $0\le t \le T-1$.
 By Lemma \ref{Thm: initialization}, we can choose $N$ to be large enough so that 
 $\bfW^{(0)}$ satisfies \eqref{eqn:initialization}. Then in the induction step, with a large enough $N$ and 
 a bounded $\xi$, we will show that $\|\bfW^{(t+1)}-\bfW^* \|_2< \|\bfW^{(t)}-\bfW^*\|_2$.  Then $\W[t]$ satisfies \eqref{eqn:initialization} naturally.  The details are as follows. 

\section{Proof of Theorem \ref{Thm: major_thm}}
\begin{proof}[Proof of Theorem \ref{Thm: major_thm}]

	The update rule of $\bfW^{(t)}$ is
	\begin{equation}
		\begin{split}
		\bfW^{(t+1)}
		=&\bfW^{(t)}-\eta\nabla \hat{f}_{\mathcal{D}_t}(\bfW^{(t)})+\beta(\W[t]-\W[t-1])\\
		=&\bfW^{(t)}-\eta\nabla {f}(\bfW^{(t)})+\beta(\W[t]-\W[t-1])\\
		&+\eta(\nabla{f}(\bfW^{(t)})-\nabla\hat{f}_{\mathcal{D}_t}(\bfW^{(t)})).
		\end{split}
	\end{equation}
	Since $\nabla^2{f}$ is a smooth function, by the intermediate value theorem, we have
	\begin{equation}
		\begin{split}
		\bfW^{(t+1)}
		=\bfW^{(t)}&- \eta\nabla^2{f}(\widehat{\bfW}^{(t)})(\bfW^{(t)}-\bfW^*)
		+\beta(\W[t]-\W[t-1])\\
		&+\eta\big(\nabla{f}(\bfW^{(t)})-\nabla\hat{f}_{\mathcal{D}_t}(\bfW^{(t)})\big),
		\end{split}
	\end{equation}	
	where $\widehat{\bfW}^{(t)}$ lies in the convex hull of $\bfW^{(t)}$ and $\bfW^*$.
	
	Next, we have
	\begin{equation}\label{eqn: major_thm_key_eqn}
	\begin{split}
	\begin{bmatrix}
	\W[t+1]-\bfW^*\\
	\W[t]-\bfW^*
	\end{bmatrix}
	=&\begin{bmatrix}
	\bfI-\eta\nabla^2{f}(\widehat{\bfW}^{(t)})+\beta\bfI &\beta\bfI\\
	\bfI& 0\\
	\end{bmatrix}
	\begin{bmatrix}
	\W[t]-\bfW^*\\
	\W[t-1]-\bfW^*
	\end{bmatrix}\\
	&+\eta
	\begin{bmatrix}
	\nabla{f}(\bfW^{(t)})-\nabla\hat{f}_{\mathcal{D}_t}(\bfW^{(t)})\\
	0
	\end{bmatrix}.
	\end{split}
	\end{equation}
	Let $\bfA(\beta)=\begin{bmatrix}
	\bfI-\eta\nabla^2{f}(\widehat{\bfW}^{(t)})+\beta\bfI &\beta\bfI\\
	\bfI& 0\\
	\end{bmatrix}$, so we have
	\begin{equation*}
	\begin{split}
	\left\|\begin{bmatrix}
	\W[t+1]-\bfW^*\\
	\W[t]-\bfW^*
	\end{bmatrix}
	\right\|_2
	=&
	\left\|\bfA(\beta)
	\right\|_2
	\left\|
	\begin{bmatrix}
	\W[t]-\bfW^*\\
	\W[t-1]-\bfW^*
	\end{bmatrix}
	\right\|_2
	+\eta
	\left\|
	\begin{bmatrix}
	\nabla{f}(\bfW^{(t)})-\nabla\hat{f}_{\mathcal{D}_t}(\bfW^{(t)})\\
	0
	\end{bmatrix}
	\right\|_2.
	\end{split}
	\end{equation*}
	From Lemma \ref{Lemma: sampling_approximation_error}, we know that 
	\begin{equation}\label{eqn: convergence2}
	\begin{split}
	&\eta\left\|\nabla{f}(\bfW^{(t)})-\nabla\hat{f}_{\mathcal{D}_t}(\bfW^{(t)})\right\|_2
	\le C_5\eta M^2\sqrt{\frac{d\log d}{N_t}}\Big(\|\bfW-\bfW^* \|_2 +\frac{|\xi|}{M}\Big)
	\end{split}
	\end{equation}
	for some constant $C_5>0$. Then, we have
	\begin{equation}\label{eqn: convergence}
	\begin{split}
	&\|\W[t+1]-\bfW^* \|_2 \\
	\le& \bigg(\| \bfA(\beta) \|_2+C_5\eta M^2\sqrt{\frac{d\log d}{N_t}}\bigg)\|\W[t]-\bfW^* \|_2 
	+ C_5\eta M\sqrt{\frac{d\log d}{N_t}}{|\xi|}\\
	:=&\nu(\beta)\|\W[t]-\bfW^* \|_2 + C_5\eta M\sqrt{\frac{d\log d}{N_t}}{|\xi|}.
	\end{split}
	\end{equation}
	Let $\nabla^2f(\widehat{\bfW}^{(t)})=\bfS\mathbf{\Lambda}\bfS^{T}$ be the eigendecomposition of $\nabla^2f(\widehat{\bfW}^{(t)})$. Then, we define
	\begin{equation}\label{eqn:Heavy_ball_eigen}
		\begin{split}
		\widetilde{\bfA}(\beta)
		: =&
		\begin{bmatrix}
		\bfS^T&\bf0\\
		\bf0&\bfS^T
		\end{bmatrix}
		\bfA(\beta)
		\begin{bmatrix}
		\bfS&\bf0\\
		\bf0&\bfS
		\end{bmatrix}
		=\begin{bmatrix}
		\bfI-\eta\bf\Lambda+\beta\bfI &\beta\bfI\\
		\bfI& 0\\
		\end{bmatrix}.
		\end{split}
	\end{equation}
	Since 
	$\begin{bmatrix}
	\bfS&\bf0\\
	\bf0&\bfS
	\end{bmatrix}\begin{bmatrix}
	\bfS^T&\bf0\\
	\bf0&\bfS^T
	\end{bmatrix}
	=\begin{bmatrix}
	\bfI&\bf0\\
	\bf0&\bfI
	\end{bmatrix}$, we know $\bfA(\beta)$ and $\widetilde{\bfA}(\beta)$
	share the same eigenvalues. 	Let $\lambda_i$ be the $i$-th eigenvalue of $\nabla^2f(\widehat{\bfW}^{(t)})$, then the corresponding $i$-th eigenvalue of $\bfA(\beta)$, denoted by $\delta_i(\beta)$, satisfies 
	\begin{equation}\label{eqn:Heavy_ball_quaratic}
	\delta_i^2-(1-\eta \lambda_i+\beta)\delta_i+\beta=0.
	\end{equation}
	Then, we have 
	\begin{equation}\label{eqn: Heavy_ball_root}
	\delta_i(\beta)=\frac{(1-\eta \lambda_i+\beta)+\sqrt{(1-\eta \lambda_i+\beta)^2-4\beta}}{2},
	\end{equation}
	and
	\begin{equation}\label{eqn:heavy_ball_beta}
	\begin{split}
	|\delta_i(\beta)|
	=\begin{cases}
	\sqrt{\beta}, \qquad \text{if}\quad  \beta\ge \big(1-\sqrt{\eta\lambda_i}\big)^2,\\
	\frac{1}{2} \left| { (1-\eta \lambda_i+\beta)+\sqrt{(1-\eta \lambda_i+\beta)^2-4\beta}}\right| , \text{otherwise}.
	\end{cases}
	\end{split}
	\end{equation}
	Note that the other root of \eqref{eqn:Heavy_ball_quaratic} is abandoned because the root in \eqref{eqn: Heavy_ball_root} is always no less than the other root with $|1-\eta \lambda_i|<1$.   
	By simple calculations, we have 
	\begin{equation}\label{eqn: Heavy_ball_result}
		\delta_i(0)>\delta_i(\beta),\quad \text{for}\quad  \forall \beta\in\big(0, (1-{\eta \lambda_i})^2\big).
	\end{equation}
	Moreover, $\delta_i$ achieves the minimum $\delta_{i}^*=|1-\sqrt{\eta\lambda_i}|$ when $\beta= \big(1-\sqrt{\eta\lambda_i}\big)^2$.
	
	Let us first assume $\bfW^{(t)}$ satisfies \eqref{eqn:initialization}, then
	from Lemma \ref{Lemma: local_convexity}, we know that 
	$$0<\frac{(1-\varepsilon_0)M}{11\kappa^2\gamma}\le\lambda_i\le6M^2K.$$
	Let $\gamma_1=\frac{(1-\varepsilon_0)M}{11\kappa^2\gamma}$ and $\gamma_2=6KM^2$.
	If we choose $\beta$ such that
	\begin{equation}
	\beta^*=\max \big\{(1-\sqrt{\eta\gamma_1})^2, (1-\sqrt{\eta\gamma_2})^2 \big\},
	\end{equation} then we have $\beta\ge (1-\sqrt{\eta \lambda_i})^2$ and $\delta_i=\max\big\{|1-\sqrt{\eta\gamma_1}|, |1-\sqrt{\eta\gamma_2}|  \big\}$ for any $i$. 
	
	Let $\eta= \frac{1}{{2\gamma_2}}$, then $\beta^*$ equals to  $\Big(1-\sqrt{\frac{\gamma_1}{2\gamma_2}}\Big)^2$. 
	Then, for any $\varepsilon_0\in (0, {1}/{2})$, we have
	\begin{equation}\label{eqn: con1}
	\begin{split}
		\|\bfA(\beta^*) \|_2
		=\max_i \delta_{i}(\beta^*)
		= 1-\sqrt{\frac{\gamma_1}{2\gamma_2}}
		=& 1-\sqrt{\frac{1-\varepsilon_0}{132\kappa^2 \gamma KM}}
		\le 1-\frac{1-(3/4)\cdot\varepsilon_0}{\sqrt{132\kappa^2 \gamma KM}}.
	\end{split}
	\end{equation}
	
	Then, let 
	\begin{equation}\label{eqn: con2}
			C_5\eta M^2\sqrt{\frac{d\log d}{N_t}}\le \frac{\varepsilon_0}{4\sqrt{132\kappa^2 \gamma KM}},
	\end{equation}
	we need $N_t\gtrsim\varepsilon_0^{-2}\kappa^2\gamma MK^3d \log d$.		
	
	Combining \eqref{eqn: con1} and \eqref{eqn: con2}, we have 
	\begin{equation}
			\nu(\beta^*)\le 1-\frac{1-\varepsilon_0}{\sqrt{132\kappa^2 \gamma KM}}.
		\end{equation}
		
	Let $\beta = 0$, we have 
	\begin{equation*}
		\begin{gathered}
			\nu(0) \ge \|\bfA(0) \|_2 = 1-\frac{1-\varepsilon_0}{{132\kappa^2 \gamma KM}},	\\
			\nu(0) \le \|\bfA(0) \|_2 + C_5\eta M^2\sqrt{\frac{d\log d}{N_t}} \le 1-\frac{1-2\varepsilon_0}{{132\kappa^2 \gamma KM}}
		\end{gathered}
	\end{equation*}
	if $N_t\gtrsim\varepsilon_0^{-2}\kappa^2\gamma M^2K^4d \log d$.
		
	Hence, with $\eta =\frac{1}{2 \gamma_2}$ and $\beta = \big(1-\frac{\gamma_1}{2\gamma_2} \big)^2$, we have
	\begin{equation}\label{eqn: induction}
		\begin{split}
			\|	\W[t+1]-\bfW^*\|_2
			\le&\Big(1-\frac{1-\varepsilon_0}{\sqrt{132\kappa^2 \gamma KM}}\Big)\|	\W[t]-\bfW^*\|_2
			+ 2C \eta M\sqrt{\frac{d\log d}{N_t}} |\xi|,
		\end{split}
	\end{equation}
	provided that $\W[t]$ satisfies \eqref{eqn:initialization}, and 
	\begin{equation}\label{eqn: N_3}
	N_t \gtrsim \varepsilon_0^{-2}\kappa^2\gamma MK^3d \log d.
	\end{equation}
	Then, we can start mathematical induction of \eqref{eqn: induction} over $t$.
			
	\textbf{Base case}: According to Lemma \ref{Thm: initialization}, we know that \eqref{eqn:initialization} holds for $\bfW^{(0)}$ if 
	\begin{equation}\label{eqn: N_1}
	N\gtrsim \varepsilon_0^{-2}\kappa^{9}\gamma^2 K^8M^2 d \log^4 d.
	\end{equation}
	According to  \eqref{eqn: sample_complexity} in Theorem 1, it is clear that the number of samples $N$  satisfies \eqref{eqn: N_1}, then  \eqref{eqn:initialization} indeed holds for $t=0$. Since \eqref{eqn:initialization} holds for $t=0$ and $N$ in \eqref{eqn: sample_complexity} satisfies \eqref{eqn: N_3} as well, we have \eqref{eqn: induction} holds for $t=0$.
		
	\textbf{Induction step}:
	Assuming \eqref{eqn: induction} holds for $\W[t]$, we need to show that \eqref{eqn: induction} holds for $\W[t+1]$. That is to say, we need (i) $N$ satisfies \eqref{eqn: N_3}; (ii) \eqref{eqn:initialization} holds for $\W[t+1]$. The requirement (i) holds naturally from \eqref{eqn: sample_complexity}. To guarantee (ii) holds, we need
	\begin{equation}
			\eta M\sqrt{\frac{d\log d}{N_t}}
			\lesssim  \frac{1-\varepsilon_0}{\sqrt{132\kappa^2 \gamma KM}}\cdot  \frac{\varepsilon_0\sigma_K}{44\kappa^2\gamma K^2M}.
	\end{equation}	
	That requires
	\begin{equation}\label{eqn: N_2}
		N_t \gtrsim \varepsilon_0^{-2} \kappa^8 \gamma^3 M^3K^6d\log d.
	\end{equation}
	Therefore, when $N_t \gtrsim \varepsilon_0^{-2} \kappa^9 \gamma^3 M^3K^8d\log^4 d$, we know that \eqref{eqn: induction} holds for all $0\le t \le T-1$ with probability at least $1-K^2M^2T\cdot d^{-10}$. By simple calculations, we can obtain
	\begin{equation}\label{eqn: Thm1_final}
		\begin{split}
			\|	\W[T]-\bfW^*\|_2
		\le&\Big(1-\frac{1-\varepsilon_0}{\sqrt{132\kappa^2 \gamma KM}}\Big)^T\|	\W[0]-\bfW^*\|_2\\
		& + C_4 \sqrt{\frac{\kappa^2\gamma MK^2d\log d}{N_t}} \cdot |\xi|
		\end{split}
	\end{equation}
	for some constant $C_4> 0$.
	\end{proof}

%% file: appendix/c4p1.tex
\section{Roadmap of Proof of Theorem \ref{Thm: major_thm_lr}}
In this section, before presenting the proof of Theorem \ref{Thm: major_thm_lr}, we start with defining some useful notations.
Recall that in \eqref{eqn:linear_regression}, the empirical risk function for linear regression problem is defined as 
	\begin{equation}\label{eqn_app: linear_regression}
	\min_{\bfW}: \quad \hat{f}_{\Omega}(\bfW)=\frac{1}{2|\Omega|}\sum_{n\in \Omega}\Big| y_n - g(\bfW;\bfa_n^T\bfX) \Big|^2.
	\end{equation}
Population risk function, which is the expectation of the empirical risk function, is defined as
\begin{equation}\label{eqn_app: linear_regression_exp}
	\min_{\bfW}: \quad {f}_{\Omega}(\bfW)=\mathbb{E}_{\bfX}\frac{1}{2|\Omega|}\sum_{n\in \Omega}\Big| y_n - g(\bfW;\bfa_n^T\bfX) \Big|^2.
	\end{equation}
Then, the road-map of the proof can be summarized in the following three steps. 

First, we show the Hessian matrix of the population risk function $f_{\Omega_t}$ is positive-definite at ground-truth parameters $\bfW^*$ and then characterize the local convexity region of $f_{\Omega_t}$ near $\bfW^*$, which is summarized in Lemma \ref{Lemma: local_convexity}. 

Second, 
$\hat{f}_{\Omega_t}$ is non-smooth because of ReLU activation, but $f_{\Omega_t}$ is smooth.
Hence, we characterize the  gradient descent term as $\nabla \hat{f}_{\Omega_t}(\bfW^{(t)}) = \langle \nabla^2f_{\Omega_t}(\widehat{\bfW}^{(t)}), \bfW^{(t)}-\bfW^* \rangle + \big(\hat{f}_{\Omega_t}({\bfW}^{(t)}) - f_{\Omega_t}({\bfW}^{(t)})\big)$. During this step, we need to apply concentration theorem to bound $\nabla\hat{f}_{\Omega_t}$  to its expectation $\nabla f_{\Omega_t}$ , which is summarized in Lemma \ref{Lemma: sampling_approximation_error}.

Third, we take the momentum term of $\beta(\bfW^{(t)}-\bfW^{(t-1)})$ into consideration and obtain the following recursive rule:
\begin{equation}\label{eqn: iteeee}
\begin{bmatrix}
\W[t+1]-\bfW^*\\
\W[t]-\bfW^*
\end{bmatrix}\\
=\bfL(\beta)
\begin{bmatrix}
\W[t]-\bfW^*\\
\W[t-1]-\bfW^*
\end{bmatrix}.
\end{equation}
Then, we know iterates $\bfW^{(t)}$ converge to the ground-truth with a linear rate which is the largest singlar value of matrix $\bfL(\beta)$. Recall that  AGD reduces to GD with $\beta=0$, so our analysis applies to GD method as well. We are able to show the convergence rate of AGD is faster than GD by proving the largest singluar value of $\bfL(\beta)$ is smaller than $\bfL(0)$ for some $\beta>0$.  Lemma \ref{Thm: initialization} provides the estimation error of $\bfW^{(0)}$ and sample complexity to guarantee $\|\bfL(\beta)\|_2$ is less than $1$ for $t=0$.

\begin{lemma}\label{Lemma: local_convexity}
	Let $f_{\Omega_t}$ be the population risk function in \eqref{eqn_app: linear_regression_exp} for regression problems, then for any $\bfW$ that satisfies
	\begin{equation}\label{eqn: initial_point_lr}
	\|\bfW^* -\bfW\|_2 \le \frac{\varepsilon_0 \sigma_K}{44\kappa^2\gamma K^2},
	\end{equation}
	the second-order derivative of $f_{\Omega_t}$ is bounded as 
	\begin{equation}
	\frac{(1-\varepsilon_0)\sigma_1^2(\bfA)}{11\kappa^2\gamma K^2}\bfI \preceq \nabla^2f_{\Omega_t}(\bfW)\preceq \frac{4\sigma_1^2(\bfA)}{K} \bfI.
	\end{equation}
\end{lemma}

\begin{lemma}\label{Lemma: sampling_approximation_error}
	Let $\hat{f}_{\Omega_t}$ and $f_{\Omega_t}$ be the  empirical and population risk functions in \eqref{eqn_app: linear_regression} and \eqref{eqn_app: linear_regression_exp} for regression problems, respectively. 
	Then, for any fixed point $\bfW$ satisfies \eqref{eqn: initial_point_lr}, we have \footnote{We use $f(d) \gtrsim(\text{ or }\lesssim, \eqsim ) g(d)$ to denote there exists some positive constant $C$ such that $f(d)\ge(\text{ or } \le, =) C\cdot g(d)$ when $d$ is sufficiently large.} 
	\begin{equation}\label{eqn: lemma3}
	\left\|\nabla f_{\Omega_t}(\bfW)-\nabla \hat{f}_{\Omega_t}(\bfW)\right\|_2
	\lesssim \sigma_1^2(\bfA)\sqrt{\frac{(1+\delta^2)d\log N}{|\Omega_t|}}\|\bfW-\bfW^* \|_2,
	\end{equation}
	with probability at least $1-K^2\cdot N^{-10}$.
\end{lemma}
\begin{lemma}\label{Thm: initialization}
	Assume the number of samples $|\Omega_t| \gtrsim \kappa^3(1+\delta^2)\sigma_1^4(\bfA)Kd\log^4 N$, the tensor initialization method via Subroutine 1 outputs $\bfW^{(0)}$ such that
	\begin{equation}\label{eqn: ini}
	\|\bfW^{(0)} -\bfW^* \|_2\lesssim \kappa^6 \sigma_1^2(\bfA)\sqrt{\frac{K^4(1+\delta^2)d\log N}{|\Omega_t|}}\|\bfW^*\|_2
	\end{equation}
	with probability at least $1-N^{-10}$.
\end{lemma}
With these three preliminary lemmas on hand, the proof of Theorem \ref{Thm: major_thm_lr} is formally summarized in the following contents. 

\section{Proof of Theorem \ref{Thm: major_thm_lr}}
\begin{proof}[Proof of Theorem \ref{Thm: major_thm_lr}]
	The update rule of $\bfW^{(t)}$ is
	\begin{equation}
	\begin{split}
	\bfW^{(t+1)}
	=&\bfW^{(t)}-\eta\nabla \hat{f}_{\Omega_t}(\bfW^{(t)})+\beta(\W[t]-\W[t-1])\\
	=&\bfW^{(t)}-\eta\nabla {f}_{\Omega_t}(\bfW^{(t)})+\beta(\W[t]-\W[t-1])+\eta(\nabla{f}_{\Omega_t}(\bfW^{(t)})-\nabla\hat{f}_{\Omega_t}(\bfW^{(t)})).
	\end{split}
	\end{equation}
	Since $\nabla^2_{\Omega_t}$ is a smooth function, by the intermediate value theorem, we have
	\begin{equation}
	\begin{split}
	\bfW^{(t+1)}
	=\bfW^{(t)}&- \eta\nabla^2{f}_{\Omega_t}(\widehat{\bfW}^{(t)})(\bfW^{(t)}-\bfW^*)\\
	&+\beta(\W[t]-\W[t-1])\\
	&+\eta\big(\nabla{f}_{\Omega_t}(\bfW^{(t)})-\nabla\hat{f}_{\Omega_t}(\bfW^{(t)})\big),
	\end{split}
	\end{equation}	
	where $\widehat{\bfW}^{(t)}$ lies in the convex hull of $\bfW^{(t)}$ and $\bfW^*$.
	
	Next, we have
	\begin{equation}\label{eqn: major_thm_key_eqn}
	\begin{split}
	\begin{bmatrix}
	\W[t+1]-\bfW^*\\
	\W[t]-\bfW^*
	\end{bmatrix}
	=&\begin{bmatrix}
	\bfI-\eta\nabla^2{f}_{\Omega_t}(\widehat{\bfW}^{(t)})+\beta\bfI &\beta\bfI\\
	\bfI& 0\\
	\end{bmatrix}
	\begin{bmatrix}
	\W[t]-\bfW^*\\
	\W[t-1]-\bfW^*
	\end{bmatrix}\\
	&+\eta
	\begin{bmatrix}
	\nabla{f}_{\Omega_t}(\bfW^{(t)})-\nabla\hat{f}_{\Omega_t}(\bfW^{(t)})\\
	0
	\end{bmatrix}.
	\end{split}
	\end{equation}
	Let $\bfL(\beta)=\begin{bmatrix}
	\bfI-\eta\nabla^2{f}_{\Omega_t}(\widehat{\bfW}^{(t)})+\beta\bfI &\beta\bfI\\
	\bfI& 0\\
	\end{bmatrix}$, so we have
	\begin{equation*}
	\begin{split}
	\left\|\begin{bmatrix}
	\W[t+1]-\bfW^*\\
	\W[t]-\bfW^*
	\end{bmatrix}
	\right\|_2
	=&
	\left\|\bfL(\beta)
	\right\|_2
	\left\|
	\begin{bmatrix}
	\W[t]-\bfW^*\\
	\W[t-1]-\bfW^*
	\end{bmatrix}
	\right\|_2
	+\eta
	\left\|
	\begin{bmatrix}
	\nabla{f}_{\Omega_t}(\bfW^{(t)})-\nabla\hat{f}_{\Omega_t}(\bfW^{(t)})\\
	0
	\end{bmatrix}
	\right\|_2.
	\end{split}
	\end{equation*}
	From Lemma \ref{Lemma: sampling_approximation_error}, we know that 
	\begin{equation}\label{eqn: convergence2}
	\begin{split}
	\eta\left\|\nabla{f}_{\Omega_t}(\bfW^{(t)})-\nabla\hat{f}_{\Omega_t}(\bfW^{(t)})\right\|_2
	\lesssim \eta \sigma_1^2(\bfA)\sqrt{\frac{(1+\delta^2)d\log N}{|\Omega_t|}}\|\bfW-\bfW^* \|_2.
	\end{split}
	\end{equation}
	Then, we have
	\begin{equation}\label{eqn: convergence}
	\begin{split}
	\|\W[t+1]-\bfW^* \|_2 
	\lesssim& \bigg(\| \bfL(\beta) \|_2+\eta\sigma_1^2(\bfA)\sqrt{\frac{(1+\delta^2)d\log N}{|\Omega_t|}}\bigg)\|\W[t]-\bfW^* \|_2\\
	:\eqsim&\nu(\beta)\|\W[t]-\bfW^* \|_2.
	\end{split}
	\end{equation}

	Let $\nabla^2f(\widehat{\bfW}^{(t)})=\bfS\mathbf{\Lambda}\bfS^{T}$ be the eigen-decomposition of $\nabla^2f(\widehat{\bfW}^{(t)})$. Then, we define
	\begin{equation}\label{eqn:Heavy_ball_eigen}
	\begin{split}
	\widetilde{\bfL}(\beta)
	: =&
	\begin{bmatrix}
	\bfS^T&\bf0\\
	\bf0&\bfS^T
	\end{bmatrix}
	\bfL(\beta)
	\begin{bmatrix}
	\bfS&\bf0\\
	\bf0&\bfS
	\end{bmatrix}
	=\begin{bmatrix}
	\bfI-\eta\bf\Lambda+\beta\bfI &\beta\bfI\\
	\bfI& 0\\
	\end{bmatrix}.
	\end{split}
	\end{equation}
	Since 
	$\begin{bmatrix}
	\bfS&\bf0\\
	\bf0&\bfS
	\end{bmatrix}\begin{bmatrix}
	\bfS^T&\bf0\\
	\bf0&\bfS^T
	\end{bmatrix}
	=\begin{bmatrix}
	\bfI&\bf0\\
	\bf0&\bfI
	\end{bmatrix}$, we know $\bfL(\beta)$ and $\widetilde{\bfL}(\beta)$
	share the same eigenvalues. 	Let $\lambda_i$ be the $i$-th eigenvalue of $\nabla^2f_{\Omega_t}(\widehat{\bfW}^{(t)})$, then the corresponding $i$-th eigenvalue of $\bfL(\beta)$, denoted by $\delta_i(\beta)$, satisfies 
	\begin{equation}\label{eqn:Heavy_ball_quaratic}
	\delta_i^2-(1-\eta \lambda_i+\beta)\delta_i+\beta=0.
	\end{equation}
	Then, we have 
	\begin{equation}\label{eqn: Heavy_ball_root}
	\delta_i(\beta)=\frac{(1-\eta \lambda_i+\beta)+\sqrt{(1-\eta \lambda_i+\beta)^2-4\beta}}{2},
	\end{equation}
	and
	\begin{equation}\label{eqn:heavy_ball_beta}
	\begin{split}
	|\delta_i(\beta)|
	=\begin{cases}
	\sqrt{\beta}, \qquad \text{if}\quad  \beta\ge \big(1-\sqrt{\eta\lambda_i}\big)^2,\\
	\frac{1}{2} \left| { (1-\eta \lambda_i+\beta)+\sqrt{(1-\eta \lambda_i+\beta)^2-4\beta}}\right| , \text{otherwise}.
	\end{cases}
	\end{split}
	\end{equation}
	Note that the other root of \eqref{eqn:Heavy_ball_quaratic} is abandoned because the root in \eqref{eqn: Heavy_ball_root} is always no less than the other root with $|1-\eta \lambda_i|<1$.   
	By simple calculations, we have 
	\begin{equation}\label{eqn: Heavy_ball_result}
	\delta_i(0)>\delta_i(\beta),\quad \text{for}\quad  \forall \beta\in\big(0, (1-{\eta \lambda_i})^2\big).
	\end{equation}
	Moreover, $\delta_i$ achieves the minimum $\delta_{i}^*=|1-\sqrt{\eta\lambda_i}|$ when $\beta= \big(1-\sqrt{\eta\lambda_i}\big)^2$.
	
	Let us first assume $\bfW^{(t)}$ satisfies \eqref{eqn: initial_point_lr}, then
	from Lemma \ref{Lemma: local_convexity}, we know that 
	$$0<\frac{(1-\varepsilon_0)\sigma_1^2(\bfA)}{11\kappa^2\gamma K^2}\le\lambda_i\le\frac{4\sigma_1^2(\bfA)}{K}.$$
	Let $\gamma_1=\frac{(1-\varepsilon_0)\sigma_1^2(\bfA)}{11\kappa^2\gamma K^2}$ and $\gamma_2=\frac{4\sigma_1^2(\bfA)}{K}$.
	If we choose $\beta$ such that
	\begin{equation}
	\beta^*=\max \big\{(1-\sqrt{\eta\gamma_1})^2, (1-\sqrt{\eta\gamma_2})^2 \big\},
	\end{equation} then we have $\beta\ge (1-\sqrt{\eta \lambda_i})^2$ and $\delta_i=\max\big\{|1-\sqrt{\eta\gamma_1}|, |1-\sqrt{\eta\gamma_2}|  \big\}$ for any $i$. 
	
	Let $\eta= \frac{1}{{2\gamma_2}}$, then $\beta^*$ equals to  $\Big(1-\sqrt{\frac{\gamma_1}{2\gamma_2}}\Big)^2$. 
	Then, for any $\varepsilon_0\in (0, {1}/{2})$, we have
	\begin{equation}\label{eqn: con1}
	\begin{split}
	\|\bfL(\beta^*) \|_2
	=\max_i \delta_{i}(\beta^*)
	= 1-\sqrt{\frac{\gamma_1}{2\gamma_2}}
	=1-\sqrt{\frac{1-\varepsilon_0}{88\kappa^2 \gamma K}}
	\le 1-\frac{1-(3/4)\cdot\varepsilon_0}{\sqrt{88\kappa^2 \gamma K}}.
	\end{split}
	\end{equation}
	Then, let 
	\begin{equation}\label{eqn: con2}
	\eta \sigma_1^2(\bfA)\sqrt{\frac{(1+\delta^2)d\log N}{|\Omega_t|}}\lesssim \frac{\varepsilon_0}{4\sqrt{88\kappa^2 \gamma K}},
	\end{equation}
	we need $|\Omega_t|\gtrsim\varepsilon_0^{-2}\kappa^2\gamma M(1+\delta^2)\sigma_1^2(\bfA)K^3d \log N$.		
	Combining \eqref{eqn: con1} and \eqref{eqn: con2}, we have 
	\begin{equation}
	\nu(\beta^*)\le 1-\frac{1-\varepsilon_0}{\sqrt{88\kappa^2 \gamma K}}.
	\end{equation}
	Let $\beta = 0$, we have 
	\begin{equation*}
	\begin{gathered}
	\nu(0) \ge \|\bfA(0) \|_2 = 1-\frac{1-\varepsilon_0}{{88\kappa^2 \gamma K}},	\\
	\nu(0) \lesssim \|\bfA(0) \|_2 + \eta \sigma_1^2(\bfA)\sqrt{\frac{(1+\delta^2)d\log N}{|\Omega_t|}} \le 1-\frac{1-2\varepsilon_0}{{88\kappa^2 \gamma K}}
	\end{gathered}
	\end{equation*}
	if $|\Omega_t|\gtrsim\varepsilon_0^{-2}\kappa^2\gamma M(1+\delta^2)\sigma_1^2(\bfA)K^3d \log N$.
	
	Hence, with $\eta =\frac{1}{2 \gamma_2}$ and $\beta = \big(1-\frac{\gamma_1}{2\gamma_2} \big)^2$, we have
	\begin{equation}\label{eqn: induction}
	\begin{split}
	\|	\W[t+1]-\bfW^*\|_2
	\le\Big(1-\frac{1-\varepsilon_0}{\sqrt{88\kappa^2 \gamma K}}\Big)\|	\W[t]-\bfW^*\|_2,
	\end{split}
	\end{equation}
	provided $\W[t]$ satisfies \eqref{eqn: initial_point_lr}, and 
	\begin{equation}\label{eqn: N_3}
	|\Omega_t|\gtrsim\varepsilon_0^{-2}\kappa^2\gamma (1+\delta^2)\sigma_1^4(\bfA)K^3d \log N.
	\end{equation}
	Then, we can start mathematical induction of \eqref{eqn: induction} over $t$.
	
	\textbf{Base case}: According to Lemma \ref{Thm: initialization}, we know that \eqref{eqn: initial_point_lr} holds for $\bfW^{(0)}$ if 
		\begin{equation}\label{eqn: N_1}
		|\Omega_t|\gtrsim \varepsilon_0^{-2}\kappa^{9}\gamma^2 (1+\delta^2)\sigma_1^4(\bfA) K^8d \log N.
		\end{equation}
		According to  Theorem \ref{Thm: major_thm_lr}, it is clear that the number of samples $|\Omega_t|$  satisfies \eqref{eqn: N_1}, then  \eqref{eqn: initial_point_lr} indeed holds for $t=0$. Since \eqref{eqn: initial_point_lr} holds for $t=0$ and $|\Omega_t|$ in Theorem \ref{Thm: major_thm_lr} satisfies \eqref{eqn: N_3} as well, we have \eqref{eqn: induction} holds for $t=0$.   
	
	\textbf{Induction step}:
		Assuming \eqref{eqn: induction} holds for $\W[t]$, we need to show that \eqref{eqn: induction} holds for $\W[t+1]$. That is to say, we need $|\Omega_t|$ satisfies \eqref{eqn: N_3}, which holds naturally from Theorem \ref{Thm: major_thm_lr}. 

		Therefore, when $|\Omega_t|\gtrsim \varepsilon_0^{-2}\kappa^{9}\gamma^2 (1+\delta^2)\sigma_1^4(\bfA)K^8 d \log N$, we know that \eqref{eqn: induction} holds for all $0\le t \le T-1$ with probability at least $1-K^2T\cdot N^{-10}$. By simple calculations, we can obtain
		\begin{equation}\label{eqn: Thm1_final}
		\begin{split}
		\|	\W[T]-\bfW^*\|_2
		\le&\Big(1-\frac{1-\varepsilon_0}{\sqrt{88\kappa^2 \gamma K}}\Big)^T\|	\W[0]-\bfW^*\|_2
		\end{split}
		\end{equation}
\end{proof}

%% file: appendix/c4p2.tex
\section{Roadmap of Proof of Theorem \ref{Thm: major_thm_cl}}
Recall that the empirical risk function in \eqref{eqn: classification} is defined as
	\begin{equation}\label{eqn_app: cl}
		\begin{split}
		\min_{\bfW}: \quad \hat{f}_{\Omega}(\bfW)
		=&\frac{1}{|\Omega|}\sum_{n\in \Omega} -y_n \log \big(g(\bfW;\bfa_n^T\bfX)\big)
		-(1-y_n) \log \big(1-g(\bfW;\bfa_n^T\bfX)\big).
		\end{split}
	\end{equation}
	The population risk function is defined as
\begin{equation}\label{eqn_app: cl_exp}
	\begin{split}
	f_{\Omega}(\bfW):=& \mathbb{E}_{\bfX,y_n}\hat{f}_{\Omega}(\bfW)\\
	=&\mathbb{E}_{\bfX}\mathbb{E}_{y_n|\bfX}\Big[\frac{1}{|\Omega|}\sum_{n\in \Omega}
	- y_n \log \big(g(\bfW;\bfa_n^T\bfX)\big)\\
	&-(1-y_n) \log \big(1-g(\bfW;\bfa_n^T\bfX)\big)\Big]\\
	=&\mathbb{E}_{\bfX}\frac{1}{|\Omega|}\sum_{n\in \Omega} - g(\bfW^*;\bfa_n^T\bfX) \log \big(g(\bfW;\bfa_n^T\bfX)\big)\\
	&-(1-g(\bfW^*;\bfa_n^T\bfX)) \log \big(1-g(\bfW;\bfa_n^T\bfX)\big).
	\end{split}
\end{equation}
The road-map of proof for Theorem \ref{Thm: major_thm_cl} follows the similar three steps as those for Theorem \ref{Thm: major_thm_lr}. The major differences lie  in three aspects: (i) in the second step,
the objective function $\hat{f}_{\Omega_t}$ is smooth since the activation function $\phi(\cdot)$ is sigmoid. Hence, we can directly apply the mean value theorem as $\nabla \hat{f}_{\Omega_t}(\bfW^{(t)}) = \langle \nabla^2\hat{f}_{\Omega_t}(\widehat{\bfW}^{(t)}), \bfW^{(t)}-\bfW^* \rangle$ to characterize the effects of the gradient descent term in each iteration, and the error bound of $\nabla^2\hat{f}_{\Omega_t}$ is provided in Lemma \ref{lemma: local_convex_cl}; 
(ii) the objective function is the sum of cross-entry loss functions, which have more complex structure of derivatives than those of square loss functions; 
(iii) as the convergent point may not be the critical point of empirical loss function,
we need to provide  the distance from the convergent point to the ground-truth parameters additionally, where Lemma \ref{lemma: first_order_cl} is used.

Lemmas \ref{lemma: local_convex_cl} and \ref{lemma: first_order_cl} are summarized in the following contents. Also, the notations $\lesssim$ and $\gtrsim$ follow the same definitions as in \eqref{eqn: lemma3}.
\begin{lemma}\label{lemma: local_convex_cl}
	For any $\bfW$ that satisfies
	\begin{equation}\label{eqn: initial_point_cl}
	\|\bfW-\bfW^*\| \le \frac{2\sigma_1^2(\bfA)}{11\kappa^2\gamma K^2}
	\end{equation}
	then the second-order derivative of the empirical risk function in \eqref{eqn_app: cl} for binary classification problems is bounded as 
	\begin{equation}
			\frac{2(1-\varepsilon_0)}{11\kappa^2 \gamma K^2}\sigma_1^2(\bfA) 
			\preceq \nabla^2\hat{f}_{\Omega_t}(\bfW)
			\preceq \sigma_1^2(\bfA).
	\end{equation}	
	provided the number of samples satisfies	
	\begin{equation}\label{eqn: sample_cl}
	|\Omega_t|\gtrsim \varepsilon_0^{-2}(1+\delta^2)\kappa^2\gamma \sigma_1^4(\bfA)K^6d\log N.
	\end{equation}
\end{lemma}

\begin{lemma}\label{lemma: first_order_cl}
	Let $\hat{f}_{\Omega_t}$ and ${f}_{\Omega_t}$ be the empirical and population risk function in \eqref{eqn_app: cl} and \eqref{eqn_app: cl_exp} for binary classification problems, respectively, then the first-order derivative of $\hat{f}_{\Omega_t}$ is close to its expectation ${f}_{\Omega_t}$ with an upper bound as
	\begin{equation}
	\begin{split}
	\|\nabla f_{\Omega_t}(\bfW) - \nabla \hat{f}_{\Omega_t}(\bfW) \|_2
	\lesssim K^2\sigma_1^2(\bfA)\sqrt{\frac{(1+\delta^2)d\log d}{|\Omega_t|}}
	\end{split}
	\end{equation}
	with probability at least $1-K^2N^{-10}$.
\end{lemma}
 With these preliminary lemmas, the proof of Theorem \ref{Thm: major_thm_cl} is formally summarized in the following contents.
 
 \section{Proof of Theorem \ref{Thm: major_thm_cl}}
\begin{proof}[Proof of Theorem \ref{Thm: major_thm_cl}]
	The update rule of $\bfW^{(t)}$ is
	\begin{equation}
	\begin{split}
	\bfW^{(t+1)}
	=&\bfW^{(t)}-\eta\nabla \hat{f}_{\Omega_t}(\bfW^{(t)})+\beta(\W[t]-\W[t-1])\\
	\end{split}
	\end{equation}
	Since $\widehat{\bfW}$ is a critical point, then we have  $\nabla \hat{f}_{\Omega_t}(\widehat{\bfW})=0$. By the intermediate value theorem, we have
	\begin{equation}
	\begin{split}
	\bfW^{(t+1)}
	=\bfW^{(t)}&- \eta\nabla^2\hat{f}_{\Omega_t}(\widehat{\bfW}^{(t)})(\bfW^{(t)}-\widehat{\bfW})\\
	&+\beta(\W[t]-\W[t-1])\\
	\end{split}
	\end{equation}	
	where $\widehat{\bfW}^{(t)}$ lies in the convex hull of $\bfW^{(t)}$ and $\widehat{\bfW}$.
	
	Next, we have
	\begin{equation}\label{eqn: major_thm_key_eqn_cl}
	\begin{split}
	\begin{bmatrix}
	\W[t+1]-\bfW^*\\
	\W[t]-\bfW^*
	\end{bmatrix}
	=&\begin{bmatrix}
	\bfI-\eta\nabla^2{\hat{f}}_{\Omega_t}(\widehat{\bfW}^{(t)})+\beta\bfI &\beta\bfI\\
	\bfI& 0\\
	\end{bmatrix}
	\begin{bmatrix}
	\W[t]-\bfW^*\\
	\W[t-1]-\bfW^*
	\end{bmatrix}
	.
	\end{split}
	\end{equation}
	Let $\bfP(\beta)=\begin{bmatrix}
	\bfI-\eta\nabla^2{\hat{f}}_{\Omega_t}(\widehat{\bfW}^{(t)})+\beta\bfI &\beta\bfI\\
	\bfI& 0\\
	\end{bmatrix}$, so we have
	\begin{equation*}
	\begin{split}
	\left\|\begin{bmatrix}
	\W[t+1]-\bfW^*\\
	\W[t]-\bfW^*
	\end{bmatrix}
	\right\|_2
	=&
	\left\|\bfP(\beta)
	\right\|_2
	\left\|
	\begin{bmatrix}
	\W[t]-\bfW^*\\
	\W[t-1]-\bfW^*
	\end{bmatrix}
	\right\|_2.
	\end{split}
	\end{equation*}
	Then, we have
	\begin{equation}\label{eqn: convergence_cl}
	\begin{split}
	\|\W[t+1]-\bfW^* \|_2 
	\lesssim& \| \bfP(\beta) \|_2 \|\W[t]-\bfW^* \|_2\\
	\end{split}
	\end{equation}	
Let $\lambda_i$ be the $i$-th eigenvalue of $\nabla^2\hat{f}_{\Omega_t}(\widehat{\bfW}^{(t)})$, and $\delta_{i}$ be the $i$-th eigenvalue of matrix $\bfP(\beta)$. Following the similar analysis in proof of Theorem \ref{Thm: major_thm_lr}, we have
	\begin{equation}\label{eqn: Heavy_ball_result_cl}
	\delta_i(0)>\delta_i(\beta),\quad \text{for}\quad  \forall \beta\in\big(0, (1-{\eta \lambda_i})^2\big).
	\end{equation}
	Moreover, $\delta_i$ achieves the minimum $\delta_{i}^*=|1-\sqrt{\eta\lambda_i}|$ when $\beta= \big(1-\sqrt{\eta\lambda_i}\big)^2$.
	
	Let us first assume $\bfW^{(t)}$ satisfies \eqref{eqn: initial_point_cl} and the number of samples satisfies \eqref{eqn: sample_cl}, then
	from Lemma \ref{lemma: local_convex_cl}, we know that 
	$$0<\frac{2(1-\varepsilon_0)\sigma_1^2(\bfA)}{11\kappa^2\gamma K^2}\le\lambda_i\le {\sigma_1^2(\bfA)}.$$
	We define $\gamma_1=\frac{2(1-\varepsilon_0)\sigma_1^2(\bfA)}{11\kappa^2\gamma K^2}$ and $\gamma_2={\sigma_1^2(\bfA)}$.
	Also, for any $\varepsilon_0\in (0, 1)$, we have
	\begin{equation}\label{eqn: con1_cl}
	\begin{split}
	\nu(\beta^*)= \|\bfP(\beta^*) \|_2
	= 1-\sqrt{\frac{\gamma_1}{2\gamma_2}}
	=1-\sqrt{\frac{1-\varepsilon_0}{11\kappa^2 \gamma K}}
	\end{split}
	\end{equation}
	Let $\beta = 0$, we have 
	\begin{equation*}
	\begin{gathered}
	\nu(0) = \|\bfA(0) \|_2 = 1-\frac{1-\varepsilon_0}{{11\kappa^2 \gamma K}}.	\\
	\end{gathered}
	\end{equation*}
	
	Hence, with probability at least $1- K^2\cdot N^{-10}$, we have
	\begin{equation}\label{eqn: induction_cl}
	\begin{split}
	\|	\W[t+1]-\bfW^*\|_2
	\le\Big(1-\sqrt{\frac{1-\varepsilon_0}{11\kappa^2 \gamma K}}\Big)\|	\W[t]-\bfW^*\|_2,
	\end{split}
	\end{equation}
	provided that $\W[t]$ satisfies \eqref{eqn: initial_point_lr}, and 
	\begin{equation}\label{eqn: N_3_cl}
	|\Omega_t|\gtrsim\varepsilon_0^{-2}\kappa^2\gamma (1+\delta^2)\sigma_1^4(\bfA)K^6d \log N.
	\end{equation}
	According to Lemma \ref{Thm: initialization}, we know that \eqref{eqn: initial_point_cl} holds for $\bfW^{(0)}$ if 
	\begin{equation}\label{eqn: N_1_cl}
	|\Omega_t|\gtrsim \varepsilon_0^{-2}\kappa^{8}\gamma^2 (1+\delta^2) K^8d \log N.
	\end{equation}
	Combining \eqref{eqn: N_3_cl} and \eqref{eqn: N_1_cl}, we need $|\Omega_t|\gtrsim \varepsilon_0^{-2}\kappa^{8}\gamma^2 (1+\delta^2)\sigma_1^4(\bfA) K^8d \log N$.	

	Finally,  by the mean value theorem, we have 
	\begin{equation}
	\hat{f}_{\Omega_t}(\widehat{\bfW}) \le \hat{f}_{\Omega_t}(\bfW^*) + {\nabla\hat{f}_{\Omega_t}(\bfW^*)}^T(\widehat{\bfW}-\bfW^*) +
	\frac{1}{2} (\widehat{\bfW}-\bfW^*)^T\nabla^2\hat{f}_{\Omega_t}(\widetilde{\bfW})(\widehat{\bfW}-\bfW^*)
	\end{equation}
	for some $\widetilde{\bfW}$ between $\widehat{\bfW}$ and $\bfW^*$.
	Since $\widehat{\bfW}$ is the local minima, we have $\hat{f}_{\Omega_t}(\widehat{\bfW})\le \hat{f}_{\Omega_t}(\bfW^*)$. That is to say 
	\begin{equation}
	{\nabla\hat{f}_{\Omega_t}(\bfW^*)}^T(\widehat{\bfW}-\bfW^*) +
	\frac{1}{2} (\widehat{\bfW}-\bfW^*)^T\nabla^2\hat{f}_{\Omega_t}(\widetilde{\bfW})(\widehat{\bfW}-\bfW^*)\le 0
	\end{equation}
	which implies
	\begin{equation}\label{eqn: cccccc3}
	\frac{1}{2} \|\nabla^2\hat{f}_{\Omega_t}(\widetilde{\bfW})\|_2\|\widehat{\bfW}-\bfW^*\|_2^2\le \|{\nabla\hat{f}_{\Omega_t}(\bfW^*)}\|_2\|\widehat{\bfW}-\bfW^*\|_2.
	\end{equation}
	From Lemma \ref{lemma: local_convex_cl}, we know that 
	\begin{equation}\label{eqn: cccccc1}
	\|\nabla^2\hat{f}_{\Omega_t}(\widetilde{\bfW})\|_2\ge \frac{2(1-\varepsilon_0)}{11\kappa^2\gamma K^2}\sigma^2(\bfA).
	\end{equation}
	From Lemma \ref{lemma: first_order_cl}, we know that
	\begin{equation}\label{eqn: cccccc2}
	\|{\nabla\hat{f}_{\Omega_t}(\bfW^*)} \|_2 = \| {\nabla\hat{f}_{\Omega_t}(\bfW^*)} - {\nabla{f}_{\Omega_t}(\bfW^*)}\|_2 \lesssim K^2 \sigma_1^2(\bfA)\sqrt{\frac{(1+\delta^2)d\log N}{|\Omega_t|}}.
	\end{equation}
	Plugging inequalities \eqref{eqn: cccccc1} and \eqref{eqn: cccccc2} back into \eqref{eqn: cccccc3}, we have
	\begin{equation}
	\|\widehat{\bfW} -\bfW^* \|_2 \lesssim (1-\varepsilon_0)^{-1}\kappa^2 \gamma K^4\sqrt{\frac{(1+\delta^2)d\log d}{|\Omega_t|}}.
	\end{equation}
	
\end{proof}

%% file: appendix/c5p1.tex
\section{Roadmap of Proofs}
We first provide an overview about techniques used in proving the landscape (Theorem \ref{Thm: convex_region}), linear convergence to the ground truth (Theorem \ref{Thm: major_thm}) and tensor initialization (Lemma \ref{Thm: initialization}). The full proofs can be found in \cite{ZWLCX21}.

{1. \textbf{Sample complexity scales in $\{r_j\}_{j=1}^K$}: To guarantee the theoretical bounds depend on $\{r_j\}_{j=1}^K$ instead of $d$, we define an equivalent empirical risk function as shown in \eqref{eqn: empirical_risk} in Appendix \ref{sec: notations},  from $\mathbb{R}^{\sum_j r_j}$ to $\mathbb{R}$. Existing concentration theorems and landscape analysis built upon \eqref{eqn: sample_risk} can no longer be used here, and thus we revised or updated the corresponding lemmas, which can be found in this appendix. In the initialization methods,  for estimating a proper weights that match new empirical risk function, the construction of high-momenta  and corresponding proofs are updated accordingly as well;}

{2. \textbf{Local convex region}: In proving Theorem \ref{Thm: convex_region} (Appendix \ref{sec: proof of theorem 1}), we 
first bound the Hessian of the expectation of the new empirical risk function
and then obtain the distance of the Hessian of the new empirical risk function to its expectation  by concentration theorem. By triangle inequality, the Hessian of the new empirical risk function is characterized in terms of sample size $N$;}

{3. \textbf{Linear Convergence}: In proving Theorem \ref{Thm: major_thm} (Appendix \ref{sec: Proof of Theorem 2}), we first characterize the gradient descent term by \textit{Intermediate Value Theorem} (IVT). However, since the empirical risk function is non-smooth due to the ReLU activation function,  IVT is applied in the expectation of the empirical risk function instead, and we later show the gradient generated by finite number of samples is close to its expectation. Therefore, the iterates still converge to the ground truth with enough samples. Further, the linear convergence rate are determined by $\|\W[t+1]-\bfW^*\bfP\|/\|\W[t]-\bfW^*\bfP\|$, which turns out to be dependent on $\beta$;}

4. \textbf{Initialization via Tensor Method}: The major challenge for tensor initialization is to construct the proper high dimensional momenta. As we mentioned above, if one directly applies the method in \cite{ZSJB17}, the sample complexity is in $\Theta(d)$. Here, we select $\widetilde{\bfx}$ (see \eqref{eqn: x_tilde}), which is the sum of the  augmented $\bfx_{\Omega_j}$.
In proving Lemma \ref{Thm: initialization}, the major idea to bound the estimations of the directions and magnitudes of $\bfw_{j,\Omega_j}$ to the ground values, respectively.   


\section{Notations}\label{sec: notations}
In this section, we first introduce some important notations that will be used in the following proofs, and the notations are
summarized  in Table 1.


First, for the convenience of proofs, some notations in main contexts, namely, $\Omega_j^*$, $r_j^*$ and $\hat{f}_{\mathcal{D}}$ will be re-defined. We emphasize here that the re-definition of these notations will not affect the presentation of theoretical results in Section \ref{sec:algorithm_and_theorem}, and the explanations can be found in the following paragraphs.  

Next, given a permutation matrix $\bfP$, we define a group of sets $\{\Omega_j^*\}_{j=1}^K$ with $|\Omega_j^{*}|=r_j^*$, and $\Omega_j^*$ denotes the indices of non-zero entries in $\bfM^*\bfP$, which is also the
non-pruned weights of the $j$-th neuron in the {oracle} model with respect to ground truth weights $\bfM^*\bfP$, instead of $\bfM^*$. Please note that the sets $\{\Omega_j^*\}_{j=1}^K$ and $\{r_j^*\}_{j=1}^K$ here are just a permutation of these in the main context. Since the permutation of $\{r_j\}_{j=1}^K$ will not change the results in Section \ref{sec:algorithm_and_theorem}, we abuse the notations for the convenience of proofs. 
 Correspondingly, for the {learner} model, the indices of non-pruned weights of the $j$-th neuron is denoted as $\Omega_j$, and $|\Omega_j|=r_j$. Therefore, we have 
\begin{equation}
    \bfw_j^T\bfx = \bfw_{j,\Omega_j}^T\bfx_{\Omega_j},
\end{equation}
where $\bfz_{\Omega_j}\in \mathbb{R}^{r_j}$ is the subvector of $\bfz$ with respect to indices $\Omega_j$ for any vector $\bfz\in\mathbb{R}^d$.

Then, recall the \textbf{\textit{empirical risk function}} defined in \eqref{eqn: sample_risk}, it can be re-written as 
\begin{equation}\label{eqn: empirical_risk}
    \hat{f}_{\mathcal{D}}(\widetilde{\bfw})
:=\frac{1}{2N}\sum_{n=1}^{N}\Big(\frac{1}{K}\sum_{j=1}^{K}\phi({\bfw}_{j,\Omega_j}^T\bfx_{n,\Omega_j})-y_n\Big)^2,
\end{equation}
where $\widetilde{\bfw} = [{\bfw}_{1,\Omega_1}^T, {\bfw}_{2,\Omega_2}^T, \cdots, {\bfw}_{K,\Omega_K}^T]^T \in \mathbb{R}^{\sum_j r_j}$. Here, we abuse the notation of $\hat{f}_{\mathcal{D}}$ to represent a mapping from $\mathbb{R}^{\sum_j r_j}$, instead of $\mathbb{R}^{K\times d}$ in \eqref{eqn: sample_risk}, to $\mathbb{R}$. In fact, under the constriant of $\bfW =\bfM \odot \bfW$, the degree of freedom of $\bfW$ is actually $\sum_j r_j$ instead of $Kd$, and the definition in \eqref{eqn: sample_risk} is a easier way for us to present the following proofs. 
Therefore, the optimization problem in \eqref{eqn: optimization_problem_5} is equivalent as 
\begin{equation}{\label{eqn: optimization_problem}}
    \min_{\widetilde{\bfw}}: \quad \hat{f}_{\mathcal{D}}(\widetilde{\bfw}).
\end{equation}
Let us define $\widetilde{\bfw}^*=[{\bfw}^{*T}_{1,\Omega_1}, {\bfw}^{*T}_{2,\Omega_2}, \cdots, {\bfw}^{*T}_{K,\Omega_K}]^T \in \mathbb{R}^{\sum_j r_j}$, where $\bfw^{*T}_j$ is the $j$-th column of $\bfW^*\bfP$. 
and it is clear that $\widetilde{\bfw}^*$ is the global optimal to  \eqref{eqn: optimization_problem}.
Additionally, the population risk function, which is the expectation of the empirical risk function over the data $\mathcal{D}$, is defined as
\begin{equation}\label{eqn: population_risk}
\begin{split}
    f(\widetilde{\bfw}) =  \mathbb{E}_{\mathcal{D}} \hat{f}_{\mathcal{D}}(\widetilde{\bfw})
=&\mathbb{E}_{\mathcal{D}}\frac{1}{2N}\sum_{n=1}^{N}\Big(\frac{1}{K}\sum_{j=1}^{K}\phi({\bfw}_{j,\Omega_j}^T\bfx_{n,\Omega_j})-y_n\Big)^2\\
=&\mathbb{E}_{\bfx}\frac{1}{2}\Big(\frac{1}{K}\sum_{j=1}^{K}\phi({\bfw}_{j,\Omega_j}^T\bfx_{\Omega_j})-y\Big)^2,\\
\end{split}
\end{equation}
where $\bfx \in \mathbb{R}^d$ belongs to standard Gaussian distribution, and $y = g(\bfW^*\bfP^*;\bfx)$. 

Moreover, 
for  the convenience of proofs, 
we use $\sigma_i$ to denote the $i$-th largest singular value of ${\bfW}^*\bfP$, and it is clear that $\sigma_{i}(\bfW^*\bfP)=\sigma_{i}(\bfW^*)$ for all $i$. Then, $\kappa$ is defined as ${\sigma_1}/{\sigma_K}$, and  $\gamma = \prod_{i=1}^K \sigma_i/\sigma_K$. Factor $\rho$ is defined in Property 3.2 \cite{ZSJB17} and a fixed constant for the ReLU activation function.
In addition, without special descriptions, ${\boldsymbol{\alpha}}=[{{\boldsymbol{\alpha}}}_1^T, {{\boldsymbol{\alpha}}}_2^T, \cdots, {{\boldsymbol{\alpha}}}_K^T ]^T$  stands for any unit vector that in $\mathbb{R}^{\sum_j r_j}$ with ${{\boldsymbol{\alpha}}}_j\in \mathbb{R}^r_j$.  Therefore, we have 
\begin{equation}\label{eqn: alpha_definition}
    \| \nabla^2 \hat{f}_{\mathcal{D}}\|_2 = \max_{\boldsymbol{\alpha}}\|\boldsymbol{\alpha}^T \nabla^2 \hat{f}_{\mathcal{D}} \boldsymbol{\alpha} \|_2
    = 
    \max_{\boldsymbol{\alpha}}
    \Big(
    \sum_{j=1}^{K} \boldsymbol{\alpha}_j^T\frac{\partial \hat{f}_{\mathcal{D}} }{\partial \bfw_j}
    \Big)^2.
\end{equation}

Finally, since we focus on order-wise analysis, some constant number will be ignored in the majority of the steps. In particular, we use $h_1(z) \gtrsim(\text{ or }\lesssim, \eqsim ) h_2(z)$ to denote there exists some positive constant $C$ such that $h_1(z)\ge(\text{ or } \le, = ) C\cdot h_2(z)$ when $z\in\mathbb{R}$ is sufficiently large.
\begin{table}[H]
{
\caption{Table of notations} 
\begin{center}
\begin{tabular}{c p{13cm}}
\hline\hline 
Notation &  {\qquad \qquad Description} \\ [0.2ex] 
\hline
$N$ &  The number of training samples; a scalar in $\mathbb{Z}$\\
\hline
$K$ &  The number of neurons in the neural network; a scalar in $\mathbb{R}$  \\
\hline
$d$ &  The dimension of input data; a scalar in $\mathbb{R}$\\
\hline
$\bfx$   & The input data/features; a vector in $\mathbb{R}^d$ \\
\hline
$y$ &    The output label; a scalar in $\mathbb{R}$\\
\hline
$\hat{f}_{\mathcal{D}}$ &  The empirical risk function defined in \eqref{eqn: empirical_risk}; a mapping from $\mathbb{R}^{\sum_jr_j}$ to $\mathbb{R}$\\
\hline
$f$   & The population risk function defined as $f=\mathbb{E}_{\mathcal{D}}\hat{f}_{\mathcal{D}}$; a mapping from $\mathbb{R}^{\sum_jr_j}$ to $\mathbb{R}$\\
\hline
$\bfP$  & The permutation matrix; a binary matrix in $\{0,1\}^{K\times K}$ \\
\hline
$\bfW^*$ &   The ground truth weights of {oracle} network; a matrix in $\mathbb{R}^{d\times K}$\\
\hline
$\bfM^*$ &   The mask matrix of the {oracle} network; a  binary matrix in $\{0,1\}^{d\times K}$\\
\hline
$r_j^*$ &   The number of non-pruned weights in the $j$-th  neuron of {oracle} network\\
\hline 
$\bfW$ &   The ground truth weights of {learner} network; a matrix in $\mathbb{R}^{d\times K}$\\
\hline
$\bfM$   & The mask matrix of the {learner} network; a binary matrix in $\{0,1\}^{d\times K}$\\
\hline
$r_j$ & The number of non-pruned weights in the $j$-th  neuron of {learner} network \\
\hline 
$r_{\min}$   & The minimal value in $\{r_j\}_{j=1}^K$ \\
\hline 
$r_{\max}$   & The maximal value in $\{r_j\}_{j=1}^K$ \\
\hline
$\Omega_j^*$   & The indices of non-pruned weights in teacher network; a set with size of $r_j^*$\\
\hline
$\Omega_j$  & The indices of non-pruned weights in {learner} network; a set with size of $r_j$\\
\hline
$\widetilde{\bfw}$   & Contains the non-pruned weights of $\bfW$ and equals to $[{\bfw}_{1,\Omega_1}^T, {\bfw}_{2,\Omega_2}^T, \cdots, {\bfw}_{K,\Omega_K}^T]^T$; a vector in $\mathbb{R}^{\sum_j r_j}$\\
\hline
$\widetilde{\bfw}^*$   & Contains the non-pruned weights of the {oracle} model; a vector in $\mathbb{R}^{\sum_j r_j}$\\
\hline
$\delta_{i,j}$ &  A binary scalar, and the value is $1$ if $\Omega_j$ and $\Omega_k$ are overlapped and $0$ otherwise\\
\hline
$\widetilde{r}$  & The value of $\frac{1}{8K^4}\big( \textstyle\sum_k\sum_j (1+\delta_{j,k})(r_j+r_k)^{\frac{1}{2}} \big)^2$\\
\hline
\hline
\end{tabular}
\end{center}
}
\label{table: notation} 
\end{table}

%% file: appendix/c5p2.tex
\section{Initialization via Tensor Method: Pruning Networks}


In this section, we present the revised tensor initialization based on that in \cite{ZSJB17}. To reduce the dependency of input dimension from $d$ to the order of $r_{\max}$, we need to define $\widetilde{\bfx}$ in \eqref{eqn: x_tilde} instead of directly using $\bfx$ to generate the high order momentum as shown in \eqref{eqn: M_1} to \eqref{eqn: M_3}. In addition, as $\bfw_{j,\Omega_j}$'s are different in dimensions, we need to define the corresponding augmented weights by inserting $0$ such that augmented $\bfw_{j,\Omega_j}$ are additive in a sense. The additional notations used in presenting are summarized in Table 2, and one can skip this part if the focus is only on the local convexity analysis (Theorem \ref{Thm: convex_region}) and convergence analysis (Theorem \ref{Thm: major_thm}). The intuitive reasons for selecting $\widetilde{\bfx}$ mainly lie in two aspects: first, $\widetilde{\bfx}$ is much lower dimensional vector considering $r_j\ll d$; second, $\widetilde{\bfx}$ belongs to zero mean Gaussian distribution, which is rotational invariant and is correlate with $\phi(\bfw_j^{*T}\bfx)$. Therefore, the magnitude and direction information of $\{\bfw_{j,\Omega_j}\}_{j=1}^K$ are separable after tensor decomposition, and the dimension of the tensors are at most in the order of $r_{\max}$.

First, we define a group of augmented vectors $\{\widetilde{\bfx}_{\widetilde{\Omega}_{j}}^{(j)}\}_{j=1}^K$ based on $\{\bfx_{\Omega_j}\}_{j=1}^K$ such that $\Omega_j\subseteq\widetilde{\Omega}_j$ with $|\widetilde{\Omega}_j|= r_{\max}$ and 
	\begin{equation}\label{eqn: x_tilde_j}
		\widetilde{x}_i^{(j)} =\begin{cases}
		x_i \quad \text{ if $i \in \Omega_j$}\\
		0 \quad \text{ if $i \in \widetilde{\Omega}_j/\Omega_j$}
		\end{cases}.
	\end{equation}
	For notation convenience, we use $\mathcal{F}_{j}$ to denote the mapping from  $\mathbb{R}^{r_j}$ to  $\mathbb{R}^{r_{\max}}$ as 
	\begin{equation}\label{eqn: mapping_tilde}
	    \mathcal{F}_j(\bfz) = [\bfz^T, \mathbf{0}_{(j)}^T]^T,
	\end{equation}
	where $\mathbf{0}$ is a zero vector in $\mathbb{R}^{r_{\max}-r_j}$. Obviously, we have  
	\begin{equation}
		\widetilde{\bfx}_{\widetilde{\Omega}_{j}}^{(j)} = \mathcal{F}_{j}(\bfx_{\Omega_j}).
	\end{equation}
	Correspondingly, the augmented weights $\{\bfu_j^*\}_{j=1}^K$ are defined as 
	\begin{equation}\label{eqn: mapping_tilde2}
	\bfu_j^* = \mathcal{F}_{j}(\bfw_{j,\Omega_j}^*)
	\end{equation} for $j\in[K]$. 

\begin{table}[H]
{
\caption{Table of additional notations for tensor method} 
\begin{center}
\begin{tabular}{c p{12cm}}
\hline\hline 
Notation &  {\qquad \qquad Description} \\ [0.2ex] 
\hline
$\widetilde{\bfx}_{\widetilde{\Omega}_j}^{(j)}$ &  The argumented vector in $\mathbb{R}^{r_{\max}}$ of $\bfx_{\Omega_j}$ by inserting $0$; defined in \eqref{eqn: x_tilde_j}
\\
\hline
$\mathcal{F}_j$ & A linear mapping that generats a augmented vector; defined in \eqref{eqn: mapping_tilde}
\\
\hline
$\mathcal{F}_j^{\dagger}$ & The pseudo inverse of $\mathcal{F}_j$; a linear mapping
\\
\hline
$\widetilde{\bfx}$ &  The value of $\frac{1}{\sqrt{K}}\sum_{j}\widetilde{\bfx}_{\widetilde{\Omega}_j}^{(j)}$;
\\
\hline
$\bfu_j^*$ & The argumented vector in $\mathbb{R}^{r_{\max}}$ of $\bfw_{j,\Omega_j}^*$ by inserting $0$; defined in \eqref{eqn: mapping_tilde2} \\
\hline
$\overline{\bfu}_j^*$ & The normalized vector of ${\bfu}_j$ as ${\bfu}_j^*/\|{\bfu}_j^*\|_2$ \\
\hline
$\widehat{\overline{\bfu}}_j^*$ & The estimation of the normalized vector of ${\bfu}_j^*$\\
\hline
$\psi_1,\psi_2, \psi_3$ & Some fixed constants depends on the distribution of $\{\bfx_{\Omega_j}\}_{j=1}^K$\\
\hline
$\bfM_1$ & A vector in $\mathbb{R}^{r_{\max}}$ defined in \eqref{eqn: M_1} \\
\hline
$\widehat{\bfM}_1$ & The estimation of $\bfM_1$ \\
\hline
${\bfM}_2$ & A matrix in $\mathbb{R}^{r_{\max}\times r_{\max}}$ defined in \eqref{eqn: M_2} \\
\hline
$\widehat{\bfM}_2$ & The estimation of $\bfM_2$ \\
\hline
${\bfM}_3$ & A tensor in $\mathbb{R}^{r_{\max}\times r_{\max}\times r_{\max}}$ defined in \eqref{eqn: M_3} \\
\hline
$\widehat{\bfM}_3$ & The estimation of $\bfM_3$ \\
\hline
$\bfV$ & The orthogonal matrix in $\mathbb{R}^{K\times K}$ that span the sub-space of the convex hull of $\{\bfu_j \}_{j=1}^K$ \\
\hline
$\widehat{\bfV}$ & The estimation of $\bfV$ \\
\hline
${\bfM}(\widehat{\bfV}, \widehat{\bfV}, \widehat{\bfV})$ & A tensor in $\mathbb{R}^{ K\times K \times K }$ defined in \eqref{eqn: M_3_vv2} \\
\hline
$\widehat{\bfM}(\widehat{\bfV}, \widehat{\bfV}, \widehat{\bfV})$ & The estimation of ${\bfM}(\widehat{\bfV}, \widehat{\bfV}, \widehat{\bfV})$ \\
\hline
$\bfs_j$ & The value of $\bfV\bfu_j^*$; a vector in $\mathbb{R}^{K}$
\\
\hline
$\widehat{\bfs}_j$ & The estimation of $\bfs_j$
\\
\hline
${\alpha}_j$ & The value of $\|\bfu_j^*\|_2$; a scalar in $\mathbb{R}$
\\
\hline
$\widehat{\alpha}_j$ & The estimation of $\alpha_j$
\\
\hline
\hline
\end{tabular}
\end{center}
}
\label{table: initilization} 
\end{table}

The steps above guarantee the augmented weights $\bfu_{j}$'s are in the same dimension so that the high order momenta are able to characterize the directions of weights simultaneously. Additionally, we define
\begin{equation}\label{eqn: x_tilde}
\widetilde{\bfx} = \frac{1}{\sqrt{K}}\sum_{j=1}^K \widetilde{\bfx}_{\widetilde{\Omega}_j}^{(j)},    
\end{equation}
and corresponding high order momenta are defined in the following 
way instead:
\begin{equation}\label{eqn: M_1}
\bfM_{1} = \mathbb{E}_{\bfx}\{y \tilde{\bfx} \}\in \mathbb{R}^{r_{\max}},
\end{equation}
\begin{equation}\label{eqn: M_2}
\bfM_{2} = \mathbb{E}_{\bfx}\Big[y\big(\tilde{\bfx}\otimes \tilde{\bfx}-\mathbb{E}_{\bfx}\tilde{\bfx}\tilde{\bfx}^T\big)\Big]\in \mathbb{R}^{r_{\max}\times r_{\max}},
\end{equation}
\begin{equation}\label{eqn: M_3}
\bfM_{3} = \mathbb{E}_{\bfx}\Big[y\big(\tilde{\bfx}^{\otimes 3}- \widetilde{\bfx}\widetilde{\otimes} \mathbb{E}_{\bfx}\widetilde{\bfx}\widetilde{\bfx}^T  \big)\Big]\in \mathbb{R}^{r_{\max}\times r_{\max} \times r_{\max}},
\end{equation}
where $\mathbb{E}_{\bfx}$ is the expectation over $\bfx$ and $\bfz^{\otimes 3} := \bfz \otimes \bfz \otimes \bfz$ defined as    
\begin{equation}
\bfv\widetilde{\otimes} \bfZ=\sum_{i=1}^{d_2}(\bfv\otimes \bfz_i\otimes \bfz_i +\bfz_i\otimes \bfv\otimes \bfz_i + \bfz_i\otimes \bfz_i\otimes \bfv),
\end{equation} 
for any vector $\bfv\in \mathbb{R}^{d_1}$ and $\bfZ\in \mathbb{R}^{d_1\times d_2}$.

Following the same calculate formulas in the Claim 5.2 \cite{ZSJB17},  there exist some known constants $\psi_i, i =1, 2, 3$, such that
\begin{equation}\label{eqn: M_1_2}
\bfM_{1} = \sum_{j=1}^{K} \psi_1\cdot \|{\bfu}_{j}^*\|_2\cdot\overline{\bfu}_j^*,
\end{equation}
\begin{equation}\label{eqn: M_2_2}
\bfM_{2} = \sum_{j=1}^{K} \psi_2\cdot \|{\bfu}_{j}^*\|_2\cdot\overline{\bfu}_{j}^* \overline{\bfu}_{j}^{*T},
\end{equation}\label{eqn: M_3_2}
\begin{equation}\label{eqn:M_3_v2}
\bfM_{3} = \sum_{j=1}^{K} \psi_3\cdot \|{\bfu}_{j}^*\|_2\cdot\overline{\bfu}_{j}^{*\otimes3},
\end{equation} 
where $\overline{\bfu}^*_{j}={\bfu}_{j}^*/\|{\bfu}_{j}^*\|_2$ in \eqref{eqn: M_1}-\eqref{eqn: M_3} is the normalization of ${\bfu}_{j}^*$.

$\bfM_{1}$, $\bfM_{2}$ and $\bfM_{3}$ can be estimated through the samples $\big\{(\bfx_n, y_n)\big\}_{n=1}^{N}$, and let $\widehat{\bfM}_{1}$, $\widehat{\bfM}_{2}$, $\widehat{\bfM}_{3}$ denote the corresponding estimates. 
First, we will decompose the rank-$k$ tensor $\bfM_{3}$ and obtain the $\{\overline{\bfu}^*_{j}\}_{j=1}^K$. By applying the tensor decomposition method \cite{KCL15} to $\widehat{\bfM}_{3}$, the outputs, denoted by $\{\widehat{\overline{\bfu}}_{j}^*\}_{j=1}^K$, are the estimations of $\{\overline{\bfu}^*_{j}\}_{j=1}^K$.
Next, we will estimate $\|\bfu_j^*\|_2$ through solving the following optimization problem:
\begin{equation}\label{eqn: int_op}
\begin{gathered}
\widehat{\boldsymbol{\alpha}} = \arg\min_{\boldsymbol{\alpha}\in\mathbb{R}^K}:\quad  \Big|\widehat{\bfM}_{1} - \sum_{j=1}^{K}\psi_1 \alpha_{j} \widehat{\overline{\bfu}}^*_{j}\Big|,\\
\end{gathered}
\end{equation}
 From \eqref{eqn: M_1_2} and \eqref{eqn: int_op},
 we know that $|\widehat{\alpha}_{j}|$ is the estimation of $\|\bfu_j^*\|_2$. Thus, 
$\widehat{\bfU}$ is given as 
$\big[
|\widehat{\alpha}_{1}|\widehat{\overline{\bfu}}^*_{1}, \cdots, |\widehat{\alpha}_{j}|\widehat{\overline{\bfu}}^*_{j}, \cdots, |\widehat{\alpha}_{K}|\widehat{\overline{\bfu}}^*_{K}
\big]$.

\floatname{algorithm}{Subroutine}
\setcounter{algorithm}{0}
\begin{algorithm}[h]
	\caption{Tensor Initialization Method}\label{Alg: initia_snn}
	\begin{algorithmic}[1]
		\State \textbf{Input:} training data $\mathcal{D}=\{(\bfx_n, y_n) \}_{n=1}^{N}$;
		\State Generate augmented inputs and weights through $\mathcal{F}_j$ as shown in \eqref{eqn: mapping_tilde} and   \eqref{eqn: mapping_tilde2};
		\State Partition $\mathcal{D}$ into three disjoint subsets $\mathcal{D}_1$, $\mathcal{D}_2$, $\mathcal{D}_3$;
		\State Calculate $\widehat{\bfM}_{1}$, $\widehat{\bfM}_{2}$ {following \eqref{eqn: M_1}, \eqref{eqn: M_2}} using $\mathcal{D}_1$, $\mathcal{D}_2$, respectively;
		\State Obtain the estimate subspace $\widehat{\bfV}$ of $\widehat{\bfM}_{2}$;
		\State Calculate $\widehat{\bfM}_{3}(\widehat{\bfV},\widehat{\bfV},\widehat{\bfV})$  through $\mathcal{D}_3$;
		\State Obtain $\{ \widehat{\bfs}_j \}_{j=1}^K$ via {tensor decomposition method \cite{KCL15}} on $\widehat{\bfM}_{3}(\widehat{\bfV},\widehat{\bfV},\widehat{\bfV})$;
		\State Obtain $\widehat{\boldsymbol{\alpha}}$ by solving  optimization problem $\eqref{eqn: int_op}$;
		\State \textbf{Return:} $\bfw^{(0)}_{j,\Omega_j}=\mathcal{F}_j^{\dagger}\big(|\widehat{\alpha}_{j}|\widehat{\bfV}\widehat{\bfs}_j\big)$, {$j=1,...,K$.}
	\end{algorithmic}
\end{algorithm}

To reduce the computational complexity of tensor decomposition, one can project $\widehat{\bfM}_{3}$ to a lower-dimensional tensor \cite{ZSJB17}. The idea is to first estimate the subspace spanned by $\{\bfw_{j}^* \}_{j=1}^{K}$, and let $\widehat{\bfV}$ denote the estimated subspace. 
 
Moreover, we have 
\begin{equation}\label{eqn: M_3_vv2}
\bfM_{3}(\widehat{\bfV},\widehat{\bfV},\widehat{\bfV}) = \mathbb{E}_{\bfx}\Big[y\big((\widehat{\bfV}^T\widetilde{\bfx})^{\otimes 3}- (\widehat{\bfV}^T\widetilde{\bfx})\widetilde{\otimes} \mathbb{E}_{\bfx}(\widehat{\bfV}^T\widetilde{\bfx})(\widehat{\bfV}^T\widetilde{\bfx})^T  \big)\Big]\in \mathbb{R}^{K\times K \times K},
\end{equation}
Then, one can decompose the estimate $\widehat{\bfM}_{3}(\widehat{\bfV}, \widehat{\bfV}, \widehat{\bfV})$ to obtain unit vectors $\{\hat{\bfs}_j \}_{j=1}^K \in \mathbb{R}^{K}$.
Since $\overline{\bfu}^*$ lies in the subspace $\bfV$, we have $\bfV\bfV^T\overline{\bfu}_j^*=\overline{\bfu}_j^*$. Then, $\widehat{\bfV}\hat{\bfs}_j$ is an estimate of $\overline{\bfu}_j^*$. 
After we obtain the estimated augmented weights $\widehat{\bfu}_j^*$, the estimated weights can be generated through $\widehat{\bfw}_{j,\Omega_j}^* =\mathcal{F}_j^{\dagger}(\widehat{\bfu}_j^*)$, where $\mathcal{F}_j^{\dagger}$ is the pseudo inverse of $\mathcal{F}_j$.
The initialization process is summarized in Subroutine \ref{Alg: initia_snn}.

%% file: appendix/c5p3.tex
\section{Proof of Theorem \ref{Thm: convex_region}}
\label{sec: proof of theorem 1}
The main idea in proving Theorem \ref{Thm: convex_region} is to use triangle inequality as shown in \eqref{eqn: Thm1_t0} by bounding the second order derivative of the population risk function and the distance between the empirical risk and population risk functions. Lemma \ref{Lemma:local_convex_population} provides the lower and upper bound for the population risk function, while Lemma \ref{Lemma: second_order} provides the error bound between the second order derivation of  empirical risk and population risk functions. 

\begin{lemma}[Weyl's inequality, \cite{B97}] \label{Lemma: weyl}
Suppose $\bfB =\bfA + \bfE$ be a matrix with dimension $m\times m$. Let $\lambda_i(\bfB)$ and $\lambda_i(\bfA)$ be the $i$-th largest eigenvalues of $\bfB$ and $\bfA$, respectively.  Then, we have 
\begin{equation}
    |\lambda_i(\bfB) - \lambda_i(\bfA)| \le \|\bfE \|_2, \quad \forall i\in [m].
\end{equation}

\end{lemma}

\begin{lemma}\label{Lemma:local_convex_population}
	Let $f$ be the population risk function in \eqref{eqn: population_risk}.       Assume $\bfW$ satisfies \eqref{eqn: convex_region},
	then the second-order derivative of $f$ over $\widetilde{\bfw}$ is bounded as 
	\begin{equation}
		\frac{(1-\varepsilon_0)\rho}{11\kappa^2\gamma K^2} \bfI \le \nabla^2f(\widetilde{\bfw}) \le \frac{7}{K}\bfI, 
	\end{equation}
	where $\widetilde{\bfw}$ only contains the elements of $\bfW$ with respect to the indices of non-pruned weights.
\end{lemma}

\begin{lemma}\label{Lemma: second_order}
	Let $\hat{f}_{\mathcal{D}}$ and $f$ be the empirical and population risk function in \eqref{eqn: empirical_risk} and \eqref{eqn: population_risk}, respectively, then the second-order derivative of $\hat{f}_{\mathcal{D}}$ is close to its expectation $f$ with an upper bound as:
	\begin{equation}
	\| \nabla^2\hat{f}_{\mathcal{D}}-\nabla^2f \|_2 \lesssim  \frac{1}{K^2}\sum_{k=1}^{K}\sum_{j=1}^{K}(1+\delta_{j,k})\sqrt{\frac{(r_j+r_k)\log q}{N}}
	\end{equation}
	with probability at least $1-q^{-r_{\min}}$.
\end{lemma}

\begin{proof}[Proof of Theorem \ref{Thm: convex_region} ]
  Let $\hat{\lambda}_{\max}$ and $\hat{\lambda}_{\min}$ denote the largest and smallest eigenvalues of $\nabla^2\hat{f}_{\mathcal{D}}$, respectively. Also,   Let ${\lambda}_{\max}$ and ${\lambda}_{\min}$ denote the largest and smallest eigenvalues of $\nabla^2{f}_{\mathcal{D}}$, respectively.
  
  Then, from Lemma \ref{Lemma: weyl}, we have 
\begin{equation}\label{eqn: Thm1_t0}
     \hat{\lambda}_{\max} \le {\lambda}_{\max} + \| \nabla^2\hat{f}_{\mathcal{D}}-\nabla^2f \|_2 
\end{equation}
and
\begin{equation}\label{eqn: Thm1_t1}
    \hat{\lambda}_{\min} \ge  
    {\lambda}_{\min}- \| \nabla^2\hat{f}_{\mathcal{D}}-\nabla^2f \|_2. 
\end{equation}
When the sample complexity satisfies
$$N\gtrsim \varepsilon_1^{-2} \rho^{-2}\kappa^4\gamma^2 K^4 \big[\frac{1}{K^2}\sum_{k=1}^{K}\sum_{j=1}^{K}(1+\delta_{j,k})\sqrt{r_j+r_k}\big]^2
\log q,$$
then from Lemma \ref{Lemma: second_order}, we have \begin{equation}\label{eqn: Thm1_t2}
    	\| \nabla^2\hat{f}_{\mathcal{D}}-\nabla^2f \|_2 \le \frac{\varepsilon_1 \rho}{11\kappa^2\gamma K^2}.
\end{equation}
Then, from \eqref{eqn: Thm1_t0}, \eqref{eqn: Thm1_t1} and \eqref{eqn: Thm1_t2}, we have 
\begin{equation}
     \hat{\lambda}_{\max} \le \frac{8}{K},
\end{equation}
and 
\begin{equation}
    \hat{\lambda}_{\min} \ge \frac{(1-\varepsilon_0-\varepsilon_1)\rho}{11\kappa^2 \gamma K^2},
\end{equation}
which completes the proof. 
\end{proof}

%% file: appendix/c5p4.tex
\section{Proof of Theorem \ref{Thm: major_thm} }\label{sec: Proof of Theorem 2}
The major idea in proving Theorem \ref{Thm: major_thm} is to first  characterize the  gradient descent term by intermediate value theorem. Let $\widetilde{\bfw}^{(t)}$ be the vectorized iterate $\bfW^{(t)}$ with respect to the non-pruned weights, then we have 
\begin{equation}
    \begin{split}
        \nabla \hat{f}_{\Omega_t}(\widetilde{\bfw}^{(t)})
        = & f_{\Omega_t}(\widetilde{\bfw}^{(t)}) + \big(\hat{f}_{\Omega_t}(\widetilde{\bfw}^{(t)}) - f_{\Omega_t}(\widetilde{\bfw}^{(t)})\big)
        \\
        = & \langle \nabla^2f_{\Omega_t}(\widehat{\bfw}^{(t)}), \widetilde{\bfw}^{(t)}-\widetilde{\bfw}^* \rangle + \big(\hat{f}_{\Omega_t}(\widetilde{\bfw}^{(t)}) - f_{\Omega_t}(\widetilde{\bfw}^{(t)})\big),
    \end{split}
\end{equation}
where $\widehat{\bfw}^{(t)}$ lies in the convex hull of $\widetilde{\bfw}^{(t)}$ and $\widetilde{\bfw}^*$. The reason that intermediate value theorem is applied on population risk function instead of empirical risk function is the non-smoothness of the empirical risk functions. Due to the non-smoothness of ReLU activation function at zero point, the empirical risk function is not smooth, either. However, the expectation of the empirical risk function over the Gaussian input $\bfx$ is smooth. Hence, compared with smooth empirical risk function, i.e., neural networks equipped with sigmoid activation function, we have an additional lemma to bound $\nabla\hat{f}_{\mathcal{D}_t}$  to its expectation $\nabla f$ , which is summarized in Lemma \ref{Lemma: first_order}.

 The momentum term $\beta(\bfW^{(t)}-\bfW^{(t-1)})$ plays an important role in determining the convergence rate, and the  recursive rule is obtained in the following way:
\begin{equation}\label{eqn: iteeee}
\begin{bmatrix}
\wt[t+1]-\widetilde{\bfw}^*\\
\wt[t]-\w^*
\end{bmatrix}\\
=\bfA(\beta)
\begin{bmatrix}
\wt[t]-\widetilde{\bfw}^*\\
\wt[t-1]-\widetilde{\bfw}^*
\end{bmatrix},
\end{equation}
where $\bfA(\beta)$ is a matrix with respect to the value of $\beta$ and defined in \eqref{eqn: A}.
Then, we know $\w^{(t)}$, which is equivalent to $\W[t]$, converges to the ground-truth with a linear rate which is the largest singular value of matrix $\bfA(\beta)$. Recall that  AGD reduces to GD with $\beta=0$, so our analysis applies to GD method as well. We are able to show the convergence rate of AGD is faster than GD by proving the largest singular value of $\bfA(\beta)$ is smaller than $\bfA(0)$ for some $\beta>0$.  

\begin{lemma}\label{Lemma: first_order}
	Let $\hat{f}_{\mathcal{D}}$ and $f$ be the empirical and population risk function in \eqref{eqn: empirical_risk} and \eqref{eqn: population_risk}, respectively, then the first-order derivative of $\hat{f}_{\mathcal{D}}$ is close to its expectation $f$ with an upper bound as: 
		\begin{equation}
	\begin{split}
	\| \nabla\hat{f}_{\mathcal{D}}(\widetilde{\bfw})-\nabla f(\widetilde{\bfw}) \|_2 
	\lesssim& 
	\frac{1}{K^2}\sum_{k=1}^{K}\sum_{j=1}^{K}(1+\delta_{j,k})
	\sqrt{\frac{r_{k}\log q}{N}}\|\w-\w^*\|_2\\
	&+\frac{1}{K}\sum_{k=1}^{K}\sqrt{\frac{r_k\log q}{N}}\cdot |\xi|    
	\end{split}
	\end{equation}
	with probability at least $1-q^{-r_{\min}}$, where $\widetilde{\bfw}$ only contains the elements of $\bfW$ with respect to the indices of non-pruned weights.
\end{lemma}

\begin{proof}[Proof of Theorem \ref{Thm: major_thm}]
Since $\|\bfW^{(t)}-\bfW^* \|_F =\| \wt[t]-\w^* \|_2$, we can explore the converges of $\{ \wt[t] \}_{t=1}^T$ instead. Recall that
\begin{equation}
\begin{split}
\wt[t+1]
=&\w^{(t)}-\eta\nabla \hat{f}_{\mathcal{D}_t}(\w^{(t)})+\beta(\wt[t]-\wt[t-1])\\
=&\w^{(t)}-\eta\nabla {f}(\w^{(t)})+\beta\big(\wt[t]-\wt[t-1]\big)\\
&+\eta\big(\nabla{f}(\w^{(t)})-\nabla\hat{f}_{\mathcal{D}_t}(\w^{(t)})\big).
\end{split}
\end{equation}
Since $\nabla^2{f}$ is a smooth function, by the intermediate value theorem, we have
\begin{equation}
\begin{split}
\w^{(t+1)}
=\w^{(t)}&- \eta\nabla^2{f}(\widehat{\bfw}^{(t)})(\w^{(t)}-\w^*)
+\beta(\wt[t]-\wt[t-1])\\
&+\eta\big(\nabla{f}(\w^{(t)})-\nabla\hat{f}_{\mathcal{D}_t}(\w^{(t)})\big),
\end{split}
\end{equation}	
where $\widehat{\bfw}^{(t)}$ lies in the convex hull of $\w^{(t)}$ and $\w^*$.\\
Next, we have
\begin{equation}\label{eqn: major_thm_key_eqn}
\begin{split}
\begin{bmatrix}
\wt[t+1]-\w^*\\
\wt[t]-\w^*
\end{bmatrix}
=&\begin{bmatrix}
\bfI-\eta\nabla^2{f}(\widehat{\bfw}^{(t)})+\beta\bfI &\beta\bfI\\
\bfI& 0\\
\end{bmatrix}
\begin{bmatrix}
\wt[t]-\w^*\\
\wt[t-1]-\w^*
\end{bmatrix}\\
&+\eta
\begin{bmatrix}
\nabla{f}(\w^{(t)})-\nabla\hat{f}_{\mathcal{D}_t}(\w^{(t)})\\
0
\end{bmatrix}
\end{split}
\end{equation}
Let 
\begin{equation}\label{eqn: A}
    \bfA(\beta)=\begin{bmatrix}
\bfI-\eta\nabla^2{f}(\widehat{\bfw}^{(t)})+\beta\bfI &\beta\bfI\\
\bfI& 0\\
\end{bmatrix},
\end{equation} so we have
\begin{equation}\label{eqn: thm2}
\begin{split}
\left\|\begin{bmatrix}
\wt[t+1]-\w^*\\
\wt[t]-\w^*
\end{bmatrix}
\right\|_2
=&
\left\|\bfA(\beta)
\right\|_2
\left\|
\begin{bmatrix}
\wt[t]-\w^*\\
\wt[t-1]-\w^*
\end{bmatrix}
\right\|_2\\
&+\eta
\left\|
\begin{bmatrix}
\nabla{f}(\w^{(t)})-\nabla\hat{f}_{\mathcal{D}_t}(\w^{(t)})\\
0
\end{bmatrix}
\right\|_2.
\end{split}
\end{equation}
From Lemma \ref{Lemma: first_order}, we know that 
\begin{equation}\label{eqn: convergence2}
\begin{split}
\eta\left\|\nabla{f}(\w^{(t)})-\nabla\hat{f}_{\mathcal{D}_t}(\w^{(t)})\right\|_2
\le& \frac{C_5\eta}{K^2}\sum_{k=1}^{K}\sum_{j=1}^K(1+\delta_{j,k})\sqrt{\frac{r_k\log q}{N_t}}\|\w-\w^* \|_2\\
&+ \frac{C_5\eta}{K}\sum_{k=1}^{K}\sqrt{\frac{r_k\log q}{N_t}}\cdot|\xi|
\end{split}
\end{equation}
for some constant $C_5>0$. Then, we have
\begin{equation}\label{eqn: convergence}
\begin{split}
\|\wt[t+1]-\w^* \|_2 
\le& \bigg(\| \bfA(\beta) \|_2+\frac{C_5\eta}{K^2}\sum_{k=1}^{K}\sum_{j=1}^K(1+\delta_{j,k})\sqrt{\frac{r_k\log q}{N_t}}\bigg)\|\wt[t]-\w^* \|_2\\ 
&+ \frac{C_5\eta}{K}\sum_{k=1}^{K}\sqrt{\frac{r_k\log q}{N_t}}\cdot|\xi|\\
:=&\nu(\beta)\|\wt[t]-\w^* \|_2 
+ \frac{C_5\eta}{K}\sum_{k=1}^{K}\sqrt{\frac{r_k\log q}{N_t}}\cdot|\xi|.
\\
\end{split}
\end{equation}

Let $\nabla^2f(\widehat{\bfw}^{(t)})=\bfS\mathbf{\Lambda}\bfS^{T}$ be the eigendecomposition of $\nabla^2f(\widehat{\bfw}^{(t)})$. Then, we define
\begin{equation}\label{eqn:Heavy_ball_eigen}
\begin{split}
{\bfA}(\beta)
: =&
\begin{bmatrix}
\bfS^T&\bf0\\
\bf0&\bfS^T
\end{bmatrix}
\bfA(\beta)
\begin{bmatrix}
\bfS&\bf0\\
\bf0&\bfS
\end{bmatrix}
=\begin{bmatrix}
\bfI-\eta\bf\Lambda+\beta\bfI &\beta\bfI\\
\bfI& 0
\end{bmatrix}
\end{split}
\end{equation}
Since 
$\begin{bmatrix}
\bfS&\bf0\\
\bf0&\bfS
\end{bmatrix}\begin{bmatrix}
\bfS^T&\bf0\\
\bf0&\bfS^T
\end{bmatrix}
=\begin{bmatrix}
\bfI&\bf0\\
\bf0&\bfI
\end{bmatrix}$, we know $\bfA(\beta)$ and ${\bfA}(\beta)$
share the same eigenvalues. \\
Let $\lambda_i$ be the $i$-th eigenvalue of $\nabla^2f(\widehat{\bfw}^{(t)})$, then the corresponding $i$-th eigenvalue of \eqref{eqn:Heavy_ball_eigen}, denoted by $\delta_i(\beta)$, satisfies 
\begin{equation}\label{eqn:Heavy_ball_quaratic}
\nu_i^2-(1-\eta \lambda_i+\beta)\delta_i+\beta=0.
\end{equation}
Then, we have 
\begin{equation}\label{eqn: Heavy_ball_root}
\delta_i(\beta)=\frac{(1-\eta \lambda_i+\beta)+\sqrt{(1-\eta \lambda_i+\beta)^2-4\beta}}{2},
\end{equation}
and
\begin{equation}\label{eqn:heavy_ball_beta}
\begin{split}
|\delta_i(\beta)|
=\begin{cases}
\sqrt{\beta}, \qquad \text{if}\quad  \beta\ge \big(1-\sqrt{\eta\lambda_i}\big)^2,\\
\frac{1}{2} \left| { (1-\eta \lambda_i+\beta)+\sqrt{(1-\eta \lambda_i+\beta)^2-4\beta}}\right| , \text{otherwise}.
\end{cases}
\end{split}
\end{equation}
Note that the other root of \eqref{eqn:Heavy_ball_quaratic} is abandoned because the root in \eqref{eqn: Heavy_ball_root} is always larger than or at least equal to the other root with $|1-\eta \lambda_i|<1$.   
By simple calculation, we have 
\begin{equation}\label{eqn: Heavy_ball_result}
\delta_i(0)>\delta_i(\beta),\quad \text{for}\quad  \forall \beta\in\big(0, (1-{\eta \lambda_i})^2\big),
\end{equation}
and specifically, $\delta_i$ achieves the minimum $\delta_{i}^*=|1-\sqrt{\eta\lambda_i}|$ when $\beta= \big(1-\sqrt{\eta\lambda_i}\big)^2$.\\
Let us first assume $\w^{(t)}$ satisfies \eqref{eqn: convex_region}, then
from Lemma \ref{Lemma:local_convex_population}, we know that 
$$0<\frac{(1-\varepsilon_0)}{11\kappa^2\gamma K^2}\le\lambda_i\le \frac{7}{K}$$
provided that $N_t\gtrsim\varepsilon_0^{-2}\rho^{-1}\kappa^2\gamma K^3 \big[\frac{1}{K^2}\sum_j\sum_k (1+\delta_{j,k})\sqrt{r_k+r_j} \big]^2 \log q$.
Let $\gamma_1=\frac{\rho(1-\varepsilon_0)}{11\kappa^2\gamma K^2}$ and $\gamma_2=\frac{7}{K}$.
If we choose $\beta$ such that
\begin{equation}
\beta^*=\max \big\{(1-\sqrt{\eta\gamma_1})^2, (1-\sqrt{\eta\gamma_2})^2 \big\},
\end{equation} then we have $\beta\ge (1-\sqrt{\eta \lambda_i})^2$ for any $i$ and $\delta_i=\max\big\{|1-\sqrt{\eta\gamma_1}|, |1-\sqrt{\eta\gamma_2}|  \big\}$ for any $i$. \\
Let $\eta= \frac{1}{{2\gamma_2}}$, then $\beta^*$ equals to  $\Big(1-\sqrt{\frac{\gamma_1}{2\gamma_2}}\Big)^2$. 
Then, for any $\varepsilon_0\in (0, \frac{1}{2})$ we have
\begin{equation}\label{eqn: con1}
\begin{split}
\|\bfA(\beta^*) \|_2
=\max_i \delta_{i}(\beta^*)
= 1-\sqrt{\frac{\gamma_1}{2\gamma_2}}
=&1-\sqrt{\frac{1-\varepsilon_0}{154\rho^{-1}\kappa^2 \gamma K}}\\
\le& 1-\frac{1-3/4\cdot\varepsilon_0}{\sqrt{154\rho^{-1}\kappa^2 \gamma K}}.
\end{split}
\end{equation}
Then, let 
\begin{equation}\label{eqn: con2}
\frac{C_5\eta}{K^2}\sum_{k=1}^{K}\sum_{j=1}^K(1+\delta_{j,k})\sqrt{\frac{r_k\log q}{N_t}}
\le \frac{\varepsilon_0}{4\sqrt{154\rho^{-1}\kappa^2 \gamma K}},
\end{equation}
we need $N_t\gtrsim\varepsilon_0^{-2}\rho^{-1}\kappa^2\gamma K^3 \big[\frac{1}{K^2}\sum_j\sum_k (1+\delta_{j,k})\sqrt{r_k} \big]^2 \log q$.

Combine \eqref{eqn: con1} and \eqref{eqn: con2}, we have 
\begin{equation}
\nu(\beta^*)\le 1-\frac{1-\varepsilon_0}{\sqrt{154\rho^{-1}\kappa^2 \gamma K}}.
\end{equation}
While let $\beta = 0$, we have 
\begin{equation}
\begin{split}
\nu(0) \ge \|\bfA(0) \|_2 = 1-\frac{1-\varepsilon_0}{{154\rho^{-1}\kappa^2 \gamma K}}\\		
\end{split}
\end{equation}
and 
\begin{equation}
\nu(0)  \le 1-\frac{1-2\varepsilon_0}{{154\rho^{-1}\kappa^2 \gamma K}}
\end{equation}
if $N_t\gtrsim\varepsilon_0^{-2}\rho^{-1}\kappa^2\gamma K^4 \big[\frac{1}{K^2}\sum_j\sum_k (1+\delta_{j,k})\sqrt{r_k+r_j} \big]^2 \log q$.

In conclusion, with $\eta =\frac{1}{2 \gamma_2}$ and $\beta = \big(1-\frac{\gamma_1}{2\gamma_2} \big)^2$, we have
\begin{equation}\label{eqn: induction}
\begin{split}
&\|	\wt[t+1]-\w^*\|_2
\le\Big(1-\frac{1-\varepsilon_0}{\sqrt{154\kappa^2 \gamma K}}\Big)\|	\wt[t]-\w^*\|_2
 + \frac{C \eta}{K}\sum_{k=1}^{K} \sqrt{\frac{r_k \log q}{N_t}} |\xi|.
\end{split}
\end{equation}
if $\wt[t+1]$ satisfies \eqref{eqn: convex_region} and $N_t\gtrsim\varepsilon_0^{-2}\rho^{-1}\kappa^2\gamma K^4 \big[\frac{1}{K^2}\sum_j\sum_k (1+\delta_{j,k})\sqrt{r_k+r_j} \big]^2 \log q$.

Then, we can start mathematical induction of \eqref{eqn: induction} over $t$.

\textbf{Base case}:  \eqref{eqn: convex_region} holds for $\w^{(0)}$ naturally from the assumption in Theorem \ref{Thm: major_thm}.
Since \eqref{eqn: convex_region}  holds and the number of samples exceeds the required bound in \eqref{eqn: induction}, we have \eqref{eqn: induction} holds for $t=0$.   

\textbf{Induction step}:
Assume \eqref{eqn: induction} holds for $t$, to make sure the mathematical induction of \eqref{eqn: induction} holds, we need $\wt[t+1]$ satisfies \eqref{eqn: convex_region}. That is  
\begin{equation}
\sum_{k=1}^K\frac{\eta}{K} \sqrt{\frac{r_k\log q}{N_t}}
\lesssim  \frac{1-\varepsilon_0}{\sqrt{132\kappa^2 \gamma K}}\cdot  \frac{\varepsilon_0\sigma_K}{44\kappa^2\gamma K^2}.
\end{equation}	
Hence, we need 
\begin{equation}\label{eqn: N_2}
N_t \gtrsim \varepsilon_0^{-2} \kappa^8 \gamma^3 K^6\Big(\frac{1}{K}\sum_k \sqrt{r_k}\Big)^2\log q.
\end{equation}

 In addition, with \eqref{eqn: convex_region} and \eqref{eqn: induction} hold for all $t\le T$, 
the following equation
\begin{equation}
\begin{split}
\left\|\begin{bmatrix}
\wt[t+1]-\w^*\\
\wt[t]-\w^*
\end{bmatrix}
\right\|_{\infty}
=&
\left\|\bfA(\beta)
\right\|_2
\left\|
\begin{bmatrix}
\wt[t]-\w^*\\
\wt[t-1]-\w^*
\end{bmatrix}
\right\|_{\infty}\\
&+\eta
\left\|
\begin{bmatrix}
\nabla{f}(\w^{(t)})-\nabla\hat{f}_{\mathcal{D}_t}(\w^{(t)})\\
0
\end{bmatrix}
\right\|_{\infty}
\end{split}
\end{equation}
holds as well, and $\left\|\bfA(\beta)
\right\|_2$ is bounded by $\nu(\beta)$. Hence,  \eqref{eqn: induction} also holds in infinity norm as 
\begin{equation}\label{eqn:infty}
\begin{split}
&\|	\wt[t+1]-\w^*\|_{\infty}
\le\Big(1-\frac{1-\varepsilon_0}{\sqrt{154\kappa^2 \gamma K}}\Big)\|	\wt[t]-\w^*\|_{\infty}
 + 2C \eta \sqrt{\frac{r \log q}{N_t}} |\xi|.
\end{split}
\end{equation}

In conclusion, when $N_t \gtrsim \varepsilon_0^{-2} \kappa^8 \gamma^3 K^6\Big(\frac{1}{K^2}\sum_k\sum_j (1+\delta_{j,k}) \sqrt{r_k+r_j}\Big)^2\log d$, we know that \eqref{eqn: induction} holds for all $1\le t \le T$ with probability at least $1-K^2T\cdot q^{-r_{\min}}$. By simple calculation, we can obtain
\begin{equation}\label{eqn: Thm1_final}
\begin{split}
\|	\wt[T]-\w^*\|_2
\le\Big(1-\frac{1-\varepsilon_0}{\sqrt{132\kappa^2 \gamma K}}\Big)^T\|	\wt[0]-\w^*\|_2
 + \frac{C}{K}\sum_{k=1}^{K} \sqrt{\frac{\kappa^2\gamma K^2r_k\log q}{N_t}} \cdot |\xi|.
\end{split}
\end{equation}
for some constant $C> 0$.

\end{proof}

%% file: appendix/c6p1.tex
\section{Roadmap of the Proofs}
We first provide an overview of the techniques used in proving Theorems 
\ref{Thm: p_convergence} and \ref{Thm: sufficient_number}. 

\textbf{1. Characterization of a proper  population risk function.} 
To characterize the performance of the iterative self-training algorithm via the stochastic gradient descent method, we need first to define a population risk function such that the following two properties hold. First, the landscape of the population risk function should be analyzable near $\{\W[\ell]\}_{\ell=0}^L$. Second, the distance between the empirical risk function in \eqref{eqn: empirical_risk_function} and the population risk function should be bounded near $\{\W[\ell]\}_{\ell=0}^L$. The generalization function defined in \eqref{eqn: GF}, which is widely used in the supervised learning problem with a sufficient number of samples, failed the second requirement. To this end, we turn to find a new population risk function defined in \eqref{eqn: p_population_risk_function}, and the illustrations of the population risk function and objection function are included in Figure \ref{fig: landscape}. 

\textbf{2. Local convex region of the population risk function.} 
The purpose is to characterize the iterations via the stochastic gradient descent method in the population risk function. 
To obtain the local convex region of the population risk function, we first bound the Hessian of the population risk function at its global optimal. Then, utilizing the corresponding lemma to obtain the Hessian of the population risk function near the global optimal. The local convex region of the population risk function is summarized in Lemma \ref{Lemma: p_second_order_derivative bound}, and the proof of Lemma \ref{Lemma: p_second_order_derivative bound} is included in this appendix.

\textbf{3. Bound between the population risk and empirical risk functions.} After the characterization of  the iterations via the stochastic gradient descent method in the population risk function, we need to bound the distance between the population risk function and empirical risk function. Therefore, the behaviors of the iterations via the stochastic gradient descent method in the empirical risk function can be described by the ones in the population risk function and the distance between these two. The key lemma is summarized in Lemma \ref{Lemma: p_first_order_derivative} (see this appendix), and the proof is included in this appendix.

\begin{figure}[h]
     \centering
     \includegraphics[width=0.6\linewidth]{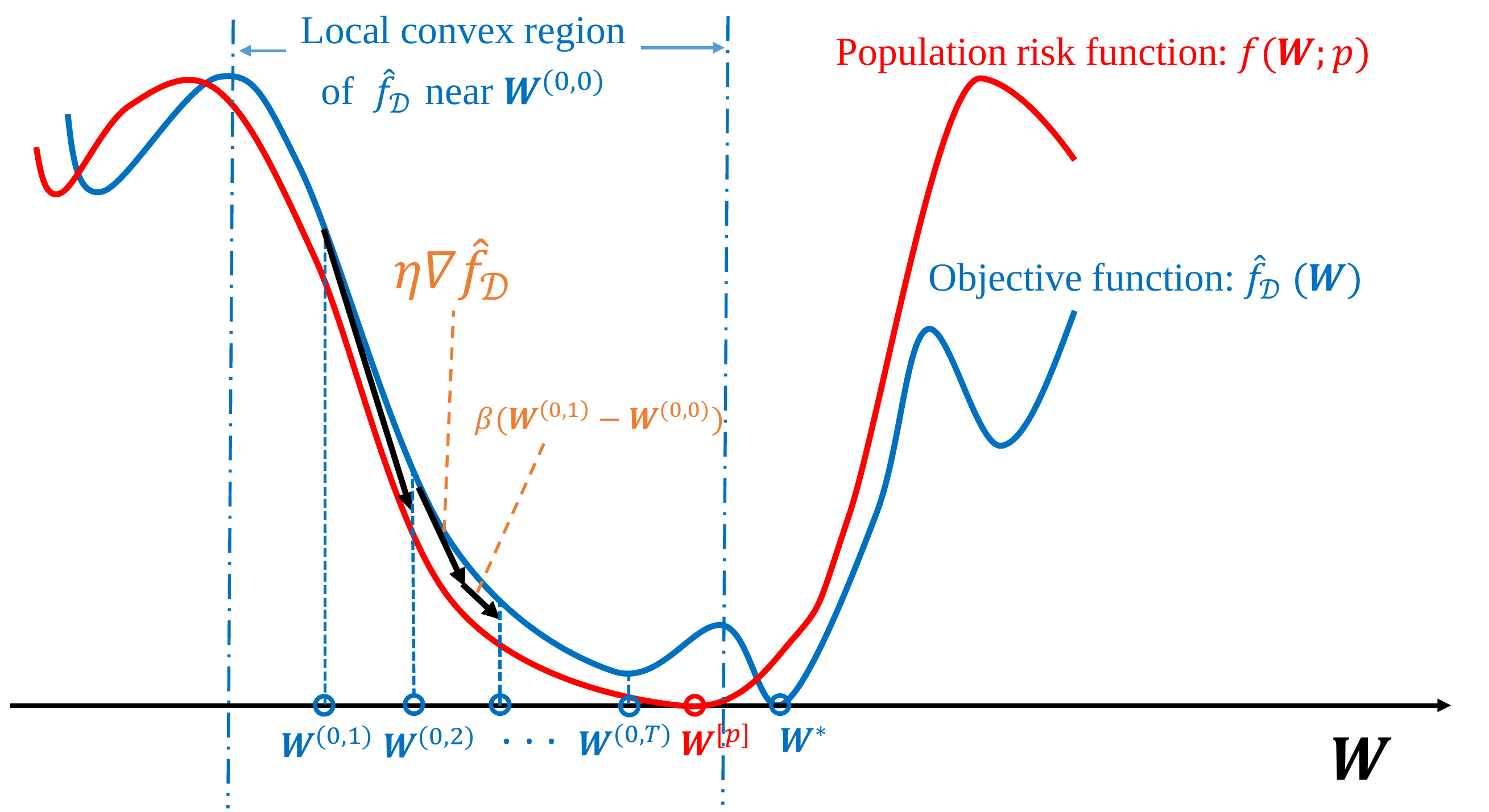}
     \caption{The landscapes of the objection function and population risk function}
     \label{fig: landscape}
\end{figure}

In the following contexts, the details of the iterative self-training algorithm are included in this appendix.  We then first provide the proof of Theorem \ref{Thm: sufficient_number} in this appendix, which can be viewed as a special case of Theorem \ref{Thm: p_convergence}. Then, with the preliminary knowledge from proving Theorem \ref{Thm: sufficient_number}, we turn to present the full proof of a more general statement summarized in Theorem \ref{Thm: p_main_thm}, which is related to Theorem \ref{Thm: p_convergence}. The definition and relative proofs of $\mu$ and $\rho$ are all included in Appendix \ref{sec: rho}.  The proofs of preliminary lemmas are included as well.

\section{Iterative Self-training Algorithm}
\label{sec: algorithm}
In this section, we implement the details of the mini-batch stochastic gradient descent used in each stage of the  iterative self-training algorithm. After $t$ number of iterations via mini-batch stochastic gradient descent at $\ell$-th stage of self-training algorithm, the learned model is denoted as $\bfW^{(\ell, t)}$. One can easily check that $\W[\ell]$ in the main context is denoted as $\W[\ell, 0]$ in this section and the following proofs. 
Last,
the pseudo-code of the iterative self-training algorithm is summarized in Algorithm
\ref{Algorithm: Alg1}.
\begin{algorithm}[h]
	\caption{Iterative Self-Training Algorithm}\label{Algorithm: Alg1}
	\begin{algorithmic}
		\State \textbf{Input:} labeled $\mathcal{D}=\{(\bfx_n, y_n) \}_{n=1}^{N}$,  
		 unlabeled data $\w[\mathcal{D}]=\{\widetilde{\bfx}_m \}_{m=1}^{M}$,
		and gradient step size $\eta$;\\
		\State \textbf{Initialization:} preliminary teacher model with weights $\bfW^{(0,0)}$;\\
		\State \textbf{Partition:} randomly and independently pick data from $\mathcal{D}$ and $\widetilde{\mathcal{D}}$ to form $T$ subsets $\{ \mathcal{D}_t \}_{t=0}^{T-1}$ and $\{\widetilde{\mathcal{D}}_t\}_{t = 0}^{T-1}$, respectively; 
		\\
		\For{$\ell = 0 , 1, \cdots, L-1$}\\
		    \State $y_m = g(\bfW^{(\ell,0)};\w[\bfx]_m)$ for $m =1, 2, \cdots, M$\\
		    \For{$t = 0, 1,\cdots, T-1$}\\
		    \State $\W[\ell, t+1] = \W[\ell, t] - \eta\cdot \nabla \hat{f}_{\mathcal{D}_t, \widetilde{\mathcal{D}}_t}(\W[\ell, t]) + \beta \cdot \big( \W[\ell, t] - \W[\ell, t-1] \big)$ \\
		    \EndFor\\
		     \State $\bfW^{(\ell+1,0)} = \bfW^{(\ell, T)}$\\
		 \EndFor
	\end{algorithmic}
\end{algorithm}

\section{Notations}
In this section, we first introduce some important notations that will be used in the following proofs, and the notations are summarized  in Table 1.

As shown in Algorithm \ref{Algorithm: Alg1},  $\bfW^{(\ell, t)}$ denotes the learned model after $t$ number of iterations via mini-batch stochastic gradient descent at $\ell$-th stage of the iterative self-training algorithm. Given a student model $\widetilde{\bfW}$, the pseudo label for $\widetilde{\bfx} \in \widetilde{\mathcal{D}}$ is generated as
\begin{equation}\label{eqn: psedu_label}
    \tilde{y} = g(\widetilde{\bfW};\widetilde{\bfx}).
\end{equation}
Further, let $\Wp = p\bfW^* + (1-p)\W[0,0]$,
we then define the \textit{population risk function} as 
\begin{equation}\label{eqn: p_population_risk_function}
    f(\bfW;p) = \frac{\lambda}{2} \mathbb{E}_{\bfx}\big(y^*(p)-g(\bfW;\bfx)\big) + \frac{\w[\lambda]}{2} \mathbb{E}_{\w[\bfx]} \big(\w[y]^*(p) - g(\bfW;\w[\bfx])\big),
\end{equation}
where $y^*(p)=g(\Wp;\bfx)$ with $\bfx\sim \mathcal{N}(0,\delta^2\bfI)$ and $\w[y]^*(p) = g(\Wp;\w[\bfx])$ with $\w[\bfx]\sim \mathcal{N}(0,\de^2\bfI)$. When $p =1$, we have $\Wp=\bfW^*$ and $y^*(p) = y$ for data in $\mathcal{D}$. 

Moreover, we use $\sigma_i$ to denote the $i$-th largest singular value of ${\bfW}^*$. Then, $\kappa$ is defined as ${\sigma_1}/{\sigma_K}$, and  $\gamma = \prod_{i=1}^K \sigma_i/\sigma_K$. Additionally,  to avoid high dimensional tensors, the first order derivative of the empirical risk function is defined in the form of vectorized $\bfW$ as 
\begin{equation}\label{eqn: vec_W}
    \nabla \hat{f}(\bfW) = \Big[\frac{\partial f}{\partial \bfw_1}^T, \frac{\partial f}{\partial \bfw_2}^T, \cdots , \frac{\partial f}{\partial \bfw_K}^T \Big]^T \in \mathbb{R}^{dK} 
\end{equation}
with $\bfW =[\bfw_1, \bfw_2, \cdots, \bfw_K]\in\mathbb{R}^{d\times K}$. 
Therefore, the second order derivative of the empirical risk function is in $\mathbb{R}^{dk\times dk}$.
Similar to \eqref{eqn: vec_W}, the high order derivatives of the population risk functions are defined based on vectorized $\bfW$ as well.

Finally, since we focus on order-wise analysis, some constant numbers will be ignored in the majority of the steps. In particular, we use $h_1(z) \gtrsim(\text{or }\lesssim, \eqsim ) h_2(z)$ to denote there exists some positive constant $C$ such that $h_1(z)\ge(\text{or } \le, = ) C\cdot h_2(z)$ when $z\in\mathbb{R}$ is sufficiently large.
\renewcommand{\arraystretch}{1.5}
\begin{table}[h]
    \centering
    \caption{Some important notations}   
    \begin{tabular}{c|p{12cm}}
    \hline
    \hline
    $\mathcal{D} =\{ \bfx_n, y_n \}_{n=1}^N$ & Labeled dataset with $N$  number of samples;\\
    \hline
    $\widetilde{\mathcal{D}} =\{ \widetilde{\bfx}_m \}_{m=1}^M$ & Unlabeled dataset with $M$  number of samples;\\
    \hline
    $\mathcal{D}_t =\{ \bfx_n, y_n \}_{n=1}^{N_t}$ & a subset of $\mathcal{D}$ with $N_t$  number of labeled data;\\
    \hline
    $\widetilde{\mathcal{D}}_t =\{ \widetilde{\bfx}_m \}_{m=1}^{M_t}$ & a subset of $\widetilde{\mathcal{D}}$ with $M_t$  number of unlabeled data;\\
    \hline
    $d$ & Dimension of the input $\bfx$ or $\widetilde{\bfx};$\\
    \hline
    $K$ & Number of neurons in the hidden layer;
    \\
    \hline
    $\bfW^*$ & Weights of the ground truth model;\\
    \hline
    $\Wp$ & $\Wp = p \bfW^* + (1-p)\W[0,0]$;\\
    \hline
    $\W[\ell,t]$ & Model returned by iterative self-training after $t$ step mini-batch stochastic gradient descent at stage $\ell$; $\W[0,0]$ is the initial model;\\
    \hline
    $\hat{f}_{ \mathcal{D}, \widetilde{\mathcal{D}} } (\text{ or } \hat{f})$  & The empirical risk function defined in \eqref{eqn: empirical_risk_function};\\
    \hline
    $f(\bfW;p)$  & The population risk function defined in \eqref{eqn: p_population_risk_function};\\
    \hline
    $\hat{\lambda}$  & The value of $\lambda \delta^2/(\lambda \delta^2 + \widetilde{\lambda}\de^2 )$;\\
    \hline
    $\mu$  & The value of $\frac{\lambda\delta^2 + \widetilde{\lambda}\de^2}{\lambda \rho(\delta) + \widetilde{\lambda}\rho(\de)}$;\\
    \hline
    $\sigma_i$  & The $i$-th largest singular value of $\bfW^*$;\\
    \hline
    $\kappa$  & The value of $\sigma_1/\sigma_{K}$;\\
    \hline
    $\gamma$  & The value of $\prod_{i=1}^K\sigma_i/\sigma_{K}$;\\
    \hline
    $q$  & Some large constant in $\mathbb{R}^+$;\\
    \hline
    \hline
    \end{tabular}
    \vspace{-1mm}
    \label{table:appendix}
\end{table}

\section{Preliminary Lemmas}
We will first start with some preliminary lemmas. As outlined at the beginning of the supplementary material, Lemma \ref{Lemma: p_second_order_derivative bound} illustrates the local convex region of the population risk function, and Lemma \ref{Lemma: p_first_order_derivative} explains the error bound between the population risk and empirical risk functions. Then, Lemma \ref{Thm: p_initialization} describes the returned initial model $\bfW^{(0,0)}$ via tensor initialization method \cite{ZSJB17} purely using labeled data. Next, Lemma \ref{Lemma: weyl} is the well known Weyl's inequality in the matrix setting. Moreover, Lemma \ref{prob} is the concentration theorem for independent random matrices. The definitions of the sub-Gaussian and sub-exponential variables are summarized in Definitions \ref{Def: sub-gaussian} and \ref{Def: sub-exponential}.  Lemmas \ref{Lemma: covering_set} and \ref{Lemma: spectral_norm_on_net} serve as the technical tools in bounding matrix norms under the framework of the confidence interval.
\begin{lemma}\label{Lemma: p_second_order_derivative bound}
   Given any $\bfW\in \mathbb{R}^{d\times K}$, let $p$ satisfy 
      \begin{equation}\label{eqn: p_initial}
        p  \lesssim \frac{\sigma_K}{\mu^{2}K\cdot \| \bfW -\bfW^* \|_F}.
    \end{equation}
    Then, we have 
    \begin{equation}
        \frac{\lambda\rho(\delta)+\w[\lambda]\rho(\de)}{11\kappa^2\gamma K^2} \preceq \nabla^2 {f}(\bfW;p) \preceq \frac{7(\lambda\delta^2+\w[\lambda]\de^2)}{K}.
    \end{equation}
\end{lemma}

\begin{lemma}\label{Lemma: p_first_order_derivative}
   Let $f$ and $\hat{f}$ be the functions defined in \eqref{eqn: p_population_risk_function} and \eqref{eqn: empirical_risk_function}, respectively. Suppose the pseudo label is generated through \eqref{eqn: psedu_label} with weights $\widetilde{\bfW}$. Then, we have
   \begin{equation}
        \begin{split}
        \|\nabla f(\bfW) -\nabla \hat{f}(\bfW) \|_2
       \lesssim&
       \frac{\lambda\delta^2}{K} \sqrt{\frac{d\log q}{N}} \cdot \| \bfW - \bfW^*\| + \frac{\w[\lambda]\de^2}{K} \sqrt{\frac{d\log q}{M}} \cdot \| \bfW -\w[\bfW] \|_2\\
       & + \frac{\big\| \lambda\delta^2 \cdot \big(\w[\bfW]-\Wp\big) +\w[\lambda]\de^2 \cdot \big(\bfW^*-\Wp\big)  \big\|_2}{K}
       \end{split}
    \end{equation}
    with probability at least $1-q^{-d}$.
\end{lemma}
\begin{lemma}[Initialization, \cite{ZSJB17}]
\label{Thm: p_initialization}
     Assuming the number of labeled data satisfies
     \begin{equation}\label{eqn: p_initial_sample}
         N \ge  p^2 N^* 
     \end{equation}
     for some large constant $q$ and $p\in[\frac{1}{K},1]$, the tensor initialization method 
    outputs  ${\bfW}^{(0,0)}$  such that
    \begin{equation}\label{eqn: p_initialization}
        \| \bfW^{(0,0)} -\bfW^* \|_F  \le \frac{\sigma_K}{p\cdot c(\kappa)\mu^{2}K}
    \end{equation}
    with probability at least $1-q^{-d}$.
\end{lemma}
\begin{lemma}[Weyl's inequality, \cite{B97}] \label{Lemma: weyl}
Suppose $\bfB =\bfA + \bfE$ be a matrix with dimension $m\times m$. Let $\lambda_i(\bfB)$ and $\lambda_i(\bfA)$ be the $i$-th largest eigenvalues of $\bfB$ and $\bfA$, respectively.  Then, we have 
\begin{equation}
    |\lambda_i(\bfB) - \lambda_i(\bfA)| \le \|\bfE \|_2, \quad \forall \quad i\in [m].
\end{equation}
\end{lemma}
\begin{lemma}[\cite{T12}, Theorem 1.6]\label{prob}
	Consider a finite sequence $\{{\bfZ}_k\}$ of independent, random matrices with dimensions $d_1\times d_2$. Assume that such random matrix satisfies
	\begin{equation*}
	\vspace{-2mm}
	\mathbb{E}({\bfZ_k})=0\quad \textrm{and} \quad \left\|{\bfZ_k}\right\|\le R \quad \textrm{almost surely}.	
	\end{equation*}
	Define
	\begin{equation*}
	\vspace{-2mm}
	\delta^2:=\max\Big\{\Big\|\sum_{k}\mathbb{E}({\bfZ}_k{\bfZ}_k^*)\Big\|,\Big\|\displaystyle\sum_{k}\mathbb{E}({\bfZ}_k^*{\bfZ}_k)\Big\|\Big\}.
	\end{equation*}
	Then for all $t\ge0$, we have
	\begin{equation*}
	\text{Prob}\Bigg\{ \left\|\displaystyle\sum_{k}{\bfZ}_k\right\|\ge t \Bigg\}\le(d_1+d_2)\exp\Big(\frac{-t^2/2}{\delta^2+Rt/3}\Big).
	\end{equation*}
\end{lemma}
\begin{defi}[Definition 5.7, \cite{V2010}]\label{Def: sub-gaussian}
	A random variable $X$ is called a sub-Gaussian random variable if it satisfies 
	\begin{equation}
		(\mathbb{E}|X|^{p})^{1/p}\le c_1 \sqrt{p}
	\end{equation} for all $p\ge 1$ and some constant $c_1>0$. In addition, we have 
	\begin{equation}
		\mathbb{E}e^{s(X-\mathbb{E}X)}\le e^{c_2\|X \|_{\psi_2}^2 s^2}
	\end{equation} 
	for all $s\in \mathbb{R}$ and some constant $c_2>0$, where $\|X \|_{\phi_2}$ is the sub-Gaussian norm of $X$ defined as $\|X \|_{\psi_2}=\sup_{p\ge 1}p^{-1/2}(\mathbb{E}|X|^{p})^{1/p}$.

	Moreover, a random vector $\bfX\in \mathbb{R}^d$ belongs to the sub-Gaussian distribution if one-dimensional marginal $\boldsymbol{\alpha}^T\bfX$ is sub-Gaussian for any $\boldsymbol{\alpha}\in \mathbb{R}^d$, and the sub-Gaussian norm of $\bfX$ is defined as $\|\bfX \|_{\psi_2}= \sup_{\|\boldsymbol{\alpha} \|_2=1}\|\boldsymbol{\alpha}^T\bfX \|_{\psi_2}$.
\end{defi}

\begin{defi}[Definition 5.13, \cite{V2010}]\label{Def: sub-exponential}
	
	A random variable $X$ is called a sub-exponential random variable if it satisfies 
	\begin{equation}
	(\mathbb{E}|X|^{p})^{1/p}\le c_3 {p}
	\end{equation} for all $p\ge 1$ and some constant $c_3>0$. In addition, we have 
	\begin{equation}
	\mathbb{E}e^{s(X-\mathbb{E}X)}\le e^{c_4\|X \|_{\psi_1}^2 s^2}
	\end{equation} 
	for $s\le 1/\|X \|_{\psi_1}$ and some constant $c_4>0$, where $\|X \|_{\psi_1}$ is the sub-exponential norm of $X$ defined as $\|X \|_{\psi_1}=\sup_{p\ge 1}p^{-1}(\mathbb{E}|X|^{p})^{1/p}$.
\end{defi}
\begin{lemma}[Lemma 5.2, \cite{V2010}]\label{Lemma: covering_set}
	Let $\mathcal{B}(0, 1)\in\{ \boldsymbol{\alpha} \big| \|\boldsymbol{\alpha} \|_2=1, \boldsymbol{\alpha}\in \mathbb{R}^d  \}$ denote a unit ball in $\mathbb{R}^{d}$. Then, a subset $\mathcal{S}_\xi$ is called a $\xi$-net of $\mathcal{B}(0, 1)$ if every point $\bfz\in \mathcal{B}(0, 1)$ can be approximated to within $\xi$ by some point $\boldsymbol{\alpha}\in \mathcal{B}(0, 1)$, i.e., $\|\bfz -\boldsymbol{\alpha} \|_2\le \xi$. Then the minimal cardinality of a   $\xi$-net $\mathcal{S}_\xi$ satisfies
	\begin{equation}
	|\mathcal{S}_{\xi}|\le (1+2/\xi)^d.
	\end{equation}
\end{lemma}
\begin{lemma}[Lemma 5.3, \cite{V2010}]\label{Lemma: spectral_norm_on_net}
	Let $\bfA$ be an $d_1\times d_2$ matrix, and let $\mathcal{S}_{\xi}(d)$ be a $\xi$-net of $\mathcal{B}(0, 1)$ in $\mathbb{R}^d$ for some $\xi\in (0, 1)$. Then
	\begin{equation}
	\|\bfA\|_2 \le (1-\xi)^{-1}\max_{\boldsymbol{\alpha}_1\in \mathcal{S}_{\xi}(d_1), \boldsymbol{\alpha}_2\in \mathcal{S}_{\xi}(d_2)} |\boldsymbol{\alpha}_1^T\bfA\boldsymbol{\alpha}_2|.
	\end{equation} 
\end{lemma}

%% file: appendix/c6p2.tex
\section{Proof of Theorem \ref{Thm: sufficient_number}}
\label{sec: sufficient_number}
With $p=1$ in \eqref{eqn: p_population_risk_function}, the \textit{population risk function} is reduced as 
\begin{equation}\label{eqn: population_risk_function}
    f(\bfW) =  \frac{\lambda}{2} \mathbb{E}_{\bfx} (y - g(\bfW;\bfx)) + \frac{\w[\lambda]}{2} \mathbb{E}_{\w[\bfx]} (\w[y]^* - g(\bfW;\w[\bfx])),
\end{equation}
where $y=g(\bfW^*;\bfx)$ with $\bfx\sim \mathcal{N}(0,\delta^2\bfI)$ and $\w[y]^* = g(\bfW^*;\w[\bfx])$ with $\w[\bfx]\sim \mathcal{N}(0,\de^2\bfI)$. In fact, \eqref{eqn: population_risk_function} can be viewed as the expectation of the \textit{empirical risk function} in \eqref{eqn: empirical_risk_function} given $\w[y]_m = g(\bfW^*;\w[\bfx]_m)$. Moreover, the ground-truth model $\bfW^*$ is the global optimal to \eqref{eqn: population_risk_function} as well. Lemmas \ref{lemma: population_second_order_derivative} and \ref{Lemma: first_order_distance} are the special case of Lemmas \ref{Lemma: p_second_order_derivative bound} and \ref{Lemma: p_first_order_derivative} with $p=1$. The proof of Theorem \ref{Thm: sufficient_number} is followed by the presentation of the two lemmas.

The main idea in proving Theorem \ref{Thm: sufficient_number} is to characterize the gradient descent term by \textit{intermediate value theorem} (IVT) as shown in \eqref{eqn: Thm3_temp1} and \eqref{eqn: Thm3_temp2}. The IVT is not directly applied in the empirical risk function because of its non-smoothness. However, the population risk functions defined in \eqref{eqn: p_population_risk_function} and \eqref{eqn: population_risk_function}, which are the expectations over the Gaussian variables, are smooth. Then, as the distance $\|\nabla f (\bfW) - \nabla f(\bfW^*) \|_F$ is upper bounded by a linear function of $\|\bfW - \bfW^*\|_F$ as shown in \eqref{eqn:eta1}, we can establish the connection between $\|\W[\ell,t+1] -\bfW^* \|_F$  and $\|\W[\ell,t]-\bfW^*\|_F$ as shown in \eqref{eqn: induction_t}. Final, by mathematical induction over $\ell$ and $t$, one can characterize $\|\W[L,0]-\bfW^*\|_F$ by $\| \W[0,0] -\bfW^* \|_F$ as shown in \eqref{eqn:outer_loop}, which completes the whole proof.

\begin{lemma}[Lemma \ref{Lemma: p_second_order_derivative bound} with $p=1$]
\label{lemma: population_second_order_derivative}
    Let $f$ and $\hat{f}$ are the functions defined in \eqref{eqn: population_risk_function} and $\eqref{eqn: empirical_risk_function}$, respectively. Then, for any $\bfW$ that satisfies
    \begin{equation}\label{eqn: initial_point}
        \|\bfW-\bfW^*\|_F \le \frac{\sigma_K}{\mu^2 K},
    \end{equation}
    we have 
    \begin{equation}
        \frac{\lambda\rho(\delta)+\w[\lambda]\rho(\de)}{11\kappa^2\gamma K^2} \preceq \nabla^2 {f}(\bfW) \preceq \frac{7(\lambda\delta^2+\w[\lambda]\de^2)}{K}.
    \end{equation}
\end{lemma}


 \begin{lemma}[Lemma \ref{Lemma: p_first_order_derivative} with $p =1$]
 \label{Lemma: first_order_distance}
   Let $f$ and $\hat{f}$ be the functions defined in \eqref{eqn: population_risk_function} and \eqref{eqn: empirical_risk_function}, respectively. Suppose the pseudo label is generated through \eqref{eqn: psedu_label} with weights $\widetilde{\bfW}$. Then, we have
      \begin{equation}
   \begin{split}
        \|\nabla f(\bfW) -\nabla \hat{f}(\bfW) \|_2
        \lesssim&
       \Big(\frac{\lambda\delta^2}{K} \sqrt{\frac{d\log q}{N}} + \frac{(1- \lambda)\de^2}{K} \sqrt{\frac{d\log q}{M}}\Big)\cdot \| \bfW -\bfW^* \|_2\\
       & + \frac{(1- \lambda)\de^2}{K}\Big( \sqrt{\frac{d\log q}{M}}+\frac{1}{2} \Big)\cdot\|\w[\bfW]-\bfW^* \|_2
   \end{split}
    \end{equation}
    with probability at least $1-q^{-d}$.
\end{lemma}

\begin{proof}[Proof of Theorem \ref{Thm: sufficient_number}]
    From Algorithm \ref{Algorithm: Alg1}, in the $\ell$-th outer loop, we have 
    \begin{equation}\label{eqn: Thm3_temp1}
    \begin{split}
       \W[\ell, t+1] 
       = &  \W[\ell, t] -\eta \nabla\hat{f}_{\mathcal{D}_t, \widetilde{\mathcal{D}}_t}(\W[\ell, t])+\beta (\W[\ell, t] - \W[\ell, t-1])\\
       = & \W[\ell, t] -\eta \nabla{f}(\W[\ell, t])+\beta (\W[\ell, t] - \W[\ell, t-1])\\
       &+ \eta \cdot \big( \nabla f(\W[\ell, t]) - \nabla \hat{f}_{\mathcal{D}_t, \widetilde{\mathcal{D}}_t}(\W[\ell, t]) \big).
    \end{split}
    \end{equation}
    Since $\nabla f$ is a smooth function and $\bfW^*$ is a local (global) optimal to $f$, then we have 
    \begin{equation}\label{eqn: Thm3_temp2}
        \begin{split}
            \nabla{f}(\W[\ell, t])
            = \nabla{f}(\W[\ell, t]) - \nabla{f}(\bfW^*)
            =\nabla^2{f}(\widehat{\bfW}^{(\ell,t)}) (\W[\ell, t]-\bfW^*),
        \end{split}
    \end{equation}
    where the last equality comes from intermediate value theorem and $\widehat{\bfW}^{(\ell,t)}$ lies in the convex hull of $\W[\ell, t]$ and $\bfW^*$.
    
    Then, we have 
    \begin{equation}\label{eqn: matrix_convergence}
        \begin{split}
         \begin{bmatrix}
        \W[\ell, t+1]-\bfW^*\\
        \W[\ell, t]-\bfW^*
        \end{bmatrix}
        =
        &\begin{bmatrix}
                \bfI -\eta\nabla^2 f(\widehat{\bfW}^{(\ell, t)})& \beta \bfI\\
                \bfI & \boldsymbol{0}
        \end{bmatrix}
         \begin{bmatrix}
        \W[\ell, t]-\bfW^*\\
        \W[\ell, t-1]-\bfW^*
        \end{bmatrix}\\
            &+\eta \begin{bmatrix}
            \nabla f(\W[\ell, t]) - \nabla \hat{f}_{\mathcal{D}_t, \widetilde{\mathcal{D}}_t}(\W[\ell, t])\\
            \boldsymbol{0}
        \end{bmatrix}.
        \end{split}
        \end{equation}
    Let $\nabla^2f(\widehat{\bfW}^{(t)})=\bfS\mathbf{\Lambda}\bfS^{T}$ be the eigen-decomposition of $\nabla^2f(\widehat{\bfW}^{(t)})$. Then, we define
\begin{equation}\label{eqn:Heavy_ball_eigen}
\begin{split}
{\bfA}(\beta)
: =&
\begin{bmatrix}
\bfS^T&\bf0\\
\bf0&\bfS^T
\end{bmatrix}
\bfA(\beta)
\begin{bmatrix}
\bfS&\bf0\\
\bf0&\bfS
\end{bmatrix}
=\begin{bmatrix}
\bfI-\eta\bf\Lambda+\beta\bfI &\beta\bfI\\
\bfI& 0
\end{bmatrix}.
\end{split}
\end{equation}
Since 
$\begin{bmatrix}
\bfS&\bf0\\
\bf0&\bfS
\end{bmatrix}\begin{bmatrix}
\bfS^T&\bf0\\
\bf0&\bfS^T
\end{bmatrix}
=\begin{bmatrix}
\bfI&\bf0\\
\bf0&\bfI
\end{bmatrix}$, we know $\bfA(\beta)$ and $\begin{bmatrix}
\bfI-\eta\bf\Lambda+\beta\bfI &\beta\bfI\\
\bfI& 0
\end{bmatrix}$
share the same eigenvalues. 
Let $\gamma^{(\bf\Lambda)}_i$ be the $i$-th eigenvalue of $\nabla^2f(\widehat{\bfw}^{(t)})$, then the corresponding $i$-th eigenvalue of \eqref{eqn:Heavy_ball_eigen}, denoted by $\gamma^{(\bfA)}_i$, satisfies 
\begin{equation}\label{eqn:Heavy_ball_quaratic}
(\gamma^{(\bfA)}_i(\beta))^2-(1-\eta \gamma^{(\bf\Lambda)}_i+\beta)\gamma^{(\bfA)}_i(\beta)+\beta=0.
\end{equation}
By simple calculation, we have 
\begin{equation}\label{eqn:heavy_ball_beta}
\begin{split}
|\gamma^{(\bfA)}_i(\beta)|
=\begin{cases}
\sqrt{\beta}, \qquad \text{if}\quad  \beta\ge \big(1-\sqrt{\eta\gamma^{(\bf\Lambda)}_i}\big)^2,\\
\frac{1}{2} \left| { (1-\eta \gamma^{(\bf\Lambda)}_i+\beta)+\sqrt{(1-\eta \gamma^{(\bf\Lambda)}_i+\beta)^2-4\beta}}\right| , \text{otherwise}.
\end{cases}
\end{split}
\end{equation}
Specifically, we have
\begin{equation}\label{eqn: Heavy_ball_result}
\gamma^{(\bfA)}_i(0)>\gamma^{(\bfA)}_i(\beta),\quad \text{for}\quad  \forall \beta\in\big(0, (1-{\eta \gamma^{(\bf\Lambda)}_i})^2\big),
\end{equation}
and  $\gamma^{(\bfA)}_i$ achieves the minimum $\gamma^{(\bfA)*}_{i}=\Big|1-\sqrt{\eta\gamma^{(\bf\Lambda)}_i}\Big|$ when $\beta= \Big(1-\sqrt{\eta\gamma^{(\bf\Lambda)}_i}\Big)^2$.
    From Lemma \ref{lemma: population_second_order_derivative}, we know that 
    \begin{equation}
        \begin{split}
        \gamma^{(\bf\Lambda)}_{\min} 
        = \frac{\lambda\rho(\delta)+\widetilde{\lambda}\rho(\de)}{11\kappa^2\gamma K^2}, 
        \quad 
        \text{and}
        \quad
        \gamma^{(\bf\Lambda)}_{\max} 
        = \frac{7(\lambda\delta^2+\widetilde{\lambda}\de^2)}{K}.
        \end{split}
    \end{equation}
    Thus, we can select $\eta = \big(\frac{1}{\sqrt{\gamma^{(\bf\Lambda)}_{\max}} + \sqrt{\gamma^{(\bf\Lambda)}_{\min} }}\big)^2$, and $\|\bfA(\beta)\|_2$ can be bounded by
    \begin{equation}\label{eqn: beta}
        \begin{split}
        \min_{\beta}\|\bfA(\beta)\|_2 
        \le& 1-\sqrt{\big(\frac{\lambda\rho(\delta)+\widetilde{\lambda}\rho(\de)}{11\kappa^2\gamma K^2}\big)/\big(2\cdot \frac{7(\lambda\delta^2+\widetilde{\lambda}\de^2)}{K} \big)} \\
        = & 1- \frac{\mu(\delta,\de)}{\sqrt{154\kappa^2\gamma K}},
        \end{split}
    \end{equation}
    where $\mu(\delta,\de) = \Big( \frac{\lambda\rho(\delta)+\widetilde{\lambda}\rho(\de)}{\lambda\delta^2+\widetilde{\lambda}\de^2} \Big)^{1/2}$.
    
    From Lemma \ref{Lemma: first_order_distance}, we have 
    \begin{equation}\label{eqn:eta1}
        \begin{split}
            \|\nabla f(\W[\ell, t]) -\nabla \hat{f}(\W[\ell, t]) \|_2
            =&\Big(\frac{\lambda\delta^2}{K} \sqrt{\frac{d\log q}{N_t}} + \frac{\w[\lambda]\de^2}{K} \sqrt{\frac{d\log q}{M_t}}\Big)\cdot \| \W[\ell, t] -\bfW^* \|_2\\
            & + \frac{\w[\lambda]\de^2}{K}\Big( \sqrt{\frac{d\log q}{M_t}}+\frac{1}{2} \Big)\cdot\|\W[\ell, 0]-\bfW^* \|_2.
        \end{split}
    \end{equation}
    Given $\varepsilon>0$ and $\tilde{\varepsilon}>0$ with $\varepsilon+\tilde{\varepsilon}<1$, let
    \begin{equation}\label{eqn:eta2}
        \begin{split}
            &\eta \cdot \frac{\lambda\delta^2}{K} \sqrt{\frac{d\log q}{N_t}}
            \le \frac{\varepsilon\mu(\delta,\de)}{\sqrt{154\kappa^2\gamma K}} \text{,}\\
            \text{and}& \quad
            \eta \cdot \frac{\widetilde{\lambda}\de^2}{K} \sqrt{\frac{d\log q}{M_t}}
            \le \frac{\tilde{\varepsilon}\mu(\delta,\de)}{\sqrt{154\kappa^2\gamma K}},
        \end{split}
    \end{equation}
    where we need 
    \begin{equation}\label{eqn:proof_sample1}
    \begin{split}
        &N_t \ge \varepsilon^{-2}\mu^{-2}\big(\frac{\lambda\delta^2}{\lambda\delta^2+\w[\lambda]\de^2}\big)^2\kappa^2\gamma K^3 d\log q,\\
        \text{and}&\quad 
        M_t \ge 
        \tilde{\varepsilon}^{-2}\mu^{-2}\big(\frac{\w[\lambda]\de^2}{\lambda\delta^2+\w[\lambda]\de^2}\big)^2\kappa^2\gamma K^3 d\log q.
        \end{split}
    \end{equation}
    Therefore, from \eqref{eqn: beta}, \eqref{eqn:eta1} and \eqref{eqn:eta2}, we have
    \begin{equation}\label{eqn: induction_t}
        \begin{split}
            &\| \W[\ell, t+1]-\bfW^* \|_2\\
            \le& \Big( 1- \frac{(1-\varepsilon-\tilde{\varepsilon})\mu(\delta,\de)}{\sqrt{154\kappa^2\gamma K}} \Big)\| \W[\ell, t]-\bfW^* \|_2\\
            &+\eta\cdot\frac{\w[\lambda]\de^2}{K}\Big( \sqrt{\frac{d\log q}{M_t}}+\frac{1}{2} \Big) \cdot\|\W[\ell, 0]-\bfW^* \|_2\\
            \le & \Big( 1- \frac{(1-\varepsilon-\tilde{\varepsilon})\mu(\delta,\de)}{\sqrt{154\kappa^2\gamma K}} \Big)\| \W[\ell, t]-\bfW^* \|_2
            +\eta \cdot \frac{\w[\lambda]\de^2}{K}\|\W[\ell, 0] -\bfW^*\|_2
        \end{split}
    \end{equation}
    when $M\ge 4d\log q$.
    By mathematical induction on \eqref{eqn: induction_t} over $t$, we have 
    \begin{equation}\label{eqn:inner_loop}
    \begin{split}
            &\| \W[\ell, t]-\bfW^* \|_2\\
            \le&\Big( 1- \frac{(1-\varepsilon-\tilde{\varepsilon})\mu}{\sqrt{154\kappa^2\gamma K}} \Big)^t\cdot\|\W[\ell, 0]-\bfW^* \|_2
            \\
            &+
            \frac{\sqrt{154\kappa^2\gamma K}}{(1-\varepsilon-\tilde{\varepsilon})\mu}\cdot \frac{\sqrt{K}}{14 (\lambda\delta^2+\w[\lambda]\de^2 )}\cdot\frac{\w[\lambda]\de^2}{K} 
            \|\W[\ell, 0]-\bfW^* \|_2\\
            \le& \bigg[\Big( 1- \frac{(1-\varepsilon-\tilde{\varepsilon})\mu}{\sqrt{154\kappa^2\gamma K}} \Big)^t+
            \frac{\sqrt{\kappa^2\gamma}\w[\lambda]\de^2}{(1-\varepsilon-\tilde{\varepsilon})\mu(\lambda\delta^2+\w[\lambda]\de^2)}\bigg]\cdot\|\W[\ell, 0]-\bfW^* \|_2
        \end{split}
    \end{equation}
    By mathematical induction on \eqref{eqn:inner_loop} over $\ell$, we have 
    \begin{equation}\label{eqn:outer_loop}
    \begin{split}
            &\| \W[\ell, T]-\bfW^* \|_2\\
            \le& \bigg[\Big( 1- \frac{(1-\varepsilon-\tilde{\varepsilon})\mu}{\sqrt{154\kappa^2\gamma K}} \Big)^T
            +
            \frac{\sqrt{\kappa^2\gamma}\w[\lambda]\de^2}{(1-\varepsilon-\tilde{\varepsilon})\mu(\lambda\delta^2+\w[\lambda]\de^2)}\bigg]^{\ell}\cdot\|\W[0, 0]-\bfW^* \|_2
        \end{split}
    \end{equation}
\end{proof}

\section{Proof of Theorem \ref{Thm: p_convergence}}
\label{sec: insuff}
Instead of proving Theorem \ref{Thm: p_convergence}, we turn to prove a stronger version, as shown in Theorem \ref{Thm: p_main_thm}. One can verify that Theorem \ref{Thm: p_convergence} is a special case of Theorem \ref{Thm: p_main_thm} by  selecting $p = \hat{\lambda}$.

The major idea in proving Theorem \ref{Thm: p_main_thm} is similar to that of Theorem \ref{Thm: sufficient_number}.
The first step is
to characterize the gradient descent term on the population risk function by \textit{intermediate value theorem} as shown in \eqref{eqn: Thm1_temp1} and \eqref{eqn: Thm1_temp2}.  Then,  the connection between $\|\W[\ell+1,0] -\Wp \|_F$  and $\|\W[\ell,0]-\Wp\|_F$ are characterized in \eqref{eqn:p_outer_loop}. Compared with proving Theorem \ref{Thm: sufficient_number}, where the induction over $\ell$ holds naturally with large size of labeled data, the induction over $\ell$ requires a proper value of $p$ as shown in \eqref{eqn_temp: 2.2}. By induction over $\ell$ on \eqref{eqn:p_outer_loop}, the relative error $\| \W[L,0] -\Wp \|_F$ can be characterized by $\|\W[0,0] -\Wp\|_F$ as shown in \eqref{eqn:p_outer_loop_2}.
\begin{theorem}\label{Thm: p_main_thm}
    Suppose the initialization $\W[0,0]$  satisfies 
    with 
    \begin{equation}\label{eqn: p_0}
       {|p -\hat{\lambda}| \le\frac{2(1-\w[\varepsilon])p-1}{\mu\sqrt{K}}}
    \end{equation}
    for some constant $\w[\varepsilon]\in (0, 1/2)$, where
    \begin{equation}
        \hat{\lambda} : =\frac{\lambda\delta^2}{\lambda\delta^2+\w[\lambda]\de^2}
        = \big(\frac{N}{\kappa^2\gamma K^3\mu^2 d \log q}\big)^{\frac{1}{2}} 
    \end{equation}
    and 
    \begin{equation}
        \mu
        = \mu( \delta, \de) 
        = \frac{\lambda\delta^2+\w[\lambda]\de^2}{\lambda\rho(\delta)+\w[\lambda]\rho(\de)}.
    \end{equation}
    Then, if the number of samples in $\w[\mathcal{D}]$ further satisfies
    \begin{equation}
        \begin{split}
        M \gtrsim \tilde{\varepsilon}^{-2} \kappa^2\gamma \mu^2 \big( 1-  \hat{\lambda}\big)^2 K^3 d\log q,
        \end{split}
    \end{equation}
    the iterates $\{ \bfW^{(\ell,t)} \}_{\ell,t=0}^{L,T}$ converge to $\bfW^{[p]}$ with $p$ satisfies \eqref{eqn: p_0} as 
    \begin{equation}\label{eqn: ap_convergence}
        \begin{split}
    &\lim_{T\to \infty}\| \bfW^{(\ell, T)} -\Wp \|_2\\
    \le& \frac{1}{1-\tilde{\varepsilon}} \cdot \Big( {1-p^*} +\mu\sqrt{K}\big|(\hat{\lambda}-p^*)\big| \Big)\cdot\|\W[0,0]-\bfW^* \|_2\\
    &+\frac{\tilde{\varepsilon} }{(1-\tilde{\varepsilon})}\cdot\|\W[\ell,0]-\Wp \|_2,
    \end{split}
    \end{equation}
    with probability at least $1-q^{-d}$.
\end{theorem}

\begin{proof}[Proof of Theorem \ref{Thm: p_main_thm}]
     From Algorithm \ref{Algorithm: Alg1}, in the $\ell$-th outer loop, we have 
    \begin{equation}\label{eqn: Thm1_temp1}
    \begin{split}
       \W[\ell, t+1] 
       =&  \W[\ell, t] -\eta \nabla\hat{f}_{\mathcal{D}_t,\widetilde{\mathcal{D}}_t}
       (\W[\ell, t])+\beta (\W[\ell, t] - \W[\ell, t-1])\\
       = & \W[\ell, t] -\eta \nabla{f}(\W[\ell, t])+\beta (\W[\ell, t] - \W[\ell, t-1])\\
       &+ \eta \cdot \Big( \nabla f(\W[\ell, t]) - \nabla \hat{f}_{\mathcal{D}_t,\widetilde{\mathcal{D}}_t}(\W[\ell, t]) \Big)
    \end{split}
    \end{equation}
    Since $\nabla f$ is a smooth function and $\Wp$ is a local (global) optimal to $f$, then we have 
    \begin{equation}\label{eqn: Thm1_temp2}
        \begin{split}
            \nabla{f}(\W[\ell, t])
            = \nabla{f}(\W[\ell, t]) - \nabla{f}(\Wp)
            =\nabla^2{f}(\widehat{\bfW}^{(\ell,t)}) (\W[\ell, t]-\Wp),
        \end{split}
    \end{equation}
    where the last equality comes from intermediate value theorem and $\widehat{\bfW}^{(\ell,t)}$ lies in the convex hull of $\W[\ell, t]$ and $\Wp$.
    
    Then, we have 
    \begin{equation}
        \| \bfW^{(\ell, t+1)} -\Wp \|_2 \le \|\bfA(\beta) \|_2 \cdot \|\bfW^{(\ell, t)}-\Wp \|_2  + \eta \cdot \|\nabla f(\W[\ell, t]) - \nabla \hat{f}_{\mathcal{D}_t,\widetilde{\mathcal{D}}_t}(\W[\ell, t]) \|_2.
    \end{equation}
    From Lemma \ref{Lemma: p_first_order_derivative}, we have 
    \begin{equation}
        \begin{split}
                &\|\nabla f(\W[\ell, t]) - \nabla \hat{f}(\W[\ell, t]) \|_2\\
            \lesssim& \frac{\lambda\delta^2}{K} \sqrt{\frac{d\log q}{N_t}} \cdot \| \W[\ell, t] - \bfW^*\| + \frac{\w[\lambda]\de^2}{K} \sqrt{\frac{d\log q}{M_t}} \cdot \| \W[\ell, t] -\W[\ell,0] \|_2\\
       & + \frac{\big| \lambda\delta^2 \cdot (\bfW^{(0,0)} -\Wp) -\w[\lambda]\de^2\cdot (\bfW^* - \Wp)  \big|}{K} \\
        \end{split}
    \end{equation}
    When $\ell = 0$, following the similar steps from \eqref{eqn:Heavy_ball_quaratic} to \eqref{eqn: beta},  we have 
    \begin{equation}
    \begin{split}
    &\|\nabla f(\W[\ell, t]) - \nabla \hat{f}(\W[\ell, t]) \|_2\\
        \lesssim & \frac{\lambda\delta^2}{K} \sqrt{\frac{d\log q}{N_t}} \cdot \| \W[\ell, t] - \Wp\| + \frac{\w[\lambda]\de^2}{K} \sqrt{\frac{d\log q}{M_t}} \cdot \| \W[\ell, t] -\Wp \|_2\\
      & + \frac{\lambda\delta^2}{K} \sqrt{\frac{d\log q}{N_t}} \cdot \| \bfW^* - \Wp\| + \frac{\w[\lambda]\de^2}{K} \sqrt{\frac{d\log q}{M_t}} \cdot \| \W[0,0] -\Wp \|_2\\
       & + \frac{\big| \lambda\delta^2 \cdot (1-p) -\w[\lambda]\de^2\cdot p  \big|}{K}\cdot\|\W[0,0]-\bfW^* \|_2
        \end{split}
    \end{equation}
    and
    \begin{equation}
    \begin{split}
        &\| \bfW^{(\ell, t+1)} -\Wp \|_2 \\
    \le& \Big( 1- \frac{1-\tilde{\varepsilon}}{\mu(\delta,\de)\sqrt{154\kappa^2\gamma K}} \Big) \cdot \|\W[\ell, t] -\Wp \|_2\\
    &+ \eta \cdot \Big( \frac{ \lambda \delta^2 (1-p) }{K} \sqrt{ \frac{d\log q}{N_t} }  + \frac{\big| \lambda \delta^2 \cdot(1-p)  -\w[\lambda]\de^2\cdot p  \big|}{K} \Big)\cdot\|\W[0,0]-\bfW^* \|_2\\
    &+\eta \cdot
    \frac{\tilde{\varepsilon} \w[\lambda]\de^2 \cdot p}{K}\cdot\sqrt{\frac{d\log q}{M_t}} \|\W[0,0]-\bfW^* \|_2.
    \end{split}
    \end{equation}
Therefore, we have 
    \begin{equation}\label{eqn:p_outer_loop}
        \begin{split}
            &\lim_{T\to \infty}\| \bfW^{(\ell, T)} -\Wp \|_2\\
            \le&  \frac{\mu\sqrt{154\kappa^2\gamma K}}{1-\tilde{\varepsilon}} \cdot \eta \cdot
            \Big[ \Big(\frac{ \lambda \delta^2 (1-p) }{K} \sqrt{ \frac{d\log q}{N_t} }  + \frac{\big| \lambda\delta^2 \cdot(1-p)  -\w[\lambda]\de^2\cdot p  \big|}{K}\Big) \\
            &\qquad\cdot\|\W[0,0]-\bfW^* \|_2
            +
    \frac{\tilde{\varepsilon} \w[\lambda]\de^2 \cdot p}{K}\cdot\sqrt{\frac{d\log q}{M_t}} \cdot\|\W[0,0]-\bfW^* \|_2\Big]\\
    \le& 
    \frac{\mu\sqrt{154\kappa^2\gamma K}}{1-\tilde{\varepsilon}}\cdot \frac{{K}}{14 (\lambda\delta^2+\w[\lambda]\de^2 )}\\
    &\cdot
    \Big[ \Big(\frac{ \lambda \delta^2 (1-p) }{K} \sqrt{ \frac{d\log q}{N_t} }  + \frac{\big| \lambda\delta^2 \cdot(1-p)  -\w[\lambda]\de^2\cdot p  \big|}{K}\Big)\cdot\|\W[0,0]-\bfW^* \|_2\\
    &  + 
    \frac{\tilde{\varepsilon} \w[\lambda]\de^2 \cdot p}{K}\cdot\sqrt{\frac{d\log q}{M_t}} \cdot\|\W[0,0]-\bfW^* \|_2\Big]\\
    \simeq & {\frac{1}{1-\tilde{\varepsilon}} \cdot \Big(  {1-p} +\sqrt{K}\cdot\big|(1-p)\mu\hat{\lambda} - p\mu(1-\hat{\lambda})\big| \Big)\cdot\|\W[0,0]-\bfW^* \|_2}\\
    &{+\frac{\tilde{\varepsilon} p}{(1-\tilde{\varepsilon})}\cdot\|\W[0,0]-\bfW^* \|_2}\\
    =&
    {\frac{1}{1-\tilde{\varepsilon}} \cdot \Big( {1-p} +\mu\sqrt{K}\big|\widehat{\lambda}-p\big| \Big)\cdot\|\W[0,0]-\bfW^* \|_2 
    +\frac{\tilde{\varepsilon} p}{(1-\tilde{\varepsilon})}\cdot\|\W[0,0]-\bfW^* \|_2},
    \end{split}
    \end{equation}
    where $\hat{\lambda} = \frac{\lambda\delta^2}{\lambda\delta^2 + \w[\lambda]\de^2}$.
    
    To guarantee the convergence in the outer loop, we require 
    \begin{equation}\label{eqn_temp: 2.1}
    \begin{split}
    &\lim_{T\to \infty}\|\W[\ell, T] -\Wp \|_2  \le \|\W[0,0]-\Wp \|_2 = p\|\W[0,0]-\bfW^* \|_2,\\
    \quad \text{and}\quad    &\lim_{T\to \infty}\| \bfW^{(\ell, T)} -\bfW^* \|_2 \le \|\W[0,0]-\bfW^* \|_2.
        \end{split}
    \end{equation}
    Since we have 
    \begin{equation}
        \begin{split}
        \| \bfW^{(\ell, T)} -\Wp \|_2  
        \le& \| \bfW^{(\ell, T)} - \bfW^* \|_2+ \| \bfW^*-\Wp \|_2 \\
        =& \| \bfW^{(\ell, T)} - \bfW^* \|_2 + (1-p)\cdot \| \bfW^*-\W[0,0] \|_2,
        \end{split}
    \end{equation}
    it is clear that \eqref{eqn_temp: 2.1} holds if and only if
    {\begin{equation}
        \frac{1}{1-\tilde{\varepsilon}} \cdot \Big( {1-p} + {\w[\varepsilon]p} +\mu\sqrt{K}\big|\hat{\lambda}-p\big| \Big) + 1-p \le 1.
    \end{equation}}
    To guarantee the iterates strictly converges to the desired point, we let 
    {\begin{equation}
        \frac{1}{1-\tilde{\varepsilon}} \cdot \Big( {1-p} + {\w[\varepsilon]p} +\mu\sqrt{K}\big|\hat{\lambda}-p\big| \Big) + 1-p \le 1 - \frac{1}{C}
    \end{equation}}
    for some larger constant $C$,
    which is equivalent to 
    {\begin{equation}\label{eqn_temp: 2.2}
        |p -\hat{\lambda}| \le\frac{2(1-\w[\varepsilon])p-1}{\mu\sqrt{K}}.
    \end{equation}}
    {To make the bound in \eqref{eqn_temp: 2.2}  meaningful, we need 
    \begin{equation}
        p\ge\frac{1}{2(1-\w[\varepsilon])}.
    \end{equation}}
    
    {When $\ell >1$, following similar steps in \eqref{eqn:p_outer_loop}, we have \begin{equation}\label{eqn:p_outer_loop_2}
        \begin{split}
            &\lim_{T\to \infty}\| \bfW^{(\ell, T)} -\Wp \|_2\\
    \le& \frac{1}{1-\tilde{\varepsilon}} \cdot \Big( {1-p} +\mu\sqrt{K}\big|(\hat{\lambda}-p)\big| \Big)\cdot\|\W[0,0]-\bfW^* \|_2 
    +\frac{\tilde{\varepsilon} p}{1-\tilde{\varepsilon}}\cdot\|\W[\ell,0]-\Wp \|_2,
    \end{split}
    \end{equation}
    Given \eqref{eqn_temp: 2.2} holds, from \eqref{eqn:p_outer_loop_2}, we have
    \begin{equation}
    \begin{split}
        &\lim_{L\to \infty, T\to \infty}\| \bfW^{(L, T)} -\Wp \|_2\\
        \le &  \frac{1}{1-\widetilde{\varepsilon}} \cdot\Big( {1-p} +\mu\sqrt{K}\big|\hat{\lambda}-p\big| \Big)\cdot\|\W[0,0]-\bfW^* \|_2\\
        \le& \frac{1}{1-\widetilde{\varepsilon}} \cdot \Big( {1-p} +\mu\sqrt{K}\big|\hat{\lambda}-p\big| \Big)\cdot\|\W[0,0]-\bfW^* \|_2.
    \end{split}
    \end{equation}}
\end{proof}

\section{Definition and Relative Proofs of $\rho$}
\label{sec: rho}
In this section, the formal definition of $\rho$ is included in Definition \ref{defi: rho}, and a corresponding claim about $\rho$ is summarized in Lemma \ref{lemma: rho_order}.  One can quickly check that ReLU activation function satisfies the conditions in Lemma \ref{lemma: rho_order}.

The major idea in proving Lemma \ref{lemma: rho_order} is to show $H_r(\delta)$ and $J_r(\delta)$ in Definition \ref{defi: rho} are in the order of $\delta^r$ when $\delta$ is small. 
\begin{defi}\label{defi: rho}
    Let $H_r(\delta) = \mathbb{E}_{z\sim \mathcal{N}(0,\delta^2)}\big( \phi^{\prime}(\sigma_Kz)z^r \big)$ and  $J_r(\delta) = \mathbb{E}_{z\sim \mathcal{N}(0,\delta^2)}\big( \phi^{\prime 2}(\sigma_Kz)z^r \big)$. Then, $\rho=\rho(\delta)$ is defined as 
    \begin{equation}\label{eqn: rho}
        \rho(\delta)=\min\Big\{ J_0(\delta)-H_0^2( \delta) -H_1^2(\delta), J_2(\delta)-H_1^2( \delta) -H_2^2(\delta), H_0(\delta)\cdot H_2( \delta) -H_1^2(\delta)  \Big\},
    \end{equation}    
    where $\sigma_K$ is the minimal singular value of $\bfW^*$.
\end{defi}
\begin{lemma}[Order analysis of $\rho$]
\label{lemma: rho_order}
    If $\rho(\delta)>0$ for $\delta\in(0, \xi)$ for some positive constant $\xi$ and the sub-gradient of $\rho(\delta)$ at $0$ can be non-zero,  then $\rho(\delta)=\Theta(\delta^2)$ when $\delta\to 0^+$.
\end{lemma}